\pdfoutput=1
\documentclass[11pt,letterpaper]{article}
\usepackage[margin=1in]{geometry}
\PassOptionsToPackage{dvipsnames}{xcolor}
\usepackage{amsmath,amsthm,amssymb,epsfig, psfrag, titlesec}
\usepackage[framemethod=tikz]{mdframed}
\usepackage{xspace}
\usepackage{mathtools}
\usepackage[colorlinks=true,linkcolor=blue!70!black,
citecolor=blue!70!black,urlcolor=blue!70!black]{hyperref}
\usepackage{algorithm}
\usepackage[noend]{algpseudocode}
\usepackage{stackrel}
\usepackage[square,numbers]{natbib}

\usepackage{enumitem}
\usepackage{breakcites}
\usepackage{transparent}
\usepackage{setspace}
\usepackage{mathrsfs}
\usepackage{nicefrac}
\usepackage{multicol}
\usepackage{xargs}
\newcommand{\loose}{\looseness=-1}

\usepackage{inconsolata}

\usepackage{etoolbox}

\usepackage{needspace}
\makeatletter
\providecommand{\figpolicy}{t}
\renewcommand{\fps@figure}{\figpolicy}
\makeatother

\usepackage{parskip}

\titlespacing*{\section}{0pt}{40pt}{40pt}

\let\oldparagraph\paragraph
\renewcommand{\paragraph}[1]{\oldparagraph{#1.}}

\usepackage{thmtools}
\declaretheorem[qed=$\triangleleft$,name=Example,style=definition,
parent=section]{example}

\makeatletter
\newcommand{\neutralize}[1]{\expandafter\let\csname c@#1\endcsname\count@}
\makeatother

\newenvironment{examplecont}[2][]
  {\renewcommand{\theexample}{\ref*{#2}}%
   \neutralize{example}\phantomsection
   \ifstrempty{#1}{\begin{example}[cont'd]}{\begin{example}[#1, cont'd]}}
  {\end{example}}

\makeatletter
\RequirePackage{xkeyval}
\RequirePackage{tikz}
\RequirePackage{amssymb}

\define@key{boxedtheorem}{titlecolor}{\def\titlecolor{#1}}
\define@key{boxedtheorem}{titlebackground}{\def\titlebackground{#1}}
\define@key{boxedtheorem}{background}{\def\background{#1}}
\define@key{boxedtheorem}{titleboxcolor}{\def\titleboxcolor{#1}}
\define@key{boxedtheorem}{boxcolor}{\def\boxcolor{#1}}
\define@key{boxedtheorem}{thcounter}{\def\thcounter{#1}}
\define@key{boxedtheorem}{size}{\def\size{#1}}
\presetkeys{boxedtheorem}{titlecolor = black, titlebackground = white, background = white,%
                         titleboxcolor = black, boxcolor = black, thcounter=, size = .9\textwidth}{}

\newcommand{\couleurs}[1][]{%
    \setkeys{boxedtheorem}{#1}
    \tikzstyle{fancytitle} =[draw=\titleboxcolor, rounded corners, fill=\titlebackground,
                            text= \titlecolor]
    \tikzstyle{mybox} = [draw=\boxcolor, fill=\background, very thick,
                        rectangle, rounded corners, inner sep=10pt, inner ysep=20pt]
}

\newsavebox{\boiboite}
\newcommand{\titre}{Titre}
\newenvironment{boite}[2][]%
    {%
    \renewcommand{\titre}{#2}
    \couleurs[#1]
    \begin{lrbox}{\boiboite}%
     \begin{minipage}[!h]{\size}
    }%
    {%
     \end{minipage}
    \end{lrbox}
    \begin{center}
    \begin{tikzpicture}
    \node [mybox] (box){\usebox{\boiboite}};
    \node[fancytitle, right=10pt] at (box.north west) {\titre};
    \end{tikzpicture}
    \end{center}
    }

\newcommand{\newboxedtheorem}[4][]{%
    \couleurs[#1]
    \@ifnotempty{#4}{%
      \@ifundefined{the#4}{\@ifundefined{\thcounter}{\newcounter{#4}}{%
      \newcounter{#4}[\thcounter ] } } { }%
    }
    \newenvironment{#2}[1][]{%
    \@ifnotempty{#4}{\refstepcounter{#4}}
    \begin{boite}[#1]{\textbf{#3\@ifnotempty{#4}{ \csname the#4\endcsname}}\@ifnotempty{##1}{
    (##1)}\textbf{.}}
    }%
    {%
    \end{boite}
    }
}
\makeatother
\usepackage{colortbl}
\newmdenv[backgroundcolor=LightBlue,linewidth=5pt, linecolor=orange!25, topline=false, bottomline=false, rightline=false, skipabove=12pt plus 3pt minus 5pt, skipbelow=9pt plus 3pt minus 5pt, suppressfirstparskip=false, nobreak=true]{whiteblueframe}
\newmdenv[linecolor=orange,backgroundcolor=LightBlue, skipabove=12pt plus 3pt minus 5pt, skipbelow=9pt plus 3pt minus 5pt, suppressfirstparskip=false, nobreak=true]{orangeblueframe}

\newcounter{algbox}[section]
\renewcommand{\thealgbox}{\thesection.\arabic{algbox}}
\newmdenv[backgroundcolor=LightBlue,linewidth=5pt,linecolor=orange!25,topline=false,bottomline=false,rightline=false,
  skipabove=12pt plus 3pt minus 5pt, skipbelow=9pt plus 3pt minus 5pt, suppressfirstparskip=false, nobreak=true,
  frametitlerule=false, frametitlebackgroundcolor=orange!12, frametitlefont=\sffamily\bfseries,
  frametitleaboveskip=4pt, frametitlebelowskip=4pt, innertopmargin=6pt]{algbox}
\newcommand{\algtitle}[2]{\refstepcounter{algbox}\label{#1}\textsf{\textbf{Algorithm \thealgbox}}\quad\textsf{\textbf{#2}}}
\newcommand{\algprotocol}[1]{\textsf{\textbf{#1}}}
\newcommand{\keepwithbox}[1]{\Needspace*{#1\baselineskip}}
\algrenewcommand{\algorithmicrequire}{\textbf{input:}}
\algrenewcommand{\algorithmicensure}{\textbf{output:}}
\definecolor{LightGrey}{rgb}{0.7, 0.7, 0.7}
\definecolor{LightBlue}{rgb}{0.95, 0.95, 0.99} 

\newcommand{\sectionbreak}{\clearpage}
\titleformat{\section}
  {\normalfont\sffamily\large\bfseries\centering\uppercase}
  {\thesection.}{.5em}{}
\titleformat{name=\section,numberless}
  {\normalfont\sffamily\large\bfseries\centering}
  {}{0pt}{\MakeUppercase}
\titleformat{\subsection}
  {\normalfont\sffamily\bfseries}
  {\thesubsection}{.5em}{}

\newtheoremstyle{TH1}
  {\topsep}%
  {\topsep}%
  {\normalfont}%
  {}%
  {\bfseries}%
  {:}%
  {.5em}%
  {\thmname{#1} \thmnumber{#2}\thmnote{~(#3)}}%
\theoremstyle{TH1}

\usetikzlibrary{calc}
\mdfdefinestyle{boxedstylemd}{
  skipabove=\baselineskip,
  skipbelow=\baselineskip,
  hidealllines=true,
  innertopmargin=4pt,
  linewidth=4pt,
  linecolor=gray!40,
  singleextra={
    \draw[line width=3pt,gray!50,line cap=rect] (O|-P) -- +(1cm,0pt);
    \draw[line width=3pt,gray!50,line cap=rect] let \p1=(O|-P), \p2=(O) in (O|-P) -- +(0pt,{-min(1cm,\y1-\y2)});
    \draw[line width=3pt,gray!50,line cap=rect] (O-|P) -- +(-1cm,0pt);
    \draw[line width=3pt,gray!50,line cap=rect] let \p1=(O|-P), \p2=(O) in (O-|P) -- +(0pt,{min(1cm,\y1-\y2)});
    },
  firstextra={
    \draw[line width=3pt,gray!50,line cap=rect] (O|-P) -- +(1cm,0pt);
    \draw[line width=3pt,gray!50,line cap=rect] let \p1=(O|-P), \p2=(O) in (O|-P) -- +(0pt,{-min(1cm,\y1-\y2)});
  },
  secondextra={
    \draw[line width=3pt,gray!50,line cap=rect] (O-|P) -- +(-1cm,0pt);
    \draw[line width=3pt,gray!50,line cap=rect] let \p1=(O|-P), \p2=(O) in (O-|P) -- +(0pt,{min(1cm,\y1-\y2)});
  }
}
\mdfdefinestyle{exerciseStylemd}{font=\small}

\newmdtheoremenv[ 
    style = boxedstylemd,
  nobreak=true,
]{theorem}{Theorem}[section]

\newmdtheoremenv[
  style = boxedstylemd,
  nobreak=true,
]{lem}{Lemma}[section]

\newmdtheoremenv[
  style = boxedstylemd,
  nobreak=true,
]{assumption}{Assumption}[section]

\newmdtheoremenv[
  style = boxedstylemd,
  nobreak=true,
]{prop}{Proposition}[section]

\newmdtheoremenv[
  style = boxedstylemd,
  nobreak=true,
]{cor}{Corollary}[section]

\newmdtheoremenv[
  style = boxedstylemd,
  nobreak=true,
]{definition}{Definition}[section]

\newmdtheoremenv[
    style = boxedstylemd,
]{rem}{Remark}[section]

\newmdtheoremenv[ style=exerciseStylemd,
  linewidth=3pt, linecolor=RoyalPurple!15, topline=false, bottomline=false, rightline=false, skipabove=5pt, skipbelow=5pt
]{exe}{Exercise}[section]

\newmdtheoremenv[
skipabove=\baselineskip,
skipbelow=\baselineskip,
hidealllines=true,
innertopmargin=4pt,
linewidth=4pt,
linecolor=gray!40,
singleextra={
  \draw[line width=3pt,gray!50,line cap=rect] (O|-P) -- +(1cm,0pt);
  \draw[line width=3pt,gray!50,line cap=rect] (O|-P) -- +(0pt,-1cm);
  \draw[line width=3pt,gray!50,line cap=rect] (O-|P) -- +(-1cm,0pt);
  \draw[line width=3pt,gray!50,line cap=rect] (O-|P) -- +(0pt,1cm);
  },
firstextra={
  \draw[line width=3pt,gray!50,line cap=rect] (O|-P) -- +(1cm,0pt);
  \draw[line width=3pt,gray!50,line cap=rect] (O|-P) -- +(0pt,-1cm);
},
secondextra={
  \draw[line width=3pt,gray!50,line cap=rect] (O-|P) -- +(-1cm,0pt);
  \draw[line width=3pt,gray!50,line cap=rect] (O-|P) -- +(0pt,1cm);
}
]{greybox}{}

\newlength{\minipagewidth}
\DeclareFontFamily{U}{mathx}{\hyphenchar\font45}
\DeclareFontShape{U}{mathx}{m}{n}{
      <5> <6> <7> <8> <9> <10>
      <10.95> <12> <14.4> <17.28> <20.74> <24.88>
      mathx10
      }{}
\DeclareSymbolFont{mathx}{U}{mathx}{m}{n}
\DeclareFontSubstitution{U}{mathx}{m}{n}
\DeclareMathAccent{\widecheck}{0}{mathx}{"71}
\DeclareMathAccent{\wideparen}{0}{mathx}{"75}

\newcommand{\Dgenshort}{D}
\newcommand{\Dgenpi}[3][\pi]{D^{#1}\prn*{#2\dmid{}#3}}

\newcommand{\Dsbrpi}[3][\pi]{D^{#1}_{\mathsf{sbr}}\prn*{#2\dmid#3}}
\newcommand{\Dsbrshort}{D_{\mathsf{sbr}}}

\newcommand{\Dflippi}[3][\pi]{\Dflipshort^{#1}\prn*{#2\dmid{}#3}}
\newcommand{\Dflipshort}{\widecheck{D}}

\newcommand{\compgenrand}[1][D]{\wb{\mathsf{dec}}_{\gamma}^{#1}}
\newcommand{\compHrand}[1][\gamma]{\wb{\mathsf{dec}}_{#1}^{\mathsf{H}}}
\newcommand{\compgenrandbasic}[1][D]{\wb{\mathsf{dec}}^{#1}}

\newcommand{\odec}{\normalfont{\textsf{o-dec}}}

\newcommand{\ocompsbr}[1][\gamma]{\odec^{\Dsbrshort}_{#1}}
\newcommand{\ocompD}[1][\gamma]{\odec^{D}_{#1}}
\newcommand{\ocompH}[1][\gamma]{\odec^{\mathsf{H}}_{#1}}

\newcommand{\EstOpt}{\mathrm{\normalfont{\textbf{OptEst}}}}
\newcommand{\EstOptD}[1][\gamma]{\EstOpt_{#1}^{D}}

\newcommand{\compH}{\compgen[\mathsf{H}]}

\newcommand{\suff}{\psi}
\newcommand{\suffhat}{\wh{\psi}}
\newcommand{\suffmap}{\mb{\psi}}
\newcommand{\Suff}{\Psi}
\newcommand{\fsuff}{f^{\suff}}
\newcommand{\pisuff}{\pi_{\suff}}

\newcommand{\etdopt}{\textsf{E\protect\scalebox{1.04}{2}D.Opt}\xspace}
\newcommand{\etdopttext}{Optimistic Estimation-to-Decisions\xspace}

\newcommand{\Lcont}{L_{\mathrm{lip}}}

\newcommand{\bz}{\mathbf{z}}

\newcommand{\bonus}{$^{\star}$}

\newcommand{\Ent}{\mathsf{Ent}}

\newcommand{\ucb}{\textsf{UCB}\xspace}
\newcommand{\linucb}{\textsf{LinUCB}\xspace}

\newcommand{\fucb}{\bar{f}}
\newcommand{\flcb}{\underline{f}}

\newcommand{\nt}[1][t]{n\ind{#1}}
\newcommand{\mt}[1][t]{m\ind{#1}}

\newcommand{\expthree}{\textsf{Exp3}\xspace}
\newcommand{\expfour}{\textsf{Exp4}\xspace}
\newcommand{\egreedy}{\texorpdfstring{$\veps$-\textsf{Greedy}}{ε-Greedy}\xspace}

\newcommand{\ucbtext}{\textsf{UCB}\xspace}
\newcommand{\ucbshort}{\textsf{UCB}\xspace}
\newcommand{\ucblong}{Upper Confidence Bound\xspace}

\newcommand{\squarecb}{\textsf{SquareCB}\xspace}

\newcommand{\subgaussian}{sub-Gaussian\xspace}
\newcommand{\subg}{\subgaussian}

\newcommand{\RegBasic}{\Reg}

\newcommand{\lhs}{left-hand side\xspace}
\newcommand{\rhs}{right-hand side\xspace}

\newcommand{\generaldm}{general decision making\xspace}

\DeclarePairedDelimiter{\abs}{\lvert}{\rvert}
\DeclarePairedDelimiter{\brk}{[}{]}
\DeclarePairedDelimiter{\crl}{\{}{\}}
\DeclarePairedDelimiter{\prn}{(}{)}
\DeclarePairedDelimiter{\nrm}{\|}{\|}
\DeclarePairedDelimiter{\tri}{\langle}{\rangle}

\DeclarePairedDelimiter{\ceil}{\lceil}{\rceil}
\DeclarePairedDelimiter{\floor}{\lfloor}{\rfloor}

\DeclareMathOperator*{\argmin}{arg\,min}
\DeclareMathOperator*{\argmax}{arg\,max}             

\newcommand{\algo}{\widehat{f}}

\newcommand{\exploss}{{\mathsf{L}}}
\newcommand{\emploss}{\widehat{\mathsf{L}}}
\newcommand{\loss}{\ell}
\newcommand{\logloss}{\loss_{\mathrm{\log}}}
\newcommand{\vloss}{\boldsymbol{\ell}}
\newcommand{\vlosstilde}{\widetilde{\vloss}}
\newcommand{\cYpred}{\cY^{\textsf{\tiny{pred}}}}
\newcommand{\indic}[1]{\mathbb{I}\left\{#1\right\}}
\newcommand{\reals}{\mathbb{R}}
\newcommand{\inner}[1]{\left\langle #1 \right\rangle}

\newcommand{\norm}[1]{\left\|#1\right\|}

\newcommand{\tr}{\ensuremath{{\scriptscriptstyle\mathsf{\,T}}}} 
\newcommand{\Prob}[1]{\mathbb{P}\left(#1\right)}

\newcommand{\compl}{\mathsf{comp}}

\def\ddefloop#1{\ifx\ddefloop#1\else\ddef{#1}\expandafter\ddefloop\fi}
\def\ddef#1{\expandafter\def\csname 
bb#1\endcsname{\ensuremath{\mathbb{#1}}}}
\ddefloop ABCDEFGHIJKLMNOPQRSTUVWXYZ\ddefloop
\def\ddefloop#1{\ifx\ddefloop#1\else\ddef{#1}\expandafter\ddefloop\fi}
\def\ddef#1{\expandafter\def\csname 
b#1\endcsname{\ensuremath{\mathbf{#1}}}}
\ddefloop ABCDEFGHIJKLMNOPQRSTUVWXYZ\ddefloop
\def\ddef#1{\expandafter\def\csname 
c#1\endcsname{\ensuremath{\mathcal{#1}}}}
\ddefloop ABCDEFGHIJKLMNOPQRSTUVWXYZ\ddefloop
\def\ddef#1{\expandafter\def\csname 
h#1\endcsname{\ensuremath{\widehat{#1}}}}
\ddefloop ABCDEFGHIJKLMNOPQRSTUVWXYZ\ddefloop
\def\ddef#1{\expandafter\def\csname 
hc#1\endcsname{\ensuremath{\widehat{\mathcal{#1}}}}}
\ddefloop ABCDEFGHIJKLMNOPQRSTUVWXYZ\ddefloop
\def\ddef#1{\expandafter\def\csname 
t#1\endcsname{\ensuremath{\widetilde{#1}}}}
\ddefloop ABCDEFGHIJKLMNOPQRSTUVWXYZ\ddefloop
\def\ddef#1{\expandafter\def\csname 
tc#1\endcsname{\ensuremath{\widetilde{\mathcal{#1}}}}}
\ddefloop ABCDEFGHIJKLMNOPQRSTUVWXYZ\ddefloop
\def\ddef#1{\expandafter\def\csname s#1\endcsname{\ensuremath{\mathsf{#1}}}}
\ddefloop ABCDEFGHIJKLMNOPQRSTUVWXYZ\ddefloop

\newcommand{\rank}{\mathrm{rank}}

\newcommand{\dimbel}{d_{\mathsf{B}}}

\newcommand{\Xm}[1][M]{X\sups{#1}}
\newcommand{\Wm}[1][M]{W\sups{#1}}
\newcommand{\Xmstar}[1][\Mstar]{X\sups{#1}}
\newcommand{\Wmstar}[1][\Mstar]{W\sups{#1}}

\newcommand{\piq}{\pi\subs{Q}}
\newcommand{\piqh}{\pi_{\sss{Q},h}}

\newcommand{\multiline}[1]{\parbox[t]{\dimexpr\linewidth-\algorithmicindent}{#1}}

\newcommand{\linqstar}{Linear-$Q^{\star}$\xspace}

\newcommand{\unitball}{\sB_2^d(1)}

\newcommand{\Enp}{\En_{\pi\sim{}p}}

\newcommand{\PiCov}{\Psi}
\newcommand{\pihsa}{\pi_{h,s,a}}

\newcommand{\pimhat}{\pi\subs{\Mhat}}

\newcommand{\ucbvi}{\textsf{UCB-VI}\xspace}
\newcommand{\lsviucb}{\textsf{LSVI-UCB}\xspace}
\newcommand{\bilinucb}{\textsf{BiLinUCB}\xspace}

\newcommand{\ones}{\mathbb{I}}

\newcommand{\phibar}{\wb{\phi}}

\newcommand{\pif}[1][f]{\pi_{#1}}
\newcommand{\pifhat}[1][\fhat]{\pi_{#1}}
\newcommand{\pifstar}[1][\fstar]{\pi_{#1}}

\newcommand{\cov}{\iota}

\newcommand{\underbracet}[2]{\underbrace{#1}_{\text{#2}}}

\newcommand{\Mo}{M_{\cO}}%
\newcommand{\Mr}{M_{\cR}}%
\newcommand{\Mhato}{\Mhat_{\cO}}%
\newcommand{\Mhatr}{\Mhat_{\cR}}%

\newcommand{\etdtext}{Estimation-to-Decisions\xspace}

\newcommand{\compc}[1][\veps]{\mathsf{dec}^{\mathrm{c}}_{#1}}
\newcommand{\deccreg}[1][\veps]{\mathsf{dec}^{\mathrm{c}}_{#1}}

\newcommand{\var}{\mathbb{V}}
\newcommand{\Var}{\var}

\newcommand{\cTm}[1][M]{\cT\sups{#1}}

\newcommand{\thetam}[1][M]{\theta\sups{#1}}

\newcommand{\wm}[1][M]{w\sups{#1}}

\newcommand{\mum}[1][M]{\mu\sups{#1}}

\newcommand{\phim}[1][M]{\phi\sups{#1}}

\newcommand{\qbar}{\bar{q}}

\newcommand{\relu}{\mathsf{relu}}

\newcommand{\El}{\mathsf{Edim}}
\newcommand{\ElCheck}{\underline{\El}}

\newcommand{\errm}[1][M]{\mathrm{err}\sups{#1}}

\newcommand{\pcigw}{\textsf{PC-IGW}\xspace}

\newcommand{\MinimaxR}[1]{\mathfrak{M}(#1,T)}
\newcommand{\MinimaxReg}{\mathfrak{M}(\cM,T)}

\renewcommand{\emptyset}{\varnothing}

\newcommand{\NullObs}{\crl{\emptyset}}

\newcommand{\filt}{\mathscr{F}}

\newcommand{\hist}{\mathcal{H}}

\newcommand{\dom}{\nu}

\newcommand{\densm}[1][M]{m^{\sss{#1}}}

\newcommand{\FrameworkShort}{DMSO\xspace}

\newcommand{\learner}{learner\xspace}

\newcommand{\Act}{\Pi}

\newcommand{\obs}{o}

\newcommand{\ObsSpace}{\mathcal{\cO}}

\newcommand{\RewardSpace}{\cR}

\newcommand{\cMhat}{\wh{\cM}}

\newcommand{\comp}[1][\gamma]{\mathsf{dec}_{#1}}
\newcommand{\decoreg}[1][\gamma]{\mathsf{dec}_{#1}}
\newcommand{\compbar}[1][\gamma]{\wb{\mathsf{dec}}_{#1}}

\newcommand{\compb}[1][\gamma]{\underline{\mathsf{dec}}_{#1}}

\newcommand{\CompText}{Decision-Estimation Coefficient\xspace}
\newcommand{\CompAbbrev}{DEC\xspace}
\newcommand{\CompShort}{\CompAbbrev}

\newcommand{\compgen}[1][D]{\comp^{#1}}

\newcommand{\etd}{\textsf{E\protect\scalebox{1.04}{2}D}\xspace}

\newcommand{\M}[1]{^{{\scriptscriptstyle M}}}
 
\newcommand{\sups}[1]{^{{\scriptscriptstyle#1}}}
\newcommand{\subs}[1]{_{{\scriptscriptstyle#1}}}

\newcommand{\sss}[1]{{\scriptscriptstyle#1}}

\newcommand{\Ens}[2]{\En^{\sss{#1},#2}}

\newcommand{\Enm}[2]{\En^{\sss{#1},#2}}

\newcommand{\bbPm}[1][M]{\bbP\sups{#1}}

\newcommand{\bbPmhat}[1][\Mhat]{\bbP\sups{#1}}

\newcommand{\fmhat}{f\sups{\Mhat}}
\newcommand{\fmhatt}{f\sups{\Mhat\ind{t}}}

\newcommand{\fm}[1][M]{f\sups{#1}}
\newcommand{\pim}[1][M]{\pi_{\sss{#1}}}

\newcommand{\cFm}{\cF_{\cM}}

\newcommand{\cMf}{\cM_{\cF}}

\newcommand{\fmstar}{f\sups{\Mstar}}
\newcommand{\pimstar}{\pi\subs{\Mstar}}

\newcommand{\fstar}{f^{\star}}
\newcommand{\pistar}{\pi^{\star}}

\newcommand{\pihat}{\wh{\pi}}

\newcommand{\Mbar}{\wb{M}}

\newcommand{\fmi}{f\sups{M_i}}
\newcommand{\pimi}{\pi\subs{M_i}}

\newcommand{\Rm}[1][M]{R\sups{#1}}
\newcommand{\Pm}[1][M]{P\sups{#1}}
\newcommand{\Pmstar}[1][\Mstar]{P\sups{#1}}

\newcommand{\Pmhat}{P\sups{\Mhat}}
\newcommand{\Rmhat}{R\sups{\Mhat}}

\newcommand{\PiRNS}{\Pi_{\mathsf{rns}}}

\newcommand{\PiGen}{\Pi_{\mathrm{RNS}}}

\newcommand{\Qmstar}[1][M]{Q^{\sss{#1},\star}}
\newcommand{\Vmstar}[1][M]{V^{\sss{#1},\star}}
\newcommand{\Qmstarstar}[1][\Mstar]{Q^{\sss{#1},\star}}
\newcommand{\Vmstarstar}[1][\Mstar]{V^{\sss{#1},\star}}
\newcommand{\Qmpi}[1][\pi]{Q^{\sss{M},#1}}
\newcommand{\Vmpi}[1][\pi]{V^{\sss{M},#1}}
\newcommand{\Vmhatpi}[1][\pi]{V^{\sss{\Mhat},#1}}

\newcommand{\Qm}[2]{Q^{\sss{#1},#2}}

\newcommand{\dmpi}[1][\pi]{d^{\sss{M},#1}}

\newcommand{\dm}[2]{d^{\sss{#1},#2}}
\newcommand{\dbar}{\bar{d}}

\newcommand{\Reg}{\mathrm{\mathbf{Reg}}}

\newcommand{\Est}{\mathrm{\mathbf{Est}}}

\newcommand{\EstHel}{\mathrm{\mathbf{Est}}_{\mathsf{H}}}

\newcommand{\EstD}{\mathrm{\mathbf{Est}}_{\mathsf{D}}}

\newcommand{\EstSq}{\mathrm{\mathbf{Est}}_{\mathsf{Sq}}}
\newcommand{\EstSqOff}{\mathrm{\mathbf{Est}}_{\mathsf{Sq}}^{\mathsf{off}}}

\newcommand{\EstKL}{\mathrm{\mathbf{Est}}_{\mathsf{KL}}}

\newcommand{\RegDM}{\Reg}

\newcommand{\RegOL}{\Reg_{\mathsf{OL}}}
\newcommand{\RegOLSq}{\RegOL^{\mathsf{Sq}}}
\newcommand{\RegOLLog}{\RegOL^{\mathsf{Log}}}

  \newcommand{\AlgEst}{\mathrm{\mathbf{Alg}}_{\mathsf{Est}}}

\newcommand{\Mhat}{\wh{M}}
\newcommand{\Mstar}{M^{\star}}

\newcommand{\thetahat}{\wh{\theta}}
\newcommand{\thetastar}{\theta^{\star}}

\newcommand{\tens}{\otimes}

\newcommand{\algcommentlight}[1]{\textcolor{blue!70!black}{\transparent{0.5}\small{\texttt{\textbf{//\hspace{2pt}#1}}}}}

\newcommand{\Otilde}{\wt{O}}

\newcommand{\trn}{\top}

\newcommand{\psdgeq}{\succeq}

\newcommand{\psdgt}{\succ}
\newcommand{\approxleq}{\lesssim}
\newcommand{\approxgeq}{\gtrsim}

\newcommand{\fhat}{\wh{f}}

\newcommand{\fbar}{\bar{f}}

\newcommand{\ind}[1]{^{{\scriptscriptstyle#1}}}

\newcommand{\bigoh}{O}
\newcommand{\bigoht}{\wt{O}}
\newcommand{\bigom}{\Omega}
\newcommand{\bigomt}{\wt{\Omega}}

\newcommand{\poly}{\mathrm{poly}}
\newcommand{\polylog}{\mathrm{polylog}}

\newcommand{\Dsq}[2]{D_{\mathrm{Sq}}\prn*{#1,#2}}

\newcommand{\Dkl}[2]{D_{\mathsf{KL}}\prn*{#1\,\|\,#2}}

\newcommand{\Dhel}[2]{D_{\mathsf{H}}\prn*{#1,#2}}
\newcommand{\Df}[2]{D_{f}\prn*{#1\dmid{}#2}}
\newcommand{\Dgen}[2]{D\prn*{#1\dmid{}#2}}

\newcommand{\Dhels}[2]{D^{2}_{\mathsf{H}}\prn*{#1,#2}}
\newcommand{\Dhelshort}{D^{2}_{\mathsf{H}}}

\newcommand{\Dtv}[2]{D_{\mathsf{TV}}\prn*{#1,#2}}
\newcommand{\Dtvs}[2]{D^2_{\mathsf{TV}}\prn*{#1,#2}}

\newcommand{\DgenX}[3]{D\prn[#1]{#2,#3}}
\newcommand{\DhelsX}[3]{D^{2}_{\mathsf{H}}\prn[#1]{#2,#3}}

\newcommand{\Ber}{\mathrm{Ber}}

\newcommand{\dmid}{\;\|\;}

\newcommand{\conv}{\mathrm{co}}

\newcommand{\Qhat}{\wh{Q}}
\newcommand{\Qbar}{\widebar{Q}}
\newcommand{\Vhat}{\wh{V}}
\newcommand{\Vbar}{\widebar{V}}

\newcommand{\Vstar}{V^{\star}}
\newcommand{\Qstar}{Q^{\star}}

\newcommand{\unif}{\mathrm{unif}}

\newcommand{\Phat}{\wh{P}}

\newcommand{\qm}{q\ind{M}}

\newcommand{\igw}{\textsf{IGW}}

\newcommand{\supp}{\mathrm{supp}}

\newcommand{\cFbar}{\widebar{\cF}}

\newcommand{\mathand}{\quad\text{and}\quad}

\def\multiset#1#2{\ensuremath{\left(\kern-.3em\left(\genfrac{}{}{0pt}{}{#1}{#2}\right)\kern-.3em\right)}}

\newcommand{\grad}{\nabla}

\newcommand{\iid}{i.i.d.\xspace}

\renewcommand{\emptyset}{\varnothing}

\usepackage[nameinlink,capitalize]{cleveref}
\crefname{algbox}{Algorithm}{Algorithms}
\Crefname{algbox}{Algorithm}{Algorithms}

\newcommand{\pref}[1]{\cref{#1}}
\newcommand{\pfref}[1]{Proof of \pref{#1}}

\crefformat{equation}{#2(#1)#3}
\Crefformat{equation}{#2(#1)#3}

\Crefformat{figure}{#2Figure #1#3}
\Crefformat{assumption}{#2Assumption #1#3}

\Crefformat{exercise}{#2Exercise #1#3}
\Crefformat{exe}{#2Exercise #1#3}
\Crefformat{prop}{#2Proposition #1#3}
\Crefformat{rem}{#2Remark #1#3}
\Crefformat{cor}{#2Corollary #1#3}
\Crefformat{lem}{#2Lemma #1#3}

\usepackage{crossreftools}
\makeatletter
\let\OldStatex\Statex
\renewcommand{\Statex}[1][3]{%
  \setlength\@tempdima{\algorithmicindent}%
  \OldStatex\hskip\dimexpr#1\@tempdima\relax}
\makeatother

\DeclareMathOperator{\En}{\mathbb{E}}

\newcommand{\mb}[1]{\boldsymbol{#1}}
\newcommand{\wt}[1]{\widetilde{#1}}
\newcommand{\wh}[1]{\widehat{#1}}
\newcommand{\wb}[1]{\widebar{#1}}

\def\ddefloop#1{\ifx\ddefloop#1\else\ddef{#1}\expandafter\ddefloop\fi}
\def\ddef#1{\expandafter\def\csname bb#1\endcsname{\ensuremath{\mathbb{#1}}}}
\ddefloop ABCDEFGHIJKLMNOPQRSTUVWXYZ\ddefloop
\def\ddefloop#1{\ifx\ddefloop#1\else\ddef{#1}\expandafter\ddefloop\fi}
\def\ddef#1{\expandafter\def\csname b#1\endcsname{\ensuremath{\mathbf{#1}}}}
\ddefloop ABCDEFGHIJKLMNOPQRSTUVWXYZ\ddefloop
\def\ddef#1{\expandafter\def\csname sf#1\endcsname{\ensuremath{\mathsf{#1}}}}
\ddefloop ABCDEFGHIJKLMNOPQRSTUVWXYZ\ddefloop
\def\ddef#1{\expandafter\def\csname c#1\endcsname{\ensuremath{\mathcal{#1}}}}
\ddefloop ABCDEFGHIJKLMNOPQRSTUVWXYZ\ddefloop
\def\ddef#1{\expandafter\def\csname h#1\endcsname{\ensuremath{\widehat{#1}}}}
\ddefloop ABCDEFGHIJKLMNOPQRSTUVWXYZ\ddefloop
\def\ddef#1{\expandafter\def\csname hc#1\endcsname{\ensuremath{\widehat{\mathcal{#1}}}}}
\ddefloop ABCDEFGHIJKLMNOPQRSTUVWXYZ\ddefloop
\def\ddef#1{\expandafter\def\csname t#1\endcsname{\ensuremath{\widetilde{#1}}}}
\ddefloop ABCDEFGHIJKLMNOPQRSTUVWXYZ\ddefloop
\def\ddef#1{\expandafter\def\csname tc#1\endcsname{\ensuremath{\widetilde{\mathcal{#1}}}}}
\ddefloop ABCDEFGHIJKLMNOPQRSTUVWXYZ\ddefloop
\def\ddefloop#1{\ifx\ddefloop#1\else\ddef{#1}\expandafter\ddefloop\fi}
\def\ddef#1{\expandafter\def\csname scr#1\endcsname{\ensuremath{\mathscr{#1}}}}
\ddefloop ABCDEFGHIJKLMNOPQRSTUVWXYZ\ddefloop

\newcommand{\ls}{\ell}

\newcommand{\eps}{\epsilon}
\newcommand{\veps}{\varepsilon}

\newcommand{\ldef}{\vcentcolon=}
\newcommand{\rdef}{=\vcentcolon}

\makeatletter
\let\save@mathaccent\mathaccent
\newcommand*\if@single[3]{%
  \setbox0\hbox{${\mathaccent"0362{#1}}^H$}%
  \setbox2\hbox{${\mathaccent"0362{\kern0pt#1}}^H$}%
  \ifdim\ht0=\ht2 #3\else #2\fi
  }
\newcommand*\rel@kern[1]{\kern#1\dimexpr\macc@kerna}
\newcommand*\widebar[1]{\@ifnextchar^{{\wide@bar{#1}{0}}}{\wide@bar{#1}{1}}}
\newcommand*\wide@bar[2]{\if@single{#1}{\wide@bar@{#1}{#2}{1}}{\wide@bar@{#1}{#2}{2}}}
\newcommand*\under@bar[2]{\if@single{#1}{\under@bar@{#1}{#2}{1}}{\under@bar@{#1}{#2}{2}}}
\newcommand*\wide@bar@[3]{%
  \begingroup
  \def\mathaccent##1##2{%
    \let\mathaccent\save@mathaccent
    \if#32 \let\macc@nucleus\first@char \fi
    \setbox\z@\hbox{$\macc@style{\macc@nucleus}_{}$}%
    \setbox\tw@\hbox{$\macc@style{\macc@nucleus}{}_{}$}%
    \dimen@\wd\tw@
    \advance\dimen@-\wd\z@
    \divide\dimen@ 3
    \@tempdima\wd\tw@
    \advance\@tempdima-\scriptspace
    \divide\@tempdima 10
    \advance\dimen@-\@tempdima
    \ifdim\dimen@>\z@ \dimen@0pt\fi
    \rel@kern{0.6}\kern-\dimen@
    \if#31
      \overline{\rel@kern{-0.6}\kern\dimen@\macc@nucleus\rel@kern{0.4}\kern\dimen@}%
      \advance\dimen@0.4\dimexpr\macc@kerna
      \let\final@kern#2%
      \ifdim\dimen@<\z@ \let\final@kern1\fi
      \if\final@kern1 \kern-\dimen@\fi
    \else
      \overline{\rel@kern{-0.6}\kern\dimen@#1}%
    \fi
  }%
  \macc@depth\@ne
  \let\math@bgroup\@empty \let\math@egroup\macc@set@skewchar
  \mathsurround\z@ \frozen@everymath{\mathgroup\macc@group\relax}%
  \macc@set@skewchar\relax
  \let\mathaccentV\macc@nested@a
  \if#31
    \macc@nested@a\relax111{#1}%
  \else
    \def\gobble@till@marker##1\endmarker{}%
    \futurelet\first@char\gobble@till@marker#1\endmarker
    \ifcat\noexpand\first@char A\else
      \def\first@char{}%
    \fi
    \macc@nested@a\relax111{\first@char}%
  \fi
  \endgroup
}
\newcommand*\under@bar@[3]{%
  \begingroup
  \def\mathaccent##1##2{%
    \let\mathaccent\save@mathaccent
    \if#32 \let\macc@nucleus\first@char \fi
    \setbox\z@\hbox{$\macc@style{\macc@nucleus}_{}$}%
    \setbox\tw@\hbox{$\macc@style{\macc@nucleus}{}_{}$}%
    \dimen@\wd\tw@
    \advance\dimen@-\wd\z@
    \divide\dimen@ 3
    \@tempdima\wd\tw@
    \advance\@tempdima-\scriptspace
    \divide\@tempdima 10
    \advance\dimen@-\@tempdima
    \ifdim\dimen@>\z@ \dimen@0pt\fi
    \rel@kern{0.6}\kern-\dimen@
    \if#31
      \underline{\rel@kern{-0.6}\kern\dimen@\macc@nucleus\rel@kern{0.4}\kern\dimen@}%
      \advance\dimen@0.4\dimexpr\macc@kerna
      \let\final@kern#2%
      \ifdim\dimen@<\z@ \let\final@kern1\fi
      \if\final@kern1 \kern-\dimen@\fi
    \else
      \underline{\rel@kern{-0.6}\kern\dimen@#1}%
    \fi
  }%
  \macc@depth\@ne
  \let\math@bgroup\@empty \let\math@egroup\macc@set@skewchar
  \mathsurround\z@ \frozen@everymath{\mathgroup\macc@group\relax}%
  \macc@set@skewchar\relax
  \let\mathaccentV\macc@nested@a
  \if#31
    \macc@nested@a\relax111{#1}%
  \else
    \def\gobble@till@marker##1\endmarker{}%
    \futurelet\first@char\gobble@till@marker#1\endmarker
    \ifcat\noexpand\first@char A\else
      \def\first@char{}%
    \fi
    \macc@nested@a\relax111{\first@char}%
  \fi
  \endgroup
}
\makeatother

\numberwithin{equation}{section}
 
\begin{document}
\begin{center}
    {\sf\Large Foundations of Reinforcement Learning and Interactive Decision Making} \\[0.8cm]
  \begin{tabular}{c@{\hspace{2cm}}c}
    {\sffamily\large Dylan J. Foster} & {\sffamily\large Alexander Rakhlin} \\[2pt]
    {\sffamily\small University of California, Berkeley} & {\sffamily\small Massachusetts Institute of Technology} \\
    {\small\texttt{dylanf@mit.edu}} & {\small\texttt{rakhlin@mit.edu}}
  \end{tabular}
\end{center}

\vspace{0.5cm}
\begin{abstract}
Interactive decision making is the problem of learning to
act well in an unknown environment, using the data that one's own
actions generate to continuously improve, and arises in situations ranging from online platforms and robotics to medical treatments. This monograph gives a
\emph{statistical perspective} on algorithm design and complexity for
interactive decision making, building from
multi-armed bandits through contextual and structured bandits to
reinforcement learning with function approximation within a single,
unified framework. Special attention is paid to
function approximation and flexible models such as neural
networks, and to the connection between supervised learning and decision
making: the reader will learn how to turn any supervised learning method
into a decision making algorithm, how to analyze the result, and how to
determine whether a given problem can be solved with few interactions. \loose

A unifying theme is that for
interactive problems, unlike supervised learning, the question of how
hard a problem is cannot be separated from the question of how to solve
it. We develop this theme through two recent ideas, which complement
classical approaches such as optimism and posterior sampling: the
Estimation-to-Decisions principle, which reduces decision making to
supervised estimation combined with an exploration rule, and the
Decision-Estimation Coefficient, a complexity measure that both
determines the exploration rule and lower bounds the regret of any
algorithm. For contextual bandits, the resulting theory is essentially
complete and its algorithms are widely deployed; for reinforcement
learning, it represents the current frontier.
\end{abstract}

\noindent
\rule{\textwidth}{1pt}
\vspace{-3em}

\setcounter{tocdepth}{2}
\begingroup
\titlespacing*{\section}{0pt}{24pt}{12pt}
\renewcommand{\sectionbreak}{}
\tableofcontents
\endgroup
\newpage

\section{Introduction}
\label{sec:intro}
\setcounter{topnumber}{1}
Interactive decision making is the problem of learning to act
well in an unknown environment, using the data that one's own actions
generate to continuously improve. A doctor choosing treatments, a recommendation system choosing
what to display, and a robot learning to navigate all face the same
tension: to make good decisions, one must know the consequences of each option,
yet consequences can only be learned by choosing, and there is no way to
learn about the options one does not try. This monograph develops the
statistical foundations of this problem: what makes a decision making
problem easy or hard, and how to design algorithms that learn from as
few interactions as possible. \loose

We give a unified answer to both questions through
\emph{Estimation-to-Decisions}, a single principle for choosing
decisions that encompasses settings ranging from classical bandit problems to general reinforcement learning. The principle splits decision making into two parts: estimation, which can be handed to whatever off-the-shelf machine learning model one prefers, and an
exploration rule that turns the model's predictions into decisions. A reader who works
through the monograph will learn to build decision making algorithms
out of any supervised learning method, analyze them, and tell, for a problem they have not seen before, whether it can be solved with few
interactions and what the cost will be.\loose

\subsection{Interactive Decision Making}
\label{sec:decision_making}

Interactive decision making is a ubiquitous problem, found in medicine, online platforms, robotics, and the sciences; the following examples illustrate the key challenges.
\begin{itemize}
\item Medical treatment: based on a patient's medical history
  and vital signs, a doctor must decide which treatment will lead to the
  best outcome.
\item Controlling a robot: based on sensor readings, a
  controller must decide what signals to send to the robot's actuators so
  that it reaches a goal.
\end{itemize}
In both examples, the decision maker or \emph{learner} (in the first example the doctor, in the second the controller)
interacts with an \emph{unknown environment}. The doctor does not know
in advance how a given patient will respond to each treatment, and the
controller does not know how the signals it sends change the robot's
configuration, or what the terrain looks like. 
But because the learner acts, it
can learn: every decision produces data, the data reduces uncertainty
about the environment, and less uncertainty means better decisions
later. The crux of interactive decision making is that each decision
plays two roles at once, and the learner must balance them: \emph{exploring},
to reduce uncertainty about the environment, and \emph{exploiting}, to perform
well given what they know now (e.g., reaching the goal as fast as possible).
In this monograph, the learner is an algorithm: our goal is to automate
the treatment decisions or robot control signals, in a way that improves
with experience.\loose

More concretely, \pref{fig:decision-making} depicts an idealized version of
the medical treatment example, to which we return throughout the monograph. At each
round $t$, the learner---here the doctor, who plays the role of the
algorithm---observes a patient's medical history and vital signs,
summarized in a \emph{context} $x\sups{t}$; chooses a treatment, the
\emph{decision} $\pi\sups{t}$; and then observes the outcome of the
treatment, in the form of a \emph{reward} $r\sups{t}$ and an auxiliary
\emph{observation} $o\sups{t}$ about, say, the progression of the
illness. Over many rounds, we wish to use the data $\crl{(x\ind{t}, \pi\ind{t}, r\ind{t}, o\ind{t})}_{t=1}^T$ collected by repeating this process to learn a good mapping
$x\mapsto\pi$ from contexts to decisions. How should such an algorithm
choose its decisions, so that it learns a good mapping from as few
interactions as possible?

\begin{figure}[t]
    \centering
    \includegraphics[width=0.8\textwidth]{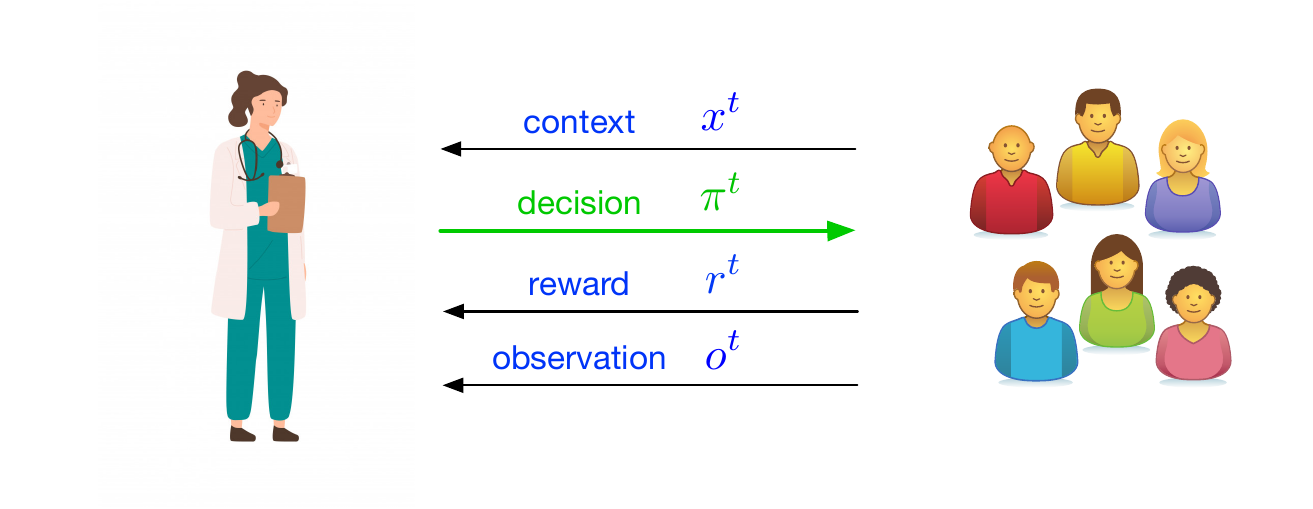}
    \caption{The general decision making problem. At each round, the learner (a doctor) selects a decision and observes a reward and an observation.}
    \label{fig:decision-making}
\end{figure}

\subsection{A Spectrum of Decision Making Problems}
\label{sec:spectrum}

The problem in \pref{fig:decision-making} can be formalized
through several settings, including multi-armed bandits, contextual
bandits, and reinforcement learning, which differ in the assumptions
they place on the environment and the decision making protocol. Each
assumption corresponds to a challenge that an algorithm must overcome,
and together they give rise to a \emph{spectrum} of decision making
problems. \cref{fig:axes} organizes this spectrum along three axes, one
for each of the core challenges we focus on:\loose

  \begin{figure}[t]
    \centering
    \includegraphics[width=0.75\textwidth]{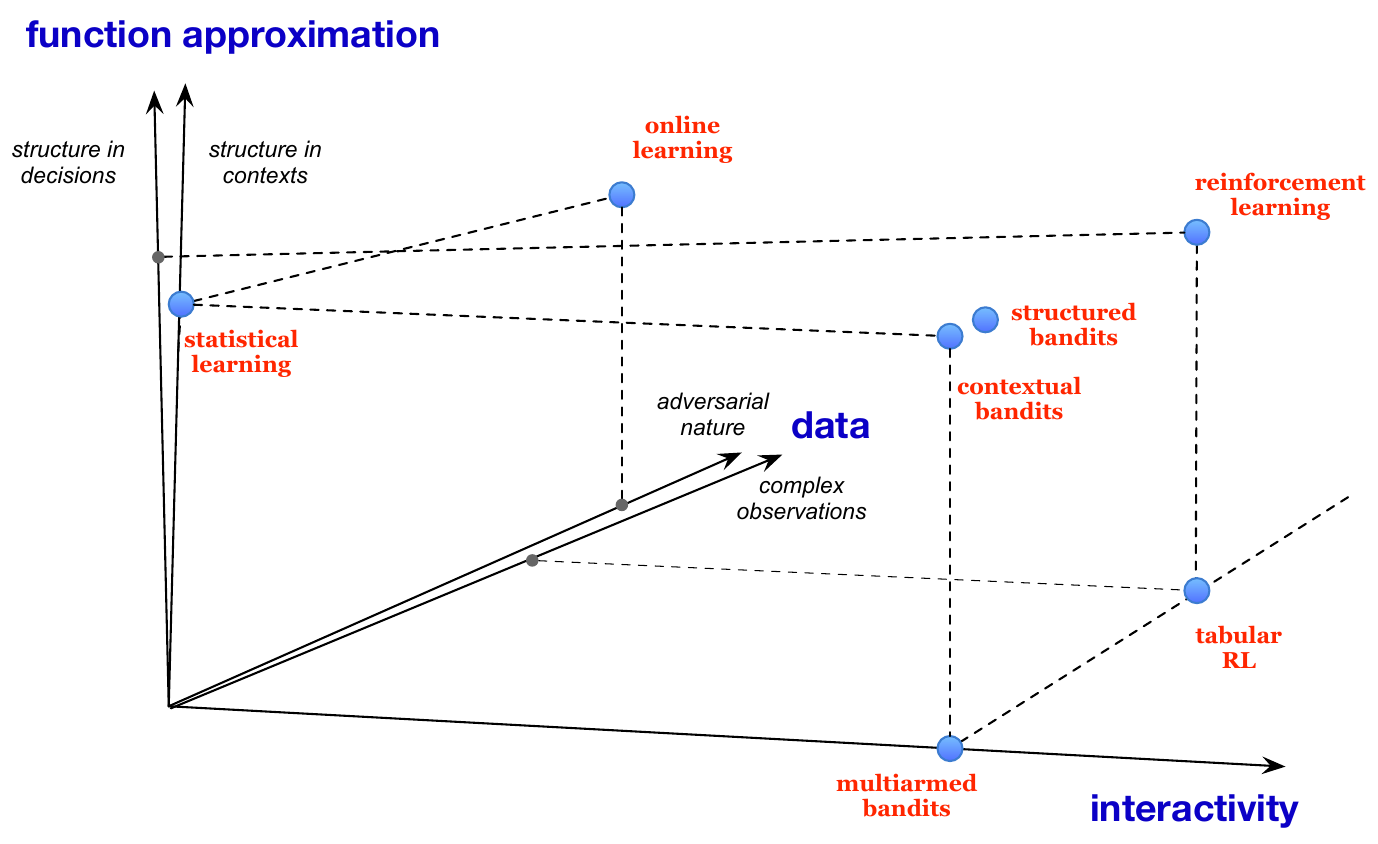}
    \caption{The landscape of decision making problems, organized along three axes: interactivity, function approximation, and data.}
    \label{fig:axes}
  \end{figure}
\begin{itemize}
\item \textbf{Interactivity.} Does the learner observe data
  passively, or do its decisions influence what data it collects? In
  \pref{fig:decision-making}, the doctor observes the effects of the
  treatments prescribed, but not the \emph{counterfactuals}: the effects of the
  treatments not given. The doctor's decisions therefore shape the data,
  and with them the ability to estimate the effects of the different
  treatments. This contrasts with classical statistical learning (\cref{sec:background}), where the dataset
  is given to the learner upfront, with no control over how it was
  collected. Multi-armed bandits (\cref{sec:mab}) isolate this
  challenge.
\item \textbf{Function approximation and generalization.}
  Supervised learning uses \emph{function approximation}
  (neural networks, kernels, forests) to generalize across the space of
  covariates. Interactive decision making algorithms can use function approximation in the same
  way, either to generalize across contexts (e.g., different patients) or to generalize across
  \emph{decisions} (e.g., different treatments). In \pref{fig:decision-making}, the context
  $x\sups{t}$ summarizing the medical history and vital signs may be a
  highly structured object, and the treatment $\pi\sups{t}$ may be a
  high-dimensional vector with interacting components, or a multi-stage
  treatment strategy. Contextual bandits (\cref{sec:cb}) and structured
  bandits (\cref{sec:structured}) isolate generalization across contexts
  and across decisions, respectively.
\item \textbf{Data.} Is the data (rewards and observations) the learner collects
  generated from their decisions through a fixed process, or does the underlying process change over time, potentially adversarially, or in response to the learner's actions? If the process
  is fixed, do we wish to model it directly, or proceed agnostically? Do we
  observe only labels, as in supervised learning, or a full trajectory of
  states, actions, and rewards, as for the robot control example? Online learning
  (\cref{sec:background}) isolates the first of these questions in a
  non-interactive setting; reinforcement learning (\cref{sec:mdp,sec:rl}) raises
  the last.
\end{itemize}
As \cref{fig:axes} shows, the basic frameworks---statistical
learning, online learning, contextual bandits, structured bandits---are
idealized problems that each capture one or two of these challenges,
while richer settings such as reinforcement learning involve all of
them.

Given this diversity, one might worry that every point on
the spectrum calls for its own algorithm, designed from scratch.
Supervised learning does not have this problem: there is essentially a
single guiding principle---fit a model to the data---and it applies
whether the model is a linear function or a neural network. For decision
making, the situation was far less clear until recently. The body of
theory this monograph presents, developed by many researchers over the
past two decades, has played a key role in clarifying it: it tells us
what each position on the spectrum demands of an algorithm, and it has
produced algorithms that are now widely used in practice, quite apart
from their theoretical interest. What a reader working
through this monograph can hope to take away depends on where on the
spectrum one looks.

\paragraph{Short-horizon problems: A settled theory}
At the near end of the spectrum, for problems in which each round
involves a single decision, such as contextual bandits, the reader will
find a theory of exploration that is, at least for the minimax questions
we ask, essentially complete, and that has been validated in practice. The first generation of contextual bandit
algorithms with provable guarantees---\linucb and its relatives
\citep{auer2002using,li2010contextual,chu2011contextual,abbasi2011improved},
Thompson sampling \citep{chapelle2011empirical,agrawal2013thompson}, and
the agnostic methods of
\citet{langford2008epoch,dudik2011efficient,agarwal2014taming}---was
developed in the 2000s and 2010s, and deployed in applications from
news recommendation \citep{li2010contextual} to mobile health
\citep{tewari2017ads}. The
regression-based methods that grew out of this line of work a decade
later \citep{foster2018practical}, \squarecb and the
Estimation-to-Decisions principle behind it
\citep{foster2020beyond,foster2021statistical}, are now standard, and are
deployed across industry.\footnote{\squarecb ships in Vowpal Wabbit
\citep{gerrits2020vowpal}, the engine behind Azure Personalizer, and in
Meta's Pearl library \citep{zhu2024pearl}. Documented deployments of
\squarecb and its inverse gap weighting rule include Uber
\citep{he2025uber}, Adyen \citep{vangara2024adyen},
and Expedia \citep{black2023adaptex}.
See
\citet{agarwal2016making} for the Decision Service, a canonical example of an early deployment, and \citet{bietti2018contextual}
for an empirical comparison of exploration algorithms.}
\cref{sec:mab,sec:cb,sec:structured} present this theory.

\paragraph{Long-horizon problems: The frontier}
At the far end of the spectrum, for problems that unfold over many
steps, as in reinforcement learning, the reader will find the current
frontier. Exploration was a central theme of the deep
reinforcement learning era of artificial intelligence that began with DQN and AlphaGo
\citep{mnih2015human,silver2016mastering}, which produced a range of
methods aimed at exploring more deliberately than the random
perturbations used by default: bootstrapped value functions, count-based
and curiosity-driven bonuses, random network distillation, and Go-Explore
\citep{osband2016deep,bellemare2016unifying,tang2017exploration,pathak2017curiosity,burda2018exploration,ecoffet2021first}.
Yet how much deliberate exploration matters in practice, and in which
domains, remains unclear, as does the role it will play as AI systems
increasingly build on powerful pre-trained models. We expect exploration
to remain a fundamental capability, but it may take a form quite
different from the one studied classically. For now, what we give is our best account of the
field's frontier of understanding for when, and how, one can explore better
than at random. \cref{sec:general_dm,sec:lower,sec:rl} present this
account as it stands today, after \cref{sec:mdp} lays the groundwork in
the tabular setting.\loose

In both regimes, the lesson as to what the theory can offer is the same: reasoning about
\emph{statistical complexity}, that is, about how many interactions a
problem fundamentally requires, tells us how to use what we know about
the problem---whether stated explicitly, through the decision making
protocol, or implicitly, through the choice of a model such as a neural
network---to decide how to explore, adaptively, as data come in. This
kind of reasoning has produced methods, \squarecb among them, that are
useful in practice and yet explore in ways that were far from obvious a
priori. Underlying this is a perspective that \cref{sec:minimax_intro}
develops: in interactive decision making, the question of how hard a
problem is cannot be separated from the question of how to solve it.

\subsection{Scope of This Monograph}
\label{sec:scope}

We aim to do one thing well: to give the reader a working
understanding of the statistical complexity of interactive decision
making. By this we mean an understanding of what determines how many
interactions a problem requires, how this depends on the structure of
the problem and on the model class, and how to design algorithms that
come as close as possible to the best achievable rate. We build up this
understanding one setting at a time, using classical tools such as
optimism, posterior sampling, and concentration inequalities alongside
the more recent Estimation-to-Decisions principle and
Decision-Estimation Coefficient, which tie the settings together into a
single framework.

\paragraph{Roadmap}
\cref{sec:background,sec:mab,sec:cb,sec:structured,sec:mdp} add one
form of structure at a time, following the axes of \cref{fig:axes};
\cref{sec:general_dm,sec:lower,sec:rl} give the general theory and
apply it to reinforcement learning.
\begin{itemize}[left=0pt]
\item \textbf{Background.} \cref{sec:background} reviews statistical
  and online learning, and states the estimation guarantees that the
  algorithms in later sections use as subroutines.
\item \textbf{Multi-armed bandits.} \cref{sec:mab} studies the simplest
  setting, with finitely many decisions and no contexts. It shows that
  the greedy strategy fails, and introduces the principle of optimism
  through the \ucb algorithm.
\item \textbf{Contextual bandits.} \cref{sec:cb} adds contexts and
  function approximation. It shows that optimism can fail for general
  function classes, and introduces the Estimation-to-Decisions principle
  through \squarecb, which reduces contextual bandits to supervised regression.
\item \textbf{Structured bandits.} \cref{sec:structured} adds structure
  across decisions, as in linear bandits. It introduces the
  Decision-Estimation Coefficient, which determines how much exploration
  a class of problems requires, and the \etd algorithm, which achieves
  it.
\item \textbf{Tabular reinforcement learning.} \cref{sec:mdp} adds
  multi-step decisions and credit assignment, in the ``tabular'' (finite state/action) setting. It
  gives the \ucbvi algorithm, based on optimism, and a first
  Estimation-to-Decisions algorithm for reinforcement learning.
\item \textbf{General decision making.} \cref{sec:general_dm} brings
  the preceding sections together. It introduces a general decision
  making framework that contains every setting above, and extends the
  Decision-Estimation Coefficient and \etd to it, so that a single
  algorithm and a single complexity measure cover all of them.
\item \textbf{Lower bounds.} \cref{sec:lower} is the punchline of this
  development. It shows that the Decision-Estimation Coefficient also
  \emph{lower bounds} the regret of any algorithm---whether or not it follows
  the Estimation-to-Decisions principle---quantifying the limit of sample-efficient learning. The upper bounds of
  \cref{sec:general_dm} are therefore optimal up to the cost of
  estimation.%
\item \textbf{Reinforcement learning with function approximation.}
  \cref{sec:rl} takes the theory to its current frontier. It shows that
  realizability alone is not enough for sample-efficient reinforcement
  learning, and gives algorithms for low-rank MDPs and for MDPs with
  low Bellman rank, and surveys the landscape of what is known to be
  tractable and what is not.
\end{itemize}
Subsections marked with (${\star}$) contain optional material that can be
skipped on a first reading. Readers interested in contextual
bandits as deployed in practice need only
\cref{sec:background,sec:mab,sec:cb}.\loose

\paragraph{What we do not cover}
Our focus on the statistical complexity of exploration, and on the
algorithms that address it, leaves out several topics that are
important in their own right. On the statistical
side, we do not cover offline reinforcement learning, in which the
learner receives a fixed dataset and never interacts
\citep{levine2020offline}, or imitation learning, in which it learns from
demonstrations \citep{pomerleau1988alvinn,ross2011reduction,foster2024behavior};
the latter in particular is of great practical importance. We also do not
cover active learning or experimental design, in which the learner
chooses what data to collect but earns no reward along the way, nor
off-policy evaluation or non-stationary environments, all of which matter in deployment. We treat adversarial environments and Bayesian
methods only briefly, and, beyond passing examples, we do not cover control or other topics such as partial monitoring. Finally, while we cover many efficient and practically successful algorithms, we do not systematically address computational complexity in reinforcement learning, a fascinating topic in its own right \citep{golowich2024exploration,rohatgi2025necessary}. \loose

On the practical
side, we do not cover policy gradient methods
\citep{williams1992simple,sutton2000policy,schulman2015trust,schulman2017proximal},
which are the workhorse of deep reinforcement learning and widely used in practice (see \citet{agarwal2021theory} for their theory). Notably, existing guarantees for these methods require the data distribution to already cover the relevant states, and little is known about exploration with policy gradient methods beyond tabular and linear settings \citep{cai2020provably,agarwal2020pc}. A related emerging topic we do not cover is the role of reinforcement
learning in foundation models, where one starts from a powerful
pre-trained model rather than from scratch
\citep{christiano2017deep,stiennon2020learning,ouyang2022training,bai2022training,rafailov2024direct,xie2024exploratory,huang2025self,huang2025best,foster2025good,chen2025when}.

For the broader picture, we recommend \citet{agarwal2019reinforcement} for theoretical guarantees for common reinforcement learning recipes; see also \citet{lattimore2020bandit} and \citet{slivkins2019introduction} for bandits, and \citet{cesa2006prediction} and \citet{hazan2016introduction} for online learning.

\subsection{The Interplay Between Statistical Complexity and Algorithm Design}
\label{sec:minimax_intro}

Two things remain before we begin. The first is to say
precisely what we mean by \emph{statistical complexity}, since the term
motivates the questions asked in every later section. The second is to
expand a claim above: that in interactive decision making, the
question of how hard a problem is cannot be separated from the question
of how to solve it. The points are related, and together they explain why
a unified treatment is possible at all.

We begin with the definition. We take a \emph{minimax} point of view. Let $\cM$ be a set of
possible models, or environments, that the learner may encounter; $\cM$
represents the learner's prior knowledge about the environment. Let
$\textsf{Alg}$ denote an algorithm, and $\textsf{Perf}_T(\textsf{Alg},
M)$ a measure of its performance on the model $M\in\cM$ after $T$ rounds
of interaction (or, in passive learning, after observing $T$ data
points). We want algorithms that perform well for any model in $\cM$, in the sense that $\textsf{Alg}$ approximately solves the minimax
problem
\begin{align}
	\label{eq:minimax_intro}
	\min_{\textsf{Alg}}~ \max_{M\in\cM}~ \textsf{Perf}_T(\textsf{Alg}, M).
\end{align}
Understanding the \emph{statistical complexity} of a problem amounts to
establishing matching, or nearly matching, upper bounds
$\wb{\phi}_T(\cM)$ and lower bounds $\underline{\phi}_T(\cM)$ on the
minimax value in \cref{eq:minimax_intro}. Notice what this definition
demands. An upper bound requires exhibiting an algorithm; a lower bound
requires reasoning about every algorithm at once. Whether these two
tasks can be separated is the question we turn to next.

In statistical learning, to a first approximation, there
is only a single algorithm. Empirical risk minimization, which fits the model that
best explains the data, is the default for nearly every problem, and theory is largely concerned with understanding how well it does: how the performance is determined by the richness of the model class and data. Thus, the question of what makes a problem
hard is largely separate from the question of how to solve it.
Interactive decision making is different. No single exploration strategy
is optimal across problems; the right one depends on the structure of
the model class, often in subtle ways; and the analysis
that reveals what makes a problem hard is the same analysis that
produces the algorithm for it. The central perspective of this monograph
is that statistical complexity and algorithm design must be studied
together. A theory of statistical complexity for decision making
therefore cannot avoid the question: for a given class $\cM$, which
algorithms solve \eqref{eq:minimax_intro}?

For many years, the answer was given one class at a time:
pick a model class $\cM$ (linear, generalized linear, and so on), design
an algorithm for it, and prove a guarantee. General techniques pointing
toward a deeper understanding existed (the eluder dimension,
information-directed sampling, Bellman rank), but an answer for
\emph{any} class $\cM$ did not. Part of the difficulty is that the
minimax problem \cref{eq:minimax_intro} ranges over the space of all
algorithms, which is far too large to search directly.

The key insight, used throughout the monograph, is that for the general
problem of \pref{fig:decision-making} we can restrict attention to
algorithms that interleave \emph{online estimation}, a form of
supervised learning described in \cref{sec:sl,sec:ol}, with an
\emph{exploration strategy} that decides how to balance exploration
against exploitation. This reduction is the \emph{Estimation-to-Decisions}
principle. As we show, algorithms built this way achieve or nearly
achieve optimality in \eqref{eq:minimax_intro} for a surprisingly wide
range of problems. The performance of this method is quantified by the \emph{Decision-Estimation Coefficient} (DEC), a central object that provides a link between complexity and algorithm design: the same quantity both (1) bounds the minimax value from below and (2) tells the algorithm how much to explore. This is what allows the theory we develop to treat every class at once.\loose
\setcounter{topnumber}{2}
 
\section{Background: Statistical and Online Learning}
\label{sec:background}
This section reviews the two non-interactive learning
problems on which the rest of the monograph builds: statistical
learning, in which the learner receives an i.i.d.\ dataset and aims to
generalize to unseen data, and online learning, in which the data arrive
sequentially and may be chosen adversarially. In the language of
\cref{fig:axes}, these are the problems obtained by switching off
interactivity. Our interest in these problems is not just for their own sake,
but in the \emph{estimation guarantees} they provide. Every decision
making algorithm in later sections reduces part of its problem to
supervised or online estimation, and inherits its sample complexity from
the results stated here.

\paragraph{Organization}
\cref{sec:sl} introduces statistical learning and empirical risk minimization. \cref{sec:ol} introduces online learning
and regret, and shows that regret controls estimation error for realizable sequential estimation problems; this is the form in which later sections use
online learning. \cref{sec:online_algs} presents algorithms that attain
these guarantees, chiefly exponential weights, and \cref{sec:tail_bounds}
collects the concentration inequalities behind the proofs. Readers
familiar with the material can skim \cref{sec:sl,sec:ol} for the main
statements and return to the rest as needed.

\subsection{Statistical Learning and Generalization}
  \label{sec:sl}

\newcommand{\sltdist}{\Mstar}
\newcommand{\sltconddist}{\Mstar_{y|x}}
\newcommand{\sltmarginal}{\Mstar_x}

We begin with a short refresher on the statistical learning
problem. Statistical learning is a purely passive problem in which the learner does not directly interact with the environment, but it captures the challenge of \emph{generalization and function approximation} in the context of \pref{fig:axes}. 

In the statistical learning problem, we receive examples
$$(x\sups{1},y\sups{1}),\ldots,(x\sups{T},y\sups{T}) \in\cX\times\cY,$$
i.i.d. from an unknown distribution $\sltdist = \sltmarginal\times \sltconddist$, where $\sltmarginal$ is the marginal distribution of the $x$ variable and $\sltconddist$ is the conditional distribution of $y$ given $x$. Here $x\ind{t}\in\cX$
are \emph{features} (sometimes called contexts or covariates), and $\cX$ is
the feature space; $y\ind{t}\in\cY$ are called \emph{outcomes}, and $\cY$ is
the outcome space. Given the data
$(x\sups{1},y\sups{1}),\ldots,(x\sups{T},y\sups{T})$, the goal is to
produce a model (or, estimator)
$\fhat:\cX\to\cYpred$ that will do a good job \emph{predicting} outcomes from
features for future examples $(x,y)$ drawn from $\sltdist$. As discussed below, there are a number of settings where the outcome space $\cY$ is different from the prediction space $\cYpred$, and thus we give them a separate notation.

To measure prediction performance, we decide on a \emph{loss
  function} $\loss:\cYpred\times\cY\to\reals$. Standard examples include:
\begin{itemize}[left=0pt]
\item Regression, where $\cY=\cYpred=\bbR$ and common losses include the \emph{square loss}
  $$\loss_{\textsf{sq}}(a,b)=(a-b)^2.$$ 
\item Classification, where $\cY=\cYpred=\crl{-1,1}$ and we may consider the
  indicator (or $0$-$1$) loss $$\loss_\text{01}(a,b)=\indic{a\neq{}b}.$$  Alternatively, if $\cY=\crl{-1,1}$ but $\cYpred=\reals$ (i.e. prediction is a real-valued ``score''; large means ``confident''), we may consider the logistic loss $$\loss_{\textsf{logistic}}(a,b) = \log\prn{1+\exp\{-ab\}}.$$
  \item Conditional density estimation with the \emph{logarithmic loss}
  (log loss). Here $\cYpred=\Delta(\cY)$, the set of distributions on $\cY$, and for $p\in\cYpred$,
  \begin{align}
    \logloss(p, y) = -\log p(y).
\end{align}
\end{itemize}
The choice of the loss function is typically far from unique, and may lead to different interpretations of the results and different computational properties of the ensuing algorithms. For instance, we rarely optimize a model to fit the indicator loss directly because of its discontinuous behavior, and instead opt for its ``softmax''  relaxation, the logistic loss.  

For a function $f:\cX\to\cYpred$, we measure the prediction performance
via the \emph{population} (or, ``test'', ``out of sample'') loss
\begin{align}
    \exploss(f) \ldef \En_{(x,y)\sim\sltdist}\brk*{\loss(f(x),y)}.
\end{align}

Letting $\cH\sups{T}\ldef\{(x\sups{t},y\sups{t})\}_{t=1}^T$ denote the
dataset, a (deterministic) \emph{algorithm} is a map that takes the dataset as input and returns a function/predictor:
\begin{align}
    \algo(\cdot ; \cH\sups{T}): \cX\to\cYpred.
\end{align}
A classical goal in designing algorithms is to ensure that 
$\En\brk[\big]{\exploss(\fhat)}$ is minimized, where $\En\brk{\cdot}$
denotes expectation with respect to the draw of the dataset
$\cH\ind{T}$. Without any assumptions, it is not possible to
learn a good predictor unless the number of examples $T$ scales with
$\abs{\cX}$ (this is sometimes called the \emph{no-free-lunch
  theorem}). The basic idea behind statistical learning is to work
with a restricted class of functions
\[
\cF\subseteq\crl*{f:\cX\to\cYpred}
\]
in order to facilitate generalization. The class $\cF$ can be thought
of as (implicitly) encoding prior knowledge about the structure of
the data. For example, in computer vision, if the features $x\ind{t}$
correspond to images and the outcomes $y\ind{t}$ are labels
(e.g., ``cat'' or ``dog''), one might expect that choosing $\cF$ to be
a class of convolutional neural networks will work well, since this
encodes spatial structure. For language modeling, recent practice has shown that autoregressive models (such as the Transformer architecture) effectively encode temporal structure through attention mechanisms.

The setting of conditional density estimation may be less familiar to the reader, so let us elaborate. In this setting, the learner is required to compute, for each $x$, a distribution rather than a point estimate (see \pref{fig:cond-density} for a cartoon illustration). We write this conditional distribution as $f(\cdot\mid{} x)$ or $f(x)$ interchangeably. For example, given an $x$ containing radar data, a probabilistic forecast could be: probability of ``rain'' is 30\% and ``sunshine'' 70\%; then, for a given realized outcome ``rain'', the assessed quality of the probabilistic prediction is the negative logarithm of the mass placed on that outcome, i.e. $-\log 0.3$. 
\begin{figure}
  \centering
  \includegraphics[width=0.7\textwidth]{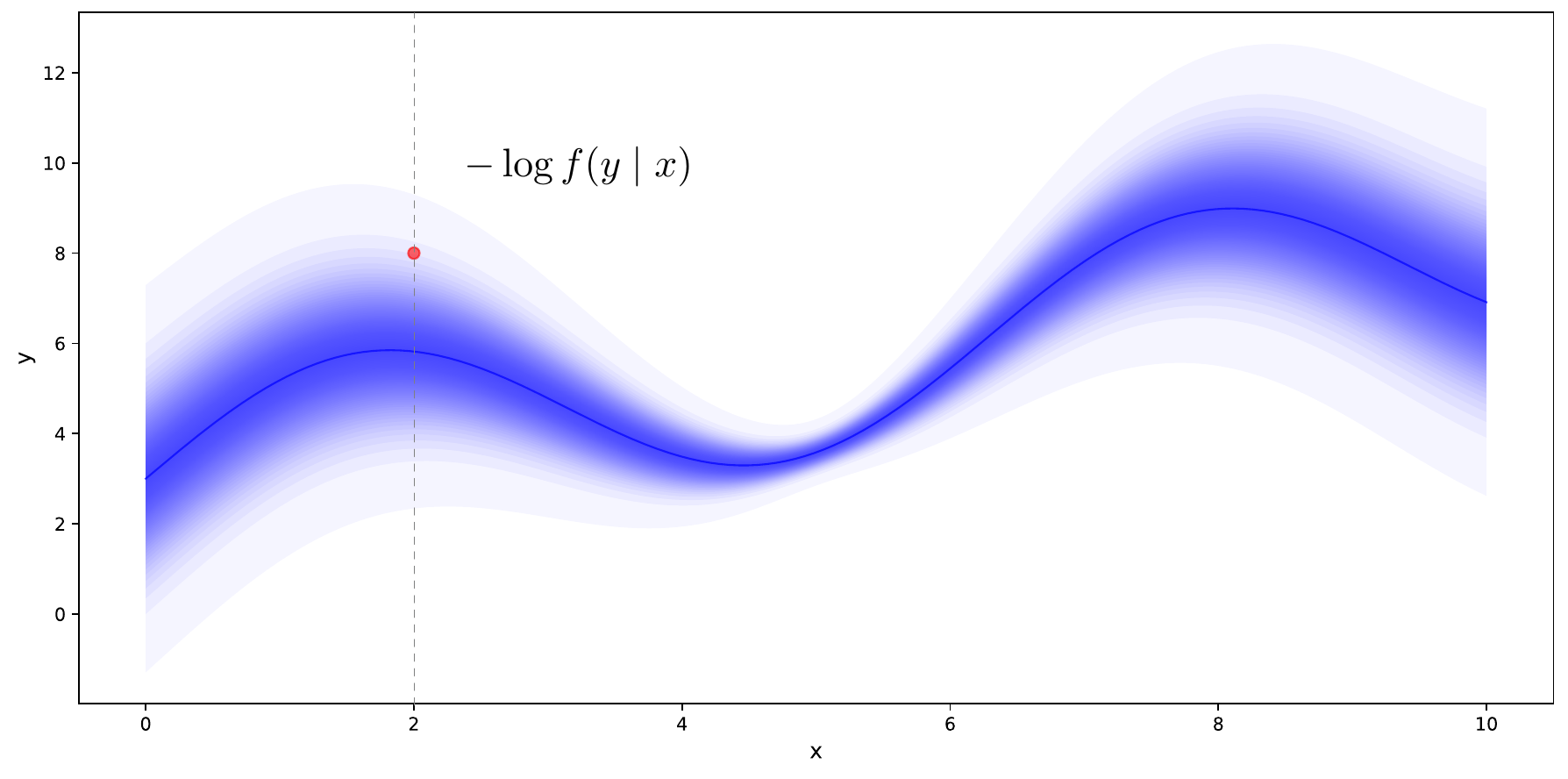}
  \caption{Conditional density estimation. Each $f(x)$ is a distribution $f(\cdot\mid{} x)$ over outcomes, and the loss on $(x,y)$ is $-\log f(y\mid{} x)$.}
  \label{fig:cond-density}
\end{figure}

In the context of language modeling, $\cX=\cY=\Sigma^*$, the set of sequences of tokens in a dictionary $\Sigma$, and $\cYpred=\Delta(\Sigma^*)$. We may then view $f(\cdot\mid{} x)$ as a generative model of (say) an answer $y$ to a question $x$.

The aforementioned loss functions induce (or suggest)  corresponding notions of ``discrepancy'' between functions in $\cF$. Since each function maps $\cX$ to $\cYpred$, we can measure the discrepancy between $f,g\in\cF$ by comparing the predictions of $f(x)$ and $g(x)$ for each $x\in\cX$ and aggregating; it is natural to average these differences over $x$ with respect to the marginal distribution $\sltmarginal$ on $\cX$. For instance, for regression problems, we can measure the (squared) distance between $f,g\in\cF$ as $\En_{x}\brk*{(f(x)- g(x))^2}$. For classification, the corresponding expected indicator becomes $\Prob{f(x)\neq{} g(x)}$, and for conditional density estimation, logarithmic loss can be related to the expectation over $x\sim \sltmarginal$ of the Kullback-Leibler divergence
\[
  \Dkl{f(\cdot \mid{} x)}{g(\cdot \mid{} x)}=\En_{y\sim f(\cdot\mid{} x)}\brk*{\log \left\{f(y\mid{} x)/g(y\mid{} x) \right\}}
\]
or the squared Hellinger distance
\[
  \Dhels{f(\cdot\mid{} x)}{g(\cdot\mid{} x)} = 2\En_{y\sim g(\cdot\mid{} x)}\brk*{1-\sqrt{\frac{f(y\mid{} x)}{g(y\mid{} x)}}}.
\]

\paragraph{Empirical risk minimization and excess risk}

The most basic and well-studied algorithmic principle for statistical
learning is \emph{Empirical Risk Minimization} (ERM). Define the empirical
loss for the dataset $\cH\ind{T}$ as
\begin{align}
    \emploss(f) = \frac{1}{T}\sum_{i=1}^T \loss( f(x\sups{i}),y\sups{i}).
\end{align}
Then, the empirical risk minimizer with respect to the class $\cF$
is given by 
\begin{align}
    \label{eq:erm_def}
    \algo \in \argmin_{f\in\cF} \emploss(f).
\end{align}
In the setting of conditional density estimation, \pref{eq:erm_def} is referred to as the Maximum Likelihood Estimate (MLE); \pref{fig:erm_fit} illustrates the two cases.

\begin{figure}
  \centering
  \includegraphics[width=0.49\textwidth]{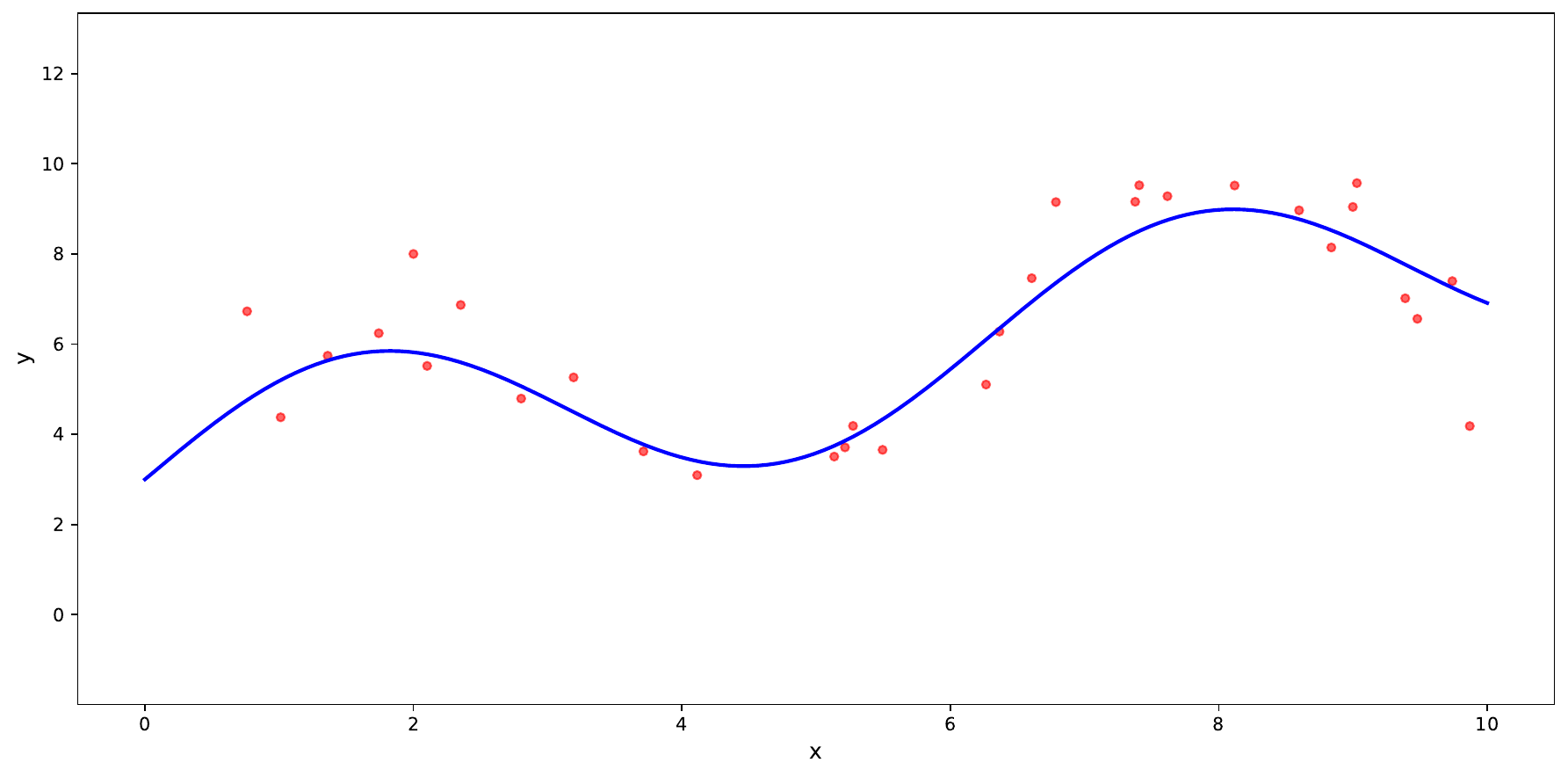}
  \includegraphics[width=0.49\textwidth]{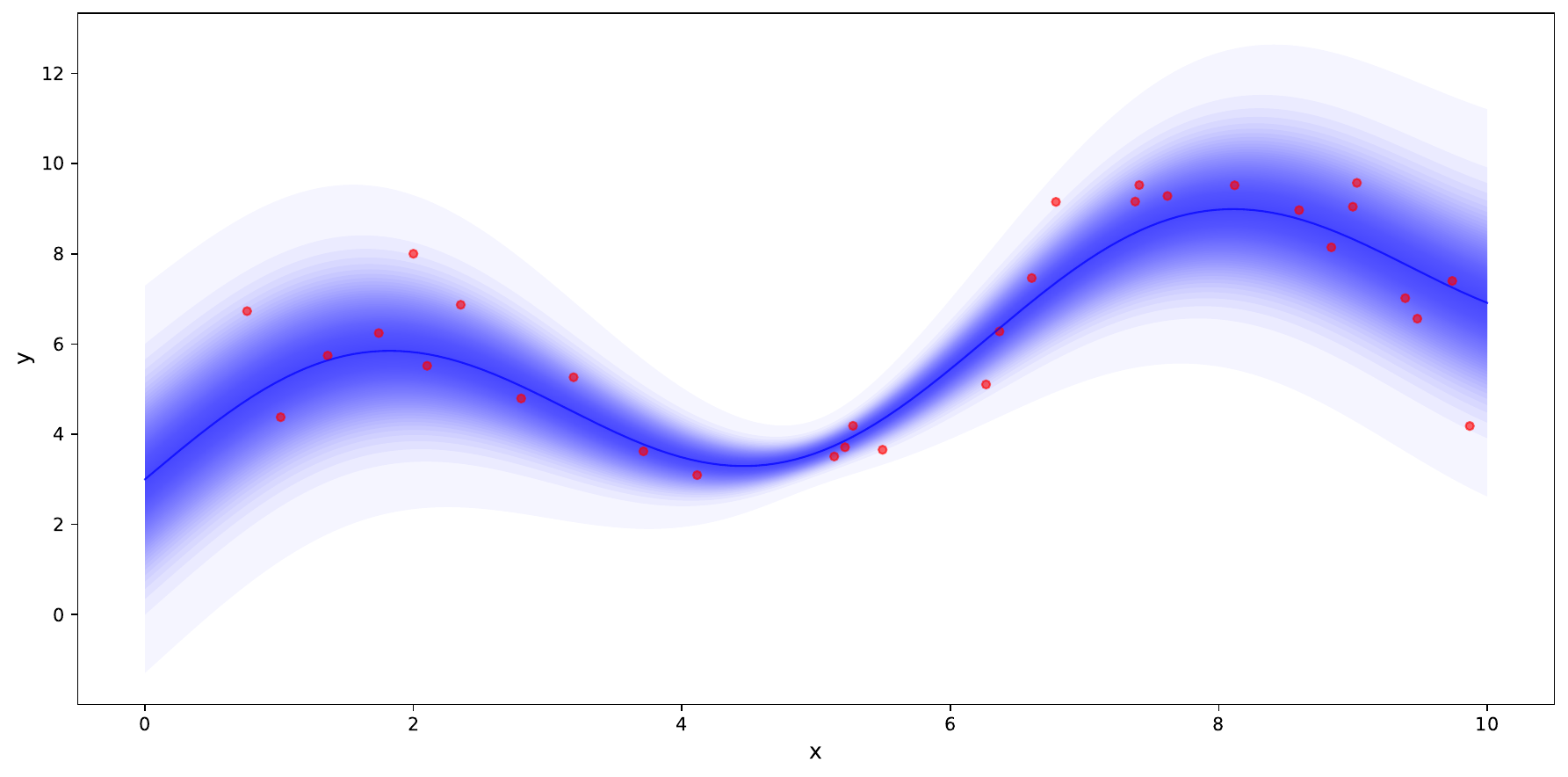}
  \caption{Least squares fit for a regression problem (left) and maximum likelihood estimate for conditional density estimation (right).}
  \label{fig:erm_fit}
\end{figure}

To measure the performance of ERM and other algorithms that use $\cF$ as encoding prior belief about a good set of models, we consider \emph{excess loss} (or, regret)
\begin{align}
  \label{eq:excess_risk}
    \mathcal{E}(f) = \exploss(f) - \min_{g\in\cF} \exploss(g).
\end{align}
Intuitively, the quantity
$\min_{g\in\cF} \exploss(g)$ in \pref{eq:excess_risk} captures
the best prediction performance any function in $\cF$ can achieve,
even with the knowledge of the \emph{true distribution} $\sltdist$. If an algorithm
$\fhat$ has low excess risk, this means that we
are predicting future outcomes nearly as well as any
algorithm based on samples can hope to perform. ERM and other
algorithms can ensure that  $\mathcal{E}(\algo)$ is small in expectation or with high probability
over draw of the dataset $\cH\ind{T}$. 

\paragraph{Connection to estimation}
An appealing feature of the formulation in \pref{eq:excess_risk} is that
it does not presuppose any relationship between the class $\cF$ and
the data distribution; in other words, it is agnostic. However, if
$\cF$ does happen to be good at modeling the data distribution, the
excess loss has an additional interpretation based on estimation.
\begin{definition}[Realizability]
    For regression problems, we say that the problem is
    \emph{realizable} (or, \emph{well-specified})\footnote{The term ``realizability'' comes from statistical and computational learning theory, whereas ``well-specified'' is the more common term in statistics. This monograph draws heavily on both bodies of literature; we adopt the former term throughout, somewhat arbitrarily, for the sake of consistency.} if the regression function $\fstar(x)\ldef\En\brk{y\mid{} x}$ is in $\cF$; we refer to this condition as \emph{realizability}, and use the same term for the analogous assumption in later sections. For conditional density estimation, we say that the problem is realizable if the conditional distribution $\sltconddist$ is in $\cF$.
\end{definition}
For the case of square loss, the regression function $\fstar$ can also be seen as a minimizer of $\exploss(f)$ over measurable functions $f$, for the same reason that $\En_{z}\brk*{(z-b)^2}$ is minimized at $b=\En\brk*{z}$. Similarly, for the logarithmic loss, $f=\sltconddist$ is a minimizer of $\exploss(f)$ over all conditional distributions coupled with the marginal distribution $\sltmarginal$.
\begin{lem}[Excess risk under realizability: square loss]
  \label{lem:square_well_specified}
	For regression, if the problem is \textit{realizable},
        then for all $f:\cX\to\cY$,
	\begin{align}
	    \label{eq:excess_loss_l2}
	    \cE(f) =  \En_{x}\brk*{(f(x)-\fstar(x))^2}
	\end{align}
\end{lem}
\begin{proof}[\pfref{lem:square_well_specified}]
	Adding and subtracting $\fstar$ in the first term of \eqref{eq:excess_risk}, we have
	\begin{align*}
		\En\brk*{(f(x)-y)^2} -  \En\brk*{(\fstar(x)-y)^2} &=  \En\brk*{(f(x)-\fstar(x))^2} + 2\En\brk*{(\fstar(x)-y)(f(x)-\fstar(x))}.
	\end{align*} 
\end{proof}
Inspecting \pref{eq:excess_loss_l2}, we see that any $f$ achieving low
      excess loss necessarily estimates the
      true regression function $\fstar$; hence, the goals of prediction and estimation coincide.
\begin{lem}[Excess risk under realizability: log loss]
  \label{lem:log_well_specified}
	For conditional density estimation, if the problem is \textit{realizable} with $\sltconddist=\fstar\in\cF$,
        then for all $f:\cX\to\Delta(\cY)$,
	\begin{align}
	    \label{eq:excess_loss_kl}
	    \cE(f) =  \En_{x}\brk*{\Dkl{\fstar}{f}}.
	\end{align}
\end{lem}

\cref{sec:ol} uses analogous quantities to measure
the cumulative estimation error of a sequence of online predictors.

\paragraph{Guarantees for ERM}
We give bounds on the excess loss of ERM for 
perhaps the simplest special case, in which $\cF$ is finite. 
\begin{prop}[ERM for finite classes]
    \label{prop:iid_finite_class}
    For any finite class $\cF$, empirical risk minimization satisfies
\begin{align}
    \label{eq:iid_finite_class}
    \En\brk[\big]{\cE(\algo)} \lesssim \compl(\cF, T),
\end{align}
where
\begin{enumerate}
    \item For any bounded loss (including classification), $\compl(\cF,T)=\sqrt{\frac{\log|\cF|}{T}}$.
    \item For square loss regression, if the
      problem is realizable, $\compl(\cF,T)=\frac{\log|\cF|}{T}$, assuming $Y$ and all predictions $f(x)$, $f\in\cF$, lie in a common bounded interval.
    \item For conditional density estimation, if the problem is realizable with $\sltconddist=\fstar\in\cF$, then the maximum likelihood estimate satisfies
      \begin{align*}
        \En\brk[\big]{\Dhels{\fhat}{\fstar}} \lesssim \frac{\log|\cF|}{T},
      \end{align*}
      where $\Dhelshort$ is the squared Hellinger distance, integrated over $x\sim\sltmarginal$.
    \end{enumerate}
    In addition, there exists a (different) algorithm that achieves
    $\compl(\cF,T)=\frac{\log|\cF|}{T}$ for both square loss regression and
    conditional density estimation, even when the problem is not realizable.
  \end{prop}
Henceforth, we use the symbol $\lesssim$ to indicate an inequality that holds up to constants, or other problem parameters deemed less important for the present discussion. As an example, the range of losses for the first part of \cref{prop:iid_finite_class} is hidden in this notation, and we only focus on the dependence of the right-hand side on $\cF$ and $T$.

The rate $\compl(\cF,T)=\sqrt{\frac{\log|\cF|}{T}}$ above is sometimes
referred to as a \emph{slow rate}, and is optimal for generic losses. The rate
$\compl(\cF,T)=\frac{\log|\cF|}{T}$ is referred to as a \emph{fast rate}, and
takes advantage of additional structure (curvature, or strong
convexity) of the square loss. Critically, both bounds scale only with
the cardinality of $\cF$, and do not depend on the size of the feature
space $\cX$, which could be infinite. This reflects the fact that
working with a restricted function class is allowing us to generalize
across the feature space $\cX$. In this context the cardinality $\log\abs{\cF}$
should be thought of a notion of \emph{capacity}, or \emph{expressiveness} for $\cF$. Intuitively,
choosing a larger, more expressive class will require a larger amount of data,
but will make the excess loss bound in \pref{eq:excess_risk} more
meaningful, since the benchmark will be stronger. 

\begin{rem}[From finite to infinite classes]
Throughout this monograph, we restrict our attention to finite
classes whenever possible in order to simplify presentation. If one
wishes to move beyond finite classes, a well-developed literature
  within statistical learning provides various notions of complexity
  for $\cF$ that lead to bounds on $\compl(\cF,T)$ for ERM and other
  algorithms. These include the Vapnik-Chervonenkis (VC) dimension for
  classification, Rademacher complexity, and covering
  numbers. Standard references include
  \citet{bousquet2004introduction,boucheron2005theory,anthony2009neural,shalev2014understanding}.
\end{rem}

Note that in statistical learning, the most naive
strategy, fitting the data as well as possible, is already near-optimal.
This is not the case for decision making, as \cref{sec:need_for_exploration}
shows.

\subsection{Online Learning and Prediction}
\label{sec:ol}

We now move on to the problem of \emph{online learning}, or sequential
prediction. The online learning problem generalizes statistical
learning on two fronts:
\begin{itemize}[left=0pt]
\item Rather than receiving a batch dataset of $T$ examples all at once, we receive
  the examples $(x\ind{t},y\ind{t})$ one by one, and must predict
  $y\ind{t}$ from $x\ind{t}$ only using the examples we have already observed.
\item Instead of assuming that examples are drawn from a fixed
  distribution, we allow examples to be generated in an arbitrary,
  potentially adversarial fashion.
\end{itemize}
\begin{algbox}[frametitle={\algprotocol{Online Learning Protocol}}]
  \begin{algorithmic}
\For{$t=1,\ldots,T$}
\State Compute predictor $\algo\sups{t}:\cX\to\cYpred$.
\State Observe $(x\sups{t}, y\sups{t})\in\cX\times\cY$.
\EndFor{}
\end{algorithmic}
\end{algbox}
In more detail, at each timestep $t$, given the examples
\begin{align}
    \cH\sups{t-1} = \{(x\sups{1}, y\sups{1}), \ldots, (x\sups{t-1}, y\sups{t-1}) \}
\end{align}
 observed so far, the algorithm produces a predictor
$$\algo\sups{t} = \algo\sups{t}(\cdot\mid{}\cH\sups{t-1}),$$
which aims to predict the outcome $y\ind{t}$ from the features
$x\ind{t}$. The algorithm's goal is to minimize the cumulative loss
over $T$ rounds, given by 
$$\sum_{t=1}^T \loss(\algo\sups{t}(x\sups{t}), y\sups{t})$$
for a known loss function $\ls:\cYpred\times\cY\to\bbR$; the cumulative
loss can be thought of as a sum of ``out-of-sample'' prediction
errors. Since we place no assumptions on the
data-generating process, it is not possible to make meaningful
statements about the cumulative loss itself. However, we can aim to
ensure that this cumulative loss is not much worse than the best
empirical explanation of the data by functions in a given class
$\cF$. That is, we measure the algorithm's performance via
\emph{regret} to $\cF$:
\begin{align}
  \label{eq:regret_deterministic}
    \RegOL = \sum_{t=1}^T \loss(\algo\sups{t}(x\sups{t}), y\sups{t}) - \min_{f\in\cF} \sum_{t=1}^T \loss(f(x\sups{t}), y\sups{t}).
\end{align}
The subscript $\mathsf{OL}$ identifies online prediction regret;
we reserve $\Reg$ for the decision-making regret studied in later
sections. When specifying the loss is useful, we write $\RegOLSq$
for square loss and $\RegOLLog$ for log loss.
Our aim is to design prediction algorithms that keep regret small for
\emph{any} sequence of data. As in statistical learning, the class $\cF$ should be thought of as
capturing our prior knowledge about the problem, and might be a linear
model or neural network. At first glance, keeping the regret small
for arbitrary sequences might seem like an
impossible task, as it stands in stark contrast with statistical
learning, where data is generated i.i.d. from a fixed
distribution.

Let us remark that it is often useful to apply online learning methods
in settings where data is not fully adversarial, but evolves according to processes too
difficult to directly model. Interactive decision making is a
case in point: the data a decision maker collects depend on its own past
decisions, and do not look like i.i.d.\ or simple time-series data.

\begin{rem}[Proper learning, improper learning, and randomization]
  The online learning protocol does not require that $\fhat\ind{t}$
  lies in $\cF$ ($\fhat\sups{t}\in\cF$). A method that chooses functions from $\cF$
  is called \emph{proper}, and one that may select predictors
  outside of $\cF$ is called \emph{improper}. It is also useful to allow
  for \emph{randomized} predictions of the form
$$\algo\sups{t} \sim q\sups{t}(\cdot |\cH\sups{t-1}),$$
where $q\sups{t}$ is a distribution on functions, typically on
elements of $\cF$. For randomized predictions, we slightly abuse
notation and write regret as
\begin{align}
  \label{eq:regret_randomized}
  \RegOL = \sum_{t=1}^T \En_{\fhat\ind{t}\sim{}q\ind{t}}\brk[\big]{\loss(\algo\sups{t}(x\sups{t}), y\sups{t})} - \min_{f\in\cF} \sum_{t=1}^T \loss(f(x\sups{t}), y\sups{t}).
\end{align}
The algorithms introduced below ensure small regret
even if data are adversarially and adaptively chosen. More precisely,
for deterministic algorithms, $(x\sups{t},y\sups{t})$ may be chosen
based on $\fhat\sups{t}$ and all the past data, while for randomized
algorithms, Nature can only base this choice on
$q\sups{t}$.
\end{rem}

In the context of \pref{fig:axes}, online learning moves
away from statistical learning along the data axis: it drops the
i.i.d.\ assumption but retains general-purpose generalization through
the class $\cF$. It is not an \emph{interactive} problem in the sense of this monograph. The learner's predictions may influence what Nature reveals next,
but the learner observes the outcome $y\ind{t}$ regardless of what it
predicts, so it never faces a tradeoff between predicting well and
gathering information. Like statistical learning, online learning is a
building block for decision making algorithms in later sections.\loose

\paragraph{Connection to statistical learning}

Online learning can be thought of as a generalization of statistical
learning, and in fact, algorithms for online learning immediately
yield algorithms for statistical learning via a technique called
\emph{online-to-batch conversion}. This result, which is formalized by
the following proposition, rests on two observations: the cumulative
loss of the algorithm looks like a sum of out-of-sample errors, and
the minimum empirical fit to realized data (over $\cF$) is, on average, a harder (that is, smaller) benchmark than the minimum expected loss in $\cF$. 
\begin{prop}[Online-to-batch conversion]
  \label{prop:online_to_batch}
  Let $(x\ind{1},y\ind{1}),\ldots,(x\ind{T},y\ind{T})$ be drawn \iid
  from a distribution $\sltdist$, and suppose the loss function
  $a\mapsto\loss(a, b)$ is convex in the first argument for all $b$. Then for any online learning
  algorithm, if we define
  \[
    \fhat(x) = \frac{1}{T}\sum_{t=1}^{T}\fhat\ind{t}(x),
  \]
  we have
  \[
\En\brk[\big]{\cE(\fhat)} \leq{} \frac{1}{T}\cdot\En\brk*{\RegOL}.
  \]
\end{prop}
\begin{proof}[\pfref{prop:online_to_batch}]
	Let $(x,y)\sim\sltdist$ be a fresh sample which is independent
        of the history $\cH\sups{T}$. 
	First, by Jensen's inequality,
	\begin{align}
		&\En\brk*{\exploss(\algo)}= \En\brk*{\En_{(x,y)}\brk*{\loss\left(\frac{1}{T}\sum_{t=1}^T \fhat\ind{t}(x), y \right)}}
		\leq \En\brk*{\frac{1}{T}\sum_{t=1}^T \En_{(x,y)}\brk*{\loss\left(\fhat\ind{t}(x), y \right)}}
	\end{align}
	which is equal to
	\begin{align}
		\En\brk*{\frac{1}{T}\sum_{t=1}^T \En_{(x\sups{t},y\sups{t})}\brk*{\loss\left(\fhat\ind{t}(x\sups{t}), y\sups{t} \right)}}
	\end{align}
	since $\fhat\ind{t}$ is a function of $\cH\sups{t-1}$ and $(x,y)$ and $(x\sups{t},y\sups{t})$ are i.i.d. 
	Second,
	\begin{align}
		\min_{f\in\cF} \exploss(f) = \min_{f\in\cF} \En\brk*{\frac{1}{T}\sum_{t=1}^T\loss(f(x\sups{t}),y\sups{t}) } \geq  \En\brk*{\min_{f\in\cF}\frac{1}{T}\sum_{t=1}^T\loss(f(x\sups{t}),y\sups{t}) }
	\end{align}
      \end{proof}
      In light of \cref{prop:online_to_batch}, one can interpret
      regret as generalizing the notion of excess risk from \iid data
      to arbitrary sequences. To motivate the algorithms below, we observe that ERM is insufficient to derive this type of guarantee: Repeatedly applying ERM to
past observations need not give small regret on arbitrary sequences.
\begin{prop}[Failure of ERM in online learning]
  \label{prop:erm_failure_online}
  For binary classification with indicator loss, there is a class of two
  constant predictors and a fixed sequence of outcomes on which online
  ERM, with a fixed deterministic tie-breaking rule, incurs regret at
  least $T/2$.
\end{prop}
\begin{proof}[\pfref{prop:erm_failure_online}]
  Use a single feature value and the predictors $f_0(x)=0$ and $f_1(x)=1$.
  Break ties in favor of $f_0$, and take $y\ind{t}=1$ for odd $t$ and
  $y\ind{t}=0$ for even $t$. At odd rounds the past counts are tied,
  so ERM predicts zero; at even rounds ones are in the majority, so it
  predicts one. Hence ERM makes $T$ mistakes. The better constant makes
  $\lfloor T/2\rfloor$ mistakes, giving regret $\lceil T/2\rceil$.
  The symmetric construction applies when ties favor $f_1$.
\end{proof}

\subsubsection{Online Learning and Estimation}

Similar to Lemma~\ref{lem:square_well_specified} in the setting of statistical learning, the regret for online learning has an
additional interpretation in terms of \emph{estimation} if the
problem is realizable.
We measure this through cumulative estimation errors, which
play a central role in the design of decision making algorithms from
\cref{sec:cb} onward.
\begin{lem}[Regret controls estimation error: square loss]
	\label{lem:well_specified_reg_est}
  Suppose that the features $x\ind{1},\ldots,x\ind{T}$ are generated
  in an arbitrary fashion, but that for all $t$, the variable $y\ind{t}$ is random with mean given by a fixed function $\fstar\in\cF$:
  \[
    \En\brk{y\sups{t}\mid{}\cH\sups{t-1},x\sups{t},\fhat\sups{t}} = \fstar(x\sups{t}).
  \]
  Define the cumulative squared estimation error by
  \begin{equation}
    \label{eq:online_est_sq}
    \EstSq \ldef \sum_{t=1}^T
      (\fhat\sups{t}(x\sups{t})-\fstar(x\sups{t}))^2.
  \end{equation}
  Then for prediction with square loss,
  \[
    \En\brk{\EstSq} \leq \En\brk{\RegOLSq}.
  \]
  \end{lem}
  Notably, this result holds even if the features
  $x\ind{1},\ldots,x\ind{T}$ are generated adversarially, with no
  prior knowledge of the sequence. This is a significant departure from classical estimation
  results in statistics, where estimation of an unknown function is
  typically done over a fixed, known sequence (``design'') $x\sups{1},\ldots,x\sups{T}$, or with respect to an i.i.d. dataset.

The analogous result for log loss relates regret to KL divergence.
\begin{lem}[Regret controls estimation error: log loss]
  \label{lem:well_specified_log_reg_est}
  Suppose that the features $x\ind{1},\ldots,x\ind{T}$ are generated
  in an arbitrary fashion, but that for a fixed $\fstar\in\cF$ and every $t$, the conditional distribution of $y\sups{t}$ given the past $\cH\sups{t-1}$, $x\sups{t}$, and the current predictor $\fhat\sups{t}$ is $\fstar(\cdot\mid x\sups{t})$. Assume the log losses appearing in the regret and comparison below are integrable.
  Define the cumulative KL estimation error by
  \begin{equation}
    \label{eq:online_est_kl}
    \EstKL \ldef \sum_{t=1}^T
      \Dkl{\fstar(\cdot\mid x\sups{t})}{\fhat\sups{t}(\cdot\mid x\sups{t})}.
  \end{equation}
  Then for prediction with log loss, any online predictor satisfies
  \[
    \En\brk{\EstKL} \leq \En\brk{\RegOLLog}.
  \]
\end{lem}
\begin{proof}[\pfref{lem:well_specified_log_reg_est}]
  Comparing with the fixed predictor $\fstar$ in the definition of regret gives
  \[
    \RegOLLog\geq\sum_{t=1}^T
    \log\frac{\fstar(y\sups{t}\mid x\sups{t})}
    {\fhat\sups{t}(y\sups{t}\mid x\sups{t})}.
  \]
  Conditional on $\cH\sups{t-1}$, $x\sups{t}$, and $\fhat\sups{t}$,
  the expected $t$th summand is precisely
  $\Dkl{\fstar(\cdot\mid x\sups{t})}{\fhat\sups{t}(\cdot\mid x\sups{t})}$.
  Taking expectations and summing gives
  $\En\brk{\EstKL}\leq\En\brk{\RegOLLog}$.
\end{proof}

\cref{lem:well_specified_reg_est} and \cref{lem:well_specified_log_reg_est}
are the form in which later sections use online learning. In particular, we design decision
making algorithms from \cref{sec:cb} onward by calling an online learner as a
subroutine, in a modular fashion that depends on the specific choice of online learner only through its cumulative estimation error. In what follows, we give concrete examples of online learning algorithms; subsequent sections will abstract such algorithms away as \emph{estimation oracles} (\cref{sec:online_regression}).\loose

\subsection{Online Learning Algorithms}
\label{sec:online_algs}

We now state our first online learning algorithm, the \emph{Bayes mixture predictor}, and prove its regret bound for conditional density estimation with log loss. This is arguably the simplest result one can imagine in online learning.
\begin{lem}[Regret of the Bayes mixture predictor]
  \label{lem:likelihood_mixture}
  Let $\cF$ be a finite class of conditional distributions. At time $t$, use
  the prediction distribution
  \begin{align*}
    \algo\sups{t}(y\mid x) &= \sum_{f\in\cF} q\sups{t}(f) f(y\mid x),
    &q\sups{t+1}(f) =
    \frac{q\sups{t}(f) f(y\sups{t}\mid x\sups{t})}
    {\algo\sups{t}(y\sups{t}\mid x\sups{t})},
  \end{align*}
  with $q\sups{1}$ uniform on $\cF$.
  For any sequence, this algorithm satisfies
  \begin{align}
    \label{eq:likelihood_mixture_regret}
    \RegOLLog = -\sum_{t=1}^T \log \algo\sups{t}(y\sups{t}\mid x\sups{t})
    - \min_{f\in\cF}\left\{-\sum_{t=1}^T \log f(y\sups{t}\mid x\sups{t})\right\}
    \leq \log|\cF|.
  \end{align}
\end{lem}
\begin{proof}[\pfref{lem:likelihood_mixture}]
  Write $Z_t=\sum_{f\in\cF}\prod_{i=1}^t f(y\sups{i}\mid x\sups{i})$, so
  $Z_0=|\cF|$. 
  By definition,
  $\algo\sups{t}(y\sups{t}\mid x\sups{t})=Z_t/Z_{t-1}$. Hence, for every $f\in\cF$,
  \[
    \prod_{t=1}^T \algo\sups{t}(y\sups{t}\mid x\sups{t})
    = \prod_{t=1}^T \frac{Z_t}{Z_{t-1}}
    = \frac{Z_T}{|\cF|}
    \geq \frac{1}{|\cF|}\prod_{t=1}^T f(y\sups{t}\mid x\sups{t}).
  \]
  Taking negative logarithms and minimizing over $f\in\cF$ proves the claim.
\end{proof}
In the realizable case (conditioning on the past as well as the current features),
\cref{lem:well_specified_log_reg_est} and \eqref{eq:likelihood_mixture_regret} imply
\begin{align}
  \label{eq:likelihood_mixture_kl}
  \En\brk{\EstKL} \leq \En\brk{\RegOLLog} \leq \log|\cF|.
\end{align}
In particular, this KL guarantee allows the features $x\sups{t}$ to be
chosen adaptively and arbitrarily. On the other hand, when the data are \iid and the problem
is realizable, the averaged predictor
$\bar f(\cdot\mid x)=\frac{1}{T}\sum_{t=1}^T\algo\sups{t}(\cdot\mid x)$
satisfies, in view of \cref{prop:online_to_batch},
\[
  \En\brk*{\En_{x\sim\sltmarginal}\brk*{
    \Dkl{\fstar(\cdot\mid x)}{\bar f(\cdot\mid x)}}}
  \leq \frac{\log|\cF|}{T},
\]
where $x$ is a fresh draw independent of the training data. 

  The average predictor $\bar f(\cdot\mid x)$ also enjoys the
$\log|\cF|/T$ expected excess log-loss bound in
\cref{cor:log_loss_misspecified}, without requiring
the problem to be realizable.
\begin{cor}[Excess log loss without realizability]
  \label{cor:log_loss_misspecified}
  Let $\cF$ be a finite, nonempty class of conditional densities and let
  $(x\sups{t},y\sups{t})_{t=1}^T$ be i.i.d. Let $\algo\sups{t}$ be
  the Bayes mixture predictor of \cref{lem:likelihood_mixture} and set
  $\bar f(\cdot\mid x)=T^{-1}\sum_{t=1}^T\algo\sups{t}(\cdot\mid x)$.
  For log loss, assuming the relevant log losses are integrable,
  \[
    \En\brk{\exploss(\bar f)}-\min_{f\in\cF}\exploss(f)
    \leq\frac{\log|\cF|}{T}.
  \]
  The data-generating conditional distribution need not belong to $\cF$.
\end{cor}
\begin{proof}[\pfref{cor:log_loss_misspecified}]
  For a fresh pair $(x,y)$, convexity of $-\log$ gives
  \[
    -\log\bar f(y\mid x)\leq \frac{1}{T}\sum_{t=1}^{T}-\log\algo\sups{t}(y\mid x).
  \]
  Since $\algo\sups{t}$ depends only on the first $t-1$ observations,
  its expected loss on a fresh pair equals its expected loss on
  $(x\sups{t},y\sups{t})$. Fix a minimizer $f_0$ of $\exploss$ over $\cF$.
  The regret bound therefore implies
  \begin{align*}
    \En\brk{\exploss(\bar f)}-\exploss(f_0)
    &\leq \frac{1}{T}\En\brk*{\sum_{t=1}^T
      \log\frac{f_0(y\sups{t}\mid x\sups{t})}
      {\algo\sups{t}(y\sups{t}\mid x\sups{t})}}
    \leq\frac{\log|\cF|}{T}.
  \end{align*}
\end{proof}

The simple strategy in \cref{lem:likelihood_mixture} can be extended to other losses, as we show next.

\paragraph{The exponential weights algorithm}
Arguably the most important online learning algorithm is
\emph{Exponential Weights}, which generalizes the Bayes mixture predictor
from log loss to arbitrary losses. The posterior weights
$q\ind{t}(f)\propto\prod_{i<t}f(y\ind{i}\mid x\ind{i})$ of the Bayes
mixture predictor are replaced by weights that decay exponentially in the
cumulative loss of each $f$. At each time $t$, the algorithm computes a distribution
$q\ind{t}\in\Delta(\cF)$ via
\begin{align}
    \label{eq:expweightsupdate}
    q\sups{t}(f) \propto \exp\left\{
    -\eta \sum_{i=1}^{t-1} \loss(f(x\sups{i}), y\sups{i})\right\},
\end{align}
where $\eta>0$ is a learning rate. Based on $q\ind{t}$, the algorithm forms the
prediction $\fhat\ind{t}$ as specified below; for reasons explained in a moment, we give two variants of the method.

\smallskip
\noindent
\begin{minipage}{.499\textwidth}
\begin{algbox}[frametitlefont=\sffamily\bfseries\small,frametitle={\algtitle{alg:expweights_avg}{Exp.\ Weights (averaged)}}]
\begin{algorithmic}
    \Require Learning rate $\eta>0$; finite class $\cF$.
\For{$t=1,\ldots,T$}
\State Compute $q\sups{t}$ in \eqref{eq:expweightsupdate}.
\State Let $\algo\sups{t} = \En_{f\sim q\sups{t}}\brk{f}$.
\State Observe $(x\sups{t}, y\sups{t})$, incur $\loss(\algo\sups{t}(x\sups{t}),y\sups{t})$.
\EndFor{}
\end{algorithmic}
\end{algbox}
\end{minipage}%
\hfill
\begin{minipage}{.499\textwidth}
\begin{algbox}[frametitlefont=\sffamily\bfseries\small,frametitle={\algtitle{alg:expweights_rand}{Exp.\ Weights (randomized)}}]
\begin{algorithmic}
    \Require Learning rate $\eta>0$; finite class $\cF$.
\For{$t=1,\ldots,T$}
\State Compute $q\sups{t}$ in \eqref{eq:expweightsupdate}.
\State Sample $\algo\sups{t}\sim q\sups{t}$.
\State Observe $(x\sups{t}, y\sups{t})$, incur  $\loss(\algo\sups{t}(x\sups{t}),y\sups{t})$.
\EndFor{}
\end{algorithmic}
\end{algbox}
\end{minipage}
The only difference between these two forms is whether we compute
the prediction $\fhat\ind{t}$ from $q\ind{t}$ via
\begin{align}
\label{eq:randomized_vs_averaged}
    \algo\sups{t} = \En_{f\sim q\sups{t}}\brk{f},~~~~~ \text{or}~~~~~ \algo\sups{t}\sim q\sups{t}.
\end{align}
\smallskip
The latter can be applied to any bounded loss functions, while the
former leads to faster rates for specific losses such as the square
loss and log loss, but is only applicable when $\cYpred$ is convex. Note
that the averaged version is inherently improper, while the second is
proper, yet randomized. From the point of view of regret, the key
difference between these two versions is the placement of
``$\En_{f\sim q\sups{t}}$'': For the averaged version it is inside the
loss function, and for the randomized version it is outside (see
\eqref{eq:regret_randomized}). The averaged version can therefore take
advantage of the structure of the loss function, such as strong
convexity, leading to faster rates. For log loss and $\eta=1$, the averaged variant is exactly the Bayes mixture predictor of \cref{lem:likelihood_mixture}.

The following result shows that
Exponential Weights leads to regret bounds for online learning, with
rates that parallel those in \pref{prop:iid_finite_class}.
\begin{prop}[Exponential weights for finite classes]
\label{prop:online_bounds}
        For any finite class $\cF$, the Exponential Weights algorithm (with appropriate choice of $\eta$) satisfies
\begin{align}
    \frac{1}{T}\RegOL \lesssim \compl(\cF, T)
\end{align}
for any sequence, where:
\begin{enumerate}
  \item For arbitrary bounded losses (including classification),
    $\compl(\cF,T)=\sqrt{\frac{\log|\cF|}{T}}$. This is achieved by
    the randomized variant.
    \item For regression with the square loss (under boundedness of  $\cY$) and conditional density
      estimation with the log loss,
    $\compl(\cF,T)=\frac{\log|\cF|}{T}$. This is achieved by the
    averaged variant.
  \end{enumerate}
\end{prop}
The losses are those introduced in \cref{sec:sl}. The bound
holds for every sequence: deterministically for the averaged variant, and
in expectation over the algorithm's own randomization for the randomized
variant, in the sense of \eqref{eq:regret_randomized}. As in
\cref{prop:iid_finite_class}, the bound for bounded losses hides the
range of the loss, which enters linearly; the fast rates take
$\cY=\cYpred=[0,1]$.

\begin{rem}[Beyond finite classes]
          As in statistical learning, there are (sequential)
          complexity measures for $\cF$ that can be used to generalize
          the regret bounds in \pref{prop:online_bounds} to infinite
          classes. In general, the optimal regret for a class $\cF$
          will reflect the statistical capacity of the class
          \citep{StatNotes2012}.
        \end{rem}

\begin{rem}[Online linear optimization]
\label{rem:online_linear_opt}
  For the slow rate in \pref{prop:online_bounds}, the
  nature of the loss and the dependence on the function $f$ is
  immaterial for the proof. The guarantee can be stated in a more
  abstract form that depends only on the vector of losses for functions in
  $\cF$ as follows. Let $\abs{\cF}=N$. For timestep $t$, define $\vloss\sups{t}\subs{f} = \loss(f(x\sups{t}),y\sups{t})$ and $\vloss\sups{t} = (\vloss\sups{t}\subs{f_1},\ldots,\vloss\sups{t}\subs{f_N})\in\bbR^{N}$ for $\cF=\{f_1,\ldots,f_N\}$. For a randomized strategy $q\sups{t}\in\Delta(\brk{N})$, expected loss of the learner can be written as
    $$\En_{\algo\sups{t}\sim q\sups{t}} \brk{\loss(\algo\sups{t}(x\sups{t}),y\sups{t})} = \inner{q\sups{t}, \vloss\sups{t}},$$
and the expected regret can be written as
\begin{align}
  \label{eq:vector_regret}
        \RegOL = \sum_{t=1}^T \inner{q\sups{t}, \vloss\sups{t}} - \min_{j\in\{1,\ldots,N\}} \sum_{t=1}^T \inner{\boldsymbol{e}_j, \vloss\sups{t}}
    \end{align}
    where $\boldsymbol{e}_j\in\bbR^{N}$ is the standard basis vector with $1$ in
    $j$th position. In its most general form, the exponential weights
    algorithm gives bounds on the regret in \pref{eq:vector_regret}
    for any sequence of vectors
    $\vloss\ind{1},\ldots,\vloss\ind{T}$, and the update takes the form
    \[
    q\sups{t}(k) \propto \exp\left\{
    -\eta \sum_{i=1}^{t-1} \vloss\ind{i}(k)\right\}.
    \]
    
    This formulation can be viewed as a special case of a problem
    known as \emph{online linear optimization}, and the exponential
    weights method can be viewed as an instance of an algorithm known
    as \emph{mirror
    descent}. 
  \end{rem}
  
To summarize, exponential weights attains, for arbitrary
sequences, the same rates that ERM attains for i.i.d.\ data, and these
rates cannot be improved in general for finite classes. ERM itself,
applied to the data observed so far, does not have low regret in general
(\cref{prop:erm_failure_online}).
For the purposes of this monograph, this is all we need from online
learning over finite classes.

\subsection{Additional Background: Random Variables and Tail Bounds}
\label{sec:tail_bounds}

The main tools for proving \pref{prop:iid_finite_class} and similar generalization
bounds are \emph{concentration inequalities} (or,
tail bounds) for random variables, which we review here with further background in \cref{app:probabilistic}.
\begin{definition}[Sub-Gaussian random variable]
\label{def:subGaussian}
  A random variable $Z$ is \subg
  with variance factor (or variance proxy) $\sigma^2$ if
	$$\forall \eta\in\reals,~~~~~\En\brk*{e^{\eta (Z-\En\brk{Z})}} \leq e^{\sigma^2 \eta^2/2}.$$
      \end{definition}
Note that if $Z\sim\cN(0,\sigma^2)$ is \emph{Gaussian} with variance $\sigma^2$, then it
is \subgaussian with variance proxy $\sigma^2$. In this sense,
sub-Gaussian random variables generalize the tail behavior of
Gaussians. A standard application of Chernoff method yields the following result.
\begin{lem}[Sub-Gaussian tail bound]
  If $Z_1,\ldots,Z_T$ are \iid sub-Gaussian random variables with variance proxy
  $\sigma^2$, then
  \begin{align}
    \label{eq:subgaussian_tail}
    \Prob{\frac{1}{T}\sum_{i=1}^T Z_i - \En\brk*{Z}\geq u}\leq \exp\left\{ -\frac{Tu^2}{2\sigma^2} \right\}.
  \end{align}
\end{lem}
Applying this result with $Z$ and $-Z$ and taking a union bound yields the following
two-sided guarantee:
\begin{align}
  \Prob{\abs*{\frac{1}{T}\sum_{i=1}^T Z_i - \En\brk*{Z}}\geq u}\leq
  2\exp\left\{ -\frac{Tu^2}{2\sigma^2}\right\}.\label{eq:hoeffding_two_sided}
\end{align}
Setting the right-hand side of \eqref{eq:hoeffding_two_sided} equal to $\delta$ and solving for $u$, we find that for any $\delta\in(0,1)$, with probability at least $1-\delta$,
\begin{align}
  \abs*{\frac{1}{T}\sum_{i=1}^T Z_i - \En\brk*{Z}}
  \leq{} \sqrt{\frac{2\sigma^2\log(2/\delta)}{T}}\label{eq:hoeffding_delta}.
\end{align}

\begin{rem}[Union bound]
  Observe that the factor $2$ under the logarithm in \cref{eq:hoeffding_delta} is
  the result of applying union bound to \eqref{eq:subgaussian_tail}. 
  Throughout this monograph, we frequently apply the union bound to
  multiple---say $N$---high probability events involving sub-Gaussian
  random variables. In this case, the union
  bound results in terms of the form $\log (N/\delta)$. The mild logarithmic dependence is
  due to the sub-Gaussian tail behavior of the averages.
\end{rem}

The following classical result states that any bounded random variable is \subg.

\begin{lem}[Hoeffding's inequality]
\label{lem:hoeffding}
	Any random variable $Z$ taking values in $[a,b]$ is \subg with variance proxy $(b-a)^2/4$, i.e.
	\begin{align}
		\label{eq:hoeffding_mgf}
		\forall \eta\in\reals,~~~~~\ln \En\brk*{\exp\{\eta (Z-\En\brk*{Z})\}} \leq \frac{\eta^2(b-a)^2}{8}.
	\end{align}
	As a consequence,
	for \iid random variables $Z_1,\ldots,Z_T$ taking values in $[a,b]$ almost surely, with probability at least $1-\delta$,
	\begin{align}
		\frac{1}{T}\sum_{i=1}^T Z_i - \En\brk*{Z}\leq (b-a)\sqrt{\frac{\log(1/\delta)}{2T}}.
	\end{align}
\end{lem}

      \begin{lem}[Bernstein's inequality]
        \label{lem:bernstein}
    Let $Z_1,\ldots,Z_T,Z$ be i.i.d. with variance
    $v^2$ and range $|Z-\En\brk{Z}|\leq B$ almost
    surely. Then with probability at least $1-\delta$,
    \begin{align}
      \label{eq:bernstein_cond_tail_sum_normalized_highprob} 
      \frac{1}{T}\sum_{i=1}^T Z_i - \En\brk{Z} \leq v \sqrt{\frac{2\log (1/\delta)}{T}} + \frac{B\log (1/\delta)}{3T}.
    \end{align}
  \end{lem}
A proof of \cref{lem:bernstein} is given in \cref{app:bernstein_proof}. We refer the reader to \cref{app:probabilistic} for further background on tail bounds. In particular, \cref{lem:martingale_chernoff} in \cref{app:probabilistic} is a common starting point for many inequalities in this monograph, including Hoeffding. 

\subsection{Bibliographic Remarks and Further Reading}\label{sec:background_bib}

\paragraph{Takeaways}
This section reviewed the two non-interactive learning problems on which
the rest of the monograph builds. Over a finite class, empirical risk
minimization attains the slow rate $\sqrt{\log\abs{\cF}/T}$ for any
bounded loss and the fast rate $\log\abs{\cF}/T$ under realizability
for square loss and for maximum likelihood in Hellinger distance
(\cref{prop:iid_finite_class} and \cref{prop:fast_rate_mle_hellinger}); exponential
weights attains the same two rates for arbitrary sequences
(\cref{prop:online_bounds}). Later sections use these estimation guarantees as a starting point to design interactive decision making algorithms. Notably, the guarantees we will study in later sections do not depend on the finite-class assumption: they depend in a modular fashion on the estimation error of the online learner, so any estimator with a comparable guarantee can be substituted.

\paragraph{Bibliographic remarks}
The statistical learning framework and empirical risk minimization go
back to \citet{vapnik1971uniform,vapnik1998statistical}; the union-bound
argument of \cref{lem:erm_uniform_dev} is standard. That ERM cannot
attain the fast rate in \cref{prop:iid_finite_class} without
realizability, whereas aggregation can, is due to
\citet{tsybakov2003optimal,juditsky2008learning,lecue2009aggregation}.
Hellinger rates for maximum likelihood are classical
\citep{wong1995probability,birge1998minimum,Sara00}; the proof of
\cref{prop:fast_rate_mle_hellinger} follows \citet{zhang2006from}.
Prediction of individual sequences originates with \citet{Hannan57} and
\citet{cover1967behavior}, the latter containing the observation behind
\cref{prop:erm_failure_online}.
The Bayes mixture predictor of \cref{lem:likelihood_mixture} takes its
name from universal coding \citep{shtarkov1987universal}, and
\cref{cor:log_loss_misspecified} is the progressive-mixture
method of
\citet{yang1999information}. Exponential weights
(\cref{alg:expweights_avg} and \cref{alg:expweights_rand}) and the slow
rate in \cref{prop:online_bounds} are due to
\citet{littlestone1994weighted,vovk1990aggregating,freund1997decision}; the fast rates for square and log
loss, and mixability more generally (\cref{rem:square}), are due to
\citet{vovk1990aggregating,vovk1995game}. Online-to-batch conversion
(\cref{prop:online_to_batch}) in high-probability form is due to
\citet{cesa2004generalization}. 

\paragraph{Further reading}
Textbook treatments of statistical learning include
\citet{vapnik1998statistical,DevroyeLugosi96,shalev2014understanding};
for fast rates beyond finite classes see
\citet{bartlett2005local,koltchinskii2006local}.
\citet{cesa2006prediction} is the standard reference for online
learning, and \citet{hazan2016introduction,orabona2019modern} develop the
online convex optimization viewpoint of \cref{rem:online_linear_opt}.
Regret for infinite classes is governed by sequential complexities developted in
\citet{StatNotes2012,rakhlin2015sequential}. 

\subsection{Deferred Proofs: Statistical Learning}

Using Hoeffding's inequality, we now prove Part
1 (the slow rate) of \pref{prop:iid_finite_class}. 
\begin{lem}[Slow rate for ERM (\pref{prop:iid_finite_class}, Part 1)]
	\label{lem:erm_uniform_dev}
Let $\cF\subseteq\{f:\cX\to\cYpred\}$ be finite, and assume that $\loss(f(x),y) \in [0,1]$ for any $f\in\cF,x\in\cX,y\in\cY$. Then with probability at least $1-\delta$, ERM satisfies    
    $$\exploss(\algo)-\min_{f\in\cF} \exploss(f) \leq \sqrt{\frac{2\log(2|\cF|/\delta)}{T}}.$$
\end{lem}
\begin{proof}[\pfref{lem:erm_uniform_dev}]
For any $f\in\cF$, we can write
\begin{align}
  \label{eq:erm_three_term_decomp}
  \exploss(\algo)-\exploss(f) = \left[
      \exploss(\algo)-\emploss(\algo) \right] +  \left[
      \emploss(\algo) - \emploss(f) \right]+ \left[\emploss(f) -
      \exploss(f) \right].
\end{align}
Further, 
    by Hoeffding's inequality,  
    \begin{align}
      \label{eq:erm_hoeffding}
      \text{for any}~ f\in \cF,~\text{with prob. at least}~1-\delta,~~~ \left|\exploss(f)-\emploss(f) \right| \leq \sqrt{\frac{\log(2/\delta)}{2T}}.
    \end{align}
    By union bound, with probability at least $1-|\cF|\delta$, inequality \eqref{eq:erm_hoeffding} holds for any $f\in \cF$ simultaneously; thus, for any $\delta\in(0,1)$, 
    \begin{align}
      \label{eq:erm_hoeffding_union}
      \text{with prob. at least}~1-\delta,~ \text{for any}~ f\in \cF,~~~ \left|\exploss(f)-\emploss(f) \right| \leq \sqrt{\frac{\log(2|\cF|/\delta)}{2T}}.
    \end{align}
    The middle term in \eqref{eq:erm_three_term_decomp} is nonpositive, and the proof is completed by applying \eqref{eq:erm_hoeffding_union} to the first and third terms.\footnote{It is important to understand the difference between \cref{eq:erm_hoeffding} and \cref{eq:erm_hoeffding_union}. In the former, the function $f$ cannot be chosen to in a data-dependent manner; in the latter, the high probability event holds for all $f\in\cF$, and thus also holds for $\fhat$.}
\end{proof}
	  To deduce the in-expectation bound of
          \cref{prop:iid_finite_class} from the high-probability tail
          bound of \cref{lem:erm_uniform_dev}, the standard technique of
          \emph{integrating out the tail} can be employed. Specifically,
          for a nonnegative random variable $U$, it
          holds that $\En\brk*{U}\leq \tau + \int_{\tau}^\infty
          \Prob{U\geq z}dz$ for all $\tau>0$; choosing $\tau\propto T^{-1/2}$
          concludes the proof.\footnote{For the in-expectation bound, only the first term in \eqref{eq:erm_three_term_decomp} requires this argument: for a fixed $f$, the third term has mean zero.}
	  
\begin{proof}[\pfref{prop:iid_finite_class}, Part 2]
  By rescaling the common bounded interval, assume that $Y$ and every
  prediction $f(x)$ lie in $[0,1]$. For each $f\in\cF$, write
  \[
    Z_i(f)=(f(x\ind{i})-y\ind{i})^2-(\fstar(x\ind{i})-y\ind{i})^2,
    \qquad R(f)=\En\brk*{(f(x)-\fstar(x))^2}.
  \]
  Realizability gives $\En\brk{Z_i(f)}=R(f)=\cE(f)$. Factoring the
  difference of squares also gives
  \[
    Z_i(f)^2\leq4(f(x\ind{i})-\fstar(x\ind{i}))^2,
    \qquad |Z_i(f)-\En\brk{Z_i(f)}|\leq2.
  \]
  Thus $\operatorname{Var}(Z_i(f))\leq4R(f)$. Apply
  \cref{lem:bernstein} to $-Z_i(f)$ and take a union bound over $\cF$.
  With $u=\log(|\cF|/\delta)$, with probability at least $1-\delta$,
  simultaneously for all $f\in\cF$,
  \[
    R(f)\leq\frac1T\sum_{i=1}^T Z_i(f)
       +\sqrt{\frac{8R(f)u}{T}}+\frac{2u}{3T}
    \leq\frac1T\sum_{i=1}^T Z_i(f)+\frac{R(f)}2+\frac{14u}{3T}.
  \]
  The last inequality uses $\sqrt{8rv}\leq r/2+4v$ for $r,v\geq0$.
  Since the empirical minimizer satisfies $\sum_{i=1}^{T} Z_i(\fhat)\leq0$, it follows that
  \begin{equation}
    \label{eq:excess_loss_bound_fast}
    \cE(\fhat)\leq\frac{28\log(|\cF|/\delta)}{3T}.
  \end{equation}
  Integrating this tail bound yields
  $\En\brk{\cE(\fhat)}\leq28(\log|\cF|+1)/(3T)$.
  For $|\cF|\geq2$ this is $O(\log|\cF|/T)$; for $|\cF|=1$
  the excess risk is zero. Rescaling back multiplies the bound by the
  squared length of the original interval.
\end{proof}
We finish this section with a proof of the Part 3 (the fast rate for conditional density estimation) of \cref{prop:iid_finite_class}.

\begin{prop}[Fast rate for MLE (\pref{prop:iid_finite_class}, Part 3)]
  \label{prop:fast_rate_mle_hellinger}
  Let $\cF$ be a finite class of conditional densities $f:\cX\to\Delta(\cY)$, and $\sltmarginal$ be a marginal distribution on $\cX$. Suppose $\sltconddist=\fstar$ for some $\fstar\in\cF$, and the data $x\sups{1},\ldots,x\sups{T}$ are i.i.d. from $\sltmarginal$ and $y\sups{t}\sim \fstar(\cdot\mid{} x\sups{t})$, $t=1,\ldots,T$, independently. 
  Let the MLE be
  \[
    \fhat = \argmin_{f\in\cF} \frac{1}{T}\sum_{t=1}^T -\log f(y\sups{t}\mid{} x\sups{t}).
  \]
  Then, for any $\delta\in(0,1)$, with probability at least $1-\delta$,
  \[
    \Dhels{\fhat}{\fstar} \leq \frac{2\log(|\cF|/\delta)}{T}
  \]
  where $\Dhels{f}{g} = \En_{x\sim \sltmarginal}\brk*{\Dhels{f(\cdot\mid{}x)}{g(\cdot\mid{}x)}}$ is the integrated squared Hellinger distance.
\end{prop}
\begin{proof}[\pfref{prop:fast_rate_mle_hellinger}]
By definition of MLE, 
\begin{align}
  \label{eq:mle_log_ratio_geq_0}
  \sum_{t=1}^T \log \fhat(y\sups{t}\mid{} x\sups{t})
  \geq \sum_{t=1}^T \log \fstar(y\sups{t}\mid{} x\sups{t}).
\end{align}
Fix $f\in\cF$ and apply \cref{lem:martingale_chernoff} to independent random variables 
$$X_t = \frac{1}{2}\left(\log f(y\sups{t}\mid{} x\sups{t}) - \log \fstar(y\sups{t}\mid{} x\sups{t})\right).$$ This yields that with probability at least $1-\delta$,
\[
  \sum_{t=1}^T \frac{1}{2}\left(\log f(y\sups{t}\mid{} x\sups{t}) - \log \fstar(y\sups{t}\mid{} x\sups{t})\right) \leq \sum_{t=1}^T \log \En_{(x\sups{t},y\sups{t})}\brk*{\exp\left\{\frac{1}{2}\log\frac{f(y\sups{t}\mid{} x\sups{t})}{\fstar(y\sups{t}\mid{} x\sups{t})}\right\}} + \log(1/\delta).
\]
Rearranging, taking a union bound over $f\in\cF$, applying to $\fhat$ and using \cref{eq:mle_log_ratio_geq_0}, we have that with probability at least $1-\delta$,
\begin{align}
  -T \log \En_{(x,y)}\brk*{\sqrt{\frac{\fhat(y\mid{} x)}{\fstar(y\mid{} x)}}} \leq \log(|\cF|/\delta).
\end{align}
Using the inequality $-\log x\geq 1-x$, for the natural logarithm, and the fact that, conditionally on $x$, $\En_{y\sim\fstar(\cdot\mid{} x)} \brk*{1-\sqrt{\frac{f(y\mid{} x)}{\fstar(y\mid{} x)}}} = \frac{1}{2}\Dhels{f(\cdot\mid{} x)}{\fstar(\cdot\mid{} x)}$, we conclude the proof.
\end{proof}

\subsection{Deferred Proofs: Online Learning}

We now prove \cref{prop:online_bounds}. Since we are
not placing any assumptions on the data generating process, we cannot hope
to control the algorithm's loss at any particular time $t$, but only
cumulatively. It is then natural to employ amortized analysis with a
potential function.

In more detail, the proof of \pref{prop:online_bounds} relies on
several steps, common to standard analyses of online learning: $(i)$
define a potential function, $(ii)$ relate the increase in potential
at each time step to the loss of the algorithm, $(iii)$ relate cumulative loss of any expert $f\in\cF$ to the final potential. For the Exponential Weights Algorithm, the proof relies on the following potential for time $t$, parameterized by $\eta>0$:
\begin{align}
    \Phi\sups{t}\subs{\eta} = -\log \sum_{f\in\cF} \exp\left\{ - \eta \sum_{i=1}^t \loss(f(x\sups{i}), y\sups{i})\right\}.
\end{align}
The choice of this potential is rather opaque, and a full explanation
of its origin is beyond the scope of this monograph, but we mention in
passing that there are principled ways of coming up with potentials in
general online learning problems.

\begin{proof}[\pfref{prop:online_bounds}]
    
  We first prove the second statement, focusing on conditional density
  with the logarithmic loss; for the square loss, see
  \cref{rem:square} below.

\noindent\emph{Proof for Part 2: Log loss.}
  Recall that for each $x$, $f(x)$ is a distribution over $\cY$, and $\logloss(f(x), y) = -\log f(y\mid{} x)$ where we abuse the notation and write $f(x)$ and $f(\cdot\mid{} x)$ interchangeably.
        With $\eta=1$, the averaged variant of exponential weights satisfies
    \begin{align}
		\label{eq:exp_weights_log_loss}
        \algo\sups{t}(y\sups{t}|x\sups{t}) = \sum_{f\in\cF} q\sups{t}(f) f(y\sups{t}|x\sups{t}) 
        = \sum_{f\in\cF} f(y\sups{t}|x\sups{t}) \frac{
         \exp
        \left\{-\sum_{i=1}^{t-1} \logloss(f(x\sups{i}), y\sups{i})
        \right\}
        }{
        \sum_{f\in\cF}\exp
        \left\{
        -\sum_{i=1}^{t-1} \logloss(f(x\sups{i}), y\sups{i})
        \right\}
        } ,
    \end{align}
    and thus
    \begin{align}
		\label{eq:change_in_potential}
      \logloss(\fhat(x\ind{t}),y\ind{t}) = -\log \algo\sups{t}(y\sups{t}|x\sups{t}) = \Phi\sups{t}\subs{1} - \Phi\sups{t-1}\subs{1}.
    \end{align}
    Hence, by telescoping
    $$\sum_{t=1}^T \logloss(\fhat(x\ind{t}),y\ind{t}) = \Phi\sups{T}\subs{1} - \Phi\sups{0}\subs{1}.$$
    Finally, observe that $\Phi\sups{0}\subs{1}=-\log|\cF|$ and, since
    $-\log$ is monotonically decreasing, we have 
	\begin{align}
		\label{eq:potential_to_benchmark}
		\Phi\sups{T}\subs{1} \leq -\log \exp \left\{-\sum_{i=1}^{T} \logloss(g(x\sups{i}), y\sups{i})
        \right\} = \sum_{i=1}^{T} \logloss(g(x\sups{i}), y\sups{i}),
	\end{align}
    for any $g\in\cF$. This establishes the result for conditional density estimation with the log loss. As already discussed, the above proof follows the strategy: the loss on each round related to change in potential \eqref{eq:change_in_potential}, and the cumulative loss of any expert is related to the final potential \eqref{eq:potential_to_benchmark}. We now aim to replicate these steps for arbitrary bounded losses.

~\\
    \noindent\emph{Proof for Part 1: Generic loss.} To prove this
    result, we build on the log loss result above. First, observe that
    without loss of generality,
	we may assume that $\loss(f(x),y)\in[0,1]$ for all $f\in\cF$ and $(x,y)$, as we can
	always re-scale the problem. The randomized variant of exponential weights \eqref{eq:randomized_vs_averaged} satisfies
	   \begin{align}
	     \En_{\fhat\ind{t}\sim{}q\ind{t}}\brk{\loss(\algo\sups{t}(x\sups{t}), y\sups{t})} = \sum_{f\in\cF} \loss(f(x\sups{t}), y\sups{t}) \frac{
	         \exp
	        \left\{-\eta\sum_{i=1}^{t-1} \loss(f(x\sups{i}), y\sups{i})
	        \right\}
	        }{
	        \sum_{f\in\cF}\exp
	        \left\{
	        -\eta\sum_{i=1}^{t-1} \loss(f(x\sups{i}), y\sups{i})
	        \right\}
	        }.
	\end{align}
	Hoeffding's inequality (\cref{lem:hoeffding}) implies that
	    \begin{align}
			\label{eq:proof_of_exp_weights_hoeffding}
	      \eta \En_{\fhat\ind{t}\sim{}q\ind{t}}\brk{
	      \loss(\algo\sups{t}(x\sups{t}), y\sups{t})} \leq -\log \sum_{f\in\cF}  \frac{
	        \exp\{-\eta\loss(f(x\sups{t}), y\sups{t})\}\exp
	        \left\{-\eta\sum_{i=1}^{t-1} \loss(f(x\sups{i}), y\sups{i})
	        \right\}
	        }{
	        \sum_{f\in\cF}\exp
	        \left\{
	        -\eta\sum_{i=1}^{t-1} \loss(f(x\sups{i}), y\sups{i})
	        \right\}
	        } 
	        + \frac{\eta^2}{8}.
	    \end{align}
	    Note that the right-hand side of this inequality is simply
	$$\Phi\sups{t}\subs{\eta}-\Phi\sups{t-1}\subs{\eta}+\frac{\eta^2}{8},$$
	establishing the analogue of \eqref{eq:change_in_potential}. Summing over $t$, this gives
	\begin{align}
	       \eta\sum_{t=1}^T \En_{\fhat\ind{t}\sim{}q\ind{t}}\brk{\loss(\algo\sups{t}(x\sups{t}), y\sups{t})} \leq \Phi\sups{T}\subs{\eta}-\Phi\sups{0}\subs{\eta}+\frac{T\eta^2}{8}.
	\end{align}
	As in the first part, for any $g\in\cF$, we can upper
	bound $$\Phi\sups{T}\subs{\eta}\leq \eta\sum_{t=1}^{T}
	\loss(g(x\sups{t}), y\sups{t}),$$ while $\Phi\sups{0}\subs{\eta}
	= - \log|\cF|$. Hence, we have that for any $g\in\cF$, 
	\[\sum_{t=1}^T
	  \En_{\fhat\ind{t}\sim{}q\ind{t}}\brk{\loss(\algo\sups{t}(x\sups{t}),
	    y\sups{t})} - \loss(g(x\sups{t}),y\sups{t}) \leq
	  \frac{T\eta}{8} + \frac{\log\abs{\cF}}{\eta}.
	\]
	With $\eta=\sqrt{\frac{8\log |\cF|}{T}}$, we conclude that
	\begin{align}
	    \sum_{t=1}^T \En_{\fhat\ind{t}\sim{}q\ind{t}}\brk{\loss(\algo\sups{t}(x\sups{t}), y\sups{t})} - \loss(g(x\sups{t}),y\sups{t}) \leq \sqrt{\frac{T\log|\cF|}{2}}.
	\end{align}
	
\end{proof}

Observe that Hoeffding's inequality was all that was needed for
        \cref{lem:erm_uniform_dev}. Curiously enough, it was also
        the only nontrivial step in the proof of
        \pref{prop:online_bounds}. In fact, the connection between
        probabilistic inequalities and online learning regret
        inequalities (that hold for arbitrary sequences) runs much
        deeper \citep{rakhlin2017equivalence,foster2018online}.

\begin{rem}[Mixable losses]
\label{rem:square}
    We did not provide a proof of \pref{prop:online_bounds} for square
    loss. It is tempting to reduce square loss regression to density
    estimation by taking the conditional density to be a Gaussian distribution. Indeed, the log loss
    of a distribution with density proportional to
    $\exp\{-(\fhat\sups{t}(x\sups{t})-y\sups{t})^2\}$ is, up to
    constants, the desired square loss. However, the mixture in
    \eqref{eq:exp_weights_log_loss} does not immediately lead to a
    prediction strategy for the square loss, as the expectation
    appears in the wrong location. This issue
    is fixed by a notion known as \emph{mixability}.

    We say that a loss $\ell$ is \emph{mixable} with parameter $\eta$ if
    there exists a constant $c>0$ such that the following holds: for any $x$ and a distribution $q\in\Delta(\cF)$, there exists a prediction $\fhat(x)\in\cYpred$ such that for all $y\in\cY$, 
	\begin{align}
		\label{eq:mixability}
		\loss(\fhat(x),y)\leq -\frac{c}{\eta}\log\left(\sum_{f\in\cF} q(f) \exp\{-\eta \loss(f(x),y)\}\right).
	\end{align}
If loss is mixable, then given the exponential weights
        distribution $q\ind{t}$, the best prediction $\widehat{y}\sups{t}=\fhat\sups{t}(x\sups{t})$ can be written (by bringing the right-hand side of \eqref{eq:mixability} to the left side) as an optimization problem
	\begin{align}
          \label{eq:mixability_minimax}
		\argmin_{\widehat{y}\sups{t}\in\cYpred} \max_{y\sups{t}\in\cY} \brk*{ \loss(\widehat{y}\sups{t},y\sups{t})+\frac{c}{\eta}\log\left(\sum_{f\in\cF} q\sups{t}(f) \exp\{-\eta \loss(f(x\sups{t}),y\sups{t})\}\right)}
	\end{align}
	which is equivalent to
	\begin{align}
		\argmin_{\widehat{y}\sups{t}\in\cYpred} \max_{y\sups{t}\in\cY} \brk*{ \loss(\widehat{y}\sups{t},y\sups{t})+\frac{c}{\eta}\log\left(\sum_{f\in\cF} \exp\{-\eta \sum_{i=1}^t \loss(f(x\sups{i}),y\sups{i})\}\right)}
	\end{align}
	once we remove the normalization factor. With this choice, mixability allows one to replicate the proof
        of \cref{prop:online_bounds} for the logarithmic loss, with
        the only difference being that \eqref{eq:exp_weights_log_loss}
        (after applying $-\log$ to both sides) becomes an
        inequality. It can be verified that square loss is mixable
        with parameter $\eta=1/2$ and $c=1$ when $\cY=\cYpred=[-1,+1]$ \citep{vovk1990aggregating}, leading to the desired fast rate for
square loss in \pref{prop:online_bounds}. The idea of translating the
English statement ``there exists a strategy such that for any
outcome...'' into a min-max inequality recurs in \cref{sec:structured}, where the
Decision-Estimation Coefficient is defined through a min-max problem of this type.
\end{rem}

\normalsize

\section{Multi-Armed Bandits: The Exploration--Exploitation Tradeoff}
\label{sec:mab}

\cref{sec:background} concerned passive learning: data is given,
and the learner's only task is to generalize. This section takes the
first step along the interactivity axis of \cref{fig:axes}. In the
\emph{multi-armed bandit} problem, the learner's decisions determine
which rewards are observed, so that collecting informative data and
earning high reward become competing goals. This is the
\emph{exploration--exploitation tradeoff}, and the multi-armed bandit isolates
it in the simplest possible form, with a finite set of decisions and no
additional structure to generalize across. The section establishes the optimal
$\sqrt{AT}$ regret rate for this problem and, more importantly,
introduces an algorithm design principle that recurs throughout the monograph: optimism in the face of uncertainty.

\paragraph{Organization}
\cref{sec:mab_setting} formalizes the multi-armed bandit problem and the regret criterion.
\cref{sec:need_for_exploration} shows that the greedy strategy inherited
from statistical learning fails, which motivates deliberate exploration.
\cref{sec:eps_greedy} analyzes the simplest remedy, \egreedy, which
explores uniformly at random. \cref{sec:ucb_bandits} introduces the
principle of optimism and the \ucb algorithm, which explores adaptively
and attains the optimal regret; its analysis is the template for a variety of 
optimistic algorithms throughout the monograph. Two optional extensions follow: \cref{sec:posterior} (marked
$\star$) treats the Bayesian setting and posterior sampling, whose
analysis is revisited in \cref{sec:structured,sec:general_dm};
\cref{sec:adversarial_bandits} (marked $\star$) treats adversarial (non-stochastic)
rewards and the \expthree algorithm.

\subsection{The Multi-Armed Bandit Problem}
\label{sec:mab_setting}
The multi-armed bandit problem is defined by the following
protocol.
\begin{algbox}[frametitle={\algprotocol{Multi-Armed Bandit Protocol}}]
  \begin{algorithmic}
\For{$t=1,\ldots,T$}
\State Select decision $\pi\sups{t}\in\Pi\ldef\{1,\ldots,A\}$.
\State Observe reward $r\sups{t}\in\bbR$.
\EndFor{}
\end{algorithmic}
\end{algbox}
The process proceeds in $T$ rounds. At each round
$t\in\brk{T}$, the learning agent selects a discrete
\emph{decision}\footnote{In the literature on bandits, decisions are
  often referred to as \emph{actions}. We use these terms
  interchangeably throughout this section.} $\pi\ind{t}\in\Pi=\crl{1,\ldots,A}$ using the data
\[
  \cH\sups{t-1} = \left\{(\pi\sups{1}, r\sups{1}),\ldots,(\pi\sups{t-1},
    r\sups{t-1}) \right\}
\]
collected so far; we refer to $\Pi$ as the \emph{decision space}
or action space, with $A\in\bbN$ denoting the size of the space. We allow the learner to randomize the decision at step
$t$ according to a distribution
$p\ind{t}=p\ind{t}(\cdot\mid{}\hist\ind{t-1})$, sampling
$\pi\ind{t}\sim{}p\ind{t}$. Based
on the decision $\pi\ind{t}$, the learner receives a reward
$r\ind{t}$, and their goal is to maximize the cumulative reward across
all $T$ rounds. As an example, one might consider an application in which the learner is a doctor (or
  personalized medical assistant) who aims to select a treatment (the
  decision) in order to make a patient feel better (maximize reward); see \pref{fig:mab}.
\begin{figure}
    \centering
    \includegraphics[width=0.8\textwidth]{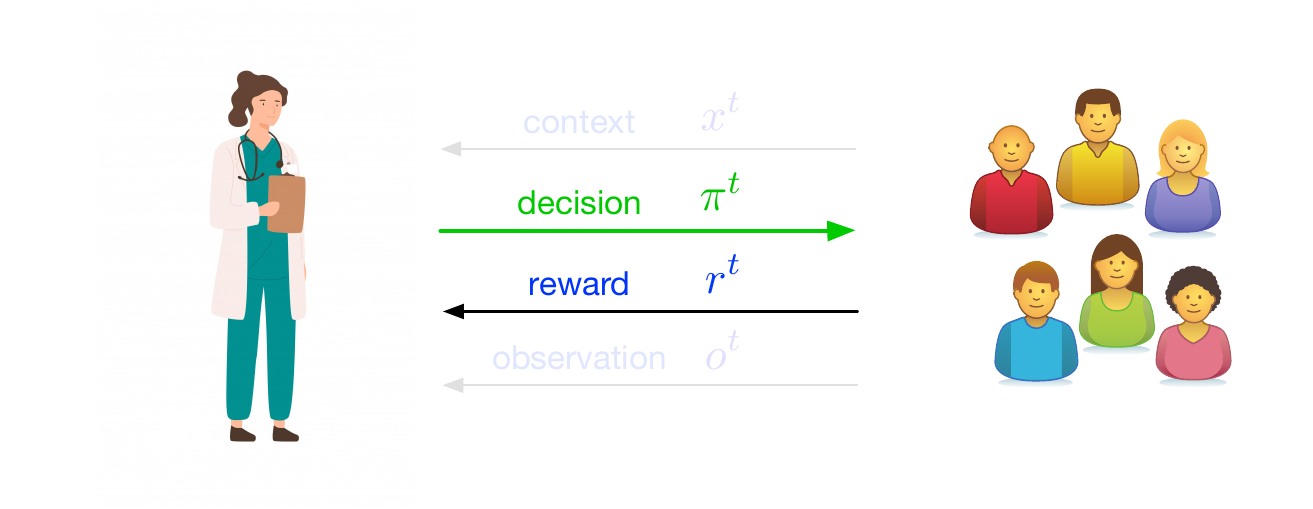}
    \caption{The multi-armed bandit problem. A doctor (the learner)
      selects a treatment (the decision) to improve a patient's vital
      signs (the reward). }
    \label{fig:mab}
\end{figure}

The multi-armed bandit problem can be studied in a stochastic framework,
in which rewards are generated from a fixed (conditional)
distribution, or an non-stochastic/adversarial framework in the vein of online
learning (\pref{sec:ol}). We focus on the stochastic framework, and
make the following assumption.
\begin{assumption}[Stochastic Rewards]
  \label{asm:stochastic_rewards_MAB}
  Rewards are generated independently via
    \begin{align}
        r\sups{t}\sim \Mstar(\cdot\mid\pi\sups{t}),
    \end{align}
    where $\Mstar(\cdot\mid\cdot)$ is the underlying \emph{model} (conditional distribution).
  \end{assumption}
  Throughout this section, we also assume that rewards are bounded:
  $r\ind{t}\in\brk{0,1}$ almost surely.\footnote{All results in this
    section extend to sub-Gaussian rewards with means in $\brk{0,1}$, at
    the cost of constants.}
  We define
  \begin{align}
    \label{eq:mab_mean_reward}
  \fstar(\pi) \ldef \En\left[r\mid{} \pi\right]    
  \end{align}
  as the mean reward function under $r\sim\Mstar(\cdot\mid\pi)$. We
  measure the learner's performance via regret to the action
  $\pistar\ldef\argmax_{\pi\in\Pi}\fstar(\pi)$ with highest reward:
  \begin{align}
    \label{eq:regret_mab}
    \Reg \ldef{} \sum_{t=1}^{T}\fstar(\pistar) - \sum_{t=1}^{T}\En_{\pi\ind{t}\sim{}p\ind{t}}\brk[\big]{\fstar(\pi\ind{t})}.
  \end{align}
  Regret is a natural notion of performance for the multi-armed bandit
  problem because
  it is \emph{cumulative}: it measures not just how well the learner can identify an action
  with good reward, but how well it can maximize reward as it
  goes. This notion is well-suited to settings like
  the personalized medicine example in \pref{fig:mab}, where regret
  captures the \emph{overall} quality of treatments, not just the quality
  of the final treatment. As in the online learning framework, we
  would like to develop algorithms that enjoy sublinear regret, i.e.
  \[
    \frac{\En\brk{\Reg}}{T}\to{}0\quad\text{as}\quad{}T\to\infty.
  \]
  Compared to the online regret \cref{eq:regret_deterministic}, the
  definition in \cref{eq:regret_mab} has the same cumulative form, but the
  benchmark is the best fixed decision $\pistar$ rather than the best
  function in a class, and---as discussed next---only the reward of the
  chosen decision is revealed. Regret bounds in this section hold with
  probability at least $1-\delta$ over the draw of the rewards unless
  stated otherwise.

  The most important feature of the multi-armed
  bandit problem, and what makes the problem fundamentally
  \emph{interactive}, is that the learner only receives a reward
  signal for the single decision $\pi\ind{t}\in\Pi$ they select at
  each round. That is, the observed reward $r\ind{t}$ gives a noisy estimate
  for $\fstar(\pi\ind{t})$, but reveals no information about the
  rewards for other decisions $\pi\neq\pi\ind{t}$.
For example in \pref{fig:mab}, if the doctor prescribes a particular treatment to
  the patient, they can observe whether the patient responds
  favorably, but they do not directly observe whether other possible
  treatments might have led to an even better outcome. This issue is
  often referred to as \emph{partial feedback} or \emph{bandit
    feedback}. Partial feedback introduces an element of \emph{active data
  collection}, as it means that the information contained
  in the dataset $\cH\ind{t}$ depends on the decisions made by the
  learner, which necessitates \emph{exploring} different
  actions. This should be contrasted with statistical learning (where
  the dataset is generated independently from the learner) and online
  learning (where losses may be chosen by nature in response to
  the learner's behavior, but where the outcome $y\ind{t}$--- and hence the full loss function
  $\ls(\cdot,y\ind{t})$---is always revealed).
  
In the context of \pref{fig:axes}, the multi-armed bandit problem
constitutes our first step along the ``interactivity'' axis, but does
not incorporate any structure in the decision space (and does not involve
features/contexts/covariates). In particular, information about one
action does not reveal information about any other actions, so there
is no hope of using function approximation to generalize across
actions.\footnote{Another way to say this is that we take
  $\cF=\bbR^{A}$, so that $\fstar\in\cF$.} As a result, the algorithms in this section
have regret that scales with $\bigom(\abs{\Pi})=\bigom(A)$. This
shortcoming is addressed by the \emph{structured bandit} framework of
\pref{sec:structured}, which allows for the use of function
approximation to model structure in the decision space.\footnote{Throughout this monograph, we
  use
  the term ``multi-armed bandit'' exclusively to refer to bandit problems with
  finite decision spaces, and the term ``structured bandit'' for
  problems with large decision spaces.}

  \begin{rem}[Alternative notions of regret]There are other, widely used notions of regret for the multi-armed bandit problem. One is the \emph{empirical regret}, defined as
    \begin{align}
    \label{eq:regret_realized_rewards}
\max_{\pi\in\Pi}\sum_{t=1}^{T}r\ind{t}(\pi) - \sum_{t=1}^{T}r\ind{t}(\pi\ind{t}),
    \end{align}
    where, for $\pi\neq\pi\ind{t}$, $r\ind{t}(\pi)$ denotes the
    \emph{counterfactual reward} the learner would have received if
    they had played $\pi$ at round $t$. Using Hoeffding's inequality, one can show that this is
equivalent to the definition in \pref{eq:regret_mab} up to additive 
$\bigoh(\sqrt{T})$ factors.
  \end{rem}

\subsection{The Need for Exploration}
\label{sec:need_for_exploration}

In statistical learning, we saw that the empirical risk minimization
algorithm, which greedily chooses the function that best fits the
data, leads to interesting bounds on excess risk. For multi-armed bandits,
since we assume the data generating process is stochastic, a natural
first attempt at designing an algorithm is to apply the greedy
principle here in the same fashion.
Concretely, at time $t$, we can compute an empirical estimate for the
reward function $\fstar$ via
\begin{equation}
  \label{eq:mab_mean}
\fhat\ind{t}(\pi)=\frac{1}{n\ind{t}(\pi)}\sum_{s<t}r\ind{s}\indic{\pi\ind{s}=\pi},
\end{equation}
where $n\ind{t}(\pi)$ is the number of times $\pi$ has been selected
up to time $t$.\footnote{If $n\ind{t}(\pi)=0$, we set $\fhat\ind{t}(\pi)=0$.} Then, we can
choose the greedy action
\[
\pihat\ind{t}=\argmax_{\pi\in\Pi}\fhat\ind{t}(\pi).
\]
Unfortunately, due to the interactive nature of the bandit problem,
this strategy can fail, leading to linear regret
($\Reg=\bigom(T)$). Consider the following problem with
$\Pi=\crl{1,2}$ ($A=2$).
\begin{itemize}
\item Decision $1$ has reward $\frac{1}{2}$ almost surely.
\item Decision $2$ has reward $\Ber(3/4)$.
\end{itemize}
Suppose we initialize by playing each decision a single time to ensure
that $n\ind{t}(\pi)>0$, then follow the greedy strategy. One can see
that with probability $1/4$,
the greedy algorithm will get stuck on action $1$, leading to regret $\bigom(T)$.

The issue in this example is that the greedy algorithm immediately gives up on the optimal
action and never revisits it. To address this, we
consider algorithms that deliberately \emph{explore} less visited actions to
ensure that their estimated rewards are not misleading.

There are two lessons to carry from this example. First, the failure is
caused by partial feedback, not by the sequential nature of the data. If
the reward of every decision were revealed at each round, the greedy
strategy would be an instance of empirical risk minimization on i.i.d.\
data, and \cref{prop:iid_finite_class} would apply. Second, an algorithm
should never discard a decision that might be optimal on the basis of a
few samples; the confidence intervals of \cref{sec:ucb_bandits} make this
idea precise.

\subsection{The \texorpdfstring{$\veps$}{ε}-Greedy Algorithm}
\label{sec:eps_greedy}
The greedy algorithm for bandits can fail because it can insufficiently explore
good decisions that initially seem bad, leading it to
get stuck playing suboptimal decisions. In light of this failure, a
reasonable solution is to manually force the algorithm to explore, so as to
ensure that this situation never occurs. This leads us to what is
known as the
\emph{\egreedy} algorithm (e.g., \citet{sutton2018reinforcement,auer2002finite}). 

\begin{algbox}[frametitle={\algtitle{alg:eps_greedy}{\egreedy}}]
  \begin{algorithmic}
    \Require Exploration parameter $\veps\in\brk{0,1}$.
\For{$t=1,\ldots,T$}
\State Compute the empirical mean rewards $\fhat\ind{t}$ as in \cref{eq:mab_mean}.
\State With probability $1-\veps$, select $\pi\ind{t}=\argmax_{\pi\in\Pi}\fhat\ind{t}(\pi)$; otherwise select $\pi\ind{t}\sim\unif(\Pi)$.
\State Observe reward $r\ind{t}$.
\EndFor{}
\end{algorithmic}
\end{algbox}

Let $\veps\in\brk{0,1}$ be the \emph{exploration parameter}. At each
time $t\in\brk{T}$, the \egreedy algorithm computes the estimated
reward function $\fhat\ind{t}$ as in \pref{eq:mab_mean}. With
probability $1-\veps$, the algorithm chooses the greedy decision
\begin{align}
    \pihat\sups{t} = \argmax_{\pi} \fhat\sups{t}(\pi),
\end{align}
and with probability $\veps$ it samples a uniform random action
$\pi\ind{t}\sim\unif(\crl{1,\ldots,A})$.

As the name suggests,
\egreedy usually plays the greedy action (\emph{exploiting} what it has
already learned), but the uniform sampling
ensures that the algorithm will also \emph{explore} unseen actions. We can
think of the parameter $\veps$ as modulating the tradeoff between
exploiting and exploring.
\begin{prop}[Regret of \egreedy]
  \label{prop:eps_greedy}
  For any $T$ and $\delta\in(0,1)$, the \egreedy algorithm with
  $\veps\propto\prn*{A\log(AT/\delta)/T}^{1/3}$ ensures that with
  probability at least $1-\delta$,
  \begin{align*}
    \Reg \approxleq A^{1/3}T^{2/3}\cdot\log^{1/3}(AT/\delta).
  \end{align*}
\end{prop}
Note that the exploration parameter must decay with the horizon $T$:
the algorithm explores less as it collects more data. This regret bound has $\frac{\En\brk{\Reg}}{T}\to{}0$ with
$T\to\infty$ as desired, though more sophisticated
strategies, beginning with \pref{sec:ucb_bandits}, attain
improved regret bounds that scale with $\sqrt{AT}$.\footnote{Note that
  $\sqrt{AT}\leq{}A^{1/3}T^{2/3}$ whenever $A\leq{}T$, and when
  $A\geq{}T$ both guarantees are vacuous.}\loose

\begin{proof}[\pfref{prop:eps_greedy}]
Recall that $\pihat\ind{t}\ldef\argmax_{\pi} \fhat\sups{t}(\pi)$ denotes the greedy
action at round $t$, and that $p\ind{t}$ denotes the distribution over
$\pi\ind{t}$. We can decompose the regret into two terms, representing the contribution from choosing the greedy action and the contribution
from exploring uniformly:
\begin{align*}
  \Reg
  &= \sum_{t=1}^{T}\En_{\pi\ind{t}\sim{}p\ind{t}}\brk*{\fstar(\pistar)
    - \fstar(\pi\ind{t})} \\
  &= (1-\veps)\sum_{t=1}^{T}\fstar(\pistar)
      - \fstar(\pihat\ind{t})
      + \veps \sum_{t=1}^{T}\En_{\pi\ind{t}\sim\unif(\brk{A})}\brk*{\fstar(\pistar)
    - \fstar(\pi\ind{t})} \\
  &\leq \sum_{t=1}^{T}\fstar(\pistar)
      - \fstar(\pihat\ind{t})
      + \veps{}T.
\end{align*}
In the last inequality, we have simply written off the contribution from exploring uniformly
by using that $\fstar(\pi)\in\brk{0,1}$. It remains to bound the
regret we incur from playing the greedy action. Here, we bound the per-step
regret in terms of estimation error using a similar decomposition to
\pref{lem:erm_uniform_dev} (note that we are now working with rewards
rather than losses):
  \begin{align}
    \fstar(\pistar)-\fstar(\pihat\sups{t}) &= [\fstar(\pistar)-\fhat\sups{t}(\pistar)] + \underbracet{[\fhat\sups{t}(\pistar) - \fhat\sups{t}(\pihat\sups{t})]}{$\leq{}0$} + [\fhat\sups{t}(\pihat\sups{t}) - \fstar(\pihat\sups{t})] \\
                                        &\leq
                                          2\max_{\pi\in\{\pistar,\pihat\sups{t}\}}
                                          |\fstar(\pi)-\fhat\sups{t}(\pi)|
                                          \leq 2\max_{\pi} |\fstar(\pi)-\fhat\sups{t}(\pi)|.
  \end{align}
  Note that this regret decomposition can also be applied to the pure
greedy algorithm, which we have already shown can fail. The reason why
\egreedy succeeds, which we use in the argument that follows, is
that because we explore, the ``effective'' number of times that each arm
will be pulled prior to round $t$ is of the order $\veps{}t/A$, which
will ensure that the sample mean converges to $\fstar$. In
particular, we show that the event
\begin{align}\cE_t = \left\{\max_{\pi} |\fstar(\pi)-\fhat\sups{t}(\pi)|\approxleq
  \sqrt{\frac{A\log(AT/\delta)}{\veps{}t}} \right\}
  \label{eq:egreedy_event}
\end{align}
occurs for all $t$ with probability at least $1-\delta$.

To prove that \cref{eq:egreedy_event} holds, we first apply Hoeffding's
inequality. Because the number of times $\pi$ has been pulled is random
and depends on the rewards observed so far, the version for a fixed
sample size does not apply directly; instead we use the version for
adaptive stopping times (\pref{lem:hoeffding_adaptive}), which gives that for any fixed $\pi$, with
probability at least $1-\delta$ over the draw of rewards,
\begin{align}
  \label{eq:mab_egreedy_proof1}
  |\fstar(\pi)-\fhat\sups{t}(\pi)| \leq \sqrt{\frac{2\log(2T/\delta)}{\nt(\pi)}}.
\end{align}
From here, taking a union bound over all $t\in\brk{T}$ and $\pi\in\Pi$ ensures that
\begin{align}
  \label{eq:mab_egreedy_proof2}
  |\fstar(\pi)-\fhat\sups{t}(\pi)| \leq \sqrt{\frac{2\log(2AT^2/\delta)}{\nt(\pi)}}
\end{align}
for all $\pi$ and $t$ simultaneously. It remains to show that the number of
pulls $n\ind{t}(\pi)$ is sufficiently large.\loose

Let $e\ind{t}\in\crl{0,1}$ be a random variable whose value indicates
whether the algorithm explored uniformly at step $t$, and let
$m\ind{t}(\pi)=\abs{\crl{i<t:\pi\ind{i}=\pi,e\ind{i}=1}}$, which has
$n\ind{t}(\pi)\geq{}m\ind{t}(\pi)$. Let
$Z\ind{t}=\indic{\pi\ind{t}=\pi,e\ind{t}=1}$. Observe that we can
write
\[
  m\ind{t}(\pi)=\sum_{i<t}Z\ind{i}.
\]
In addition, $Z\ind{t}\sim\Ber(\veps/A)$, so we have
$\En\brk*{m\ind{t}(\pi)}=\veps{}(t-1)/A$. Using Bernstein's inequality
(\pref{lem:bernstein}) with $Z\ind{1},\ldots,Z\ind{t-1}$, we have that
for any fixed $\pi$ and all $u>0$, with probability at least $1-2e^{-u}$,
$$\abs*{\mt(\pi) - \frac{\veps{}(t-1)}{A}} \leq \sqrt{2\Var\brk{Z}(t-1)u} + \frac{u}{3}\leq \sqrt{\frac{2\veps{}(t-1)u}{A}} + \frac{u}{3} \leq
\frac{\veps (t-1)}{2A} + \frac{4u}{3},$$
where we have used that
$\Var\brk{Z}=\veps/A\cdot(1-\veps/A)\leq\veps/A$,  and then applied
the arithmetic mean-geometric mean (AM-GM) inequality,
which states that
$\sqrt{xy}\leq\frac{x}{2}+\frac{y}{2}$ for $x,y\geq{}0$. Rearranging,
this gives
\begin{align}
  \label{eq:eps_greedy_N_at_least}
  \mt(\pi) \geq \frac{\veps(t-1)}{2A}  - \frac{4u}{3}.
\end{align}
Setting $u=\log(2AT/\delta)$ and taking a union bound, we are
guaranteed that with probability at least $1-\delta$, for all
$\pi\in\Pi$ and $t\in\brk{T}$
\begin{align}
  \label{eq:eps_greedy_N_at_leas2}
  \mt(\pi) \geq \frac{\veps(t-1)}{2A}  - \frac{4\log(2AT/\delta)}{3}.
\end{align}
As long as $\veps{}t\approxgeq{}A\log(AT/\delta)$ (we can write off
the rounds where this does not hold), this yields
\[
  \nt(\pi) \geq \mt(\pi) \approxgeq \frac{\veps t}{A}.
\]
Taking a union bound and combining with \pref{eq:mab_egreedy_proof2},
this implies that with probability at least $1-\delta$, for all $t$,
\[\max_{\pi} |\fstar(\pi)-\fhat\sups{t}(\pi)|\approxleq
  \sqrt{\frac{A\log(AT/\delta)}{\veps{}t}}.
\]
which leads to the overall regret bound
\begin{align}
  \Reg \leq  \sum_{t=1}^{T}\max_{\pi} |\fstar(\pi)-\fhat\sups{t}(\pi)|
  + \veps{}T 
  &\approxleq  \sum_{t=1}^{T}\sqrt{\frac{A\log(AT/\delta)}{\veps{}t}}
  + \veps{}T\notag\\
  &  \leq \sqrt{\frac{AT\log(AT/\delta)}{\veps{}}}
  + \veps{}T.\label{eq:egreedy_final}
\end{align}
To balance the terms on the right-hand side, we set
\[
  \veps \propto \prn*{\frac{A\log(AT/\delta)}{T}}^{1/3},
\]
which gives the final result.%
\end{proof}

This proof shows that the \egreedy strategy allows the learner to
acquire information uniformly for all actions, but we pay for this in
terms of regret (specifically, through the $\veps{}T$ factor in the
final regret bound \cref{eq:egreedy_final}). This issue here is that the \egreedy strategy
continually explores all actions, even though we might expect to rule
out actions with very low reward after a relatively small amount of exploration.
\Cref{sec:ucb_bandits} addresses this shortcoming with more adaptive strategies.

\begin{rem}[Explore-then-commit]
  A relative of \egreedy is the explore-then-commit (ETC) algorithm
  (e.g., \citet{robbins1952some,langford2008epoch}),
  which uniformly explores actions for the first $N$ rounds, then
  estimates rewards based on the data collected and commits to the
  greedy action for the remaining $T-N$ rounds. This strategy can 
  be shown to attain $\Reg\approxleq{}A^{1/3}T^{2/3}$ for an
  appropriate choice of $N$, matching \egreedy.
\end{rem}

\subsection{Optimism and the Upper Confidence Bound (UCB) Algorithm}
\label{sec:ucb_bandits}

The next algorithm we study is the \ucblong (\ucbtext)
algorithm \citep{lai1985asymptotically,agrawal1995sample,auer2002finite}. The \ucbtext algorithm attains a regret bound of the order
$\bigoht(\sqrt{AT})$, which improves upon the regret bound for
\egreedy, and is optimal (in a worst-case sense) up to logarithmic
factors. In addition to optimality, the algorithm offers several secondary benefits,
including adaptivity to favorable structure in the underlying reward function.

The \ucb algorithm is based on the notion of \emph{optimism in the face
  of uncertainty}, a general principle that recurs
throughout this monograph in
increasingly rich settings. The idea behind the principle is that at each time $t$, we should
adopt the most optimistic perspective of the world possible given the
data collected so far, and then choose the decision $\pi\ind{t}$ based on
this perspective.

To apply the idea of optimism to the multi-armed bandit
problem, suppose that for each step $t$, we can construct ``confidence intervals''
\begin{align}
    \label{eq:def_upper_lower_conf}
    \flcb\sups{t}, \fucb\sups{t}:\Pi\to\reals,
\end{align}
with the following property: with probability at least $1-\delta$,
\begin{align}
    \forall t\in[T], \pi\in\Pi,~~~ f^*(\pi) \in [\flcb\sups{t}(\pi), \fucb\sups{t}(\pi) ].
\end{align}
\begin{figure}
    \centering
    \includegraphics[width=0.6\textwidth]{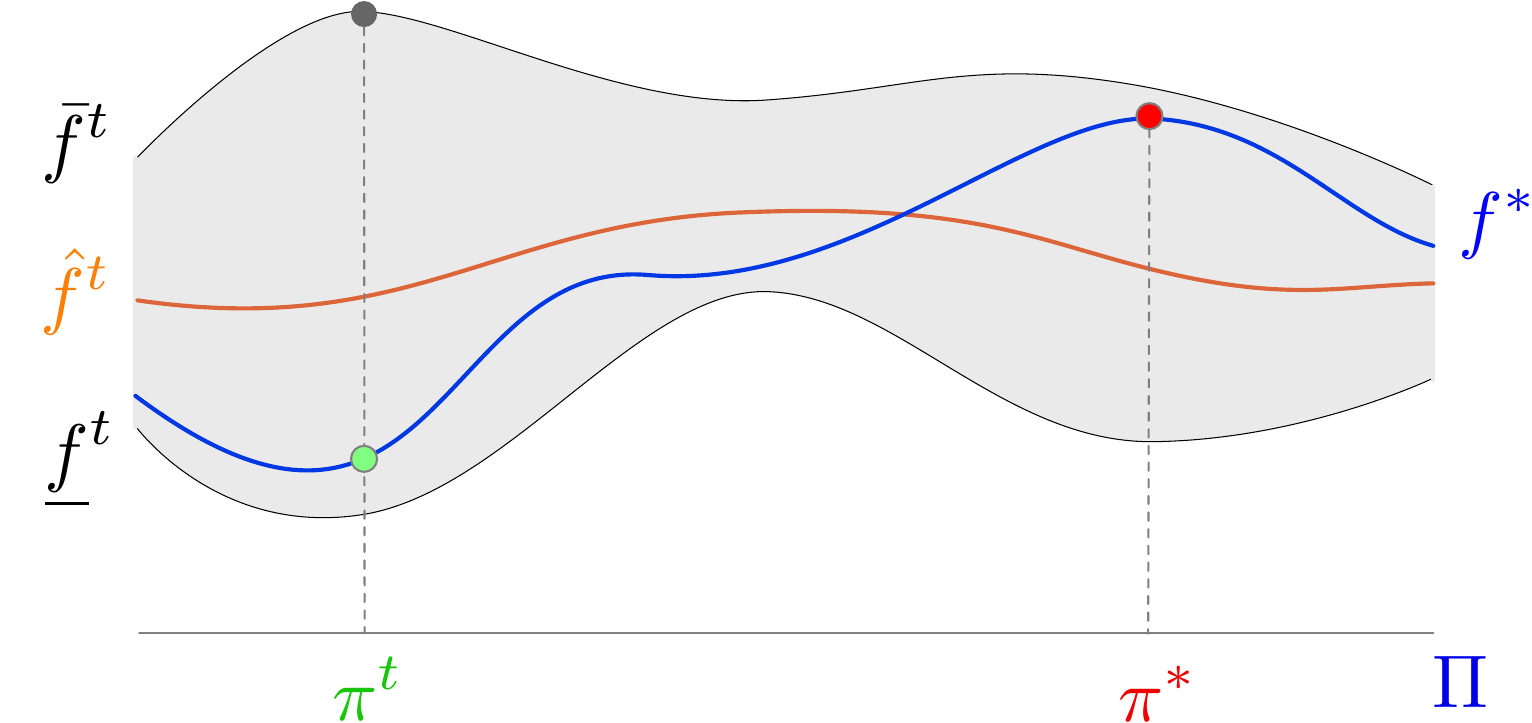}
    \caption{The \ucbshort algorithm. Selecting $\pi\ind{t}$
      optimistically ensures that its suboptimality never exceeds the
      confidence width.}
    \label{fig:confidence_width}
\end{figure}We refer to $\flcb\ind{t}$ as a \emph{lower confidence bound} and
$\fucb\ind{t}$ as a \emph{upper confidence bound}, since we are
guaranteed that with high probability, they lower (resp. upper) bound
$\fstar$. Given confidence intervals, the \ucbtext algorithm simply chooses $\pi\ind{t}$ as
the ``optimistic'' action that maximizes the upper confidence bound:
\begin{align}
    \label{eq:ucb_mab}
\pi\ind{t}=\argmax_{\pi\in\Pi}\fucb\ind{t}(\pi).
\end{align}
\begin{algbox}[frametitle={\algtitle{alg:ucb_mab}{Upper Confidence Bound (\ucb)}}]
  \begin{algorithmic}
    \Require Confidence parameter $\delta\in(0,1)$.
\For{$t=1,\ldots,T$}
\State Compute upper confidence bound $\fucb\ind{t}$ as in \cref{eq:ucb_confidence_bound}.
\State Select $\pi\ind{t}=\argmax_{\pi\in\Pi}\fucb\ind{t}(\pi)$.
\State Observe reward $r\ind{t}$.
\EndFor{}
\end{algorithmic}
\end{algbox}
The following lemma shows that the instantaneous regret for this
strategy is bounded by the width of the confidence interval; see
\pref{fig:confidence_width} for an illustration.
\begin{lem}[Optimism bounds regret by confidence width]
  \label{lem:regret_optimistic}
  Fix $t$, and suppose that
  $\fstar(\pi)\in\brk{\flcb\ind{t}(\pi),\fucb\ind{t}(\pi)}$ for all
  $\pi$. Then the optimistic action
  \[
    \pi\ind{t}=\argmax_{\pi\in\Pi}\fucb\ind{t}(\pi)
  \]
  has
  \begin{align}
	  \label{eq:regret_optimistic_conclusion}
	  \fstar(\pistar) - \fstar(\pi\sups{t}) \leq \fucb\sups{t}(\pi\sups{t}) - \fstar(\pi\sups{t}) \leq \fucb\sups{t}(\pi\sups{t}) - \flcb\sups{t}(\pi\sups{t}).
  \end{align}
\end{lem}
\begin{proof}[\pfref{lem:regret_optimistic}]
  The result follows immediate from the observation that for any $t\in[T]$ and any $\pistar\in\Pi$, we have
    $$\fstar(\pistar)\leq \fucb\sups{t}(\pistar) \leq \fucb\sups{t}(\pi\sups{t})~~~\text{and}~~~~ -\fstar(\pi\sups{t}) \leq -\flcb\sups{t}(\pi\sups{t}).$$
  \end{proof}
  \pref{lem:regret_optimistic} implies that as long as we can build
  confidence intervals for which the width  $\fucb\sups{t}(\pi\sups{t}) - \flcb\sups{t}(\pi\sups{t})$
  shrinks, the regret for the \ucbshort strategy will be small. To
  construct such intervals, here we appeal to Hoeffding's inequality for
  adaptive stopping times
  (\pref{lem:hoeffding_adaptive}).\footnote{While asymptotic
    confidence intervals in classical statistics arise from limit
    theorems, we are interested in valid \emph{non-asymptotic}
    intervals, and thus appeal to concentration inequalities.} Since
  $r\ind{t}\in\brk{0,1}$, a union bound gives
  that with probability at least $1-\delta$, for all $t\in\brk{T}$ and $\pi\in\Pi$,
  \begin{equation}
    \label{eq:ucb_good_event}
\abs{\fhat\ind{t}(\pi)-\fstar(\pi)}\leq\sqrt{\frac{2\log
          (2T^2A/\delta)}{n\sups{t}(\pi)}},
    \end{equation}
    where we recall that $\fhat\ind{t}$ is the sample mean and $n\ind{t}(\pi)\ldef{}\sum_{i<t}\indic{\pi\ind{i}=\pi}$. This suggests that with
    \begin{align}
      \label{eq:ucb_confidence_bound}
      \fucb\ind{t}(\pi) = \fhat\ind{t}(\pi) + \sqrt{\frac{2\log
      (2T^2A/\delta)}{n\sups{t}(\pi)}}
      ,\mathand
            \flcb\ind{t}(\pi) = \fhat\ind{t}(\pi) - \sqrt{\frac{2\log
      (2T^2A/\delta)}{n\sups{t}(\pi)}},
    \end{align}
    we obtain a valid confidence interval. With this choice---along
    with \pref{lem:regret_optimistic}---we are in
    a favorable position, because for a given round $t$, one of two
    things must happen:
    \begin{itemize}
    \item The optimistic action has high reward, so the instantaneous
      regret is small.
    \item The instantaneous regret is large, which by
      \pref{lem:regret_optimistic} implies that confidence width is
      large as well (and $n\ind{t}(\pi\ind{t})$ is small). This can only
      happen a small number of times, since $n\ind{t}(\pi\ind{t})$
      will increase as a result, causing the width to shrink.
    \end{itemize}
Using this idea, we can prove the following regret bound.
\begin{theorem}[Regret of \ucb]
  \label{prop:ucb}
  Using the confidence bounds in \pref{eq:ucb_confidence_bound}, the
  \ucbshort algorithm ensures that with probability at least $1-\delta$,
  \begin{align*}
    \Reg \approxleq \sqrt{AT\log(AT/\delta)}.
  \end{align*}
\end{theorem}
This result is optimal up to the $\log(AT)$ factor: \cref{sec:lower_mab}
shows that every algorithm suffers $\Omega(\sqrt{AT})$ regret on some
instance with Gaussian rewards. The logarithmic factor can be removed by using the same
algorithm with a slightly more sophisticated confidence interval
construction \citep{audibert2009minimax}. Note that compared
to the statistical learning and online learning setting, where we were
able to attain regret bounds that scaled logarithmically with the size
of the benchmark class, here the
optimal regret scales \emph{polynomially} with
$\abs{\Pi}=A$. This is the price we pay for partial/bandit feedback,
and reflects that fact that we must explore all actions to learn.

\begin{proof}[\pfref{prop:ucb}]
Let us condition on the event in \pref{eq:ucb_good_event}. Whenever
this occurs, we have that
$\fstar(\pi)\in\brk*{\flcb\ind{t}(\pi),\fucb\ind{t}(\pi)}$ for all
$t\in\brk{T}$ and $\pi\in\Pi$, so the confidence intervals are
valid. As a result, \pref{lem:regret_optimistic}  bounds regret in
terms of the confidence width:
\begin{align}
  \sum_{t=1}^{T}\fstar(\pistar)-\fstar(\pi\ind{t})
  \leq{} \sum_{t=1}^{T}\fucb\ind{t}(\pi\ind{t}) -
  \flcb\ind{t}(\pi\ind{t})
  = \sum_{t=1}^{T}2\sqrt{\frac{2\log
          (2T^2A/\delta)}{n\sups{t}(\pi\ind{t})}}\wedge{}1;\label{eq:ucb_intermediate}
\end{align}
here, the ``$\wedge{}1$'' term appears because we write off the
regret for early rounds where $n\ind{t}(\pi\ind{t})=0$ as $1$.\loose

To bound the right-hand side, we use a potential argument. The basic
idea is that at every round, $n\ind{t}(\pi)$ must increase for some
action $\pi$, and since there are only $A$ actions, this means that
$1/\sqrt{n\ind{t}(\pi\ind{t})}$ can only be large for a small number
of rounds. This can be thought of as a quantitative instance of the
pigeonhole principle.%
\begin{lem}[Confidence width potential lemma]
  \label{lem:confidence_width_potential}
  We have
  \[
    \sum_{t=1}^T \frac{1}{\sqrt{n\sups{t}(\pi\sups{t})}}\wedge{}1
    \approxleq \sqrt{AT}.
  \]
\end{lem}
\begin{proof}[\pfref{lem:confidence_width_potential}]
  We begin by writing.
  \begin{align}
    \sum_{t=1}^T \frac{1}{\sqrt{n\sups{t}(\pi\sups{t})}}\wedge{}1  =  \sum_{\pi} \sum_{t=1}^T\frac{\indic{\pi\sups{t}=\pi}}{\sqrt{n\sups{t}(\pi)}}\wedge{}1 =\sum_{\pi} \sum_{t=1}^{n\sups{T+1}(\pi)}\frac{1}{\sqrt{t-1}}\wedge{}1.
  \end{align}
  For any $n\in\bbN$, we have
  $\sum_{t=1}^{n}\frac{1}{\sqrt{t-1}}\wedge{}1\leq{} 1 + 2\sqrt{n}$, which allows us
  to bound by
  \[
    A + 2 \sum_{\pi} \sqrt{n\sups{T}(\pi)}.
  \]
The factor of $A$ above is a lower-order term (recall that we have $A\leq\sqrt{AT}$ whenever $A\leq{}T$, and if
$A>T$ the regret bound we are proving is vacuous). To bound the second
term, using Jensen's inequality, we have
  \begin{align*}
    \sum_{\pi} \sqrt{n\sups{T}(\pi)}
    \leq{} A\sqrt{ \sum_{\pi}\frac{n\sups{T}(\pi)}{A}}
    = A\sqrt{T/A} = \sqrt{AT}.
  \end{align*}
\end{proof}
The main regret bound now follows from
\pref{lem:confidence_width_potential} and \cref{eq:ucb_intermediate}.

\end{proof}

To summarize, the key steps in the proof of \pref{prop:ucb} were to:
\begin{enumerate}
\item Use the optimistic property and validity of the confidence
  bounds to bound regret by the sum of confidence widths.
\item Use a potential argument to show that the sum of confidence
  widths is small.
\end{enumerate}
Subsequent sections revisit and generalize both ideas for
more sophisticated settings, including contextual bandits, structured
bandits, and reinforcement learning.

\begin{rem}[Instance-dependent regret for \ucb]
  The $\bigoht(\sqrt{AT})$ regret bound attained by \ucbshort holds
  uniformly for all models, and is (nearly) minimax-optimal, in the sense that for any algorithm, there
exists a model $\Mstar$ for which the regret must scale as
$\bigom(\sqrt{AT})$. Minimax optimality is a useful notion of
performance, but may be overly pessimistic. As an alternative, it is
possible to show that the \ucb attains what is known as an \emph{instance-dependent} regret bound, which reflects the fact that the algorithm has better regret on ``nice'' problem
instances.\loose

Let $\Delta(\pi)\ldef\fstar(\pistar)-\fstar(\pi)$ be
the \emph{suboptimality gap} for decision $\pi$. Then \ucbshort
can be shown to achieve
\[
  \Reg
  \approxleq{} \sum_{\pi:\Delta(\pi)>0} \frac{\log(AT/\delta)}{\Delta(\pi)}.
\]
If we keep the underlying model fixed and take $T\to\infty$, this
regret bound scales only \emph{logarithmically} in $T$, which improves
upon the $\sqrt{T}$-scaling of the minimax regret bound.
  
\end{rem}

This concludes the main development of this section. The remaining subsections cover two extensions, the Bayesian viewpoint and adversarial or non-stochastic 
rewards.

\subsection{Bayesian Bandits and the Posterior Sampling
  Algorithm\bonus}
\label{sec:posterior}
\newcommand{\RegBayes}{\Reg_{\mathsf{Bayes}}}

Up to this point, we have been designed and analyzed algorithms from
a \emph{frequentist} viewpoint, in which we aim to minimize regret for
a \emph{worst-case} choice of the underlying model $\Mstar$. An
alterative is to adopt a \emph{Bayesian} viewpoint, and assume that
the underlying model is drawn from a known \emph{prior}
$\mu\in\Delta(\cM)$.\footnote{It is important that $\mu$ is known,
  otherwise this is no different from the frequentist setting.} In this case, rather than worst-case performance,
we are concerned with the average regret under the prior, defined via
\[
  \RegBayes(\mu)
  \ldef{} \En_{\Mstar\sim\mu}\En\sups{\Mstar}\brk*{\RegBasic},
\]
where $\En\sups{\Mstar}\brk*{\cdot}$ denotes the algorithm's expected regret when
$\Mstar$ is the underlying reward distribution.

Working in the Bayesian setting opens up additional avenues for
designing algorithms, because we can take advantage of our knowledge
of the prior to compute quantities of interest that are not available
in the frequentist setting, such as posterior distribution over
$\pistar$ after observing the dataset $\cH\ind{t-1}$. The most basic and well-known strategy
here is \emph{posterior sampling} (also known as Thompson sampling or
probability matching) \citep{thompson1933likelihood,russo2014learning}. 
\begin{algbox}[frametitle={\algtitle{alg:posterior_sampling}{Posterior Sampling}}]
\begin{algorithmic}
    \Require Prior distribution over models.
\For{$t=1,\ldots,T$}
\State Set $p\ind{t}(\pi)=\bbP\prn*{\pistar=\pi\mid{}\hist\ind{t-1}}$,
where $\hist\ind{t-1}=(\pi\ind{1},r\ind{1}),\ldots,(\pi\ind{t-1},r\ind{t-1})$.
\State Sample $\pi\ind{t}\sim{}p\ind{t}$ and observe $r\ind{t}$.
\EndFor{}
\end{algorithmic}
\end{algbox}
The basic idea is as follows. At each time $t$, we can use our
knowledge of the prior to compute the distribution 
$\bbP\prn*{\pistar=\cdot\mid{}\hist\ind{t-1}}$, which represents the posterior
distribution over $\pistar$ given all of the data we have collected
from rounds $1,\ldots,t-1$. The posterior sampling algorithm simply
samples the learner's action $\pi\ind{t}$ from this distribution,
thereby ``matching'' the posterior distribution of $\pistar$.

\begin{prop}[Bayesian regret of posterior sampling]
  \label{prop:posterior_mab}
  For any prior $\mu$, the posterior sampling algorithm ensures that
  \begin{align}
\RegBayes(\mu) \leq \sqrt{AT\log(A)},
  \end{align}
\end{prop}

\newcommand{\Rho}{\mathrm{P}}%
\newcommand{\rhostar}{\rho^{\star}}%
  \newcommand{\fmbart}{f^{\sss{\Mbar\ind{t}}}}%
  \newcommand{\pt}{p\ind{t}}%
  \newcommand{\Mbartrho}{\Mbar\ind{t}_{\rho}}%
  \newcommand{\Mbartrhostar}{\Mbar\ind{t}_{\rhostar}}%

Here we prove a simplified version of
\cref{prop:posterior_mab}; the full proof is given in
\cref{sec:mab_deferred}. The argument has three steps, and the reader may
find it useful to keep them in mind: (i) a \emph{decoupling} inequality
bounds the per-round Bayesian regret by the estimation error of the
posterior mean under the learner's own distribution $p\ind{t}$; (ii) this
estimation error is bounded by the \emph{information gain} about
$\fstar$ from the round's observation; (iii) the total information gain
over $T$ rounds telescopes and is at most the entropy of the prior. Step
(i) is the part that generalizes furthest: \cref{sec:dec_posterior}
recasts it as a bound on the Decision-Estimation Coefficient, a general complexity measure for interactive decision making.
  
  \begin{proof}[\pfref{prop:posterior_mab} (simplified version)]%
    We make the following simplifying assumptions:
    \begin{itemize}
    \item We restrict to reward distributions where
      $\Mstar(\cdot\mid{}\pi)=\cN(\fstar(\pi),1)$, departing, for
      simplicity, from the boundedness convention of this section. That is, $\fstar$
      is the only part of the reward distribution that is unknown.
    \item $\fstar$ belongs to a known class $\cF$, and rather than proving
      the regret bound in \pref{prop:posterior_mab}, we prove a
      bound of the form
      \[
        \RegBayes(\mu) \approxleq \sqrt{AT\log\abs{\cF}},
      \]
      which
      replaces the $\log{}A$ factor in the proposition with
      $\log\abs{\cF}$.
    \end{itemize}
    Since the mean reward function $\fstar$ is the only part of the
    reward distribution $\Mstar$ that is unknown, we can simplify by
    considering an equivalent formulation where the prior has the form $\mu\in\Delta(\cF)$. That is, we
    have a prior over $\fstar$ rather than $\Mstar$.

    Before proceeding, let us introduce some notation. The process through which we sample $\fstar\sim\mu$ and run the
bandit algorithm induces a joint law  over $(\fstar, \cH\ind{T})$,
which we call $\bbP$. Throughout the proof, we use $\En\brk*{\cdot}$
to denote the expectation under this law. We also define $\En_t\brk*{\cdot}=\En\brk*{\cdot\mid{}\hist\ind{t}}$ and
$\bbP_t\brk{\cdot}=\bbP\brk{\cdot\mid\hist\ind{t}}$.

We begin by using
the law of total expectation to express the expected regret as
\[
\RegBayes(\mu) = 
\En\brk*{\sum_{t=1}^{T}\En_{t-1}\brk*{\fstar(\pi_{\fstar}) - \fstar(\pi\ind{t})}}.
\]
Above, we have written $\pistar=\pi_{\fstar}$ to make explicit the fact that
this is a random variable whose value is a function of $\fstar$.

We first simplify the expected regret for each step $t$. Let
$\mu\ind{t}(f)\ldef{}\bbP(\fstar=f\mid{}\cH\ind{t-1})$ be the
posterior distribution at timestep $t$. The learner's decision
$\pi\ind{t}$ is \emph{conditionally independent} of $\fstar$ given
$\cH\ind{t-1}$, so we can write\ \begin{align*}
  \En_{t-1}\brk*{\fstar(\pi_{\fstar}) - \fstar(\pi\ind{t})}
  &=
    \En_{\fstar\sim\mu\ind{t},\pi\ind{t}\sim{}p\ind{t}}\brk*{\fstar(\pi_{\fstar})
    - \fstar(\pi\ind{t})}.
\end{align*}
If we define
$\fbar\ind{t}(\pi)=\En_{\fstar\sim\mu\ind{t}}\brk*{\fstar(\pi)}$ as
the expected reward function under the posterior, we can further write
this as
\begin{align*}
      \En_{\fstar\sim\mu\ind{t},\pi\ind{t}\sim{}p\ind{t}}\brk*{\fstar(\pi_{\fstar})
    - \fbar\ind{t}(\pi\ind{t})}.
\end{align*}
By the design of the posterior sampling algorithm,
$\pi\ind{t}\sim{}p\ind{t}$ is identical in distribution to
$\pi_{\fstar}$ under $\fstar\sim{}\mu\ind{t}$, so this is equal to 
\begin{align*}
      \En_{\fstar\sim\mu\ind{t}}\brk*{\fstar(\pi_{\fstar})
    - \fbar\ind{t}(\pi_{\fstar})}.
\end{align*}
This quantity captures---on average---how far a given realization of $\fstar$
deviates from the posterior mean $\fbar\ind{t}$, for a specific
decision $\pi_{\fstar}$ which is coupled to $\fstar$. The expression
above might appear to be unrelated to the learner's decision
distribution, but the next lemma
shows that it is possible to relate this quantity back to the learner's
decision distribution using a notion of
\emph{information gain} (or, estimation error).
  \begin{lem}[Decoupling]
    \label{lem:mab_decoupling_basic}
For any function $\fbar:\Pi\to\bbR$, it holds that 
    \begin{equation}
      \label{eq:mab_decoupling_basic}
      \En_{\fstar\sim\mu\ind{t}}\brk*{\fstar(\pi_{\fstar})-\fbar(\pi_{\fstar})}
      \leq \sqrt{A\cdot\En_{\fstar\sim\mu\ind{t}}\En_{\pi\ind{t}\sim{}p\ind{t}}\brk*{(\fstar(\pi\ind{t})-\fbar(\pi\ind{t}))^2}}.
    \end{equation}
  \end{lem}
  \begin{proof}[\pfref{lem:mab_decoupling_basic}]
    We show a more general result. Namely, for any
    $\nu\in\Delta(\cF)$ and $\fbar:\Pi\to\bbR$, if we define
    $p(\pi)=\bbP_{f\sim\nu}\prn*{\pi_f=\pi}$, then
    \begin{equation}
      \label{eq:decoupling_general}
      \En_{f\sim\nu}\brk*{f(\pi_{f})-\fbar(\pi_{f})}
      \leq \sqrt{A\cdot\En_{f\sim\nu}\En_{\pi\sim{}p}\brk*{(f(\pi)-\fbar(\pi))^2}}.
    \end{equation}
    This can be thought of as a ``decoupling'' lemma. On the \lhs, the
    random variables $f$ and $\pi_f$ are coupled, but on the \rhs,
    $\pi$ is drawn from the \emph{marginal distribution} over $\pi_f$,
    independent of the draw of $f$ itself.

    To prove the result, we use Cauchy-Schwarz as follows:
    \begin{align*}
      \En_{f\sim\nu}\brk*{f(\pi_{f})-\fbar(\pi_{f})}
      &=
        \En_{f\sim\nu}\brk*{\frac{p^{1/2}(\pi_f)}{p^{1/2}(\pi_f)}\prn*{f(\pi_{f})-\fbar(\pi_{f})}} \\
        &\leq{}
                \prn*{\En_{f\sim\nu}\brk*{\frac{1}{p(\pi_f)}}}^{1/2}
                \cdot\prn*{\En_{f\sim\nu}\brk*{p(\pi_f)\prn*{f(\pi_{f})-\fbar(\pi_{f})}^2}
                }^{1/2}.
    \end{align*}
    For the first term, we have
    \[
      \En_{f\sim\nu}\brk*{\frac{1}{p(\pi_f)}}
      = \sum_{f}\frac{\nu(f)}{p(\pi_f)}
      = \sum_{\pi}\sum_{f:\pi_f=\pi}\frac{\nu(f)}{p(\pi)}
            = \sum_{\pi}\frac{p(\pi)}{p(\pi)} = A.
          \]
          For the second term, we have
          \begin{align*}
            \En_{f\sim\nu}\brk*{p(\pi_f)\prn*{f(\pi_{f})-\fbar(\pi_{f})}^2}
            \leq{}\En_{f\sim\nu}\brk*{\sum_{\pi}p(\pi)\prn*{f(\pi)-\fbar(\pi)}^2}
            = \En_{f\sim\nu}\En_{\pi\sim{}p}\brk*{(f(\pi)-\fbar(\pi))^2}.
          \end{align*}
Putting these bounds together yields \pref{eq:decoupling_general}.    
\end{proof}
Using \pref{lem:mab_decoupling_basic}, we have that
\begin{align*}
  \RegBayes(\mu)
  &\leq{} \En\brk*{
  \sum_{t=1}^{T}
  \sqrt{A\cdot\En_{\fstar\sim\mu\ind{t}}\En_{\pi\ind{t}\sim{}p\ind{t}}\brk*{(\fstar(\pi\ind{t})-\fbar\ind{t}(\pi\ind{t}))^2}}
    } \\
  &\leq{} 
    \sqrt{AT\cdot \En\brk*{\sum_{t=1}^{T}\En_{\fstar\sim\mu\ind{t}}\En_{\pi\ind{t}\sim{}p\ind{t}}\brk*{(\fstar(\pi\ind{t})-\fbar\ind{t}(\pi\ind{t}))^2}}
    }.
\end{align*}
To finish up, we show that
\[
\En\brk*{\sum_{t=1}^{T}\En_{\fstar\sim\mu\ind{t}}\brk*{\En_{\pi\ind{t}\sim{}p\ind{t}}\brk*{(\fstar(\pi\ind{t})-\fbar\ind{t}(\pi\ind{t}))^2}}}\leq\frac{5}{2}\log\abs{\cF}.
\]
To do this, we need some additional information-theoretic tools.
\begin{itemize}
\item For a random variable $X$ with distribution $\bbP$, $\Ent(X)\equiv\Ent(\bbP)\ldef{}\sum_{x}p(x)\log(1/p(x))$.
\item For random variables $X$ and $Y$,
  $\Ent(X\mid{}Y=y)\ldef{}\Ent(\bbP_{X\mid{}Y=y})$ and $\Ent(X\mid{}Y)\ldef\En_{y\sim{}p_Y}\brk*{\Ent(X\mid{}Y=y)}$.
\item For distributions $\bbP$ and $\bbQ$, $\Dkl{\bbP}{\bbQ}=\sum_{x}p(x)\log(p(x)/q(x))$.
\end{itemize}
\newcommand{\bpi}{\mb{\pi}}
\newcommand{\bpistar}{\mb{\pistar}}
\newcommand{\bfstar}{\mb{\fstar}}
\newcommand{\bcH}{\mb{\cH}}
\newcommand{\br}{\mb{r}}
To keep notation as clear as possible going forward, let us use boldface script ($\bpi\ind{t}$,
$\bpistar$, $\bfstar$, $\bcH\ind{t}$) to refer to the abstract random
variables under consideration, and use non-boldface script
($\pi\ind{t}$, $\pistar$, $\fstar$, $\cH\ind{t}$)  to
refer to their realizations. Our aim is to use the conditional
entropy $\Ent(\bfstar\mid{}\bcH\ind{t})$ as a potential function, and show that for each
$t$,
\begin{align}
  \label{eq:entropy_potential}
   \frac{2}{5}\En\brk*{\En_{\fstar\sim\mu\ind{t}}\En_{\pi\ind{t}\sim{}p\ind{t}}\brk*{(\fstar(\pi\ind{t})-\fbar\ind{t}(\pi\ind{t}))^2}}
  \mathrel{\leq}\Ent(\bfstar\mid{}\bcH\ind{t-1}) - \Ent(\bfstar\mid{}\bcH\ind{t}).
\end{align}
From here the result follows, because
\begin{align*}
  \frac{2}{5}\En\brk*{\sum_{t=1}^{T}\En_{\fstar\sim\mu\ind{t}}\En_{\pi\ind{t}\sim{}p\ind{t}}\brk*{(\fstar(\pi\ind{t})-\fbar\ind{t}(\pi\ind{t}))^2}}
  &\mathrel{\leq} \sum_{t=1}^{T}\brk*{\Ent(\bfstar\mid{}\bcH\ind{t-1}) -
  \Ent(\bfstar\mid{}\bcH\ind{t})}\\
  &= \Ent(\bfstar\mid{}\bcH\ind{0}) - \Ent(\bfstar\mid{}\bcH\ind{T}) \\
  &\leq{}  \Ent(\bfstar\mid{}\bcH\ind{0}) \\
  &\leq \log\abs{\cF},
\end{align*}
where the last inequality follows because the entropy of a random
variable $X$ over a set $\cX$ is always bounded by $\log\abs{\cX}$.

We now prove \pref{eq:entropy_potential}.
Conditional on the history and current action, the predictive reward
is the sum of a random mean in $[0,1]$ and independent $\cN(0,1)$ noise.
By Hoeffding's lemma, the centered random mean is $1/4$-\subgaussian,
so the predictive reward is $5/4$-\subgaussian.
The reward conditional on $\fstar$ is also $5/4$-\subgaussian.
Thus, Pinsker's inequality for sub-Gaussian random variables (\pref{lem:pinsker_subgaussian}) implies that
\[
\frac{2}{5}(\fstar(\pi\ind{t})-\fbar\ind{t}(\pi\ind{t}))^2 \leq{} \Dkl{\bbP_{\br\ind{t}\mid{}\fstar,\pi\ind{t},\cH\ind{t-1}}}{
  \bbP_{\br\ind{t}\mid{}\pi\ind{t},\cH\ind{t-1}}}.
\]
and
\[
  \frac{2}{5}\En_{\fstar\sim\mu\ind{t}}\En_{\pi\ind{t}\sim{}p\ind{t}}\brk*{(\fstar(\pi\ind{t})-\fbar\ind{t}(\pi\ind{t}))^2}
  \leq \En_{\fstar\sim\mu\ind{t}}\En_{\pi\ind{t}\sim{}p\ind{t}}\brk*{
\Dkl{\bbP_{\br\ind{t}\mid{}\fstar,\pi\ind{t},\cH\ind{t-1}}}{
  \bbP_{\br\ind{t}\mid{}\pi\ind{t},\cH\ind{t-1}}}
}
\]
  Since KL divergence satisfies
  $\En_{x\sim{}\bbP_X}\brk*{\Dkl{\bbP_{Y\mid{}X=x}}{\bbP_Y}}=\En_{y\sim{}\bbP_Y}\brk*{\Dkl{\bbP_{X\mid{}Y=y}}{\bbP_X}}$,
  this equals
  \begin{align}
    \En_{t-1}\brk*{
    \Dkl{\bbP_{\bfstar\mid{}\pi\ind{t},r\ind{t},\cH\ind{t-1}}}{\bbP_{\bfstar\mid{}\cH\ind{t-1}}}
    } =
    \En_{t-1}\brk*{
    \Dkl{\bbP_{\bfstar\mid{}\cH\ind{t}}}{\bbP_{\bfstar\mid{}\cH\ind{t-1}}}
    }.\label{eq:mab_info_gain_basic}
  \end{align}
  Taking the expectation over $\cH\ind{t-1}$, we can write this
    as
    \begin{align*}
     \En\brk*{ \En_{t-1}\brk*{
    \Dkl{\bbP_{\bfstar\mid{}\cH\ind{t}}}{\bbP_{\bfstar\mid{}\cH\ind{t-1}}}
      }}
      = 
      \En_{\cH\ind{t-1}}\En_{\cH\ind{t}\mid\cH\ind{t-1}}\brk*{
    \Dkl{\bbP_{\bfstar\mid{}\cH\ind{t}}}{\bbP_{\bfstar\mid{}\cH\ind{t-1}}}
    }.
    \end{align*}
    A simple exercise shows that for random variables $X,Y,Z$,
    \begin{align*}
      \En_{(x,y)\sim{}\bbP_{X,Y}}\brk*{
      \Dkl{\bbP_{Z\mid{}X=x,Y=y}}{\bbP_{Z\mid{}X=x}
      }}
      = \Ent(Z\mid{}X) - \Ent(Z\mid{}X,Y).
    \end{align*}
    Applying this result above (and using that
    $\cH\ind{t-1}\subset\cH\ind{t}$) gives
    \begin{align*}
      \En_{\cH\ind{t-1}}\En_{\cH\ind{t}\mid\cH\ind{t-1}}\brk*{
    \Dkl{\bbP_{\bfstar\mid{}\cH\ind{t}}}{\bbP_{\bfstar\mid{}\cH\ind{t-1}}}
    } = \Ent(\bfstar\mid\bcH\ind{t-1})-\Ent(\bfstar\mid{}\bcH\ind{t}).
    \end{align*}
\end{proof}

The analysis above critically makes use of the fact that we are
concerned with Bayesian regret, and have access to the \emph{true
  prior}. One might hope that by choosing a sufficiently uninformative
prior, this approach might continue to work in the frequentist
setting. In fact, this is indeed the case for bandits, though a different analysis
is required \cite{agrawal2012analysis,agrawal2013thompson}. However,
one can show (\pref{sec:structured,sec:general_dm}) that the Bayesian
analysis we have given here extends to significantly richer decision
making settings, while the frequentist counterpart is limited to
simple variants of the multi-armed bandit.

\begin{rem}[Equivalence of min-max frequentist regret and max-min
  Bayesian regret]
  Using the minimax theorem, it is possible to show that under
  appropriate technical conditions
  \begin{align*}
    \min_{\mathsf{Alg}}\max_{\Mstar}\En^{\sss{\Mstar}, \mathsf{Alg}}\brk*{\Reg}
    =
    \max_{\mu\in\Delta(\cM)}\min_{\mathsf{Alg}}\En_{\Mstar\sim\mu}\En^{\sss{\Mstar},
    \mathsf{Alg}}\brk*{\Reg}.
  \end{align*}
  That is, if we take the worst-case value of the Bayesian regret over
  all possible choices of prior, this coincides with the minimax value
  of the frequentist regret.
\end{rem}

The entropy potential in the proof plays the same role as the
pigeonhole argument of \cref{lem:confidence_width_potential}: there, each
round that incurs large regret must increase the count of some action, of
which there are only $A$; here, each round that incurs large regret must
reveal information about $\fstar$, but only 
$\bigoh(\log\abs{\cF})$ bits of information can be revealed in total.

\subsection{Adversarial Bandits and the Exp3 Algorithm\bonus}
\label{sec:adversarial_bandits}

The stochastic model of \cref{asm:stochastic_rewards_MAB} is
restrictive: in applications such as online advertising or routing, the
reward of a decision may drift over time, or may be chosen by an
adversary. This subsection gives a brief introduction to the multi-armed
bandit problem with non-stochastic rewards, which dispenses with
\cref{asm:stochastic_rewards_MAB} entirely. In the context of
\cref{fig:axes}, this moves the problem along the ``data'' axis, in the
same way that online learning generalizes statistical learning in
\cref{sec:background}; the algorithm and analysis below make this
analogy exact.  To simplify the presentation, suppose that the collection of rewards  $$\left\{ r\sups{t}(\pi)\in[0,1]: \pi\in [A], t\in[T]\right\}$$
for each action and time step is arbitrary and fixed ahead of the interaction by an oblivious adversary. Since we do not posit a stochastic model for rewards, we define regret as in \eqref{eq:regret_realized_rewards}.

As one might expect,
the solution for non-stochastic bandits leverages the online learning tools introduced in \cref{sec:ol}. We consider an algorithm known as \expthree (\emph{Exponential Weights for
Exploration and Exploitation}) \citep{auer2002non}, which builds upon the exponential weights
algorithm for online learning in \cref{sec:ol}. To make the connection as clear as possible, we make a temporary switch from rewards to losses, mapping $r\ind{t}$ to $1-r\ind{t}$, a transformation that does not change the problem itself.\loose

\begin{algbox}[frametitle={\algtitle{alg:exp3}{\expthree}}]
  \begin{algorithmic}
    \Require Learning rate $\eta>0$.
    \State Initialize $p\ind{1}=\unif(\Pi)$.
\For{$t=1,\ldots,T$}
\State Select $\pi\ind{t}\sim p\ind{t}$ and observe loss $\vloss\ind{t}(\pi\ind{t})$.
\State Form the importance-weighted loss estimate $\vlosstilde\ind{t}(\pi)=\vloss\ind{t}(\pi)\indic{\pi\ind{t}=\pi}/p\ind{t}(\pi)$ for all $\pi\in\Pi$.
\State Update $p\ind{t+1}(\pi)\propto p\ind{t}(\pi)\exp\prn*{-\eta\,\vlosstilde\ind{t}(\pi)}$.
\EndFor{}
\end{algorithmic}
\end{algbox}
\begin{prop}[Regret of \expthree]
  \label{prop:exp3}
  For any sequence of losses in $\brk{0,1}^{A}$ fixed in advance, the
  \expthree algorithm with an appropriate choice of $\eta$ ensures that
  \[
    \En\brk*{\Reg}\lesssim\sqrt{AT\log{}A},
  \]
  where the expectation is over the algorithm's randomization.
\end{prop}
The remainder of this subsection sketches the proof.

Recall that $p\ind{t}$ denotes the learner's randomization
distribution at round $t$. Averaging over the sampled actions, the
expected bandit regret is
\begin{align}
    \En\brk*{\Reg} = \En\brk*{\sum_{t=1}^T \inner{p\ind{t},\vloss\ind{t}}}
    - \min_{\pi\in[A]} \sum_{t=1}^T \vloss\ind{t}(\pi),
\end{align}
where $\vloss\ind{t}\in[0,1]^A$ is the vector of losses at time $t$.
The expectation also accounts for the randomness in $p\ind{t}$
induced by past observations.

Since only the loss (equivalently, reward) of the chosen action
$\pi\ind{t}\sim p\ind{t}$ is observed, we cannot directly appeal to
the exponential weights algorithm, which requires knowledge of the
full vector $\vloss\ind{t}$. To address this, we build an
\emph{unbiased estimate} of the vector  $\vloss\ind{t}$ from a single
real-valued observation $\vloss\ind{t}(\pi\ind{t})$. At first, this
might appear impossible, but it is straightforward to show that
\begin{align}
    \vlosstilde\sups{t}(\pi) = \frac{\vloss\sups{t}(\pi)}{p\sups{t}(\pi)}\times \indic{\pi\sups{t}=\pi} 
\end{align}
is an unbiased estimate for all $\pi\in[A]$, or 
in vector notation
\begin{align}
    \En_{\pi\sups{t}\sim p\sups{t}}\brk[\big]{\vlosstilde\sups{t}} = \vloss\sups{t}.
\end{align}
We apply exponential weights to the estimated loss vectors
$\vlosstilde\ind{t}$. As in \pref{rem:online_linear_opt}, the online
learner's regret on these vectors is
\begin{align}
  \RegOL = \sum_{t=1}^T\inner{p\ind{t},\vlosstilde\ind{t}}
    -\min_{\pi\in[A]}\sum_{t=1}^T\vlosstilde\ind{t}(\pi).
\end{align}
These estimated losses are nonnegative but can exceed one, so we
use the second-order exponential weights bound.
This bound holds for each realized sequence of estimated losses;
taking expectations and bounding their second moments gives
$\En\brk{\RegOL}\lesssim\sqrt{AT\log A}$.
Unbiasedness then transfers the guarantee to bandit regret:
\begin{align}
  \En\brk{\Reg}
    &= \En\brk*{\sum_{t=1}^T\inner{p\ind{t},\vlosstilde\ind{t}}}
       -\min_{\pi\in[A]}\En\brk*{\sum_{t=1}^T\vlosstilde\ind{t}(\pi)}\\
    &\leq \En\brk{\RegOL}\lesssim\sqrt{AT\log A}.
\end{align}
The inequality follows by comparing with the
best fixed action for the original loss sequence.\loose

\subsection{Bibliographic Remarks and Further Reading}\label{sec:mab_bib}
\paragraph{Takeaways}
This section introduced the multi-armed bandit as the simplest
interactive decision making problem, and used it to isolate the
exploration--exploitation tradeoff created by partial feedback. The contrast between \egreedy and \ucb---uniform exploration paying
$T^{2/3}$ versus adaptive exploration paying $\sqrt{T}$---reappears in
the contextual bandit setting of \cref{sec:cb}. More broadly, three
ideas recur throughout the monograph. First, optimism: by
\cref{lem:regret_optimistic}, the regret of \ucb is controlled by the
sum of confidence widths. Second, the pigeonhole (or potential)
argument of \cref{lem:confidence_width_potential}, which shows that
the widths cannot remain large for many rounds; it reappears in later
sections as the elliptic potential lemma. Third, the
information-theoretic analysis of posterior sampling, in which a
decoupling inequality (\cref{lem:mab_decoupling_basic}) and an
entropy potential replace confidence intervals altogether; this
argument foreshadows the Decision-Estimation Coefficient of
\cref{sec:general_dm}.\loose

\paragraph{Bibliographic remarks}
The multi-armed bandit problem was introduced by
\citet{thompson1933likelihood} in the context of clinical trials and
given its modern formulation by \citet{robbins1952some}; the name
refers to a slot machine, or ``one-armed bandit.'' Upper confidence
indices originate with \citet{lai1985asymptotically}, who also proved the
asymptotic instance-dependent lower bound; the finite-time \ucb
algorithm and analysis of \cref{prop:ucb}, including the
instance-dependent bound in the remark that follows it, are due to
\citet{auer2002finite}, and the $\Omega(\sqrt{AT})$ minimax lower bound
is from \citet{auer2002non}. Posterior sampling is due to
\citet{thompson1933likelihood}; the information-theoretic analysis
behind \cref{prop:posterior_mab} follows \citet{russo2016information},
building on \citet{russo2014learning}, and frequentist regret bounds
for posterior sampling are due to \citet{agrawal2012analysis} and
\citet{kaufmann2012thompson}. \expthree and the importance-weighted
loss estimator are from \citet{auer2002non}.

\paragraph{Further reading}
Comprehensive treatments of bandits are given by
\citet{lattimore2020bandit}, \citet{slivkins2019introduction}, and
\citet{bubeck2012regret}. Refinements of \ucb include KL-UCB
\citep{garivier2011kl}, which attains the constant in the Lai--Robbins
lower bound, and MOSS \citep{audibert2009minimax}, which removes the
logarithmic factor from the minimax bound. The Bayesian analysis of
posterior sampling extends to adversarial rewards through minimax
duality \citep{bubeck2015bandit,lattimore2019information}, a
perspective we return to in \cref{sec:posterior_and_dec}.

\subsection{Deferred Proofs}
\label{sec:mab_deferred}
\newcommand{\bpi}{\mb{\pi}}
\newcommand{\bpistar}{\mb{\pistar}}
\newcommand{\bfstar}{\mb{\fstar}}
\newcommand{\bcH}{\mb{\cH}}
\newcommand{\br}{\mb{r}}
\begin{proof}[\pfref{prop:posterior_mab} (full version)]%
Let $\En_t\brk*{\cdot}=\En\brk*{\cdot\mid{}\hist\ind{t}}$ and
$\bbP_t\brk{\cdot}=\bbP\brk{\cdot\mid\hist\ind{t}}$. We begin by using
the law of total expectation to express the expected regret as
\[
\RegBayes(\mu) = 
\En\brk*{\sum_{t=1}^{T}\En_{t-1}\brk*{\fstar(\pistar) - \fstar(\pi\ind{t})}}.
\]
Here and throughout the proof, $\En\brk*{\cdot}$ denotes the joint
expectation over both $\Mstar\sim\mu$ and over the sequence
$\hist\ind{T}=(\pi\ind{1},r\ind{1}),\ldots,(\pi\ind{T},r\ind{T})$ that
the algorithm
generates by interacting with $\Mstar$. 

We first simplify the (conditional) expected regret for each step $t$. Let
$\fbar\ind{t}(\pi)\ldef{}\En_{t-1}\brk*{\fstar(\pi)}$
denote the \emph{posterior mean reward} function at time $t$, which
should be thought of as the expected value of $\fstar$ given everything we have
learned so far. Next, let $\fbar\ind{t}_{\pi'}(\pi) =
\En_{t-1}\brk*{\fstar(\pi)\mid\pistar=\pi'}$, which is the expected
reward given the current history, assuming that
$\pistar=\pi'$. We write the expression\loose 
\[
\En_{t-1}\brk*{\fstar(\pistar) - \fstar(\pi\ind{t})}
\]
in terms of these quantities. For the learner's reward, we observe
that  $\fstar$ is conditionally independent of $\pi\ind{t}$
given $\hist\ind{t-1}$, we have
\[
\En_{t-1}\brk*{\fstar(\pi\ind{t})} = \En_{\pi\sim\pt}\brk{\fbar\ind{t}(\pi)}.
\]
For the reward of the optimal action, we begin by writing
\begin{align*}
  \En_{t-1}\brk*{\fstar(\pistar)}
  &= \sum_{\pi\in\Pi}\bbP_{t-1}(\pistar=\pi)
    \En_{t-1}\brk*{\fstar(\pi)\mid\pistar=\pi} \\
  &= \sum_{\pi\in\Pi}\bbP_{t-1}(\pistar=\pi) \fbar\ind{t}_{\pi}(\pi)
  \\
    &= \En_{\pi\sim{}p\ind{t}}\brk*{\fbar\ind{t}_{\pi}(\pi)},
\end{align*}
where we have used that $p\ind{t}$ was chosen to match the posterior
distribution over $\pistar$. This gives
\[
  \En_{t-1}\brk*{\fstar(\pistar) - \fstar(\pi\ind{t})}
=   \En_{\pi\sim{}p\ind{t}}\brk*{\fbar\ind{t}_{\pi}(\pi) - \fbar\ind{t}(\pi)}.
\]
We now make use of the following decoupling-type inequality, which
follows from \cref{eq:decoupling_general}:
    \begin{equation}
      \label{eq:mab_decoupling}
      \En_{\pi\sim{}p\ind{t}}\brk*{\fbar\ind{t}_{\pi}(\pi) -
        \fbar\ind{t}(\pi)}
\leq \sqrt{A\cdot\En_{\pi,\pistar\sim{}p\ind{t}}\brk*{(\fbar\ind{t}_{\pistar}(\pi)-\fbar\ind{t}(\pi))^2}}.
    \end{equation}
To keep notation as clear as possible going forward, let us use boldface script ($\bpi\ind{t}$,
$\bpistar$, $\bfstar$, $\bcH\ind{t}$) to refer to the abstract random
variables under consideration, and use non-boldface script
($\pi\ind{t}$, $\pistar$, $\fstar$, $\cH\ind{t}$)  to
refer to their realizations. As in the simplified proof, we show that the \rhs in \pref{eq:mab_decoupling} is related to a
  notion of \emph{information gain} (that is, information about
  $\pistar$ acquired at step $t$). Using Pinsker's inequality, we have
  \begin{align*}
    \En_{\pi\ind{t},\pistar\sim{}p\ind{t}}\brk*{(\fbar\ind{t}_{\pistar}(\pi\ind{t})-\fbar\ind{t}(\pi\ind{t}))^2}
    \leq{} \En_{t-1}\brk*{
    \Dkl{\bbP_{\br\ind{t}\mid{}\pistar,\pi\ind{t},\cH\ind{t-1}}}{\bbP_{\br\ind{t}\mid{}\pi\ind{t},\cH\ind{t-1}}}
    }.
  \end{align*}
  Since KL divergence satisfies
  $\En_{X}\brk*{\Dkl{\bbP_{Y\mid{}X}}{\bbP_Y}}=\En_{Y}\brk*{\Dkl{\bbP_{X\mid{}Y}}{\bbP_X}}$,
  this is equal to
  \begin{align}
    \En_{t-1}\brk*{
    \Dkl{\bbP_{\bpistar\mid{}\pi\ind{t},r\ind{t},\cH\ind{t-1}}}{\bbP_{\bpistar\mid{}\cH\ind{t-1}}}
    } =
    \En_{t-1}\brk*{
    \Dkl{\bbP_{\bpistar\mid{}\cH\ind{t}}}{\bbP_{\bpistar\mid{}\cH\ind{t-1}}}
    }.\label{eq:mab_info_gain}
  \end{align}
  This is quantifying how much information about $\pistar$ we gain by
  playing $\pi\ind{t}$ and observing $r\ind{t}$ at step $t$, relative to what we knew at step $t-1$.
Applying \pref{eq:mab_decoupling} and \pref{eq:mab_info_gain}, we
have
\begin{align*}
\sum_{t=1}^{T}      \En_{\pi\sim{}p\ind{t}}\brk*{\fbar\ind{t}_{\pi}(\pi) -
        \fbar\ind{t}(\pi)} 
&\leq
                                \sum_{t=1}^{T}\sqrt{A\cdot\En_{\pi,\pistar\sim{}p\ind{t}}\brk*{(\fbar\ind{t}_{\pistar}(\pi)-\fbar\ind{t}(\pi))^2}}\\
  &\leq \sum_{t=1}^{T}\sqrt{A\cdot \En_{t-1}\brk*{
    \Dkl{\bbP_{\bpistar\mid{}\cH\ind{t}}}{\bbP_{\bpistar\mid{}\cH\ind{t-1}}}
    }} \\
    &\leq \sqrt{AT\cdot \sum_{t=1}^{T}\En\brk*{
      \Dkl{\bbP_{\bpistar\mid{}\cH\ind{t}}}{\bbP_{\bpistar\mid{}\cH\ind{t-1}}}
    }}.
\end{align*}
We can write%
\[ 
\En\brk*{
    \Dkl{\bbP_{\bpistar\mid{}\cH\ind{t}}}{\bbP_{\bpistar\mid{}\cH\ind{t-1}}}
  }
= \Ent(\bpistar\mid{}\bcH\ind{t-1}) - \Ent(\bpistar\mid\bcH\ind{t}),
\]
so telescoping gives
\[
  \sum_{t=1}^{T}\En\brk*{\Dkl{\bbP_{\bpistar\mid{}\cH\ind{t}}}{\bbP_{\bpistar\mid{}\cH\ind{t-1}}}}
  = \Ent(\bpistar\mid{}\bcH\ind{0}) - \Ent(\bpistar\mid{}\bcH\ind{T}) \leq \log(A).
\]
\end{proof}

\section{Contextual Bandits: Exploration and Generalization}
\label{sec:cb}

The multi-armed bandit of \cref{sec:mab} assumes that a single decision
is best for everyone. Few applications look like this: a treatment that
is right for one patient is wrong for another. This section adds
\emph{contexts}: before each decision, the learner observes side
information, and the goal is to compete with the best \emph{policy}
mapping contexts to decisions. In the language of \cref{fig:axes}, this
is the first step along the function approximation axis, and the central
question is whether the statistical efficiency of supervised
learning can be retained when the learner must also explore. The answer is yes, but not
by optimism alone, and introduces a reduction from
decision making to online regression that resolves the problem for every
class. This reduction, the Estimation-to-Decisions principle, is the
organizing idea of the rest of the monograph.

\paragraph{Organization}
\cref{sec:cb_setting} formalizes the contextual bandit problem, the regret criterion, and
the role of function approximation. \cref{sec:optimism_template} adapts
the principle of optimism to contextual bandits, replacing confidence
intervals by confidence sets built from least squares. \cref{sec:linucb}
instantiates the template for linear classes (\linucb), and
\cref{sec:beyond_linear} complements this by showing that optimistic methods can have poor performance for general classes.

\cref{sec:online_regression} introduces the \emph{online regression oracle}, the main abstraction through which we reduce decision making to supervised learning. \cref{sec:eps_greedy_cb} analyzes the simplest algorithm built on it, \egreedy, as a first attempt. This culminates in \cref{sec:igw}, which gives an optimal reduction, Inverse Gap Weighting and the \squarecb algorithm.\loose

\subsection{The Contextual Bandit Problem}
\label{sec:cb_setting}
The simplicity of the multi-armed bandit problem comes at a cost: few
real-world problems can be modeled as a multi-armed bandit directly. For example, for the
problem of selecting medical treatments, the multi-armed bandit
formulation presupposes that one treatment rule (action/decision) is
good for all patients, which is clearly unreasonable. To address this,
we augment the problem formulation by allowing the decision-maker to
select the action $\pi\ind{t}$ after observing a \emph{context}
$x\ind{t}$; this is called the \emph{contextual bandit} problem.
The context $x\ind{t}$, which may also be thought of as a feature
vector or collection of covariates (e.g., a patient's medical history,
or the profile of a user arriving at a website), can be used by the
learner to better maximize rewards by tailoring decisions to the
specific patient or user under consideration.

\begin{algbox}[frametitle={\algprotocol{Contextual Bandit Protocol}}]
  \begin{algorithmic}
\For{$t=1,\ldots,T$}
\State Observe context $x\sups{t}\in\cX$.
\State Select decision $\pi\sups{t}\in\Pi=\{1,\ldots,A\}$.
\State Observe reward $r\sups{t}\in\bbR$.
\EndFor{}
\end{algorithmic}
\end{algbox}
The protocol proceeds in $T$ rounds. At each round $t$, the learner
observes a context $x\ind{t}\in\cX$, selects a decision
$\pi\ind{t}\in\Pi$---possibly at random, $\pi\ind{t}\sim{}p\ind{t}$,
where the distribution $p\ind{t}\in\Delta(\Pi)$ may depend on
$x\ind{t}$ and on the history
$\cH\ind{t-1}=(x\ind{1},\pi\ind{1},r\ind{1}),\ldots,(x\ind{t-1},\pi\ind{t-1},r\ind{t-1})$---and
observes a reward $r\ind{t}$; see \cref{fig:cb}.\loose
\begin{figure}
    \centering
    \includegraphics[width=0.8\textwidth]{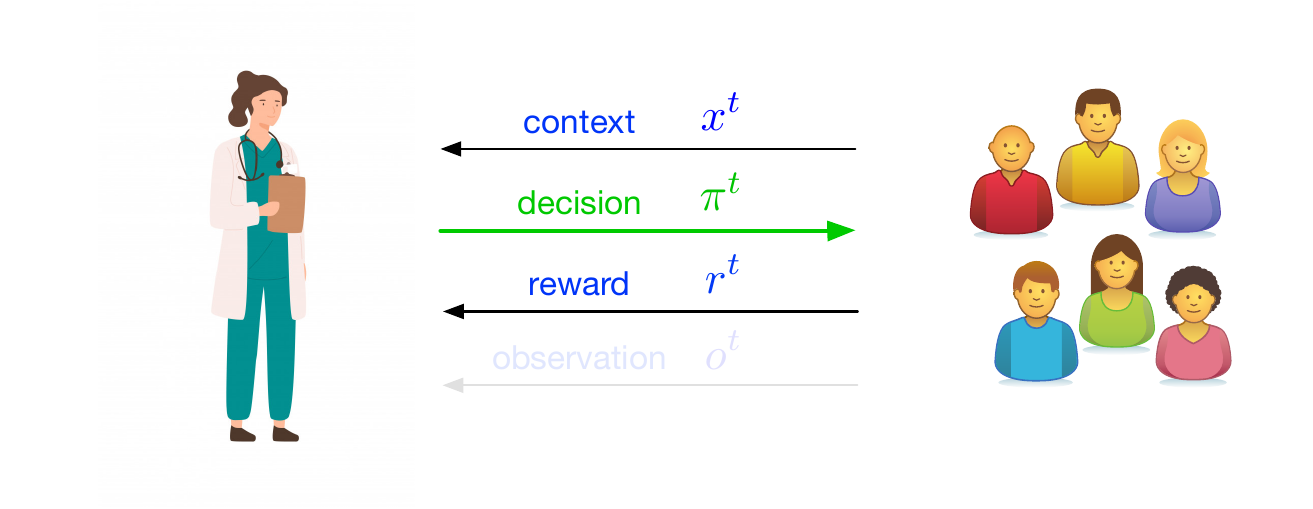}
    \caption{The contextual bandit problem. A doctor (the learner)
      selects a treatment (the decision) based on the patient's context
      (medical history, symptoms).}
    \label{fig:cb}
\end{figure}

As with multi-armed bandits, contextual bandits can be studied in a
stochastic framework or in an adversarial framework. In this monograph,
we allow the contexts $x\ind{1},\ldots,x\ind{T}$ to be generated
in an arbitrary, potentially adversarial fashion, but assume that rewards are
generated from a fixed conditional distribution.
\begin{assumption}[Stochastic Rewards]
  \label{asm:stochastic_rewards_CB}
  Rewards are generated independently via
    \begin{align}
        r\sups{t}\sim \Mstar(\cdot\mid x\ind{t}, \pi\sups{t}),
    \end{align}
    where $\Mstar(\cdot\mid\cdot,\cdot)$ is the underlying \emph{model} (or
    conditional distribution).
  \end{assumption}
This generalizes the stochastic multi-armed bandit framework in
\pref{sec:mab}. As in \cref{sec:mab}, we assume throughout this section
that rewards are bounded, $r\ind{t}\in\brk{0,1}$ almost surely, and
regret bounds hold with probability at least $1-\delta$ unless stated
otherwise; all results hold for arbitrary (possibly adaptively chosen)
context sequences. We define
  \begin{align}
    \label{eq:cb_mean_reward}
  \fstar(x,\pi) \ldef \En\left[r\mid{} x, \pi\right]    
  \end{align}
  as the mean reward function under $r\sim\Mstar(\cdot\mid x,
  \pi)$, and define $\pistar(x) \ldef\argmax_{\pi\in\Pi}\fstar(x,
  \pi)$ as the optimal \emph{policy}, which maps each context $x$ to
  the optimal action for the context. Unlike the multi-armed bandit, the
  reward landscape now varies with the context, so no single action is
  best; the natural benchmark is a policy rather than an action. We measure performance via regret relative
  to $\pistar$:
  \begin{align}
    \label{eq:regret_cb}
    \Reg \ldef{} \sum_{t=1}^{T}\fstar(x\ind{t}, \pistar(x\ind{t})) - \sum_{t=1}^{T}\En_{\pi\ind{t}\sim{}p\ind{t}}\brk{\fstar(x\ind{t}, \pi\ind{t})},
  \end{align}
where, as above, $p\ind{t}$ is the learner's action distribution at step $t$. This provides a (potentially) much stronger notion of performance than
what we considered for the multi-armed bandit: Rather than competing
with the reward of the single best action, we are competing with
the reward of the best sequence of decisions tailored to the context sequence
we observe.

\begin{rem}[Contextual bandits versus reinforcement learning]
  To readers already familiar with reinforcement learning, the
  contextual bandit setting may appear quite similar at first
  glance, with the term ``context'' replacing ``state''. The key
  difference is that in reinforcement learning, we aim to control the
  evolution of $x\ind{1},\ldots,x\ind{T}$ (which is why they are
  referred to as state), whereas in contextual bandits, we take the
  sequence as a given, and only aim to maximize our rewards conditioned
  on the sequence.
\end{rem}
  
\paragraph{Function approximation and the goal}
As a first idea, if $\cX$, the set of possible contexts, is finite, one can imagine
running a separate multi-armed bandit (MAB) algorithm for each context, updating each algorithm only when its corresponding context is encountered. This leads to the following guarantee (the proof is deferred to \cref{sec:cb_deferred}).
\begin{prop}[Unstructured contextual bandits]
  \label{prop:unstructured_cb}
  Suppose $|\cX|=K<\infty$, $|\Pi|=A<\infty$, and rewards lie in $[0,1]$. If we run a separate \ucb instance for each context, then for any (potentially adaptive) context sequence,
  \[
    \En[\Reg]\lesssim\sqrt{AKT\log(eAKT)}.
  \]
\end{prop} We see that the regret bound scales with $\abs{\cX}$, an
undesirable property which reflects the fact that this approach does
not allow for generalization across contexts. Instead, we would like
to share information between different contexts. After all, a doctor
prescribing treatments might never observe exactly the same medical
history and symptoms twice, but they might see similar patients or
recognize underlying patterns. In the spirit of statistical learning
(\pref{sec:background}) this means assuming access to a class $\cF$ that
can model the mean reward function, and aiming for regret bounds that
scale with $\log\abs{\cF}$ (reflecting the statistical capacity of $\cF$), with \emph{no dependence} on the cardinality of $\cX$. To
facilitate this, we make the following realizability assumption.
\begin{assumption}[Realizability]
    The decision-maker has access to a class $\cF\subset\{f:\cX\times\Pi\to\reals\}$ such that $\fstar\in\cF$.
\end{assumption} 
In keeping with the reward convention, we also assume that
$f(x,\pi)\in\brk{0,1}$ for all $f\in\cF$ unless otherwise specified.
Using the class $\cF$, we would like to develop algorithms that can
model the underlying reward function for better decision making
performance. Concretely, the goal of this section is regret that scales with
$\log\abs{\cF}$ and has no dependence on $\abs{\cX}$. The best possible
rate of this form, $\sqrt{AT\log\abs{\cF}}$, is attained in
\cref{sec:igw}; for linear classes, the number of actions $A$ is replaced
by the dimension $d$ (\cref{sec:linucb}). In the language of
\cref{fig:axes}, contexts introduce generalization across $\cX$ through
function approximation, but the decision space remains unstructured.

With our goal in mind, it is natural to try
leveraging the algorithms and respective guarantees we have already
developed for statistical and online supervised learning. At this point,
however, the decision-making problem---with its
exploration-exploitation dilemma---appears to be quite distinct from these supervised learning 
frameworks. Indeed, naively applying supervised learning methods,
which do not account for the interactive nature of the problem, can
lead to failure, as we saw with the greedy algorithm in \cref{sec:need_for_exploration}. In spite of these apparent difficulties, the remainder of
this section shows that it is possible to leverage supervised
learning methods to develop provable decision making methods, thereby
bridging the two methodologies.

\subsection{Optimism for Contextual Bandits: A Generic Template}
\label{sec:optimism_template}
What algorithmic principles should we employ to solve the
contextual bandit problem? One approach is to adapt
solutions from the multi-armed bandit setting. There, we saw
that the principle of \emph{optimism} (in particular, the \ucb
algorithm) led to (nearly) optimal rates for bandits, so a natural
question is whether optimism can be adapted to give optimal guarantees in the presence of contexts. The answer to this last question is: \emph{it
depends}. We first describe some positive results under
assumptions on $\cF$ (\cref{sec:optimism_template,sec:linucb}), then provide a negative example (\cref{sec:beyond_linear}), and finally turn to an entirely different algorithmic principle (\cref{sec:igw}).

\paragraph{Optimism via confidence sets}

Let us describe a general approach (or, template) for applying
the principle of optimism to contextual bandits \citep{chu2011contextual,abbasi2011improved,russo2013eluder,foster2018practical}. Suppose that at each time, we have a
way to construct
a \emph{confidence set}
\[
\cF\ind{t} \subseteq\cF
\]
based on the data observed so far, with the important property that
$\fstar\in\cF\ind{t}$. Given such a confidence set we can define upper
and lower \emph{confidence functions} $\flcb\sups{t}, \fucb\sups{t}:\cX\times\Pi\to\reals$
via
\begin{align}
	\label{eq:upper_lower_envelopes}
    \flcb\sups{t}(x,\pi) = \min_{f\in\cF\ind{t}} f(x,\pi),~~~ \fucb\sups{t}(x,\pi) = \max_{f\in\cF\ind{t}} f(x,\pi).
\end{align}
These functions generalize the upper and lower confidence bounds we
constructed in \cref{sec:mab}. Since $\fstar\in\cF\ind{t}$, they have the property that
\begin{align}
  \label{eq:fstar_in_CI}
\flcb\ind{t}(x,\pi) \leq \fstar(x,\pi) \leq \fucb\ind{t}(x,\pi)
\end{align}
for all $x\in\cX, \pi\in\Pi$. As such, if we consider a contextual
analogue of the \ucb algorithm, given by
\begin{equation}
    \pi\ind{t}=\argmax_{\pi\in\Pi}\fucb\ind{t}(x\ind{t},\pi),
  \end{equation}
  then as in \pref{lem:regret_optimistic}, the optimistic action satisfies 
    $$\fstar(x\ind{t},\pistar(x\ind{t})) - \fstar(x\ind{t},\pi\sups{t}) \leq
    \fucb\sups{t}(x\ind{t},\pi\sups{t}) -
    \flcb\sups{t}(x\ind{t},\pi\sups{t}).$$
That is, the suboptimality is bounded by the width of the confidence
interval at $(x\ind{t},\pi\sups{t})$, and the total regret is bounded as
\begin{equation}
  \label{eq:width_cb}
\Reg \leq \sum_{t=1}^{T}\fucb\sups{t}(x\ind{t},\pi\sups{t}) -
    \flcb\sups{t}(x\ind{t},\pi\sups{t}).
\end{equation}
To make this approach concrete and derive sublinear bounds on the
regret, we need a way to construct the
confidence set $\cF\ind{t}$, ideally so that the width in
\pref{eq:width_cb} shrinks as fast as possible. 

\textbf{How should we construct the confidence set?}~~
Recall from \eqref{eq:ucb_confidence_bound} that the optimistic approach for multi-armed bandits is to select an action $\pi\ind{t}$ that maximizes $\fucb\ind{t}(\pi) = \fhat\ind{t}(\pi) + \frac{\beta}{\sqrt{n\sups{t}(\pi)}}
$ for some $\beta$ that depends on $T,A,\delta$. What are the corresponding quantities for our setting of contextual bandits with a class $\cF$? Unfortunately, since in the present formulation we may never observe the same $x\ind{t}$ more than once, the corresponding meaning of $\fhat\ind{t}(x,\pi)$ as an average of observed rewards makes no sense. It turns out that appealing to the variational representation of averages as minimizers of the squared loss gracefully resolves this conundrum. The idea is as follows. 

Denote by $S\ind{t}(\pi)$ the set of rewards obtained when action $\pi$ was chosen prior to time $t$, with $|S\ind{t}(\pi)|=n\ind{t}(\pi)$. As mentioned, we may write the mean reward as $$\fhat\ind{t}(\pi) = \argmin_{v\ind{\star}\in\reals} \sum_{r\in S\ind{t}(\pi)} (v\ind{\star}-r)^2.$$
Similarly, the confidence set 
\begin{align}
\label{eq:conf_interval_reals_0}
\crl*{v\in\reals: \abs{v- \fhat\ind{t}(\pi)}\leq \frac{\beta}{\sqrt{n\sups{t}(\pi)}}}
\end{align}
can be shown to coincide with 
\begin{align}
\label{eq:conf_interval_reals}
\crl*{v\in\reals: \sum_{r\in S\ind{t}(\pi)}\prn*{v- r}^2 - \min_{v\ind{\star}\in\reals} \sum_{r\in S\ind{t}(\pi)}\prn*{v\ind{\star}- r}^2 \leq \beta^2},
\end{align}
the set of almost-minimizers of the mean squared error.
To verify this when $n=n\ind{t}(\pi)>0$, let
$\bar r \ldef n^{-1}\sum_{r\in S\ind{t}(\pi)}r$. Since
$\sum_{r\in S\ind{t}(\pi)}(r-\bar r)=0$, expansion gives
\[
  \sum_{r\in S\ind{t}(\pi)}(v-r)^2
    =n(v-\bar r)^2+\sum_{r\in S\ind{t}(\pi)}(\bar r-r)^2.
\]
The minimum is attained at $v=\bar r$, so the excess squared loss is
$n(v-\bar r)^2$. Its sublevel set at $\beta^2$ is exactly
\eqref{eq:conf_interval_reals_0}. When $n=0$, the loss in
\eqref{eq:conf_interval_reals} is identically zero and the confidence
set is all of $\reals$.
As we show next, these re-formulations open the door to constructing the confidence sets $\cF\ind{t}$.

\paragraph{Constructing confidence sets with least squares}
We construct confidence sets by appealing to a supervised learning method, empirical risk minimization with the
square loss (or, least squares). Recall that $f(x,\pi)\in\brk{0,1}$ for
all $f\in\cF$ and $r\ind{t}\in[0,1]$ almost surely. Let
\begin{align}
\label{eq:erm_t}
	\fhat\ind{t} = \argmin_{f\in\cF} \sum_{i=1}^{t-1} (f(x\ind{i},\pi\ind{i})-r\ind{i})^2
\end{align}
be the empirical risk minimizer at round $t$, and with $\beta^2\ldef{}8\log(|\cF|/\delta)$ define $\cF\ind{1}=\cF$ and
\begin{align}
	\label{eq:def_Ft_confidence_ball}
	\cF\ind{t} = \crl*{f\in\cF: \sum_{i=1}^{t-1} (f(x\ind{i},\pi\ind{i}) - r\ind{i})^2 \leq \sum_{i=1}^{t-1} (\fhat\ind{t}(x\ind{i},\pi\ind{i}) - r\ind{i})^2 + \beta^2}
\end{align}
for $t>1$. These definitions are the bona fide generalizations of \cref{eq:conf_interval_reals_0,eq:conf_interval_reals}, aiming to employ the structure of functions in $\cF$ and enable generalization across the covariates and decisions. Observe that the optimization problems in \cref{eq:erm_t,eq:def_Ft_confidence_ball} pull together information across all the triples $(x\ind{i},\pi\ind{i}, r\ind{i})$ instead of bucketing the data according to a particular context.  

In summary, our confidence set $\cF\ind{t}$ is the collection of all
functions in $\cF$ that have empirical squared error close to that of
$\fhat\ind{t}$. We set $\beta$ ``just
large enough'' to ensure that we do not accidentally exclude
$\fstar$---a property necessary for establishing \cref{eq:fstar_in_CI}---with the precise value for $\beta$ informed by the probabilistic 
inequalities we explored in \pref{sec:background}. The only catch here is that
we need to use variants of these inequalities that handle dependent data,
since $x\ind{t}$ and $\pi\ind{t}$ are not \iid in \cref{eq:erm_t}. The following result
shows that $\cF\ind{t}$ is indeed valid and, moreover, that all
functions $f\in\cF\ind{t}$ have low estimation error on the history.

\begin{lem}[Validity of least-squares confidence sets]
	\label{lem:valid_CI}
	Let $\pi\ind{1},\ldots,\pi\ind{T}$ be chosen by an arbitrary (and possibly randomized) decision-making algorithm. 
	With probability at least $1-\delta$, $\fstar\in\cF\ind{\tau}$ for all $\tau\in[T]$. Moreover, with 	probability at least $1-\delta$, for all $\tau\in [T]$, all $f\in\cF\ind{\tau}$ satisfy
	\begin{align}
          \label{eq:confidence_set_estimation}
		\sum_{t=1}^{\tau-1} \En_{\pi\ind{t}\sim{}p\ind{t}}\brk*{(f(x\ind{t},\pi\ind{t})-\fstar(x\ind{t},\pi\ind{t}))^2} \leq 4\beta^2,
	\end{align}
  where $\beta^2=8\log(|\cF|/\delta)$.
\end{lem}

Importantly, the definition of $\cF\ind{t}$ in \cref{eq:def_Ft_confidence_ball} does not involve the unknown $\fstar$, yet \cref{lem:valid_CI}---the statistical consequence of least squares---guarantees closeness to $\fstar$.

\begin{proof}[\pfref{lem:valid_CI}]
For $f\in\cF$, define
\begin{align}
	U\ind{t}(f) = (f(x\ind{t},\pi\ind{t}) - r\ind{t})^2 - (\fstar(x\ind{t},\pi\ind{t}) - r\ind{t})^2. 
\end{align}
It is straightforward to check that\footnote{We leave $\En_{t-1}$ on the right-hand side to include the case of randomized decisions $\pi\ind{t}\sim p\ind{t}$.}
\begin{align}
	\label{eq:cond_exp_square_loss_difference}
	\En_{t-1} U\ind{t}(f) = \En_{t-1}(f(x\ind{t},\pi\ind{t})-\fstar(x\ind{t},\pi\ind{t}))^2,
\end{align}
where $\En_{t-1}\brk{\cdot}\ldef{}\En\brk{\cdot\mid\cH\ind{t-1}, x\ind{t}}$. Then $Z\ind{t}(f) = \En_{t-1} U\ind{t}(f) - U\ind{t}(f)$ is a martingale difference sequence and 
$\sum_{t=1}^\tau Z\ind{t}(f)$ is a martingale. Because $x\ind{t}$ and
$\pi\ind{t}$ depend on the past, Bernstein's inequality for independent
variables does not apply; we use instead its martingale analogue,
Freedman's inequality (\pref{lem:freedman}), which appears here for the
first time and is used repeatedly in later sections. Since increments
$Z\ind{t}(f)$ are bounded as $\abs{Z\ind{t}(f)}\leq 2$ (this holds whenever $f\in[0,1], r\ind{t}\in[0,1]$), according to \pref{lem:freedman}
with $\eta=\frac{1}{8}$, with probability at least $1-\delta$, for all
$\tau\leq T$,
\begin{align}
	\label{eq:freedman_application}
      \sum_{t=1}^{\tau}Z\ind{t}(f) \leq{} \frac{1}{8}\sum_{t=1}^{\tau}\En_{t-1}\brk*{Z\ind{t}(f)^{2}} + 8\log(\delta^{-1}).
\end{align}
To control the right-hand side, we again use that
$f,r\ind{t}\in\brk{0,1}$ to bound
\begin{align}
	\En_{t-1}\brk*{Z\ind{t}(f)^{2}} &\leq \En_{t-1}\brk*{\prn*{(f(x\ind{t},\pi\ind{t}) - r\ind{t})^2 - (\fstar(x\ind{t},\pi\ind{t}) - r\ind{t})^2}^2} \\
	&\leq 4\En_{t-1}\brk*{ (f(x\ind{t},\pi\ind{t}) - \fstar(x\ind{t},\pi\ind{t}))^2 } = 4\En_{t-1} U\ind{t}(f).
\end{align}
Then, after rearranging, \eqref{eq:freedman_application} becomes
\begin{align}
	\label{eq:freedman_application2}
      \frac{1}{2}\sum_{t=1}^{\tau} \En_{t-1} U\ind{t}(f) \leq \sum_{t=1}^{\tau} U\ind{t}(f)  + 8\log(\delta^{-1}).
\end{align}
Since the left-hand side is nonnegative, we conclude that with probability at least $1-\delta$,
\begin{align}
	\sum_{t=1}^{\tau} (\fstar(x\ind{t},\pi\ind{t}) - r\ind{t})^2 \leq \sum_{t=1}^{\tau} (f(x\ind{t},\pi\ind{t}) - r\ind{t})^2 + 8\log(\delta^{-1}).
\end{align}
Taking a union bound over $f\in \cF$, gives that with probability at least $1-\delta$,
\begin{align}
	\forall f\in\cF,\; \forall \tau\in [T],~~~~ \sum_{t=1}^{\tau} (\fstar(x\ind{t},\pi\ind{t}) - r\ind{t})^2 \leq \sum_{t=1}^{\tau} (f(x\ind{t},\pi\ind{t}) - r\ind{t})^2 + 8\log(|\cF|/\delta),
\end{align}
and in particular, shifting index $\tau$ by $1$ for brevity, and applying to the data-dependent $\fhat\ind{\tau}$,
\begin{align}
	\forall \tau\in [T+1],~~~~ \sum_{t=1}^{\tau-1} (\fstar(x\ind{t},\pi\ind{t}) - r\ind{t})^2 \leq \sum_{t=1}^{\tau-1} (\fhat\ind{\tau}(x\ind{t},\pi\ind{t}) - r\ind{t})^2 + 8\log(|\cF|/\delta);
\end{align}
that is, we have $\fstar\in\cF\ind{\tau}$ for all $\tau\in\{1,\ldots,T+1\}$, proving the first claim. For the second part of the claim, observe that any $f\in\cF\ind{\tau}$ must satisfy
$$\sum_{t=1}^{\tau-1} U\ind{t}(f) \leq \beta^2$$
since the empirical risk of $\fstar$ is never better than the
empirical risk of the minimizer $\fhat\ind{t}$. Thus from \eqref{eq:freedman_application2}, with probability at least $1-\delta$, for all $\tau\leq T$,
\begin{align}
	\sum_{t=1}^{\tau-1} \En_{t-1} U\ind{t}(f) \leq 2\beta^2  + 16\log(\delta^{-1}).
\end{align}
The second claim follows by taking union bound over $f\in\cF\ind{\tau}\subseteq \cF$, and by \eqref{eq:cond_exp_square_loss_difference}. 
\end{proof}

\pref{lem:valid_CI} is valid for any algorithm, but it is particularly useful for \ucb as it establishes the validity of the
      confidence bounds as per \eqref{eq:fstar_in_CI}; however, it is not yet enough to show
      that the algorithm attains low regret. Indeed, to bound the
      regret, we need to control the confidence widths in
      \pref{eq:width_cb}, but there is a mismatch: for step $\tau$, the
      regret bound in \pref{eq:width_cb} considers the width at
      $(x\ind{\tau},\pi\ind{\tau})$, but \pref{eq:confidence_set_estimation}
      only ensures closeness of functions in $\cF\ind{\tau}$ under
      $(x\ind{1},\pi\ind{1}),\ldots,(x\ind{\tau-1},\pi\ind{\tau-1})$. \cref{sec:linucb} shows that for linear models, it is possible
      to control this mismatch; \cref{sec:beyond_linear} shows that this is not possible in general.

\subsection{Optimism for Linear Models: The LinUCB Algorithm}
\label{sec:linucb}

We now instantiate the general template for optimistic algorithms
developed in the
previous section for the special case where $\cF$ is a class of
\emph{linear functions}.

\paragraph{Linear models}
We fix a \emph{feature map} $\phi:\cX\times \Pi\to \sB_2^d(1)$, where
$\sB_2^d(1)$ is the unit-norm Euclidean ball in $\reals^d$. The
feature map is assumed to be known to the learning agent. For example,
in the case of medical treatments, $\phi$ transforms the medical
history and symptoms $x$ for the patient, along with a possible
treatment $\pi$, to a representation $\phi(x,\pi)\in\sB_2^d(1)$. In practice, feature
maps come from domain knowledge or from a representation learned
offline (for example, the penultimate layer of a pretrained neural
network); for the analysis, all that matters is that $\phi$ is fixed and
known. We take $\cF$ to be the set of linear functions given by
\begin{align}
\label{eq:lin_class}
    \cF=\crl*{(x,\pi)\mapsto\inner{\theta,\phi(x,\pi)}\mid{}\theta\in\Theta},
\end{align}
where $\Theta\subseteq \sB_2^d(1)$ is the parameter set. As before, we assume $\fstar\in\cF$; we let $\theta^*$ denote the
corresponding parameter vector, so that $\fstar(x,\pi)=\tri*{\thetastar,\phi(x,\pi)}$. With some abuse of notation, we associate the set of parameters $\Theta$ to the corresponding functions in $\cF$.

\paragraph{The \linucb algorithm}
The following figure displays an algorithm we refer to as \linucb \citep{auer2002using,chu2011contextual,abbasi2011improved},
which adapts the generic template for optimistic algorithms to
the case where $\cF$ is linear in the sense of \eqref{eq:lin_class}.
\begin{algbox}[frametitle={\algtitle{alg:linucb}{\linucb}}]
  \begin{algorithmic}
\Require Confidence parameter $\beta>0$.
\For{$t=1,\ldots,T$}
\State Compute the least squares solution $\widehat{\theta}\ind{t}$
(over $\theta\in\Theta$) given by
\[
\thetahat\ind{t}=\argmin_{\theta\in\Theta}\sum_{i=1}^{t-1}\prn*{\tri{\theta,\phi(x\ind{i},\pi\ind{i})}-r\ind{i}}^2.
\]
\State Define
$$\widetilde{\Sigma}\ind{t}=\sum_{i=1}^{t-1}\phi(x\ind{i},\pi\ind{i})\phi(x\ind{i},\pi\ind{i})^\tr + I.$$
\State Given $x\ind{t}$, select action 
$$\pi\ind{t} \in \argmax_{\pi\in\Pi}~
\tri[\big]{\widehat{\theta}\ind{t}, \phi(x_t,\pi)}  + \sqrt{16\beta^2+4} \norm{ \phi(x_t,\pi) }_{(\widetilde{\Sigma}\ind{t})^{-1}}.
$$
\State Observe reward $r\sups{t}$
\EndFor{}
\end{algorithmic}
\end{algbox}
\linucb instantiates the template of \cref{sec:optimism_template} with the
least-squares confidence set \cref{eq:def_Ft_confidence_ball}: for linear
classes this set is contained in an ellipsoid centered at
$\thetahat\ind{t}$, and maximizing a linear function over an ellipsoid
gives the closed-form upper confidence bound in the last step of the
algorithm; the proof below makes this precise. The following result shows that \linucb enjoys a regret bound that
scales with the complexity $\log\abs{\cF}$ of the model class and
the feature dimension $d$.\footnote{The theorem allows rewards in $[-1,1]$ rather than the $[0,1]$
convention of this section, because linear predictions with $\theta$
and $\phi$ in the unit ball take values in $[-1,1]$. \pref{lem:valid_CI}
applies after rescaling, with $\beta^2=32\log(\abs{\cF}/\delta)$.\loose}

\begin{theorem}[Regret of \linucb]
  \label{prop:linucb}
  Let $\Theta\subseteq\sB_2^d(1)$ be a nonempty finite set, let
  $\theta^*\in\Theta$, and fix $\phi:\cX\times\Pi\to\sB_2^d(1)$.
  Assume $r\ind{t}\in[-1,1]$ almost surely, and let $T\geq1$ and
  $\delta\in(0,1)$. Break ties using a fixed rule. Let $\cF$ be the corresponding class
  of linear functions in \eqref{eq:lin_class}, and take
  $\beta^2=32\log(|\cF|/\delta)$.
  Then \linucb satisfies, with probability at least $1-\delta$,
  \[
    \Reg\lesssim\sqrt{(\beta^2+1)dT\log(1+T/d)}
      \lesssim\sqrt{dT\log(1+T/d)\bigl(1+\log(|\cF|/\delta)\bigr)}
  \]
  for any potentially adaptive sequence of contexts.
  More generally, for any compact, potentially infinite set
  $\Theta\subseteq\sB_2^d(1)$ containing $\theta^*$, the same first bound
  holds with
  \[
    \beta^2=32\bigl[d\log(1+2T)+\log(1/\delta)\bigr]+6.
  \]
  In particular, with probability at least $1-\delta$,
  \[
    \Reg\lesssim\sqrt{dT\log(1+T/d)
      \bigl[d\log(1+T)+\log(1/\delta)+1\bigr]}.
  \]
\end{theorem}
Notably, this regret bound has no explicit dependence on the context space size
$\abs{\cX}$. Interestingly, the bound is also independent of the
number of actions $\abs{\Pi}$, which is replaced by the dimension $d$; this reflects that the linear structure
of $\cF$ allows the learner to generalize not just across contexts, but across
decisions. \cref{sec:structured} expands upon the idea of generalizing across actions.

\begin{remark}[Connection to \ucb for multi-armed bandits]
  \label{ex:linucb_ucb_connection}
  We can embed the multi-armed bandit setting as a special case of linear contextual bandits by considering a single context $x_0$ and one-hot features $\phi(x_0,a)=e_a$, so that $d=A$. Let $N_a(t)$ count selections of action $a$ before time $t$. Then
  \[
    \Sigma\ind{t}=\operatorname{diag}(N_1(t),\ldots,N_A(t)),\qquad
    \|e_a\|_{(\widetilde\Sigma\ind{t})^{-1}}^2=\frac1{N_a(t)+1}.
  \]
  Consequently, \linucb's score is
  $\widehat\theta\ind{t}_a+\sqrt{16\beta^2+4}/\sqrt{N_a(t)+1}$, and we see that in this setting, the diagonal covariance recovers the inverse-square-root count dependence of \ucb's bonus.
\end{remark}

We prove the here result for finite parameter sets; a straightforward extension to continuous parameter sets uses discretization, and is deferred to
\cref{sec:cb_deferred}.

\begin{proof}[\pfref{prop:linucb}, finite parameter sets]
The confidence set \eqref{eq:def_Ft_confidence_ball} in the generic optimistic algorithm
template is 
\begin{align}
	\cF\ind{t} = \crl*{\theta\in\Theta: \sum_{i=1}^{t-1} (\inner{\theta,\phi(x\ind{i},\pi\ind{i})} - r\ind{i})^2 \leq \sum_{i=1}^{t-1} (\tri{\widehat{\theta}\ind{t},\phi(x\ind{i},\pi\ind{i})} - r\ind{i})^2 + \beta^2},
\end{align}  
where $\widehat{\theta}\ind{t}$ is the least squares solution computed
in \linucb. By \pref{lem:valid_CI}, with the rescaling to $[-1,1]$ described above, with probability at least $1-\delta$, for all $t\in[T]$, all $\theta\in\cF\ind{t}$ satisfy
\begin{align}
	\sum_{i=1}^{t-1} (\inner{\theta-\theta^*,\phi(x\ind{i},\pi\ind{i})})^2 \leq 4\beta^2,
\end{align}
which means that $\cF\ind{t}$ is a subset of\footnote{For a PSD matrix
  $\Sigma\psdgeq{}0$, we define $\nrm{x}_{\Sigma}=\sqrt{\tri*{x,\Sigma{}x}}$.}
\begin{align}
	\Theta' = \crl*{\theta\in\Theta: \norm{\theta-\theta^*}_{\Sigma\ind{t}}^2 \leq 4\beta^2},~~~~\text{where}~~~~ \Sigma\ind{t}=\sum_{i=1}^{t-1}\phi(x\ind{i},\pi\ind{i})\phi(x\ind{i},\pi\ind{i})^\tr.
\end{align}
Since $\widehat{\theta}\ind{t}\in\cF\ind{t}$, we have that for any $\theta\in\Theta'$, by triangle inequality, $\nrm[\big]{\theta-\widehat{\theta}\ind{t}}_{\Sigma\ind{t}}^2 \leq 16\beta^2$.
Furthermore, since $\widehat{\theta}\ind{t}\in \Theta\subseteq\sB_2^d(1)$,
$\nrm[\big]{\theta-\widehat{\theta}\ind{t}}_2 \leq 2$.
Combining the two constraints into one, we find that $\Theta'$ (and, hence, $\cF\ind{t}$) is a subset of
\begin{align}
	\Theta'' = \crl*{\theta\in\reals^d: \nrm[\big]{\theta-\widehat{\theta}\ind{t}}_{\widetilde{\Sigma}\ind{t}}^2 \leq 16\beta^2+4},~~\text{where}~~ \widetilde{\Sigma}\ind{t}=\sum_{i=1}^{t-1}\phi(x\ind{i},\pi\ind{i})\phi(x\ind{i},\pi\ind{i})^\tr + I.
\end{align}
The purpose of establishing the containment $\cF\ind{t}\subseteq \Theta''$ is purely computational, as it gives us a clean upper bound on $\fucb\sups{t}$ in \eqref{eq:upper_lower_envelopes}:
\begin{align}
	\label{eq:linucb_form_with_bonus}
	\fucb\sups{t}(x,\pi) \leq \max_{\theta:  \norm{\theta-\widehat{\theta}\ind{t}}_{\widetilde{\Sigma}\ind{t}} \leq \sqrt{16\beta^2+4}} \inner{\theta, \phi(x,\pi)} = \tri[\big]{\widehat{\theta}\ind{t}, \phi(x,\pi)}  + \sqrt{16\beta^2+4} \norm{ \phi(x,\pi) }_{(\widetilde{\Sigma}\ind{t})^{-1}},
\end{align}
and similarly $\flcb\sups{t}(x,\pi) \geq \tri[\big]{\widehat{\theta}\ind{t}, \phi(x,\pi)}  - \sqrt{16\beta^2+4} \nrm*{ \phi(x,\pi) }_{(\widetilde{\Sigma}\ind{t})^{-1}}$. 
Since $\theta^*\in\Theta''$, the upper and lower envelopes
of $\Theta''$ bracket the true reward at every action. \linucb maximizes
this enlarged upper envelope, so the optimistic-regret argument
(\pref{lem:regret_optimistic}) bounds its regret by the width of
$\Theta''$ at the selected action. Thus,
\begin{align}
	\label{eq:reg_bd_linear_potential}
	\Reg \leq 2\sqrt{16\beta^2+4} \sum_{t=1}^T \norm{ \phi(x\ind{t},\pi\ind{t}) }_{(\widetilde{\Sigma}\ind{t})^{-1}}\lesssim\sqrt{(\beta^2+1)T\sum_{t=1}^T\norm{\phi(x\ind{t},\pi\ind{t})}_{(\widetilde{\Sigma}\ind{t})^{-1}}^2}.
\end{align}
The above upper bound has the same flavor as the one in
\pref{lem:confidence_width_potential}: as we obtain more and more
information in some direction $v$, the matrix
$\widetilde{\Sigma}\ind{t}$ has a larger and larger component in that
direction, and for that direction $v$, the term $\norm{v
}_{(\widetilde{\Sigma}\ind{t})^{-1}}^2$ becomes smaller and
smaller. Unlike the multi-armed bandit case, where each round increments
the count of a single action, a single feature vector
$\phi(x\ind{t},\pi\ind{t})$ can have components in many directions at
once. To conclude, we apply a potential argument,
\cref{lem:elliptic_potential} below, to bound
\begin{align*}
\sum_{t=1}^T \norm{ \phi(x\ind{t},\pi\ind{t}) }_{(\widetilde{\Sigma}\ind{t})^{-1}}^2\lesssim  d\log(1+T/d).
\end{align*}
\end{proof}

The following result is referred to as the elliptic potential lemma, and it can be viewed as a generalization of \pref{lem:confidence_width_potential}.
\begin{lem}[Elliptic potential lemma]
\label{lem:elliptic_potential}
	Let $a_1,\ldots,a_T\in \reals^d$ satisfy $\norm{a_t}\leq 1$
        for all $t\in[T]$, and let $V_t = I + \sum_{s\leq t} a_s a_s^\tr$. Then
	\begin{align}
	    \label{eq:elliptic_potential_bound}
          \sum_{t=1}^T \norm{a_t}^2_{V_{t-1}^{-1}} \leq 2d \log \prn*{1+T/d}.
	\end{align}
        More generally, for any $\lambda>0$, with
        $V^{\lambda}_t=\lambda{}I+\sum_{s<t}a_sa_s^{\trn}$,
	\begin{align}
	    \label{eq:elliptic_potential_log}
          \sum_{t=1}^T \log\prn*{1+\norm{a_t}^2_{(V^{\lambda}_{t})^{-1}}} \leq d \log \prn*{1+\lambda^{-1}T/d}.
	\end{align}
      \end{lem}
\begin{proof}[\pfref{lem:elliptic_potential}]
  Set $u_t=a_t^\tr V_{t-1}^{-1}a_t$. Since $V_{t-1}\succeq I$ and
  $\|a_t\|_2\leq1$, we have $0\leq u_t\leq1$. The matrix determinant
  lemma gives $\det V_t=\det V_{t-1}(1+u_t)$. Also,
  $u\leq2\log(1+u)$ on $[0,1]$, since the derivative of
  $2\log(1+u)-u$ is $(1-u)/(1+u)\geq0$. Therefore
  \[
    \sum_{t=1}^T u_t\leq2\sum_{t=1}^T\log(1+u_t)
      =2\log\frac{\det V_T}{\det V_0}=2\log\det V_T.
  \]
  We further observe that if $\lambda_1,\ldots,\lambda_d$ are the positive eigenvalues of $V_T$,
  the arithmetic--geometric mean inequality implies
  \[
    \det V_T=\prod_{j=1}^d\lambda_j
      \leq\left(\frac{\operatorname{tr}V_T}{d}\right)^d
      =\left(1+\frac1d\sum_{t=1}^T\|a_t\|_2^2\right)^d
      \leq(1+T/d)^d.
  \]
  Combining the two displays proves the claim.
\end{proof}
To summarize, the analysis of \linucb combined two ingredients that
carry forward: confidence sets defined through excess square loss
(\cref{lem:valid_CI}), and a potential argument
(\cref{lem:elliptic_potential}) showing that the confidence widths at
the selected actions cannot remain large. The elliptic potential lemma
reappears in the analysis of \lsviucb in \cref{sec:rl_linear}.

\subsection{Failure of Optimism Beyond Linear Classes}
\label{sec:beyond_linear}

We now present an example of a class $\cF$ for which optimistic
methods necessarily incur regret that scales linearly with either the
cardinality of $\cF$ or with cardinality of $\cX$, meaning that we do
not achieve the desired $\log\abs{\cF}$ scaling of regret that one
might expect in (offline or online) supervised learning.
This construction is from \citet{foster2018practical}.

\begin{example}[Failure of optimism for contextual bandits]
\label{ex:failure_optimism}
  Let $A=2$, and let $N\in\bbN$ be given. Let $\pi\subs{g}$ and $\pi\subs{b}$ be two actions
  available in each context, so that
  $\Pi=\crl{\pi\subs{g},\pi\subs{b}}$. Let $\cX=\{x\subs{1},\ldots,x\subs{N}\}$
  be a set of distinct contexts, and define a class $\cF=\{\fstar,f\subs{1},\ldots,f\subs{N}\}$ of cardinality $N+1$ as follows. Let $\fstar(x,\pi\subs{g})=1/2$ and $\fstar(x,\pi\subs{b})=0$ for any $x\in\cX$. For each $i\in[N]$, $f\subs{i}(x\subs{j},\pi\subs{g})=1/2$ and $f\subs{i}(x\subs{j},\pi\subs{b})=0$ for $j\neq i$, while
        $f\subs{i}(x\subs{i},\pi\subs{g})=0$ and $f\subs{i}(x\subs{i},\pi\subs{b})=1$; see \cref{fig:defector} for an illustration.

    \begin{figure}
        \centering
        \includegraphics[width=0.8\linewidth]{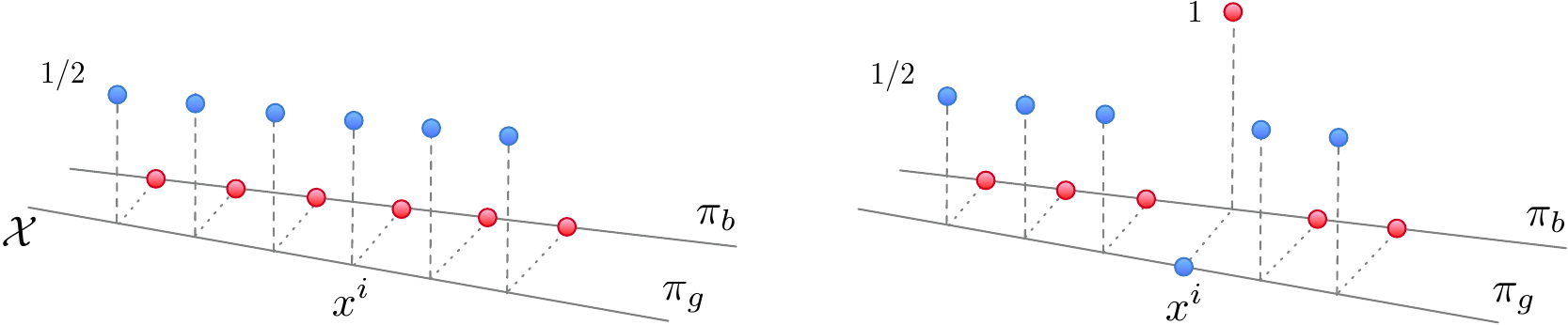}
        \caption{The construction in \cref{ex:failure_optimism}. The true function $\fstar$ (left) and a ``defector'' $f\subs{i}$ (right).}
        \label{fig:defector}
    \end{figure}

        Now, consider a (realizable) problem instance in which rewards are deterministic and given by
        \[r\ind{t}=\fstar(x\ind{t},\pi\ind{t}),\] which we note is a
        constant function with respect to the context. Since $\fstar$ is the true model, $\pi\subs{g}$ is always
        the best action, bringing a reward of $1/2$ per round. Any
        time $\pi\subs{b}$ is chosen, the decision-maker incurs
        instantaneous regret $1/2$. We now argue that if we
        apply the generic optimistic algorithm from
        \cref{sec:optimism_template}, it chooses $\pi\subs{b}$ every time a new context
        is encountered, incurring $\bigom(N)$ regret.\loose

        Let $S\ind{t}$ be the set of distinct contexts
        encountered before round $t$. Clearly, the exact minimizers of
        empirical square loss (see \eqref{eq:erm_t}) are $\fstar$,
        and all $f\subs{i}$ where $i$ is such that $x_i\notin
        S\ind{t}$. Hence, for any choice of $\beta\geq 0$, the
        confidence set in \eqref{eq:def_Ft_confidence_ball} contains
        all $f\subs{i}$ for which $x\subs{i}\notin S\ind{t}$. This
        implies that for each $t\in[T]$ where $x\ind{t}=x\subs{i}\notin S\ind{t}$, action
        $\pi\subs{b}$ has a higher upper confidence bound (defined in \eqref{eq:upper_lower_envelopes}) than $\pi\subs{g}$, since
        \[
          \fbar\ind{t}(x\ind{t},\pi\subs{b})
           = f_i(x_i,\pi\subs{b}) = 1 >
           \fbar\ind{t}(x\ind{t},\pi\subs{g}) = \fstar(x\ind{t},\pi\subs{g})=1/2.
        \]
        Hence, the
        cumulative regret grows by $1/2$ every time a new context
        is presented, and thus scales as $\bigom(N/2)$
        if the contexts are presented in order. That is, since
        $N=|\cX|=|\cF|-1$, the confidence-based algorithm fails to
        achieve logarithmic dependence on $\cF$.

        Let us remark that this failure continues even if contexts are
        stochastic. If the contexts are chosen via the uniform
        distribution on $\cX$, then for $T\geq{}N$, at least a
        constant proportion of the domain will be presented, which still leads to a lower bound of
        \[\En\brk{\Reg} = \Omega(N)=\Omega(\min\crl{\abs{\cF},\abs{\cX}}).\qedhere\]
     \end{example}

This stands in contrast to supervised learning, where empirical risk
minimization attains excess risk scaling with $\log\abs{\cF}$ for every
finite class (\cref{prop:iid_finite_class}). It is precisely the interactive nature
of the problem that breaks the classical guarantee.
     
What is behind the failure of optimism in this example? The
structure of $\cF$ forces optimistic methods to \emph{over-explore}, as the
algorithm puts too much hope into trying the arm $\pi\subs{b}$ for
each new context. As a result, the confidence widths in
\pref{eq:width_cb} do not shrink quickly enough. In what follows, we develop alternative methods which \emph{do} enjoy
logarithmic dependence on the size of $\cF$, with the best of these
methods (\cref{sec:igw}) achieving regret $O(\sqrt{AT\log\abs{\cF}})$. 

We mention in passing that even though optimism does not succeed in general, it is useful to
understand in what cases it works. We saw that the structure
of linear classes in $\reals^d$ only allowed for $d$ ``different''
directions, while in the example above, the optimistic algorithm gets
tricked by each new context, and is not able to shrink the confidence
band quickly enough over the domain. \cref{sec:structured} introduces the
eluder dimension, a structural property of the class $\cF$ which generalizes this notion, and is sufficient for optimistic methods to achieve low regret.

    \subsection{Reduction to Online Regression}
    \label{sec:online_regression}

    Given that the principle of optimism only leads to low regret for
    classes $\cF$ with special structure, we are left wondering
    whether there are more general algorithmic principles for decision
    making that can succeed for \emph{any} class $\cF$. This
    subsection develops the tool behind such principles: a reduction of
    decision making to \emph{online regression}.
    \cref{sec:eps_greedy_cb,sec:igw} then present two algorithms built on
    it.\loose

      For motivation, inspecting \cref{prop:linucb}, the meaning of $\beta^2$ in \cref{eq:def_Ft_confidence_ball}, and the statistical consequence of \cref{lem:valid_CI}, we observe that the regret bound of \linucb ``morally'' has the form
    $$\Reg\lesssim \prn*{d\cdot \sum_{\tau=1}^T \sum_{t=1}^{\tau} \En_{\pi\ind{t}\sim{}p\ind{t}}\brk*{(\fhat\ind{\tau+1}(x\ind{t},\pi\ind{t})-\fstar(x\ind{t},\pi\ind{t}))^2}}^{1/2}$$
    where $\fhat\ind{\tau+1}$ is the least squares (or, supervised learning with square loss) solution with respect to $(x\ind{1},\pi\ind{1},r\ind{1}),\ldots,(x\ind{\tau},\pi\ind{\tau},r\ind{\tau})$. However, \cref{ex:failure_optimism} indicates that relying on past data too much is not a viable approach in general. Indeed, even if covariates $x\ind{t}$ are i.i.d., our decisions $\pi\ind{t}$ are not, and---according to \cref{prop:erm_failure_online}---ERM may not have low regret. In hindsight, the fact that we were able to base our result for \linucb on simple ``offline'' regression appears to be more of a happy accident than an instance of a deeper principle.\loose
    
    To make progress, we build upon \emph{online} supervised learning as opposed to offline. To make the use of supervised
    learning as modular as possible, we the specific algorithm away using
    the notion of an \emph{online regression oracle} \citep{foster2020beyond}.

\begin{definition}[Online Regression Oracle]
    \label{def:online_regression_oracle}
	At each time $t\in[T]$, an \emph{online regression oracle} returns, given  $$(x\ind{1},\pi\ind{1},r\ind{1}),\ldots,(x\ind{t-1},\pi\ind{t-1},r\ind{t-1})$$ with $\En\brk{r\ind{i}|x\ind{i},\pi\ind{i}}=\fstar(x\ind{i},\pi\ind{i})$ and $\pi\ind{i}\sim p\ind{i}$, a function $\fhat\ind{t}:\cX\times\Pi\to\reals$. Define the cumulative estimation error
	$$\EstSq \ldef \sum_{t=1}^T \En_{\pi\ind{t}\sim p\ind{t}}\brk*{(\fhat\ind{t}(x\ind{t},\pi\ind{t})-\fstar(x\ind{t},\pi\ind{t}))^2}.$$
	The oracle is required to ensure that
	$$\EstSq \leq \EstSq(\cF, T, \delta)$$
	with probability at least $1-\delta$, where $\EstSq(\cF,T,\delta)$ is a deterministic bound. For the results that
        follow, $p\ind{i} = p\ind{i}(\cdot |
        x\ind{i}, \cH\ind{i-1})$ will represent the randomization distribution of
        a decision-maker.
      \end{definition}

The quantity $\EstSq$ is the cumulative squared estimation error
from \eqref{eq:online_est_sq}, now averaged over the action at each
round; it is a random variable that depends on the realized data. We use
$\EstSq(\cF,T,\delta)$ for a deterministic upper bound
on this error that holds with probability at least $1-\delta$. Throughout,
the bare symbol refers to the realized error and the argument form to the
guarantee; results whose proofs are pathwise are stated in terms of the
former, and the latter is substituted in when tuning parameters.

We can construct such an oracle using the online learning methods
from \cref{sec:ol}. The online predictor aims to keep its square-loss
regret $\RegOLSq$ small when predicting $r\ind{t}$ from
$(x\ind{t},\pi\ind{t})$. Under realizability,
\pref{lem:well_specified_reg_est} converts this regret guarantee into
an expected squared prediction-error bound. Freedman's inequality
(\pref{lem:freedman} in the appendix) gives the high-probability
guarantee required here.
For example, for finite classes of $[0,1]$-valued functions and bounded
rewards, averaged exponential weights has
$\RegOLSq\lesssim\log\abs{\cF}$ and yields
$\EstSq(\cF,T,\delta)\lesssim\log(\abs{\cF}/\delta)$. Note that we make use of
\emph{online regression} oracles for the results that follow because we aim to
derive regret bounds that hold for arbitrary, potentially adversarial
sequences $x\ind{1},\ldots,x\ind{T}$. If we instead assume that contexts are \iid, it is
reasonable to make use of algorithms for \emph{offline estimation},
or statistical learning with $\cF$. See \cref{sec:offline_regression}
for further discussion.

\subsection{A First Attempt: The \texorpdfstring{$\veps$}{ε}-Greedy Algorithm}
\label{sec:eps_greedy_cb}

With the oracle in hand, what remains is the choice of the exploration
distribution $p\ind{t}$ as a function of the estimate $\fhat\ind{t}$. The first general-purpose contextual bandit algorithm we study, illustrated below, is a contextual counterpart to the \egreedy method introduced in \cref{sec:mab}. 

    \begin{algbox}[frametitle={\algtitle{alg:egreedy_cb}{\egreedy for Contextual Bandits}}]
  \begin{algorithmic}
    \Require Exploration parameter $\veps\in(0,1)$; online regression oracle.
\For{$t=1,\ldots,T$}
\State Obtain $\fhat\ind{t}$ from online regression oracle for $(x\ind{1},\pi\ind{1},r\ind{1}),\ldots,(x\ind{t-1},\pi\ind{t-1},r\ind{t-1})$. 
\State Observe $x\ind{t}$.
\State With prob. $\veps$, select $\pi\sups{t}\sim \unif(\brk{A})$, and with prob. $1-\veps$, choose the greedy action
$$\pihat\ind{t} = \argmax_{\pi\in\brk{A}} \fhat\ind{t}(x\ind{t},\pi).$$
\State Observe reward $r\sups{t}$.
\EndFor{}
\end{algorithmic}
\end{algbox}
At each step $t$, the algorithm uses an
online regression oracle to compute a reward estimator
$\fhat\ind{t}(x,\pi)$ based on the data $\hist\ind{t-1}$ collected so
far (the estimator is computed before $x\ind{t}$ is revealed, and is
then evaluated at $x\ind{t}$). Given this estimator, the algorithm uses the same sampling
strategy as in the non-contextual case: with
probability $1-\veps$, the algorithm chooses the greedy decision
\begin{align}
    \pihat\sups{t} = \argmax_{\pi} \fhat\sups{t}(x\ind{t},\pi),
\end{align}
and with probability $\veps$ it samples a uniform random action
$\pi\ind{t}\sim\unif(\crl{1,\ldots,A})$. The following theorem shows
that whenever the online estimation oracle has low estimation error
$\EstSq(\cF,T,\delta)$, this method achieves low regret.
\begin{prop}[Regret of \egreedy for contextual bandits]
  \label{prop:eps_greedy_cb}
Suppose the decision-maker has access to an online regression oracle (\pref{def:online_regression_oracle}) with a guarantee $\EstSq(\cF, T, \delta)$. Then by choosing $\veps$ appropriately, the \egreedy
  algorithm ensures that with probability at least $1-\delta$,
  \begin{align*}
    \Reg \approxleq A^{1/3}T^{2/3}\cdot\EstSq(\cF,T, \delta)^{1/3}
  \end{align*}
  for any sequence $x\ind{1},\ldots,x\ind{T}$. As a special case, when
  $\cF$ is finite, if we use the (averaged) exponential
weights algorithm as an online regression oracle, the \egreedy
algorithm has
\begin{align*}
    \Reg \approxleq A^{1/3}T^{2/3}\cdot\log^{1/3}(\abs{\cF}/\delta).
\end{align*}
\end{prop}
Notably, this result scales with $\log\abs{\cF}$ for \emph{any finite
  class}, analogous to regret bounds for offline/online supervised
learning. This is the first regret bound in the monograph obtained by
\emph{reducing} decision making to supervised learning: the only property
of the oracle used in the proof is its estimation error, so any online
regression method can be plugged in. The $T^{2/3}$-dependence in the regret bound is suboptimal
(as seen for the special case of non-contextual bandits), which
\cref{sec:igw} addresses via more deliberate exploration.\loose

\begin{proof}[\pfref{prop:eps_greedy_cb}]
  Recall that $p\ind{t}$ denotes the randomization strategy on round
  $t$, computed after observing $x\ind{t}$. Following the same steps
  as the proof of \cref{prop:eps_greedy}, we can bound regret by\loose
\begin{align*}
  \Reg
  = \sum_{t=1}^{T}\En_{\pi\ind{t}\sim{}p\ind{t}}\brk*{\fstar(x\ind{t},\pistar(x\ind{t}))
    - \fstar(x\ind{t}, \pi\ind{t})} 
  \leq \sum_{t=1}^{T}\fstar(x\ind{t},\pistar(x\ind{t}))
      - \fstar(x\ind{t},\pihat\ind{t})
      + \veps{}T,
\end{align*}
where the $\veps{}T$ term represents the bias incurred by exploring uniformly.

Fix $t$ and abbreviate $\pistar(x\ind{t})=\pistar$. We have
  \begin{align*}
	&\fstar(x\ind{t},\pistar)
	      - \fstar(x\ind{t},\pihat\ind{t})\\
     &= [\fstar(x\ind{t},\pistar)-\fhat\sups{t}(x\ind{t},\pistar)] + [\fhat\sups{t}(x\ind{t},\pistar) - \fhat\sups{t}(x\ind{t},\pihat\sups{t})] + [\fhat\sups{t}(x\ind{t},\pihat\sups{t}) - \fstar(x\ind{t},\pihat\sups{t})] \\
	 &\leq \sum_{\pi\in\{\pihat\ind{t}, \pistar\}} |\fstar(x\ind{t},\pi)-\fhat\sups{t}(x\ind{t},\pi)| \\
	 &=\sum_{\pi\in\{\pihat\ind{t}, \pistar\}} \frac{1}{\sqrt{p\ind{t}(\pi)}} \sqrt{p\ind{t}(\pi)}|\fstar(x\ind{t},\pi)-\fhat\sups{t}(x\ind{t},\pi)| .
  \end{align*}
By the Cauchy-Schwarz inequality, the last expression is at most
\begin{align}
  \label{eq:eps_greedy_min_probability}
	&\crl*{\sum_{\pi\in\{\pihat\ind{t}, \pistar\}} \frac{1}{p\ind{t}(\pi)}}^{1/2}\crl*{\sum_{\pi\in\{\pihat\ind{t}, \pistar\}}  p\ind{t}(\pi)\prn*{\fstar(x\ind{t},\pi)-\fhat\sups{t}(x\ind{t},\pi)}^2}^{1/2} \\
	&\leq \sqrt{\frac{2A}{\veps}} \crl*{\En_{\pi\ind{t}\sim p\ind{t}}\prn*{\fstar(x\ind{t},\pi\ind{t})-\fhat\sups{t}(x\ind{t},\pi\ind{t})}^2}^{1/2}.
\end{align}
Summing across $t$, this gives
\begin{align}
	\sum_{t=1}^{T}\fstar(x\ind{t},\pistar(x\ind{t}))
	      - \fstar(x\ind{t},\pihat\ind{t}) 
		  &\leq \sqrt{\frac{2A}{\veps}} \sum_{t=1}^{T} \crl*{\En_{\pi\ind{t}\sim p\ind{t}}\prn*{\fstar(x\ind{t},\pi\ind{t})-\fhat\sups{t}(x\ind{t},\pi\ind{t})}^2}^{1/2} \\
		  &\leq \sqrt{\frac{2AT}{\veps}} \crl*{\sum_{t=1}^{T} \En_{\pi\ind{t}\sim p\ind{t}}\prn*{\fstar(x\ind{t},\pi\ind{t})-\fhat\sups{t}(x\ind{t},\pi\ind{t})}^2}^{1/2}.
\end{align}
Now observe that the online regression oracle guarantees that with
probability $1-\delta$,
$$\sum_{t=1}^{T} \En_{\pi\ind{t}\sim
  p\ind{t}}\prn*{\fstar(x\ind{t},\pi\ind{t})-\fhat\sups{t}(x\ind{t},\pi\ind{t})}^2\leq
\EstSq(\cF,T, \delta).$$
Whenever this occurs, we have
\begin{align*}
\Reg \approxleq \sqrt{\frac{AT\EstSq(\cF,T,\delta)}{\veps}} + \veps T.
\end{align*}
Choosing $\veps$ to balance the two terms leads to the claimed result
(if the balancing value exceeds one, we take $\veps=1$; in that case
$T\lesssim{}A\,\EstSq(\cF,T,\delta)$, and the claimed bound holds trivially).\loose
\end{proof}

\subsection{Inverse Gap Weighting and the Estimation-to-Decisions
  Principle}
\label{sec:igw}
To conclude this section, we present a general, oracle-based algorithm for
contextual bandits which achieves
\begin{align*}
    \Reg \approxleq \sqrt{AT\log\abs{\cF}}
\end{align*}
for any finite class $\cF$. As with \egreedy, this approach has no
dependence on the cardinality $\abs{\cX}$ of the context space,
reflecting the ability to generalize across contexts. The dependence
on $T$ improves upon \egreedy, and is optimal.

To motivate the approach, recall that conceptually, the key step of the  proof of \pref{prop:eps_greedy_cb} involved relating the instantaneous regret 
\begin{align}
\label{eq:cb_inst_reg}
    \En_{\pi\ind{t}\sim{}p\ind{t}}\brk*{\fstar(x\ind{t},\pistar(x\ind{t}))
    - \fstar(x\ind{t}, \pi\ind{t})}
\end{align}
of the decision maker at time $t$ to the instantaneous estimation error
\begin{align}
\label{eq:cb_inst_sq_error}
\En_{\pi\ind{t}\sim p\ind{t}}\brk*{\prn[\big]{\fstar(x\ind{t},\pi\ind{t})-\fhat\sups{t}(x\ind{t},\pi\ind{t})}^2}
\end{align}
between $\fhat\ind{t}$ and $\fstar$ under the randomization
distribution $p\ind{t}$. The \egreedy exploration distribution gives a
way to relate these quantities, but the algorithm's regret is
suboptimal because the randomization distribution puts mass at least
$\veps/A$ on every action, even those that are clearly suboptimal and
should be discarded.  One can ask whether there exists a better
randomization strategy that still admits an upper bound on
\eqref{eq:cb_inst_reg} in terms of
\eqref{eq:cb_inst_sq_error}. \pref{prop:igw_mab} below establishes
exactly that. At first glance, this distribution might appear to be
somewhat arbitrary or even magical, but \cref{sec:structured,sec:general_dm} show
that it arises as a special case of a more general---and in
some sense, universal---principle for
designing decision making algorithms.
The distribution below was introduced by \citet{abe1999associative}; the
form and the use we make of it here follow \citet{foster2020beyond}.\loose
\begin{definition}[Inverse Gap Weighting]
	Given a vector $\fhat=(\fhat(1), \ldots, \fhat(A))\in\bbR^{A}$, the Inverse Gap Weighting distribution $p=\igw_\gamma(\fhat(1), \ldots, \fhat(A))$ with parameter $\gamma\geq 0$ is defined as
\begin{equation}
      \label{eq:igw}
      p(\pi) = \frac{1}{\lambda + 2\gamma(\fhat(\pihat)-\fhat(\pi))},
    \end{equation}
where $\pihat=\argmax_\pi \fhat(\pi)$ is the greedy action, and where $\lambda\in\brk{1,A}$ is chosen such that
$\sum_{\pi}p(\pi)=1$.
\end{definition}
Above, the normalizing constant
  $\lambda\in\brk*{1,A}$ is always guaranteed to exist, because we have
  $\frac{1}{\lambda}\leq{}\sum_{\pi}p(\pi)\leq\frac{A}{\lambda}$,
  and because $\lambda\mapsto\sum_{\pi}p(\pi)$ is continuous over
  $\brk{1,A}$.

Let us give some intuition behind the distribution in \cref{eq:igw}.  
We can interpret the parameter $\gamma$ as trading off exploration and
exploitation. Indeed, $\gamma\to 0$ gives a uniform distribution,
while $\gamma\to\infty$ amplifies the gap between the greedy action
$\pihat$ and any action with $\fhat(\pi)<\fhat(\pihat)$, resulting in
a distribution supported only on actions that achieve the largest
estimated value $\fhat(\pihat)$.

The following fundamental technical result shows that playing the
Inverse Gap Weighting distribution always suffices to
link the instantaneous regret in \eqref{eq:cb_inst_reg} in to the
instantaneous estimation error in \eqref{eq:cb_inst_sq_error}.
  \begin{prop}[Inverse gap weighting links regret to estimation error]
    \label{prop:igw_mab}
    Consider a finite decision space $\Pi=\crl{1,\ldots,A}$. For any
    vector $\fhat\in\reals^A$ and $\gamma>0$, define $p=\igw_\gamma(\fhat(1), \ldots, \fhat(A))$.
     This strategy guarantees that for all $\fstar\in\reals^A$,
     \begin{align}
       \label{eq:igw_mab}
     \En_{\pi\sim{}p}\brk*{\fstar(\pistar)-\fstar(\pi)} \leq 
 \frac{A}{\gamma} + \gamma\cdot\En_{\pi\sim{}p} \brk[\big]{(\fhat(\pi)-\fstar(\pi))^2}.
\end{align}
\end{prop}
\begin{proof}[\pfref{prop:igw_mab}]
  We break the ``regret'' term on the \lhs of \cref{eq:igw_mab} into three terms:
	\begin{align*}
		\hspace{-8mm}\En_{\pi\sim{}p}\brk[\big]{ \fstar(\pistar)-\fstar(\pi)}
		= \underbrace{\En_{\pi\sim{}p}\brk[\big]{ \fhat(\pihat)-\fhat(\pi)} }_{\text{(I) exploration bias}} 
		+  \underbrace{\En_{\pi\sim{}p}\brk[\big]{
          \fhat(\pi)-\fstar(\pi) }  }_{\text{(II) est. error on policy}} 
		+ \underbrace{ \fstar(\pistar)-\fhat(\pihat) }_{\text{(III) est. error at opt}} .
	\end{align*}
	The first term asks ``how much would we lose by exploring, if
        $\fhat$ were the true reward function?'', and is equal to
	\begin{align*}
		\sum_{\pi} \frac{\fhat(\pihat) - \fhat(\pi)}{\lambda + 2\gamma\prn[\big]{\fhat(\pihat) - \fhat(\pi)}} \leq \frac{A-1}{2\gamma},
	\end{align*}
	while the second term is at most
	\begin{align*}
		&\sqrt{\En_{\pi\sim{}p}\brk[\big]{( \fhat(\pi)-\fstar(\pi))^2}} \leq \frac{1}{2\gamma} + \frac{\gamma}{2}\En_{\pi\sim{}p}( \fhat(\pi)-\fstar(\pi))^2.  
	\end{align*}
	The third term can be further written as
	\begin{align*}
    	\fstar(\pistar)-\fhat(\pistar) 
			- (\fhat(\pihat) -\fhat(\pistar)) 
			&\leq \frac{\gamma}{2} p(\pistar) \prn{\fstar(\pistar)-\fhat(\pistar)  }^2 + \frac{1}{2\gamma p(\pistar)} - (\fhat(\pihat) -\fhat(\pistar)) \\
			&\leq \frac{\gamma}{2} \En_{\pi\sim{}p}( \fstar(\pi) - \fhat(\pi))^2
			+ \brk*{\frac{1}{2\gamma p(\pistar)} - (\fhat(\pihat) -\fhat(\pistar))}.
	\end{align*}
	The term in brackets above is equal to
	\begin{align*}
		\frac{\lambda + 2\gamma(\fhat(\pihat) -\fhat(\pistar))}{2\gamma} - (\fhat(\pihat) -\fhat(\pistar)) = \frac{\lambda}{2\gamma}\leq \frac{A}{2\gamma}.
	\end{align*}
\end{proof}

The simple result we just proved is remarkable. The special \igw{} strategy guarantees a relation between regret and estimation error for \textit{any} estimator $\fhat$ and any $\fstar$, irrespective of the problem structure or the class $\cF$. \pref{prop:igw_mab} is the core of the development for the rest of the monograph; \cref{sec:structured,sec:general_dm} greatly generalize it to general decision making problems and reinforcement learning.\loose

Below, we present a contextual bandit algorithm called \squarecb
\citep{foster2020beyond} which makes use of the
Inverse Gap Weighting distribution to achieve optimal regret.
\begin{algbox}[frametitle={\algtitle{alg:squarecb}{\squarecb}}]
  \begin{algorithmic}
    \Require Exploration parameter $\gamma>0$; online regression oracle.
\For{$t=1,\ldots,T$}
\State Obtain $\fhat\ind{t}$ from online regression oracle with $(x\ind{1},\pi\ind{1},r\ind{1}),\ldots,(x\ind{t-1},\pi\ind{t-1},r\ind{t-1})$. 
\State Observe $x\ind{t}$.
\State Compute $p\ind{t} = \igw_\gamma\prn*{\fhat\ind{t}(x\ind{t},1), \ldots, \fhat\ind{t}(x\ind{t},A)}$.
\State Select action $\pi\ind{t}\sim p\ind{t}$.
\State Observe reward $r\sups{t}$.
\EndFor{}
\end{algorithmic}
\end{algbox}
At each step $t$, the algorithm uses an
online regression oracle to compute a reward estimator
$\fhat\ind{t}(x,\pi)$ based on the data $\hist\ind{t-1}$ collected so
far. Given this estimator, the algorithm uses Inverse Gap Weighting to
compute
$p\ind{t}=\igw_\gamma(\fhat\ind{t}(x\ind{t},\cdot))$ as an exploratory
distribution, then samples $\pi\ind{t}\sim{}p\ind{t}$.

The following result, which is a near-immediate consequence of
\cref{prop:igw_mab}, gives a regret bound for this algorithm.
\begin{theorem}[Regret of \squarecb]
  \label{prop:squarecb}
    Given a class $\cF$ with $\fstar\in\cF$, for any $\gamma>0$ and any
    sequence $x\ind{1},\ldots,x\ind{T}$, \squarecb satisfies
    $$\Reg \leq \frac{TA}{\gamma} + \gamma\cdot\EstSq$$
    for every realization. In particular, if the decision-maker
    has access to an online regression oracle
    (\pref{def:online_regression_oracle}) with estimation error $\EstSq(\cF, T, \delta)$, then \squarecb with $\gamma = \sqrt{TA / \EstSq(\cF, T, \delta)}$ attains a regret bound of
    $$\Reg \lesssim \sqrt{A T \EstSq(\cF, T, \delta)}$$
    with probability at least $1-\delta$.
    As a special case, when $\cF$ is finite, the averaged exponential
    weights algorithm achieves $\EstSq(\cF, T, \delta)\lesssim \log(\abs{\cF}/\delta)$, leading to
    $$\Reg \lesssim \sqrt{A T \log(\abs{\cF}/\delta}).$$
\end{theorem}
\begin{proof}[\pfref{prop:squarecb}]
  We begin with regret, then add and subtract the squared estimation
  error as follows:
    \begin{align*}
    \Reg
    &= \sum_{t=1}^{T}\En_{\pi\ind{t}\sim{}p\ind{t}}\brk*{\fstar(x\ind{t}, \pistar) - \fstar(x\ind{t}, \pi\ind{t})} \\
    &\leq \sum_{t=1}^{T}\En_{\pi\ind{t}\sim{}p\ind{t}}\brk*{\fstar(x\ind{t}, \pistar)
    - \fstar(x\ind{t}, \pi\ind{t})
    - \gamma\cdot{}(\fstar(x\ind{t}, \pi\ind{t})-\fhat\ind{t}(x\ind{t}, \pi\ind{t}))^2} + \gamma\cdot\EstSq.
    \end{align*}
    By appealing to \pref{prop:igw_mab} with $\fhat(x\ind{t},\cdot)$
    and $\fstar(x\ind{t},\cdot)$, for each step $t$, we have
    \begin{align*}
        \En_{\pi\ind{t}\sim{}p\ind{t}}\brk*{\fstar(x\ind{t}, \pistar) - \fstar(x\ind{t}, \pi\ind{t}) - \gamma\cdot{}(\fstar(x\ind{t}, \pi\ind{t})-\fhat\ind{t}(x\ind{t}, \pi\ind{t}))^2} \leq \frac{A}{\gamma},
    \end{align*}
    and thus
    $$\Reg \leq \frac{TA}{\gamma} + \gamma \cdot \EstSq.$$
    This holds for every realization. On the event, of probability at
    least $1-\delta$, that the oracle guarantee
    $\EstSq\leq\EstSq(\cF,T,\delta)$ holds, choosing $\gamma$ to balance
    the two terms yields the result.
\end{proof}

If the online regression oracle is minimax optimal (that is,
$\EstSq(\cF, T, \delta)$ is the ``best possible'' for $\cF$) then
\squarecb is also minimax optimal for $\cF$. Thus, $\igw$ not only provides a connection between online supervised learning and decision making, but it does so in an optimal fashion. Establishing minimax optimality is beyond the scope of this monograph: it requires understanding of minimax optimality of online regression with arbitrary $\cF$, as well as lower bound on regret of contextual bandits with arbitrary sequences of contexts. We refer to \citet{foster2020beyond} for details.\loose

\paragraph{Interpretation: The Estimation-to-Decisions principle}
\squarecb is the first instance of a design principle that organizes the
rest of this monograph, which we call \emph{Estimation-to-Decisions}. The
principle has two ingredients: an \emph{estimation oracle}, a supervised
(online) learning algorithm that produces an estimate $\fhat\ind{t}$ for
the unknown model with no regard for exploration, and an
\emph{exploration strategy}, a fixed rule that maps $\fhat\ind{t}$ to a
distribution $p\ind{t}$ over decisions. The central property is that
\emph{the exploration strategy ensures that our performance for decision
  making is controlled by our performance for estimation}. For Inverse
Gap Weighting, this is the inequality in \cref{prop:igw_mab},
\begin{equation}
  \label{eq:e2d_principle}
  \underbrace{\En_{\pi\sim{}p}\brk[\big]{\fstar(\pistar)-\fstar(\pi)}}_{\text{regret of $p$}}
  \;\leq\; \frac{A}{\gamma}
  + \gamma\cdot\underbrace{\En_{\pi\sim{}p}\brk[\big]{(\fhat(\pi)-\fstar(\pi))^2}}_{\text{estimation error of $\fhat$ under $p$}}
  \qquad\text{for all $\fhat$ and $\fstar$,}
\end{equation}
which, summed over rounds, bounds the regret of the algorithm by the
estimation error of the oracle on the data the algorithm itself
collects. That this is possible at all should come as a surprise: the
estimate may be arbitrarily wrong on decisions the algorithm never
plays, but \cref{eq:e2d_principle} says that with the right choice for $p$, such errors cannot hurt.

The value of this reduction is that it separates the two sources of difficulty: the
estimation oracle handles generalization across the class $\cF$ (and can
be any method from statistical or online learning), while the
exploration strategy handles exploration. It also makes the algorithm
\emph{oracle-efficient}: each round costs one call to the regression
oracle and $O(A)$ additional arithmetic. Generalizing this principle, in \cref{sec:structured} we
ask what the best possible exploration strategy is for a \emph{general} decision
space $\Pi$; the answer is given by a complexity measure, the
Decision-Estimation Coefficient, which systematically generalizes the \igw{}  distribution to general decision spaces. The principle is developed in even greater generality---subsuming reinforcement learning---in \cref{sec:general_dm}, and \cref{sec:lower} further shows that it leads to a universal lower bound on regret for decision making.

This concludes the main development of this section. The remaining
subsection covers an extension and can be skipped on a first reading.

\subsubsection{Extension: SquareCB with Offline Regression Oracles\bonus}
\label{sec:offline_regression}

When $x\ind{1},\ldots,x\ind{T}$ are i.i.d., it is natural to ask
whether an online regression method that works for arbitrary sequences
is necessary, or whether one can work with a weaker oracle tuned to
i.i.d. data. For \squarecb, it turns out that any oracle for
\emph{offline regression} (defined below) is sufficient. 

\begin{definition}[Offline Regression Oracle]
    \label{def:offline_regression_oracle}
		Given $$(x\ind{1},\pi\ind{1},r\ind{1}),\ldots,(x\ind{t-1},\pi\ind{t-1},r\ind{t-1})$$ where $x\ind{1},\ldots, x\ind{t-1}$ are i.i.d., $\pi\ind{i}\sim p(x\ind{i})$ for fixed  $p:\cX\to\Delta(\Pi)$ and $\En\brk{r\ind{i}|x\ind{i}, \pi\ind{i}} = \fstar(x\ind{i}, \pi\ind{i})$,
		an \emph{offline regression oracle} returns a function $\fhat:\cX\times\Pi\to\reals$ such that
	$$\En_{x, \pi\sim p(x)} (\fhat(x,\pi)-\fstar(x,\pi))^2 \leq t^{-1} \EstSqOff(\cF, t, \delta)$$
	with probability at least $1-\delta$.
\end{definition}
Note that the normalization $t^{-1}$ above is introduced to keep the
scaling consistent with our conventions for offline estimation.

Below, we state a variant of \squarecb which is adapted to offline
oracles \citep{simchi2020bypassing}. Compared to the \squarecb for online oracles, the main change
is that we update the estimation oracle and exploratory
distribution on an epoched schedule as opposed to updating at every
round. In addition, the parameter $\gamma$ for the Inverse Gap
Weighting distribution changes as a function of the epoch.
    \begin{algbox}[frametitle={\algtitle{alg:squarecb_offline}{\squarecb with Offline Regression}}]
  \begin{algorithmic}
    \Require Exploration parameters $\gamma_1,\gamma_2,\ldots$ and epoch sizes $\tau_1,\tau_2,\ldots$; offline regression oracle.
    \For{$m=1,2,\ldots$}
    \State Obtain $\fhat\ind{m}$ from offline regression oracle with\\ \qquad\qquad$(x\ind{\tau_{m-2}+1},\pi\ind{\tau_{m-2}+1},r\ind{\tau_{m-2}+1}),\ldots,(x\ind{\tau_{m-1}},\pi\ind{\tau_{m-1}},r\ind{\tau_{m-1}})$. 
    \For{$t=\tau_{m-1}+1,\ldots,\tau_m$}
\State Observe $x\ind{t}$.
\State Compute $p\ind{t} = \igw_{\gamma_m}\prn*{\fhat\ind{m}(x\ind{t},1), \ldots, \fhat\ind{m}(x\ind{t},A)}$.
\State Select action $\pi\ind{t}\sim p\ind{t}$.
\State Observe reward $r\sups{t}$.
\EndFor{}
\EndFor{}
\end{algorithmic}
\end{algbox}
While this algorithm is quite intuitive, proving a regret bound for it
is quite non-trivial---much more so than the online oracle
variant. The key challenge is that, while the contexts
$x\ind{1},\ldots,x\ind{T}$ are \iid, the decisions
$\pi\ind{1},\ldots,\pi\ind{T}$ evolve in a time-dependent fashion,
which makes it unclear to invoke the guarantee in
\cref{def:offline_regression_oracle}. Nonetheless, the following
remarkable result, due to \citet{simchi2020bypassing}, shows that this algorithm attains a regret bound
similar to that of \cref{prop:squarecb}.
\begin{prop}[Regret of \squarecb with offline oracles]
  For $m=1,2,\ldots$, let $\tau_m = 2^{m}$ and $\gamma_m = \sqrt{A\tau_{m-1}/\EstSqOff(\cF,
      \tau_{m-1}, \delta)}$. Then with probability
    at least $1-\delta$, regret of \squarecb with an offline
    oracle is at most
    $$
        \Reg \lesssim  \sum_{m=1}^{\lceil \log T\rceil} \sqrt{ A \cdot \tau_m \cdot \EstSqOff(\cF, \tau_{m}, \delta/m^2)}.
    $$
\end{prop}
Under mild assumptions, above bound scales as
    $$
        \Reg \lesssim \sqrt{ A \cdot T \cdot \EstSqOff(\cF, T, \delta/\log T)}.
    $$
For a finite class $\cF$, we recall from \cref{sec:background} that
empirical risk with the square loss (least squares) achieves
$\EstSqOff(\cF,T,\delta)\approxleq\log(\abs{\cF}/\delta)$, which gives
    $$
        \Reg \lesssim \sqrt{ AT \log(\abs{\cF}/\delta)}.
        $$
Thus, when contexts are i.i.d., the Estimation-to-Decisions reduction
goes through with an offline oracle as well, at the price of an epoch
schedule and a more delicate analysis.

\subsection{Bibliographic Remarks and Further Reading}\label{sec:cb_bib}
\paragraph{Takeaways}
This section showed that supervised learning algorithms can be lifted into interactive decision making algorithms without sacrificing statistical efficiency: with a
realizable class $\cF$, one can achieve regret scaling with
$\log\abs{\cF}$ and no dependence on the number of contexts. Optimism
accomplishes this for linear classes (\linucb), but fails for general
classes (\cref{ex:failure_optimism}). The key idea is the
Estimation-to-Decisions principle: pair an off-the-shelf regression
oracle with an exploration rule, Inverse Gap Weighting, that guarantees
that the regret of the chosen distribution is controlled by the
estimation error of the oracle under that same distribution
(\cref{prop:igw_mab}). \squarecb is the simplest instance of this principle, and the
rest of the monograph generalizes this to structured bandits
and reinforcement learning.\loose

\paragraph{Bibliographic remarks}
The optimism template of \cref{sec:optimism_template}, with confidence
sets defined through excess square loss (\cref{lem:valid_CI}), follows
\citet{russo2013eluder}. Optimism for linear contextual bandits
originates with \citet{auer2002using}; the name \linucb and the
algorithm in \cref{alg:linucb} are from \citet{li2010contextual}, with
the first regret analysis given by \citet{chu2011contextual}. The
confidence ellipsoids and elliptic potential lemma behind
\cref{prop:linucb} are due to \citet{dani2008stochastic} and
\citet{abbasi2011improved}. \cref{ex:failure_optimism} is from
\citet{foster2018practical}. The $T^{2/3}$ rate of
\cref{prop:eps_greedy_cb} for uniform exploration (\cref{alg:egreedy_cb})
with a supervised learning oracle goes back to the epoch-greedy algorithm of
\citet{langford2008epoch}; the online regression oracle
(\cref{def:online_regression_oracle}) is the abstraction of
\citet{foster2020beyond}. The Inverse Gap Weighting distribution
\cref{eq:igw} was introduced by \citet{abe1999associative};
\cref{prop:igw_mab} for linear contextual bandits. \squarecb (\cref{alg:squarecb}), and
\cref{prop:squarecb} are due to \citet{foster2020beyond}, who observed that inverse gap weighting encompasses the general contextual bandit problem and is minimax optimal. The offline variant of \squarecb, \cref{alg:squarecb_offline}, and its regret bound are due to \citet{simchi2020bypassing}, who termed this version of the algorithm as FALCON. The perspective of \emph{Estimation-to-Decisions} as the general principle behind these methods is from \citet{foster2021statistical}.

\paragraph{Further reading}
Textbook treatments of linear and contextual bandits are given by
\citet{lattimore2020bandit} and \citet{slivkins2019introduction}.
Refinements of \squarecb include instance-dependent regret bounds
governed by a disagreement coefficient \citep{foster2020instance},
first-order bounds \citep{foster2021efficient}, and adaptation to
misspecification \citep{foster2020adapting}. Not covered in this monograph is the \emph{agnostic formulation} of the contextual bandit problem, where one competes with a class of policies rather than assuming realizability. This is the setting of \expfour \citep{auer2002non} and of the oracle-efficient algorithm of \citet{agarwal2014taming}.

\subsection{Deferred Proofs}
\label{sec:cb_deferred}

\begin{proof}[\pfref{prop:unstructured_cb}]
  The case $T=1$ is immediate from $\Reg\leq1$, so assume $T\geq2$.
  For each pair $(x,a)$, consider an independent stream of rewards with
  distribution $\Mstar(\cdot\mid x,a)$, revealing the next entry whenever
  that pair is selected. This is equivalent to the stochastic multi-armed bandit protocol,
  even when contexts and actions depend on the past. Let $\bar r_{x,a,n}$
  be the mean of the first $n$ entries of this stream. Hoeffding's inequality
  and a union bound over $KAT$ choices of $(x,a,n)$ give, for
  $L=\log(2AKT^2)$,
  \[
    \bbP\left(\exists x,a,\ 1\leq n\leq T:
      |\bar r_{x,a,n}-\fstar(x,a)|>\sqrt{\frac{L}{2n}}\right)
    \leq 2KATe^{-L}=\frac1T.
  \]
  Condition on the complement of this event. Once every action has been tried
  for a given context, the usual optimistic comparison shows that choosing
  an action with $n\geq1$ previous observations incurs regret at most
  $2\sqrt{L/(2n)}$. When $n=0$, \ucb chooses an untried action, whose
  regret is at most one.

  Let $N_{x,a}$ be the number of selections of $(x,a)$ up to time $T$
  and $N_x=\sum_aN_{x,a}$. The first selections contribute at most
  $\min\crl{A,N_x}\leq\sqrt{AN_x}$ in context $x$. Since
  $\sum_{n=1}^{m-1}n^{-1/2}\leq2\sqrt m$, the remaining selections give
  \[
    \Reg_x\leq\sqrt{AN_x}
      +2\sqrt{2L}\sum_{a=1}^A\sqrt{N_{x,a}}
    \leq(1+2\sqrt{2L})\sqrt{AN_x}.
  \]
  Summing over contexts and applying Cauchy--Schwarz once more yields
  \[
    \Reg\leq(1+2\sqrt{2L})\sqrt A\sum_x\sqrt{N_x}
      \leq(1+2\sqrt{2L})\sqrt{AKT}.
  \]
  Taking expectations and using
  $L\lesssim\log(eAKT)$ proves the result.
\end{proof}

\paragraph{\linucb with continuous parameter sets}
We now extend \pref{prop:linucb} to continuous parameter sets. The idea
is to approximate the parameter set by finitely many points, apply the
concentration argument to those points, and then bound the
error from the approximation. We begin with the elementary covering
lemma needed for this argument.

\begin{lem}[Finite nets for linear predictors]
  \label{lem:cb_linear_net}
  For any nonempty $\Theta\subseteq\sB_2^d(1)$ and $0<\veps\leq1$,
  there is a finite $\Theta_\veps\subseteq\Theta$ such that
  \[
    |\Theta_\veps|\leq(1+2/\veps)^d,\qquad
    \forall\theta\in\Theta\quad
    \min_{\theta'\in\Theta_\veps}\|\theta-\theta'\|_2\leq\veps.
  \]
  For any such approximating pair and any feature map with
  $\|\phi(x,\pi)\|_2\leq1$,
  \[
    |f_\theta(x,\pi)-f_{\theta'}(x,\pi)|\leq\veps
    \qquad\text{for all }(x,\pi).
  \]
\end{lem}

\begin{proof}[\pfref{lem:cb_linear_net}]
  Choose a maximal collection $\Theta_\veps\subseteq\Theta$ whose distinct
  points are more than $\veps$ apart. The open Euclidean balls of radius
  $\veps/2$ centered at these points are disjoint and contained in the
  ball of radius $1+\veps/2$. Comparing their volumes gives
  \[
    |\Theta_\veps|(\veps/2)^d\leq(1+\veps/2)^d,
  \]
  so the collection is finite with the claimed cardinality. In particular,
  greedily adding separated points must terminate, which establishes
  existence even when $\Theta$ is not closed. Maximality makes this
  collection an $\veps$-net. Finally, Cauchy--Schwarz gives
  \[
    |f_\theta(x,\pi)-f_{\theta'}(x,\pi)|
      =|\langle\theta-\theta',\phi(x,\pi)\rangle|
      \leq\|\theta-\theta'\|_2\|\phi(x,\pi)\|_2\leq\veps.
  \]
\end{proof}

\begin{proof}[\pfref{prop:linucb}, continuous parameter sets]
  Recall that in this setup, \linucb still computes least squares over the full set $\Theta$; we appeal to covering numbers and a finite approximation only to establish the confidence set guarantee.
  
  Fix $0<\veps\leq1$ and the net $\Theta_\veps$ from
  \pref{lem:cb_linear_net}; write $N=|\Theta_\veps|$. For any
  $\theta\in\Theta$, define
  \[
    U_i(\theta)=(f_\theta(x\ind{i},\pi\ind{i})-r\ind{i})^2
      -(\fstar(x\ind{i},\pi\ind{i})-r\ind{i})^2,
    \qquad D_i(\theta)=\En_{i-1}\brk*{U_i(\theta)}.
  \]
  Here $\En_{i-1}$ conditions on the past and the current context,
  before choosing the action and observing its reward. The stochastic
  reward assumption gives
  \[
    D_i(\theta)=\En_{\pi\ind{i}\sim p\ind{i}}
      [(f_\theta(x\ind{i},\pi\ind{i})-\fstar(x\ind{i},\pi\ind{i}))^2].
  \]
  All rewards and predictions lie in $[-1,1]$. Consequently,
  $|D_i(\theta)-U_i(\theta)|\leq8$ and
  \[
    \En_{i-1}[(D_i(\theta)-U_i(\theta))^2]
      \leq\En_{i-1}[U_i(\theta)^2]\leq16D_i(\theta).
  \]
  Apply \pref{lem:freedman} with $\eta=1/32$ to each fixed
  $\theta'\in\Theta_\veps$, and take a union bound over the net.
  With probability at least $1-\delta$, simultaneously for all net
  points $\theta'\in\Theta_\veps$ and all $0\leq n\leq T$,
  \[
    \frac12\sum_{i=1}^nD_i(\theta')
      \leq\sum_{i=1}^nU_i(\theta')+32\log(N/\delta).
  \]
  The comparison function in this argument is the true $\fstar$;
  it need not belong to the net.

  For any $\theta\in\Theta$, choose $\theta'$ with $\nrm{\theta-\theta'}_2\veps$.
  Factoring differences of squared losses, and using boundedness, gives
  \[
    |U_i(\theta)-U_i(\theta')|\leq4\veps,
    \qquad |D_i(\theta)-D_i(\theta')|\leq4\veps.
  \]
  Thus the same event ensures, for every $\theta\in\Theta$ and prefix $n\leq T$,
  \begin{equation}
    \label{eq:cb_cover_confidence}
    \frac12\sum_{i=1}^nD_i(\theta)
      \leq\sum_{i=1}^nU_i(\theta)+b,
    \qquad b=32\log(N/\delta)+6T\veps.
  \end{equation}
  Let $\beta^2\geq b$, and consider the confidence sets
  \eqref{eq:def_Ft_confidence_ball} over the original class.
  Compactness of $\Theta$ ensures that the least-squares minimizer exists.
  Applying \eqref{eq:cb_cover_confidence} to that minimizer, and using
  $D_i\geq0$, shows that its cumulative loss is at least the cumulative
  loss of $\fstar$ minus $b$. Hence $\fstar\in\cF\ind{t}$ for every $t$.
  Every $\theta\in\cF\ind{t}$ also has
  $\sum_{i<t}U_i(\theta)\leq\beta^2$, because the empirical minimizer
  has no greater loss than $\fstar$. Therefore
  \[
    \sum_{i<t}D_i(\theta)\leq2(\beta^2+b)\leq4\beta^2.
  \]
  This is exactly the confidence guarantee used in the finite-class
  \linucb proof. Since \linucb chooses its maximizing action deterministically
  (with fixed tie-breaking), these conditional squared errors are the
  quadratic forms in its empirical covariance matrix. The rest of the proof and \pref{lem:elliptic_potential} then give
  \[
    \Reg\lesssim\sqrt{(\beta^2+1)dT\log(1+T/d)}.
  \]
  Finally, setting $\veps=1/T$, the $\veps$-net bound gives
  $b\leq32[d\log(1+2T)+\log(1/\delta)]+6$, proving the stated result.
\end{proof}

\section{Structured Bandits: Exploration in Large Decision Spaces}
\label{sec:structured}

Contextual bandits generalize across \emph{contexts}; this section turns
to the other axis of \cref{fig:axes}, generalizing across
\emph{decisions}. In the \emph{structured bandit} problem, the decision
space may be continuous or combinatorially large, and the learner must use structured in the reward landscape to infer the rewards of decisions it may have never played from those it has.
Compared to contextual bandits, the situation is more subtle: regret must
scale with $\log\abs{\cF}$ \emph{and} with a measure of how hard the
class $\cF$ is to explore, and the right way to explore depends on
$\cF$. The section's main tool is the Decision-Estimation
Coefficient, which identifies the best possible exploration strategy for
any class, and the resulting algorithm, Estimation-to-Decisions, which
turns the principle of \cref{sec:cb} into a single method for that can be applied to all
structured bandit problems.

\paragraph{Organization}
\cref{sec:structured_setting} formalizes the structured bandit problem and desiderata.
\cref{sec:structured_warmup}
adapts the principle of optimism to the structured setting and analyzes it based on a complexity measure called the eluder dimension, but shows that it can be exponentially suboptimal in general. \cref{sec:dec_structured} introduces the
Decision-Estimation Coefficient and the \etd algorithm, and shows that
Inverse Gap Weighting is the special case for finitely many actions.
\cref{sec:dec_examples_structured} computes the DEC for examples including linear bandits and nonparametric bandits.
Lastly, \cref{sec:dec_posterior} relates the DEC back to optimism and to
posterior sampling. Two optional extensions (marked $\star$)
follow: \cref{sec:structured_contexts} adds contexts,
and \cref{sec:dec_properties} collects structural
properties of the DEC.

\subsection{The Structured Bandit Problem}
\label{sec:structured_setting}
The structured bandit problem generalizes the (non-contextual)
multi-armed bandit problem by allowing for large, potentially infinite
or continuous decision spaces. The protocol is as follows.
\begin{algbox}[frametitle={\algprotocol{Structured Bandit Protocol}}]
  \begin{algorithmic}
\For{$t=1,\ldots,T$}
\State Select decision $\pi\sups{t}\in\Pi$.\hfill\algcommentlight{$\Pi$ is
  large and potentially continuous.}
\State Observe reward $r\sups{t}\in\bbR$.
\EndFor{}
\end{algorithmic}
\end{algbox}

The only change from the multi-armed bandit protocol of \pref{sec:mab}
is that $\Pi$ is no longer restricted to $\crl{1,\ldots,A}$. This added
generality is natural in many applications:
\begin{itemize}
\item In medicine, the treatment may be a continuous
  variable, such as a dosage. The treatment could even be a
  high-dimensional vector (such as dosages for many different
  medications). See \pref{fig:structured_bandit}.
\item In pricing applications, a seller might aim to select a continuous price or vector or
  prices in order to maximize their returns.
\item In routing applications, the decision space may be finite, but
  combinatorially large. For example, the decision might be a path or
  flow in a graph.
\end{itemize}

\begin{figure}
    \centering
    \includegraphics[width=0.8\textwidth]{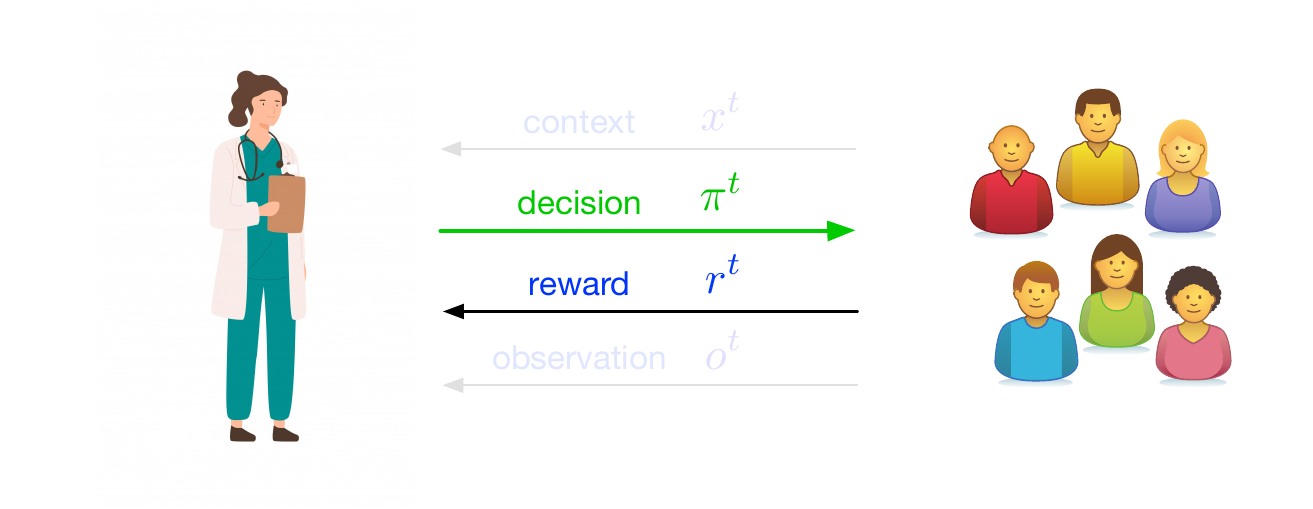}
    \caption{The structured bandit problem. A doctor (the learner)
      selects a continuous, high-dimensional treatment (the decision).}
    \label{fig:structured_bandit}
  \end{figure}

  \paragraph{Assumptions and regret}
As in preceding sections, we assume that rewards are
stochastic, and generated from a fixed model.
\begin{assumption}[Stochastic Rewards]
  \label{asm:stochastic_rewards_structured}
  Rewards are generated independently via
    \begin{align}
        r\sups{t}\sim \Mstar(\cdot\mid \pi\sups{t}),
    \end{align}
    where $\Mstar(\cdot\mid\cdot)$ is the underlying \emph{model}.
  \end{assumption}
We define
  \begin{align}
    \label{eq:mean_reward_structured}
  \fstar(\pi) \ldef \En\left[r\mid{} \pi\right]
  \end{align}
  as the mean reward function under $r\sim\Mstar(\cdot\mid 
  \pi)$, and measure regret via
  \begin{align}
    \label{eq:regret_structured}
    \Reg \ldef{} \sum_{t=1}^{T}\fstar(\pistar) - \sum_{t=1}^{T}\En_{\pi\ind{t}\sim{}p\ind{t}}\brk{\fstar(\pi\ind{t})}.
  \end{align}
  Here, $\pistar \ldef\argmax_{\pi\in\Pi}\fstar(\pi)$ as usual. We define the history as $\cH\ind{t}=(\pi\ind{1},r\ind{1}),\ldots,(\pi\ind{t},r\ind{t})$.
  As in the preceding sections, we assume that $r\ind{t}\in\brk{0,1}$
  almost surely, and regret bounds hold with probability at least
  $1-\delta$ unless stated otherwise.

\paragraph{Function approximation}

A first attempt to tackle the structured bandit problem might be to
apply algorithms for the multi-armed bandit setting, such as \ucb. This
would give regret $\bigoht(\sqrt{\abs{\Pi}{}T})$, which could be
vacuous if $\Pi$ is large relative to $T$. However, with no
further assumptions on the underlying reward function $\fstar$, this is unavoidable. To
allow for better regret, we make assumptions on the structure of
$\fstar$ that allow us to share information across
decisions, and to generalize to decisions that we may not have played.
This is well-suited for the applications described above, where $\Pi$ is a
continuous set (e.g., $\Pi\subseteq\bbR^{d}$), but we expect $\fstar$
to be continuous, or perhaps even linear with respect some well-designed
set of features. To make this idea precise, we follow the same
approach as in statistical learning and contextual bandits, and assume
access to a \emph{realizable} function class $\cF$ that aims to
capture our prior knowledge about $\fstar$.
\begin{assumption}[Realizability]
    The decision-maker has access to a class $\cF\subset\{f:\Pi\to\reals\}$ such that $\fstar\in\cF$.
\end{assumption} 
As for rewards, we assume that
$f(\pi)\in\brk{0,1}$ for all $f\in\cF$ unless stated otherwise.
Given such a class, a reasonable goal---particularly in light of the
development in \cref{sec:background} and \cref{sec:cb}---would be to achieve guarantees that
scale with the complexity of supervised learning or estimation with
$\cF$, e.g. $\log\abs{\cF}$ for finite classes; this is what we were able to achieve for
contextual bandits, after all. Unfortunately, this is too
good to be true, as the following example shows.
\begin{example}[Necessity of structural assumptions]
\label{ex:structured_realizability}
  Let $\Pi=\brk{A}$, and let $\cF=\crl{f_i}_{i\in\brk{A}}$, where
  \[
    f_i(\pi)\ldef{} \frac{1}{2} + \frac{1}{2}\indic{\pi=i}.
  \]
  It is clear that one needs $\Reg\approxgeq{}A \wedge T$ for this setting, yet
  $\log\abs{\cF}=\log(A)$, so a regret bound of the form
  $\Reg\approxleq{}\sqrt{T\log\abs{\cF}}$ is not possible if $A$ is
  large relative to $T$.
\end{example}
What this example highlights is that generalizing across decisions is
fundamentally different (and, in some sense, more challenging) than
generalizing across contexts (the two can also be combined; see
\cref{sec:structured_contexts}). In light of this, we aim for guarantees that scale with
$\log\abs{\cF}$, but additionally scale with an appropriate notion of \emph{complexity
  of exploration} for the decision space $\Pi$. Such a notion of complexity should
reflect how much information is shared across decisions, which depends
on the interplay between $\Pi$ and $\cF$. The complexity measure we
arrive at, the Decision-Estimation Coefficient, does exactly this, and
characterizes the optimal regret for structured bandits up to the
$\log\abs{\cF}$ estimation term (\cref{sec:lower}).

\subsection{Warm-Up: Optimism and the Eluder Dimension}
\label{sec:structured_warmup}
To see what such guarantees look like and how they can be obtained,
we first investigate the optimism principle and the \ucb algorithm in
the structured setting. We show that:
\begin{enumerate}
\item \ucb attains guarantees that scale with $\log\abs{\cF}$, and
  additionally scale with a notion of complexity called the \emph{eluder dimension}, which
  is small for simple problems such as bandits with linear rewards.
\item In general, \ucb is not optimal, and can have regret that is
  exponentially large compared to the optimal rate.
\end{enumerate}

\subsubsection{UCB for Structured Bandits}

We can adapt the \ucb algorithm from
multi-armed bandits to
structured bandit by appealing to least squares and
confidence sets, similar to the approach we took for contextual
bandits \citep{russo2013eluder}. Recall that $f(\pi)\in[0,1]$ for all $f\in\cF$ and $r\ind{t}\in[0,1]$ almost surely. Let
\begin{align}
\label{eq:erm_structured}
	\fhat\ind{t} = \argmin_{f\in\cF} \sum_{i=1}^{t-1} (f(\pi\ind{i})-r\ind{i})^2
\end{align}
be the empirical minimizer on round $t$, and with
$\beta^2\ldef{}8\log(|\cF|/\delta)$, define confidence sets $\cF\ind{1}=\cF$ and
\begin{align}
  \label{eq:confidence_set_structured}
	\cF\ind{t} = \crl*{f\in\cF: \sum_{i=1}^{t-1} (f(\pi\ind{i}) - r\ind{i})^2 \leq \sum_{i=1}^{t-1} (\fhat\ind{t}(\pi\ind{i}) - r\ind{i})^2 + \beta^2}.
\end{align}  
Defining $\fbar\ind{t}(\pi)\ldef{}\max_{f\in\cF\ind{t}}f(\pi)$ as the
upper confidence bound, the generalized \ucb algorithm is given by
\begin{equation}
\pi\ind{t}=\argmax_{\pi\in\Pi}\fbar\ind{t}(\pi).\label{eq:ucb_structured}
\end{equation}
\oldparagraph{When does the confidence width shrink?}
Using \cref{lem:valid_CI}, one can see the generalized \ucb algorithm ensures
that $\fstar\in\cF\ind{t}$ for all $t$ with high probability. Whenever
this happens, regret is bounded by the upper confidence width:
\begin{align}
  \Reg \leq \sum_{t=1}^{T}\fbar\ind{t}(\pi\ind{t})-\fstar(\pi\ind{t}).\label{eq:width_structured}
\end{align}
This bound holds for all structured bandit problems, with no
assumption on the structure of $\Pi$ and $\cF$. Hence, to derive a
regret bound, the only question we need to
answer is \emph{when will the confidence widths shrink?}

For the unstructured multi-armed bandit, we need to shrink the width
for every arm separately, and the best bound on
\pref{eq:width_structured} we can hope for is
$\bigoh(\sqrt{\abs{\Pi}T})$. One might hope that if $\Pi$ and $\cF$
have nice structure, we can do better. In fact, we have already seen
one such case: For linear models, where
\begin{align}
    \cF=\crl*{\pi\mapsto\inner{\theta,\phi(\pi)}\mid{}\theta\in\Theta\subset
    \sB_2^d(1)},\label{eq:linear_structured}
\end{align}
\pref{prop:linucb} shows that we can bound \pref{eq:width_structured}
by $\sqrt{dT\log\abs{\cF}}$. Here, the number of decisions $\abs{\Pi}$
is replaced by the dimension $d$, which reflects the fact that there
are only $d$ truly unique directions to explore before we can start
\emph{extrapolating} to new actions. Is there a more general version
of this phenomenon when we move beyond linear models?

\subsubsection{Eluder Dimension: Bounding Confidence Widths}
The eluder dimension \citep{russo2013eluder} is a complexity measure
that aims to capture the extent to which the function class $\cF$
facilitates extrapolation (i.e., generalization to unseen decisions),
and gives a generic way of bounding the confidence width in
\pref{eq:width_structured}. It is defined for a class $\cF$ as follows.
\begin{definition}[Eluder Dimension]
\label{def:eluder}
  Let $\cF\subset(\Act\to\bbR)$ and $\fstar:\Pi\to\bbR$ be given, and define $\ElCheck_{\fstar}(\cF,\veps)$ as the length of the longest sequence
  of decisions $\pi\ind{1},\ldots,\pi\ind{d}\in\Act$ such that for all $t\in\brk{d}$, there
  exists $f\ind{t}\in\cF$ such that
  \begin{equation}
    \label{eq:eluder}
\abs*{f\ind{t}(\pi\ind{t})-\fstar(\pi\ind{t})}>\veps,\quad\text{and}\quad\sum_{i<t}(f\ind{t}(\pi\ind{i})-\fstar(\pi\ind{i}))^2\leq\veps^{2}.
  \end{equation}
  The eluder dimension is defined as
  $\El_{\fstar}(\cF,\veps)=\sup_{\veps'\geq{}\veps}\ElCheck_{\fstar}(\cF,\veps')\vee{}1$. We abbreviate $\El(\cF,\veps) = \max_{\fstar\in\cF}\El_{\fstar}(\cF,\veps)$.
\end{definition}
The intuition behind the eluder dimension is simple: It asks, for a
worst-case sequence of decisions, how many
times we can be ``surprised'' by a new decision $\pi\ind{t}$ if we
 can estimate the underlying model $\fstar$ well on all of the
 preceding points. In particular, if we form confidence sets as in
 \pref{eq:confidence_set_structured} with $\beta=\veps$, then the
 number of times the upper confidence width in
 \cref{eq:width_structured} can be
 larger than $\veps$ is at most $\El_{\fstar}(\cF,\veps)$. We consider the definition
 $\El_{\fstar}(\cF,\veps)=\sup_{\veps'\geq{}\veps}\ElCheck_{\fstar}(\cF,\veps')\vee{}1$
 instead of directly working with $\ElCheck_{\fstar}(\cF,\veps)$ to ensure monotonicity with respect to $\veps$, which is used
 in the proof of \cref{prop:ucb_eluder}.

The following result gives a regret bound for \ucb for generic
structured bandit problems. The regret bound has no dependence on the
size of the decision space, and scales only with 
 $\El(\cF,\veps)$ and $\log\abs{\cF}$.
\begin{prop}[Regret of \ucb via the eluder dimension]
  \label{prop:ucb_eluder}
For a finite set of functions $\cF\subset(\Pi\to\brk{0,1})$, using $\beta^2
= 8\log(|\cF|/\delta)$, the generalized \ucb algorithm guarantees that
with probability at least $1-\delta$,
\begin{equation}
  \Reg \lesssim \min_{\veps>0}\crl*{
    \sqrt{\El(\cF,\veps)\cdot{}T\log(\abs{\cF}/\delta)}
+ \veps{}T
} \approxleq
    \sqrt{\El(\cF,T^{-1/2})\cdot{}T\log(\abs{\cF}/\delta)}.
\end{equation}
\end{prop}
For the case of linear models in \pref{eq:linear_structured}, 
it is possible to use the elliptic potential lemma
(\pref{lem:elliptic_potential}) to show that
\[
\El(\cF,\veps) \approxleq{} d\log(\veps^{-1}).
\]
For finite classes, this gives
$\Reg\approxleq{}\sqrt{dT\log(\abs{\cF}/\delta)\log(T)}$, which
recovers the guarantee in \pref{prop:linucb}. Another well-known
example is that of \emph{generalized linear models}. Here,
we fix \emph{link function} $\sigma:\brk{-1,+1}\to\bbR$ and define
\[
    \cF=\crl*{\pi\mapsto\sigma\prn[\big]{\inner{\theta,\phi(\pi)}}\mid{}\theta\in\Theta\subset
    \sB_2^d(1)}.
\]
This is a more flexible model than linear bandits. A well-known
special case is the logistic bandit problem, where
$\sigma(z)=1/(1+e^{-z})$. One can show \citep{russo2013eluder} that
for any choice of $\sigma$, if there exist
$\mu,L>0$ such that $\mu<\sigma'(z)<L$ for all $z\in\brk{-1,+1}$, then
\begin{equation}
  \label{eq:eluder_glm}
\El(\cF,\veps) \approxleq{} \frac{L^2}{\mu^2}\cdot{}d\log(\veps^{-1}).
\end{equation}
This leads to a regret bound that scales with
$\frac{L}{\mu}\sqrt{dT\log\abs{\cF}}$, generalizing the regret bound
for linear bandits.

In general, the eluder dimension can be quite large. Consider the generalized linear model setup above with
$\sigma(z)=+\relu(z)$ or $\sigma(z)=-\relu(z)$ (either choice of sign
works), where $\relu(z)\ldef{}\max\crl{z,0}$ is the ReLU function;
this can be interpreted as a neural network with a single neuron. Here, we can
have $\sigma'(z)=0$, so \pref{eq:eluder_glm} does not apply, and it
turns out \citep{li2021eluder} that 
\begin{align}
  \label{eq:eluder_relu}
  \El(\cF,\veps) \approxgeq e^{d}
\end{align}
for constant $\veps$. That is, even for a single ReLU neuron, the
eluder dimension is already  exponential, which is a bit
disappointing. Fortunately, the eluder dimension can be overly pessimistic:
\cref{sec:dec_examples_structured} shows that it is possible to do
better for this class, but this requires changing the algorithm.\loose

\begin{proof}[\pfref{prop:ucb_eluder}]
  \newcommand{\ValueEluderCL}[1]{\ElCheck(\cF,#1)}
  \newcommand{\ValueEluderL}[1]{\El(\cF,#1)}

Define
\[
\cFbar\ind{t} =
\crl*{f\in\cF\mid{}\sum_{i<t}(f(\pi\ind{i})-\fstar(\pi\ind{i}))^2 \leq 4\beta^2}.
\]
By \pref{lem:valid_CI}, we have that with probability at least
$1-\delta$, for all $t$:
\begin{enumerate}
  \item $\fstar\in\cF\ind{t}$.
  \item $\cF\ind{t}\subseteq\cFbar\ind{t}$.
  \end{enumerate}
  Let us condition on this event. As in
  \pref{lem:regret_optimistic}
  , since $\fstar\in\cF\ind{t}$,
we can upper bound
\begin{align*}
  \Reg \leq \sum_{t=1}^{T}\fbar\ind{t}(\pi\ind{t})-\fstar(\pi\ind{t}).
\end{align*}
Now,
define
\[
w\ind{t}(\pi) = \sup_{f\in\cFbar\ind{t}}\brk*{f(\pi)-\fstar(\pi)},
\]
which is a useful upper bound on the upper confidence width at time
$t$. Since $\cF\ind{t}\subseteq\cFbar\ind{t}$, we have
\[
  \Reg \leq \sum_{t=1}^{T}w\ind{t}(\pi\ind{t}).
\]

We now appeal to the following technical lemma concerning the eluder dimension, which is \citet[Lemma 3]{russo2013eluder}.
\begin{lem}[Eluder dimension controls confidence widths]
  \label{lem:eluder_indicator_bound}
  Fix a function class $\cF$, function $\fstar\in\cF$, and parameter
  $\beta>0$. For any sequence $\pi\ind{1},\ldots,\pi\ind{T}$, if we
  define
  \[
    w\ind{t}(\pi) = \sup_{f\in\cF}\crl*{f(\pi)-\fstar(\pi) : \sum_{i<t}(f(\pi\ind{i})-\fstar(\pi\ind{i}))^2\leq\beta^2},
  \]
then for all $\alpha>0$,
  \[
    \sum_{t=1}^{T}\indic{w\ind{t}(\pi\ind{t})>\alpha} \leq{} \prn*{\frac{\beta^2}{\alpha^2}+1}\cdot\ElCheck_{\fstar}(\cF,\alpha).
  \]
\end{lem}
Note that for the special case where $\beta=\alpha$, the bound in
\pref{lem:eluder_indicator_bound} immediately follows from the
definition of the eluder dimension. The point of this lemma is to show
that a similar bound holds for all scales $\alpha$ simultaneously, but
with a pre-factor $\frac{\beta^2}{\alpha^2}$ that grows large when $\alpha\ll\beta$.

To apply this result, fix $\veps>0$, and bound
\begin{equation}
\sum_{t=1}^{T}w\ind{t}(\pi\ind{t})
  \leq \sum_{t=1}^{T}w\ind{t}(\pi\ind{t})\indic{w\ind{t}(\pi\ind{t})>\veps} + \veps{}T.\label{eq:eluder_width_decomp}
\end{equation}
Let us order the indices $\crl{1,\ldots,T}$ as $\crl{i_1,\ldots,i_T}$,
so that $w\ind{i_1}(\pi\ind{i_1})\geq{}w\ind{i_2}(\pi\ind{i_2})\geq\ldots\geq{}w\ind{i_{\tau}}(\pi\ind{i_{\tau}})$. Consider any index $\tau$ 
for which $w\ind{i_\tau}(\pi\ind{i_\tau})>\veps$. For any $\alpha>\veps$, if we have $w\ind{i_\tau}(\pi\ind{i_\tau})>\alpha$, then \pref{lem:eluder_indicator_bound} (since $\alpha\leq{}1\leq\beta^2$) implies that
\begin{equation}
  \label{eq:wt_eluder_bound}
\tau \leq{} \sum_{t=1}^{T}\indic{w\ind{t}(\pi\ind{t})>\alpha} \leq{} \prn*{\frac{4\beta^2}{\alpha^2}+1}\ElCheck_{\fstar}(\cF,\alpha)\leq \frac{5\beta^2}{\alpha^2}\ElCheck_{\fstar}(\cF,\alpha).
\end{equation}
Since we have restricted to $\alpha\geq\veps$ and
$\alpha\mapsto\El_{\fstar}(\cF,\alpha)$ is decreasing, rearranging yields
\[
  w\ind{i_\tau}(\pi\ind{i_\tau}) \leq{} \beta\sqrt{\frac{5\ValueEluderL{\veps}}{\tau}}.
\]
With this, we can bound the main term in \pref{eq:eluder_width_decomp} by
\begin{align*}
  \sum_{t=1}^{T}w\ind{t}(\pi\ind{t})\indic{w\ind{t}(\pi\ind{t})>\veps}
  \approxleq{}  \sum_{t=1}^{T}\beta\sqrt{\frac{\ValueEluderL{\veps}}{t}}
  \approxleq{}  \beta\sqrt{\ValueEluderL{\veps}T}.
\end{align*}
Combining this with \pref{eq:eluder_width_decomp} gives
$\Reg\approxleq{}\beta\sqrt{\ValueEluderL{\veps}T}+\veps{}T$. Since
$\veps>0$ was arbitrary, we are free to minimize over it.
  
\end{proof}

\begin{proof}[\pfref{lem:eluder_indicator_bound}]
Let us adopt the shorthand $d=\ElCheck_{\fstar}(\cF,\alpha)$. We begin with a definition. We say $\pi$ is $\alpha$-independent of $\pi\ind{1},\ldots,\pi\ind{t}$
  if there exists $f\in\cF$ such that
  $\abs*{f(\pi)-\fstar(\pi)}>\alpha$ and
  $\sum_{i=1}^{t}\prn*{f(\pi\ind{i})-\fstar(\pi\ind{i})}^{2}\leq\alpha^{2}$. We say
  $\pi$ is $\alpha$-dependent on $\pi\ind{1},\ldots,\pi\ind{t}$ if for all $f\in\cF$ with
  $\sum_{i=1}^{t}\prn*{f(\pi\ind{i})-\fstar(\pi\ind{i})}^{2}\leq\alpha^{2}$, $\abs*{f(\pi)-\fstar(\pi)}\leq{}\alpha$.

  We first claim that for any $t$, if
  $w\ind{t}(\pi\ind{t})>\alpha$, then $\pi_t$ is $\alpha$-dependent on
  at most $\beta^2/\alpha^2$ disjoint subsequences of
  $\pi\ind{1},\ldots,\pi\ind{t-1}$. Indeed, let $f$ be such that
  $\abs*{f(\pi\ind{t})-\fstar(\pi\ind{t})}>\alpha$. If $\pi\ind{t}$ is
  $\alpha$-dependent on a particular subsequence
  $\pi\ind{i_1},\ldots,\pi\ind{i_k}$ but $w\ind{t}(\pi\ind{t})>\alpha$, we must have
  \[
    \sum_{j=1}^{k}(f(\pi\ind{i_j})-\fstar(\pi\ind{i_j}))^{2}\geq{}\alpha^{2}.
  \]
  If there are $M$ such disjoint sequences, we have
  \[
    M\alpha^{2}\leq{}\sum_{i<t}(f(\pi\ind{i})-\fstar(\pi\ind{i}))^2\leq{}\beta^2,
  \]
  so $M\leq{} \frac{\beta^2}{\alpha^{2}}$.

Next, we claim that for $\tau$ and any sequence $(\pi\ind{1},\ldots,\pi\ind{\tau})$, there is
some $j$ such that $\pi\ind{j}$ is $\alpha$-dependent on at least
$\floor{\tau/d}$ disjoint subsequences of $\pi\ind{1},\ldots,\pi\ind{j-1}$. Let
$N=\floor{\tau/d}$, and let $B_1,\ldots,B_N$ be subsequences of
$\pi\ind{1},\ldots,\pi\ind{\tau}$. We initialize with $B_i = (\pi\ind{i})$. If $\pi\ind{N+1}$ is
$\alpha$-dependent on $B_i=\prn*{\pi\ind{i}}$ for all $1\leq{}i\leq{}N$ we are done. Otherwise,
choose $i$ such that $\pi\ind{N+1}$ is $\alpha$-independent of
$B_i$, and add it to $B_i$. Repeat this process until we reach $j$
such that either $\pi\ind{j}$ is $\alpha$-dependent on all $B_i$ or 
$j=\tau$. In the first case we are done, while in the second case, we
have $\sum_{i=1}^{N}\abs*{B_i}\geq{}\tau\geq{}dN$. Moreover,
$\abs*{B_i}\leq{}d$, since each $\pi\ind{j}\in{}B_i$ is
$\alpha$-independent of its prefix (this follows from the definition
of eluder dimension). We conclude that
$\abs*{B_i}=d$ for all $i$, so in this case $\pi\ind{\tau}$ is $\alpha$-dependent
on all $B_i$.

Finally, let $(\pi\ind{t_1},\ldots,\pi\ind{t_{\tau}})$ be the subsequence $\pi\ind{1},\ldots,\pi\ind{T}$ consisting of all elements for
which $w\ind{t_i}(\pi\ind{t_i})>\alpha$. Each element is dependent on
at most $\beta^2/\alpha^{2}$ disjoint subsequences of
$(\pi\ind{t_1},\ldots,\pi\ind{t_{\tau}})$, and by the argument above, one
element is dependent on at least $\floor{\tau/d}$ disjoint
subsequences, so we must have $\floor{\tau/d}\leq{}
\beta^2/\alpha^{2}$, which implies that $\tau\leq{} (\beta^2/\alpha^{2}+1)d$.
\end{proof}

\subsubsection{Failure of Optimism: The Cheating Code}
The following example shows a function class $\cF$ for which the regret
experienced by \ucb is exponentially large compared to the regret
obtained by a simple alternative algorithm. The construction is adapted from
\citet{amin2011bandits} and \citet{jun2020crush}.
\begin{example}[Cheating Code]
  \label{ex:cheating_code}
  Let $A\in\bbN$ be a
  power of $2$ and consider the following function class $\cF$ (in this example, we depart from the convention of this section and allow rewards and functions to take values in $[-1,1]$; rescaling to $[0,1]$ changes nothing). 
  \begin{itemize}
    \item The decision space is $\Pi=\brk{A}\cup\cC$, where
      $\cC=\crl{c_1,\ldots,c_{\log_2(A)}}$ is a set of ``cheating''
      actions.
    \item For all actions $\pi\in\brk{A}$, $f(\pi)\in\brk*{0,1}$ for
      all $f\in\cF$, but we otherwise make no assumption on the
      reward.
    \item For each $f\in\cF$, rewards for actions in $\cC$ take the
      following form. Let $\pi_f\in\brk{A}$ denote the action in
      $\brk{A}$ with highest reward. Let $b(f)=(b_1(f),\ldots,b_{\log_2(A)}(f))\in\crl{0,1}^{\log_2(A)}$ be
      a binary encoding for the index of $\pi_f\in\brk{A}$ (e.g., if
      $\pi_f=1$, $b(f)=(0,0,\ldots,0)$, if $\pi_f=2$,
      $b(f)=(0,0,\ldots,0,1)$, and so on). For each action $c_i\in\cC$, we set
      \[
        f(c_i) = -b_i(f).
      \]
    \end{itemize}
    The idea here is that if we ignore the actions $\cC$, this looks like a
    standard multi-armed bandit problem, and the optimal regret is
    $\Theta(\sqrt{AT})$. However, we can use the actions in $\cC$
    to ``cheat'' and get an exponential improvement in sample
    complexity. The argument is as follows.

    Suppose for simplicity that rewards are Gaussian with
    $r\sim{}\cN(\fstar(\pi),1)$ under $\pi$. For each cheating action
    $c_i\in\cC$, since $\fstar(c_i)=-b_i(\fstar)\in\crl{0,-1}$, we can determine
    whether the value is $b_i(\fstar)=0$ or $b_i(\fstar)=1$ with high probability using
    $\bigoht(1)$ action pulls. If we do this for each $c_i\in\cC$, which
    will incur $\bigoht(\log(A))$ regret (there are $\log(A)$ such
    actions and each one leads to constant regret), we can infer the
    binary encoding $b(\fstar)=b_1(\fstar),\ldots,b_{\log_2(A)}(\fstar)$ for the optimal action
    $\pifstar$ with high probability. At this point, we can simply
    stop exploring, and commit to playing $\pifstar$ for the remaining rounds, which
    will incur no more regret. If one is careful with the details, this gives that with probability
    at least $1-\delta$,
    \[
      \Reg \approxleq \log_2(A)\cdot\log(\log_2(A)/\delta),
    \]
    since there are $\log_2(A)$ cheating actions, and a union bound over
    them requires each to be pulled $O(\log(\log_2(A)/\delta))$ times.
    In other words, by exploiting the cheating actions, our regret has
    gone from linear to \emph{logarithmic} in $A$ (we have also
    improved the dependence on $T$, an added bonus).\loose

    Now, let us consider the behavior of the generalized \ucb
    algorithm. Unfortunately, since all actions $c_i\in\cC$ have
    $f(c_i)\leq{}0$ for all $f\in\cF$, we have
    $\fbar\ind{t}(c_i)\leq{}0$. As a result, the generalized \ucb algorithm
    will only ever pull actions in $\brk{A}$, ignoring the cheating actions
    and effectively turning this
    into a vanilla multi-armed bandit problem, which means that
    \[
      \Reg \approxgeq \sqrt{AT}.
    \]
  \end{example}
  
  This example shows that \ucb can behave suboptimally in the presence of decisions
  that reveal useful information but do not necessarily lead to high
  reward. Since the ``cheating'' actions are guaranteed to have
  low reward, \ucb avoids them even though they are very
  informative. We conclude that:
  \begin{enumerate}
  \item Obtaining optimal sample complexity for structured bandits
    requires algorithms that more deliberately balance the tradeoff
    between optimizing reward and acquiring information.
  \item In general, the optimal strategy for picking decisions
    can be very different depending on the choice of the class
    $\cF$. This contrasts the contextual bandit setting, where
    we saw that the Inverse Gap Weighting algorithm attained optimal
    sample complexity for any choice of class $\cF$, and all that
    needed to change was how to perform estimation.
  \end{enumerate}

  \begin{rem}[Suboptimality of posterior sampling]
    Recall the Bayesian bandit setting in \cref{sec:posterior}, where
    we showed that the posterior sampling algorithm attains regret
    $\bigoht(\sqrt{AT})$ when $\Pi=\crl{1,\ldots,A}$. Posterior
    sampling is a general-purpose algorithm, and can be applied to
    directly to arbitrary structured bandit problems (as long as a prior is available). However, similar
    to \ucb, the cheating
    code construction in \cref{ex:cheating_code} implies that posterior
    sampling is not optimal in general. Indeed, posterior sampling
    will never select the cheating arms in $\cC$, as these have
    sub-optimal reward for all models in $\cF$. As a result, the
    Bayesian regret of the algorithm will scale with $\Reg \approxgeq
    \sqrt{AT}$ for a worst-case prior.
  \end{rem}

  \subsection{The \CompText: Estimation-to-Decisions for Structured Bandits}
  \label{sec:dec_structured}

  \cref{sec:structured_warmup} leaves us with two requirements: an
  algorithm must deliberately trade reward for information, and the
  tradeoff must adapt to the class $\cF$. This section shows that both
  are met by a general version of the Estimation-to-Decisions principle
  of \cref{sec:cb}: reduce decision making to supervised online
  estimation, and pair the estimation oracle with an exploration strategy
  that ensures that regret is controlled by estimation error.

  Concretely, we keep the online regression oracle of
  \cref{sec:cb} and generalize the exploration strategy. First, we require the following oracle for supervised estimation.
\begin{definition}[Online Regression Oracle]
    \label{def:online_regression_oracle_structured}
	At each time $t\in[T]$, an \emph{online regression oracle} returns, given  $$(\pi\ind{1},r\ind{1}),\ldots,(\pi\ind{t-1},r\ind{t-1})$$ with $\En\brk{r\ind{i}|\pi\ind{i}}=\fstar(\pi\ind{i})$ and $\pi\ind{i}\sim p\ind{i}$, a function $\fhat\ind{t}:\Pi\to\reals$. Define the cumulative estimation error
	$$\EstSq \ldef \sum_{t=1}^T \En_{\pi\ind{t}\sim p\ind{t}}\brk*{(\fhat\ind{t}(\pi\ind{t})-\fstar(\pi\ind{t}))^2}.$$
	The oracle is required to ensure that $\EstSq \leq \EstSq(\cF, T, \delta)$
	with probability at least $1-\delta$, where $\EstSq(\cF,T,\delta)$ is a deterministic bound. Here, $p\ind{i}(\cdot | \cH\ind{i-1})$ is the randomization distribution for the decision-maker.
      \end{definition}
This is the oracle from \pref{def:online_regression_oracle}
with the context omitted. As discussed there, a square-loss
online regret bound on $\RegOLSq$ yields an estimation-error
bound $\EstSq(\cF,T,\delta)$ under realizability. In particular, for finite
classes of bounded functions and bounded rewards, averaged exponential
weights gives $\EstSq(\cF,T,\delta)\approxleq\log(\abs{\cF}/\delta)$.
      
\keepwithbox{16}%
The following algorithm, which we call \etdtext or \etd \citep{foster2021statistical,foster2023tight}, is a
general-purpose meta-algorithm for structured bandits.
    \begin{algbox}[frametitle={\algtitle{alg:etd_structured}{Estimation-to-Decisions (\etd) for Structured Bandits}}]
  \begin{algorithmic}
    \Require Exploration parameter $\gamma>0$; online regression oracle.
    \For{$t=1,\ldots,T$}
    \State Obtain $\fhat\ind{t}$ from online regression oracle with $(\pi\ind{1},r\ind{1}),\ldots,(\pi\ind{t-1},r\ind{t-1})$. 
    \State Compute
    \[
      p\ind{t}=\argmin_{p\in\Delta(\Pi)}\max_{f\in\cF}\En_{\pi\sim{}p}\biggl[f(\pif)-f(\pi)
    -\gamma\cdot(f(\pi)-\fhat\ind{t}(\pi))^2
    \biggr].
      \]
\State Select action $\pi\ind{t}\sim p\ind{t}$ and observe reward $r\ind{t}$.
\EndFor{}
\end{algorithmic}
\end{algbox}
As in \squarecb, each round consists of an estimation step and an
exploration step; the only change is the latter, which now solves a
min-max problem over $p\in\Delta(\Pi)$ and $f\in\cF$ instead of
applying a closed-form rule. Why a min-max problem? Inverse Gap
Weighting was tailored to
finitely many unstructured actions, and the examples in
\cref{sec:structured_warmup} show that for a general class $\cF$, the
right way to explore depends on $\cF$ in a delicate way. Rather than
designing an exploration strategy by hand for each $\cF$, \etd lets the
class itself determine the strategy: given the estimate $\fhat\ind{t}$,
it asks for the distribution over decisions that achieves the
\emph{best possible} tradeoff between regret and estimation error.
The value of this tradeoff is a complexity measure (or,
structural parameter) for $\cF$ called
the \emph{\CompText} \citep{foster2021statistical,foster2023tight}, given by
      \begin{equation}
  \label{eq:dec_structured}
  \comp(\cF,\fhat) =
  \min_{p\in\Delta(\Act)}\max_{f\in\cF}\En_{\pi\sim{}p}\biggl[\hspace{1pt}\underbrace{f(\pif)-f(\pi)}_{\text{regret
    of decision}}
    -\gamma\cdot\hspace{-12pt}\underbrace{(f(\pi)-\fhat(\pi))^2}_{\text{information
        gain for obs.}}\hspace{-8pt}
    \biggr].
  \end{equation}
That is, the \CompText is the value of a game in which the learner
(the min player) chooses a distribution over decisions so that, against
a worst-case instance $f$ (the max player), the \emph{regret} of the
decision is controlled by the \emph{information gain} (or, estimation
error) about $f$ relative to a reference model $\fhat$---a guess for the
true model. With enough information, the learner can confirm or rule out
the guess, and the scale parameter $\gamma$ controls how much regret
they are willing to incur to do so. In general, the larger the value of
$\comp(\cF,\fhat)$, the more difficult it is to explore.\loose

To state a regret bound for \etd, we define
  \begin{equation}
    \label{eq:dec_max_structured}
        \comp(\cF) = \sup_{\fhat\in\conv(\cF)}\comp(\cF,\fhat).
      \end{equation}
Here, $\conv(\cF)$ denotes the set of all convex combinations of
elements in $\cF$. The reason we consider the set $\conv(\cF)$ is that in
general, online estimation algorithms such as exponential weights will
produce improper predictions with $\fhat\in\conv(\cF)$. In fact, it
turns out (see \pref{prop:dec_unconstrained}) that even if we allow
$\fhat$ to be unconstrained above, the maximizer always lies in
$\conv(\cF)$ without loss of generality.

The main result for this section shows that the regret for \etd is controlled by the value of the \CompShort and the estimation error
$\EstSq(\cF,T,\delta)$ for the
online regression oracle.
    The result is from \citet{foster2021statistical}.
    \begin{theorem}[Regret of \etd for structured bandits]
      \label{prop:dec_bandit}
      The \etd algorithm with exploration parameter $\gamma>0$
      guarantees that for every realization,
\begin{equation}
  \label{eq:dec_bandit}
\Reg \leq{}
  \comp(\cF)\cdot{}T + \gamma\cdot\EstSq.
\end{equation}
In particular, with an online regression oracle
(\pref{def:online_regression_oracle_structured}), with probability at least $1-\delta$,
\begin{equation}
  \label{eq:dec_bandit_hp}
\Reg \leq{}
  \comp(\cF)\cdot{}T + \gamma\cdot\EstSq(\cF,T,\delta).
\end{equation}
\end{theorem}
Compare \cref{eq:dec_bandit} with the guarantee for
\squarecb in \cref{prop:squarecb}, which has the same form with
$\comp(\cF)$ replaced by $A/\gamma$: regret is controlled by
the estimation error of the oracle, and the price of exploration, which
was $A/\gamma$ for finitely many actions, is now the \CompShort.
We can optimize over the parameter $\gamma$ in the result
above, which yields
\begin{align*}
\Reg &\leq{}
\inf_{\gamma>0}\crl[\bigg]{\comp(\cF)\cdot{}T
  + \gamma\cdot\EstSq(\cF,T,\delta)} \\
       &\leq 
2\cdot\inf_{\gamma>0}\max\crl[\bigg]{\comp(\cF)\cdot{}T,
  \gamma\cdot\EstSq(\cF,T,\delta)}.
\end{align*}
For finite classes, we can use the exponential weights method to
obtain $\EstSq(\cF,T,\delta)\approxleq{}\log(\abs{\cF}/\delta)$, and
this bound specializes to
\begin{align}
\Reg &\approxleq{}
       \inf_{\gamma>0}\max\crl[\bigg]{\comp(\cF)\cdot{}T,
  \gamma\cdot\log(\abs{\cF}/\delta)}.\label{eq:dec_finite}
\end{align}
As desired, regret scales only with the estimation complexity
$\log\abs{\cF}$ and the exploration complexity $\comp(\cF)$. Before
interpreting the result further, we give the proof: it is the proof of
\cref{prop:squarecb} with \cref{eq:e2d_principle} replaced by the
definition of the \CompShort.
    \begin{proof}[\pfref{prop:dec_bandit}]
  We write
\begin{align*}
  \Reg &=
           \sum_{t=1}^{T}\En_{\pi\ind{t}\sim{}p\ind{t}}\brk*{\fstar(\pistar)-\fstar(\pi\ind{t})}\\
         &\leq
           \sum_{t=1}^{T}\En_{\pi\ind{t}\sim{}p\ind{t}}\brk*{\fstar(\pistar)-\fstar(\pi\ind{t})}
           - \gamma{}\cdot
           \En_{\pi\ind{t}\sim{}p\ind{t}}\brk*{(\fstar(\pi\ind{t})-\fhat\ind{t}(\pi\ind{t}))^2}
           + \gamma\cdot{}\EstSq.
\end{align*}
For each $t$, since $\fstar\in\cF$, we have
\begin{align}
  &\En_{\pi\ind{t}\sim{}p\ind{t}}\brk*{\fstar(\pistar)-\fstar(\pi\ind{t})}
           - \gamma{}\cdot
  \En_{\pi\ind{t}\sim{}p\ind{t}}\brk*{(\fstar(\pi\ind{t})-\fhat\ind{t}(\pi\ind{t}))^2}\notag\\
&  \leq  \sup_{f\in\cF}\crl*{\En_{\pi\ind{t}\sim{}p\ind{t}}\brk*{f(\pif)-f(\pi\ind{t})}
           - \gamma{}\cdot
                                                                                                  \En_{\pi\ind{t}\sim{}p\ind{t}}\brk*{
                                                                                                  (f(\pi\ind{t})-\fhat\ind{t}(\pi\ind{t}))^2
                                                                                                  }}\notag\\
  &  = \inf_{p\in\Delta(\Act)}\sup_{f\in\cF}\En_{\pi\sim{}p}\brk*{f(\pif)-f(\pi)
    - \gamma{}\cdot
    (f(\pi\ind{t})-\fhat\ind{t}(\pi\ind{t}))^2
    }\notag\\
  &  = \comp(\cF,\fhat\ind{t}),\label{eq:minimax_regret_structured}
\end{align}
where the first equality above uses that $p\ind{t}$ is chosen as the
minimizer for $\comp(\cF,\fhat\ind{t})$. Summing across rounds, we conclude that
\[
  \Reg \leq{} \max_{t\in\brk{T}}\comp(\cF,\fhat\ind{t})\cdot{}T + \gamma\cdot\EstSq.
\]
It remains to bound $\comp(\cF,\fhat\ind{t})$ by $\comp(\cF)$. If the
oracle's estimates satisfy $\fhat\ind{t}\in\conv(\cF)$, as is the case
for exponential weights, this is immediate from the definition
\cref{eq:dec_max_structured}. For arbitrary estimates
$\fhat\ind{t}:\Pi\to\bbR$, it follows from \cref{prop:dec_unconstrained},
which shows that references outside $\conv(\cF)$ do not increase the
\CompShort.
The second statement follows on the event that the oracle guarantee
$\EstSq\leq\EstSq(\cF,T,\delta)$ holds.
\end{proof}

Of course, the bound is only useful if the \CompShort is indeed
bounded, which is not immediately apparent. \cref{sec:lower} shows that
boundedness is not just sufficient but \emph{necessary} for low regret,
in a fairly strong quantitative sense; for now, we build intuition
through examples. 
\subsubsection{Warmup: Computing the DEC for Multi-Armed Bandits} 
To understand how to apply the general DEC regret bound to specific classes of interest, we begin with the multi-armed bandit, where
$\Pi=\brk{A}$ and $\cF=\bbR^{A}$.
Here the Inverse Gap Weighting distribution of \pref{sec:cb} turns out
to be the exact minimizer.
    \begin{prop}[\igw{} minimizes the DEC]
      \label{prop:igw_exact}
      For the Multi-Armed Bandit setting, where $\Pi=\brk{A}$ and
      $\cF=\bbR^{A}$, the Inverse Gap Weighting distribution $p=\igw_{2\gamma}(\fhat)$
      in \pref{eq:igw} is the \emph{exact} minimizer for
      $\comp(\cF,\fhat)$, and certifies that $\comp(\cF,\fhat)=\frac{A-1}{4\gamma}$.
    \end{prop}
By rewriting \pref{prop:igw_mab}, it is straightforward to
deduce that the \CompShort is bounded by $\frac{A}{\gamma}$, but
\pref{prop:igw_exact} shows that \igw{} is actually the \emph{best
  possible distribution} for this minimax problem. In this sense, the
\squarecb algorithm can be seen as a (contextual) special case of the \etdtext principle. Note that to
attain the exact optimal value (instead of a bound that is
optimal up to constants), we use $\igw_{2\gamma}$ as opposed
$\igw_{\gamma}$ as in \pref{prop:igw_mab}; the reason why this
choice for $\gamma$ is optimal is related to the fact that the inequality
$xy\leq{}x^2+\frac{1}{4}y^2$ is tight in general.

    \begin{proof}[\pfref{prop:igw_exact}]
      \newcommand{\istar}{i^{\star}}
      We rewrite the minimax problem as
      \begin{align*}
        &\min_{p\in\Delta(\brk{A})}\max_{f\in\bbR^{A}}\En_{\pi\sim{}p}\brk*{f(\pi_f)-f(\pi)-\gamma(f(\pi)-\fhat(\pi))^2}\\
        &=\min_{p\in\Delta(\brk{A})}\max_{f\in\bbR^{A}}\max_{\pistar\in\brk{A}}\En_{\pi\sim{}p}\brk*{f(\pistar)-f(\pi)-\gamma(f(\pi)-\fhat(\pi))^2}\\
        &=\min_{p\in\Delta(\brk{A})}\max_{\pistar\in\brk{A}}\max_{f\in\bbR^{A}}\En_{\pi\sim{}p}\brk*{f(\pistar)-f(\pi)-\gamma(f(\pi)-\fhat(\pi))^2}.
      \end{align*}
      For any fixed $p$ and $\pistar$, first-order
      conditions for optimality imply that the choice for $f$ that maximizes this
      expression is
      \begin{align*}
        f(\pi) = \fhat(\pi) - \frac{1}{2\gamma} + \frac{1}{2\gamma{}p(\pistar)}\indic{\pi=\pistar}.
      \end{align*}
      This choice gives
      \[
        \En_{\pi\sim{}p}\brk*{f(\pistar) - f(\pi)}
        = \En_{\pi\sim{}p}\brk*{\fhat(\pistar) - \fhat(\pi)} + \frac{1-p(\pistar)}{2\gamma{}p(\pistar)}
      \]
      and
      \[
        \gamma\En_{\pi\sim{}p}\brk*{(f(\pi)-\fhat(\pi))^2}
        = \frac{1-p(\pistar)}{4\gamma} +
        \frac{(1-p(\pistar))^2}{4\gamma{}p(\pistar)} =
        \frac{1}{4\gamma{}p(\pistar)} - \frac{1}{4\gamma}.
      \]
      Plugging in and simplifying, we compute that the original
      minimax game is
      equivalent to
      \begin{align}
        \label{eq:igw_simplified}
\min_{p\in\Delta(\brk{A})}\max_{\pistar\in\brk{A}}\crl*{\En_{\pi\sim{}p}\brk*{\fhat(\pistar)-\fhat(\pi)}
        + \frac{1}{4\gamma{}p(\pistar)}} - \frac{1}{4\gamma}.
      \end{align}
      \noindent\emph{Finishing the proof: Ad-hoc approach.} Observe that for any $p\in\Delta(\Pi)$, we have
      \begin{align*}
        \max_{\pistar\in\brk{A}}\crl*{\En_{\pi\sim{}p}\brk*{\fhat(\pistar)-\fhat(\pi)}
        + \frac{1}{4\gamma{}p(\pistar)}}
        \geq{} \En_{\pistar\sim{}p}\brk*{\En_{\pi\sim{}p}\brk*{\fhat(\pistar)-\fhat(\pi)}
        + \frac{1}{4\gamma{}p(\pistar)}} = \frac{A}{4\gamma},
      \end{align*}
      so no $p$ can attain value better than $\frac{A}{4\gamma}$. If we can show that \igw{} achieves this
      value, we are done.

Observe that by setting $p=\igw_{2\gamma}(\fhat)$, we have that for
all $\pistar$,
\begin{align}
  \label{eq:igw_equalizing}
  \En_{\pi\sim{}p}\brk*{\fhat(\pistar)-\fhat(\pi)}
  + \frac{1}{4\gamma{}p(\pistar)}
  &=   \En_{\pi\sim{}p}\brk*{\fhat(\pistar)-\fhat(\pi)}
  + \frac{\lambda}{4\gamma{}} + \fhat(\pihat) - \fhat(\pistar) \\
    &=   \En_{\pi\sim{}p}\brk*{\fhat(\pihat)-\fhat(\pi)}
        + \frac{\lambda}{4\gamma{}}.\notag
\end{align}
Note that the value on the right-hand side is independent of
$\pistar$. That is, the inverse gap weighting distribution is an equalizing
strategy. This means that for this choice of $p$, we have
\begin{align*}
  \max_{\pistar\in\brk{A}}\crl*{\En_{\pi\sim{}p}\brk*{\fhat(\pistar)-\fhat(\pi)}
    + \frac{1}{4\gamma{}p(\pistar)}} 
    &=\min_{\pistar\in\brk{A}}\crl*{\En_{\pi\sim{}p}\brk*{\fhat(\pistar)-\fhat(\pi)}
      + \frac{1}{4\gamma{}p(\pistar)}} \\
  &=\En_{\pistar\sim{}p}\crl*{\En_{\pi\sim{}p}\brk*{\fhat(\pistar)-\fhat(\pi)}
  + \frac{1}{4\gamma{}p(\pistar)}}  = \frac{A}{4\gamma}.
\end{align*}
Hence, $p=\igw_{2\gamma}(\fhat)$ achieves the optimal value.

\noindent\emph{Finishing the proof: Principled approach.}
We begin by relaxing to $p\in\bbR^{A}_{+}$. Define
\[
g_{\pistar}(p) = \fhat(\pistar)
        + \frac{1}{4\gamma{}p(\pistar)}.
\]
Let $\nu\in\bbR$ be a Lagrange multiplier and $p\in\bbR_{+}^{A}$, and consider the Lagrangian
\begin{align*}
  \cL(p,\nu) = g_{\pistar}(p) - \sum_{\pi}p(\pi)\fhat(\pi) + \nu\prn*{\sum_{\pi}p(\pi)-1}.
\end{align*}
By the KKT conditions, if we wish to show that $p\in\Delta(\Pi)$ is
optimal for the objective in \pref{eq:igw_simplified}, it suffices to find $\nu$ such that\footnote{If
  $p\in\Delta(\Pi)$, the KKT condition that $\frac{d}{d\nu}\cL(p,\nu)=0$
  is already satisfied.}
\[
  \mathbf{0}\in\partial_{p}\cL(p,\nu),
\]
where $\partial_{p}$ denotes the subgradient with respect to
$p$. Recall that for a convex function $h(x)=\max_{y}g(x,y)$, we have
$\partial{}_xh(x)=\conv(\crl*{\grad{}g(x,y)\mid{}g(x,y)=\max_{y'}g(x,y')})$. As
a result, 
\[
  \partial_{p}\cL(p,\nu)
  = \nu\mathbf{1} - \fhat + \conv(\crl{\grad_{p}g_{\pistar}(p)\mid{}g_{\pistar}(p)=\max_{\pi'}g_{\pi'}(p)}).
\]
Now, let $p=\igw_{2\gamma}(\fhat)$. We argue
that $\mathbf{0}\in\partial_p\cL(p,\nu)$ for an appropriate choice of $\nu$. By
\pref{eq:igw_equalizing}, we know that $g_{\pi}(p)=g_{\pi'}(p)$ for
all $\pi,\pi'$ ($p$ is equalizing), so the expression above simplifies to
\begin{equation}
  \label{eq:igw_subgradient}
  \partial_{p}\cL(p,\nu)
  = \nu\mathbf{1} - \fhat + \conv(\crl{\grad_{p}g_{\pistar}(p)}_{\pistar\in\Pi}).
\end{equation}
Noting that
$\grad_pg_{\pistar}(p)=-\frac{1}{4\gamma{}p^2(\pistar)}e_{\pistar}$,
we 
compute
\[
  \mb{\delta}\ldef{}\sum_{\pi}p(\pi)\grad{}g_{\pi}(p) =
  \crl*{-\frac{1}{4\gamma{}p(\pi)}}_{\pi\in\Pi} =
  \crl*{-\frac{\lambda}{4\gamma} - \fhat(\pihat) + \fhat(\pi)}_{\pi\in\Pi},
\]
which has $\mb{\delta} \in
\conv(\crl{\grad_{p}g_{\pistar}(p)}_{\pistar\in\Pi})$. By choosing
$\nu=\frac{\lambda}{4\gamma}+\fhat(\pihat)$, we have
\[
  \nu\mathbf{1} - \fhat + \mb{\delta} = \mathbf{0},
\]
so \pref{eq:igw_subgradient} is satisfied.
    \end{proof}

\subsubsection{Approximate Solutions and Certified Bounds on the \CompShort}
    Before proceeding, let us mention that to apply \etd, it is not
    necessary to exactly solve the minimax problem
    \pref{eq:dec_structured}. Instead, let us say that a distribution
    $p=p(\fhat,\gamma)$ \emph{certifies an upper bound} on the \CompShort if, given
    $\fhat$ and $\gamma>0$, it ensures that
    \begin{align*}
      \sup_{f\in\cF}\En_{\pi\sim{}p}\brk*{
      f(\pif) - f(\pi)
      - \gamma\cdot(f(\pi)-\fhat(\pi))^2
      } \leq \compbar(\cF,\fhat)
    \end{align*}
    for some known upper bound
    $\compbar(\cF,\fhat)\geq\comp(\cF,\fhat)$. In this case, if we define
    $\compbar(\cF)\ldef\sup_{\fhat}\compbar(\cF,\fhat)$, then it is
    simple to see that using this distribution $p\ind{t}=p(\fhat\ind{t},\gamma)$
    within \etd{} gives
    \begin{align*}
      \Reg \leq{}
  \compbar(\cF)\cdot{}T + \gamma\cdot\EstSq
    \end{align*}
    for every realization.
   
    \subsection{Examples: Computing the \CompText}
\label{sec:dec_examples_structured}

    We now compute the \CompShort for classes where, unlike the multi-armed
    bandit, the right exploration strategy is not obvious. In each case
    we exhibit a distribution $p$ that certifies a bound, and read off regret
    via \cref{prop:dec_bandit}.

\subsubsection{Cheating Code}

We begin with the cheating code of \cref{ex:cheating_code}, on which
both \ucb and posterior sampling incur $\sqrt{AT}$ regret.

\begin{prop}[\CompShort for Cheating Code]
  \label{prop:cheating}
  Consider the cheating code in \pref{ex:cheating_code}. For this
  class $\cF$, we have
  \begin{align*}
    \comp(\cF) \approxleq{} \frac{\log_2(A)}{\gamma}.
  \end{align*}
\end{prop}
Combining \cref{prop:cheating} with \cref{prop:dec_bandit} gives
$\Reg\approxleq{}\sqrt{\log_2(A)\,T\log\abs{\cF}}$, whether \etd plays
the certifying distribution or the exact minimizer; a refinement of \etd
\citep{foster2023tight} improves this to the $\log_2(A)$ bound of
\cref{ex:cheating_code}.

\begin{proof}[\pfref{prop:cheating}]
  To simplify exposition, we present a bound on $\comp(\cF,\fhat)$ for this
    example only for
    $\fhat\in\cF$, not for $\fhat\in\conv(\cF)$. A similar approach
    (albeit with a slightly different choice for $p$) leads to the
    same bound on $\comp(\cF)$.  Let
  $\fhat\in\cF$ and $\gamma>0$ be given, and define
  \[
    p = (1-\veps)\pifhat + \veps\cdot\unif(\cC).
  \]
  We show that if we choose $\veps=2\frac{\log_2(A)}{\gamma}$, this strategy certifies that
  \begin{align*}
    \comp(\cF,\fhat) \approxleq{} \frac{\log_2(A)}{\gamma}.
  \end{align*}
  Let $f\in\cF$ be fixed, and consider the value
  \[
  \En_{\pi\sim{}p}\brk*{f(\pif)-f(\pi)
    -\gamma\cdot(f(\pi)-\fhat(\pi))^2
}.
  \]
  We  consider two cases. For the first, if $\pi_f=\pifhat$, then we can upper bound
  \begin{align*}
    \En_{\pi\sim{}p}\brk*{f(\pif)-f(\pi)
    -\gamma\cdot(f(\pi)-\fhat(\pi))^2
    }\leq{}  \En_{\pi\sim{}p}\brk*{f(\pif)-f(\pi)}
    =  \En_{\pi\sim{}p}\brk*{f(\pifhat)-f(\pi)}\leq{}2\veps,
  \end{align*}
  since $f\in\brk{-1,1}$. 

  For the second case, suppose that $\pi_f\neq{}\pi_{\fhat}$. We begin
  by bounding 
  \begin{align*}
    \En_{\pi\sim{}p}\brk*{f(\pif)-f(\pi)
    -\gamma\cdot(f(\pi)-\fhat(\pi))^2
    }\leq{} 2
    -\gamma\cdot\En_{\pi\sim{}p}\brk*{(f(\pi)-\fhat(\pi))^2},
  \end{align*}
  using that $f\in\brk{-1,1}$. To proceed, we want to argue that the
  negative offset term above is sufficiently large; informally, this means that
  we are exploring ``enough''. Observe that since $\pi_f \neq
  \pi_{\fhat}$, if we let
  $b_1,\ldots,b_{\log_2(A)}$ and $b'_1,\ldots,b'_{\log_2(A)}$ denote
  the binary representations for $\pi_f$ and $\pi_{\fhat}$, there
  exists $i$ such that $b_i\neq{}b'_{i}$. As a result, we have
  \begin{align*}
    \En_{\pi\sim{}p}\brk*{(f(\pi)-\fhat(\pi))^2}
    \geq{} \frac{\veps}{\log_2(A)}(f(c_i)-\fhat(c_i))^2
    = \frac{\veps}{\log_2(A)}(b_i-b'_i)^2 = \frac{\veps}{\log_2(A)}.
  \end{align*}
  We conclude that in the second case,
  \[
\En_{\pi\sim{}p}\brk*{f(\pif)-f(\pi)
    -\gamma\cdot(f(\pi)-\fhat(\pi))^2
    } \leq{} 2- \gamma\frac{\veps}{\log_2(A)}.
  \]

  Putting the cases together, we have
  \begin{align*}
    \En_{\pi\sim{}p}\brk*{f(\pif)-f(\pi)
    -\gamma\cdot(f(\pi)-\fhat(\pi))^2
    }
    \leq{} \max\crl*{2\veps, 2- \gamma\frac{\veps}{\log_2(A)}}.
  \end{align*}
  To balance these terms, we set
  \[
    \veps = 2\frac{\log_2(A)}{\gamma},
  \]
  which leads to the result.
\end{proof}

\subsubsection{Linear Bandits}
\label{sec:linear}
\emph{Linear bandits}
\citep{abe1999associative,auer2002finite,dani2008stochastic,chu2011contextual,abbasi2011improved}
are the linear contextual bandit problem of \pref{sec:cb} without
contexts: $\Act$ is arbitrary and
$\cF=\crl*{\pi\mapsto\tri{\theta,\phi(\pi)}\mid{}\theta\in\Theta}$,
where $\Theta\subseteq\sB_2^{d}(1)$ is a parameter set and
$\phi:\Pi\to\sB_2^{d}(1)$ is a fixed feature map known to the learner.

To prove bounds on the \CompShort for this setting, we make use of a
primitive from convex analysis and experimental design known as the \emph{G-optimal design}; the following is the Kiefer--Wolfowitz theorem \citep{kiefer1960equivalence}.\loose
\begin{prop}[G-optimal design]
    For any compact set $\cZ\subseteq\bbR^{d}$ with $\mathrm{dim}\,\mathrm{span}(\cZ)=d$, there exists a
    distribution $p\in\Delta(\cZ)$, called the \emph{G-optimal
      design}, which has
      \begin{equation}
    \label{eq:optimal_design}
    \sup_{z\in\cZ}\tri[\big]{\Sigma_p^{-1}z,z}\leq{}d,
  \end{equation}
  where $\Sigma_p\ldef\En_{z\sim{}p}\brk*{zz^{\trn}}$. 
\end{prop}
The G-optimal design ensures
coverage in every direction of the decision space, generalizing
the notion of uniform exploration for finite decision spaces. In this
sense, it can be thought of as a ``universal'' exploratory
distribution for linearly structured decision spaces. Special cases include:
\begin{itemize}
\item When $\cZ=\Delta(\brk{A})$, we can take
  $p=\unif(e_1,\ldots,e_A)$ as an optimal design
\item When $\cZ=\sB_2^d(1)$, we can again take
  $p=\unif(e_1,\ldots,e_d)$ as an optimal design.
\item For any positive definite matrix $A\psdgt0$, the set
  $\cZ=\crl*{z\in\bbR^{d}\mid{}\tri{Az,z}\leq{}1}$ is an ellipsoid. Letting
  $\lambda_1,\ldots,\lambda_d$ and $v_1,\ldots,v_d$ denote the
  eigenvalues and eigenvectors for $A$, respectively, the distribution
  $p=\unif(\lambda_1^{-1/2}v_1,\ldots,\lambda_d^{-1/2}v_d)$ is an
  optimal design.
\end{itemize}

To see how the G-optimal design can be used for exploration, consider
the following generalization of the $\veps$-greedy algorithm.
\begin{itemize}
\item Let $q\in\Delta(\Pi)$ be the G-optimal design for the set $\crl{\phi(\pi)}_{\pi\in\Pi}$.
  \item At each step $t$, obtain $\fhat\ind{t}$ from a supervised
    estimation oracle. Play $\pihat\ind{t}=\pi_{\fhat\ind{t}}$ with
    probability $1-\veps$, and sample $\pi\ind{t}\sim{}q$ otherwise.
  \end{itemize}
  It is straightforward
  to show that this strategy gives
  $\Reg\approxleq{}d^{1/3}T^{2/3}(\log\abs{\cF})^{2/3}$ for linear
  bandits. The basic idea is to replace
  \pref{eq:eps_greedy_min_probability} in the proof of
  \pref{prop:eps_greedy_cb} with the optimal design property \pref{eq:optimal_design}, using  that the reward functions under consideration are linear.

  The following result shows that by combining optimal design
  with inverse gap weighting, we can obtain a $d/\gamma$ bound on the
  \CompShort, which leads to an improved $\sqrt{dT}$ regret bound.

  \begin{prop}[\CompShort for Linear Bandits]
    \label{prop:dec_linear}
    Consider the linear bandit setting. Let a linear function $\fhat$ and
    $\gamma>0$ be given, consider the following distribution $p$:
    \begin{itemize}
    \item Define $\phibar(\pi) =
      \phi(\pi)/\sqrt{1+\frac{\gamma}{d}\prn[\big]{\fhat(\pifhat)-\fhat(\pi)}}$,
      where $\pifhat=\argmax_{\pi\in\Pi}\fhat(\pi)$.
    \item Let $\qbar\in\Delta(\Pi)$ be the G-optimal design for
      the set $\crl{\phibar(\pi)}_{\pi\in\Pi}$, and define $q=\frac{1}{2}\qbar + \frac{1}{2}\ones_{\pifhat}$.
    \item For each $\pi\in\Pi$, set
      \[
        p(\pi) = \frac{q(\pi)}{\lambda+\frac{\gamma}{d}(\fhat(\pifhat)-\fhat(\pi))},
      \]
      where $\lambda\in\brk{1/2,1}$ is chosen such that
      $\sum_{\pi}p(\pi)=1$.\footnote{The normalizing constant $\lambda$ always
        exists because we have $\frac{1}{2\lambda}\leq\sum_{\pi}p(\pi)\leq\frac{1}{\lambda}$.}
    \end{itemize}
  This strategy certifies that
  \[
    \comp(\cF) \approxleq{} \frac{d}{\gamma}.
  \]
\end{prop}
One can show that $\comp(\cF)\approxgeq\frac{d}{\gamma}$ for this
setting as well, so this is the best bound we can hope for. Combining
this result with \pref{prop:dec_bandit} and using the averaged exponential
weights algorithm for estimation as in \cref{eq:dec_finite} gives $\Reg\approxleq{}
\sqrt{dT\log(\abs{\cF}/\delta)}$.

\begin{proof}[\pfref{prop:dec_linear}]%
  \newcommand{\termthree}{\text{(III)}}%
  \newcommand{\termfour}{\text{(IV)}}%
  \newcommand{\Sigmabar}{\wb{\Sigma}}%
  Fix $f\in\cF$. Let us abbreviate $\eta=\frac{\gamma}{d}$. As in \pref{prop:igw_mab}, we
  break regret into three terms:
	\begin{align*}
		\hspace{-8mm}\En_{\pi\sim{}p}\bigl[ f(\pif)-f(\pi)\bigr] 
		= \underbrace{\En_{\pi\sim{}p}\bigl[ \fhat(\pifhat)-\fhat(\pi)\bigr] }_{\text{(I) exploration bias}} 
		+  \underbrace{\En_{\pi\sim{}p}\bigl[ \fhat(\pi)-f(\pi)\bigr] }_{\text{(II) est error on policy}} 
		+ \underbrace{ f(\pif)-\fhat(\pifhat) }_{\text{(III) est error at opt}} .
	\end{align*}
	The first term captures the loss in exploration that we would
        incur if $\fhat$ were the true reward function, and is equal to:
	\begin{align*}
          \sum_{\pi} \frac{q(\pi)(\fhat(\pifhat) -
          \fhat(\pi))}{\lambda + \eta\left(\fhat(\pifhat) -
          \fhat(\pi)\right)} \leq
          \sum_{\pi}\frac{q(\pi)}{\eta} \leq \frac{1}{\eta},
	\end{align*}
	and the second term, as before, is at most
	\begin{align*}
		&\sqrt{\En_{\pi\sim{}p}\brk[\big]{( \fhat(\pi)-f(\pi))^2}} \leq \frac{1}{2\gamma} + \frac{\gamma}{2}\En_{\pi\sim{}p}( \fhat(\pi)-f(\pi))^2.  
	\end{align*}
	The third term can be written as
	\begin{align*}
    \termthree = 	f(\pif)-\fhat(\pif) 
			- (\fhat(\pifhat) -\fhat(\pif)) 
			&= \tri[\big]{\theta-\thetahat,\phi(\pif)} - (\fhat(\pifhat) -\fhat(\pif)),
        \end{align*}
        where $\theta,\thetahat\in\Theta$ are parameters such that
        $f(\pi)=\tri{\theta,\phi(\pi)}$ and
        $\fhat(\pi)=\tri{\thetahat,\phi(\pi)}$. Defining
        $\Sigma_p=\En_{\pi\sim{}p}\brk*{\phi(\pi)\phi(\pi)^{\trn}}$,
        we can bound
        \begin{align*}
          \tri[\big]{\theta-\thetahat,\phi(\pif)}
          &=
          \tri[\big]{\Sigma_p^{1/2}(\theta-\thetahat),\Sigma_p^{-1/2}\phi(\pif)} \\
            &\leq{}
              \nrm{\Sigma_p^{1/2}(\theta-\thetahat)}_2\nrm{\Sigma_p^{-1/2}\phi(\pif)}_2
              &\leq{}
                \frac{\gamma}{2}\nrm{\Sigma_p^{1/2}(\theta-\thetahat)}_2^2
                + \frac{1}{2\gamma}\nrm{\Sigma_p^{-1/2}\phi(\pif)}_2^2.
        \end{align*}
        Note that
        $\nrm{\Sigma_p^{1/2}(\theta-\thetahat)}_2^2=\En_{\pi\sim{}p}\brk{(\fhat(\pi)-f(\pi))^2}$
        and
        $\nrm{\Sigma_p^{-1/2}\phi(\pif)}_2^2=\tri{\phi(\pif),\Sigma_p^{-1}\phi(\pif)}$, 
        so we have
        \begin{align*}
          \termthree
          \leq{}
          \frac{\gamma}{2}\En_{\pi\sim{}p}\brk{(\fhat(\pi)-f(\pi))^2}
          + \underbrace{\frac{1}{2\gamma}\tri{\phi(\pif),\Sigma_p^{-1}\phi(\pif)} - (\fhat(\pifhat) -\fhat(\pif))}_{\termfour}.
        \end{align*}
        To proceed, observe that
        \begin{align*}
          \Sigma_p &\psdgeq \frac{1}{2}\sum_{\pi}\frac{\qbar(\pi)}{\lambda +
          \eta(\fhat(\pifhat)-\fhat(\pi))}\phi(\pi)\phi(\pi)^{\trn}\\
          &\psdgeq \frac{1}{2}\sum_{\pi}\frac{\qbar(\pi)}{1 +
          \eta(\fhat(\pifhat)-\fhat(\pi))}\phi(\pi)\phi(\pi)^{\trn} 
          \psdgeq \frac{1}{2}\sum_{\pi}\qbar(\pi)\phibar(\pi)\phibar(\pi)^{\trn}\rdef\frac{1}{2}\wb{\Sigma}_{\qbar}
        \end{align*}
        This means that we can bound
        \begin{align*}
          \tri{\phi(\pif),\Sigma_p^{-1}\phi(\pif)}
          &\leq{} 2 \tri{\phi(\pif),\Sigmabar_{\qbar}^{-1}\phi(\pif)} \\
          &=
            2(1+\eta(\fhat(\pifhat)-\fhat(\pif)))\tri{\phibar(\pif),\Sigmabar_{\qbar}^{-1}\phibar(\pif)}
          \\
          &\leq{} 2d(1+\eta(\fhat(\pifhat)-\fhat(\pif))),
        \end{align*}
        where the last line uses that $\qbar$ is the G-optimal design for
        $\crl{\phibar(\pi)}_{\pi\in\Pi}$. We conclude that
        \begin{align*}
          \termfour
          \leq{} \frac{2d}{2\gamma}
          + \frac{2d\eta}{2\gamma}(\fhat(\pifhat)-\fhat(\pif)) - (\fhat(\pifhat)-\fhat(\pif))\leq\frac{d}{\gamma}.
        \end{align*}
  
\end{proof}

  \begin{rem}[Log-determinant strategy for linear bandits]
    When $\Theta=\bbR^{d}$ and $\fhat\in\cF$, an alternative
    strategy for linear bandits is given by
      \[
        p\in\argmax_{p\in\Delta(\Act)}\crl*{\En_{\pi\sim{}p}\brk[\big]{
            \fhat(\pi)} +
          \frac{1}{4\gamma}\log\det\prn*{\En_{\pi\sim{}p}\brk{\phi(\pi)\phi(\pi)^{\trn}}
            -\bar\phi_p\bar\phi_p^{\trn}}
        }.
      \]
    Here $\bar\phi_p=\En_{\pi\sim{}p}\brk{\phi(\pi)}$, so the matrix
    inside the determinant is the centered covariance of the features.\footnote{The determinant is computed on the span of the feature
    differences $\phi(\pi)-\phi(\pi')$, for $\pi,\pi'\in\Pi$.}
    This strategy certifies $\comp(\cF,\fhat)\leq d/(4\gamma)$;
    see \citet[Lemma~2]{foster2020adapting}.
    \end{rem}

\subsubsection{Nonparametric Bandits}
\label{sec:nonparametric}
\newcommand{\met}{\rho}%
\newcommand{\Mcov}{\cN_{\met}}%

For all of the examples so far, we have shown that
\[
\comp(\cF) \approxleq\frac{\textsf{eff-dim}(\cF,\Pi)}{\gamma},
\]
where $\textsf{eff-dim}(\cF,\Pi)$ is some quantity that 
(informally) reflects the amount of exploration required for the class
$\cF$ under consideration ($A$ for bandits, $\log_2(A)$ for the
cheating code, and $d$ for linear bandits). In general though, the 
\CompText does not always shrink at a $\gamma^{-1}$ rate, and can have
slower decay for problems where the optimal rate is worse than
$\sqrt{T}$. We now consider such a setting: a standard \emph{nonparametric} bandit
problem called \emph{Lipschitz
bandits in metric spaces}
\citep{auer2007improved,kleinberg2019bandits}.

We
take $\Act$ to be a metric space equipped with metric $\met$, and define
\[
\cF = \crl*{f:\Act\to\brk{0,1} \mid{} \text{$f$ is $1$-Lipschitz w.r.t $\met$}}.
\]
We give a bound on the \CompText which depends on the \emph{covering
  number} for the space $\Pi$ (with respect to
the metric $\rho$). Let us say that $\Act'\subseteq\Act$ is an $\veps$-cover with
respect to $\met$ if
\[
    \forall{}\pi\in\Act\quad\exists{}\pi'\in\Act'\quad\text{s.t.}\quad
    \met(\pi,\pi')\leq\veps,
  \]
  and let $\Mcov(\Act,\veps)$ denote the size of the smallest such cover.
\begin{prop}[\CompShort for Lipschitz Bandits]
    \label{prop:dec_lipschitz}
    Consider the Lipschitz bandit setting, and suppose
    that there exists $d>0$ such that
    $\Mcov(\Act,\veps)\leq\veps^{-d}$ for all $\veps>0$.
    Let $\fhat:\Pi\to\brk{0,1}$ and $\gamma\geq{}1$ be given and consider the following distribution:
    \begin{enumerate}
    \item Let $\Act'\subseteq\Act$ witness the covering number
      $\Mcov(\Act,\veps)$ for $\veps\in(0,1]$.
    \item Let $p$ be the result of applying the inverse gap weighting strategy in \pref{eq:igw}
      to $\fhat$, restricted to the (finite) decision space $\Act'$.
    \end{enumerate}
    By setting $\veps\propto\gamma^{-\frac{1}{d+1}}$, this strategy certifies that
    \[
            \comp(\cF,\fhat) \approxleq\gamma^{-\frac{1}{d+1}}.
          \]
        \end{prop}
        Ignoring dependence on $\EstSq(\cF,T,\delta)$, this result
        leads to regret bounds that scale as $T^{\frac{d+1}{d+2}}$
        (after tuning $\gamma$ in \cref{prop:dec_bandit}),
        which cannot be improved.

  \begin{proof}[\pfref{prop:dec_lipschitz}]
      Let $f\in\cF$ be fixed. Let $\Act'$ be the $\veps$-cover for $\Pi$. Since $f$
  is $1$-Lipschitz, for all $\pi\in\Pi$ there exists a corresponding
  covering element $\cov(\pi)\in\Act'$ such that
  $\rho(\pi,\cov(\pi))\leq\veps$, and consequently for any distribution $p$,
  \begin{align*}
    \En_{\pi\sim{}p}\brk*{f(\pif) -f(\pi)}
    &\leq \En_{\pi\sim{}p}\brk*{f(\cov(\pif)) -f(\pi)} + \abs{f(\pif) -
    f(\cov(\pif))}\\
    &\leq{} \En_{\pi\sim{}p}\brk*{f(\cov(\pif)) -f(\pi)} +
    \rho(\pif,\cov(\pif))\\
    &\leq{} \En_{\pi\sim{}p}\brk*{f(\cov(\pif)) -f(\pi)} + \veps.
  \end{align*}
  At this point, since $\cov(\pif)\in\Act'$, \pref{prop:igw_mab} ensures
  that if we choose $p$ using inverse gap weighting over $\Act'$, we have
  \[
    \En_{\pi\sim{}p}\brk*{f(\cov(\pif))-f(\pi)}
    \leq{} \frac{\abs{\Act'}}{\gamma} + \gamma\cdot\En_{\pi\sim{}p}\brk*{(f(\pi)-\fhat(\pi))^2}.
  \]
  From our assumption on the growth of $\Mcov(\Act,\veps)$,
  $\abs{\Act'}\leq\veps^{-d}$, so the value is at most
  \[
    \veps +  \frac{\veps^{-d}}{\gamma}.
  \]
  We choose $\veps\propto\gamma^{-\frac{1}{d+1}}$ to balance the terms,
  leading to the result.
    
  \end{proof}

    \subsubsection{Further Examples: Eluder Dimension, Concave, and ReLU Bandits}
    We state some additional upper bounds
    on the \CompShort without proof; details can be found in \cite{foster2021statistical}.

    \begin{example}[\CompText subsumes Eluder Dimension]
        Consider any class $\cF$ with values in $\brk{0,1}$. For all 
  $\gamma\geq{}e$, we have
  \begin{equation}
    \label{eq:comp_eluder}
    \comp(\cF) \approxleq\inf_{\veps>0}\crl*{\veps
      + \frac{\El(\cF-\cF,\veps)\log^{2}(\gamma)}{\gamma}} + \gamma^{-1}.
  \end{equation}
\end{example}
As a special case, this example implies that \etd enjoys a regret
bound for generalized linear bandits similar to that of \ucb.

\begin{example}[Bandits with Concave Rewards]
  The concave (or convex, if one considers losses rather than rewards) bandit problem
\citep{kleinberg2004nearly,flaxman2005online,agarwal2013stochastic,bubeck2016multi,bubeck2017kernel,lattimore2020improved}
is a generalization of the linear bandit. We take
$\Act\subseteq\sB_2^d(1)$ and define
\[
  \cF=\crl*{f:\Act\to\brk{0,1}\mid{}\text{$f$ is concave and
      $1$-Lipschitz w.r.t $\ls_2$}}.
\]
For this setting, whenever $\cF\subseteq(\Pi\to\brk{0,1})$, results of
\citet{lattimore2020improved} imply that
\begin{equation}
  \label{eq:dec_convex}
    \comp(\cF) \approxleq
      \frac{d^{4}}{\gamma}\cdot\polylog(d,\gamma)
  \end{equation}
  for all $\gamma>0$.
\end{example}
For the function class
\[
\cF=\crl*{f(\pi)=1-\relu(\tri{\pi,\theta})\mid{}\theta\in\Theta\subset\sB_2^{d}(1)},
\]
\eqref{eq:dec_convex} leads to a $\sqrt{\poly(d)T}$ regret bound for \etd. This highlights a
case where the Eluder dimension is overly pessimistic, since we saw
that it grows exponentially for this class.

\subsection{Relationship to Optimism and Posterior Sampling}
\label{sec:dec_posterior}
This subsection relates the \CompText and \etd to the two other
approaches covered so far, optimism and posterior sampling. 

\subsubsection{E2D with Confidence Sets: Recovering Optimism}
\newcommand{\red}[1]{#1}
\newcommand{\green}[1]{#1}
\newcommand{\pit}{\pi\ind{t}}
\newcommand{\fhatt}{\fhat}
\newcommand{\pistarr}{\pistar}

The \etd meta-algorithm and the \CompText can be combined with
the idea of \emph{confidence sets} that we used in the \ucb
algorithm. Consider the following variant of \etd.
    \begin{algbox}[frametitle={\algtitle{alg:etd_confidence}{Estimation-to-Decisions (\etd) with Confidence Sets}}]
  \begin{algorithmic}
    \Require Exploration parameter $\gamma>0$; confidence radius $\beta>0$; online regression oracle.
    \For{$t=1,\ldots,T$}
    \State Obtain $\fhat\ind{t}$ from online regression oracle with
    $(\pi\ind{1},r\ind{1}),\ldots,(\pi\ind{t-1},r\ind{t-1})$.
    \State Set
    \[
      \cF\ind{t} = \crl*{f\in\cF\mid{} \sum_{i<t}\En_{\pi\ind{i}\sim{}p\ind{i}}\brk*{(\fhat\ind{i}(\pi\ind{i})-f(\pi\ind{i}))^2}\leq\beta^2}.
    \]
    \State Compute
    \[
      p\ind{t}=\argmin_{p\in\Delta(\Pi)}\max_{f\in\cF\ind{t}}\En_{\pi\sim{}p}\biggl[f(\pif)-f(\pi)
    -\gamma\cdot(f(\pi)-\fhat\ind{t}(\pi))^2
    \biggr].
      \]
\State Select action $\pi\ind{t}\sim p\ind{t}$.
\EndFor{}
\end{algorithmic}
\end{algbox}
This strategy is the same as the basic \etd algorithm, except that at
each step, we compute a confidence set $\cF\ind{t}$ and modify
the minimax problem so that the max player is restricted to choose
$f\in\cF\ind{t}$.\footnote{Note that compared to the confidence
sets used in \ucb, a slight difference
is that we compute
$\cF\ind{t}$ using the estimates $\fhat\ind{1},\ldots,\fhat\ind{T}$
produced by the online regression oracle (this is sometimes referred
to as ``online-to-confidence set conversion'') as opposed to using ERM; this
difference is unimportant, and the latter would work as well.} With this change, the distribution $p\ind{t}$ can be
interpreted as the minimizer for
$\comp(\cF\ind{t},\fhat\ind{t})$. 

To analyze this algorithm, we show that as long as
$\fstar\in\cF\ind{t}$ for all $t$, the same per-step analysis as in
\pref{prop:dec_bandit} goes through, with $\cF$ replaced by
$\cF\ind{t}$. The argument uses nothing about the confidence sets beyond this
containment, so we state the result for an arbitrary sequence of
confidence sets; the construction above is one way to obtain them.
\begin{prop}[Regret of \etd with confidence sets]
  \label{prop:dec_structured_conf}
  Let $\cF\ind{1},\ldots,\cF\ind{T}\subseteq\cF$ be any sequence of
  confidence sets, where $\cF\ind{t}$ may depend on $\cH\ind{t-1}$,
  such that $\fstar\in\cF\ind{t}$ for all $t\in\brk{T}$ with probability
  at least $1-\delta$, and assume the conditions of
  \pref{prop:dec_unconstrained} hold for each $\cF\ind{t}$. Then for any
  $\gamma>0$, \etd with these confidence sets, run with an online
  regression oracle, ensures that with probability at least $1-2\delta$,
  \begin{equation}
    \label{eq:dec_structured_conf}
    \Reg
    \leq \sum_{t=1}^{T}\comp(\cF\ind{t}) + \gamma\cdot{}\EstSq(\cF,T,\delta).
  \end{equation}
  For the construction in \cref{alg:etd_confidence} with
  $\beta^2=\EstSq(\cF,T,\delta)$, the containment holds on the oracle's
  own event, and the bound holds with probability at least $1-\delta$.
\end{prop}
\begin{proof}[\pfref{prop:dec_structured_conf}]
  Condition on the intersection of the event guaranteed by the online
  regression oracle, on which
  $\sum_{t=1}^T\En_{\pi\sim p\ind{t}}[(\fhat\ind{t}(\pi)-\fstar(\pi))^2]\leq\EstSq(\cF,T,\delta)$,
  and the event that $\fstar\in\cF\ind{t}$ for all $t$; by a union
  bound, this intersection has probability at least $1-2\delta$. For the
  construction in \cref{alg:etd_confidence} with
  $\beta^2=\EstSq(\cF,T,\delta)$, every partial sum of the oracle's
  error is at most $\beta^2$ on the oracle's event, so $\fstar\in\cF\ind{t}$
  for every $t$ on that event alone. By the choice of $p\ind{t}$ and the definition of the DEC, we then have
  \begin{align*}
    \En_{\pi\sim p\ind{t}}[\fstar(\pi_{\fstar})-\fstar(\pi)]
    &\leq\comp(\cF\ind{t},\fhat\ind{t})
      +\gamma\En_{\pi\sim p\ind{t}}[(\fhat\ind{t}(\pi)-\fstar(\pi))^2]\\
    &\leq\comp(\cF\ind{t})
      +\gamma\En_{\pi\sim p\ind{t}}[(\fhat\ind{t}(\pi)-\fstar(\pi))^2].
  \end{align*}
  The second inequality uses \pref{prop:dec_unconstrained}; this step is
  needed because $\fhat\ind{t}$ need not lie in $\conv(\cF\ind{t})$.
  Summing over $t$ and using the oracle bound proves the claim.
\end{proof}

This bound is never worse than the one in \pref{prop:dec_bandit}, but
it can be smaller if the confidence sets
$\cF\ind{1},\ldots,\cF\ind{T}$ shrink quickly.

\paragraph{Relation to confidence width and \ucb}
It turns out that the usual \ucb algorithm, which selects
$\pi\ind{t}=\argmax_{\pi\in\Pi}\fbar\ind{t}(\pi)$ for
$\fbar\ind{t}(\pi)=\max_{f\in\cF\ind{t}}f\ind{t}(\pi)$, certifies a
bound on $\comp(\cF\ind{t})$ which is never worse than usual
confidence width we use in the \ucb analysis.
\begin{prop}[\ucb certifies a bound on the DEC]
  \label{prop:dec_width0}
  For any $t$, the \ucb strategy $\pi\ind{t}=\argmax_{\pi\in\Pi}\fbar\ind{t}(\pi)$ certifies that
  \begin{equation}
    \label{eq:dec_width0}
    \comp[0](\cF\ind{t}) \leq \fucb\ind{t}(\pi\ind{t}) - \flcb\ind{t}(\pi\ind{t}).
  \end{equation}
\end{prop}
\begin{proof}[\pfref{prop:dec_width0}]
  By choosing $\pi\ind{t}=\argmax_{\pi\in\Pi}\fbar\ind{t}(\pi)$, we
  have that for any $\fhat$,
  \begin{align*}
    \comp[0](\cF\ind{t},\fhat)
    &=
    \inf_{p\in\Delta(\Pi)}\sup_{f\in\cF\ind{t}}\En_{\pi\sim{}p}\brk*{\max_{\pistar}f(\pistar)
      - f(\pi)} \\
        &\leq
          \sup_{f\in\cF\ind{t}}\brk*{\max_{\pistar}f(\pistar)
          - f(\pi\ind{t})} \\
            &\leq
          \sup_{f\in\cF\ind{t}}\brk*{\max_{\pistar}\fbar\ind{t}(\pistar)
              - f(\pi\ind{t})}\\
                &=
          \sup_{f\in\cF\ind{t}}\brk*{\fbar\ind{t}(\pi\ind{t})
                  - f(\pi\ind{t})}
                  = \fucb\ind{t}(\pi\ind{t}) - \flcb\ind{t}(\pi\ind{t}).
  \end{align*}
  
\end{proof}
As we saw in the analysis of \ucb for multi-armed bandits with
$\Pi=\crl{1,\ldots,A}$ (\pref{sec:ucb_bandits}), the
confidence width in \pref{eq:dec_width0} might be large for a given
round $t$, but by the pigeonhole argument (\pref{lem:confidence_width_potential}), when we sum over all rounds
we have
\[
  \sum_{t=1}^{T}    \comp[0](\cF\ind{t}) \leq
  \sum_{t=1}^{T}\fucb\ind{t}(\pi\ind{t}) - \flcb\ind{t}(\pi\ind{t})
  \leq\bigoht(\sqrt{AT}).
\]
Hence, although \ucb does not minimize the \CompShort, the bound it
certifies is small whenever the confidence widths shrink quickly.

Note that \cref{prop:dec_width0} bounds the \CompShort at $\gamma=0$;
this is only meaningful for shrinking classes, since for $\cF=\bbR^{A}$,
\pref{prop:igw_exact} gives $\comp(\cF)\approxgeq{}A/\gamma\to\infty$
as $\gamma\to{}0$. Allowing $\gamma>0$ gives the following slightly
stronger result, which replaces $\flcb\ind{t}$ by $\fhat\ind{t}$.
        \begin{prop}[\ucb certifies a bound on the DEC, $\gamma>0$]
          \label{prop:ucb_conf_width}
          For any $\gamma>0$, the \ucb strategy
          $\pi\ind{t}=\argmax_{\pi\in\Pi}\fbar\ind{t}(\pi)$ certifies
          that
          \[
            \comp(\cF\ind{t},\fhat\ind{t})
            \leq{} \fucb\ind{t}(\pi\ind{t}) - \fhat\ind{t}(\pi\ind{t})
            + \frac{1}{4\gamma}.
          \]
        \end{prop}
        \begin{proof}[\pfref{prop:ucb_conf_width}] This is a slight
          generalization of the proof of \cref{prop:dec_width0}. By choosing
          $\pi\ind{t}=\argmax_{\pi\in\Pi}\fbar\ind{t}(\pi)$, we have
	\begin{align*}
	    \comp(\red{\cF}\ind{t},\fhatt\ind{t}) &= \min_{p\in\Delta(\Pi)} \max_{f\in\red{\cF\ind{t}}}
		~\En_{\green{\pi}\sim{}p}\biggl[\hspace{1pt}\max_{\pistarr} f(\pistarr)-f(\green{\pi})
	      -\gamma\cdot\hspace{-2pt}(\fhatt\ind{t}(\green{\pi}) - f(\green{\pi}))^2 \hspace{-2pt}
	      \biggr] \\
		  	&\leq\max_{f\in\red{\cF\ind{t}}} 
                     \biggl[\hspace{1pt}\max_{\pistar}f(\pistar)-f(\pi\ind{t})
	      -\gamma\cdot\hspace{-2pt}(\fhatt\ind{t}(\pi\ind{t}) - f(\pi\ind{t}))^2
	      \biggr] \\
  		  	&\leq \max_{f\in\red{\cF\ind{t}}} 
  		\biggl[\hspace{1pt}\fucb\ind{t}(\green{\pi\ind{t}})-f(\green{\pi\ind{t}})
  	      -\gamma\cdot\hspace{-2pt}(\fhatt\ind{t}(\green{\pi\ind{t}}) - f(\green{\pi\ind{t}}))^2 
  	      \biggr] \\  
				&= \max_{f\in\red{\cF\ind{t}}} \underbrace{\biggl[\hspace{1pt}\fhatt\ind{t}(\green{\pi\ind{t}})-f(\green{\pi\ind{t}})
			  	   	      -\gamma\cdot\hspace{-2pt}(\fhatt\ind{t}(\green{\pi\ind{t}}) - f(\green{\pi\ind{t}}))^2
			  	   	      \biggr]}_{\leq \frac{1}{4\gamma}} + \fbar\ind{t}(\pi\ind{t})-\fhat\ind{t}(\pi\ind{t}).
	 \end{align*}          
        \end{proof}

        \subsubsection{The Bayesian DEC: Recovering Posterior Sampling}
        \label{sec:posterior_and_dec}

The \CompText \pref{eq:dec_structured} is a min-max problem. Its
\emph{dual} (max-min) analogue is
\begin{equation}
  \label{eq:comp_dual}
  \compb(\cF,\fhat) =
  \sup_{\mu\in\Delta(\cF)}\inf_{p\in\Delta(\Act)}\En_{f\sim\mu}\En_{\pi\sim{}p}\brk*{f(\pif)-f(\pi)
    -\gamma\cdot(f(\pi)-\fhat(\pi))^2
    }.
  \end{equation}
The dual has a Bayesian interpretation: the adversary selects a
\emph{prior} $\mu$ over $\cF$, and the learner, knowing $\mu$, trades
off regret and information on average over $f\sim\mu$.

Using the minimax theorem (\pref{lem:sion}), one can show that the
\CompText and its Bayesian counterpart coincide.
\begin{prop}[Equivalence of primal and dual DEC]
  \label{prop:swap_dec}
  If $\cF$ and $\Act$ are finite, or more generally under the regularity
  conditions of the minimax theorem (\pref{lem:sion}), we have
  \begin{equation}
    \label{eq:minimax_swap_dec}
    \comp(\cF,\fhat) = \compb(\cF,\fhat).
  \end{equation}
\end{prop}
\begin{proof}[\pfref{prop:swap_dec}]
  For fixed $\fhat$, define
  \[
    L(p,\mu;\fhat)=\En_{f\sim\mu,\pi\sim p}
       [f(\pi_f)-f(\pi)-\gamma(f(\pi)-\fhat(\pi))^2].
  \]
  Maximizing over $f\in\cF$ is equivalent to taking the supremum over
  $\mu\in\Delta(\cF)$. The payoff $L$ is affine in each distribution.
  For finite $\cF$ and $\Act$, both simplices are compact and $L$ is
  continuous, so the minimax theorem (\pref{lem:sion}) implies that the infimum and supremum can be exchanged. The same argument applies more generally whenever the regularity conditions of the minimax theorem hold.
\end{proof}

Thus, any bound on the dual \CompShort immediately yields
a bound on the primal \CompShort. As an example,
we can adapt the posterior sampling/probability matching strategy
introduced in \pref{sec:mab}. When applied to the Bayesian
DEC---this approach selects $p$ to be the action distribution induced by sampling $f\sim\mu$ and selecting
$\pif$.
The following proposition makes the resulting bound explicit.
\begin{prop}[Posterior sampling bounds the DEC]
  \label{prop:dec_posterior_sampling}
  If $|\Act|=A<\infty$ and $\cF$ is uniformly bounded, then for every
  $\gamma>0$ and reference $\fhat$,
  \[
    \compb(\cF,\fhat)\leq\frac{A}{4\gamma}.
  \]
  Consequently, whenever the minimax identity \eqref{eq:minimax_swap_dec}
  applies, $\comp(\cF)\leq A/(4\gamma)$.
\end{prop}
\begin{proof}[\pfref{prop:dec_posterior_sampling}]
  Fix a prior $\mu$ and let $\bar f(\pi)=\En_{f\sim\mu}f(\pi)$.
  Choose the probability-matching distribution
  $p(\pi)=\bbP_{f\sim\mu}(\pi_f=\pi)$, using a fixed rule to choose
  each maximizer $\pi_f$. Write
  \[
    V=\En_{f\sim\mu,\pi\sim p}[(f(\pi)-\bar f(\pi))^2],\qquad
    I=\En_{f\sim\mu,\pi\sim p}[(f(\pi)-\fhat(\pi))^2].
  \]
  Here $f$ and $\pi$ are drawn independently. Variance decomposition
  gives $I=V+\En_{\pi\sim p}[(\bar f(\pi)-\fhat(\pi))^2]\geq V$.
  Since the marginal distribution of $\pi_f$ is $p$, the terms involving
  $\bar f$ cancel in the expected regret:
  \begin{align*}
    \En_{f\sim\mu,\pi\sim p}[f(\pi_f)-f(\pi)]
    &=\En_{f\sim\mu}[f(\pi_f)-\bar f(\pi_f)]
    \leq\sqrt{AV},
  \end{align*}
  where the inequality is \pref{lem:mab_decoupling_basic}.
  Thus the objective in the dual DEC, evaluated at this $p$, is at most
  \[
    \sqrt{AV}-\gamma I\leq\sqrt{AI}-\gamma I
    =\frac{A}{4\gamma}
       -\gamma\left(\sqrt I-\frac{\sqrt A}{2\gamma}\right)^2
    \leq\frac{A}{4\gamma}.
  \]
  Taking the supremum over priors proves the dual bound. The primal
  bound follows from \pref{prop:swap_dec} and the supremum over references.
\end{proof}
More generally, most analysis techniques for the Bayesian setting can be viewed as
implicitly providing bounds on the dual \CompText
\citep{russo2014learning,bubeck2015bandit,bubeck2016multi,russo2018learning,lattimore2019information,lattimore2020improved}. Notably,
the dual \CompShort is always bounded by a Bayesian
complexity measure known as the \emph{information ratio}, which is
used throughout the literature on Bayesian bandits and reinforcement
learning \citep{foster2022complexity}. Beyond the primal and dual
\CompText, there are deeper connections between the \CompShort and
Bayesian algorithms, including a Bayesian counterpart to the \etd
algorithm itself \citep{foster2021statistical}.

\subsection{Incorporating Contexts\bonus}
\label{sec:structured_contexts}
The \CompText and \etd extend directly to \emph{contextual}
structured bandits, and hence to a version of \squarecb for general
decision spaces. Consider the following protocol.
\begin{algbox}[frametitle={\algprotocol{Contextual Structured Bandit Protocol}}]
  \begin{algorithmic}
\For{$t=1,\ldots,T$}
\State Observe context $x\sups{t}\in\cX$.
\State Select decision $\pi\sups{t}\in\Pi$. \hfill\algcommentlight{$\Pi$ is
  large and potentially continuous.}
\State Observe reward $r\sups{t}\in\bbR$.
\EndFor{}
\end{algorithmic}
\end{algbox}
This is the same as the contextual bandit protocol in \pref{sec:cb},
except that we allow $\Pi$ to be large and potentially continuous. As
in that section, we allow the contexts $x\ind{1},\ldots,x\ind{T}$ to
be generated in an arbitrary, potentially adversarial fashion, but
assume that
\[
r\ind{t}\sim\Mstar(\cdot\mid{}x\ind{t},\pi\ind{t}),
\]
and define $\fstar(x,\pi)=\En_{r\sim\Mstar(\cdot\mid{}x,\pi)}$. We
assume access to a function class $\cF$ such that $\fstar\in\cF$, and
assume access to an estimation oracle for $\cF$ that ensures that with
probability at least $1-\delta$,
	$$\sum_{t=1}^T \En_{\pi\ind{t}\sim
          p\ind{t}}(\fhat\ind{t}(x\ind{t},\pi\ind{t})-\fstar(x\ind{t},\pi\ind{t}))^2
        \leq \EstSq(\cF, T, \delta).$$
        For $f\in\cF$, we define $\pif(x)=\argmax_{\pi\in\Pi}f(x,\pi)$.
        
        To extend the \etd algorithm to this setting, at each time $t$
        we solve
        the minimax problem corresponding to the \CompShort, but
        condition on the context $x\ind{t}$.
    \begin{algbox}[frametitle={\algtitle{alg:etd_contextual}{Estimation-to-Decisions (\etd) for Contextual Structured Bandits}}]
  \begin{algorithmic}
    \Require Exploration parameter $\gamma>0$; online regression oracle.
    \For{$t=1,\ldots,T$}
    \State Observe $x\ind{t}\in\cX$.
    \State Obtain $\fhat\ind{t}$ from online regression oracle with $(x\ind{1},\pi\ind{1},r\ind{1}),\ldots,(x\ind{t-1},\pi\ind{t-1},r\ind{t-1})$. 
    \State Compute
    \[
      p\ind{t}=\argmin_{p\in\Delta(\Pi)}\max_{f\in\cF}\En_{\pi\sim{}p}\biggl[f(x\ind{t},\pif(x\ind{t}))-f(x\ind{t},\pi)
    -\gamma\cdot(f(x\ind{t},\pi)-\fhat\ind{t}(x\ind{t},\pi))^2
    \biggr].
      \]
\State Select action $\pi\ind{t}\sim p\ind{t}$ and observe reward $r\ind{t}$.
\EndFor{}
\end{algorithmic}
\end{algbox}
For $x\in\cX$, define
\[
\cF(x,\cdot)=\crl*{f(x,\cdot)\mid{}f\in\cF}
\]
as the projection of $\cF$ onto $x\in\cX$. The following result shows
that whenever the \CompShort is bounded conditionally---that is,
whenever it is bounded for $\cF(x,\cdot)$ for all $x$---this strategy
has low regret.
    \begin{prop}[Regret of \etd for contextual structured bandits]
      \label{prop:dec_contextual_structured}
The \etd algorithm with exploration parameter $\gamma>0$ guarantees that
\begin{equation}
\Reg \leq{}
\sup_{x\in\cX}\comp(\cF(x,\cdot))\cdot{}T + \gamma\cdot\EstSq
\end{equation}
for every realization, and hence, with probability at least $1-\delta$,
the same bound holds with $\EstSq(\cF,T,\delta)$ in place of $\EstSq$.
\end{prop}
The proof is nearly identical to that of \pref{prop:dec_bandit},
conditioning on $x\ind{t}$ at each round, and is omitted.

In particular, since \igw{} exactly minimizes the \CompShort when
$\cF=\bbR^{A}$ (\cref{prop:igw_exact}), \squarecb is the special case
of Contextual \etd with $\cF=\bbR^{A}$.

Going beyond the finite-action setting, it is simplest to interpret
\pref{prop:dec_contextual_structured} when $\cF(x,\cdot)$ has the same
structure for all contexts. One example is \emph{contextual bandits with
  linearly structured decision spaces}. Here, we take
\[
  \cF = \crl*{f(x,a)=\tri*{\phi(x,a),g(x)}\mid{}g\in\cG},
\]
where $\phi(x,a)\in\bbR^{d}$ is a fixed feature map and
$\cG\subset(\cX\to{} \sB_2^d(1))$ is an arbitrary function class. This
setting generalizes the linear contextual bandit problem from
\pref{sec:cb}, which corresponds to the case where $\cG$ is a set of
constant functions. We can apply \pref{prop:dec_linear} to conclude that
$\sup_{x\in\cX}\comp(\cF(x,\cdot))\approxleq{}\frac{d}{\gamma}$, so that
\pref{prop:dec_contextual_structured} gives $\Reg\approxleq\sqrt{dT\cdot\EstSq(\cF,T,\delta)}$.

\subsection{Additional Properties of the \CompText{}\bonus}
\label{sec:dec_properties}
The following proposition is what allows the supremum in
\cref{eq:dec_max_structured} to be taken over $\conv(\cF)$, and
justifies the final step in the proofs of
\cref{prop:dec_bandit,prop:dec_structured_conf} when the oracle's
estimates do not lie in $\conv(\cF)$.
\begin{prop}[References outside the convex hull do not increase the DEC]
  \label{prop:dec_unconstrained}
      Suppose $\cF$ and $\Pi$ are finite. More generally, it suffices that
      the Bayesian DEC equivalence in \eqref{eq:minimax_swap_dec} holds for every reference and that the
      pointwise mean of every prior on $\cF$ belongs to $\conv(\cF)$.
      Then, for any $\gamma>0$,
      \begin{align*}
      \sup_{\fhat:\Pi\to\bbR}\comp(\cF,\fhat)
        = \sup_{\fhat\in\conv(\cF)}\comp[\gamma](\cF,\fhat).
      \end{align*}
    \end{prop}
\begin{proof}[\pfref{prop:dec_unconstrained}]
  The supremum over all reference functions is at least the supremum
  over $\conv(\cF)$. For the reverse inequality, fix any reference
  $\fhat$ and prior $\mu$, and define $\bar f_\mu=\En_{f\sim\mu}\brk{f}\in\conv(\cF)$.
  At each action $\pi$, variance decomposition gives
  \[
    \En_{f\sim\mu}[(f(\pi)-\fhat(\pi))^2]
    =\En_{f\sim\mu}[(f(\pi)-\bar f_\mu(\pi))^2]
       +(\bar f_\mu(\pi)-\fhat(\pi))^2.
  \]
  Let $L(p,\mu;g)$ denote the expected DEC payoff with reference $g$,
  as in the proof of \pref{prop:swap_dec}. The identity above implies
  $L(p,\mu;\fhat)\leq L(p,\mu;\bar f_\mu)$ for every $p$.
  Hence, using \cref{eq:minimax_swap_dec} for equality, we have that
  \begin{align*}
    \comp(\cF,\fhat)
    &=\sup_\mu\inf_p L(p,\mu;\fhat)\\
    &\leq\sup_\mu\inf_p L(p,\mu;\bar f_\mu)\\
    &\leq\sup_\mu\sup_\nu\inf_p L(p,\nu;\bar f_\mu)
      =\sup_\mu\comp(\cF,\bar f_\mu)\\
    &\leq\sup_{g\in\conv(\cF)}\comp(\cF,g).
  \end{align*}
  Taking the supremum
  over $\fhat$ proves the reverse inequality.
\end{proof}

\subsection{Bibliographic Remarks and Further Reading}\label{sec:structured_bib}
\paragraph{Takeaways}
This section extended the Estimation-to-Decisions principle from finite
action sets to structured decision spaces, where the learner must
generalize across decisions rather than across contexts. We showed that optimism
remains a viable strategy when confidence widths shrink quickly, and the
eluder dimension quantifies when this happens, but the cheating code
shows that this rule is not optimal for every class $\cF$. The \CompText resolves this by letting the class
itself determine the exploration distribution: \etd solves a
per-round min-max problem whose value links regret to estimation error, and recovers \igw{} for multi-armed bandits, optimal design for
linear bandits, and covering-based exploration for Lipschitz bandits. The \CompShort is the quantity that \cref{sec:general_dm} generalizes
beyond bandits and that \cref{sec:lower} shows to be necessary.

\paragraph{Bibliographic remarks}
The generalized \ucb algorithm with least-squares confidence sets, the
eluder dimension, \cref{prop:ucb_eluder}, and
\cref{lem:eluder_indicator_bound} are due to \citet{russo2013eluder}.
The cheating code in
\cref{ex:cheating_code} is adapted from the constructions of
\citet{amin2011bandits} and \citet{jun2020crush}. The \CompText, the
\etd meta-algorithm (\cref{alg:etd_structured}), \cref{prop:dec_bandit},
and the bounds in \cref{prop:cheating}, \cref{prop:dec_linear}, and
\cref{prop:dec_lipschitz} are due to \citet{foster2021statistical}; the
$\log(T)$ refinement for the cheating code and the treatment of
reference models in $\conv(\cF)$ appear in \citet{foster2023tight}. The
closed-form minimizers in \cref{prop:igw_exact} and for linear bandits
are from \citet{foster2020adapting,foster2021statistical}. The
existence of the G-optimal design is the Kiefer--Wolfowitz theorem
\citep{kiefer1960equivalence}; our presentation follows
\citet[Chapter 21]{lattimore2020bandit}. Lipschitz bandits in metric
spaces are due to \citet{kleinberg2008multi,kleinberg2019bandits} and
\citet{auer2007improved}. 

The results relating the \CompShort to
optimism and posterior sampling---\etd with confidence sets
(\cref{alg:etd_confidence} and \cref{prop:dec_structured_conf}), the
observation that \ucb certifies a bound on the \CompShort
(\cref{prop:dec_width0} and \cref{prop:ucb_conf_width}), and the
Bayesian \CompShort (\cref{prop:swap_dec}) and its bound by the
information ratio---are from \citet{foster2021statistical}; the
probability-matching argument in \cref{prop:dec_posterior_sampling} behind
the Bayesian bound is the decoupling argument of
\citet{russo2014learning,russo2016information}. \etd for contextual
structured bandits (\cref{alg:etd_contextual},
\cref{prop:dec_contextual_structured}) is from
\citet{foster2021statistical}, and \cref{prop:dec_unconstrained} is
implicit in \citet{foster2021statistical,foster2023tight}.

\paragraph{Further reading}
Textbook treatments of linear and structured bandits are given by
\citet[Part V]{lattimore2020bandit} and
\citet[Chapter 4]{slivkins2019introduction}. \citet{lattimore2017end}, complementing the minimax perspective in this section, show that optimism is asymptotically suboptimal even for finite-armed linear bandits. The information ratio
originates in the Bayesian analysis of posterior sampling and
information-directed sampling
\citep{russo2014learning,russo2016information,russo2018learning}; its
use for minimax regret is due to \citet{lattimore2019information}. Extensions of the \CompText to adversarial decision making can be found in \citet{foster2022complexity,xu2023bayesian}.

\section[Reinforcement Learning: Basic Principles and Finite State Spaces]{Reinforcement Learning:\\ Basic Principles and Finite State Spaces}
\label{sec:mdp}
The bandit problems of \cref{sec:mab,sec:cb,sec:structured} feature one-shot interaction at each learning episode: a decision is made, a reward is observed, and the interaction
ends. In this section, we take the next step along the interactivity axis of
\cref{fig:axes}. In \emph{reinforcement learning}, each decision is a
multi-stage policy that acts on a state evolving in response to the
learner's actions, and the learner observes the intermediate states and
rewards along the way. This extra information supports \emph{credit
  assignment}---recognizing which actions made a policy good or
bad so that the decision can be broken down into a sequence of single-step choices. This section treats the simplest version of the problem,
\emph{tabular} reinforcement learning, in which states and actions are
finite and there is no structure to generalize across. We return to the setting, this time with function approximation and generalization, in \cref{sec:rl}.

\paragraph{Organization}
\cref{sec:mdp_setting} formalizes the episodic reinforcement learning
problem. We first set statistical uncertainty
aside: \cref{sec:mdp_planning} shows how to compute an optimal policy
when the MDP is known, via dynamic programming. We then turn to the
learning problem, following the structure of the preceding sections.
\cref{sec:mdp_uniform_failure} shows that uniform exploration fails, now
with regret exponential in the horizon. \cref{sec:mdp_analysis_tools}
introduces analysis tools for reinforcement learning that are used throughout the remainder of the
monograph. \cref{sec:mdp_optimism} adapts the principle of optimism to
unknown MDPs, and \cref{sec:ucbvi} instantiates it with the \ucbvi algorithm.
Finally, \cref{sec:dec_tabular} gives a second approach, based on
Estimation-to-Decisions.

\subsection{The Episodic Reinforcement Learning Problem}
\label{sec:mdp_setting}

  We consider an episodic finite-horizon reinforcement
  learning framework. With $H$ denoting the \emph{horizon}, a
  \emph{Markov Decision Process} (MDP) $M$ takes the form
  $$M=\crl*{\cS, \cA, \crl{\Pm_h}_{h=1}^{H}, \crl{\Rm_h}_{h=1}^{H},
    d_1},$$ where $\cS$ is the state space, $\cA$ is the action space,
  $$\Pm_h:\cS\times\cA\to\Delta(\cS)$$ is the probability transition
  kernel at step $h$, $$\Rm_h:\cS\times\cA\to\Delta(\bbR)$$ is
  the reward distribution, and $d_1\in\Delta(\cS)$ is the initial
  state distribution. We write $\Pm_h(s,a)\in\Delta(\cS)$ for the
  transition distribution at $(s,a)$ and $\Pm_h(s'\mid{}s,a)$ for its
  probability mass at $s'$, and likewise for $\Rm_h(s,a)$ and
  $\Rm_h(\cdot\mid{}s,a)$. We allow
 the reward distribution and transition
  kernel to vary across MDPs, but assume for simplicity that the initial
  state distribution is fixed and known.

  For a fixed MDP $M$, an \emph{episode} proceeds under the
  following protocol.
  At the beginning of the episode, the learner selects a
  randomized, non-stationary \emph{policy}
  $$\pi=(\pi_1,\ldots,\pi_H),$$ where $\pi_h:\cS\to\Delta(\cA)$; we
  let $\PiRNS$ for ``randomized, non-stationary'' denote the set of all such policies. The episode
  then evolves through the following process, beginning from
  $s_1\sim{}d_1$. For $h=1,\ldots,H$:
  \begin{itemize}
  \item $a_h\sim\pi_h(s_h)$.
  \item $r_h\sim\Rm_h(s_h,a_h)$ and $s_{h+1}\sim{}P\sups{M}_h(s_h,a_h)$.
  \end{itemize}
    For notational convenience, we take $s_{H+1}$ to be a deterministic terminal
  state.  The \emph{Markov} property refers to the fact that under this evolution, 
  $$\mathbb{P}^{\sss{M}}(s_{h+1}=s'\mid{} s_h, a_h) = \mathbb{P}^{\sss{M}}(s_{h+1}=s'\mid{} s_h, a_h, s_{h-1}, a_{h-1},\ldots, s_1, a_1).$$

  The value for a policy $\pi$ under $M$ is given by
  \begin{align}
	  \label{eq:value_function_def}
	  \fm(\pi)\ldef\Ens{M}{\pi}\brk*{\sum_{h=1}^{H}r_h},
  \end{align} where
  $\Ens{M}{\pi}\brk{\cdot}$ denotes expectation under the process above. We define an optimal policy for model $M$ as
\begin{align}
	\label{eq:def_opt_pim}
	\pim \in \argmax_{\pi\in\PiRNS}\fm(\pi).
\end{align}

\paragraph{Online reinforcement learning}
For reinforcement learning, we focus on the \emph{online reinforcement learning problem}, in which we interact with an
unknown MDP $\Mstar$ for $T$ episodes.
\begin{algbox}[frametitle={\algprotocol{Online Reinforcement Learning Protocol}}]
  \begin{algorithmic}
\For{$t=1,\ldots,T$}
\State Select policy $\pi\ind{t}\in\PiRNS$ (possibly at random, $\pi\ind{t}\sim{}p\ind{t}$).
\State Execute $\pi\ind{t}$ in $\Mstar$ for one episode.
\State Observe the trajectory $\tau\ind{t}=(s_{1}\ind{t},a_{1}\ind{t},r_{1}\ind{t}),\ldots,(s_{H}\ind{t},a_{H}\ind{t},r_H\ind{t})$.
\EndFor{}
\end{algorithmic}
\end{algbox}
The learner may randomize, selecting $\pi\ind{t}\sim{}p\ind{t}$ for a
distribution $p\ind{t}\in\Delta(\PiRNS)$ that depends on the
trajectories observed so far; note that this randomization is in
addition to the randomness within the policy itself.
The goal is to minimize the total regret
\begin{align}
	\label{eq:mdp_regret}
	\Reg \ldef \sum_{t=1}^{T}\En_{\pi\ind{t}\sim{}p\ind{t}}\brk*{\fstar(\pistar)
  - \fstar(\pi\ind{t})}
\end{align}
against the optimal policy $\pistar\ldef\pimstar$ for $\Mstar$. Here and
throughout, we abbreviate $\fstar\ldef\fmstar$ and $\pistar\ldef\pimstar$
whenever the true MDP $\Mstar$ is meant, retaining the model subscript or
superscript for a generic $M$. Throughout this section, rewards are
bounded, $r_h\in\brk{0,1}$, and normalized so that
$\sum_{h=1}^{H}r_h\in\brk{0,1}$ almost surely; regret bounds hold with
probability at least $1-\delta$ unless stated otherwise.

The online RL framework is a strict generalization of (structured) bandits and
contextual bandits (with i.i.d. contexts). Indeed, if $\cS=\{s_0\}$
and $H=1$, each episode amounts to choosing an action
$a\sups{t}\in\cA$ and observing a reward $r\sups{t}$ with mean
$f\sups{M}(a\sups{t})$, which is precisely a bandit problem. On the
other hand, taking $\cS=\cX$ and $H=1$ puts us in the setting of
contextual bandits, with $d_1$ being the distribution of contexts. In
both cases, the notion of regret \eqref{eq:mdp_regret} coincides with
the notion of regret in the respective setting. Conversely, one
could view an episodic MDP as a structured bandit whose decisions are
policies and whose rewards are episode returns. What distinguishes
reinforcement learning is the additional feedback: the learner observes
the intermediate states and rewards, not just the return, and this is
what makes credit assignment possible.

In the language of
\cref{fig:axes}, reinforcement learning moves further along the
interactivity axis; the tabular setting, like the multi-armed bandit,
assumes no structure across states and actions, so their cardinalities $S=\abs{\cS}$ and $A=\abs{\cA}$ appear in
the rates; function approximation is the subject of \cref{sec:rl}.

\begin{rem}[Alternative MDP formulations]
We mention in passing that many alternative formulations for Markov
decision processes and for the reinforcement learning problem appear
throughout the literature. For example, MDPs can be studied with infinite horizon
(with or without discounting), and an alternative to minimizing regret
is to consider \emph{PAC-RL} which aims to minimize the sub-optimality
of a final output policy produced after exploring for $T$ rounds. The core ideas are largely the same, but we find the formluation here to be cleanest, while capturing the essential features of the problem.
\end{rem}

\subsection{Planning via Dynamic Programming}
\label{sec:mdp_planning}

We first set statistical uncertainty aside: suppose that the MDP
$M$ is known, so that the only task is to find a policy maximizing
$\fm(\pi)$, as in \eqref{eq:def_opt_pim}. This is the \emph{planning}
problem. It arises on its own in games such as chess or backgammon,
where the transition rules are given, and it is a prerequisite for the
learning problem, in which the learner must additionally estimate the
unknown $\Mstar$ from interaction. We show that planning is solved by
the principle of \emph{dynamic programming}, which addresses credit
assignment directly: it breaks the choice of a multi-stage policy into a
sequence of single-step choices, each made with knowledge of the value
of the actions that follow.

\paragraph{Value functions}
The objects that make this decomposition possible are
intermediate ``reward-to-go'' functions. For a given model $M$ and policy $\pi$, we define the \emph{state-action
value function} and \emph{state value function} via
\[
Q_h^{\sss{M},\pi}(s,a)=\En^{\sss{M},\pi}\brk*{\sum_{h'=h}^{H}r_{h'}\mid{}s_h=s,
  a_h=a},
\mathand V_h^{\sss{M},\pi}(s)=\En^{\sss{M},\pi}\brk*{\sum_{h'=h}^{H}r_{h'}\mid{}s_h=s}.
\]
Hence, the definition in \eqref{eq:value_function_def} reads
$$\fm(\pi) = \En_{s\sim{}d_1, a\sim{}\pi_1(s)}\brk*{Q_1^{\sss{M},\pi}(s,a)} = \En_{s\sim{}d_1}\brk*{V_1^{\sss{M},\pi}(s)}$$

We begin by observing that the optimal policy $\pim$ in \eqref{eq:def_opt_pim} may not be uniquely
defined. For instance, if $d_1$ assigns zero probability to some state
$s_1$, the behavior of $\pim$ on this state is immaterial. We now
introduce a fundamental result, \pref{prop:bellman}, which
guarantees existence of an optimal policy $\pim=(\pi\subs{M,1},\ldots,\pi\subs{M,H})$ that 
maximizes $\Vmpi[\pi]_1(s)$ over $\pi\in\PiRNS$ for all states
$s\in\cS$ \emph{simultaneously} (rather than just on average, as in
\eqref{eq:def_opt_pim}). The fact that such a policy exists may seem magical at first, but it is rather straightforward. Indeed, if $\pi\subs{M,h}(s)$ is defined for all $s\in\cS$ and $h=2,\ldots,H$, then defining the optimal $\pi\subs{M,1}(s)$ at any $s$ is a matter of greedily choosing an action that maximizes the sum of the expected immediate reward and the remaining expected reward under the optimal policy. Indeed, this observation is Bellman's principle of optimality, stated in its original form as follows \citep{bellman1954theory}:\loose
\begin{figure}[H]
  \centering
    \includegraphics[width=.8\textwidth]{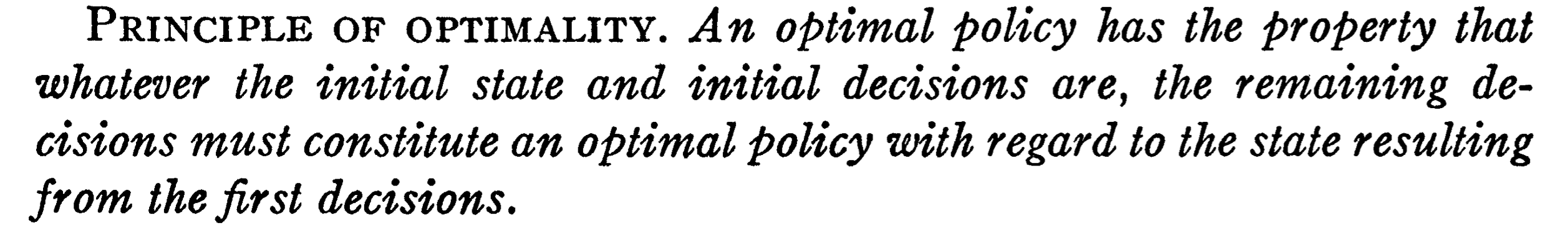}
\end{figure}
To state the result formally, we introduce the \emph{optimal value functions} as 
\begin{align}
	\label{eq:def_opt_value_fn}
	\Qmstar_h(s,a)=\max_{\pi\in\PiRNS}  \En^{\sss{M}, \pi}\brk*{\sum_{h'=h}^H r_{h'} \mid{} s_h=s, a_h=a }
\mathand \Vmstar_h(s)= \max_{a} \Qmstar_h(s,a)
\end{align}
for all $s\in\cS$, $a\in\cA$, and $h\in[H]$; we adopt the convention
that $\Vmstar_{H+1}(s)=\Qmstar_{H+1}(s,a)=0$. Since these optimal
values are separate maximizations for each $s,a,h$, it is reasonable
to ask whether there exists a single policy that maximizes all these
value functions simultaneously. Indeed, the following lemma shows that there exists $\pim$ such that for all $s,a,h$,
\begin{align}
	\label{eq:def_opt_value_fn2}
	\Qmstar_h(s,a) = \Qmpi[\pim]_h(s,a),
\mathand \Vmstar_h(s)=\Vmpi[\pim]_h(s).
\end{align}

\begin{prop}[Bellman Optimality]
	\label{prop:bellman}
	The optimal value function \eqref{eq:def_opt_value_fn} for
        MDP $M$ can be computed via $\Vmpi[\pim]_{H+1}(s)\ldef 0$, and for each $s\in\cS$,
	\begin{align}
		\Vmpi[\pim]_{h}(s) = \max_{a\in\cA} \En^{\sss{M}}\brk*{r_h + \Vmpi[\pim]_{h+1}(s_{h+1}) \mid{} s_h=s, a_h=a}.
	\end{align}
        The optimal policy is given by
	\begin{align}
		\pi\subs{M,h}(s) \in \argmax_{a\in\cA} \En^{\sss{M}}\brk*{r_h + \Vmpi[\pim]_{h+1}(s_{h+1}) \mid{} s_h=s, a_h=a}.
	\end{align}
	Equivalently, for all $s\in\cS$, $a\in\cA$, 
	\begin{align}
		\label{eq:bellman_backup}
		\Qmpi[\pim]_{h}(s, a) = \En^{\sss{M}}\brk*{r_h + \max_{a'\in\cA} \Qmpi[\pim]_{h+1}(s_{h+1}, a') \mid{} s_h=s, a_h=a},
	\end{align}
	and an the optimal policy is given by
	\begin{align}
		\pi\subs{M,h}(s) \in \argmax_{a\in\cA} \Qmpi[\pim]_{h}(s, a).
	\end{align}
\end{prop}
\paragraph{The Bellman operator}
The update in \eqref{eq:bellman_backup} is referred to as \emph{value iteration} (VI). It is useful to introduce a more succinct notation for this update. For an MDP $M$, define the \emph{Bellman Operators}
$\cT\sups{M}_1,\ldots,\cT\sups{M}_H$ via
\begin{align}
	[\cT_h\ind{\sss{M}} Q](s, a) = \En_{r_h\sim R\sups{M}_h(s,a), s_{h+1}\sim{}P_h\sups{M}(s,a)}\brk*{r_h + \max_{a'\in\cA} Q(s_{h+1}, a')}
\end{align}
for any $Q:\cS\times\cA\to\reals$. We write the
expectation above more succinctly as
\begin{align}
 	[\cT_h^{\sss{M}} Q](s, a) = \En^{\sss{M}}\brk*{r_h(s_h, a_h) + \max_{a'\in\cA} Q(s_{h+1}, a') \mid{} s_h=s, a_h=a}
\end{align}
\begin{figure}
    \begin{center}
    \includegraphics[width=\textwidth]{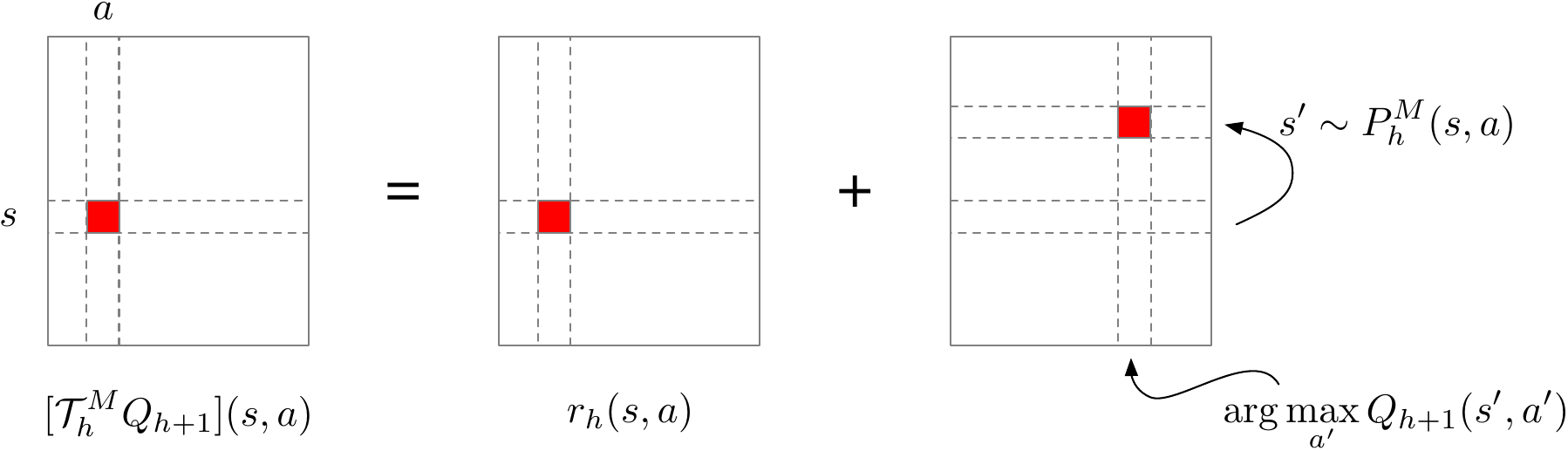}
    \end{center}
    \caption{The Bellman operator. Applying $\cT_h^{\sss{M}}$ to $Q_{h+1}$ adds the immediate reward to the expected greedy value at the next state.}
    \label{fig:bellman_operator}
\end{figure}\cref{fig:bellman_operator} visualizes the operator: $\cT_h^{\sss{M}}$
maps a function of the next layer's state--action pairs to a function of
the current layer's, by taking one step of the dynamics and one greedy
maximization. In the language of Bellman operators, \eqref{eq:bellman_backup} can be written as
\begin{align}
	\label{eq:bellman_backup_operator}
	\Qmpi[\pim]_h = \cT\sups{M}_h \Qmpi[\pim]_{h+1}.
\end{align}

\subsection{The Need for Strategic Exploration}
\label{sec:mdp_uniform_failure}

The task of planning using dynamic programming---which requires
knowledge of the MDP---is fairly straightforward, at least if we
disregard the computational concerns. In this monograph, however, we are
interested in the problem of learning to make decisions in the face of
an unknown environment. Minimizing regret in an unknown MDP requires exploration. As the next example shows, exploration in MDPs is a more delicate issue than in bandits. 

Recall that naive exploration with \egreedy, while having a suboptimal rate in $T$ ($T^{2/3}$ as opposed to $\sqrt{T}$), gives a meaningful and useful regret guarantee for multi-armed and contextual bandits. The following classical example, a so-called ``combination lock,'' shows that for reinforcement learning, the performance of this strategy can be disastrous rather than merely suboptimal, as it leads to exponential regret in the horizon $H$. 

Consider the MDP depicted in \pref{fig:graphics_combination_lock},
with $H+2$ states, and two actions $a_g$ and $a_b$, and a starting
state $1$. The ``good'' action $a_g$ deterministically leads to the
next state in the chain, while the ``bad'' action deterministically
leads to a terminal state. The only place where a non-zero reward can
be received is the last state $H$, if the good action is chosen. The
starting state is $1$, and so the only way to receive non-zero reward
is to select $a_g$ for \emph{all} the $H$ time steps within the
episode. Since the length of the episode is also $H$, selecting
actions uniformly brings no information about the optimal sequence of
actions, unless by chance all of the actions sampled
happen to be good; the probability that this occurs is exponentially
small in $H$. Of course, the learner does not know in advance which
action is the good one at each state, so any algorithm that does not adjust its exploration strategy adaptively must guess all $H$ of them at once. This means that $T$
needs to be at least $\Omega(2^H)$ to achieve nontrivial regret.
\begin{figure}
  \centering
    \includegraphics[width=.7\textwidth]{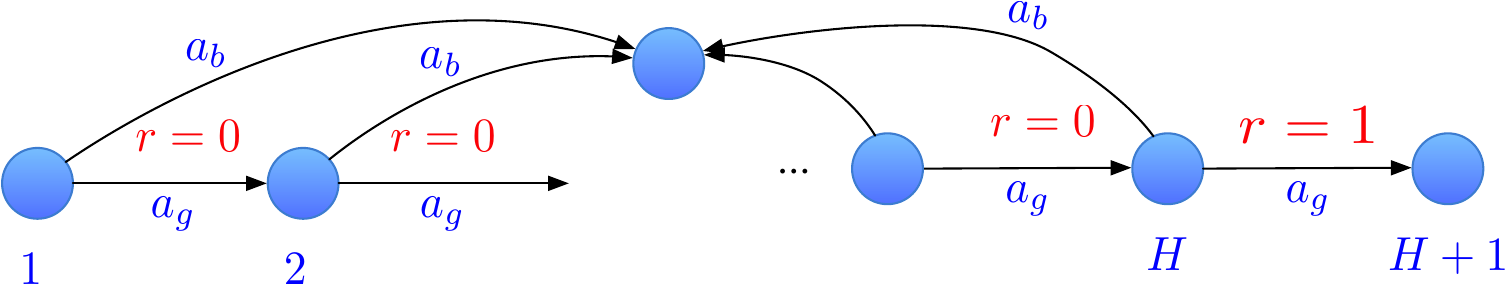}
  \caption{The combination lock MDP. The good action $a_g$ advances along the chain and the bad action $a_b$ leads to a terminal state; reward is received only at the end of the chain.}
  \label{fig:graphics_combination_lock}
\end{figure}

Given the failure of \egreedy for this example, one can ask whether other algorithmic
principles also fail. As we now show, appropriate adaptations of the optimism principle and the Estimation-to-Decisions principle succeed, and yield regret bounds that are
\emph{polynomial} in the parameters $\abs{\cS}$, $\abs{\cA}$, and $H$. Before diving into the details, we present a collection of standard tools for
analysis in MDPs, which are used throughout the remainder of the
monograph.

\subsection{Analysis Tools: Performance Difference, Bellman Residuals, and Simulation}
\label{sec:mdp_analysis_tools}

One of the most basic tools employed in the analysis of reinforcement learning algorithms is the \emph{performance difference lemma}, which expresses the difference in values for two policies in terms of differences in \emph{single-step actions} taken by the two policies. The simple proof, stated below, proceeds by successively changing one policy into another and keeping track of the ensuing differences in expected rewards. One may also interpret this lemma as a version of the credit assignment mechanism.

Henceforth, we adopt the following simplified notation. When a policy
$\pi$ is applied to the random variable $s_{h}$, we drop the
subscript $h$ and write $\pi(s_{h})$ instead of
$\pi_{h}(s_{h})$, whenever this does not cause confusion.\loose
\begin{lem}[Performance Difference Lemma]
	\label{lem:perf_diff_lemma}
	For any $s\in\cS$, and $\pi,\pi'\in\PiRNS$,
	\begin{align}
		V_1^{\sss{M},\pi'}(s) - V_1^{\sss{M},\pi}(s)  = \sum_{h=1}^H \En^{\sss{M},\pi}\brk*{  Q_h^{\sss{M},\pi'}(s_h,\pi'(s_h)) - Q_h^{\sss{M},\pi'}(s_h,a_h) \mid{} s_1=s}
	\end{align}	
\end{lem}
\begin{proof}[\pfref{lem:perf_diff_lemma}]
  Fix a pair of policies $\pi,\pi'$ and define
	$$\pi^h =  (\pi_1,\ldots,\pi_{h-1}, \pi'_h,\ldots, \pi'_H),$$
	with $\pi^1 = \pi'$ and $\pi^{H+1} = \pi$. By telescoping, we can write
	\begin{align}
		V_1^{\sss{M},\pi'}(s)- V_1^{\sss{M},\pi}(s)  = \sum_{h=1}^{H}   V_1^{\sss{M},\pi^h}(s)- V_1^{\sss{M},\pi^{h+1}}(s) .
	\end{align}
        Observe that for each $h$, we have
	\begin{align}
		\label{eq:perf_diff_proof_1}
		 V_1^{\sss{M},\pi^h}(s)-V_1^{\sss{M},\pi^{h+1}}(s) =    \En^{\sss{M},\pi^h}\brk*{\sum_{h'=1}^{H}r_{h'}\mid{}s_1=s} - \En^{\sss{M},\pi^{h+1} }\brk*{\sum_{h'=1}^{H}r_{h'}\mid{}s_1=s}.
	\end{align}
        Here, one process evolves according to $(M,\pi^{h})$ and the
        other according to $(M,\pi^{h+1})$. The processes only
        differ in the action taken once the state $s_h$ is reached. In the former, the action $\pi'(s_h)$ is taken, whereas in the latter it is $\pi(s_h)$. Hence, \eqref{eq:perf_diff_proof_1} is equal to 
	\begin{align}
		\En^{\sss{M},\pi}\brk*{ Q_h^{\sss{M},\pi'}(s_h,\pi'(s_h))- Q_h^{\sss{M},\pi'}(s_h,\pi(s_h))\mid{}s_1=s}
	\end{align}
	which can be written as
	\begin{align}
		\En^{\sss{M},\pi}\brk*{ Q_h^{\sss{M},\pi'}(s_h,\pi'(s_h)) -Q_h^{\sss{M},\pi'}(s_h,a_h) \mid{}s_1=s}.
	\end{align}
\end{proof}
We state the performance difference lemma for completeness. It is not
needed for the main results in this monograph, but it is the starting point
for many analyses in reinforcement learning, and the decomposition that
follows can be viewed as a variant of it.

In contrast to the performance difference lemma, which relates the
values of two policies under the same MDP, the next result relates the
performance of the same policy under two different MDPs. Specifically,
the difference in initial value for two MDPs is decomposed into a sum
of errors between layer-wise value functions. The quantity
$Q_h-\cT^{\sss{M}}_hQ_{h+1}$ appearing in the second part is called the
\emph{Bellman residual} of $Q$ at layer $h$ under $M$; it measures how
far $Q$ is from being self-consistent under one application of the
Bellman operator, and appears as a central notion of estimation error in the
analyses that follow.\loose

\begin{lem}[Bellman residual decomposition]
  \label{lem:bellman_residual}
  For any pair of MDPs $M=(\Pm,\Rm)$ and $\Mhat=(\Pmhat,\Rmhat)$, for any $s\in\cS$, and policies $\pi\in\PiGen$,
    \begin{equation}
      \label{eq:bellman_residual0}
      \Vmpi_1(s)- \Vmhatpi_1(s) =  \sum_{h=1}^{H}\Enm{\Mhat}{\pi}\brk*{\Qmpi_h(s_h,a_h) - r_h -
        \Vmpi_{h+1}(s_{h+1}) \mid s_1=s}
  \end{equation}
  Hence, for $M,\Mhat$ with the same initial state distribution,
  \begin{equation}
    \label{eq:bellman_residual1}
    \fm(\pi)- \fmhat(\pi) =  \sum_{h=1}^{H}\Enm{\Mhat}{\pi}\brk*{\Qmpi_h(s_h,a_h) - r_h -
      \Vmpi_{h+1}(s_{h+1})}.
\end{equation}
In addition, for any MDP $M$ and function $Q=(Q_1,\ldots,Q_H,Q_{H+1})$
with $Q_{H+1}\equiv 0$, letting $\piqh(s)=\argmax_{a\in\cA}Q_h(s,a)$, we have
    \begin{equation}
      \label{eq:bellman_residual3}
      \max_{a\in\cA}Q_1(s,a)- \Vmpi[\piq]_1(s) =  \sum_{h=1}^{H}\Enm{M}{\piq}\brk[\big]{Q_h(s_h,a_h) - \brk*{\cTm_h Q_{h+1}}(s_h,a_h) \mid s_1=s}.
  \end{equation}
  and, hence,
  \begin{equation}
    \label{eq:bellman_residual2}
    \En_{s_1\sim{}d_1}\brk[\big]{\max_{a\in\cA}Q_1(s_1,a)}- \fm(\piq) =  \sum_{h=1}^{H}\Enm{M}{\piq}\brk[\big]{Q_h(s_h,a_h) - \brk*{\cTm_h Q_{h+1}}(s_h,a_h)}.
\end{equation}
\end{lem}
Note that for the second part of \pref{lem:bellman_residual} $Q=(Q_1,\ldots,Q_H)$ can be
any sequence of functions, and need not be a value function
corresponding to a particular policy or MDP. It is worth noting that $Q$ gives rise to the greedy policy $\piq$, which, in turn, gives rise to $\Qmpi[\piq]$ (the value of $\piq$ in model $M$), but it may well be the case that $\Qmpi[\piq]\neq Q$.
\begin{proof}[\pfref{lem:bellman_residual}]
We prove \pref{eq:bellman_residual1}, and omit the proof
for \cref{eq:bellman_residual0}, which is similar but more verbose. We have
  \begin{align*}
    \sum_{h=1}^{H}\Enm{\Mhat}{\pi}\brk*{\Qmpi_h(s_h,a_h) - r_h -
    \Vmpi_{h+1}(s_{h+1})}
    &= \sum_{h=1}^{H}\Enm{\Mhat}{\pi}\brk*{\Qmpi_h(s_h,a_h)  -
    \Vmpi_{h+1}(s_{h+1})}
      - \Enm{\Mhat}{\pi}\brk*{\sum_{h=1}^{H}r_h}\\
        &= \sum_{h=1}^{H}\Enm{\Mhat}{\pi}\brk*{\Qmpi_h(s_h,a_h)  -
    \Vmpi_{h+1}(s_{h+1})}
          - \fmhat(\pi).
  \end{align*}
  On the other hand, since
  $\Vmpi_h(s)=\En_{a\sim\pi_h(s)}\brk{\Qmpi_h(s,a)}$, a telescoping
  argument yields
  \begin{align*}
    \sum_{h=1}^{H}\Enm{\Mhat}{\pi}\brk*{\Qmpi_h(s_h,a_h)  -
      \Vmpi_{h+1}(s_{h+1})}
    &=\sum_{h=1}^{H}\Enm{\Mhat}{\pi}\brk*{\Vmpi_h(s_h)  -
      \Vmpi_{h+1}(s_{h+1})}\\
        &=\Enm{\Mhat}{\pi}\brk*{\Vmpi_1(s_1)}  -
          \Enm{\Mhat}{\pi}\brk*{\Vmpi_{H+1}(s_{H+1})}\\
    &=\fm(\pi),
  \end{align*}
  where we have used that $\Vmpi_{H+1}=0$, and that both MDPs have the
  same initial state distribution. We prove
  \pref{eq:bellman_residual2} (omitting the proof of
  \cref{eq:bellman_residual3}) using a similar argument. We have
  \begin{align*}
    &\sum_{h=1}^{H}\Enm{M}{\piq}\brk[\big]{Q_h(s_h,a_h) - r_h -
    \max_{a\in\cA}Q_{h+1}(s_{h+1},a)} \\
    &= \sum_{h=1}^{H}\Enm{M}{\piq}\brk[\big]{Q_h(s_h,a_h)  -
      \max_{a\in\cA}Q_{h+1}(s_{h+1},a)} -
      \Enm{M}{\piq}\brk*{\sum_{h=1}^{H}r_h} \\
    &= \sum_{h=1}^{H}\Enm{M}{\piq}\brk[\big]{Q_h(s_h,a_h)  -
      \max_{a\in\cA}Q_{h+1}(s_{h+1},a)} - \fm(\piq).
  \end{align*}
  Since $a_{h+1}=\piqh(s_{h+1})=\argmax_{a\in\cA}Q_{h+1}(s_{h+1},a)$,
we have 
$$\Enm{M}{\piq}\brk[\big]{Q_h(s_h,a_h)  -
      \max_{a\in\cA}Q_{h+1}(s_{h+1},a)}=\Enm{M}{\piq}\brk[\big]{Q_h(s_h,a_h)  -
        Q_{h+1}(s_{h+1},a_{h+1})},$$ and the result follows by telescoping.
  
\end{proof}
A third tool, the \emph{simulation lemma}, compares the value of a
single policy under two different MDPs. We state it in two parts: the
first bounds the value difference by the distance between the two
trajectory distributions, and the second decomposes it into per-layer
errors in transition probabilities and rewards. Both are used in
\cref{sec:dec_tabular} to analyze model-based exploration.
  \begin{lem}[Simulation lemma]
    \label{lem:simulation_basic}
      For any pair of MDPs $M=(\Pm,\Rm)$ and $\Mhat=(\Pmhat,\Rmhat)$ with
  the same initial state distribution, and any policy $\pi$, we have
  \begin{align}
  \label{eq:simulation}
    \abs*{\fm(
    \pi)-\fmhat(\pi)}
  &\leq{} \Dtv{M(\pi)}{\Mhat(\pi)} \\
  &\leq \Dhel{M(\pi)}{\Mhat(\pi)}
  \leq{} \frac{1}{2\eta} + \frac{\eta}{2}\Dhels{M(\pi)}{\Mhat(\pi)}\quad\forall{}\eta>0.
  \end{align}
  \end{lem}
  \begin{lem}[Layer-wise simulation lemma]
    \label{lem:simulation_layerwise}
      In the setting of \cref{lem:simulation_basic},
    \begin{align}
      \label{eq:simulation_basic1}
      &\fm(\pi)- \fmhat(\pi) \notag\\
      &=         \sum_{h=1}^{H}\Enm{\Mhat}{\pi}\brk*{\brk*{(\Pm_h-\Pmhat_h)\Vmpi_{h+1}}(s_h,a_h)
        }\\
        &\qquad+ \sum_{h=1}^{H}\Enm{\Mhat}{\pi}\brk*{
        \En_{r_h\sim\Rm_h(s_h,a_h)}\brk{r_h}
              - \En_{r_h\sim\Rmhat_h(s_h,a_h)}\brk{r_h}
        }\\
      &\leq{}
              \sum_{h=1}^{H}\Enm{\Mhat}{\pi}\brk*{\Dtv{\Pm_h(s_h,a_h)}{\Pmhat_h(s_h,a_h)}
              + \Dtv{\Rm_h(s_h,a_h)}{\Rmhat_h(s_h,a_h)}
        }.
        \label{eq:simulation_basic2}
    \end{align}
  \end{lem}

  \begin{proof}[\pfref{lem:simulation_basic}]
Let $X=\sum_{h=1}^{H}r_h$. Since $X\in\brk*{0,1}$ almost surely, we have
  \begin{align*}
    \abs*{\fm(\pi)-\fmhat(\pi)}
    =   \abs*{\Enm{M}{\pi}\brk{X}-\Enm{\Mhat}{\pi}\brk{X}}
    \leq{} \Dtv{M(\pi)}{\Mhat(\pi)} \leq \Dhel{M(\pi)}{\Mhat(\pi)}.
  \end{align*}
   The final result now follows from the AM-GM inequality.
\end{proof}

  \begin{proof}[\pfref{lem:simulation_layerwise}]
From  \pref{lem:bellman_residual}, we have
    \begin{align*}
    \fm(\pi)- \fmhat(\pi) &=  \sum_{h=1}^{H}\Enm{\Mhat}{\pi}\brk*{\Qmpi_h(s_h,a_h) - r_h -
                            \Vmpi_{h+1}(s_{h+1})} \\
            &=
              \sum_{h=1}^{H}\Enm{\Mhat}{\pi}\brk*{\brk*{\Pm_h\Vmpi_{h+1}}(s_h,a_h)-
              \Vmpi_{h+1}(s_{h+1})
              +\En_{r_h\sim\Rm_h(s_h,a_h)}\brk{r_h}
              - \En_{r_h\sim\Rmhat_h(s_h,a_h)}\brk{r_h}
              }
      \\
      &=
        \sum_{h=1}^{H}\Enm{\Mhat}{\pi}\brk*{\brk*{(\Pm_h-\Pmhat_h)\Vmpi_{h+1}}(s_h,a_h)
        }\\
        &\qquad+ \sum_{h=1}^{H}\Enm{\Mhat}{\pi}\brk*{
        \En_{r_h\sim\Rm_h(s_h,a_h)}\brk{r_h}
              - \En_{r_h\sim\Rmhat_h(s_h,a_h)}\brk{r_h}
        }\\
            &\leq{}
              \sum_{h=1}^{H}\Enm{\Mhat}{\pi}\brk*{\Dtv{\Pm_h(s_h,a_h)}{\Pmhat_h(s_h,a_h)}
              + \Dtv{\Rm_h(s_h,a_h)}{\Rmhat_h(s_h,a_h)}
        },
    \end{align*}
where we have used that $\Vmpi_{h+1}(s)\in\brk*{0,1}$.
    
\end{proof}
When rewards are known and deterministic (\cref{assm:ucbvi}
below), the reward distributions of any two MDPs coincide, so the reward
terms in \cref{eq:simulation_basic1,eq:simulation_basic2} vanish.

\subsection{Optimism in Reinforcement Learning}
\label{sec:mdp_optimism}

To develop algorithms for regret minimization in unknown MDPs, we turn
to the principle of optimism, which we have seen is successful in
tackling multi-armed bandits and linear bandits (in small
dimension). Recall that for bandits, \cref{lem:regret_optimistic} gave
a way to decompose the regret of optimistic algorithms into width of
confidence intervals. What is the analogue of
\pref{lem:regret_optimistic} for MDPs? Thinking of optimistic
estimates at the level of expected rewards for policies $\pi$ is
unwieldy, and we need to dig into the structure of these multi-stage
decisions. In particular, the approach we employ is to construct a
sequence of optimistic \emph{value functions}
$\Qbar_1,\ldots,\Qbar_H$ which are guaranteed to over-estimate the
optimal value function $\Qmstarstar$ for the true MDP $\Mstar$. Since the
optimal value functions of $\Mstar$ appear frequently, we
abbreviate $\Qstar_h\ldef\Qmstarstar_h$ and $\Vstar_h\ldef\Vmstarstar_h$,
retaining the model superscript only when a generic $M\in\cM$ is intended.
For multi-armed bandits,
implementing optimism amounted to adding ``bonuses,'' constructed from past
data, to estimates for the reward function. We construct
optimistic value functions in a similar fashion. Before giving the
construction, we introduce a technical lemma, which quantifies the
error in using such optimistic estimates in terms of \emph{Bellman
  residuals}; Bellman residuals measure self-consistency of the
optimistic estimates under the application of the Bellman
operator.
\begin{lem}[Error decomposition for optimistic policies]
	\label{lem:reg_decomp_optimistic}
	Let $\{\Qbar_1,\ldots,\Qbar_H\}$ be a sequence of functions
        $\Qbar_h:\cS\times\cA\to\reals$ with the property that  for all $(s,a)$,
	\begin{align}
		\label{eq:Qbar_optimistic}
		\Qmstar_h(s,a)\leq \Qbar_h(s,a)
	\end{align}
	and set $\Qbar_{H+1}\equiv 0$. Let
        $\pihat=(\pihat_1,\ldots,\pihat_H)$ be such that $\pihat_h(s)
        = \argmax_{a} \Qbar_h(s,a)$. Then for all $s\in\cS$,
	\begin{align}
		\label{eq:reg_decomp_optimistic}
		\Vmstar_1(s) - \Vmpi[\pihat]_1(s) \leq \sum_{h=1}^H \En^{\sss{M}, \pihat}\brk*{(\Qbar_h- \cT_h^{\sss{M}} \Qbar_{h+1})(s_h, \pihat(s_h)) \mid{} s_1=s}.
	\end{align}
\end{lem}
\cref{lem:reg_decomp_optimistic} is an analogue of
\cref{lem:regret_optimistic} for bandits, with Bellman residuals playing
the role of confidence widths. Optimism lets us pass from the unknown
$\pim$ to the algorithm's own policy $\pihat$, and what remains is a
notion of estimation error for the optimistic value functions.

\cref{lem:reg_decomp_optimistic} tells us that closeness of $\Qbar_h$
to the Bellman backup $\cT_h^{\sss{M}} \Qbar_{h+1}$ implies closeness
of $\pihat$ to $\pim$ in terms of the value. As a sanity check, if
$\Qbar_h=\Qmstar_h$, the right-hand side of
\eqref{eq:reg_decomp_optimistic} is zero, since
$\Qmstar_h=\cT_h^{\sss{M}} \Qmstar_{h+1}$. Crucially, errors do not
accumulate too fast as a function of the horizon. This fact should not
be taken for granted: in general, if $\Qbar$ is not optimistic, it could have been the case that small changes in $\Qbar_h$ exponentially degrade the quality of the policy $\pihat$.

Another important aspect of the decomposition
\eqref{eq:reg_decomp_optimistic} is the \emph{on-policy} nature of the
terms in the sum. Observe that the law of $s_h$ for each of the terms
is given by executing $\pihat$ in model $M$. The distribution of $s_h$
is often referred to as the \emph{roll-in distribution}; when this
distribution is induced by the policy executed by the algorithm, we
may have a better control of the error than in the \emph{off-policy}
case when the roll-in distribution is given by $\pim$ or another
unknown policy.

\begin{proof}[\pfref{lem:reg_decomp_optimistic}]
	Let $\Vbar_h(s) \ldef \max_{a\in\cA} \Qbar_h(s,a)$. Just as in
        the proof of \pref{lem:regret_optimistic}, the assumption that
        $\Qbar_h$ is ``optimistic'' implies that
	$$\Qmstar_h(s_h,\pi_{\sss{M}}(s_h)) \leq \Qbar_h(s_h,\pi_{\sss{M}}(s_h)) \leq \Qbar_h(s_h,\pihat(s_h))$$
	and, hence, $\Vmstar_1(s)\leq \Vbar_1(s)$.
	Then, \eqref{eq:bellman_residual3} applied to $Q=\Qbar$ and $\piq=\pihat$ states that
      \begin{equation}
        \Vbar_1(s)- \Vmpi[\pihat]_1(s) =  \sum_{h=1}^{H}\Enm{M}{\pihat}\brk[\big]{\Qbar_h(s_h,a_h) - \brk*{\cTm_h \Qbar_{h+1}}(s_h,a_h) \mid s_1=s}.
    \end{equation}
\end{proof}

\begin{rem}[Optimism is only needed at the first step]
  In fact, the proof of \cref{lem:reg_decomp_optimistic} only uses
  that the initial value $\Qbar_1$ is optimistic. However, to
  construct a value function with this property, the algorithms we
  consider proceed by backwards induction, producing
  optimistic estimates $\Qbar_1,\ldots,\Qbar_H$ for all steps $h$ in the process.
\end{rem}

\subsection{Optimism for Tabular MDPs: The UCB-VI Algorithm}
\label{sec:ucbvi}
We now instantiate the principle of optimism to give regret bounds for
online reinforcement learning in \emph{tabular MDPs}. Tabular RL may
be thought of as an analogue of finite-armed bandits: we assume no
structure across states and actions, but require that the state and
action spaces are small. The regret bounds we present
depend polynomially on $S=|\cS|$ and $A=|\cA|$, as well as the horizon $H$.

\paragraph{Preliminaries}

For simplicity, we assume that the reward function is known to the
learner, so that only the transition probabilities are unknown. This
does not change the difficulty of the problem in a meaningful way, but
allows us to keep notation light.
\begin{assumption}[Known deterministic rewards]
	\label{assm:ucbvi}
	Rewards are deterministic and known to the learner: $R^{\sss{M}}_h (s,a) = \delta_{r_{h}(s,a)}$ for known $r_{h}:\cS\times\cA\to [0,1]$, for all $M$.
\end{assumption}
Together with the normalization $\sum_hr_h\in\brk{0,1}$ of
\cref{sec:mdp_setting}, this implies $\Vmstar_1(s)\in\brk{0,1}$ for all
$s$ and all $M$.

Define, with a slight abuse of notation, 
\begin{align*}
	n\sups{t}_{h}(s,a) = \sum_{i=1}^{t-1} \indic{(s_{h}^i, a_{h}^i)=(s,a)}, \mathand  n\sups{t}_{h}(s,a,s') = \sum_{i=1}^{t-1} \indic{(s_{h}^i, a_{h}^i, s_{h+1}^i)=(s,a,s')},
\end{align*}
as the empirical state-action and state-action-next state frequencies.
We can estimate the transition probabilities via
\begin{align}
	\label{eq:def_Pt_ucb-vi}
	\widehat{P}\sups{t}_{h} (s'\mid{} s,a) = \frac{n\sups{t}_{h}(s,a,s')}{n\sups{t}_{h}(s,a)}.
\end{align}
Together with the known rewards, the estimated transitions
define an \emph{estimated MDP}
$\Mhat\sups{t}\ldef\crl[\big]{\cS,\cA,\crl{\widehat{P}\sups{t}_h}_{h=1}^{H},\crl{r_h}_{h=1}^{H},d_1}$.
Under \cref{assm:ucbvi}, an MDP is determined by its transition kernels,
so estimating $\Mstar$ and estimating $\crl{\Pmstar_h}_{h=1}^{H}$ are the
same thing, and we use the two viewpoints interchangeably. The analysis
of \ucbvi below works directly with $\widehat{P}\sups{t}_h$; the
estimated MDP $\Mhat\sups{t}$ is the main object in
\cref{sec:dec_tabular}.

\keepwithbox{24}%
\paragraph{The \ucbvi algorithm}

The following algorithm, \ucbvi (``Upper Confidence Bound
Value Iteration'') \citep{azar2017minimax}, combines the notion of optimism with dynamic
programming. 
\begin{algbox}[frametitle={\algtitle{alg:ucbvi}{\ucbvi}}]
	\begin{algorithmic}
    \Require Confidence parameter $\delta\in(0,1)$.
	\For{$t=1,\ldots,T$}
	\State Let $\Vbar\ind{t}_{H+1}\equiv 0$.
		\For{$h=H,\ldots,1$}
		\State Update $n\sups{t}_{h}(s,a)$,
                $n\sups{t}_{h}(s,a,s')$, and $b\ind{t}_{h,\delta}(s,a)$, for
                all $(s,a)\in\cS\times \cA$.
                \Statex[2]\hfill\algcommentlight{$b\ind{t}_{h,\delta}(s,a)$
                  is the bonus computed in \pref{eq:ucbvi_bonus}.}
		\State
                Compute:
		\begin{align}\Qbar\sups{t}_{h}(s,a) = \crl*{ r_{h}(s,a)+ \En_{s'\sim{} \widehat{P}\sups{t}_{h}(\cdot \mid{} s,a)}  \Vbar\sups{t}_{h+1}(s') + b\sups{t}_{h,\delta}(s,a) } \wedge  1. \label{eq:ucbvi_update}\end{align}
		\State Set $\Vbar\sups{t}_{h}(s) = \max_{a\in\cA} \Qbar\sups{t}_{h}(s,a)$ and $\pihat\sups{t}_{h}(s)=\argmax_{a\in\cA} \Qbar\sups{t}_{h}(s,a).$
		\EndFor{}
	\State Collect trajectory $(s_{1}\ind{t},a_{1}\ind{t},r_{1}\ind{t}),\ldots,(s_{H}\ind{t},a_{H}\ind{t},r_{H}\ind{t})$ according to $\pihat\sups{t}$.
	\EndFor{}
	\end{algorithmic}
      \end{algbox}
      We analyze the \ucbvi algorithm using
      \pref{lem:reg_decomp_optimistic}. In constructing functions
      $\Qbar_h$, we need to satisfy two goals: (1) ensure that
      with high probability \eqref{eq:Qbar_optimistic} is satisfied,
      i.e. the value functions are optimistic; and (2) that the functions $\Qbar_h$ are
      ``self-consistent,'' in the sense that the Bellman residuals in
      \eqref{eq:reg_decomp_optimistic} are small. The second
      requirement already suggests that we should define $\Qbar_h$
      approximately as a Bellman backup $\cT_h^{\sss{\Mstar}} \Qbar_{h+1}$,
      going backwards for $h=H+1,\ldots,1$ as in dynamic programming,
      while ensuring the first requirement. In addition to these
      considerations, we must use a surrogate for the Bellman
      operator $\cT_h^{\sss{\Mstar}}$, since the model $\Mstar$ is not
      known. This is achieved by estimating $\Mstar$ using empirical
      transition frequencies. Putting these ideas together gives the
      update in \cref{eq:ucbvi_update}. We apply the principle of
      value iteration, except that
      \begin{enumerate}
      \item For each episode $t$, we augment the rewards $r_h(s,a)$ with a ``bonus''
        $b\ind{t}_{h,\delta}(s,a)$ designed to ensure optimism. 
      \item The Bellman operator $\cT_h^{\sss{\Mstar}}$ is replaced by
        that of the estimated MDP, $\cT_h^{\sss{\Mhat\sups{t}}}$, which uses the
        estimated transition probabilities in \cref{eq:def_Pt_ucb-vi}.
      \end{enumerate}
The bonus functions play precisely the same role as the width of the
confidence interval in \eqref{eq:ucb_confidence_bound}: these bonuses
ensure that \eqref{eq:Qbar_optimistic} holds with high probability, as
\pref{lem:ucb_vi_optimism} shows.
      
The following theorem shows that with an appropriate choice of bonus,
this algorithm achieves a polynomial regret bound.
\begin{theorem}[Regret of \ucbvi]
  \label{thm:ucb-vi}
  For any $\delta>0$, \ucbvi with 
        \begin{align}
          b\ind{t}_{h,\delta}(s,a) =
          2\sqrt{\frac{\log(2SAHT/\delta)}{n_{h}\sups{t}(s,a)}}\label{eq:ucbvi_bonus}
        \end{align}
	guarantees that with probability at least $1-\delta$, 
	$$\Reg \lesssim HS\sqrt{AT} \cdot \sqrt{\log(SAHT/\delta)}$$
      \end{theorem}
	  
	  We mention that a slight variation on \pref{lem:ucbvi_conc}
          below (using the Freedman inequality instead of the Azuma-Hoeffding
          inequality) yields an improved rate of
          $O(H\sqrt{SAT}+\poly(H,S,A)\log T)$. The optimal rate is
          $\Theta(\sqrt{HSAT})$ (a formal lower bound is
          \cref{prop:lower_tabular} in \cref{sec:lower}), which is achieved using the same template with a more careful choice
          for the bonus $b\ind{t}_{h,\delta}$. We remark that care should be taken in comparing results in the literature, as scaling conventions for the
          individual and cumulative rewards (as in \pref{assm:ucbvi}), as well as stationary versus non-stationary transition assumptions, can vary.

          \subsubsection{Optimism and Bellman Residuals Within an Episode}
          In what follows, we prove \cref{thm:ucb-vi}. The analysis features two of the same ingredients as
          that of \ucb for multi-armed bandits: bonuses to ensure optimism, and a potential argument that controls the number of times the algorithm can be surprised by bad state-action pairs. What changes is that we need \cref{lem:reg_decomp_optimistic} to relate the algorithm's (global) suboptimality to its (local) per-state action bonuses.

          To begin, we first prove several helper lemmas
          concerning the performance within each episode $t$. Throughout
          this subsubsection, we fix $t$ and drop the superscript $t$. Given the estimated transitions
$\crl[\big]{\widehat{P}_h(\cdot\mid{}s,a)}_{h,s,a}$, define the
estimated MDP $\Mhat=\crl*{\cS, \cA, \crl{\Phat_h}_{h=1}^{H}, \crl{\Rm[\Mstar]_h}_{h=1}^{H},
    d_1}$. The associated Bellman operator is
\begin{align}
	\label{eq:empirical_bellman_op}
	\cT_h^{\sss{\Mhat}} Q(s,a) =  r_{h}(s,a) + \En_{s'\sim{} \widehat{P}_h(\cdot\mid{} s,a)} \max_a Q(s',a)
\end{align}
for $Q:\cS\times\cA\to\reals$. 
Consider the sequence of functions $\Qbar_h:\cS\times\cA\to[0,1], \Vbar_h:\cS\to[0,1]$, for $h=1,\ldots,H+1$, with $\Qbar_{H+1}\equiv 0$ and 
\begin{align}
	\label{eq:planning_in_estimated_model_optimism}
	\Qbar_{h}(s,a)=\left\{ \brk{\cT_h^{\sss{\Mhat}} \Qbar_{h+1}}(s,a) + b_{h,\delta}(s,a) \right\}  \wedge 1, \mathand \Vbar_h(s) = \max_a \Qbar_h(s,a)
\end{align}
for bonus functions $b_{h,\delta}:\cS\times\cA\to \reals$ that we
leave unspecified for now.  Henceforth, we follow the usual notation that for functions $f,g$ over the same domain, $f\leq g$ indicates pointwise inequality over the domain.

The first lemma we present shows that as long as the bonuses
$b_{h,\delta}$ are large enough to bound the error between the
estimated transition probabilities and true transition probabilities,
the functions $\Qbar_1,\ldots,\Qbar_H$ constructed above are optimistic.
\begin{lem}[Bonuses ensure optimism]
	\label{lem:ucb_vi_optimism}
	Suppose we have estimates $\crl[\big]{\widehat{P}_h(\cdot\mid{}s,a)}_{h,s,a}$ and a function $b_{h,\delta}:\cS\times\cA\to\reals$ with the property that for all $s\in\cS,a\in\cA$,
	\begin{align}
		\abs*{\sum_{s'}\widehat{P}_h(s'\mid{}s,a) \Vstar_h(s')-\sum_{s'} \Pmstar_h(s'\mid{}s,a) \Vstar_h(s')} \leq b_{h,\delta}(s,a).
	\end{align}
Then for all $h\in\brk{H}$, we have
	\begin{align}
		\Qbar_{h} \geq \Qstar_h, \mathand \Vbar_h \geq \Vstar_h
	\end{align}
	for $\Qbar_h, \Vbar_h$ defined in \eqref{eq:planning_in_estimated_model_optimism}.
\end{lem}
\begin{proof}[\pfref{lem:ucb_vi_optimism}]
	The proof proceeds by backward induction on the statement 
	$$\Vbar_h \geq \Vstar_h$$ with $h=H+1$ down to $h=1$. We
        start with the base case $h=H+1$, which is trivial because
        $\Vbar_{H+1} = \Vstar_{H+1} \equiv 0$. Now, assume
        $\Vbar_{h+1}\geq \Vstar_{h+1}$, and let us prove the induction step. Fix $(s,a)\in \cS\times\cA$. If $\Qbar_{h}(s,a)=1$, then, trivially, $\Qbar_{h}(s,a)\geq \Qstar_{h}(s,a)$. Otherwise, $\Qbar_{h}(s,a) = \cT_h^{\sss{\Mhat}} \Qbar_{h+1}(s,a) + b_{h,\delta}(s,a)$, and thus
	\begin{align*}
		\Qbar_{h}(s,a) - \Qstar_h(s,a) &= b_{h,\delta}(s,a) + \En_{s'\sim{} \widehat{P}_h(\cdot\mid{} s,a)} \Vbar_{h+1}(s')  - \En_{s'\sim{} \Pmstar_h(\cdot\mid{} s,a)} \Vstar_{h+1}(s') \\
		&\geq b_{h,\delta}(s,a) + \En_{s'\sim{} \widehat{P}_h(\cdot\mid{} s,a)} \Vstar_{h+1}(s')  - \En_{s'\sim{} \Pmstar_h(\cdot\mid{} s,a)} \Vstar_{h+1}(s') \geq 0.
	\end{align*}
	This, in turn, implies that $\Vbar_{h}(s) = \max_{a} \Qbar_h(s,a) \geq \max_{a} \Qstar_h(s,a) = \Vstar_h(s)$, concluding the induction step.
\end{proof}

We now analyze the effect of using an estimated model $\Mhat$ for the Bellman operator rather than the true unknown $\cT_h^{\sss{\Mstar}}$.
\begin{lem}[Bellman residuals of the optimistic value functions]
	\label{lem:ucb_vi_bellman_error}
	Suppose we have estimates $\crl[\big]{\widehat{P}_h(\cdot\mid{}s,a)}_{h,s,a}$ and $b_{h,\delta}'(s,a):\cS\times\cA\to\reals$ with the property that 
	\begin{align}
		\max_{V\in \{0,1\}^{S}}\abs*{\sum_{s'}\widehat{P}_h(s'\mid{}s,a) V(s')-\sum_{s'} \Pmstar_h(s'\mid{}s,a) V(s')} \leq b_{h,\delta}'(s,a).
	\end{align}
Then the Bellman residual satisfies
	\begin{align}
          \Qbar_{h}-\cT_h^{\sss{\Mstar}} \Qbar_{h+1} \leq (b_{h,\delta} + b_{h,\delta}')\wedge{}1
	\end{align}
	for $\Qbar_h, \Vbar_h$ defined in \eqref{eq:planning_in_estimated_model_optimism}.
\end{lem}
\begin{proof}[\pfref{lem:ucb_vi_bellman_error}]
  That $\Qbar_{h}-\cT_h^{\sss{\Mstar}} \Qbar_{h+1}\leq{}1$ is immediate. To
  prove the main result, observe that
\begin{align}
  \Qbar_{h}-\cT_h^{\sss{\Mstar}} \Qbar_{h+1}=\left\{ \cT_h^{\sss{\Mhat}} \Qbar_{h+1} + b_{h,\delta} \right\}  \wedge 1 - \cT_h^{\sss{\Mstar}} \Qbar_{h+1} \leq (\cT_h^{\sss{\Mhat}}-\cT_h^{\sss{\Mstar}}) \Qbar_{h+1} + b_{{h,\delta}}
\end{align}
For any $Q\in \cS\times \cA\to [0,1]$, 
\begin{align}
	(\cT_h^{\sss{\Mhat}}-\cT_h^{\sss{\Mstar}}) Q (s,a) &= \En_{s'\sim{} \widehat{P}_h(\cdot\mid{} s,a)} \max_a Q(s',a) - \En_{s'\sim{} \Pmstar_h(\cdot\mid{} s,a)} \max_a Q(s',a) \\
	&\leq \max_{V\in [0,1]^{S}} \abs{ \En_{s'\sim{} \widehat{P}_h(\cdot\mid{} s,a)} V(s') - \En_{s'\sim{} \Pmstar_h(\cdot\mid{} s,a)} V(s')}.
\end{align}
Since the maximum is achieved at a vertex of $[0,1]^S$, the statement follows.
\end{proof}

\subsubsection{Concentration and Regret Bound}
We now bring back the time index $t$ and show that the estimated
transition probabilities in \ucbvi satisfy conditions of
\pref{lem:ucb_vi_optimism} and \pref{lem:ucb_vi_bellman_error},
ensuring that the functions $\Qbar\ind{t}_1,\ldots,\Qbar\ind{t}_H$ are optimistic.
\begin{lem}[Concentration for empirical transitions]
  \label{lem:ucbvi_conc}
	Let
        $\crl[\big]{\widehat{P}\sups{t}_{h}}_{h\in[H],t\in[T]}$ be
        defined as in \eqref{eq:def_Pt_ucb-vi}. Then with probability
        at least $1-\delta$, the functions
	$$b\sups{t}_{h,\delta}(s,a) = 2\sqrt{\frac{\log(2SAHT/\delta)}{n_{h}\sups{t}(s,a)}},\mathand
b\sups{'t}_{h,\delta}(s,a) = 8\sqrt{\frac{S\log(2SAHT/\delta)}{n_{h}\sups{t}(s,a)}}$$
	satisfy the assumptions of \pref{lem:ucb_vi_optimism} and
        \pref{lem:ucb_vi_bellman_error}, respectively, for all
        $s\in\cS$, $a\in\cA$, $h\in[H]$, and $t\in[T]$ simultaneously.
\end{lem}
\begin{proof}[\pfref{lem:ucbvi_conc}]
	We leave the proof as an exercise to the reader, as it follows similar reasoning to other concentration results throughout the monograph.
\end{proof}

\begin{proof}[\pfref{thm:ucb-vi}]
  Putting everything together, we can now prove
  \pref{thm:ucb-vi}. Under the event in \pref{lem:ucbvi_conc}, the
  functions $\Qbar\ind{t}_1,\ldots,\Qbar\ind{t}_H$ are optimistic,
  which means that the conditions of \pref{lem:reg_decomp_optimistic}
  hold, and the instantaneous regret on round $t$ (conditionally on
  $s_1\sim d_1$) is at most
\begin{align*}
  &\sum_{h=1}^H \En^{\sss{\Mstar},
  \pihat\sups{t}}\brk*{(\Qbar\sups{t}_{h}- \cT_{h}^{\sss{\Mstar}}
  \Qbar\sups{t}_{h+1})(s\sups{t}_{h},
  \pihat\sups{t}_{h}(s\sups{t}_{h})) \mid{} s_1=s}\\
  &\qquad\leq
\sum_{h=1}^H \En^{\sss{\Mstar},
  \pihat\sups{t}}\brk*{(b_{h, \delta}(s\sups{t}_{h},
\pihat\sups{t}_{h}(s\sups{t}_{h}))+
b_{h,\delta}'(s\sups{t}_{h},
\pihat\sups{t}_{h}(s\sups{t}_{h})))\wedge{}1},
\end{align*}
where the second
inequality invokes \cref{lem:ucb_vi_bellman_error}; note that the
expectation is under $\pihat\ind{t}$ itself, the on-policy roll-in
distribution discussed after \cref{lem:reg_decomp_optimistic}, which is
what allows the bonuses to be summed along the algorithm's own
trajectories. Summing over
$t=1,\ldots,T$, and applying the Azuma-Hoeffding inequality, we have
that with probability at least $1-\delta$, the regret of \ucbvi is bounded by
\begin{align*}
  &\sum_{t=1}^{T}\sum_{h=1}^H \En^{\sss{\Mstar},
  \pihat\sups{t}}\brk*{(b_{h, \delta}(s\sups{t}_{h},
\pihat\sups{t}_{h}(s\sups{t}_{h}))+
b_{h,\delta}'(s\sups{t}_{h},
  \pihat\sups{t}_{h}(s\sups{t}_{h})))\wedge{}1}\\
  &\approxleq
  \sum_{t=1}^{T}\sum_{h=1}^H(b_{h, \delta}(s\sups{t}_{h},
\pihat\sups{t}_{h}(s\sups{t}_{h}))+
b_{h,\delta}'(s\sups{t}_{h},
\pihat\sups{t}_{h}(s\sups{t}_{h})))\wedge{}1 + \sqrt{HT\log(1/\delta)}.
\end{align*}
Using the bonus definition in
\cref{eq:ucbvi_bonus}, the bonus term above is bounded by
\begin{align}
  \sum_{t=1}^T\sum_{h=1}^H  \sqrt{\frac{S\log(2SAHT/\delta)}{n_{h}\sups{t}(s\sups{t}_{h},\pihat\sups{t}_{h}(s\sups{t}_{h}))}}\wedge{}1 \leq \sqrt{S\log(2SAHT/\delta)}\sum_{t=1}^T\sum_{h=1}^H  \frac{1}{\sqrt{n_{h}\sups{t}(s\sups{t}_{h},\pihat\sups{t}_{h}(s\sups{t}_{h}))}}\wedge{}1
\end{align}
The double summation can be handled in the same fashion as
\pref{lem:confidence_width_potential}:
\begin{align*}
  \sum_{t=1}^T\sum_{h=1}^H  \frac{1}{\sqrt{n_{h}\sups{t}(s\sups{t}_{h},\pihat\sups{t}_{h}(s\sups{t}_{h}))}}\wedge{}1 &= \sum_{h=1}^H  \sum_{(s,a)} \sum_{t=1}^T \frac{\indic{(s\sups{t}_{h}, \pihat\sups{t}_{h}(s\sups{t}_{h}))=(s,a)}}{\sqrt{n_{h}\sups{t}(s,a)}}\wedge{}1\\ &\approxleq \sum_{h=1}^H  \sum_{(s,a)} \sqrt{n_{h}\sups{T}(s,a)} \leq H\sqrt{SAT}.
\end{align*}
\end{proof}

\subsection{Estimation-to-Decisions for Tabular Reinforcement Learning}
\label{sec:dec_tabular}

We now give a second approach to tabular reinforcement
learning, based on the Estimation-to-Decisions principle of
\cref{sec:cb,sec:structured} rather than on optimism. The estimation
step is the same as in \ucbvi: from the trajectories observed so far, we
form an estimated MDP $\Mhat\ind{t}$. However, instead of planning optimistically in $\Mhat\ind{t}$, we adapt the
Estimation-to-Decisions principle to turn the estimate into an
exploratory distribution over policies whose regret is controlled by the
estimation error of $\Mhat\ind{t}$. Doing so requires an exploration
strategy that exploits the structure of MDPs, which serves as a preview for the most general form of the \CompText we introduce in \cref{sec:general_dm}.\loose

Throughout this subsection, we adopt \cref{assm:ucbvi}.
We write $\cM$ for the set of all tabular MDPs with state space
$\cS=\brk{S}$, action space $\cA=\brk{A}$, horizon $H$, and the known
rewards $r_h$; a model $M\in\cM$ is thus specified by its transition
kernels. For $M\in\cM$ and $\pi\in\PiRNS$, $M(\pi)$ denotes the law of
the trajectory $(s_1,a_1,r_1,\ldots,s_H,a_H,r_H)$ under $\pi$ in $M$.

\paragraph{Occupancy measures}
The results we present make use of the notion of \emph{occupancy
  measures} for an MDP $M$. Let $\bbP^{\sss{M},\pi}\prn{\cdot}$ denote the law of a trajectory evolving
under MDP $M$ and policy $\pi$. We define state occupancy
measures via
$$d^{\sss{M},\pi}_h(s)=\bbP^{\sss{M},\pi}(s_h=s)$$ and state-action occupancy
measures via 
$$d^{\sss{M},\pi}_h(s,a)=\bbP^{\sss{M},\pi}(s_h=s,a_h=a).$$ Note that we
have $d^{\sss{M},\pi}_1(s)=d_1(s)$ for all $M$ and $\pi$.

\paragraph{Instantiating Estimation-to-Decisions: Policy Cover Inverse Gap Weighting}
The main ingredient in our algorithm is the following subroutine which, given an estimated MDP $\Mhat\in\cM$ and a parameter $\eta>0$, computes a distribution $p$ over policies. The strategy, \emph{Policy Cover Inverse Gap Weighting}, combines inverse gap weighting
with the notion of a \emph{policy cover}---a collection of policies that
ensures good coverage of every state
\citep{du2019latent,misra2020kinematic,jin2020reward}.

\begin{algbox}[frametitle={\algtitle{alg:policy_cover_igw}{Policy Cover Inverse Gap Weighting (\pcigw)}}]
	\begin{algorithmic}
    \setstretch{1.3}
        \Require Estimated model $\Mhat\in\cM$; exploration parameter $\eta>0$.
        \Ensure Distribution $p\in\Delta(\PiRNS)$.
       \State Define \emph{inverse gap weighted policy cover} $\PiCov=\crl{\pi_{h,s,a}}_{h\in\brk{H},s\in\brk{S},a\in\brk{A}}$ via
       \begin{equation}
         \pi_{h,s,a} =
         \argmax_{\pi\in\PiRNS}\frac{\dm{\Mhat}{\pi}_h(s,a)}{2HSA +
           \eta(\fmhat(\pimhat)-\fmhat(\pi))}.
         \label{eq:pc_igw1}
       \end{equation}
       \State For each policy
       $\pi\in\PiCov\cup\crl{\pimhat}$, define
       \begin{equation}
         p(\pi) = \frac{1}{\lambda + \eta(\fmhat(\pimhat)-\fmhat(\pi))},\label{eq:pc_igw2}
       \end{equation}
       where $\lambda\in\brk{1,2HSA}$ is chosen such that $\sum_{\pi}p(\pi)=1$.
       \State \Return $p$.
\end{algorithmic}
\end{algbox}

The algorithm consists of two steps. First, in \pref{eq:pc_igw1}, we
compute the collection of policies
$\PiCov=\crl{\pi_{h,s,a}}_{h\in\brk{H},s\in\brk{S},a\in\brk{A}}$ that
constitutes the policy cover. The basic idea here is that each policies
in the policy cover should balance (i) regret and (ii) \emph{coverage}---that
is---ensure that all the states are sufficiently reached, which means
we are exploring. We accomplish this by using policies of the form\loose
\[
  \pi_{h,s,a} \ldef
  \argmax_{\pi\in\PiRNS}\frac{\dm{\Mhat}{\pi}_h(s,a)}{2HSA +
    \eta(\fmhat(\pimhat)-\fmhat(\pi))}
\]
which---for each $(s,a,h)$ tuple---maximize the ratio of the occupancy measure for $(s,a)$ at layer
$h$ to the
regret gap under $\Mhat$. This \emph{inverse gap weighted policy
  cover} balances exploration and exploitation by trading off coverage
with suboptimality.
With the policy cover in hand, the second step of
the algorithm computes the exploratory distribution $p$
by simply applying inverse gap weighting to the elements of the cover
and the greedy policy $\pimhat$. A naive alternative would be to apply
inverse gap weighting to all of $\PiRNS$ at once; this would pay
$\abs{\PiRNS}/\gamma$, which is exponential in
$S$ and $H$. The cover replaces this count by $HSA$, as the following proposition shows.
\begin{prop}[\pcigw links regret to estimation error]
  \label{prop:igw_tabular}
  Under \cref{assm:ucbvi}, for any $\gamma>0$ and
  $\Mhat\in\cM$, the \pcigw strategy with $\eta=\frac{\gamma}{21H^2}$ ensures that for all $M\in\cM$,
  \begin{align*}
    \En_{\pi\sim{}p}\brk*{\fm(\pim)-\fm(\pi)}
    \approxleq{} \frac{H^3SA}{\gamma}
    + \gamma\cdot\En_{\pi\sim{}p}\brk*{\DhelsX{\big}{M(\pi)}{\Mhat(\pi)}}.
  \end{align*}
\end{prop}
\cref{prop:igw_tabular} has the same form as
\cref{eq:e2d_principle}, with $H^3SA/\gamma$ in place of $A/\gamma$ and
Hellinger distance between trajectory distributions in place of squared
error of rewards. That such an inequality holds at all is the
non-trivial content of the result. The estimate $\Mhat$ is an entire
MDP, and the Hellinger error on the right-hand side is measured only
along trajectories that $p$ actually generates; an error in the
transitions at a state that $p$ rarely visits is invisible to the
right-hand side, yet could in principle be exactly what makes a policy
suboptimal. The policy cover is what rules this out: every state--action
pair that can be reached with non-negligible probability under
\emph{some} policy is visited by $p$ with probability inversely
proportional to that policy's estimated suboptimality, so any
misestimation that could cause regret is paid for by the oracle.

\paragraph{From the inequality to an algorithm}
We turn the inequality in \cref{prop:igw_tabular} into an algorithm using the same Estimation-to-Decisions template as in \cref{sec:igw}. In each episode $t$ an
online estimation oracle produces estimated transition kernels, and hence
an estimated MDP $\Mhat\ind{t}\in\cM$, from the trajectories observed so
far, and we play
$\pi\ind{t}\sim{}p\ind{t}\ldef\pcigw(\Mhat\ind{t})$. The oracle now
estimates distributions rather than means, so its error is measured in
Hellinger distance: in direct
analogy with \cref{def:online_regression_oracle}, we assume an oracle
whose cumulative Hellinger error
\[
  \EstHel\ldef\sum_{t=1}^{T}\En_{\pi\ind{t}\sim{}p\ind{t}}\brk*{\Dhels{\Mstar(\pi\ind{t})}{\Mhat\ind{t}(\pi\ind{t})}}
\]
satisfies, with probability at least $1-\delta$,
$\EstHel\leq\EstHel(\cM,T,\delta)$ for some deterministic bound
$\EstHel(\cM,T,\delta)$. Informally, this guarantee states that $\Mstar$ and $\Mhat\ind{t}$ must induce similar distributions over trajectories, on average over the policies sampled from $p\ind{t}$. This is the online estimation oracle of \cref{def:online_estimation_oracle} in \cref{sec:general_dm}, restricted to proper estimates $\Mhat\ind{t}\in\cM$.

To minimize the estimation error, the following
result shows that a natural choice---the empirical transition
frequencies of \ucbvi, smoothed by adding one pseudo-count to each
next state---has the required guarantee. This is Laplace's rule of
succession \citep{laplace1814essai}, applied separately for each
$(h,s,a)$; equivalently, this is the continuous-prior analogue of the Bayes mixture predictor of
\cref{lem:likelihood_mixture}, with the uniform distribution over the
simplex $\Delta(\cS)$ as the prior. The log-loss regret bound is
classical \citep[Chapter 9]{cesa2006prediction}.
\begin{prop}[Hellinger estimation for tabular MDPs]
  \label{prop:tabular_estimation}
  Let $\Mhat\ind{t}\in\cM$ be the MDP with transition kernels
  \begin{equation}
    \label{eq:laplace_estimator}
    \widehat{P}\ind{t}_h(s'\mid{}s,a) = \frac{n\ind{t}_h(s,a,s')+1}{n\ind{t}_h(s,a)+S},
  \end{equation}
  where $n\ind{t}_h(s,a)$ and $n\ind{t}_h(s,a,s')$ are the empirical
  counts defined in \cref{sec:ucbvi}. For any $\delta\in(0,1)$, with
  probability at least $1-\delta$,
  \[
    \EstHel \leq{} HS^2A\log(e(T+S)) + 2\log(\delta^{-1}).
  \]
  This holds for any sequence of distributions $p\ind{1},\ldots,p\ind{T}$
  used to select $\pi\ind{1},\ldots,\pi\ind{T}$.
\end{prop}
The proof is deferred to \cref{sec:mdp_deferred}. Summing
\cref{prop:igw_tabular} over episodes and plugging in
\cref{prop:tabular_estimation} gives the following end-to-end regret
bound.

\begin{theorem}[Regret of Estimation-to-Decisions with \pcigw]
  \label{prop:e2d_tabular}
  Under \cref{assm:ucbvi}, the algorithm that plays
  $\pi\ind{t}\sim\pcigw(\Mhat\ind{t})$ with exploration parameter
  $\gamma>0$ has the following guarantees.
  \begin{enumerate}
  \item For any estimation oracle and every realization,
    \[
      \Reg \approxleq \frac{H^3SA}{\gamma}\cdot{}T + \gamma\cdot\EstHel.
    \]
    Consequently, if the oracle has estimation error
    $\EstHel(\cM,T,\delta)$, then with probability at least $1-\delta$,
    \[
      \Reg \approxleq \sqrt{H^3SA\cdot{}T\cdot\EstHel(\cM,T,\delta)}
      \qquad\text{for}\qquad
      \gamma=\sqrt{H^3SA\cdot{}T/\EstHel(\cM,T,\delta)}.
    \]
  \item With the estimator \cref{eq:laplace_estimator} and
    $\gamma=\sqrt{H^3SA\cdot{}T/(HS^2A\log(e(T+S))+2\log(\delta^{-1}))}$,
    with probability at least $1-\delta$,
    \[
      \Reg \approxleq \sqrt{H^{4}S^{3}A^{2}T\log(e(T+S)) + H^{3}SAT\log(\delta^{-1})}.
    \]
  \end{enumerate}
\end{theorem}
\begin{proof}[\pfref{prop:e2d_tabular}]
  The first bound is proven identically to
  \cref{prop:squarecb}, with \cref{prop:igw_tabular} in place of
  \cref{prop:igw_mab}: adding and subtracting
  $\gamma\cdot\En_{\pi\ind{t}\sim{}p\ind{t}}\brk*{\Dhels{\Mstar(\pi\ind{t})}{\Mhat\ind{t}(\pi\ind{t})}}$
  in each episode and applying \cref{prop:igw_tabular} with $M=\Mstar$ and
  $\Mhat=\Mhat\ind{t}$ bounds each episode's regret by $H^3SA/\gamma$ plus
  $\gamma$ times that episode's estimation error; summing over episodes gives
  the bound. The second bound in the first part follows on the event
  that $\EstHel\leq\EstHel(\cM,T,\delta)$, by the choice of $\gamma$. 
\end{proof}

The rate in \cref{prop:e2d_tabular} is polynomial in all
parameters, though worse than the $HS\sqrt{AT}$ rate of \ucbvi
(\cref{thm:ucb-vi}); the point of the argument is its generality rather
than its constants. We return to this construction in
\cref{sec:general_dm} (\cref{ex:dec_tabular_rl}), as an instance of the more general DEC theory.\loose

\begin{rem}[Computational efficiency]
  The \pcigw strategy can be implemented in a computationally efficient fashion. Briefly, the idea is
to solve \pref{eq:pc_igw1} by taking a dual approach and optimizing
over occupancy measures rather
than policies. With this parameterization, \pref{eq:pc_igw1} becomes a
linear-fractional program, which can then be transformed into a standard
linear program using classical techniques.
\end{rem}

\subsubsection{Preliminary Lemma: Change of Measure}
The proof of \pref{prop:igw_tabular} uses the simulation
lemma (\cref{lem:simulation_basic}) from \cref{sec:mdp_analysis_tools},
together with the following ``change-of-measure'' lemma, which allows one
to move between quantities involving an estimator $\Mhat$ and those
involving another model $M$.

\begin{lem}[Change of measure for RL]
  \label{lem:change_of_measure}
  Consider any MDPs $M,\Mhat\in\cM$ (so that both satisfy \cref{assm:ucbvi}). For all
  $p\in\Delta(\Pi)$ and $\eta>0$ we have
  \begin{align}
  &\En_{\pi\sim{}p}\brk*{\fm(\pim) - \fm(\pi)}\notag\\
  &\leq{}
    \En_{\pi\sim{}p}\brk*{\fm(\pim) - \fmhat(\pi)}
    + \eta\En_{\pi\sim{}p}\brk*{
    \Dhels{M(\pi)}{\Mhat(\pi)}
    } + \frac{1}{4\eta}.     \label{eq:com0}
  \end{align}
and
    \begin{align}
    \En_{\pi\sim{}p}\Enm{\Mhat}{\pi}\brk*{\sum_{h=1}^{H}    \Dtvs{\Pm_h(s_h,a_h)}{\Pmhat_h(s_h,a_h)}}
    \leq 8H\En_{\pi\sim{}p}\brk*{\Dhels{M(\pi)}{\Mhat(\pi)}}.       \label{eq:com1}
  \end{align}
\end{lem}

  \begin{proof}[\pfref{lem:change_of_measure}]
  We first prove \pref{eq:com0}. Writing $\fm(\pim)-\fm(\pi)
  = \fm(\pim)-\fmhat(\pi) + \fmhat(\pi)-\fm(\pi)$ and applying
  \cref{eq:simulation} (with $2\eta$ in place of $\eta$) to the second
  difference, we have that for all $\eta>0$,
\begin{align*}
  &\En_{\pi\sim{}p}\brk*{\fm(\pim) - \fm(\pi)}\\
  &\leq{}
    \En_{\pi\sim{}p}\brk*{\fm(\pim) - \fmhat(\pi)}
    + \eta\En_{\pi\sim{}p}\brk*{
    \Dhels{M(\pi)}{\Mhat(\pi)}
    } + \frac{1}{4\eta}.
  \end{align*}
  We now prove \pref{eq:com1}.
Using \pref{lem:hellinger_pair}, we have that for all $h$,
\[
  \Enm{\Mhat}{\pi}\brk*{\Dhels{\Pm_h(s_h,a_h)}{\Pmhat_h(s_h,a_h)}}
  \leq{} 8\Dhels{M(\pi)}{\Mhat(\pi)}.
  \]
As a result, using that $\Dtvs{\cdot}{\cdot}\leq\Dhels{\cdot}{\cdot}$,
  \begin{align*}
    \Enm{\Mhat}{\pi}\brk*{\sum_{h=1}^{H}          \Dtvs{\Pm_h(s_h,a_h)}{\Pmhat_h(s_h,a_h)}
    }
    &\leq 8H\Dhels{M(\pi)}{\Mhat(\pi)}.
  \end{align*}
  Since this holds uniformly for all $\pi$, taking expectation over
  $\pi\sim{}p$ gives \pref{eq:com1}.
  
\end{proof}

\subsubsection{Proof of \cref*{prop:igw_tabular}}

\begin{proof}[\pfref{prop:igw_tabular}]%
  \newcommand{\allpolicies}{\PiCov\cup\crl{\pimhat}}%
  \newcommand{\etaalt}{\eta'}%
Let $M\in\cM$ be fixed. The main effort in the proof is to bound
the quantity
  \begin{align*}
    \En_{\pi\sim{}p}\brk*{\fm(\pim) - \fmhat(\pi)}
  \end{align*}
  in terms of the quantity on the right-hand side of \pref{eq:com1},
  then apply change of measure (\pref{lem:change_of_measure}). We
  begin with the decomposition
  \begin{align}
    \En_{\pi\sim{}p}\brk*{\fm(\pim) - \fmhat(\pi)}
    =     \underbrace{\En_{\pi\sim{}p}\brk*{\fmhat(\pimhat) -
    \fmhat(\pi)}}_{(\mathrm{I})}
    + \underbrace{\fm(\pim) - \fmhat(\pimhat)}_{(\mathrm{II})}.
    \label{eq:igw_tabular0}
  \end{align}
  For the first term $(\mathrm{I})$, which may be thought of as exploration bias, we have
  \begin{align}
    \En_{\pi\sim{}p}\brk*{\fmhat(\pimhat) - \fmhat(\pi)}
    =\sum_{\pi\in\PiCov\cup\crl{\pimhat}}\frac{\fmhat(\pimhat) - \fmhat(\pi)}{\lambda
    + \eta(\fmhat(\pimhat) - \fmhat(\pi))}
    \leq{} \frac{2HSA}{\eta},
    \label{eq:igw_tabular0.5}
  \end{align}
  where we have used that $\lambda\geq{}0$. We next bound the second
  term $(\mathrm{II})$, which entails showing that the \pcigw
  distribution \emph{explores enough}. We have
  \begin{equation}
    \label{eq:igw_tabular1}
    \fm(\pim) - \fmhat(\pimhat)
    = \fm(\pim) - \fmhat(\pim) - (\fmhat(\pimhat) - \fmhat(\pim)).
  \end{equation}
We use the simulation lemma
  to bound
    \begin{align*}
      \fm(\pim)- \fmhat(\pim)
  &\leq{}
    \sum_{h=1}^{H}\Enm{\Mhat}{\pim}\brk*{\Dtv{\Pm_h(s_h,a_h)}{\Pmhat_h(s_h,a_h)}}
      =\sum_{h=1}^{H}\sum_{s,a}\dm{\Mhat}{\pim}_h(s,a)\errm_h(s,a),
    \end{align*}
where $\errm_h(s,a) \ldef \Dtv{\Pm_h(s,a)}{\Pmhat_h(s,a)}$ (the reward
terms in \cref{eq:simulation_basic2} vanish under \cref{assm:ucbvi}). Define $\dbar_h(s,a) =
    \En_{\pi\sim{}p}\brk[\big]{\dm{\Mhat}{\pi}_h(s,a)}$. Then, using the
    AM-GM inequality, we have that for any $\etaalt>0$,
    \begin{align*}
      \sum_{h=1}^{H}\sum_{s,a}\dm{\Mhat}{\pim}_h(s,a)\brk*{\errm_h(s,a)}
      &=
        \sum_{h=1}^{H}\sum_{s,a}\dm{\Mhat}{\pim}_h(s,a)\prn*{\frac{\dbar_h(s,a)}{\dbar_h(s,a)}}^{1/2}\brk*{\errm_h(s,a)}\\
      &\leq{} \frac{1}{2\etaalt}\sum_{h=1}^{H}\sum_{s,a}\frac{(\dm{\Mhat}{\pim}_h(s,a))^{2}}{\dbar_h(s,a)}
+
        \frac{\eta'}{2}\sum_{h=1}^{H}\sum_{s,a}\dbar_h(s,a) (\errm_h(s,a))^2\\
      &= \frac{1}{2\etaalt}\sum_{h=1}^{H}\sum_{s,a}\frac{(\dm{\Mhat}{\pim}_h(s,a))^2}{\dbar_h(s,a)}
              +        \frac{\eta'}{2}\sum_{h=1}^{H}\En_{\pi\sim{}p}\Enm{\Mhat}{\pi}\brk*{(\errm_h(s_h,a_h))^2}.
    \end{align*}
    The second term is exactly the upper bound we want, so it remains to bound
    the ratio of occupancy measures in the first term. Observe that
    for each $(h,s,a)$, we have
    \begin{align*}
      \frac{\dm{\Mhat}{\pim}_h(s,a)}{\dbar_h(s,a)}
      \leq{}
      \frac{\dm{\Mhat}{\pim}_h(s,a)}{\dm{\Mhat}{\pihsa}_h(s,a)}\cdot\frac{1}{p(\pihsa)}
      \leq{}
      \frac{\dm{\Mhat}{\pim}_h(s,a)}{\dm{\Mhat}{\pihsa}_h(s,a)}\prn*{2HSA +
      \eta(\fmhat(\pimhat) -\fmhat(\pihsa)},
    \end{align*}
    where the second inequality follows from the definition of $p$ and
    the fact that $\lambda\leq{}2HSA$. Furthermore, since
    \[
      \pihsa = \argmax_{\pi\in\PiRNS}\frac{\dm{\Mhat}{\pi}_h(s,a)}{2HSA +
           \eta(\fmhat(\pimhat)-\fmhat(\pi))},
       \]
       and since $\pim\in\PiRNS$, we can upper bound by
       \begin{equation}
       \label{eq:pcigw_multiplicative}
               \frac{\dm{\Mhat}{\pim}_h(s,a)}{\dm{\Mhat}{\pim}_h(s,a)}\prn*{2HSA +
        \eta(\fmhat(\pimhat) -\fmhat(\pim)}
      = 2HSA + \eta(\fmhat(\pimhat) -\fmhat(\pim).
    \end{equation}
    As a result, we have
    \begin{align*}
      \sum_{h=1}^{H}\sum_{s,a}\frac{(\dm{\Mhat}{\pim}_h(s,a))^2}{\dbar_h(s,a)}
      &\leq{}
      \sum_{h=1}^{H}\sum_{s,a}\dm{\Mhat}{\pim}_h(s,a)(2HSA +
      \eta(\fmhat(\pimhat) -\fmhat(\pim)) \\
      &= 2H^{2}SA +\eta{}H(\fmhat(\pimhat) -\fmhat(\pim)).
    \end{align*}
    Putting everything together and returning to \pref{eq:igw_tabular1}, this establishes that
    \begin{align*}
      &\fm(\pim) - \fmhat(\pimhat) \\
      &\leq{} \frac{H^{2}SA}{\eta'}
      +\frac{\eta'}{2}\sum_{h=1}^{H}\En_{\pi\sim{}p}\Enm{\Mhat}{\pi}\brk*{(\errm_h(s_h,a_h))^2}
      +\frac{\eta{}H}{2\eta'}(\fmhat(\pimhat) -\fmhat(\pim))
                 -(\fmhat(\pimhat) - \fmhat(\pim)).
    \end{align*}
    We set $\eta' = \frac{\eta{}H}{2}$ so that the latter terms cancel
    and we are left with
    \[
      \fm(\pim) - \fmhat(\pimhat)
      \leq{} \frac{2HSA}{\eta} + \frac{\eta{}H}{4}\sum_{h=1}^{H}\En_{\pi\sim{}p}\Enm{\Mhat}{\pi}\brk*{(\errm_h(s_h,a_h))^2}.
    \]
    Combining this with \pref{eq:igw_tabular0} and
    \pref{eq:igw_tabular0.5} gives
    \begin{align*}
      &\En_{\pi\sim{}p}\brk*{\fm(\pim) - \fmhat(\pi)}\\
      &\leq{} \frac{4HSA}{\eta} +
        \frac{\eta{}H}{4}\sum_{h=1}^{H}\En_{\pi\sim{}p}\Enm{\Mhat}{\pi}\brk*{(\errm_h(s_h,a_h))^2}\\
      &= \frac{4HSA}{\eta} + \frac{\eta{}H}{4}\sum_{h=1}^{H}\En_{\pi\sim{}p}\Enm{\Mhat}{\pi}\brk*{\Dtvs{\Pm_h(s_h,a_h)}{\Pmhat_h(s_h,a_h)}}.
    \end{align*}
We conclude by applying the change-of-measure lemma (\pref{lem:change_of_measure}), which
implies that for any $\eta'>0$,
\begin{align*}
\En_{\pi\sim{}p}\brk*{\fm(\pim) - \fm(\pi)}
    \leq{} \frac{4HSA}{\eta} + (4\eta')^{-1} + (2H^2\eta+\eta')\cdot\En_{\pi\sim{}p}\brk*{
    \Dhels{M(\pi)}{\Mhat(\pi)}
    }.
\end{align*}
The result follows by choosing $\eta=\eta'=\frac{\gamma}{21H^2}$ (we
have made no effort to optimize the constants here).
\end{proof}

\subsection{Bibliographic Remarks and Further Reading}\label{sec:mdp_bib}
\paragraph{Takeaways}
This section introduced the finite-horizon episodic MDP, showed that
planning in a known MDP reduces to dynamic programming, and gave two
approaches to online reinforcement learning in tabular MDPs. Both build
the same estimated MDP $\Mhat\ind{t}$ from empirical transition
frequencies and differ only in what they do with it: \ucbvi plans
optimistically in $\Mhat\ind{t}$ and is analyzed through the Bellman
residuals of its optimistic value functions, while \pcigw uses
$\Mhat\ind{t}$ to build an exploratory distribution over a policy cover,
analyzed in Hellinger distance; \cref{sec:general_dm} shows that the
latter is the \CompShort construction for tabular MDPs. The
analysis tools in \cref{sec:mdp_analysis_tools} recur throughout the
monograph.

\paragraph{Bibliographic remarks}
Bellman's principle of optimality and value iteration
(\cref{prop:bellman}) go back to \citet{bellman1954theory}; see
\citet{puterman1994markov} and \citet{sutton2018reinforcement} for
textbook treatments. The performance difference lemma (\cref{lem:perf_diff_lemma})
is due to \citet{kakade2003sample}, and the simulation lemma
(\cref{lem:simulation_basic}) takes its name from
\citet{kearns2002near}; the Bellman residual decomposition
(\cref{lem:bellman_residual}) is a telescoping identity that appears in
many forms, with \citet{jiang2018open} being perhaps the most notable example.

Optimism for tabular RL originates with E$^3$ and R-max
\citep{kearns2002near,brafman2002r} and, for regret minimization, with
UCRL2 \citep{jaksch2010near}. \ucbvi (\cref{alg:ucbvi}) and
\cref{thm:ucb-vi} are from \citet{azar2017minimax};
\cref{lem:reg_decomp_optimistic}, \cref{lem:ucb_vi_optimism}, and
\cref{lem:ucb_vi_bellman_error} follow their analysis. The
$\Omega(\sqrt{HSAT})$ lower bound quoted after \cref{thm:ucb-vi} is from
\citet{domingues2021episodic}, refining
\citet{jaksch2010near,osband2016lower}.\loose

The \pcigw{} strategy
(\cref{alg:policy_cover_igw}), \cref{prop:igw_tabular}, and the
change-of-measure lemma (\cref{lem:change_of_measure}) are adapted from
\citet{foster2021statistical}, where the same argument bounds the
Decision-Estimation Coefficient for tabular MDPs and
\cref{prop:e2d_tabular} is a special case of the general \etd guarantee. The notion of policy cover used in the algorithm is developed in \citet{du2019latent,misra2020kinematic,jin2020reward}. The estimator in
\cref{prop:tabular_estimation} is classical \citep[Chapter 9]{cesa2006prediction}. 

\paragraph{Further reading}
The monograph of \citet{agarwal2019reinforcement} covers tabular RL in
depth. The bound in \cref{thm:ucb-vi} is loose in $S$ and $H$;
Bernstein-type bonuses give the minimax rate
\citep{azar2017minimax,zanette2019tighter}, and
\citet{zhang2021reinforcement} remove the remaining horizon dependence
under the normalization of \cref{assm:ucbvi}; see \citet{jiang2018open}
for why this normalization matters. Model-free counterparts of \ucbvi,
which maintain optimistic $Q$-functions without an estimated model,
begin with \citet{jin2018q}. Posterior sampling for tabular RL is
analyzed by \citet{osband2013more,osband2017posterior}.

\subsection{Deferred Proofs}
\label{sec:mdp_deferred}

This proof uses somewhat standard machinery from the theory of online learning; we refer to \cref{lem:log_loss_hellinger} for further background.
\begin{proof}[\pfref{prop:tabular_estimation}]
  Under \cref{assm:ucbvi}, the density of a trajectory
  $\tau=(s_1,a_1,\ldots,s_H,a_H)$ under policy $\pi$ in an MDP $M\in\cM$ is
  $d_1(s_1)\prod_{h=1}^{H}\pi_h(a_h\mid{}s_h)\Pm_h(s_{h+1}\mid{}s_h,a_h)$
  (rewards are a deterministic function of the trajectory, and the
  final transition can be ignored). Hence, for the trajectory $\tau\ind{t}$
  observed at round $t$, the log-loss difference between the true model
  and the estimate satisfies
  \begin{equation}
    \label{eq:tabular_est_decomp}
    \log\frac{\Mstar(\tau\ind{t}\mid\pi\ind{t})}{\Mhat\ind{t}(\tau\ind{t}\mid\pi\ind{t})}
    = \sum_{h=1}^{H}\log\frac{\Pmstar_h(s\ind{t}_{h+1}\mid{}s\ind{t}_h,a\ind{t}_h)}{\widehat{P}\ind{t}_h(s\ind{t}_{h+1}\mid{}s\ind{t}_h,a\ind{t}_h)},
  \end{equation}
  since the initial state and policy terms cancel. For each cell
  $(h,s,a)$, let $\cT_{h,s,a}\ldef\crl{t\in[T]: (s\ind{t}_h,a\ind{t}_h)=(s,a)}$
  be the set of rounds in which the cell is visited, and define
  \begin{equation}
    \label{eq:tabular_est_cell_regret}
    R_{h,s,a} \ldef \sum_{t\in\cT_{h,s,a}}
    \log\frac{\Pmstar_h(s\ind{t}_{h+1}\mid{}s,a)}{\widehat{P}\ind{t}_h(s\ind{t}_{h+1}\mid{}s,a)},
  \end{equation}
  the log-loss regret of the estimator \cref{eq:laplace_estimator}
  relative to the fixed distribution $\Pmstar_h(\cdot\mid{}s,a)$ on the
  subsequence of next states observed at this cell. Since every pair
  $(t,h)$ belongs to exactly one cell, summing \cref{eq:tabular_est_decomp}
  over $t$ gives
  \begin{equation}
    \label{eq:tabular_est_sum}
    \sum_{t=1}^{T}\log\frac{\Mstar(\tau\ind{t}\mid\pi\ind{t})}{\Mhat\ind{t}(\tau\ind{t}\mid\pi\ind{t})}
    = \sum_{h=1}^{H}\sum_{(s,a)\in\cS\times\cA}R_{h,s,a}.
  \end{equation}

  We bound $R_{h,s,a}$ for a fixed cell. Write $n=\abs{\cT_{h,s,a}}$ for the
  number of visits, $y_1,\ldots,y_n\in\cS$ for the next states observed
  on these visits in order, and $n_j$ for the number of occurrences of
  state $j\in\cS$ among $y_1,\ldots,y_n$, so that $\sum_{j}n_j=n$. Let
  $n_j(k)$ be the number of occurrences of $j$ among $y_1,\ldots,y_k$. At
  the $k$-th visit, the counts in \cref{eq:laplace_estimator} are those of
  the first $k-1$ next states, so the estimator assigns probability
  \[
    \widehat{P}_k(y_k) = \frac{n_{y_k}(k-1)+1}{k-1+S}
  \]
  to the state $y_k$ that is actually observed. In this notation,
  \[
    R_{h,s,a} = \sum_{k=1}^{n}\log\frac{\Pmstar_h(y_k\mid{}s,a)}{\widehat{P}_k(y_k)}
    = \log\frac{\prod_{k=1}^{n}\Pmstar_h(y_k\mid{}s,a)}{\prod_{k=1}^{n}\widehat{P}_k(y_k)}.
  \]
  We compute the two products separately.

  \emph{The estimator's product likelihood.} The denominators of $\widehat{P}_1(y_1),\ldots,\widehat{P}_n(y_n)$ are
  $S,S+1,\ldots,n+S-1$, whose product is $\frac{(n+S-1)!}{(S-1)!}$. For
  the numerators, fix a state $j$ and consider the visits $k$ at which
  $y_k=j$: at the first such visit $n_j(k-1)+1=1$, at the second it is
  $2$, and so on up to $n_j$, so these visits contribute the factor
  $n_j!$. Multiplying over $j$,
  \begin{equation}
    \label{eq:laplace_product}
    \prod_{k=1}^{n}\widehat{P}_k(y_k)
    = \frac{(S-1)!\,\prod_{j\in\cS}n_j!}{(n+S-1)!}.
  \end{equation}
  (This quantity is also $\int_{\Delta(\cS)}\prod_{k=1}^{n}q(y_k)\,dq$ for the
  uniform distribution over the simplex, which is why
  \cref{eq:laplace_estimator} is the Bayes mixture predictor with a
  uniform prior \citep[Chapter 9]{cesa2006prediction}; we do not need
  this fact.)

  \emph{The true distribution's product likelihood.} Writing $q=\Pmstar_h(\cdot\mid{}s,a)$
  and $\bar{q}(j)=n_j/n$ for the empirical distribution of
  $y_1,\ldots,y_n$, we have $\prod_{k=1}^{n}q(y_k)=\prod_{j}q(j)^{n_j}$, and
  \[
    \log\prod_{j}q(j)^{n_j}
    = n\sum_{j}\bar{q}(j)\log q(j)
    = n\sum_{j}\bar{q}(j)\log\bar{q}(j) - n\Dkl{\bar{q}}{q}
    \leq \log\prod_{j}\bar{q}(j)^{n_j},
  \]
  since KL divergence is nonnegative (\cref{lem:divergence_inequality}). In
  words: the empirical distribution maximizes the likelihood of the
  observed sequence, so the true distribution can do no better than it.

  \emph{Putting everything together.} Dividing the second product by \cref{eq:laplace_product},
  \[
    R_{h,s,a}
    \leq \log\prn*{\frac{(n+S-1)!}{(S-1)!\,\prod_{j}n_j!}\prod_{j}\prn[\Big]{\frac{n_j}{n}}^{n_j}}
    = \log\binom{n+S-1}{S-1}
    + \log\prn*{\frac{n!}{\prod_{j}n_j!}\prod_{j}\prn[\Big]{\frac{n_j}{n}}^{n_j}}.
  \]
  The second logarithm is at most zero: its argument is the multinomial
  probability of observing exactly the counts $(n_j)_{j\in\cS}$ in $n$
  independent draws from $\bar{q}$, which is at most $1$. For the first,
  the
  standard estimate $\binom{m}{r}\leq(em/r)^{r}$ with $m=n+S-1$ and
  $r=S-1$ gives (cf.\ \citet[Chapter 9]{cesa2006prediction}, which
  proves the same bound for the Laplace forecaster)
  \[
    R_{h,s,a}
    \leq (S-1)\log\prn*{\frac{e(n+S-1)}{S-1}}
    \leq S\log(e(T+S)),
  \]
  where the last step uses $n\leq{}T$ and $\frac{n+S-1}{S-1}\leq n+S-1\leq T+S$.

  Summing over the $HSA$ cells in \cref{eq:tabular_est_sum}, the
  log-loss regret of the sequence $\Mhat\ind{1},\ldots,\Mhat\ind{T}$
  relative to $\Mstar$ is at most $HS^2A\log(e(T+S))$. Finally, \cref{lem:log_loss_hellinger} (stated
  in \cref{sec:online_estimation_hellinger}) converts log-loss regret relative to $\Mstar$ into
  Hellinger estimation error: with probability at least $1-\delta$,
  $\EstHel$ is at most this regret plus $2\log(\delta^{-1})$. Since
  $\Mhat\ind{t}$ is determined by the history and $\tau\ind{t}$ is drawn
  from $\Mstar(\pi\ind{t})$, the lemma applies for any choice of the
  distributions $p\ind{t}$.\loose
\end{proof}

\section{General Decision Making: A Unified Framework}
\label{sec:general_dm}
\newcommand{\cMgb}{\cM_{\textsf{MAB-G}}}
\newcommand{\cMsn}{\cM_{\textsf{MAB-SN}}}
\renewcommand{\pm}[1][M]{p_{\sss{#1}}}

So far, we have covered three frameworks for interactive
decision making beyond the multi-armed bandit: contextual bandits,
structured bandits, and episodic reinforcement learning. For each setting, we developed algorithms as instances of the Estimation-to-Decisions
principle. For structured bandits, the \CompText (\CompShort) identified
the best possible exploration strategy for a given class, yielding one
algorithm, \etd, for all classes. For tabular reinforcement learning, the
same principle was applied (with the main difference being to move from regression to estimation of entire models), but the exploration strategy
and analysis were built by hand. In this section, we unify these developments by introducing a general decision making framework that captures the settings of the preceding chapters and deriving general versions of Estimation-to-Decisions and the Decision-Estimation Coefficient in this setting.

\paragraph{Organization}
\cref{sec:gdm_framework} introduces our general decision making framework and shows how
the earlier settings are special cases. \cref{sec:divergences} reviews
the information-theoretic divergences needed to formally define the algorithms and complexity measures we study.
\cref{sec:dec_general} defines the most general form of the \CompShort and computes it for examples, and \cref{sec:algo_gen_dm} builds on this to give the general form of the \etd algorithm.
Finally, \cref{sec:constrained_comp} introduces a second variant of the DEC, the \emph{constrained \CompShort}, which leads to sharper regret bounds and sets up for \cref{sec:lower}.

\subsection{General Decision Making Framework}
\label{sec:gdm_framework}

We work in a simple framework we refer to as \emph{Decision Making with Structured Observations} (\FrameworkShort), which we will show subsumes the different decision making frameworks we have encountered so far.
far. The protocol has two ingredients beyond those of a structured
bandit: a \emph{decision space} $\Act$, which may consist of actions,
policies, or any other objects, and an \emph{observation space}
$\ObsSpace$, which carries whatever feedback the learner receives
beyond the reward.
\begin{algbox}[frametitle={\algprotocol{General Decision Making Protocol}}]
  \begin{algorithmic}
\For{$t=1,\ldots,T$}
\State Learner selects a decision $\pi\ind{t}\in\Act$ (possibly at random, $\pi\ind{t}\sim{}p\ind{t}$).
\State Nature selects a reward $r\ind{t}\in\RewardSpace\subseteq\bbR$ and observation $\obs\ind{t}\in\ObsSpace$ based on the decision.
\State Learner observes $(r\ind{t},\obs\ind{t})$.
\EndFor{}
\end{algorithmic}
\end{algbox}
We focus on a stochastic variant of the \FrameworkShort framework.
\begin{assumption}[Stochastic Rewards and Observations]
  \label{asm:stochastic_dmso}
  Rewards and observations are generated independently via
    \begin{align}
        (r\sups{t},o\ind{t}) \sim \Mstar(\cdot\mid\pi\sups{t}),
    \end{align}
    where $\Mstar:\Pi\to\Delta(\cR\times\cO)$ is the underlying \emph{model}.
  \end{assumption}
To facilitate the use of learning and function approximation,
we assume the \learner has access to a \emph{model class} $\cM$ that
contains the model
$\Mstar$.
As usual, depending on the
problem domain, $\cM$ might consist of linear models, neural networks,
random forests, or
other complex function approximators; this generalizes the role of the
reward function class $\cF$ used in contextual/structured bandits. We make the following standard realizability
assumption, which asserts that $\cM$ is flexible enough to express the true model.%
\begin{assumption}[Realizability]
  \label{ass:realizability}
  The model class $\cM$ contains the true model $\Mstar$.
\end{assumption}
Throughout, we assume that mean rewards are bounded,
$\Enm{M}{\pi}\brk{r}\in[0,1]$ for all $M\in\cM$ and $\pi\in\Pi$; the
rewards themselves may be unbounded (Gaussian, for example). Regret
bounds hold with probability at least $1-\delta$ over the draw of the
rewards and observations and the learner's randomization.
For a model $M\in\cM$, let $\Enm{M}{\pi}\brk*{\cdot}$ denote the expectation under
$(r,\obs)\sim{}M(\pi)$.
Further, following the notation in \pref{sec:mdp}, let $$\fm(\pi)\ldef{}\Enm{M}{\pi}\brk*{r}$$ denote the
mean reward function, and let 
$$\pim\ldef{}\argmax_{\pi\in\Act}\fm(\pi)$$ 
denote the optimal decision with
maximal expected reward. Finally, define
\begin{align}
	\label{eq:def_F_from_M}
	\cFm\ldef{}\crl*{\fm\mid{}M\in\cM}
\end{align} 
as the induced class of mean reward functions. We evaluate the \learner's performance in terms of \emph{regret} to the optimal
decision for $\Mstar$:
\begin{equation}
  \label{eq:regret}
  \RegDM
  \ldef \sum_{t=1}^{T}\En_{\pi\ind{t}\sim{}p\ind{t}}\brk*{\fmstar(\pimstar) - \fmstar(\pi\ind{t})},
\end{equation}
where $p\ind{t}\in\Delta(\Act)$ is the
learner's distribution over decisions at round $t$. As in \cref{sec:mdp}, we abbreviate $\fstar=\fmstar$ and
$\pistar=\pimstar$.
In the language of \cref{fig:axes}, this framework sits at the far
corner: it is interactive, it allows arbitrary function approximation
through $\cM$, and the observation $o\ind{t}$ can be as rich as one
likes. Every setting studied so far is recovered by restricting the
decision and observation spaces in different ways. Let us first see how it subsumes the
structured bandit and contextual bandit problems.

\begin{example}[Structured bandits]
  \label{ex:bandit}
When there are no observations
(i.e., $\ObsSpace=\crl*{\emptyset}$), the \FrameworkShort framework is equivalent to 
\emph{structured bandits} studied earlier in \pref{sec:structured}.
Therein, we defined a structured bandit instance by specifying a class
$\cF$ of mean reward functions and a general class of reward
distributions, such as sub-Gaussian or bounded. In the \FrameworkShort
framework, we may equivalently start with a set of models $\cM$ and
let $\cFm$ be the induced class \eqref{eq:def_F_from_M}. By changing
the class $\cF$, this encompasses all of the concrete examples of
structured bandit problems we studied in \cref{sec:structured},
including linear bandits, nonparametric bandits, and concave/convex bandits.

\end{example}

\begin{example}[Contextual bandits]
	The \FrameworkShort framework captures contextual
        bandits (\pref{sec:cb}) with i.i.d.\ contexts, a special
        case of the setting of \pref{sec:cb}, which also allows
        adversarial contexts. To make this precise, we
        slightly abuse notation and think of $\pi\sups{t}$ as
        \emph{functions} mapping the context $x\sups{t}$ to an action in $[A]$. To this end, on round $t$, the decision-maker selects a mapping $\pi\sups{t}:\cX\to[A]$ from contexts to actions, and the context $o\sups{t}=x\sups{t}$ is observed at the end of the round. This is equivalent to first observing $x\sups{t}$ and selecting $\pi\sups{t}(x\sups{t})\in[A]$.
	
	Formally, let $\ObsSpace=\cX$ be the space of contexts,
        $[A]$ be the set of actions, and $\Pi:\cX\to[A]$ be the
        space of decisions. The distribution $(r,x)\sim M(\pi)$ then
        has the following structure: $x\sim\cD\sups{M}$ and $r\sim
        R\sups{M}(\cdot|x, \pi(x))$ for some context distribution
        $\cD\sups{M}$ and reward distribution $R\sups{M}$. In other
        words, the distribution $\cD\sups{M}$ for the context $x$
        (treated as an observation) is part of the model $M$.        
\end{example}

  \begin{example}[Online reinforcement learning]
  \label{ex:rl}
  
The online reinforcement learning framework we introduced in
\cref{sec:mdp} immediately falls into the \FrameworkShort framework
by taking $\Act=\PiRNS$, $r\ind{t}=\sum_{h=1}^{H}r_h\ind{t}$, and
$\obs\ind{t}=\tau\ind{t}$. 
Tabular MDPs are one choice of model class $\cM$. Reinforcement
learning with function approximation, the subject of \cref{sec:rl},
corresponds to other restrictions on $\cM$ (e.g., low-rank or smooth structure in the transition dynamics); we will apply the tools of this section to each of these settings.\loose
\end{example}
Other problems included in the \FrameworkShort framework
which we will not cover in this monograph include partially observable MDPs
\citep{kaelbling1998planning,jin2020sample,liu2022partially,golowich2022learning},
where the observation is the sequence of observations and actions in an
episode, and bandits with graph-structured feedback
\citep{mannor2011bandits,alon2015online}, where the observation contains
the rewards of neighboring actions. In each case the framework is the
same, all that changes is the model class $\cM$.\loose

\subsection{Refresher: Information-Theoretic Divergences}
\label{sec:divergences}
To develop algorithms and complexity measures for \generaldm, we need
a way to measure the distance between distributions over abstract
observations (this was not a concern for the structured and
contextual bandit settings, where we only needed to consider the mean
reward function). To do this, we use the notion \emph{Csisz\'ar $f$-divergence}, which includes the Kullback-Leibler (KL)
divergence, total variation distance, and Hellinger distance as special
cases (we already made use of the latter in \cref{sec:mdp}, but it will be helpful to review some additional properties before we continue).\loose

Let $\bbP$ and $\bbQ$ be probability distributions over a measurable space
$(\Omega,\filt)$. We say that $\bbP$ is \emph{absolutely continuous}
with respect to $\bbQ$ if for all events $A\in\filt$,
$\bbQ(A)=0\implies\bbP(A)=0$; we denote this by $\bbP\ll\bbQ$. For a convex function $f:(0,\infty)\to\bbR$ with $f(1)=0$, the associated
$f$-divergence for $\bbP$ and $\bbQ$ is
given by
\begin{align}
  \label{eq:fdiv1}
  \Df{\bbP}{\bbQ} \ldef \En_{\bbQ}\brk*{f\prn*{\frac{d\bbP}{d\bbQ}}}
\end{align}
whenever $\bbP\ll\bbQ$. More generally, defining
$p=\frac{d\bbP}{d\nu}$ and $q=\frac{d\bbQ}{d\nu}$ for a common
dominating measure $\nu$, we have
\begin{align}
  \label{eq:fdiv2}
  \Df{\bbP}{\bbQ} \ldef \int_{q>0}qf\prn*{\frac{p}{q}}d\nu +
  \bbP(q=0)\cdot{}f'(\infty), 
\end{align}
where $f'(\infty)\ldef{}\lim_{x\to{}0^{+}}xf(1/x)$. We make use of the following $f$-divergences, all of which have
unique properties that make them useful in different contexts.\footnote{Another important example which appears throughout the decision making literature, but not this monograph, is the Neyman-Pearson
$\chi^2$-divergence.}
\begin{itemize}
\item  Choosing $f(t)=\frac{1}{2}\abs{t-1}$ gives the \emph{total
    variation (TV) distance}
  \[
    \Dtv{\bbP}{\bbQ} =
    \frac{1}{2}\int\abs*{\frac{d\bbP}{d\nu}-\frac{d\bbQ}{d\nu}}d\nu,
  \]
  which can also be written as
  \begin{equation*}
  \Dtv{\bbP}{\bbQ}=\sup_{A\in\filt}\abs{\bbP(A)-\bbQ(A)}.
\end{equation*}
\item Choosing $f(t) = (1-\sqrt{t})^2$ gives \emph{squared Hellinger distance}
  \[
    \Dhels{\bbP}{\bbQ}=\int\prn*{\sqrt{\frac{d\bbP}{d\nu}}-\sqrt{\frac{d\bbQ}{d\nu}}}^{2}d\nu.
  \]
\item Choosing $f(t)=t\log(t)$ gives the \emph{Kullback-Leibler divergence}:
  \[
    \Dkl{\bbP}{\bbQ} =\left\{
      \begin{array}{ll}
        \int\log\prn[\big]{
        \frac{d\bbP}{d\bbQ}
        }d\bbP,\quad{}&\bbP\ll\bbQ,\\
        +\infty,\quad&\text{otherwise.}
      \end{array}\right.
  \]
\end{itemize}
Note that for TV distance and Hellinger distance, we use the notation
$D(\cdot,\cdot)$ rather than $D(\cdot\dmid\cdot)$ to emphasize that
the divergence is symmetric. The divergences satisfy the following relationship.\loose
\begin{lem}[Total variation, Hellinger, and KL]
  \label{lem:divergence_inequality}
  For all distributions $\bbP$ and $\bbQ$,
  \begin{equation}
    \label{eq:divergence_inequality}
    \Dtvs{\bbP}{\bbQ} \leq{} \Dhels{\bbP}{\bbQ}\leq\Dkl{\bbP}{\bbQ}.
  \end{equation}
\end{lem}
\begin{proof}[\pfref{lem:divergence_inequality}]
  Take the common dominating measure $\nu=\bbP+\bbQ$ and write
  $p=d\bbP/d\nu$ and $q=d\bbQ/d\nu$. By Cauchy--Schwarz,
  \begin{align*}
    \Dtv{\bbP}{\bbQ}
    &=\frac12\int|\sqrt p-\sqrt q|(\sqrt p+\sqrt q)\,d\nu\\
    &\leq\frac12\left(\int(\sqrt p-\sqrt q)^2\,d\nu\right)^{1/2}
       \left(\int(\sqrt p+\sqrt q)^2\,d\nu\right)^{1/2}
    \leq\sqrt{\Dhels{\bbP}{\bbQ}},
  \end{align*}
  since $\int\sqrt{pq}\,d\nu\leq1$ and thus the second integral is
  at most $4$. Squaring proves the first inequality.
  The second is immediate if $\Dkl{\bbP}{\bbQ}=+\infty$.
  Otherwise $\bbP\ll\bbQ$, and $-\log u\geq1-u$ gives
  \begin{align*}
    \Dkl{\bbP}{\bbQ}
    &=-2\int_{p>0}p\log\sqrt{q/p}\,d\nu
      \geq2\int_{p>0}p(1-\sqrt{q/p})\,d\nu\\
    &=2-2\int\sqrt{pq}\,d\nu
      =\Dhels{\bbP}{\bbQ}.
  \end{align*}
\end{proof}

It is known that $\Dtv{\bbP}{\bbQ} = 1$ if and only if
$\Dhels{\bbP}{\bbQ}=2$, and $\Dtv{\bbP}{\bbQ} = 0$ if and only if
$\Dhels{\bbP}{\bbQ}=0$ (more generally, $\Dhels{\bbP}{\bbQ}\leq{}2\Dtv{\bbP}{\bbQ}$). Moreover, they induce same topology, i.e. a
sequence converges in one distance if and only if it converges the
other. KL divergence cannot be bounded by TV distance or Hellinger
distance in general, but the following lemma shows that it is possible
to relate these quantities if the density ratios under consideration
are bounded.
  \begin{lem}[KL is bounded by Hellinger under bounded density ratios]
    \label{lem:kl_hellinger}
    Let $\bbP$ and $\bbQ$ be probability distributions over a measurable space
    $(\Omega,\filt)$. If
    $\sup_{F\in\filt}\frac{\bbP(F)}{\bbQ(F)}\leq{}V$, then
    \begin{equation}
      \label{eq:kl_hellinger}
      \Dkl{\bbP}{\bbQ}\leq{}(2+\log(V))\Dhels{\bbP}{\bbQ}.
    \end{equation}
  \end{lem}
  Other properties we use include:
\begin{itemize}[leftmargin=*]
    \item Boundedness of TV (by $1$) and Hellinger (by $2$).
    \item The triangle inequality for total variation distance
      and for (unsquared) Hellinger distance $D_{\mathsf{H}}$.
    \item Convexity: every $f$-divergence is jointly convex in
      its arguments. In particular, for a distribution $\nu$ over
      distributions $\bbQ$ with mixture $\bar{\bbQ}=\En_{\bbQ\sim\nu}\brk{\bbQ}$,
      $\Dhels{\bbP}{\bar{\bbQ}}\leq\En_{\bbQ\sim\nu}\brk{\Dhels{\bbP}{\bbQ}}$.
    \item The \emph{data-processing inequality}, satisfied by
      every $f$-divergence: for any measurable map $g$, if $\bbP_g$ and
      $\bbQ_g$ denote the laws of $g(X)$ under $X\sim\bbP$ and
      $X\sim\bbQ$,
      \begin{equation}
        \label{eq:data_processing}
        \Df{\bbP_g}{\bbQ_g}\leq\Df{\bbP}{\bbQ}.
      \end{equation}
      In words, discarding part of an observation cannot increase the
      divergence.
    \item \emph{Chain rule} and subadditivity properties for KL and
      Hellinger divergence (\cref{lem:kl_chain_rule} and
      \cref{lem:Hellinger-chain}), stated in \cref{sec:lower} where they
      are used.
    \item A variational representation for TV distance: 
	\begin{align}
		\label{eq:var_repr_tv}
		\Dtv{\bbP}{\bbQ}=\sup_{g:\Omega\to\brk{0,1}}\abs*{\En_{\bbP}\brk{g} - \En_{\bbQ}\brk{g}}
	\end{align}
      \end{itemize}
      See \citet{polyanskiy2020lecture} for further
background.

\subsection{\CompText for General Decision Making}
\label{sec:dec_general}
To extend the \CompText beyond bandits, the notion of
information gain has to change. In a multi-armed bandit the mean rewards
are a sufficient statistic, so squared error between mean reward
functions is the right measure. In general decision making, the
observation may carry information that the mean reward does not (a
trajectory, in reinforcement learning), and so may the shape of the
reward distribution itself. To address this, we measure information gain by a divergence
between the distributions of rewards and observations. This handles both cases
at once, and is the only change from \cref{sec:structured}.

For a model class $\cM$, reference model $\Mhat$, and scale
parameter $\gamma>0$, the \CompText for \generaldm
\citep{foster2021statistical,foster2023tight} is
defined via
    \begin{equation}
  \label{eq:dec}
  \comp(\cM,\Mhat) =
  \inf_{p\in\Delta(\Act)}\sup_{M\in\cM}\En_{\pi\sim{}p}\biggl[\hspace{1pt}\underbrace{\fm(\pim)-\fm(\pi)}_{\text{regret
    of decision}}
    -\gamma\cdot\hspace{-2pt}\underbrace{\DhelsX{\big}{M(\pi)}{\Mhat(\pi)}}_{\text{information
        gain for obs.}}\hspace{-2pt}
    \biggr].
  \end{equation}
We further define
\begin{align}
\comp(\cM)=\sup_{\Mhat\in\conv(\cM)}\comp(\cM,\Mhat).\label{eq:dec_max}
\end{align}

Compared with the structured bandit version in
\cref{eq:dec_structured}, two things have changed: the coefficient is
defined over a class of \emph{models} rather than mean reward
functions, and the information gain
$\En_{\pi\sim{}p}\brk*{\Dhels{M(\pi)}{\Mhat(\pi)}}$ measures the distance
between the \emph{distributions} of rewards and observations under $M$
and $\Mhat$, not the distance between means. We have already met
\cref{eq:dec} once, without naming it: for tabular MDPs,
\cref{prop:igw_tabular} shows that \pcigw certifies
$\comp(\cM,\Mhat)\approxleq{}H^3SA/\gamma$ for every $\Mhat\in\cM$
(\cref{ex:dec_tabular_rl} below).

The supremum in \cref{eq:dec_max} is over $\conv(\cM)$ rather
than over all models because, as for structured bandits
(\cref{prop:dec_unconstrained}), allowing arbitrary reference models
changes the coefficient by at most a constant factor in $\gamma$.
\begin{prop}[References outside the convex hull]
  \label{prop:dec_unconstrained_general}
  Suppose that $\Pi$ and $\cM$ are finite and rewards are
  bounded, so that the minimax theorem (\cref{lem:sion}) applies; more
  generally, it suffices that the conclusion of \cref{lem:sion} holds for
  the payoff in \cref{eq:dec}, as in \cref{prop:swap_dec}. Then for all $\gamma>0$,
  \begin{equation}
    \label{eq:dec_unconstrained_general}
    \sup_{\Mhat}\comp(\cM,\Mhat)\leq\sup_{\Mhat\in\conv(\cM)}\comp[\gamma/4](\cM,\Mhat)=\comp[\gamma/4](\cM),
  \end{equation}
  where the supremum on the left is over all models $\Mhat:\Pi\to\Delta(\cR\times\cO)$.
\end{prop}
This is a slight generalization of \cref{prop:dec_unconstrained},
which gives the same statement for the squared-error coefficient with
no loss in $\gamma$; the factor of $4$ is the price of replacing squared
error by Hellinger distance. The proof is deferred to the end of the
section.

\begin{rem}[The choice of Hellinger distance]
  Any $f$-divergence could be substituted in \cref{eq:dec}, but we will see that Hellinger satisfies two properties that are essential for the analysis in this section and \cref{sec:lower}. First, it is dominated by KL divergence, so an estimator that minimizes the logarithmic loss controls it (\cref{lem:log_loss_hellinger}), which is
  what we need for upper bounds. Second, it dominates
  total variation (\cref{lem:divergence_inequality}), so a small
  information gain means the models are hard to distinguish by any test,
  which is what the lower bounds of \cref{sec:lower} need. The latter results also use that Hellinger distance obeys a chain rule over sequences.
\end{rem}

\subsubsection{Examples}
To build intuition as to how the general \CompText
adapts to the structure of the model class $\cM$, let us review a few examples---some familiar, and some new.

\begin{example}[Multi-armed bandit with Gaussian rewards]
  \label{ex:bandit_gaussian}
  Let $\Pi=[A]$, $\RewardSpace=\reals$, $\ObsSpace=\{\emptyset\}$. We define
  \[
    \cMgb = \crl*{M: M(\pi) = \cN(f(\pi), 1), f:\Pi\to[0,1]}.
  \]
  We claim that 
  \begin{equation}
    \label{eq:dec_mab_g}
    \comp(\cMgb) \propto \frac{A}{\gamma}.
  \end{equation}
  To prove this, consider the case where $\Mhat\in\cM$ for
  simplicity. Recall that we have previously shown that this
  behavior holds for the squared error version of the \CompShort
  defined in \eqref{eq:dec_structured}. Thus, it is sufficient to argue that the squared Hellinger divergence for Gaussian distributions reduces to square difference between the means:
  $$\Dhels{M(\pi)}{\Mhat(\pi)} \propto (f\sups{M}(\pi) - f\sups{\Mhat}(\pi))^2.$$
  Once this is shown, the claim follows from \pref{prop:igw_exact}. To prove
  this, first note that
  \begin{align}
  \Dhels{M(\pi)}{\Mhat(\pi)} \leq \Dkl{M(\pi)}{\Mhat(\pi)} =
    \frac{1}{2}(f\sups{M}(\pi) - f\sups{\Mhat}(\pi))^2.\label{eq:hellinger_gaussian_ub}
  \end{align}
  In the other direction, one can directly compute
  $$\Dhels{M(\pi)}{\Mhat(\pi)} = 2 - 2\exp\crl*{-\frac{1}{8}(f\sups{M}(\pi) - f\sups{\Mhat}(\pi))^2}$$
  and using that $1-\exp\{-x\}\geq (1-e^{-1})x$ for $x\in[0,1]$, we establish
  $$\Dhels{M(\pi)}{\Mhat(\pi)} \geq c\cdot(f\sups{M}(\pi) - f\sups{\Mhat}(\pi))^2$$
  for $c=\frac{1-1/e}{8}$.
\end{example}
In fact, one can show that the general DEC in \cref{eq:dec} coincides
with the basic squared error version from 
\pref{sec:structured} for general structured bandit problems, not just
multi-armed bandits \citep{foster2021statistical}.

Let us next consider a twist on the bandit problem that is more
information-theoretic in nature, and shows why the divergence
in \cref{eq:dec} is taken between distributions rather than means.
\begin{example}[Bandits with structured noise]
  \label{ex:bandit_structured_noise}
   Let $\Pi=[A]$, $\RewardSpace=\reals$, $\ObsSpace=\{\emptyset\}$. We define
  \[
    \cMsn = \crl*{M_1,\ldots,M_A}\cup\crl*{M_0}
  \]
  where
  $M_i(\pi)\ldef{}\cN(\nicefrac{1}{2}, 1)$ for $\pi\neq i$ and
  $M_i(\pi)\ldef{}\Ber(\nicefrac{3}{4})$ for $\pi=i$; we further
  define $M_0(\pi)\ldef{}\cN(\nicefrac{1}{2}, 1)$ for all
  $\pi\in\Pi$. We take $M_0$ as the reference model. Before proceeding with the calculations, observe that
  we can solve the general decision making problem when the underlying model
  is $\Mstar\in\cM$ with a simple algorithm. It is sufficient to select every action in $[A]$ only once:
  all optimal actions have Bernoulli rewards and give
  $r\in\crl{0,1}$ almost surely, while the suboptimal action has Gaussian
  rewards, and gives $r\notin\crl{0,1}$ almost surely. Thus, if we
  select an action and observe a reward $r\in\crl{0,1}$, we know
  that we have identified the optimal action.
  
  The valuable information contained in the reward distribution is reflected in
  the Hellinger divergence, which attains its maximum value when
  comparing a continuous distribution to a discrete one:
  $$\Dhels{M_i(\pi)}{M_0(\pi)} = 2\indic{\pi=i}.$$ To use this
  property to
  derive the upper bound on $\comp(\cMsn,M_0)$, first note that the maximum over $M$ in the definition of $\comp(\cMsn,M_0)$ is not attained at $M=M_0$, since in that case both the divergence and regret terms are zero, irrespective of $p$. Now, take $p=\unif[A]$. Then for any $M\in\{M_1,\ldots,M_A\}$, 
  $$\En_{\pi\sim{}p} \brk*{\fm(\pim)-\fm(\pi)} = (1-1/A)(\nicefrac{3}{4}-\nicefrac{1}{2}),$$
  and 
  $$\comp(\cMsn,M_0)\lesssim (1-1/A)(\nicefrac{3}{4}-\nicefrac{1}{2}) - \gamma \frac{2}{A} \lesssim \indic{\gamma\leq{}A/4}.$$  This leads to an upper bound
  \begin{equation}
    \label{eq:dec_mab_sn}
    \comp(\cMsn,M_0) \lesssim \indic{\gamma\leq{}A/4}
  \end{equation}
  which can also be shown to be tight.
\end{example}
  
\begin{example}[Bandits with Full Information]
  \label{ex:full_info_lower}
  Consider a ``full-information'' learning setting. We have $\Act=\brk{A}$ and $\cR=\brk{0,1}$, and for a given decision $\pi$ we
  observe a reward $r$ as in the standard multi-armed bandit, but also receive an observation
  $o = (r(\pi'))_{\pi'\in\brk{A}}$ consisting of (counterfactual) rewards for \emph{every} action.

  For a given model $M$, let $\Mr(\pi)$ denote the distribution over
  the reward $r$ for $\pi$, and let $\Mo(\pi)$ denote the distribution of $o$. Then for any decision $\pi$, since all rewards are observed,
the data processing inequality implies that for all $M,\Mhat\in\cM$
and $\pi'\in\Pi$,
  \begin{align}
    \label{eq:hellinger_full_info}
    \Dhels{M(\pi)}{\Mhat(\pi)}&\geq{}\Dhels{\Mo(\pi)}{\Mhato(\pi)}\\
    &= \Dhels{\Mo(\pi')}{\Mhato(\pi')} \geq \Dhels{\Mr(\pi')}{\Mhatr(\pi')}.
  \end{align}
Using this property, we show that for any $\Mhat\in\cM$,
\begin{align}
  \comp(\cM,\Mhat)\leq\frac{1}{\gamma}.\label{eq:dec_full_info}
\end{align}
Comparing to
  the finite-armed bandit, we see that the \CompShort for this example is
  independent of $A$, which reflects the extra information contained
  in the observation $o$.

  To prove \cref{eq:dec_full_info}, for a given model $\Mhat\in\cM$ we
  choose $p=\mathbb{I}_{\pimhat}$ (i.e. the decision maker selects $\pimhat$ deterministically), and bound
  $\En_{\pi\sim{}p}\brk*{\fm(\pim)-\fm(\pi)}$ by
  \begin{align*}
    \fm(\pim)-\fm(\pimhat)
    & \leq{}     \fm(\pim)-\fm(\pimhat) +     \fmhat(\pimhat)-\fmhat(\pim)\\
    &\leq{}
    2\cdot{}\max_{\pi\in\crl{\pim,\pimhat}}\abs{\fm(\pi)-\fmhat(\pi)} \\
    &\leq{}  2\cdot{}\max_{\pi\in\crl{\pim,\pimhat}}\Dhel{\Mr(\pi)}{\Mhatr(\pi)}.
  \end{align*}
  We then use the AM-GM inequality, which implies that for any $\gamma>0$,
  \begin{align*}
    \max_{\pi\in\crl{\pim,\pimhat}}\Dhel{\Mr(\pi)}{\Mhatr(\pi)}
    &\approxleq{}
      \gamma\cdot{}\max_{\pi\in\crl{\pim,\pimhat}}\Dhels{\Mr(\pi)}{\Mhatr(\pi)}
    + \frac{1}{\gamma} \\
    &\leq{}
    \gamma\cdot{}\Dhels{M(\pimhat)}{\Mhat(\pimhat)}
    + \frac{1}{\gamma},
  \end{align*}
  where the final inequality uses \pref{eq:hellinger_full_info}. This
  certifies that for all $M\in\cM$, the choice for $p$ above
  satisfies \[\En_{\pi\sim{}p}\brk*{\fm(\pim)-\fm(\pi)-\gamma\cdot\Dhels{M(\pi)}{\Mhat(\pi)}}\approxleq{}\frac{1}{\gamma},\]
  so we have $\comp(\cM,\Mhat)\approxleq{}\frac{1}{\gamma}$.
\end{example}  
\begin{example}[Tabular reinforcement learning]
  \label{ex:dec_tabular_rl}
  Let $\cM$ be the class of tabular MDPs with $S$ states, $A$
  actions, horizon $H$, and known rewards (\cref{assm:ucbvi}), viewed as a
  general decision making problem as in \cref{ex:rl}. By
  \cref{prop:igw_tabular}, the \pcigw strategy certifies that
  \[
    \comp(\cM,\Mhat)\approxleq{}\frac{H^3SA}{\gamma}
    \qquad\text{for all }\Mhat\in\cM.
  \]
  This is the first example in which the observation $o=\tau$ carries
  information beyond the reward, and the Hellinger distance between
  trajectory distributions is what makes the bound polynomial in $S$, $A$,
  and $H$ rather than exponential.
\end{example}

\subsection{Estimation-to-Decisions Algorithm for General Decision Making}
\label{sec:algo_gen_dm}

The most general version of the Estimation-to-Decisions meta-algorithm (\etd) is given below.%

\begin{algbox}[frametitle={\algtitle{alg:etd}{Estimation-to-Decisions (\etd) for General Decision Making}}]
	\begin{algorithmic}
       \Require Exploration parameter $\gamma>0$; online estimation oracle.
  \For{$t=1, 2, \cdots, T$}
  \State Obtain $\Mhat\ind{t}$ from online estimation oracle
  with $(\pi\ind{1},r\ind{1},o\ind{1}),\ldots,(\pi\ind{t-1},r\ind{t-1},o\ind{t-1})$.
  \State Compute \hfill\algcommentlight{Minimizer
  for $\comp(\cM,\Mhat\ind{t})$.}
  \begin{equation}
    \label{eq:algorithm_argmin}
    p\ind{t}=\argmin_{p\in\Delta(\Act)}\sup_{M\in\cM}
\En_{\pi\sim{}p}\biggl[\fm(\pim)-\fm(\pi)
    -\gamma\cdot\DhelsX{\big}{M(\pi)}{\Mhat\ind{t}(\pi)}
    \biggr].
  \end{equation}
\State{}Sample decision $\pi\ind{t}\sim{}p\ind{t}$ and update estimation
algorithm with $(\pi\ind{t},r\ind{t}, \obs\ind{t})$.
\EndFor
\end{algorithmic}

\end{algbox}

 Compared to the meta-algorithm we studied for structured bandits, the main difference is
that rather than trying to estimate the reward function $\fstar$, we
now estimate the underlying model $\Mstar$, as we did for
tabular reinforcement learning in \cref{sec:dec_tabular}. To do so, we appeal once
again to the notion of an \emph{online estimation oracle}, but this
time for \emph{model estimation}.
\begin{definition}[Online Estimation Oracle]
  \label{def:online_estimation_oracle}
  At each time $t\in[T]$, an \emph{online estimation oracle} returns, given
  $$(\pi\ind{1},r\ind{1},o\ind{1}),\ldots,(\pi\ind{t-1},r\ind{t-1},o\ind{t-1})$$
  with $(r\ind{i},o\ind{i})\sim\Mstar(\pi\ind{i})$ and $\pi\ind{i}\sim
  p\ind{i}$, a model $\Mhat\ind{t}\in\cMhat$, where
  $\cMhat\supseteq\cM$ is a fixed class of models associated with the
  oracle. Define the cumulative
  (Hellinger) estimation error
  \begin{equation}
    \label{eq:hellinger_error}
    \EstHel \ldef{} \sum_{t=1}^{T}\En_{\pi\ind{t}\sim{}p\ind{t}}\brk*{\DhelsX{\Big}{\Mstar(\pi\ind{t})}{\Mhat\ind{t}(\pi\ind{t})}}.
  \end{equation}
  The oracle is required to ensure that $\EstHel\leq\EstHel(\cM,T,\delta)$
  with probability at least $1-\delta$, where $\EstHel(\cM,T,\delta)$ is a
  deterministic bound. Here, $p\ind{i}(\cdot\mid\cH\ind{i-1})$ is the
  randomization distribution for the decision-maker.
\end{definition}
This mirrors the online regression oracle of
\cref{def:online_regression_oracle_structured}, with two changes: the
oracle estimates the model rather than the mean reward function, and
estimation error is measured by squared Hellinger distance between the
distributions $\Mstar(\pi)$ and $\Mhat\ind{t}(\pi)$ rather than by
squared error between means. As in the preceding sections, the bare
symbol $\EstHel$ denotes the realized error and $\EstHel(\cM,T,\delta)$
the guarantee. We use $\cMhat$ to denote any set where the oracle's predictions are guaranteed to live. A \emph{proper} oracle has $\cMhat=\cM$; the estimator
for tabular MDPs in \cref{prop:tabular_estimation} is one example.
Standard online estimation algorithms such as exponential weights
produce improper (mixture) estimates, for which $\cMhat=\conv(\cM)$; see
\cref{sec:online_estimation_hellinger}.

\etd solves a min-max problem rather than applying an explicit
rule such as inverse gap weighting, for the reasons discussed in
\cref{sec:dec_structured}: the min-max problem is the definition of the
best possible exploration strategy for the class, and any explicit rule
that certifies an upper bound on its value can be used in its place.

\begin{theorem}[Regret of \etd]
\label{prop:upper_main}
\etd with exploration parameter $\gamma>0$ and an online
estimation oracle with estimates in $\cMhat$ guarantees that for every
realization,
\begin{equation}
  \label{eq:upper_main}
\RegDM \leq{}
  \sum_{t=1}^{T}\comp(\cM,\Mhat\ind{t}) + \gamma\cdot\EstHel
  \leq \sup_{\Mhat\in\cMhat}\comp(\cM,\Mhat)\cdot{}T + \gamma\cdot\EstHel.
\end{equation}
Furthermore, if \cref{prop:dec_unconstrained_general} applies, then for
any $\cMhat$ we have
$\sup_{\Mhat\in\cMhat}\comp(\cM,\Mhat)\leq\comp[c\gamma](\cM)$ for an
absolute constant $c\geq{}1/4$,
so that with probability at least $1-\delta$,
\begin{equation}
  \label{eq:upper_main_hp}
\RegDM \leq{}
  \comp[c\gamma](\cM)\cdot{}T + \gamma\cdot\EstHel(\cM,T,\delta).
\end{equation}
\end{theorem}
For proper oracles, and for oracles with $\cMhat=\conv(\cM)$,
the first statement already gives $\comp(\cM)$ with $c=1$.
Optimizing over the parameter $\gamma$ in \cref{eq:upper_main_hp} yields
\[
\RegDM \leq{}
\inf_{\gamma>0}\crl[\bigg]{\comp(\cM)\cdot{}T
  + \gamma\cdot\EstHel(\cM,T,\delta)}
\leq
2\cdot\inf_{\gamma>0}\max\crl[\bigg]{\comp(\cM)\cdot{}T,
  \gamma\cdot\EstHel(\cM,T,\delta)},
\]
up to the constant $c$ inside $\comp[c\gamma]$, which we
suppress.
\cref{sec:online_estimation_hellinger} shows that for any finite class $\cM$, the
averaged exponential weights algorithm with the logarithmic loss is an
online estimation oracle with $\EstHel(\cM,T,\delta)\approxleq\log(\abs{\cM}/\delta)$
and $\cMhat=\conv(\cM)$. Putting these
facts together, we see that for any finite class, it is possible to
achieve
\begin{align}
  \label{eq:etd_finite}
\RegDM \leq{}
  \comp(\cM)\cdot{}T + \gamma\cdot\log(\abs{\cM}/\delta)
\end{align}
with probability at least $1-\delta$.
\begin{proof}[\pfref{prop:upper_main}]
  We write
\begin{align*}
  \RegDM &=
           \sum_{t=1}^{T}\En_{\pi\ind{t}\sim{}p\ind{t}}\brk*{\fmstar(\pimstar)-\fmstar(\pi\ind{t})}\\
         &=
           \sum_{t=1}^{T}\En_{\pi\ind{t}\sim{}p\ind{t}}\brk*{\fmstar(\pimstar)-\fmstar(\pi\ind{t})}
           - \gamma{}\cdot
           \En_{\pi\ind{t}\sim{}p\ind{t}}\brk*{\DhelsX{\big}{\Mstar(\pi\ind{t})}{\Mhat\ind{t}(\pi\ind{t})}}
           + \gamma\cdot{}\EstHel.
\end{align*}
For each $t$, since $\Mstar\in\cM$, we have
\begin{align}
  &\En_{\pi\ind{t}\sim{}p\ind{t}}\brk*{\fmstar(\pimstar)-\fmstar(\pi\ind{t})}
           - \gamma{}\cdot
  \En_{\pi\ind{t}\sim{}p\ind{t}}\brk*{\DhelsX{\big}{\Mstar(\pi\ind{t})}{\Mhat\ind{t}(\pi\ind{t})}}\notag\\
&  \leq  \sup_{M\in\cM}\En_{\pi\ind{t}\sim{}p\ind{t}}\brk*{\fm(\pim)-\fm(\pi\ind{t})}
           - \gamma{}\cdot
                                                                                                   \En_{\pi\ind{t}\sim{}p\ind{t}}\brk*{\DhelsX{\big}{M(\pi\ind{t})}{\Mhat\ind{t}(\pi\ind{t})}}\notag\\
  &  = \inf_{p\in\Delta(\Act)}\sup_{M\in\cM}\En_{\pi\sim{}p}\brk*{\fm(\pim)-\fm(\pi)
           - \gamma{}\cdot
    \DhelsX{\big}{M(\pi)}{\Mhat\ind{t}(\pi)}}\notag\\
  &  = \comp(\cM,\Mhat\ind{t}).\label{eq:minimax_regret_basic}
\end{align}
Summing over all rounds $t$ gives the first inequality in
\cref{eq:upper_main}, and using that $\Mhat\ind{t}\in\cMhat$ for all $t$
gives the second. The final statement follows from
$\sup_{\Mhat\in\cMhat}\comp(\cM,\Mhat)\leq\comp[c\gamma](\cM)$, on the
event that the oracle guarantee $\EstHel\leq\EstHel(\cM,T,\delta)$
holds.
\end{proof}

\paragraph{Examples for the upper bound}
We now revisit the examples from \cref{sec:dec_general} and use \etd
and \cref{prop:upper_main} to derive regret bounds for them.
\begin{examplecont}[Multi-armed bandit with Gaussian rewards]{ex:bandit_gaussian}
  For the Gaussian bandit problem from \pref{ex:bandit_gaussian}, plugging the bound
  $\comp(\cMgb)\approxleq{}A/\gamma$ into \pref{prop:upper_main} yields
  \[
    \RegDM \approxleq{}
    \frac{AT}{\gamma} + \gamma\cdot{}\EstHel(\cM,T,\delta)
  \]
  with probability at least $1-\delta$. Choosing $\gamma=\sqrt{AT/\EstHel(\cM,T,\delta)}$ balances the terms above and
  gives
    \[
    \RegDM \approxleq{}
    \sqrt{AT\cdot\EstHel(\cM,T,\delta)}.
  \]
\end{examplecont}

\begin{examplecont}[Bandits with structured noise]{ex:bandit_structured_noise}
  For the bandit-type problem with structured noise from
  \pref{ex:bandit_structured_noise}, the bound
  $\comp(\cMsn)\approxleq{}\indic{\gamma\leq{}A/4}$ yields
    \[
    \RegDM \approxleq{}
    \indic{\gamma\leq{}A/4}\cdot{}T + \gamma\cdot{}\EstHel(\cM,T,\delta)
  \]
  with probability at least $1-\delta$.   We can choose $\gamma=A$, which gives
  \[
    \RegDM \approxleq{}
    A\cdot{}\EstHel(\cM,T,\delta).
  \]
\end{examplecont}

\begin{examplecont}[Tabular reinforcement learning]{ex:dec_tabular_rl}
  For tabular reinforcement learning, the estimator of
  \cref{prop:tabular_estimation} is proper ($\Mhat\ind{t}\in\cM$), so
  \cref{ex:dec_tabular_rl} bounds each round's term in
  \cref{eq:minimax_regret_basic} by
  $\comp(\cM,\Mhat\ind{t})\approxleq{}H^3SA/\gamma$, and the proof of
  \cref{prop:upper_main} gives
  \[
    \RegDM \approxleq{}
    \frac{H^3SA\cdot{}T}{\gamma} + \gamma\cdot{}\EstHel.
  \]
  This is just a restatement of the first part of \cref{prop:e2d_tabular}. \pcigw is an explicit strategy that certifies an upper
  bound on the value of \cref{eq:algorithm_argmin} rather than solving
  it exactly, which suffices by the discussion of certified bounds in
  \cref{sec:structured}.
\end{examplecont}
\paragraph{Further examples}
The examples above are deliberately simple. Richer instances of the
\CompShort appear throughout the monograph, and we collect pointers here. For
reinforcement learning with function approximation, \cref{sec:dec_bellman_rank}
bounds the \CompShort for the class of MDPs with low Bellman rank
(\cref{prop:dec_bellman_rank}), which subsumes tabular MDPs, low-rank
MDPs, and several other models; this is the most
substantial application of \cref{prop:upper_main} in the monograph. For bandit problems, the calculations in
\cref{sec:dec_examples_structured}---the cheating code
(\cref{prop:cheating}), linear bandits (\cref{prop:dec_linear}), Lipschitz
bandits (\cref{prop:dec_lipschitz}), and the eluder dimension, concave, and
ReLU examples---carry over to the general \CompShort by the argument in \cref{ex:bandit_gaussian} (under Gaussian or
bounded noise). 

\subsubsection{Online Estimation with Hellinger Distance}
\label{sec:online_estimation_hellinger}

Let us now give some more detail as to how to perform the online model
estimation required by \cref{prop:upper_main}. Model estimation is a
more challenging problem than regression, since we are estimating the
underlying \emph{conditional distribution} rather than just the
conditional mean. In spite of this difficulty, estimating the model $\Mstar$ with respect to Hellinger distance is a
classical problem that we can solve using the online learning tools
introduced in \pref{sec:ol}; in particular, online \emph{conditional density estimation} with the log loss, which we already saw in a special case for tabular reinforcement learning in \cref{prop:tabular_estimation}.

Instead of directly performing estimation with respect to Hellinger distance,
the simplest way to develop conditional density estimation algorithms is to work with the
logarithmic loss. Given a tuple $(\pi\ind{t},
r\ind{t}, \obs\ind{t})$, define the
logarithmic loss for a model $M$ as\looseness=-1
\begin{equation}
  \label{eq:log_loss}
  \logloss\ind{t}(M) = \log\prn*{
    \frac{1}{\densm(r\ind{t}, \obs\ind{t}\mid{}\pi\ind{t})}
  },
\end{equation}
where we define $\densm(\cdot,\cdot\mid{}\pi)$ as the conditional
density for $(r,\obs)$ under $M$. As in \cref{sec:ol}, we
write $\RegOLLog$ for the online predictor's regret under log loss:
this compares the cumulative log loss of its model estimates with
that of the best fixed model in $\cM$. Thus,
\begin{equation}
  \label{eq:log_loss_regret}
  \RegOLLog
  = \sum_{t=1}^{T}\logloss\ind{t}(\Mhat\ind{t}) - \inf_{M\in\cM}\sum_{t=1}^{T}\logloss\ind{t}(M).
\end{equation}
The following result converts $\RegOLLog$ into a bound on
Hellinger estimation error $\EstHel$, which is what
\cref{prop:upper_main} requires.
\begin{lem}[Log-loss regret controls Hellinger estimation error]
  \label{lem:log_loss_hellinger}
  For any online estimation algorithm,
  whenever \pref{ass:realizability} holds, we have
  \begin{equation}
    \En\brk*{\RegOLLog} \geq \En\brk*{
      \sum_{t=1}^{T}\Dkl{\Mstar(\pi\ind{t})}{\Mhat\ind{t}(\pi\ind{t})}
    },
  \end{equation}
  so that
  \begin{equation}
	  \label{eq:exp_hel_kl}
    \En\brk*{\EstHel} \leq \En\brk*{\RegOLLog}.
  \end{equation}
Furthermore, for any $\delta\in(0,1)$, with
probability at least $1-\delta$,
\begin{equation}
	\label{eq:high_prob_hel_kl}
    \EstHel \leq{} \RegOLLog + 2\log(\delta^{-1}).
\end{equation}
\end{lem}
This result is desirable because regret minimization with the
logarithmic loss is a well-studied problem in
online learning. Efficient algorithms are known for model classes of interest
\citep{cover1991universal,vovk1995game,kalai2002efficient,hazan2015online,orseau2017soft,rakhlin2015sequential,foster2018logistic,luo2018efficient},
and this is complemented by theory which provides minimax rates for generic
model classes
\citep{shtarkov1987universal,opper99logloss,cesabianchi99logloss,bilodeau2020tight}.
One example we have already seen (\pref{sec:background}) is the averaged exponential
weights method, which 
guarantees
\[
  \RegOLLog\leq\log\abs{\cM}
\]for finite
classes $\cM$; by \cref{eq:high_prob_hel_kl}, it is therefore an online
estimation oracle in the sense of \cref{def:online_estimation_oracle}
with $\EstHel(\cM,T,\delta)=\log\abs{\cM}+2\log(\delta^{-1})\approxleq\log(\abs{\cM}/\delta)$.
More generally, any algorithm whose log-loss regret is bounded by a
deterministic quantity $R(T)$ is an online estimation oracle with
$\EstHel(\cM,T,\delta)=R(T)+2\log(\delta^{-1})$. Another example is that for linear models (i.e., $\densm(r,o\mid{}\pi)=\tri*{\phi(r,o,\pi),\theta}$ for a
fixed feature map $\phi\in\bbR^{d}$), algorithms with
$\RegOLLog=\bigoh(d\log(T))$ are known
\citep{rissanen1986complexity,shtarkov1987universal}. All of these
algorithms output mixtures of models in $\cM$, so
$\cMhat=\conv(\cM)$. The estimator for
tabular MDPs in \cref{prop:tabular_estimation} is another instance: its
proof bounds the log-loss regret of Laplace's estimator by
$HS^2A\log(e(T+S))$ and then applies \cref{eq:high_prob_hel_kl}. We refer the
reader to Chapter 9 of \citet{cesa2006prediction} for further examples and
discussion.

While \eqref{eq:exp_hel_kl} is straightforward,
\eqref{eq:high_prob_hel_kl} is rather remarkable, as the remainder
term does not scale with $T$. Indeed, a naive attempt at applying
concentration inequalities to control the deviations of the random
quantities $\EstHel$ and $\RegOLLog$ would require boundedness of the
loss function, which is problematic because the logarithmic loss can
be unbounded. The proof exploits unique properties of the moment
generating function for the log loss.

\subsection{The Constrained \CompText}
\label{sec:constrained_comp}
\newcommand{\alphalowerabs}{\alpha(\veps,\gamma)}
\newcommand{\pirev}{\pi_{\circ}}

As we have defined it so far, the \CompShort of \cref{eq:dec} subtracts the information gain
from the regret. To close the section, we introduce a second form of the DEC that places a \emph{hard
constraint} on the information gain instead. This \emph{constrained}
\CompShort yields a sharper regret bound than
\cref{prop:upper_main}, and is a form that we revisit for lower bounds in \cref{sec:lower}. In what follows, we define the constrained \CompShort, state the regret bound it gives,
and then compare the two forms.

\subsubsection{Definition and Regret Upper Bound}
\label{sec:constrained_def}

For a model class $\cM$, reference model $\Mhat$, and radius
$\veps>0$, the constrained \CompText is defined as\footnote{We adopt the convention that the
  value of $\compc(\cM,\Mhat)$ is zero if there exists $p$ such that
  the set of $M\in\cM$ with
  $\En_{\pi\sim{}p}\brk*{\Dhels{M(\pi)}{\Mhat(\pi)}}\leq\veps^2$ is empty.}
\begin{align}
  \label{eq:dec_constrained}
  \compc(\cM,\Mhat)=
  \inf_{p\in\Delta(\Pi)}\sup_{M\in\cM}\crl*{\En_{\pi\sim{}p}\brk*{
  \fm(\pim) - \fm(\pi)}
  \mid\En_{\pi\sim{}p}\brk*{\Dhels{M(\pi)}{\Mhat(\pi)}}\leq\veps^2
  },
\end{align}
with
\begin{align}
  \label{eq:dec_constrained_max}
  \compc(\cM)\ldef{}\sup_{\Mhat\in\conv(\cM)}  \compc(\cM\cup\crl{\Mhat},\Mhat).
\end{align}
We henceforth refer to the coefficient in \cref{eq:dec} as the
\emph{offset} \CompShort to contrast it with this version. The interpretation is the same: the min player
chooses a distribution over decisions, and the max player chooses a
model that is hard to distinguish from the reference model yet incurs
large regret; the difference is that ``hard to distinguish'' is now the
constraint $\En_{\pi\sim{}p}\brk*{\Dhels{M(\pi)}{\Mhat(\pi)}}\leq\veps^2$
rather than a penalty. The natural radius is $\veps\approx{}T^{-1/2}$:
after $T$ rounds, models at Hellinger distance $\veps$ from the truth
cannot be told apart, so $\compc(\cM)$ at this radius is the regret the
learner cannot avoid in a single round. The union with $\crl{\Mhat}$ is a technical convenience which makes the reference model itself an admissible alternative for the max player, so that the constrained supremum is over a nonempty set (and the value is nonnegative) even when $\Mhat\in\conv(\cM)\setminus\cM$.

The constrained \CompShort leads to upper bounds on regret through
a refinement of \etd. We state the result for finite classes and omit
the proof.
\begin{prop}[Regret upper bound via constrained \CompShort]
  \label{prop:upper_constrained}
  Let $\cM$ be finite with $\Mstar\in\cM$. Suppose that for
  some $C_{\mathrm{reg}}>1$, the constrained DEC satisfies the growth condition
  \begin{equation}
    \label{eq:dec_regularity}
    \compc(\cM)\leq{}C_{\mathrm{reg}}^2\cdot\compc[\veps/C_{\mathrm{reg}}](\cM)
    \quad\text{for all }\veps\in(0,2),
  \end{equation}
  which holds in particular whenever $\compc(\cM)\propto\veps^{\rho}$ for
  some $\rho\leq{}2$. Then there is an algorithm that, for any
  $\delta\in(0,1)$, ensures that with probability at least $1-\delta$,
  \begin{equation}
    \label{eq:upper_constrained}
    \RegDM \leq{} C\cdot\prn*{\compc[\bar\veps](\cM)\cdot{}T\log{}T +
      \sqrt{T\log(\abs{\cM}/\delta)}},
    \qquad
    \bar\veps\ldef{}C'\sqrt{\frac{\log(\abs{\cM}/\delta)}{T}},
  \end{equation}
  where $C,C'>0$ are constants depending only on $C_{\mathrm{reg}}$.
\end{prop}
This result and the algorithm are due to
\citet{foster2023tight}; \citet{glasgow2023tight} give a simpler
single-pass algorithm with sharper constants. The idea behind the algorithm is to run \etd with the constrained min-max problem in place of the offset one: at round $t$ it plays the minimizer $p\ind{t}$ in the definition of $\compc[\bar\veps](\cM\cup\{\Mhat\ind{t}\},\Mhat\ind{t})$, where $\Mhat\ind{t}$ is produced by a specialized estimation oracle restricted to a
\emph{confidence set} of models consistent with the past estimates.

The regularity condition
\cref{eq:dec_regularity} rules out coefficients that collapse too
quickly as $\veps$ decreases, and can be replaced by the same condition
on any upper bound $r(\veps)\geq\compc(\cM)$, in which case $r(\bar\veps)$
replaces $\compc[\bar\veps](\cM)$ in \cref{eq:upper_constrained}. The
factor $\log{}T$ arises from summing the coefficient over geometrically
decreasing radii and can be removed whenever $\compc(\cM)\propto\veps^{\rho}$
for some $\rho<2$ \citep{foster2023tight,glasgow2023tight}; it is
needed only in the boundary case $\rho=2$.

\subsubsection{Comparison to the Offset \CompShort}
\label{sec:constrained_vs_offset}

The offset \CompShort can be viewed as a Lagrangian relaxation of the
constrained DEC, and always upper bounds it: 
\begin{align}
  \deccreg(\cM,\Mhat) &= \inf_{p\in\Delta(\Pi)}\sup_{M\in\cM}\crl*{
  \En_{\pi\sim{}p}\brk*{\fm(\pim) -\fm(\pi)}\mid{}\Enp\brk*{\Dhels{M(\pi)}{\Mhat(\pi)}}\leq\veps^2
                        } \notag\\
  &= \inf_{p\in\Delta(\Pi)}\sup_{M\in\cM}\inf_{\gamma\geq{}0}\crl*{
  \En_{\pi\sim{}p}\brk*{\fm(\pim) -\fm(\pi)} - \gamma\prn*{\Enp\brk*{\Dhels{M(\pi)}{\Mhat(\pi)}}-\veps^2}
    }\vee{}0 \notag\\
      &\leq \inf_{\gamma\geq{}0}\inf_{p\in\Delta(\Pi)}\sup_{M\in\cM}\crl*{
  \En_{\pi\sim{}p}\brk*{\fm(\pim) -\fm(\pi)} - \gamma\prn*{\Enp\brk*{\Dhels{M(\pi)}{\Mhat(\pi)}}-\veps^2}
        }\vee{}0\notag\\
        &= \inf_{\gamma\geq{}0}\crl*{
          \decoreg(\cM,\Mhat) + \gamma\veps^2
  }\vee{}0. \label{eq:constrained_to_offset}
\end{align}
In the opposite direction, it is straightforward to show that
\begin{align*}
  \comp(\cM)
  \approxleq{} \deccreg[\gamma^{-1/2}](\cM).
\end{align*}
This inequality is lossy, and there are
indeed classes for which the constrained \CompShort is much smaller than the
offset \CompShort, so that \cref{prop:upper_constrained} gives a
strictly better regret bound than \cref{eq:etd_finite}. The following
example, adapted from \citet{foster2023tight}, exhibits the gap.

\begin{prop}[Improvement from constrained \CompShort]
  \label{prop:constrained_gap}
  Fix $A\geq{}4$ and $\beta\in(0,1)$. Let
  $\Pi=[A]\cup\crl{\pirev}$ and $\cO=[A]\cup\crl{\bot}$, and let
  $\cM_\beta=\crl{M_0,M_1,\ldots,M_A}$ be the following class. Under
  $M_i$ for $i\in[A]$, each arm $\pi\in[A]$ gives deterministic reward
  $\frac12+\frac12\indic{\pi=i}$ and observation $\bot$, while the
  \emph{revealing decision} $\pirev$ gives reward $0$ and observation
  $o=i$ with probability $\beta$ and $o=\bot$ otherwise. Under $M_0$,
  each arm gives reward $\frac12$ and observation $\bot$, while $\pirev$
  gives reward $0$ and observation $o\sim\unif([A])$ with probability
  $\beta$ and $o=\bot$ otherwise. Then for all $\gamma>0$ and all
  $\veps^2\leq\beta/8$,
  \begin{equation}
    \label{eq:constrained_gap}
    \comp(\cM_\beta)\geq{}\frac{1}{4+16\gamma\beta}-\frac{2\gamma}{A}
    \qquad\text{and}\qquad
    \compc(\cM_\beta)\leq{}\frac{8\veps^2}{\beta}.
  \end{equation}
  Consequently, with $\beta=T^{-1/2}$ and $A=T^2$, the offset bound
  \cref{eq:etd_finite} is $\Omega(T^{3/4})$ for every choice of $\gamma$,
  whereas \cref{prop:upper_constrained} gives $\RegDM\leq\Otilde(\sqrt{T})$.
\end{prop}
  The proof is deferred to the end of the section, but the mechanism is easy to describe. The revealing decision is expensive, costing regret $\frac12$ per play, but its information is cheap in aggregate: an information budget of $\veps^2$ buys probability
$\veps^2/\beta$ of learning the identity of the good arm outright. The
constrained \CompShort is therefore quadratic in $\veps$. The offset DEC misses this detail: for any $\gamma$, the offset \CompShort is
of order $1/(\gamma\beta)$, and
$\inf_\gamma\crl{1/(\gamma\beta)+\gamma\veps^2}=2\veps/\sqrt{\beta}$ is
only linear in $\veps$. In the language of \cref{eq:constrained_to_offset},
strong duality fails, and no single multiplier $\gamma$ certifies the
quadratic rate.

Despite this example, we view the constrained \CompShort as a
conceptual tool rather than a replacement for the offset \CompShort. For
most of the classes in this monograph, the two coefficients behave the
same way, and the main use of the constrained form is for lower bounds,
where it is the natural quantity; this is how \cref{sec:lower} uses it.
For algorithm design and upper bounds we work with the offset \CompShort
throughout, for two reasons. It is simpler to derive algorithms from it,
since \etd needs only the offset min-max problem and an estimation
oracle. It is also usually simpler to find distributions $p$ that certify bounds
on it, as the examples in \cref{sec:dec_general} and the calculations in
\cref{sec:structured} illustrate.

\subsection{Bibliographic Remarks and Further Reading}\label{sec:general_dm_bib}
\paragraph{Takeaways}
This section replaced the separate settings studied so far by a single
framework, decision making with structured observations, in which a model
$M\in\cM$ maps each decision to a distribution over rewards and
observations. Two ingredients carry over from structured bandits with
minimal change: the \CompShort, now defined with squared Hellinger
distance between models in place of squared error between mean rewards,
and the \etd meta-algorithm, which pairs a minimizer for the \CompShort
with an online estimation oracle. \cref{prop:upper_main} bounds regret in terms of the \CompText and Hellinger estimation error, and
\cref{lem:log_loss_hellinger} shows that any online density estimator
with low log-loss regret is an estimation oracle in the required sense.

\paragraph{Bibliographic remarks}
The \FrameworkShort framework of \cref{sec:gdm_framework} is from
\citet{foster2021statistical}. 
The general \CompShort in \cref{eq:dec}, the online estimation oracle definition from 
\cref{def:online_estimation_oracle}, \etd (\cref{alg:etd}), and
\cref{prop:upper_main} are from \citet{foster2021statistical}, as are
\cref{ex:bandit_gaussian}, \cref{ex:bandit_structured_noise}, and
\cref{ex:full_info_lower}; The in-expectation part of
\cref{lem:log_loss_hellinger} is folklore; the high-probability part is from
\citet{foster2021statistical}, and the argument through the moment
generating function of the log-likelihood ratio goes back to
\citet{wong1995probability,Sara00,zhang2006from}. \cref{prop:dec_unconstrained_general} is Proposition 4.11 of
\citet{foster2023tight}. The constrained \CompShort and the
upper bound in \cref{prop:upper_constrained} (Theorem 3.2 there) are from
\citet{foster2023tight}; \citet{glasgow2023tight} give a single-pass
algorithm with sharper constants for bandit problems and extend the
theory to $\gamma$-regret. \cref{prop:constrained_gap} is adapted from
Example 5.1 and Proposition 4.4 of \citet{foster2023tight}.

\paragraph{Further reading}
The $f$-divergences in the refresher were
introduced by \citet{csiszar1967information} and \citet{ali1966general};
\citet{polyanskiy2014lecture} provide a comprehensive treatment of the properties of $f$-divergences. The information ratio of \citet{russo2016information} is a Bayesian
relative of the \CompShort that trades regret against information about
the optimal decision; \citet{lattimore2021mirror} develop frequentist
variants and connections to partial monitoring, and
\citet{foster2021statistical} connect the two quantities formally to the DEC.
Adversarial rewards and observations, and bounds that scale with
$\log\abs{\Pi}$ rather than $\log\abs{\cM}$ through exploration by
optimization, are treated by \citet{lattimore2020exploration} and
\citet{foster2022complexity}. Extensions of the framework include
model-free variants whose estimation cost scales with a value function
class \citep{foster2022note} (see \cref{sec:optimistic}), PAC and
reward-free versions \citep{chen2022unified}, and partially observable
settings \citep{liu2022partially,chen2022partially}.

\subsection{Deferred Proofs}

\begin{proof}[\pfref{prop:dec_unconstrained_general}]
  Fix a reference model $\Mhat$. Since $\Pi$ and $\cM$ are finite, the
  payoff in \cref{eq:dec} is bilinear in $p\in\Delta(\Pi)$ and in a
  prior $\mu\in\Delta(\cM)$ over models, and \cref{lem:sion} gives
  \begin{equation}
    \label{eq:hull_proof_sion1}
    \comp(\cM,\Mhat)
    =\sup_{\mu\in\Delta(\cM)}\inf_{p\in\Delta(\Pi)}\En_{M\sim\mu}\En_{\pi\sim{}p}\brk*{\fm(\pim)-\fm(\pi)-\gamma\cdot\Dhels{M(\pi)}{\Mhat(\pi)}}.
  \end{equation}
  Fix $\mu$, and let $\Mbar_\mu\ldef\En_{M\sim\mu}\brk{M}\in\conv(\cM)$
  be the mixture model. For every $\pi$,
  \begin{equation*}
    \En_{M\sim\mu}\brk*{\Dhels{M(\pi)}{\Mbar_\mu(\pi)}}
    \leq\En_{M,M'\sim\mu}\brk*{\Dhels{M(\pi)}{M'(\pi)}}
    \leq{}4\,\En_{M\sim\mu}\brk*{\Dhels{M(\pi)}{\Mhat(\pi)}},
  \end{equation*}
  where the first inequality is convexity of squared Hellinger distance
  in its second argument and the second is the triangle inequality for
  Hellinger distance with $\Mhat(\pi)$ as the intermediate point (both
  from \cref{sec:divergences}). Thus the information gain relative to
  $\Mhat$ at parameter $\gamma$ dominates the information gain relative to
  $\Mbar_\mu$ at parameter $\gamma/4$, and substituting into
  \cref{eq:hull_proof_sion1} gives
  \begin{align*}
    \comp(\cM,\Mhat)
    &\leq\sup_{\mu\in\Delta(\cM)}\inf_{p\in\Delta(\Pi)}\En_{M\sim\mu}\En_{\pi\sim{}p}\brk*{\fm(\pim)-\fm(\pi)-\frac{\gamma}{4}\cdot\Dhels{M(\pi)}{\Mbar_\mu(\pi)}}\\
    &\leq\sup_{\mu\in\Delta(\cM)}\comp[\gamma/4](\cM,\Mbar_\mu)
    \leq\sup_{\Mbar\in\conv(\cM)}\comp[\gamma/4](\cM,\Mbar),
  \end{align*}
  where the second inequality replaces $\mu$ in the payoff by a free
  prior and applies \cref{lem:sion} once more, with $\Mbar_\mu$ as the
  reference. Since $\Mhat$ was arbitrary, this proves
  \cref{eq:dec_unconstrained_general}.
\end{proof}

\begin{proof}[\pfref{lem:log_loss_hellinger}]
\newcommand{\indt}{\cI\ind{t}}
    We first prove the in-expectation bound. By assumption, we have that
  \[
    \sum_{t=1}^{T}\logloss\ind{t}(\Mhat\ind{t})
    -
    \sum_{t=1}^{T}\logloss\ind{t}(\Mstar)\leq{}\RegOLLog.
  \]
  Taking expectations, \pref{ass:realizability} implies that
  \[
    \sum_{t=1}^{T}\En\brk*{\Dkl{\Mstar(\pi\ind{t})}{\Mhat\ind{t}(\pi\ind{t})}}
    \leq{}\En\brk*{\RegOLLog}.
  \]
  The bound now follows from \pref{lem:divergence_inequality}.

  We now prove the high-probability bound using \cref{lem:martingale_chernoff}.
  Define
$Z_t=\frac{1}{2}(\logloss\ind{t}(\Mhat\ind{t}) -
\logloss\ind{t}(\Mstar))$. Applying
\pref{lem:martingale_chernoff} with the sequence $(-Z_t)_{t\leq{}T}$,
we are guaranteed that with probability at least $1-\delta$,
\[
  \sum_{t=1}^{T}-\log\prn*{\En_{t-1}\brk*{e^{-Z_t}}}
  \leq{} \sum_{t=1}^{T}Z_t + \log(\delta^{-1})
  = \frac{1}{2}\sum_{t=1}^{T}\prn*{\logloss\ind{t}(\Mhat\ind{t}) -
      \logloss\ind{t}(\Mstar)} + \log(\delta^{-1}).
\]
Let $t$ be fixed, and define abbreviate
$z\ind{t}=(r\ind{t},o\ind{t})$. Let $\nu(\cdot\mid{}\pi)$ be any
(conditional) dominating measure for $\densm[\Mhat\ind{t}]$ and
$\densm[\Mstar]$, and observe that
\begin{align*}
  \En_{t-1}\brk*{e^{-Z_t}\mid{}\pi\ind{t}}
  &= \En_{t-1}\brk*{
  \sqrt{\frac{\densm[\Mhat\ind{t}](z\ind{t}\mid{}\pi\ind{t})}{\densm[\Mstar](z\ind{t}\mid{}\pi\ind{t})}}\mid{}\pi\ind{t}
  }\\
    &=
\int{}\densm[\Mstar](z\mid{}\pi\ind{t})\sqrt{\frac{\densm[\Mhat\ind{t}](z\mid{}\pi\ind{t})}{\densm[\Mstar](z\mid{}\pi\ind{t})}}\dom(dz\mid{}\pi\ind{t})\\
  &=
    \int{}\sqrt{\densm[\Mstar](z\mid{}\pi\ind{t})\densm[\Mhat\ind{t}](z\mid{}\pi\ind{t})}\dom(dz\mid{}\pi\ind{t})
    = 1 - \frac{1}{2}\Dhels{\Mstar(\pi\ind{t})}{\Mhat\ind{t}(\pi\ind{t})}.
\end{align*}
Hence,
\[
  \En_{t-1}\brk*{e^{-Z_t}} = 1 - \frac{1}{2}\En_{t-1}\brk*{\Dhels{\Mstar(\pi\ind{t})}{\Mhat\ind{t}(\pi\ind{t})}}
\]
and, since $-\log(1-x)\geq{}x$ for $x\in\brk*{0,1}$, we conclude
that
\[
  \frac{1}{2}\sum_{t=1}^{T}\En_{t-1}\brk*{\Dhels{\Mstar(\pi\ind{t})}{\Mhat\ind{t}(\pi\ind{t})}}
  \leq{} \frac{1}{2}\sum_{t=1}^{T}\prn*{\logloss\ind{t}(\Mhat\ind{t}) -
      \logloss\ind{t}(\Mstar)} + \log(\delta^{-1}).
\]
\end{proof}

\begin{proof}[\pfref{prop:constrained_gap}]
  Write $\cM=\cM_\beta$. Under every model in $\cM$, the revealing
  decision $\pirev$ has reward $0$ and produces the observation $\bot$
  with probability $1-\beta$.

  \paragraph{Lower bound for the offset \CompShort}
  Since $M_0\in\cM\subseteq\conv(\cM)$, it suffices to lower bound
  $\comp(\cM,M_0)$. Fix $p\in\Delta(\Pi)$, and let $i\in[A]$ be an arm
  with $p(i)\leq{}1/A$.
  We consider two choices for the max player. If $M=M_i$, the optimal
  decision is arm $i$, and
  \[
    \En_{\pi\sim{}p}\brk*{f^{M_i}(i)-f^{M_i}(\pi)}
    = p(\pirev) + \tfrac12\prn*{1-p(\pirev)-p(i)}
    \geq{}\tfrac14.
  \]
  Since $M_i$ and $M_0$ differ only at arm $i$, where the squared
  Hellinger distance is $2$, and at $\pirev$, where it is at most
  $2\beta$, we also have
  \[
    \En_{\pi\sim{}p}\brk*{\Dhels{M_i(\pi)}{M_0(\pi)}}\leq{}\frac{2}{A}+2\beta\,p(\pirev).
  \]
  If $M=M_0$, the regret of $p$ is $\frac12p(\pirev)$ and the
  information gain is zero. Writing $x=p(\pirev)$, we conclude that
  \begin{align*}
    \sup_{M\in\cM}\En_{\pi\sim{}p}\brk*{\fm(\pim)-\fm(\pi)-\gamma\cdot\Dhels{M(\pi)}{M_0(\pi)}}
    &\geq{}\max\crl*{\tfrac14-2\gamma\beta{}x-\tfrac{2\gamma}{A},\ \tfrac{x}{2}}\\
    &\geq{}\frac{1}{4+16\gamma\beta}-\frac{2\gamma}{A},
  \end{align*}
  where the last inequality uses that the two terms in the maximum
  cross at $x=(\frac14-\frac{2\gamma}{A})/(\frac12+2\gamma\beta)$. Since
  $p$ was arbitrary, this proves the first claim.

  \paragraph{Upper bound for the constrained \CompShort}
  Fix a distribution $\nu\in\Delta(\crl{0,\ldots,A})$, let
  $\Mbar=\sum_{j}\nu_jM_j\in\conv(\cM)$ be the corresponding mixture, and
  let $i\in\argmax_{j\in[A]}\nu_j$.
  Under $\Mbar$, arm $j$ has mean reward $\frac12+\frac12\nu_j$, so
  $\pi_{\Mbar}=i$, and $\nu_j\leq\frac12$ for $j\neq{}i$. Set
  \[
    p=(1-\lambda)\cdot\mathbb{I}_{i}+\lambda\cdot\mathbb{I}_{\pirev},
    \qquad \lambda\ldef{}\frac{8\veps^2}{\beta}\leq{}1.
  \]
  For $j\neq{}i$, $M_j(\pirev)$ places no mass on observations in
  $[A]\setminus\crl{j}$, while $\Mbar(\pirev)$ places mass
  $\beta(\nu_k+\nu_0/A)$ on each $k\in[A]$. Hence
  \[
    \Dhels{M_j(\pirev)}{\Mbar(\pirev)}
    \geq{}\beta\prn*{1-\nu_j-\tfrac{\nu_0}{A}}\geq{}\frac{\beta}{4},
  \]
  using $A\geq{}4$, and consequently
  \[
    \En_{\pi\sim{}p}\brk*{\Dhels{M_j(\pi)}{\Mbar(\pi)}}\geq{}\frac{\lambda\beta}{4}=2\veps^2>\veps^2.
  \] Every $M_j$ with $j\neq{}i$ is therefore excluded
  from the constraint set in \cref{eq:dec_constrained}, which leaves
  $M_i$, $M_0$, and $\Mbar$. For each of these, arm $i$ is optimal (for
  $M_0$, every arm is optimal) and $\pirev$ has reward $0$, so the regret
  of $p$ is at most $\lambda$. This proves that
  $\compc(\cM\cup\crl{\Mbar},\Mbar)\leq{}8\veps^2/\beta$, and the second
  claim follows by taking the supremum over $\Mbar$.

  \paragraph{Rates}
  Take $\beta=T^{-1/2}$ and $A=T^2$, so that $\log\abs{\cM}\leq{}3\log{}T$.
  Consider the offset bound \cref{eq:etd_finite}. If
  $\gamma\leq{}T^{3/4}$, the first claim gives
  \[
    \comp(\cM)\cdot{}T\geq{}\frac{T}{4+16T^{1/4}}-\frac{2}{T}\geq{}\frac{T^{3/4}}{40}
  \]
  for $T$ large enough, and if $\gamma\geq{}T^{3/4}$, then
  $\gamma\log(\abs{\cM}/\delta)\geq{}T^{3/4}\log{}2$. In either case
  \cref{eq:etd_finite} is $\Omega(T^{3/4})$. For the constrained bound,
  $\bar\veps^2\leq\beta/8$ for $T$ large enough, so the second claim gives
  \[
    \compc[\bar\veps](\cM)\cdot{}T\log{}T\leq{}8C'^2\log(\abs{\cM}/\delta)\sqrt{T}\log{}T.
  \]
  Finally, $r(\veps)=8\veps^2/\beta$ satisfies the regularity condition
  \cref{eq:dec_regularity} with $C_{\mathrm{reg}}=2$, so
  \cref{prop:upper_constrained} applies with $r(\bar\veps)$ in place of
  $\compc[\bar\veps](\cM)$ and gives $\RegDM\leq\Otilde(\sqrt{T})$.
\end{proof}

\section[Lower Bounds: Necessity of the Decision-Estimation Coefficient]{Lower Bounds:\\ Necessity of the Decision-Estimation Coefficient}
\label{sec:lower}

Up to this point, we have developed algorithms and proved
upper bounds on their regret. This section asks the converse question:
for a given class of models $\cM$, what is the best regret that
\emph{any} algorithm can achieve? We show that the \CompText gives an answer to this question, lower bounding the minimax regret for any class $\cM$ of models.

\paragraph{Organization}
\cref{sec:minimax_regret} formalizes algorithms and minimax regret, and 
\cref{sec:lower_mab} proves the classical $\sqrt{AT}$ lower bound for
the multi-armed bandit by hand as an introduction to basic lower bound techniques.
\cref{sec:lower_general} proves a first lower bound based on the \CompShort as a warm-up, and \cref{sec:lower_bounds_observable_rewards} strengthens this to the main result. Finally, \cref{sec:lower_examples} instantiates
the lower bound for bandits and reinforcement learning.

\subsection{Minimax Regret and Optimality}
\label{sec:minimax_regret}

\cref{sec:minimax_intro} introduced the minimax perspective
informally, but to state lower bounds, we need to be a bit more formal about how we define algorithms and minimax regret. For all the problems considered so far, an algorithm is fully specified\footnote{This includes contextual bandits if we define $\pi$ as a mapping from context to action, as in the DMSO formulation.} by the collection $p=(p\ind{1},\ldots,p\ind{T})$, where $p\ind{t}=p\ind{t}(\cdot\mid\hist\ind{t-1})\in\Delta(\Pi)$ is the algorithm's
strategy for step $t$ (a function of the history $\hist\ind{t-1}$).  Let $\Enm{M}{p}$ denote the expectation with respect to the observation process
  $(r\ind{t},o\ind{t})\sim{}M(\pi\ind{t})$ and any randomization
  used by the algorithm, when $\Mstar=M$ is the true model.
  When the algorithm is fixed and clear from context, which is the case in every
  proof below, we drop it from the notation and write $\En\sups{M}$ and
  $\bbP\sups{M}$ for the expectation and for the law of the history
  $\hist\ind{T}$ under the model $M$.

For a model class $\cM$, we define the minimax regret
via
\begin{equation}
  \label{eq:minimax_regret}
  \MinimaxReg = \inf_{p\ind{1},\ldots,p\ind{T}}\sup_{M\in\cM}\Enm{M}{p}\brk*{\Reg(T)},
\end{equation}
where, throughout this section only, we write regret as $\Reg(T)$ to make the
dependence on $T$ explicit;  the number of rounds governs how well any
algorithm can distinguish between models, and thus it is useful to parametrize lower bounds by $T$.

\subsection{Warm-Up: Lower Bound for the Multi-Armed Bandit}
\label{sec:lower_mab}

We begin with the multi-armed bandit, the simplest setting in which to see how a lower bound on regret is proved. Concretely, consider the multi-armed bandit problem with $A$ actions, as studied in \pref{sec:mab}. To have meaningful lower bounds, we need to make sure the noise model is sufficiently confusing to the learner. To this end, we consider the Gaussian noise model, as in \pref{ex:bandit_gaussian}: 
  \[
    \cMgb = \crl*{M: M(\pi) = \cN(f(\pi), 1), f:\brk{A}\to[0,1]}.
  \]
The following result shows that no algorithm can improve on the $\sqrt{AT}$
regret of \ucb and \etd by more than a constant factor.
\begin{prop}[Lower bound for Gaussian bandits]
  \label{lem:gaussian_mab_lower_bound}
For any algorithm, there exists a model $M\in\cMgb$ where $\En\sups{M}\brk*{\Reg(T)} \approxgeq \sqrt{AT}\wedge T$. In particular, $\MinimaxR{\cMgb} \approxgeq \sqrt{AT} \wedge T$.\loose
\end{prop}
Proving a lower bound of this type requires showing that \emph{every} algorithm, including ones we have not thought of, incurs this regret on some model in the class.
However, not all the models in this collection are equally difficult for a given decision-maker. For example, the model corresponding to $f=(1,0,\ldots,0)$ is easy for an algorithm that always starts by trying the first action. The art of proving lower bounds is to find a pair or sub-family of models that (1) cannot be optimized simultaneously by any algorithm (i.e., disagree on the optimal action), yet (2) are difficulty for any algorithm to distinguish. In what follows, we sketch the idea behind \cref{lem:gaussian_mab_lower_bound}, which is to manually build such a family of models by setting the means to be nearly, but not quite the same. 

\subsubsection{Anatomy of a Lower Bound}

Let $\fhat = \Delta \cdot \boldsymbol{e}_1$, where $\Delta\in(0,1/2)$ and $\boldsymbol{e}_i$ is the $i$-th standard basis vector in $\reals^A$. Further, define $f_{j}=\fhat+2\Delta\cdot \boldsymbol{e}_j$ for $j=2,\ldots,A$. Let $\Mhat=\cN(\fhat, I\subs{A})$ and $M_{j} = \cN(f_j, I\subs{A})$ for each $j=2,\ldots,A$. The first arm is the best decision in $\Mhat$, while the $j$th arm is the best decision in model $M_j$. We define $\cM'\subseteq\cMgb$ as the coressponding sub-family of $A$ models:
\begin{align}
  \label{eq:lower_mab_models}
    \cM' = \crl*{\Mhat} \cup \crl*{M_{j}: j=2,\ldots, A}.
\end{align}
We want to show that any algorithm must incur large regret on one of
the models in $\cM'$. Fix $T\in\bbN$ and an algorithm
$p\ind{1},\ldots,p\ind{T}$. Each $p\ind{t}$ is a function of the history
$\hist\ind{t-1}$, hence a random variable whose distribution depends on
the true model. Averaging over rounds, and over the history under a
model $M$, gives the algorithm's \emph{average profile},
  \begin{align}
    \label{eq:average_profile}
    \pm\ldef\En\sups{M}\brk*{\frac{1}{T}\sum_{t=1}^{T}p\ind{t}}\in\Delta(\Pi),
  \end{align}
the distribution of the decision at a uniformly random round
when $M$ is the true model. The idea behind the proof, which recurs in
every lower bound of this section, is that both (1) the regret the
algorithm incurs and (2) the information it collects about the model
\emph{only depend on the algorithm through its average profile}. That is, for the
purpose of proving lower bounds, the profile is a sufficient statistic. The first fact we need to make precise relates regret to the average profile.\loose
\begin{lem}[Regret via the average profile]
  \label{lem:regret_average_profile}
  For any model $M$ and algorithm $(p\sups{1},\ldots,p\sups{T})$, it holds that
  \begin{align}
    \En\sups{M}\brk*{\frac{1}{T}\Reg(T)} = \En_{\pi\sim{}p\subs{M}}\brk*{\fm(\pim)-\fm(\pi)},
  \end{align}
\end{lem}
The second fact we need concerns information: how well can the algorithm distinguish two models after $T$
rounds? For this proof, we will reason about this using the KL divergence between the laws of
the history under the two models. The following lemma, the \emph{chain rule for KL divergence}, allows us to decompose the joint law of the history under the two models (a complex object to reason about) into a sum of divergences between the conditional laws over per-step observations.\loose
\begin{lem}[Chain rule for KL Divergence]
  \label{lem:kl_chain_rule}
  Let $(\cX_1,\filt_1),\ldots,(\cX_n,\filt_n)$ be a sequence of
  measurable spaces, and let $\cX\ind{i}=\prod_{t=1}^{i}\cX_t$ and
  $\filt\ind{i}=\bigotimes_{t=1}^{i}\filt_t$. For each $i$, let
  $\bbP\ind{i}(\cdot\mid{}\cdot)$ and $\bbQ\ind{i}(\cdot\mid{}\cdot)$ be probability kernels from
  $(\cX\ind{i-1},\filt\ind{i-1})$ to $(\cX_i,\filt_i)$. Let $\bbP$ and
  $\bbQ$ be
  the laws of $X_1,\ldots,X_n$ under
  $X_i\sim{}\bbP\ind{i}(\cdot\mid{}X_{1:i-1})$ and
  $X_i\sim{}\bbQ\ind{i}(\cdot\mid{}X_{1:i-1})$ respectively. Then it
  holds that
\begin{align}
  \label{eq:kl_chain_rule}
  \Dkl{\bbP}{\bbQ}
  &=
\En_{\bbP}\brk*{\sum_{i=1}^{n}\Dkl{\bbP\ind{i}(\cdot\mid{}X_{1:i-1})}{\bbQ\ind{i}(\cdot\mid{}X_{1:i-1})}}.
\end{align}
\end{lem}
Applied to the joint law $\bbPm$ over the history of an algorithm, this allows us to relate the joint KL divergence to the average profile.
\begin{cor}[Chain rule for interactive algorithms]
  \label{cor:kl_chain_rule_interactive}
  Let $M$ and $\Mhat$ be models, and fix an algorithm with average
  profile $p\subs{\Mhat}$ under $\Mhat$. Then the laws $\bbPmhat$ and
  $\bbPm$ of the history $\cH\ind{T}$ under the two models satisfy
  \begin{align*}
    \Dkl{\bbPmhat}{\bbPm}
    =
      \En\sups{\Mhat}\brk*{\sum_{t=1}^{T}\En_{\pi\ind{t}\sim{}p\ind{t}}\brk*{\Dkl{\Mhat(\pi\ind{t})}{M(\pi\ind{t})}}} 
    =  T\cdot \En_{\pi\sim p\subs{\Mhat}}\brk*{ \Dkl{ \Mhat(\pi) }{ M(\pi)} }.
  \end{align*}
\end{cor}
\begin{proof}[\pfref{cor:kl_chain_rule_interactive}]
  We apply \cref{lem:kl_chain_rule} to the
  sequence $\pi\ind{1}, z\ind{1}, \ldots, \pi\ind{T}, z\ind{T}$ with
  $z\ind{t}=(r\ind{t}, o\ind{t})$. Let us use the bold
  notation $\bz\ind{t}$ to refer to a random variable under
  consideration, and let $z\ind{t}$ refer to its realization. Then we have
    \begin{align*}
  &\Dkl{\bbPmhat}{\bbPm}\\
  &=
    \En\sups{\Mhat}\brk*{ \sum_{t=1}^{T}
		\Dkl{\bbPmhat(\bz\ind{t}|\cH\ind{t-1}, \pi\ind{t})}{\bbPm(\bz\ind{t}|\cH\ind{t-1}, \pi\ind{t})} 
	+ 
		\Dkl{\bbPmhat(\mb{\pi}\ind{t}|\cH\ind{t-1}}{\bbPm(\mb{\pi}\ind{t}|\cH\ind{t-1})}
	}  \\
 &= \En\sups{\Mhat}\brk*{ \sum_{t=1}^{T}\Dkl{\Mhat(\pi\ind{t})}{M(\pi\ind{t})}}
  \end{align*}
  since conditionally on $\cH\ind{t-1}$, the law of $\pi\ind{t}$ does not depend on the model.
\end{proof}

\subsubsection{Proof of Proposition \ref*{lem:gaussian_mab_lower_bound}}

Equipped with \cref{lem:regret_average_profile} and \cref{cor:kl_chain_rule_interactive}, we are now ready to prove \cref{lem:gaussian_mab_lower_bound}.
\begin{proof}[\pfref{lem:gaussian_mab_lower_bound}]
We choose $\Delta=C\sqrt{A/T}$ for an absolute constant $C\leq{}1/8$, and show that any algorithm must have large regret on either $\Mhat$ or $M_{j}$ for some $j\in \{2, \ldots, A\}$. Fix an algorithm $(p\sups{1},\ldots,p\sups{T})$. First, observe that if $p\subs{\Mhat}(1)\leq 1/2$, then the algorithm incurs regret at least $T\Delta/2\approxgeq \sqrt{AT}$ on $\Mhat$ by \cref{lem:regret_average_profile}, and we are done. Hence, we may assume that $p\subs{\Mhat}(1)> 1/2$. In this case, to choose a second model, we take any smallest coordinate $j$ of $p\subs{\Mhat}$, i.e., $$j=\argmin_{k\in[A]} p\subs{\Mhat}(k),$$ 
and set $M=M_{j}$. Equivalently, 
\begin{align}
\label{eq:lower_mab_2nd_model_choice}
  M_{j}=\argmax_{M\in\cM'} \En_{\pi\sim{}p\subs{\Mhat}}\brk*{\fm(\pim)-\fm(\pi)}.
\end{align}
We now show that the algorithm incurs large regret on $M$. To this end, we first note that since $p\subs{\Mhat}(1)> 1/2$, we have $p\subs{\Mhat}(j)\leq 1/(2(A-1))$. If we can show that $p\subs{M}(j)$ is also not too large, then we are done by \cref{lem:regret_average_profile}. This step, which we may colloquially refer to as change-of-measure, is present in all the lower bounds we prove in this section.
 
The simplest way to do change-of-measure is to appeal to the definition of the total variation distance. Define the set $\cE=[A]\setminus \{j\}$. We have, from \pref{eq:var_repr_tv}, the data processing inequality \pref{eq:data_processing} , and \pref{eq:divergence_inequality},\loose
\begin{align}
  \label{eq:lower_mab_change_of_measure}
  p\subs{\Mhat}(\cE) \leq p\subs{M}(\cE) + \Dtv{p\subs{M}}{p\subs{\Mhat}} \leq p\subs{M}(\cE) + \sqrt{\Dkl{\bbPmhat}{\bbPm}},
\end{align}
where we recall that $\bbPmhat$ and $\bbPm$ are the joint laws of the histories $\cH\sups{T}$ under the models $\Mhat$ and $M$, respectively. Now, by \cref{cor:kl_chain_rule_interactive}, we have 
\begin{align}
  \Dkl{\bbPmhat}{\bbPm} &= \En\sups{\Mhat}\brk*{\sum_{t=1}^{T}\Dkl{\Mhat(\pi\sups{t})}{M(\pi\sups{t})}} \notag\\
  &= T\cdot \En_{\pi\sim{}p\subs{\Mhat}}\brk*{\Dkl{\Mhat(\pi)}{M(\pi)}} \notag\\
  &= T\cdot p\subs{\Mhat}(j)\cdot\Dkl{\Mhat(j)}{M(j)} \notag\\
  &= T\cdot p\subs{\Mhat}(j)\cdot\frac{(2\Delta)^2}{2} \notag\\
  &\leq \frac{T}{2(A-1)}\cdot\frac{(2\Delta)^2}{2}. \label{eq:kl_bound_on_mab}
\end{align}
Here, the third equality uses that $\Mhat$ and $M$ differ only on
coordinate $j$, so only rounds in which arm $j$ is played contribute. The
fourth equality uses that the KL divergence between two Gaussian
distributions with variance $1$ is half the squared difference of their
means, which is $(2\Delta)^2/2$. The final inequality is where the
choice of $j$ enters: the algorithm plays arm $j$ with probability at
most $1/(2(A-1))$ under $\Mhat$, on average over rounds, so the two
models are rarely distinguished.

We now choose $\Delta=C\sqrt{A/T}$ for a constant $C\leq{}1/8$.
Then the right-hand side of \cref{eq:kl_bound_on_mab} is
$C^2A/(A-1)\leq{}2C^2$, which implies that the two models are statistically
indistinguishable by the algorithm. Indeed, by \cref{eq:lower_mab_change_of_measure},
\[
  p\subs{M}(\cE)
  \geq{}p\subs{\Mhat}(\cE)-\sqrt{\Dkl{\bbPmhat}{\bbPm}}
  \geq{}1-\frac{1}{2(A-1)}-\sqrt{2}\,C
  \geq{}\frac14,
\]
where we used $p\subs{\Mhat}(\cE)=1-p\subs{\Mhat}(j)$ and $A\geq{}2$.

Finally, every action in $\cE$ is suboptimal under $M=M_j$, by
$2\Delta$ for arms other than $1$ and by $\Delta$ for arm $1$. Hence
\cref{lem:regret_average_profile} gives
\[
  \En\sups{M}\brk*{\Reg(T)}\geq{}T\Delta\cdot{}p\subs{M}(\cE)\geq{}\frac{T\Delta}{4}=\frac{C}{4}\sqrt{AT},
\]
which is the claimed bound.
\end{proof}

The proof is specific to the Gaussian bandit only through the
choice of the sub-family $\cM'$. To isolate the general principle
underlying it, and to prepare for the theorems of the next two
subsections, let us summarize the main steps of the proof:
\begin{enumerate}
\item Pick a reference model $\Mhat$.
\item Fix any algorithm $(p\sups{1},\ldots,p\sups{T})$.
\item Compute the corresponding average profile $p\subs{\Mhat}$.
\item Pick a second model $M$ based on the average profile $p\subs{\Mhat}$ by solving \eqref{eq:lower_mab_2nd_model_choice} over a ``difficult enough'' (with respect to $\Mhat$) subset $\cM'$ of the original model class $\cMgb$.
\end{enumerate}

What does ``difficult enough'' mean in the last step? The models
in $\cM'$ must be hard to distinguish from $\Mhat$ \emph{by the
algorithm}, that is, with respect to the profile $p\subs{\Mhat}$. In
the proof, this was certified by the KL bound \cref{eq:kl_bound_on_mab},
which is small precisely when
$\En_{\pi\sim{}p\subs{\Mhat}}\brk{\Dkl{\Mhat(\pi)}{M(\pi)}}\approxleq{}1/T$:
the second model may differ from $\Mhat$ only where the algorithm
rarely looks. This suggests replacing the hand-picked sub-family by the
set of all models within a fixed information budget of $\Mhat$, which
depends on the profile $p$ under consideration,
\[
  \cM'(p)=\crl*{M\in\cM: \En_{\pi\sim{}p}\brk*{\Dkl{\Mhat(\pi)}{M(\pi)}}\leq \veps^2}
\]
for some $\veps\propto\sqrt{1/T}$. With this definition, the four steps become the
optimization problem
\begin{align}
  \label{eq:lower_mab_opt_problem}
  \sup_{\Mhat}\inf_{p\in\Delta(\Pi)}\sup_{M\in\cM'(p)} \En_{\pi\sim{}p}\brk*{\fm(\pim)-\fm(\pi)}.
\end{align}
This is precisely the maximal value of $\compc(\cM,\Mhat)$, the constrained \CompShort of \cref{eq:dec_constrained}, except that the \CompShort measures
information with Hellinger distance rather than KL divergence; we switch to the former for technical reasons explained in \cref{sec:lower_general}. That we can define this optimization over profiles $p$, rather than over algorithms, is the
payoff of the sufficient-statistic idea: regret and information both depend
on the algorithm only through $p\subs{\Mhat}$, so it is enough to
consider every profile an algorithm could produce.\loose

\begin{rem}[Connection to the modulus of continuity]
  \label{rem:modulus_of_continuity}
  The constrained \CompShort can be thought of as an interactive
  counterpart to the \emph{modulus of continuity} in statistical
  estimation \citep{donoho1987geometrizing,donoho1991geometrizingii}. In
  the classical setting, one observes $T$ i.i.d.\ samples from an unknown
  distribution $P$ belonging to a class $\cP$, and wishes to estimate a
  functional $\theta(P)$ of the distribution, such as its mean, a
  quantile, or the value of its density at a point. The modulus of
  continuity of $\theta$ over $\cP$ is
  \[
    \omega(\veps)=\sup\crl*{\abs{\theta(P)-\theta(Q)} : P,Q\in\cP,\ \Dhel{P}{Q}\leq\veps},
  \]
  the largest change in the functional across a Hellinger ball of radius
  $\veps$. It characterizes the minimax rate of estimation: the two-point
  method \citep{lecam1973convergence} shows that no estimator can achieve
  error better than $\omega(c/\sqrt{T})$, since two distributions at
  Hellinger distance $1/\sqrt{T}$ cannot be distinguished from $T$
  samples. In general the modulus gives only a lower bound, but for affine
  functionals over convex classes, under mild regularity conditions, this
  rate is also achievable \citep{polyanskiy2019dualizing}.\loose
    
  One can immediately see that the constrained \CompShort places a similar hard constraint on Hellinger distance, and the lower bounds of the next two subsections
  use the same radius $\veps\propto{}1/\sqrt{T}$. 
  The constrained \CompShort differs from the modulus of continuity only through reasons related to interaction. The functional of interest is the regret of a decision
  distribution $p$, rather than a fixed functional of the distribution,
  and Hellinger distance is measured under $p$ itself, since the
  learner's decisions determine what data it observes. Moreover, whereas
  the classical two-point method selects the pair of distributions
  obliviously, the pair $(\Mhat,M)$ in \cref{eq:lower_mab_opt_problem},
  as in the proof of \cref{lem:gaussian_mab_lower_bound}, is chosen
  adversarially based on the algorithm. The latter is necessary for the number of decisions to
  appear in the lower bound, even for multi-armed bandits. 
\end{rem}

\subsection{A First Lower Bound via the Constrained \CompText}
\label{sec:lower_general}

\newcommand{\Loss}{\mathsf{L}}
\newcommand{\ueps}{\underline{\varepsilon}}
\newcommand{\compcsq}[1][\veps]{\mathsf{dec}^{\mathrm{c},\mathsf{sq}}_{#1}}

This subsection turns the recipe of \cref{sec:lower_mab} into a general lower bound based on the constrained DEC. Throughout this subsection, we take $\Loss:\cM\times\Pi\to[0,1]$ to be an
arbitrary function, and define $\compc(\cM,\Mhat)$ as the constrained
\CompShort of \cref{eq:dec_constrained} with $\Loss(M,\pi)$ in place of
$\fm(\pim)-\fm(\pi)$. Our motivating case of interest is
\begin{align}
  \label{eq:loss_observable_rewards}
  \Loss(M,\pi) = \fm(\pim)-\fm(\pi),
\end{align}
for which $\sum_{t=1}^{T}\Loss(M,\pi\ind{t})$ is the regret and
$\compc(\cM,\Mhat)$ is the constrained \CompShort as defined in \cref{sec:constrained_comp}. The result we prove is a warm-up for the main theorem of
\cref{sec:lower_bounds_observable_rewards} in one respect: an additional side
condition ties the radius to the value of the \CompShort, which weakens
the rate; we remove this limitation in \cref{sec:lower_bounds_observable_rewards}

Before getting to the proof, we state a Hellinger analogue of
\cref{cor:kl_chain_rule_interactive}. In \cref{sec:lower_mab} we measured
the closeness of two models with KL divergence, which was sufficient
there and has an exact, elementary chain rule. From here on we use
Hellinger distance instead. Even though KL divergence is well-behaved for some classes, in general it can be much larger than Hellinger distance, and is infinite when the supports differ, so a lower bound through a KL ball would be weaker and potentially fall short of the DEC-based upper bound in \cref{prop:upper_constrained}. The price is that
squared Hellinger distance does not satisfy an exact chain rule. It is,
however, sub-additive over rounds up to a constant factor, which suffices for our purposes \citep[Lemma D.2]{foster2024online} (see also \citet{jayram2009hellinger} and \citet[Lemma 11.5.3]{duchi-notes}).
\begin{lem}[Sub-additivity for squared Hellinger distance]
\label{lem:Hellinger-chain}
  
  With the notation of \cref{lem:kl_chain_rule},
  \begin{align}
    \label{eq:hellinger_chain_rule}
    \Dhels{\bbP}{\bbQ}\leq 
  7~\En_{\bbP}\brk*{\sum_{i=1}^{n}\Dhels{\bbP\ind{i}(\cdot\mid X_{1:i-1})}{\bbQ\ind{i}(\cdot\mid X_{1:i-1})}}.
  \end{align}
\end{lem}
Applying this to the history as in the proof of
\cref{cor:kl_chain_rule_interactive} gives the following bound, in
which the information collected by the algorithm again depends on it
only through its profile.\loose
\begin{cor}[Hellinger distance for $T$-round algorithms]
  \label{cor:hellinger_distance}
  Let $M$ and $\Mhat$ be models, and fix an algorithm with average
  profile $p\subs{\Mhat}$ under $\Mhat$. Then
  \begin{align*}
      \Dtvs{\bbP\sups{M}}{\bbP\sups{\Mhat} }
      \leq \Dhels{ \bbP\sups{M}}{\bbP\sups{\Mhat} }
      \leq 
      7T\cdot\En_{\pi\sim p\subs{\Mhat}}\brk*{ \Dhels{ M(\pi)}{ \Mhat(\pi) } }.
  \end{align*}
\end{cor}

We can now state the first lower bound. It is expressed through
the supremum of the constrained \CompShort $\compc[\ueps](\cM,\Mhat)$ over \emph{arbitrary}
reference models $\Mhat:\Pi\to\Delta(\cR\times\cO)$, rather than
$\Mhat\in\conv(\cM)$ as in \cref{eq:dec_constrained_max}. The alternatives
$M$ in the definition still range over $\cM$, since the true model must lie in the class; this is why the reference model is not added to the class as in \cref{eq:dec_constrained_max}. Recall that we state the result for
an arbitrary bounded function $\Loss(M,\pi)$ of the model and the
decision, not only for the regret.\footnote{The proof uses nothing about the
form of $\Loss$ beyond $\Loss\in[0,1]$; in particular, it places no
assumption on the reward realizations, which may be unbounded (as an
aside, it applies even in settings where rewards are not observed at all).}
\begin{prop}[First lower bound via the constrained \CompShort]
    \label{prop:lower_bound_simple}
  Let \( \Loss:\cM\times\Pi \to [0,1] \) be arbitrary, and define $\Reg = \sum_{t=1}^T \Loss(M, \pi\sups{t})$. Let any \( \ueps>0 \) satisfy
  \begin{equation}
    \label{eq:cdec_fixed_point}
    T\ueps^2 \le \frac{1}{56}\cdot\sup_{\Mhat}\compc[\ueps](\cM,\Mhat),
  \end{equation}
  where the supremum ranges over all models $\Mhat:\Pi\to\Delta(\cR\times\cO)$. Then every algorithm satisfies
\[
\sup_{M\in\cM}\En\sups{M}\brk*{\frac{1}{T}\Reg} \ge \frac{1}{6} \, \sup_{\Mhat}\compc[\ueps](\cM,\Mhat).
\]
\end{prop}
Let us interpret the condition \cref{eq:cdec_fixed_point}, which is the main feature of \cref{prop:lower_bound_simple} that leaves room for improvement. For the Gaussian bandit, where
$\compc(\cM,\Mhat)\approx\veps\sqrt{A}$ (see \cref{sec:lower_examples}),
the side condition $T\ueps^2\leq\ueps\sqrt{A}/56$ forces
$\ueps\approxleq\sqrt{A}/T$, and the proposition yields only
$\En\brk{\Reg(T)}\approxgeq{}T\cdot\ueps\sqrt{A}\approx{}A$. While this correctly shows that the lower bound must scale with the number of actions, it does not grow with $T$, whereas the true minimax regret is of order $\sqrt{AT}$. The main theorem of the next subsection removes the side condition and recovers the $\sqrt{AT}$ rate. To do so, it specializes $\Loss$ to the regret \cref{eq:loss_observable_rewards} and assumes that the reward realizations lie in $[0,1]$; as we explain there, this is what permits a stronger change-of-measure argument.

The proof is a change-of-measure argument in the style of
\cref{sec:lower_mab}. Note that the reference model $\Mhat$ need not belong
to $\cM$: it serves only as the measure under which the algorithm's behavior
is analyzed, and nature ultimately plays a model $M\in\cM$. Lower bounds of
this form, with arbitrary reference models, appear in \citet{foster2023multiagent}.
\begin{proof}[\pfref{prop:lower_bound_simple}]
We follow the template of \pref{sec:lower_mab}. Write $D=\sup_{\Mhat}\compc[\ueps](\cM,\Mhat)$, and choose $\Mhat:\Pi\to\Delta(\cR\times\cO)$ with $\compc[\ueps](\cM,\Mhat)=D$; if the supremum is not attained, an approximate maximizer suffices, at the cost of an arbitrarily small loss in the constants below. Fix an arbitrary algorithm and let $p\subs{\Mhat}$ be its average profile under $\Mhat$. Choose $M\in\cM$ as a best response to this profile in the constrained \CompShort, so that
\[
\En_{\pi\sim p\subs{\Mhat}}\brk*{\Loss(M,\pi)}\geq D,
\qquad
\En_{\pi\sim p\subs{\Mhat}}\brk*{\Dhels{M(\pi)}{\Mhat(\pi)}}\leq\ueps^2.
\]
The second condition says that $M$ is hard to distinguish from
$\Mhat$: by \cref{cor:hellinger_distance} and the side condition \cref{eq:cdec_fixed_point} on
$\ueps$,
		\begin{align*}
			\Dhels{\bbP\sups{M}}{\bbP\sups{\Mhat}}\leq 7 T\cdot\En_{\pi\sim p\subs{\Mhat}}\brk*{\Dhels{M(\pi)}{\Mhat(\pi)}} \leq 7 T\cdot\ueps^2 \leq \frac{D}{8}.
		\end{align*}
The first condition says that the profile $p\subs{\Mhat}$ has
large regret under $M$; it remains to transfer this to the profile
$p\subs{M}$, which is what the algorithm actually plays under $M$. Since
$p\subs{M}$ and $p\subs{\Mhat}$ are the laws of the decision at a random
round under the two models, the data processing inequality gives
$\Dhels{p\subs{M}}{p\subs{\Mhat}}\leq\Dhels{\bbP\sups{M}}{\bbP\sups{\Mhat}}\leq{}D/8$,
and the change-of-measure inequality for bounded functions
(\cref{lem:mp_min} in \cref{sec:change_of_measure}) yields
	\begin{align*}
		\En_{\pi\sim p\subs{\Mhat}}\brk*{\Loss(M,\pi)} \leq 3\En_{\pi\sim p\subs{M}}\brk*{\Loss(M,\pi)} + 4\Dhels{p\subs{M}}{p\subs{\Mhat}} \leq 3\En_{\pi\sim p\subs{M}}\brk*{\Loss(M,\pi)} + \frac{D}{2}.
	\end{align*}
	Hence, by \cref{lem:regret_average_profile},
	$$\En\sups{M}\brk*{\frac{1}{T}\Reg} = \En_{\pi\sim p\subs{M}}\brk*{\Loss(M,\pi)} \geq \frac{D}{6},$$
which proves the claim.
\end{proof}

\subsection{Main Lower Bound for the Constrained \CompText}
\label{sec:lower_bounds_observable_rewards}

\newcommand{\alg}{\textsf{alg}}

We now state the main result of this section: a lower bound
in terms of $\compc(\cM)$ at radius $\ueps\propto{}T^{-1/2}$, which qualitatively matches the upper bounds in preceding sections.
\begin{theorem}[Lower bound via the constrained DEC]\label{thm:decc-lower-general}
  Let $T\geq 1$, and let
  $\cM\subseteq\bigl(\Pi\to\Delta([0,1]\times\cO)\bigr)$ be a model
  class with rewards in $[0,1]$. Define
  $\Loss(M,\pi)=\fm(\pim)-\fm(\pi)$. Then, for any $T$-round
  algorithm,
  \begin{align*}
    \sup_{M\in\cM}\En\sups{M}\brk*{\frac{\Reg(T)}{T}}
    \geq \frac18\prn*{
      \compc[\ueps](\cM)-6\ueps-\frac1T
    },
    \qquad
    \ueps=\frac{1}{24\sqrt T}.
  \end{align*}
\end{theorem}

It is instructive to compare \cref{thm:decc-lower-general} to the upper bounds of \cref{sec:general_dm}. The lower bound gives 
$\En\brk{\Reg(T)}\approxgeq\compc[\ueps](\cM)\cdot{}T$ with
$\ueps\propto{}T^{-1/2}$, while the constrained-\CompShort variant of \etd
in \cref{prop:upper_constrained} achieves
$\Reg(T)\approxleq\compc[\bar\veps](\cM)\cdot{}T\log{}T+\sqrt{T\log(\abs{\cM}/\delta)}$
with $\bar\veps\propto\sqrt{\log(\abs{\cM}/\delta)/T}$. The two radii differ
by a factor of $\sqrt{\log\abs{\cM}}$, so up to this factor and the
estimation term, the upper and lower bounds coincide. In other words, the
exploration strategy prescribed by the \CompShort is the best possible;
whatever is lost relative to the minimax regret is lost in estimation,
not in exploration.

Compared with \cref{prop:lower_bound_simple}, the theorem is
specialized to the regret, $\Loss(M,\pi)=\fm(\pim)-\fm(\pi)$, and assumes
that the reward realizations lie in $[0,1]$. This specialization,
together with boundedness (more generally, the condition in
\cref{rem:decc_unbounded_rewards}), facilitates a stronger
change-of-measure argument. The proof otherwise follows the template of
\cref{sec:lower_mab}, with one new device: the average profile is
computed only up to a stopping time, the first round at which the
algorithm has accumulated a prescribed amount of regret under $\Mhat$,
so that the regret under $\Mhat$ and under the second model $M$ can be
compared on the same event.

\begin{proof}[\pfref{thm:decc-lower-general}]
Fix any reference model $\Mhat\in\conv(\cM)$ and an algorithm.
Set
\begin{align*}
  R=\sup_{M\in\cM}\En\sups{M}\brk*{\frac{\Reg(T)}{T}},
  \qquad \veps=\frac{1}{24\sqrt T}.
\end{align*}
We prove
\begin{align}
  \label{eq:decc-lower-converse}
  \compc[\veps](\cM)\leq 8R+6\veps+\frac1T.
\end{align}
Since all losses are at most one, the claim is immediate when
$8R+6\veps+1/T\geq1$. We may therefore assume the reverse inequality.

Let $\pim[\Mhat]$ maximize $\fm[\Mhat]$. For a sufficiently small
$\eta>0$, put
\begin{align*}
  a=4R+2\veps+\eta<1.
\end{align*}
Run the algorithm under $\Mhat$, and write
\begin{align*}
  g_t
  &=\fm[\Mhat](\pim[\Mhat])
    -\En_{\pi\sim p\ind{t}(\cdot\mid\cH\ind{t-1})}\brk*{\fm[\Mhat](\pi)},
  &
  G_t&=\sum_{s=1}^t g_s.
\end{align*}
Let $\tau$ be the first $t\leq T$ for which $G_t\geq Ta$, with
$\tau=T$ if the threshold is never reached. Define the stopped-and-padded
average profile
\begin{align}
  \label{eq:decc-lower-profile}
  p\subs{\Mhat}
  =\frac1T\En\sups{\Mhat}\brk*{
      \sum_{t=1}^{\tau}p\ind{t}(\cdot\mid\cH\ind{t-1})
      +(T-\tau)\delta_{\pim[\Mhat]}
    }.
\end{align}
Because $0\leq g_t\leq1$, this profile has reference loss
\begin{align}
  \label{eq:decc-lower-reference-loss}
  \En_{\pi\sim p\subs{\Mhat}}
    \brk*{\fm[\Mhat](\pim[\Mhat])-\fm[\Mhat](\pi)}
  =\frac1T\En\sups{\Mhat}\brk*{G_\tau}
  <a+\frac1T.
\end{align}

For this particular profile, let
\begin{align}
  \label{eq:decc-lower-best-response}
  B=\sup_{\substack{N\in\cM\cup\{\Mhat\}\\
      \En_{\pi\sim p\subs{\Mhat}}
        \brk*{\Dhels{N(\pi)}{\Mhat(\pi)}}\leq\veps^2}}
      \En_{\pi\sim p\subs{\Mhat}}
        \brk*{f^N(\pi_N)-f^N(\pi)}.
\end{align}
Since $p\subs{\Mhat}$ is one admissible choice in the DEC infimum,
\begin{align}
  \label{eq:decc-lower-dec-vs-best-response}
  \compc[\veps](\cM\cup\{\Mhat\},\Mhat)\leq B.
\end{align}
We claim that
\begin{align}
  \label{eq:decc-lower-certificate}
  B\leq 2a+\sqrt2\,\veps+\frac1T.
\end{align}

Suppose otherwise. Choose a model $M$ attaining the supremum in
\eqref{eq:decc-lower-best-response}; if it is not attained, choose a
sufficiently accurate approximate maximizer. Denote its objective value by
$b$, so that
\begin{align}
  \label{eq:decc-lower-large-loss}
  b>2a+\sqrt2\,\veps+\frac1T.
\end{align}
The choice $M=\Mhat$ has objective value less than $a+1/T$ by
\eqref{eq:decc-lower-reference-loss}; hence the selected model belongs to
$\cM$. It also satisfies
\begin{align}
  \label{eq:decc-lower-feasible}
  \En_{\pi\sim p\subs{\Mhat}}
    \brk*{\Dhels{M(\pi)}{\Mhat(\pi)}}\leq\veps^2.
\end{align}
Since the reward is observed and belongs to $[0,1]$,
\begin{align*}
  \abs{\fm(\pi)-\fm[\Mhat](\pi)}^2
  \leq 2\Dhels{M(\pi)}{\Mhat(\pi)}.
\end{align*}
Thus, \eqref{eq:decc-lower-feasible} and Cauchy--Schwarz imply
\begin{align*}
  \En_{\pi\sim p\subs{\Mhat}}
    \brk*{\abs{\fm(\pi)-\fm[\Mhat](\pi)}}
  \leq\sqrt2\,\veps.
\end{align*}
Together with \eqref{eq:decc-lower-reference-loss} and
\eqref{eq:decc-lower-large-loss}, this gives
\begin{align}
  \fm(\pim)-\fm[\Mhat](\pim[\Mhat])
  &=b+\En_{\pi\sim p\subs{\Mhat}}
      \brk*{\fm(\pi)-\fm[\Mhat](\pi)} \notag\\
  &\quad-\En_{\pi\sim p\subs{\Mhat}}
      \brk*{\fm[\Mhat](\pim[\Mhat])-\fm[\Mhat](\pi)}
  >a. \label{eq:decc-lower-gap}
\end{align}

Now run the original algorithm under $M$, defining $G_t$ and $\tau$ from
its realized history by the same formulas. Let
\begin{align*}
  X=\sum_{t=1}^{\tau}
    \En_{\pi\sim p\ind{t}(\cdot\mid\cH\ind{t-1})}
      \brk*{\abs{\fm(\pi)-\fm[\Mhat](\pi)}}.
\end{align*}
For a realized history, write the conditional pseudo-regret through time
$s$ as
\begin{align*}
  \overline{\Reg}_M(s)
  =\sum_{t=1}^s\brk*{
    \fm(\pim)-\En_{\pi\sim p\ind{t}(\cdot\mid\cH\ind{t-1})}\brk*{\fm(\pi)}
  }.
\end{align*}
For every stopped history, \eqref{eq:decc-lower-gap} implies
\begin{align}
  \label{eq:decc-lower-barrier}
  \overline{\Reg}_M(\tau)
  \geq \tau a+G_\tau-X
  \geq Ta-X.
\end{align}
Indeed, if $\tau=T$, the first term is at least $Ta$; if $\tau<T$, the
threshold term $G_\tau$ is at least $Ta$.

Under $\Mhat$, feasibility and the definition of $p\subs{\Mhat}$ give
\begin{align}
  \label{eq:decc-lower-x-bound}
  \En\sups{\Mhat}\brk*{X^2}
  &\leq T\En\sups{\Mhat}\brk*{
    \sum_{t=1}^{\tau}
      \En_{\pi\sim p\ind{t}(\cdot\mid\cH\ind{t-1})}
      \brk*{\abs{\fm(\pi)-\fm[\Mhat](\pi)}^2}} \notag\\
  &\leq 2T^2\veps^2.
\end{align}
Consequently, $X\leq2T\veps$ with probability at least $1/2$ under
$\Mhat$. Padding each stopped transcript with a common cemetery symbol and
applying the Hellinger chain rule (\cref{cor:hellinger_distance}) gives
\begin{align*}
  \Dhels{\bbP\sups{\Mhat}_{\tau}}
         {\bbP\sups{M}_{\tau}}
  &\leq 7T\veps^2,
  &
  \Dtv{\bbP\sups{\Mhat}_{\tau}}
       {\bbP\sups{M}_{\tau}}
  &\leq\sqrt{14T\veps^2}<\frac14.
\end{align*}
Thus, $X\leq2T\veps$ with probability at least $1/4$ under $M$. Since
conditional instantaneous regret is nonnegative,
\begin{align*}
  \En\sups{M}\brk*{\frac{\Reg(T)}{T}}
  &=\En\sups{M}\brk*{\frac{\overline{\Reg}_M(T)}{T}}\\
  &\geq\En\sups{M}\brk*{\frac{\overline{\Reg}_M(\tau)}{T}}
  \geq\frac14(a-2\veps)
  =R+\frac\eta4,
\end{align*}
contradicting the definition of $R$. This proves
\eqref{eq:decc-lower-certificate}.

Combining \eqref{eq:decc-lower-dec-vs-best-response} and
\eqref{eq:decc-lower-certificate}, and then letting $\eta\downarrow0$, gives
\begin{align*}
  \compc[\veps](\cM\cup\{\Mhat\},\Mhat)
  \leq 8R+(4+\sqrt2)\veps+\frac1T
  \leq 8R+6\veps+\frac1T.
\end{align*}
Taking the supremum over $\Mhat\in\conv(\cM)$ proves
\eqref{eq:decc-lower-converse}. Rearranging completes the proof.
\end{proof}

The theorem is stated for rewards in $[0,1]$; the following
remark extends general reward distributions with bounded means, which covers the Gaussian examples below.
\begin{rem}[Unbounded rewards]
  \label{rem:decc_unbounded_rewards}
  The conclusion of \pref{thm:decc-lower-general} extends, up to
  absolute constants, to unbounded rewards provided that the mean rewards
  lie in $[0,1]$ and, for some absolute constant $c$, all $M\in\cM$,
  $\Mbar\in\conv(\cM)$, and $\pi\in\Pi$ satisfy
  \[
    \abs{\fm(\pi)-\fm[\Mbar](\pi)}^2
    \leq c\cdot\Dhels{M(\pi)}{\Mbar(\pi)}.
  \]
  Indeed, these are the only consequences of bounded rewards used in the
  proof. The numerical coefficient of $\ueps$ may depend on $c$.
  This includes unit-variance Gaussian rewards with means in $[0,1]$,
  including their mixtures as reference models. Indeed, a standard Cauchy--Schwarz argument gives
  \[
    \abs{\En_{P}\brk{r}-\En_{Q}\brk{r}}^2
    \leq 2\prn*{\En_{P}\brk{r^2}+\En_{Q}\brk{r^2}}\cdot
      \Dhels{P}{Q},
  \]
  and the second moments are at most two, so $c=8$ suffices.
\end{rem}

\subsection{Examples}
\label{sec:lower_examples}

We now instantiate the lower bound in
\pref{thm:decc-lower-general} for concrete model classes. We begin by revisiting examples from \cref{sec:dec_general}.
Each example lower bounds $\compc(\cM,\Mhat)$ for a single,
conveniently chosen reference model $\Mhat\in\conv(\cM)$, which suffices
since $\compc(\cM,\Mhat)\leq\compc(\cM)$ by \cref{eq:dec_constrained_max}.\loose
\begin{examplecont}[Multi-armed bandit with Gaussian rewards]{ex:bandit_gaussian}
  Let us lower bound the constrained \CompShort for the Gaussian
  bandit problem from \pref{ex:bandit_gaussian}. Set
  $\Mhat(\pi)=\cN(1/2,1)$, and let $\crl{M_1,\ldots,M_A}\subseteq\cM$
  be a sub-family of models with $M_i(\pi)=\cN(\fmi(\pi),1)$, where
  $\fmi(\pi)\ldef{}\frac{1}{2}+\Delta\indic{\pi=i}$ for a
  parameter $\Delta$ to be chosen below. Observe that for all
  $i$,
  $\En_{\pi\sim{}p}\brk*{\Dhels{M_i(\pi)}{\Mhat(\pi)}}\leq\frac{1}{2}\Delta^2p(i)$
  by \pref{eq:hellinger_gaussian_ub}, and
  $\En_{\pi\sim{}p}\brk*{\fmi(\pimi)-\fmi(\pi)}=(1-p(i))\Delta$, so we
  can lower bound
  \begin{align*}
    \compc(\cM,\Mhat) &= \inf_{p\in\Delta(\Pi)}\sup_{M\in\cM}\crl*{\En_{\pi\sim{}p}\brk*{
  \fm(\pim) - \fm(\pi)}
  \mid\En_{\pi\sim{}p}\brk*{\Dhels{M(\pi)}{\Mhat(\pi)}}\leq\veps^2
                        } \\
    &\geq \inf_{p\in\Delta(\Pi)}\max_{i}\crl*{ (1-p(i))\Delta
  \mid p(i)\frac{\Delta^2}{2}\leq\veps^2
  }
  \end{align*}
For any $p$, there exists $i$ such that
$p(i)\leq{}1/A$. If we choose $\Delta=\veps\cdot{}\sqrt{2A}$, this
choice for $i$ satisfies the constraint
$p(i)\frac{\Delta^2}{2}\leq\veps^2$, and we are left with
\begin{align*}
  \compc(\cM,\Mhat) \geq{} (1-p(i))\Delta \geq{}\veps\sqrt{A/2},
\end{align*}
since $1-p(i)\geq{}1/2$.

Plugging this lower bound on the constrained \CompText into \pref{thm:decc-lower-general}, with the extension in \pref{rem:decc_unbounded_rewards}, yields
    \[
      \En\brk*{\Reg} \geq{}
      \bigom(\sqrt{AT}).
  \]
  This recovers \cref{lem:gaussian_mab_lower_bound}, with the \CompShort automatically selecting the hard alternative model.
\end{examplecont}

Generalizing the argument above, the following proposition gives a useful recipe for lower bounding the \CompText for any model class that ``embeds'' a family of models with properties similar to the multi-armed bandit problem.
\begin{prop}[DEC lower bound for hard families]
  \label{prop:hard_family}
Let a reference model $\Mhat$ be given, and suppose that a class $\cM$
contains a sub-class $\crl{M_1,\ldots,M_N}$ and collection of decisions
  $\pi_1,\ldots,\pi_N$ with the property that for
  all $i$:
  \begin{enumerate}
  \item
    $\Dhels{M_i(\pi)}{\Mhat(\pi)}\leq\beta^2\cdot\indic{\pi=\pi_i}$.
  \item $\fmi(\pimi)-\fmi(\pi)\geq{}\alpha\cdot\indic{\pi\neq\pi_i}$.
  \end{enumerate}
  Then
  \[
    \compc(\cM,\Mhat)\approxgeq{}\alpha\cdot\indic{\veps\geq{}\beta/\sqrt{N}}.
  \]
\end{prop}
The examples that follow can be obtained by applying this result with
an appropriate sub-family; see \citet{foster2021statistical,foster2023tight} for detailed proofs.

\begin{examplecont}[Bandits with structured noise]{ex:bandit_structured_noise}
  Recall the bandit-type problem with structured noise from
  \pref{ex:bandit_structured_noise}, where we have
  $\cMsn=\crl{M_1,\ldots,M_A}\cup\crl{M_0}$, with
  $M_i(\pi)=\cN(1/2,1)\indic{\pi\neq{}i}+\Ber(3/4)\indic{\pi=i}$
  and $M_0(\pi)=\cN(1/2,1)$ for every $\pi$. With $M_0$ as the
  reference model, the sub-family $\crl{M_1,\ldots,M_A}$ satisfies the
  conditions of \pref{prop:hard_family} with $\alpha=1/4$ and
  $\beta^2=2$. As a result, we have
  $\compc(\cMsn)\approxgeq{}\indic{\veps\geq{}\sqrt{2/A}}$, which yields
    \[
      \En\brk*{\Reg} \approxgeq{} \bigomt(A)
    \]
    if we apply \cref{thm:decc-lower-general}, with the extension in
    \pref{rem:decc_unbounded_rewards}.
  
\end{examplecont}

\begin{examplecont}[Bandits with full information]{ex:full_info_lower}
  Consider the full-information variant of the bandit
  setting in \pref{ex:full_info_lower}. By adapting the argument in
  \pref{ex:bandit_gaussian}, one can show that
  \begin{align*}
    \compc(\cM)\approxgeq \veps,
  \end{align*}
  which leads to a lower bound of the form
  \[
    \En\brk*{\Reg} \approxgeq \sqrt{T}.
  \]
\end{examplecont}

Next, we revisit some of the structured bandit classes considered in
\pref{sec:structured}. For these, it is convenient to have a
version of the lower bound stated directly in terms of the reward
function class $\cF$, in the language of \cref{sec:structured}. For
$\fhat:\Pi\to\bbR$, define the square-loss constrained \CompShort
\[
  \compcsq(\cF,\fhat)=
  \inf_{p\in\Delta(\Pi)}\sup_{f\in\cF}\crl*{\En_{\pi\sim{}p}\brk*{f(\pi_f)-f(\pi)}
  \mid\En_{\pi\sim{}p}\brk*{(f(\pi)-\fhat(\pi))^2}\leq\veps^2},
\]
the analogue of \cref{eq:dec_constrained} with squared error in place of
Hellinger distance.
\begin{cor}[Lower bound for structured bandits]
  \label{cor:lower_structured}
  Let $\cF\subseteq(\Pi\to[0,1])$, and let
  $\cM_{\cF}=\crl{M_f:f\in\cF}$ with $M_f(\pi)=\cN(f(\pi),1)$ be the induced
  structured bandit class with unit-variance Gaussian rewards. Then for
  all $\veps>0$,
  \[
    \compc(\cM_{\cF})\geq\sup_{\fhat\in\cF}\compcsq[\sqrt{2}\veps](\cF,\fhat),
  \]
  and consequently every algorithm has
  $\En\brk{\Reg(T)}\approxgeq{}T\cdot\sup_{\fhat\in\cF}\compcsq[c/\sqrt{T}](\cF,\fhat)-\sqrt{T}$
  for an absolute constant $c>0$.
\end{cor}
We omit a formal proof, but briefly, the first claim holds because, by
\cref{eq:hellinger_gaussian_ub},
$\Dhels{M_f(\pi)}{M_{\fhat}(\pi)}\leq\frac12(f(\pi)-\fhat(\pi))^2$, so
the square-loss ball of radius $\sqrt2\veps$ around $\fhat$ is contained
in the Hellinger ball of radius $\veps$ around $M_{\fhat}\in\cM_{\cF}$,
and enlarging the max player's feasible set can only increase the value.

The examples that follow are instances
of this corollary.

\begin{example}[Linear bandits]
  \label{ex:linear_lower}
  Consider the linear bandit setting in \pref{sec:linear}, with $\cF=\crl*{\pi\mapsto\tri{\theta,\phi(\pi)}\mid{}\theta\in\Theta}$,
where $\Theta\subseteq\sB_2^{d}(1)$ is a parameter set and 
$\phi:\Pi\to\bbR^{d}$ is a fixed feature map that is known to the
learner. Let $\cM$ be the set of all reward distributions with
$\fm\in\cF$ and $1$-\subgaussian noise. Then
\[
  \compc(\cM)\approxgeq \veps\sqrt{d},
\]
which gives
\[
\En\brk*{\Reg}\approxgeq\sqrt{dT}.
\]
\end{example}

\begin{example}[Nonparametric bandits]
  \label{ex:nonparametric_lower}
  Consider the Lipschitz bandit setting in \pref{sec:nonparametric},
  where $\Act$ is a metric space with metric $\met$, and
\[
\cF = \crl*{f:\Act\to\brk{0,1} \mid{} \text{$f$ is $1$-Lipschitz w.r.t $\met$}}.
\]
 Let $\cM$ be the set of all reward distributions with
 $\fm\in\cF$ and $1$-\subgaussian noise. Let $d>0$ be such that the
 covering number for $\Pi$ satisfies
 \[
   \Mcov(\Act,\veps) \geq \veps^{-d}.
   \]
 Then
\[
  \compc(\cM)\approxgeq \veps^{\frac{2}{d+2}},
\]
which leads to $\En\brk*{\Reg}\approxgeq T^{\frac{d+1}{d+2}}$.
\end{example}

Each of these lower bounds is matched, up to logarithmic factors,
by an upper bound proved earlier (\cref{prop:dec_linear} for linear
bandits and \cref{prop:dec_lipschitz} for Lipschitz bandits), both via
\etd with the corresponding offset \CompShort. In both cases the offset
and constrained coefficients have the same order, which is why the lower
bounds are proven through the latter and the upper bounds are proven through the
former without loss.

\paragraph{Tabular reinforcement learning}
Our final example is the class of tabular MDPs from
\cref{sec:mdp}, for which \cref{sec:dec_tabular} gave the upper bound
$\Reg\approxleq\sqrt{\mathrm{poly}(H,S,A)\cdot{}T}$ via \pcigw. We state,
but do not prove, a complementary lower bound on the constrained
\CompShort \citep{foster2021statistical,foster2023tight}.
\begin{prop}[DEC lower bound for tabular RL]
  \label{prop:lower_tabular}
  Let $\cM$ be the class of tabular MDPs with $S\geq{}2$ states, $A\geq{}2$ actions,
  and $\sum_{h=1}^{H}r_h\in\cR\ldef\brk*{0,1}$. If
  $H\geq{}2\log_2(S/2)$, then
\[
  \compc(\cM) \approxgeq \veps\sqrt{HSA}.
\]
\end{prop}
Using \pref{thm:decc-lower-general}, this gives $\En\brk{\Reg}\approxgeq\sqrt{HSAT}$, matching the upper bound of \cref{sec:dec_tabular} up to polynomial factors in $H$.

\subsection{Tighter Regret Bounds for the Decision-Estimation Coefficient\bonus}
\label{sec:lower_tighter}

The main gap between the most general upper and lower bounds based
on the \CompText concerns estimation. All of the upper bounds we have given pay an estimation
cost, represented by $\log\abs{\cM}$ for a finite model class, that is
absent from the DEC lower bound. Each of the bounds we have presented can
be tight, depending on the model class, but sharper guarantees are possible
many special cases of interest. In this optional epilogue, we present
some general techniques for obtaining such refinements by tailoring estimation
to the decision-making task. Obtaining a complete understanding of when and how the estimation gap can be closed remains an active area of research; see, for example,
\citet{chen2024assouad}.

\subsubsection{General Divergences and Randomized Estimators}
\label{sec:general_divergence}

In this section we give a generalization of the \etd algorithm
that incorporates two extra features: \emph{general divergences} and
\emph{randomized estimators}. 

\paragraph{General divergences}
  The \CompText measures estimation error via the Hellinger distance
  $\DhelsX{\big}{M(\pi)}{\Mhat(\pi)}$, which is fundamental in the
  sense that it leads to lower bounds on the optimal regret
  (\cref{thm:decc-lower-general}). Nonetheless, for specific applications and model classes,
it can be useful to work with alternative distance measures and divergences. For a
non-negative function (``divergence'') $\Dgenpi{\cdot}{\cdot}$, we define
\begin{equation}
  \label{eq:comp_general}
  \compgen(\cM,\Mhat) =
    \inf_{p\in\Delta(\Act)}\sup_{M\in\cM}\En_{\pi\sim{}p}\biggl[\fm(\pim)-\fm(\pi)
    -\gamma\cdot\Dgenpi{\Mhat}{M}
    \biggr].
  \end{equation}
  This variant of the \CompShort naturally leads to regret bounds in
  terms of estimation error under $\Dgenpi{\cdot}{\cdot}$. Note that
  we use notation $\Dgenpi{\Mhat}{M}$ instead of say,
  $\DgenX{\big}{\Mhat(\pi)}{M(\pi)}$, to reflect that fact that the divergence
  may depend on $M$ (resp. $\Mhat$) and $\pi$ through properties other
  than $M(\pi)$ (resp. $\Mhat(\pi)$).

\paragraph{Randomized estimators}
The basic version of \etd assumes that at each
round, the online estimation oracle provides a point estimate
$\Mhat\ind{t}$. In some settings, it is useful to consider
\emph{randomized estimators} that, at each round, produce a
distribution $\nu\ind{t}\in\Delta(\cM)$ over models. For this setting, we further
generalize the \CompShort by defining
\begin{equation}
  \label{eq:comp_general_randomized}
  \compgenrand(\cM,\nu) =
    \inf_{p\in\Delta(\Act)}\sup_{M\in\cM}\En_{\pi\sim{}p}\biggl[\fm(\pim)-\fm(\pi)
    -\gamma\cdot\En_{\Mhat\sim\nu}\brk*{\Dgenpi{\Mhat}{M}}
    \biggr]
  \end{equation}
  for distributions $\nu\in\Delta(\cM)$. We additionally define $\compgenrand(\cM)=\sup_{\nu\in\Delta(\cM)}\compgenrand(\cM,\nu)$.

  \begin{algbox}[frametitle={\algtitle{alg:etd_general_div}{\etd for General Divergences and Randomized Estimators}}]
        \setstretch{1.3}
         \begin{algorithmic}
           \Require Exploration parameter $\gamma>0$; divergence $\Dgen{\cdot}{\cdot}$; randomized online estimation oracle.
      \For{$t=1, 2, \cdots, T$}
      \State Obtain randomized estimate $\nu\ind{t}\in\Delta(\cM)$ from estimation oracle
      with $\crl*{(\pi\ind{i},r\ind{i},o\ind{i})}_{i<t}$.
      \State Compute \hfill \algcommentlight{Eq. \pref{eq:comp_general_randomized}}.%
      \[
        p\ind{t}=\argmin_{p\in\Delta(\Act)}\sup_{M\in\cM}
    \En_{\pi\sim{}p}\biggl[\fm(\pim)-\fm(\pi)
        -\gamma\cdot\En_{\Mhat\sim\nu\ind{t}}\brk*{\Dgenpi{\Mhat}{M}}
        \biggr].
    \]
    \State{}Sample decision $\pi\ind{t}\sim{}p\ind{t}$ and update estimation
    algorithm with $(\pi\ind{t},r\ind{t}, \obs\ind{t})$.
    \EndFor
    \end{algorithmic}
    \end{algbox}

\paragraph{Algorithm}  
  A generalization of \etd that incorporates general divergences and
  randomized estimators is given at the start of this subsection. The
  algorithm is identical to \etd (\cref{sec:algo_gen_dm}), with the only
  differences being that i) we play the distribution that solves the minimax
  problem \pref{eq:comp_general_randomized} with the user-specified divergence
  $\Dgenpi{\cdot}{\cdot}$ rather than squared Hellinger distance, and ii)
  we use the randomized estimate $\nu\ind{t}$ rather than a point estimate. Our
  performance guarantee for this algorithm depends on the estimation
  performance of the oracle's randomized estimates $\nu\ind{1},\ldots,\nu\ind{T}\in\Delta(\cM)$ with respect to the given
  divergence $\Dgenpi{\cdot}{\cdot}$, which we define as
    \begin{equation}
    \label{eq:general_error}
        \EstD \ldef{} \sum_{t=1}^{T}\En_{\pi\ind{t}\sim{}p\ind{t}}\En_{\Mhat\ind{t}\sim\nu\ind{t}}\brk*{\Dgenpi[\pi\ind{t}]{\Mhat\ind{t}}{\Mstar}}.
  \end{equation}
We have the following guarantee.
\begin{prop}[Regret of \etd with general divergences]
  \label{thm:upper_general_distance}
The algorithm \etd for General Divergences
and Randomized Estimators with exploration parameter $\gamma>0$ guarantees
that
\begin{equation}
  \label{eq:upper_general_distance}
\Reg \leq{}\compgenrand(\cM)\cdot{}T + \gamma\cdot\EstD
\end{equation}
for every realization.
\end{prop}

\paragraph{Sufficient statistics and benefits of general divergences}
Many divergences of interest have
the useful property that they depend on the estimated model $\Mhat$
only through a ``sufficient statistic'' for the model class under
consideration. Formally, there exists a
\emph{sufficient statistic space} $\Suff$ and \emph{sufficient
  statistic} $\suffmap:\cM\to\Psi$ with the property that we can write
(overloading notation)
\[
\Dgenpi{M}{M'} = \Dgenpi{\suffmap(M)}{M'},\quad\fm(\pi)=f^{\suffmap(M)}(\pi),\mathand\pim=\pi_{\suffmap(M)}
\]
for all models $M,M'$. In this case, it suffices for the online
estimation oracle to directly estimate the sufficient statistic by
producing a randomized estimator $\nu\ind{t}\in\Delta(\Suff)$, and we
can write the estimation error as
\begin{align}
  \EstD \ldef{}
\sum_{t=1}^{T}\En_{\pi\ind{t}\sim{}p\ind{t}}\En_{\suffhat\ind{t}\sim\nu\ind{t}}\brk*{\Dgenpi[\pi\ind{t}]{\suffhat\ind{t}}{\Mstar}}.
\end{align}
The benefit of this perspective is that for many examples of interest,
since the divergence depends on the estimate only through $\psi$, we
can derive bounds on $\Est$ that scale with $\log\abs{\Psi}$ instead
of $\log\abs{\cM}$.

For example, in structured bandit problems, one can work with the
divergence
$$\Dsq{\Mhat(\pi)}{M(\pi)}\ldef{}(\fm(\pi)-\fmhat(\pi))^2$$ which
uses the mean reward function as a sufficient statistic, i.e. $\suff(M)=\fm$. Here, it is
clear that one can achieve $\EstD\approxleq\log\abs{\cF}$, which 
improves upon the rate $\EstHel\approxleq\log\abs{\cM}$ for Hellinger
distance, and recovers the specialized version of the \etd algorithm we considered in
\pref{sec:structured}. Analogously, for reinforcement learning, one
can consider value functions as a sufficient statistic, and use an appropriate divergence based on Bellman residuals to derive
estimation guarantees that scale with the complexity $\log\abs{\cQ}$ of a given value
function class $\cQ$; see \cref{sec:rl} for
details.

\oldparagraph{Does randomized estimation help?}
Note that whenever $D$
is convex in the first argument, we have
$\compgenrand(\cM)\leq\sup_{\Mhat\in\conv(\cM)}\compgen(\cM,\Mhat)=\compgen(\cM)$
(that is, the randomized DEC is never larger than the vanilla DEC), but it is not immediately
apparent whether the opposite direction of this inequality holds, and
one might hope that working with the randomized \CompShort in
\pref{eq:comp_general_randomized} would lead to improvements over the
non-randomized counterpart. The
next result shows that this is not the case: Under mild assumptions on the divergence $D$,
randomization offers no improvement.%
\begin{prop}[Randomized estimators do not improve the DEC]
  \label{prop:randomized_equivalence}
Let $\Dgenshort$ be any bounded divergence with the property that
for all models $M,M',\Mhat$ and $\pi\in\Pi$,
\begin{equation}
  \label{eq:triangle}
  \Dgenpi{M}{M'}
  \leq{} C\prn*{
    \Dgenpi{\Mhat}{M}
    + \Dgenpi{\Mhat}{M'}
    }.
\end{equation}
Then for all $\gamma>0$,
\begin{align}
  \sup_{\Mhat}\compgen(\cM,\Mhat)
  \leq{} \compgenrandbasic_{\gamma/(2C)}(\cM).
\end{align}
\end{prop}
Squared Hellinger distance is symmetric and satisfies Condition \pref{eq:triangle} with
$C=2$. Hence, writing $\compHrand(\cM)$ as shorthand for $\compgenrand(\cM)$ with
$D=\Dhels{\cdot}{\cdot}$, we obtain the following corollary.
\begin{prop}[Randomized DEC for Hellinger distance]
  \label{prop:hellinger_randomized}
  Suppose that $\cR\subseteq\brk*{0,1}$. Then for all
  $\gamma>0$,
  \begin{align*}
    \compHrand(\cM)
    \leq{} \sup_{\Mhat\in\conv(\cM)}\compH(\cM,\Mhat)
    \leq \sup_{\Mhat}\compH(\cM,\Mhat)
    \leq{} \compHrand[\gamma/4](\cM).
          \end{align*}
\end{prop}
This shows that for Hellinger distance---at least from a statistical perspective---there is
no benefit to using the randomized
\CompShort compared to the original version. In some cases, however,
strategies $p$ that minimize $\compHrand(\cM,\nu)$ can be simpler to
compute than strategies that minimize $\compH(\cM,\Mhat)$ for $\Mhat\in\conv(\cM)$.

\subsubsection{Optimistic Estimation}
\label{sec:optimistic}

To derive stronger regret bounds that allow for estimation with
general divergences, we can combine \etdtext with a
specialized estimation approach introduced by \citet{zhang2022feel}
(see also \citet{dann2021provably,agarwal2022model,zhong2022posterior}),
which we refer to as \emph{optimistic estimation}. The results we
present here are based on \citet{foster2022note}.

Let a divergence $\Dgenpi{\cdot}{\cdot}$ be fixed. An \emph{optimistic estimation oracle}
$\AlgEst$ is an algorithm which, at each step $t$, given
$\hist\ind{t-1}=(\pi\ind{1},r\ind{1},o\ind{1}),\ldots,(\pi\ind{t-1},r\ind{t-1},o\ind{t-1})$,
produces a randomized estimator $\nu\ind{t}\in\Delta(\cM)$.
Compared to the previous section, the only change is that for a parameter $\gamma>0$, we
will measure the performance of the oracle via \emph{optimistic
  estimation error}, defined as
\begin{align}
  \label{eq:est_error_basic}
  \EstOptD \ldef{}
\sum_{t=1}^{T}\En_{\pi\ind{t}\sim{}p\ind{t}}\En_{\Mhat\ind{t}\sim\nu\ind{t}}\brk*{\Dgenpi[\pi\ind{t}]{\Mhat\ind{t}}{\Mstar}
  + \gamma^{-1}(\fmstar(\pimstar)-\fmhatt(\pi\subs{\Mhat\ind{t}}))}.
\end{align}
This quantity is similar to \pref{eq:general_error}, but incorporates a bonus
term
\[
  \gamma^{-1}(\fmstar(\pimstar)-\fmhatt(\pi\subs{\Mhat\ind{t}})),
\]
which encourages the estimation algorithm to \emph{over-estimate} the
optimal value $\fmstar(\pimstar)$ for the underlying model, leading to
a form of optimism. The online estimation algorithm must
control this combined criterion to obtain the regret bounds below.

\begin{example}[Structured bandits]
  Consider any structured bandit problem with decision space $\Pi$,
  function class $\cF\subseteq(\Pi\to\brk{0,1})$, and
  $\cO=\NullObs$. Let $\cMf$ be the class
  \[
\cMf=\crl*{M\mid{}\fm\in\cF, \text{$M(\pi)$ is $1$-\subgaussian}\;\forall{}\pi}.
\]
To derive bounds on the optimistic estimation error, we can appeal to
an augmented version of the (randomized) exponential weights algorithm which, for
  a learning rate parameter $\eta>0$, sets
  \[
    \nu\ind{t}(\fm) \propto\exp\prn*{ -\eta\prn*{\sum_{i<t}\brk*{(\fm(\pi\ind{i})-r\ind{i})^2 -
        \gamma^{-1}\fm(\pim)}} }.
  \]
  For an appropriate choice of $\eta$, this method achieves
  $\En\brk[\big]{\EstOptD}\approxleq{} \log\abs{\cF} +
  \sqrt{T\log\abs{\cF}}/\gamma$ for $D=\Dsq{\cdot}{\cdot}$ \citep{zhang2022feel}.
\end{example}

\keepwithbox{19}%
\paragraph{Optimistic Estimation-to-Decisions}

We consider an \emph{optimistic} variant of \etd, which we
refer to as \etdopt. At each timestep $t$, the algorithm calls the estimation oracle to
obtain a randomized estimator $\nu\ind{t}$ using the data
$(\pi\ind{1},r\ind{1},o\ind{1}),\ldots,(\pi\ind{t-1},r\ind{t-1},o\ind{t-1})$ collected so
far. The algorithm then uses the estimator to compute a distribution
$p\ind{t}\in\Delta(\Pi)$ and samples $\pi\ind{t}$ from this
distribution.

\begin{algbox}[frametitle={\algtitle{alg:etd_opt}{Optimistic \etd (\etdopt)}}]
      \setstretch{1.3}
       \begin{algorithmic}
         \Require Exploration parameter $\gamma>0$; divergence $\Dgen{\cdot}{\cdot}$; optimistic estimation oracle.
    \For{$t=1, 2, \cdots, T$}
    \State \multiline{Obtain randomized estimate $\nu\ind{t}\in\Delta(\cM)$ from
    optimistic estimation oracle
    with $\crl*{(\pi\ind{i},r\ind{i},o\ind{i})}_{i<t}$.}
    \State Compute \hfill \algcommentlight{Eq. \pref{eq:optimistic_dec}}.%
    \[
      p\ind{t}=\argmin_{p\in\Delta(\Act)}\sup_{M\in\cM}
  \En_{\pi\sim{}p}\En_{\Mhat\sim\nu\ind{t}}\biggl[\fmhat(\pimhat)-\fm(\pi)
      -\gamma\cdot\Dgenpi{\Mhat}{M}
      \biggr].
  \]
  \State{}Sample decision $\pi\ind{t}\sim{}p\ind{t}$ and update estimation
  algorithm with $(\pi\ind{t},r\ind{t}, \obs\ind{t})$.
  \EndFor
  \end{algorithmic}
  \end{algbox}

 The main change relative to the version of \etd for general divergences
in \cref{sec:general_divergence} is that the minimax problem in \etdopt is
derived from an ``optimistic'' variant of the \CompShort tailored to
the optimistic estimation error in \cref{eq:est_error_basic}. This
quantity, which
we refer to as the \emph{Optimistic \CompText}, is defined for
$\nu\in\Delta(\cM)$ as
\begin{align}
  \ocompD(\cM,\nu)
  =
  \inf_{p\in\Delta(\Pi)}\sup_{M\in\cM}\En_{\pi\sim{}p}\En_{\Mhat\sim\nu}\brk*{
  \fmhat(\pimhat) - \fm(\pi)  - \gamma\cdot{}\Dgenpi{\Mhat}{M}
  }.
  \label{eq:optimistic_dec}
\end{align}
and
\begin{align}
  \label{eq:optimistic_max}
  \ocompD(\cM) = \sup_{\nu\in\Delta(\cM)}\ocompD(\cM,\nu).
\end{align}
The Optimistic \CompShort the same as the generalized \CompShort in
\cref{eq:comp_general_randomized}, except that
the optimal value $\fm(\pim)$ in
\cref{eq:comp_general_randomized} is replaced by the optimal value $\fmhat(\pimhat)$ for the
(randomized) reference model $\Mhat\sim\nu$. This seemingly small change is the main
advantage of incorporating optimistic
estimation, and makes it possible to bound the Optimistic
\CompShort for certain divergences $D$ for which the value of the
generalized \CompShort in \cref{eq:comp_general_randomized} would otherwise be unbounded.

\begin{rem}[Sufficient statistics]
  When the divergence $D$ admits a sufficient statistic
  $\suffmap:\cM\to\Psi$, for any distribution $\nu\in\Delta(\cM)$, if
  we define $\nu\in\Delta(\Psi)$ via
  $\nu(\psi) = \nu(\crl{M\in\cM: \suffmap(M)=\psi})$, we have
  \[
    \ocompD(\cM,\nu) =
    \inf_{p\in\Delta(\Pi)}\sup_{M\in\cM}\En_{\pi\sim{}p}\En_{\suff\sim\nu}\brk*{
      \fsuff(\pisuff) - \fm(\pi) - \gamma\cdot{}\Dgenpi{\suff}{M} }.
  \]
  In this case, by overloading notation slightly, we may simplify the
  definition in \pref{eq:optimistic_max} to
  \[
    \ocompD(\cM) = \sup_{\nu\in\Delta(\Psi)}\ocompD(\cM,\nu).
    \]
\end{rem}

\paragraph{Regret bound for optimistic \etd}
The following result shows that the regret of \etdopttext is controlled by
the Optimistic \CompShort and the optimistic estimation error for the
oracle.
\begin{prop}[Regret of optimistic \etd]
  \label{thm:optimistic}
  \etdopt ensures that
  \begin{align}
    \label{eq:main_single_sample}
    \RegDM \leq \ocompD(\cM)\cdot{}T + \gamma\cdot{}\EstOptD
  \end{align}
  for every realization.
\end{prop}
This regret bound has the same structure as that of \pref{thm:upper_general_distance}, but the
\CompShort and estimation error are replaced by their optimistic
counterparts.

\paragraph{When does optimistic estimation help?}
When does the regret bound in \cref{thm:optimistic} improve upon its
non-optimistic counterpart in \cref{thm:upper_general_distance}? It
turns out that for \emph{asymmetric} divergences
such as those found in the context of reinforcement learning, the
regret bound in \cref{eq:main_single_sample} can be much smaller than
the corresponding bound in \cref{eq:upper_general_distance}; see
\cref{sec:dec_bellman} for an example. However, for symmetric divergences such as Hellinger distance, we now
show that the result never
improves upon \cref{thm:upper_general_distance}.

Given a divergence $D$, we define the \emph{flipped divergence}, which
swaps the first and second arguments, by
\[
  \Dflippi{\Mhat}{M}\ldef\Dgenpi{M}{\Mhat}.
\]
\begin{prop}[Equivalence of optimistic DEC and randomized DEC]
  \label{prop:equiv_symmetry}
  Assume that For all pairs of models $M,\Mhat\in\conv(\cM)$, we have
  $(\fmhat(\pi) - \fm(\pi))^2 \leq{}
  \Lcont^2\cdot\Dgenpi{\Mhat}{M}$ for a constant $\Lcont>0$. Then for all $\gamma>0$,
  \begin{align}
\compgenrandbasic[\Dflipshort]_{3\gamma/2}(\cM)
    - \frac{\Lcont^2}{2\gamma} \leq    \ocompD(\cM) \leq{} \compgenrandbasic[\Dflipshort]_{\gamma/2}(\cM)
    + \frac{\Lcont^2}{2\gamma}.
  \end{align}
\end{prop}
This result shows that the optimistic DEC with divergence $D$ is
equivalent to the generalized DEC in
\cref{eq:comp_general_randomized}, but with the arguments to the
divergence flipped. Thus, for symmetric divergences, the quantities
are equivalent. In particular, we can combine \cref{prop:equiv_symmetry} with \cref{prop:randomized_equivalence} to derive the
following corollary for Hellinger distance.
\begin{prop}[Optimistic DEC for Hellinger distance]
  Suppose that rewards are bounded in $\brk*{0,1}$. Then for all
  $\gamma>0$,
  \begin{align*}
    \ocompH[2\gamma](\cM) - \frac{1}{\gamma}
    \leq{} \sup_{\Mbar}\compH(\cM,\Mbar)
    \leq{} \ocompH[\gamma/6](\cM) + \frac{3}{\gamma}.
  \end{align*}
\end{prop}

For asymmetric
divergences, in settings where there exists an estimation oracle for which the flipped
estimation error
\[\Est^{\Dflipshort}=\sum_{t=1}^{T}\En_{\pi\sim{}p\ind{t}}\En_{\Mhat\ind{t}\sim{}\nu\ind{t}}\brk*{\Dgenpi[\pi\ind{t}]{\Mstar}{\Mhat\ind{t}}}\]
is controlled, \cref{prop:equiv_symmetry} shows that to match the guarantee in
\pref{thm:optimistic}, optimism is not required, and it suffices to
run the non-optimistic algorithm of
\cref{sec:general_divergence}. However, as discussed in
\cref{sec:dec_bellman}, for certain divergences found in the
context of reinforcement learning, estimation with respect to the
flipped divergence is not feasible, yet working with the optimistic
\CompShort and \etdopt leads to meaningful guarantees.

\subsection{Bibliographic Remarks and Further Reading}\label{sec:lower_bib}
\paragraph{Takeaways}
This section closes the loop on the \CompShort: \cref{thm:decc-lower-general}
shows that the constrained \CompShort is not only sufficient for low regret,
as \etd established in \cref{prop:upper_constrained}, but necessary, so that
whatever separates \etd from the minimax regret is a cost of estimation.
Every lower bound in the section rests on one mechanism: a change of
measure between two models that are statistically close (with closeness
aggregated over the interaction by the chain rule for KL or Hellinger
divergence) yet whose near-optimal decisions are far apart. Here, the potentially surprising takeaway is that the DEC---dual to its role in the Estimation-to-Decisions upper bound---can encompass this lower bound mechanism into a single optimization problem.

\paragraph{Bibliographic remarks}
The two-point method is classical \citep{lecam1973convergence}; see
\citet[Chapter 2]{tsybakov2008introduction} or \citet{donoho1987geometrizing}. The
$\Omega(\sqrt{AT})$ bandit lower bound in \cref{lem:gaussian_mab_lower_bound}
is due to \citet{auer2002non}; the presentation through
\cref{lem:regret_average_profile} and \cref{cor:kl_chain_rule_interactive}
follows \citet[Chapter 15]{lattimore2020bandit}, and \cref{lem:Hellinger-chain}
is from \citet{jayram2009hellinger,foster2024online}. 

Lower bounds in terms of the
\CompShort were introduced by \citet{foster2021statistical}; the constrained
form and \cref{thm:decc-lower-general} are due to \citet{foster2023tight}, simplified for this text. The general-loss lower bound of \cref{prop:lower_bound_simple}, with arbitrary reference models, is from \citet{foster2023multiagent}, where it arises from decision making with hidden rewards. The hard family of \cref{prop:hard_family} and
the examples in \cref{sec:lower_examples} follow
\citet{foster2021statistical,foster2023tight}, recovering classical rates:
\citet{dani2008stochastic} for \cref{ex:linear_lower},
\citet{kleinberg2008multi} for \cref{ex:nonparametric_lower}, and
\citet{jaksch2010near,osband2016lower,domingues2021episodic} for
\cref{prop:lower_tabular}.
In the starred material, \cref{thm:upper_general_distance},
\cref{prop:randomized_equivalence}, \cref{prop:hellinger_randomized},
\cref{prop:equiv_symmetry}, and \cref{thm:optimistic} are from
\citet{foster2022note}; optimistic estimation itself is due to
\citet{zhang2022feel}.

\paragraph{Further reading}
The estimation term separating the upper and lower bounds cannot be
removed in general \citep{foster2023tight};
\citet{chen2024assouad}, building on \citet{foster2022complexity}, close this gap up to polynomial factors for the special case of convex
classes, and highlight barriers to further improvement. For more background on the modulus of continuity in statistical estimation, which can be viewed as a non-interactive predecessor to the \CompText, see \citet{donoho1987geometrizing,donoho1991geometrizingii}. \citet{chen2025decision} extend the upper and lower bounds of this section to environments that change over time, within a hybrid stochastic/adversarial framework that covers local differential privacy, statistical-query learning, and robust decision making.

\section[Reinforcement Learning: Function Approximation and Large State Spaces]{Reinforcement Learning:\\ Function Approximation and Large State Spaces}
\label{sec:rl}

This section returns to reinforcement learning, now without the
assumption that the state space is small. The protocol is the same as
\cref{sec:mdp}, but the goal is regret that does not depend on the
number of states. The main tool we use is function approximation: a class of
models for the transition dynamics, or a class of value functions, that
lets the learner generalize across states as it explores. Compared with
the contextual and structured bandit problems, where function
approximation already appeared, reinforcement learning adds credit
assignment on top of generalization, and the combination is what makes
this problem challenging. There are two questions this section is organized around: what structure,
beyond realizability, makes sample-efficient learning possible, and
which algorithms exploit it. In the language of \cref{fig:axes},
this section covers the far corner of the figure, covering the frontier of theoretical understanding.

\paragraph{Organization}
\cref{sec:rl_realizability} recalls the setting and shows that
realizability alone, of the model, the value function, or the policy,
is not enough for sample-efficient learning. \cref{sec:rl_linear}
studies the first structural assumption, low-rank transitions, and the
\lsviucb algorithm. \cref{sec:rl_bellman_rank} introduces the more
general notion of Bellman rank and the \bilinucb algorithm, and
\cref{sec:bellman_rank_example} works through MDP families that have
low Bellman rank. Finally,
\cref{sec:dec_bellman_rank} instantiates the Estimation-to-Decisions
principle by bounding the \CompText for classes with low Bellman rank,
and \cref{sec:rl_conclusion} closes with a survey of the landscape of
tractability and open directions.

\subsection{Setting and Realizability}
\label{sec:rl_realizability}

Recall (\cref{sec:mdp}) that for reinforcement learning, an MDP $M$ takes the form
\[M=\crl*{\cS, \cA, \crl{\Pm_h}_{h=1}^{H}, \crl{\Rm_h}_{h=1}^{H},
    d_1},\] where $\cS$ is the state space, $\cA$ is the action space,
  $\Pm_h:\cS\times\cA\to\Delta(\cS)$ is the probability transition
  kernel at step $h$, $\Rm_h:\cS\times\cA\to\Delta(\bbR)$ is
  the reward distribution, and $d_1\in\Delta(\cS)$ is the initial
  state distribution. All of the results in this section take
  $\Act = \PiRNS$, and we assume that rewards are nonnegative with
  $\sum_{h=1}^{H}r_h\in\brk{0,1}$ almost surely, unless otherwise
  specified. As in
  \cref{sec:mdp}, $\Qstar_h\ldef\Qmstarstar_h$ and $\Vstar_h\ldef\Vmstarstar_h$
  denote the optimal value functions of the true MDP $\Mstar$. Note that the online reinforcement learning framework
  is a special case of the general decision making setting in
  \cref{sec:general_dm}, but the algorithms we develop in this section
  are tailored to the MDP structure.

For the frameworks we have considered so far (contextual and
structured bandits, general decision making), all of the algorithms we
analyzed leveraged the assumption of \emph{realizability}, which
asserts that we have a function class that is capable of modeling the
underlying environment well. For reinforcement learning, there are
various realizability assumptions one can consider:
\begin{itemize}
\item \emph{Model realizability}: We have a model class $\cM$ of MDPs
  that contains the true MDP $\Mstar$.
\item \emph{Value function realizability}: We have a class $\cQ$ of
  state-action value functions ($Q$-functions) that contains the optimal function $\Qstar$
  for the underlying MDP.
\item \emph{Policy realizability}: We have a class $\Pi$ of policies
  that contains an optimal policy $\pistar$.
\end{itemize}
To relate these assumptions, define the \emph{induced classes}
\begin{align}
  \label{eq:rl_induced_classes}
  \cQ_{\cM}\ldef\crl*{\Qmstar:M\in\cM},
  \qquad
  \Pi_{\cQ}\ldef\crl*{\piq:Q\in\cQ},
\end{align}
where $\piq$ is the greedy policy for $Q$, with a fixed rule for breaking
ties. We see that model realizability for $\cM$ implies value function realizability
for $\cQ_{\cM}$. Likewise, value function realizability for $\cQ$ implies
policy realizability for $\Pi_{\cQ}$, since a greedy policy for $\Qstar$
is optimal by \pref{prop:bellman}. In general however, the converse implications do not hold.

Ideally, we would hope that whenever one of (model, value function, policy) realizability holds, we can
obtain regret bounds that scale with the complexity of the corresponding function class (e.g., $\log\abs{\cM}$ for model realizability, or
$\log\abs{\cQ}$ for value function realizability), but do not
depend on the number of states $\abs{\cS}$ or other
properties of the underlying MDP, analogous to the situation for
statistical learning. Unfortunately, the following result, due to
\citet{krishnamurthy2016pac}, shows that this is too much to ask for.
\begin{prop}[Realizability is insufficient]
  \label{prop:realizability_insufficient}
For any $S\in\bbN$ and $H\in\bbN$, there exists a class of horizon-$H$
MDPs $\cM$
with $\abs{\cS}=S$, $\abs{\cA}=2$, and $\log\abs{\cM}=\log(S)$, yet any algorithm must
have 
\[
  \En\brk*{\Reg} \approxgeq \sqrt{\min\crl{S,2^{H}}\cdot{}T} \wedge T.
\]
\end{prop}

The interpretation of this result is that even if model realizability
holds, any algorithm needs regret that scales with
$\min\crl{\abs{\cS},\abs{\cM},2^{H}}$. This means additional structural assumptions on the
underlying MDP $\Mstar$---beyond realizability---are required if we want to obtain
sample-efficient learning guarantees. The same construction
also rules out such guarantees under value function or policy
realizability: the induced classes in \pref{eq:rl_induced_classes}
satisfy
\[
  \abs{\Pi_{\cQ_{\cM}}}\leq\abs{\cQ_{\cM}}\leq\abs{\cM},
\]
so passing to these classes does not increase the class-size complexity.

\begin{rem}[Comparison to structured bandits]
  \cref{prop:realizability_insufficient} is analogous to the impossibility result we proved for
  structured bandits (\pref{ex:structured_realizability}), which is
  subsumed by the RL framework. That result required a large number of actions,
  while \cref{prop:realizability_insufficient} holds even when $\abs{\cA}=2$.
\end{rem}
\begin{rem}[Further notions of realizability]
  There are many notions of realizability beyond those we
consider above. For example, for value function approximation, one can
assume that $\Qm{\Mstar}{\pi}\in\cQ$ \emph{for all $\pi$}
\citep{zanette2020learning,jin2021bellman}, or assume
that the class $\cQ$ obeys certain notions of consistency with respect
to the Bellman operator for $\Mstar$ (\emph{Bellman
completeness}; see \citet{munos2008finite,chen2019information}).
\end{rem}

The remainder of this section explores different structural assumptions that
facilitate low regret for reinforcement learning with function
approximation. Briefly, the idea is to make assumptions that
either i) allow for extrapolation across the state space, or ii)
control the number of ``effective'' state distributions the algorithm
can encounter. We begin with reinforcement learning with
linear models (\cref{sec:rl_linear}), then explore a general structural property known as
Bellman rank (\cref{sec:rl_bellman_rank}).

\subsection{Low-Rank MDPs and \lsviucb}
\label{sec:rl_linear}

Toward understanding the complexity of RL with function
approximation, let us consider perhaps the simplest possible modeling
approach: Linear function approximation. A natural idea here is to
assume linearity of the underlying $Q$-function corresponding to the true model $\Mstar$, generalizing the
linear bandit setting in \pref{sec:structured}:
\begin{align}
  \label{eq:linear_qstar}
  \Qstar_h(s,a) = \tri*{\phi(s,a),\thetam[\Mstar]_h},\quad\forall{}h\in\brk{H}
\end{align}
where $\phi(s,a)\in\sB_2^d(1)$ is a feature map that is known to the
learner and $\thetam[\Mstar]_h\in\unitball$ is an unknown parameter
vector. Equivalently, we
can define
\begin{align}
  \label{eq:linear_value}
\cQ = \crl*{ Q=(Q_1,\ldots,Q_H) ~\text{with}~ Q_h(s,a) = \tri*{\phi(s,a),\theta_h} \mid{}\theta_h\in\sB_2^d(1)\;\forall{}h},
\end{align}
and assume that $\Qstar\in\cQ$. This is called the \emph{\linqstar} model.

Linearity is a strong assumption, and it is reasonable to imagine that this
would be sufficient for low regret. Indeed, one might hope that using
linearity, we can extrapolate the value of $\Qstar$ once we estimate
it for a small number of states. Unfortunately, even for this very simple class of functions,
it turns out that realizability is still insufficient; the following
result is due to \citet{weisz2021exponential} (see also \citet{wang2021exponential}).
\begin{prop}[Linear $\Qstar$ is insufficient]
  \label{prop:realizability_insufficient_linear}
For any $d\in\bbN$ and $H\in\bbN$ sufficiently large, any algorithm
for the \linqstar model must have
\[
  \En\brk*{\Reg} \approxgeq \min\crl*{2^{\bigom(d)}, 2^{\bigom(H)}} \wedge T.
\] 
\end{prop}
This contrasts the situation for contextual bandits and linear
bandits, where linear rewards were sufficient for low regret. The
intuition is that, even though $\Qstar$ is linear, it might take a
very long time to estimate the value for even a small number of
states. That is, linearity of the optimal value function is not
a useful assumption unless there is some kind of additional structure that can guide us
toward the optimal value function to begin with.

We mention in passing that
\cref{prop:realizability_insufficient_linear} can be proven by lower
bounding the \CompText \citep{foster2021statistical}.

\paragraph{The Low-Rank MDP model}
We therefore make a stronger assumption, which asserts that
the transition probabilities themselves have linear structure: For all
$s\in\cS$, $a\in\cA$, and $h\in\brk{H}$, we have\loose
\begin{align}
  \label{eq:low_rank_mdp}
  \Pm_h(s'\mid{}s,a) = \tri*{\phi(s,a),\mum_h(s')},\mathand \En\brk*{r_h|s,a}=\tri*{\phi(s,a),\wm_h}.
\end{align}
Here, $\phi(s,a)\in\unitball$ is a feature map that is known to the
learner, $\mum_h(s')\in\reals^d$ is another feature map which is
\emph{unknown} to the learner, and $\wm_h\in\sB_2^d(\sqrt{d})$ is an unknown
parameter vector. Additionally, for simplicity, we assume that 
\begin{align}
\label{eq:ucbvi_scaling}
\norm{\sum_{s'\in\cS} |\mum_h(s')|}\leq \sqrt{d}
\end{align}
which in particular holds in the tabular example below. 
As before, assume that both cumulative and individual-step rewards are
in $[0,1]$. For the remainder of the subsection, we let $\cM$ denote the set of MDPs with these properties.

The linear structure in \cref{eq:low_rank_mdp} implies that the
transition matrix has rank at most $d$, thus facilitating information sharing and generalization across states,
even when the cardinality of $\cS$ and $\cA$ is large or infinite. For
this reason, we refer to MDPs with this structure as \emph{low-rank
  MDPs} \citep{RendleFS10,YaoSPZ14,jin2020provably,agarwal2020flambe}.

Just as linear bandits generalize unstructured multi-armed bandits,
the low-rank MDP model \cref{eq:low_rank_mdp} generalizes tabular RL,
which corresponds to the special case in which $d=|\cS|\cdot |\cA|$,
$\phi(s,a)=\boldsymbol{e}_{s,a}$, and $(\mu_h(s'))_{s,a}=\Pm_h(s'\mid{}s,a)$.

\paragraph{Properties of low-rank MDPs}

The linear structure of the transition probabilities and mean rewards is a significantly more stringent assumption than linearity of $\Qstar_h(s,a)$ in \eqref{eq:linear_qstar}. Notably, it implies that Bellman backups of \emph{arbitrary functions} are linear.
\begin{lem}[Bellman backups for low-rank MDPs]
	\label{lem:lsvi_bellman_linear}
	For any low-rank MDP $M\in\cM$ and any $Q:\cS\times\cA \to \reals$ and any $h\in[H]$, the Bellman operator is linear in $\phi$:
	$$\brk*{\cTm_h Q}(s,a) = \tri*{\phi(s,a), \theta\subs{Q}\sups{M}}$$
	for some $\theta\subs{Q}\sups{M}\in \reals^d$. Furthermore, for $Q:\cS\times\cA\to[0,1]$, under the scaling assumption \eqref{eq:ucbvi_scaling}, it holds that $\norm{\theta\subs{Q}\sups{M}}\leq 2\sqrt{d}$.
    Finally, for any policy $\pi=(\pi\subs{1},\ldots,\pi\subs{H})$, functions $Q\sups{M,\pi}_{h}$ are linear in $\phi$ for every $h$. 
    \end{lem}
      
\begin{proof}[\pfref{lem:lsvi_bellman_linear}]
	We have
  \begin{align}
	  \brk*{\cTm_h Q}(s,a) &= 
	  	\tri*{\phi(s,a),\wm_h} +  
	  	\sum_{s'} \Pm_h(s'\mid{}s,a)  \max_{a'}Q (s', a') \\
		&= \tri*{\phi(s,a),\wm_h} + \sum_{s'} \tri*{\phi(s,a),\mum_h(s')} \max_{a'} Q(s', a') \\
		&= \tri*{\phi(s,a), \wm_h + \sum_{s'} \mum_h(s')  \max_{a'}Q(s', a')}. 
	\end{align}
	and, under the aforementioned assumptions,
	\begin{align}
		\label{eq:norm_of_bellman_backups}
		\norm{\theta\subs{Q}\sups{M}}\leq \norm{\wm_h} + \norm{\sum_{s'} \mum_h(s') \max_{a'} Q(s',a')} \leq 2\sqrt{d}.
	\end{align}
\end{proof}

\subsubsection{The LSVI-UCB Algorithm}
\newcommand{\bhdelta}{b_{h,\delta}}

To provide regret bounds for the low-rank MDP model, we analyze an
algorithm called \lsviucb (``Least Squares Value Iteration \ucb''), which was introduced and analyzed in an influential
paper of \citet{jin2020provably}. Similar to the \ucbvi algorithm we
analyzed for tabular RL, the main idea behind the algorithm is to
compute a state-action value $\Qbar\ind{t}$ with the optimistic property that
\[
\Qbar_h\ind{t}(s,a) \geq \Qstar_h(s,a)
\]
for all $s,a,h$. This is achieved by combining the principle of
dynamic programming with an appropriate choice of bonus to ensure
optimism. However, unlike \ucbvi, the algorithm does not directly
estimate transition probabilities (which is not feasible when $\mum$
is unknown), and instead implements approximate value iteration by
solving a certain least squares objective.
\begin{algbox}[frametitle={\algtitle{alg:lsviucb}{\lsviucb}}]
	\begin{algorithmic}
          \Require Radius $R>0$.
	\For{$t=1,\ldots,T$}
	\State Let $\Qbar\ind{t}\subs{H+1}\equiv 0$. 
        \For{$h=H,\ldots,1$}
        \State \multiline{Compute least-squares estimator 
        \[
          \thetahat_h\ind{t} =
          \argmin_{\theta\in\sB_2^d(2\sqrt{d})}\sum_{i<t}\prn*{\tri*{\phi(s_h\ind{i},a_h\ind{i}),\theta}
            - r_h\ind{i} - \max_{a}\Qbar_{h+1}\ind{t}(s_{h+1}\ind{i}, a)}^2,
        \]
        and let 
        $\Qhat_h\ind{t}(s,a)\ldef\tri[\big]{\phi(s,a),\thetahat_h\ind{t}}$.}
      \State Define
      \[
        \wb{\Sigma}_h\ind{t}=\sum_{i<t}\phi(s_h\ind{i},a_h\ind{i})
        \phi(s_h\ind{i},a_h\ind{i})^{\trn} + I.
      \]
      \State Compute bonus:
      \[
        b\subs{h,\delta}\ind{t}(s,a) = \sqrt{R} \nrm*{\phi(s,a)}_{(\wb{\Sigma}_h\ind{t})^{-1}}.
      \]
        \State
		\State
                Compute optimistic value function:
		\[
                  \Qbar\sups{t}\subs{h}(s,a) = \max\crl*{0,
                  \min\crl*{1,\Qhat_h\ind{t}(s,a) + b\sups{t}\subs{h,\delta}(s,a)}}.
                \]
		\State Set $\Vbar\sups{t}\subs{h}(s) = \max_{a\in\cA} \Qbar\sups{t}\subs{h}(s,a)$ and $\pihat\sups{t}\subs{h}(s)=\argmax_{a\in\cA} \Qbar\sups{t}\subs{h}(s,a).$
		\EndFor{}
	\State Collect trajectory $(s\subs{1}\ind{t},a\subs{1}\ind{t},r\subs{1}\ind{t}),\ldots,(s\subs{H}\ind{t},a\subs{H}\ind{t},r\subs{H}\ind{t})$ according to $\pihat\sups{t}$.
	\EndFor{}
	\end{algorithmic}
      \end{algbox}
      Two features of the algorithm deserve comment. First, the
      regression target at layer $h$ is $r_h+\max_a\Qbar_{h+1}\ind{t}(s_{h+1},a)$,
      the empirical Bellman backup of the optimistic value function
      computed one layer later; \cref{lem:lsvi_bellman_linear} shows that this can always be approximated by a linear function.
      Second, the bonus
      $b\subs{h,\delta}\ind{t}$ is the elliptic bonus from \linucb
      (\cref{sec:linucb}), but now applied layer by layer.

The \lsviucb algorithm enjoys the following regret bound.      
      \begin{theorem}[Regret of \lsviucb]
        \label{prop:lsvi_ucb}
        Under the scaling assumption \cref{eq:ucbvi_scaling},
        let $\delta\in(0,1)$ and set $R=c\cdot{}d^2\log(2dHT/\delta)$
        for a sufficiently large numerical constant $c$.
        Then, with probability at least $1-\delta$, \lsviucb satisfies
        \begin{align}
          \label{eq:lsvi_ucb_regret}
          \Reg \approxleq H\sqrt{d^3T}\log(2dHT/\delta).
        \end{align}
      \end{theorem}
The bound is polynomial in $H$, $d$, and $\log(1/\delta)$,
and independent of the number of states. The reason the horizon enters
only polynomially, in contrast with \cref{prop:realizability_insufficient_linear},
is that low-rank structure controls how regression errors propagate:
the error at layer $h$ is a linear function of $\phi$, so it is
confined to the $d$ directions the data have not yet covered. The
dependence on $d$ and $H$ in \cref{eq:lsvi_ucb_regret} is not optimal, but can be sharpened with similar techniques \citep{agarwal2023vo}.

      \subsubsection{Proof of \cref*{prop:lsvi_ucb}}
      The starting point of our analysis for \ucbvi was
      \pref{lem:reg_decomp_optimistic}, which states that it is
      sufficient to construct optimistic estimates
      $\{\Qbar_1,\ldots,\Qbar_H\}$ (i.e. $\Qstar_h \leq \Qbar_h$)
      such that the Bellman residuals $\En^{\sss{\Mstar},
        \pihat}\brk*{(\Qbar_h- \cT_h^{\sss{\Mstar}} \Qbar_{h+1})(s_h,
        a_h)}$ are small under the greedy (with respect to $\Qbar$'s)
      policy $\pihat$. In order to control these residuals, we
      constructed an estimated model $\Mhat$ and defined empirical
      Bellman operators $\cT_h^{\sss{\Mhat}}$ in terms of estimated
      transition kernels. We then set $\Qbar_h$ to be the empirical
      Bellman backup $\cT_h^{\sss{\Mhat}}\Qbar_{h+1}$, plus an
      optimistic bonus term. In contrast, \lsviucb does not directly estimate the model. Instead, it performs regression with a target that is an empirical Bellman backup. Subtleties arise in the analysis of this regression step due to lack of independence.

\paragraph{Technical lemmas for regression}

Recall from the definition of the Bellman operator (\cref{sec:mdp_planning}) that for any fixed $Q:\cS\times\cA\to\reals$,
  \begin{align}
	  \En^{\sss{\Mstar}}\brk*{r_h\ind{i} + \max_{a}Q(s_{h+1}\ind{i}, a) \mid s\ind{i}_h,a\ind{i}_h} = \brk*{\cTm[\Mstar]_h Q}(s_h\ind{i},a_h\ind{i}).
  \end{align}
However, for layer $h$, the regression problem within \lsviucb concerns a
\emph{data-dependent} function $Q=\Qbar_{h+1}\ind{t}$ (with $i<t$),
which is chosen as a function of all the trajectories
$\tau\ind{1},\ldots,\tau\ind{t-1}$. This dependence implies that the
regression problem solved by \lsviucb is not of the type studied, say,
in \pref{prop:iid_finite_class}. Instead, in the language of
\pref{sec:sl}, the mean of the outcome variable is itself a function
that depends on all the data. The saving grace here is that this
dependence does not result in arbitrarily complex functions, which
allows us to appeal to uniform convergence arguments. In
particular, for every $h$ and $t$, $\Qbar\ind{t}_h$ belongs to the class
\begin{align}
	\label{eq:lsvi_set_Qs}
        \wb{\cQ} \ldef \left\{\begin{aligned}
          (s,a)\mapsto{}&\max\crl*{0,\min\crl*{1,
          \inner{\theta,\phi(s,a)}+\sqrt{R}\nrm*{\phi(s,a)}_{(\wb{\Sigma})^{-1}}}}:\\
          &\norm{\theta}\leq 2\sqrt{d},\quad
          \sigma_{\textsf{min}}(\wb{\Sigma})\geq 1
        \end{aligned}\right\}.
\end{align}
To make use of this fact, we first state a result on least
squares with adaptively collected covariates and data-dependent targets.
It comes in two parts: a finite-class version, which we prove in full,
and the version for the classes at hand, whose proof is a covering
argument that we only sketch.
\begin{lem}[Least squares with data-dependent targets]
	\label{lem:uniform_regression}
	Let $(\filt_i)_{i\geq{}0}$ be a filtration, and let
        $x_1,\ldots,x_T\in\cX$ be such that $x_i$ is
        $\filt_{i-1}$-measurable. Let $\cG$ be an index set, and for
        each $g\in\cG$ let $y_1(g),\ldots,y_T(g)\in\reals$ be outcomes
        such that $y_i(g)$ is $\filt_i$-measurable and, for a function
        $f_g\in\cF\subseteq\crl{f:\cX\to\reals}$ and a constant $R\geq{}1$,
        \begin{equation}
          \label{eq:uniform_regression_model}
          \En\brk*{y_i(g)\mid\filt_{i-1}}=f_g(x_i)
          \quad\text{and}\quad
          \abs*{y_i(g)-f_g(x_i)}\leq{}R\quad\text{almost surely.}
        \end{equation}
        The random variables $\crl{y_i(g)}_{g\in\cG}$ may be
        correlated. For each $g\in\cG$, define the least-squares
        solution
	$$\fhat_g \in \argmin_{f\in\cF}\sum_{i=1}^T (y_i(g) - f(x_i))^2.$$
        \begin{enumerate}
        \item If $\cF$ and $\cG$ are finite, then with probability at
          least $1-\delta$, simultaneously for all $g\in\cG$,
          \begin{equation}
            \label{eq:uniform_regression_finite}
            \sum_{i=1}^T (\fhat_g(x_i) - f_g(x_i))^2 \leq{}8R^2\log(\abs{\cF}\abs{\cG}/\delta).
          \end{equation}
        \item Suppose instead that $\cX=\sB_2^d(1)$, that
          $\cF=\crl{x\mapsto\tri{\theta,x}:\nrm{\theta}\leq{}2\sqrt{d}}$,
          and that $\cG=\wb{\cQ}$ is the class in \cref{eq:lsvi_set_Qs},
          with the maps $g\mapsto{}y_i(g)$ and $g\mapsto{}f_g$ both
          $1$-Lipschitz with respect to the sup-norm on $\wb{\cQ}$. Then
          with probability at least $1-\delta$, simultaneously for all
          $g\in\cG$,
          \begin{equation}
            \label{eq:uniform_regression_linear}
            \sum_{i=1}^T (\fhat_g(x_i) - f_g(x_i))^2 \approxleq{}R^2d^2\log(dT/\delta).
          \end{equation}
        \end{enumerate}
\end{lem}
\begin{proof}[\pfref{lem:uniform_regression}]
  We prove Part 1 in full. Fix $g\in\cG$ and $f\in\cF$, and
  write $\eps_i\ldef{}y_i(g)-f_g(x_i)$ for the noise and
  $v_i\ldef{}f(x_i)-f_g(x_i)$ for the difference of the two functions
  at $x_i$. By \cref{eq:uniform_regression_model}, $\eps_i$ is
  conditionally mean zero and bounded by $R$ given $\filt_{i-1}$, while
  $v_i$ is $\filt_{i-1}$-measurable. Hoeffding's lemma
  (\cref{eq:hoeffding_mgf}) therefore gives, for every $\lambda>0$,
  \[
    \log\En_{i-1}\brk*{e^{\lambda\eps_iv_i}}\leq\frac{\lambda^2R^2v_i^2}{2}.
  \]
  Applying \cref{lem:martingale_chernoff} to the sequence
  $\lambda\eps_iv_i$, we conclude that with probability at least
  $1-\delta'$,
  \begin{equation}
    \label{eq:uniform_regression_chernoff}
    \sum_{i=1}^{T}\eps_iv_i
    \leq\frac{\lambda{}R^2}{2}\sum_{i=1}^{T}v_i^2+\frac{\log(1/\delta')}{\lambda}.
  \end{equation}
  Taking $\delta'=\delta/(\abs{\cF}\abs{\cG})$ and a union bound,
  \cref{eq:uniform_regression_chernoff} holds simultaneously for all
  $f\in\cF$ and $g\in\cG$ with probability at least $1-\delta$. On this
  event, fix $g$ and apply \cref{eq:uniform_regression_chernoff} with
  $f=\fhat_g$, which is admissible because the bound holds for every
  $f\in\cF$ at once. Optimality of $\fhat_g$ gives
  \[
    \sum_{i=1}^T (y_i(g) - \fhat_g(x_i))^2 \leq \sum_{i=1}^T (y_i(g) - f_g(x_i))^2,
  \]
  and expanding the squares with $y_i(g)=f_g(x_i)+\eps_i$ yields the
  basic inequality
  \[
    \sum_{i=1}^{T}(\fhat_g(x_i)-f_g(x_i))^2
    \leq{}2\sum_{i=1}^{T}\eps_i\,(\fhat_g(x_i)-f_g(x_i)).
  \]
  Combining the two displays with $\lambda=1/(2R^2)$,
  \[
    \sum_{i=1}^{T}(\fhat_g(x_i)-f_g(x_i))^2
    \leq\frac12\sum_{i=1}^{T}(\fhat_g(x_i)-f_g(x_i))^2
    +4R^2\log(\abs{\cF}\abs{\cG}/\delta),
  \]
  and rearranging proves \cref{eq:uniform_regression_finite}.

  For Part 2 we give only a sketch. Let $\cF'$ be a cover of
  $\cF$ at scale $1/T$ in the sup-norm on $\sB_2^d(1)$ and $\cG'$ a cover
  of $\wb{\cQ}$ at scale $1/T$ in the sup-norm; standard volumetric
  arguments give $\log\abs{\cF'}\approxleq{}d\log(dT)$, since $\cF$ is
  parameterized by $\theta\in\bbR^d$, and
  $\log\abs{\cG'}\approxleq{}d^2\log(dT)$, since $\wb{\cQ}$ is
  parameterized by $\theta\in\bbR^{d}$ and a matrix $\wb{\Sigma}\in\bbR^{d\times{}d}$
  with $\sigma_{\textsf{min}}(\wb{\Sigma})\geq{}1$.
  Clipping to $[0,1]$ does not increase the sup-norm distance
  between functions, so the same covering bound applies to the clipped class.
  Part 1 applied to the
  covers gives \cref{eq:uniform_regression_finite} with
  $\log(\abs{\cF'}\abs{\cG'}/\delta)\approxleq{}d^2\log(dT/\delta)$. Passing
  from the covers back to $\cF$ and $\wb{\cQ}$ costs an additive constant:
  by the Lipschitz assumptions, replacing $g$ by its nearest cover point
  changes every outcome $y_i(g)$ and every value $f_g(x_i)$ by at most
  $1/T$, and replacing $\fhat_g$ by its nearest cover point changes each
  value by at most $1/T$, so that all sums of squares over $T$ terms
  change by $O(1)$. We omit the details.
\end{proof}	

We may now apply \pref{lem:uniform_regression} to analyze the regression step of \lsviucb. 

\begin{lem}[Confidence bound for \lsviucb regression]
  \label{lem:lsvi_ucb_confidence}
With probability at least $1-\delta$, we have that for all $t$ and
$h$,
\begin{align}
  \label{eq:19}
  \sum_{i<t}\prn*{\Qhat\ind{t}_h(s_h\ind{i},a_h\ind{i})-\brk*{\cTm[\Mstar]_h\Qbar_{h+1}\ind{t}}(s_h\ind{i},a_h\ind{i})}^2
  \approxleq d^2\log(2dHT/\delta).
\end{align}
\end{lem}
\begin{proof}[\pfref{lem:lsvi_ucb_confidence}]
  Fix $t\in\brk{T}$ and $h\in\brk{H}$. We apply Part 2 of
  \cref{lem:uniform_regression} with $T$ replaced by $t-1$, then take a
  union bound over $t$ and $h$.

  Let $\filt_{i}$ be the sigma-field
  generated by the first $i$ trajectories together with
  $(s_h\ind{i+1},a_h\ind{i+1})$, and for $Q\in\wb{\cQ}$ set
  \[
    x_i=\phi(s_h\ind{i},a_h\ind{i}),
    \qquad
    y_i(Q)=r_h\ind{i}+\max_{a}Q(s_{h+1}\ind{i},a).
  \]
  Then $x_i$ is $\filt_{i-1}$-measurable and $y_i(Q)$ is
  $\filt_i$-measurable, and by the Markov property,
  \[
    \En\brk*{y_i(Q)\mid\filt_{i-1}}
    =\brk*{\cTm[\Mstar]_hQ}(s_h\ind{i},a_h\ind{i})
    =\tri*{x_i,\theta\subs{Q}\sups{\Mstar}},
  \]
  where $\nrm{\theta\subs{Q}\sups{\Mstar}}\leq{}2\sqrt{d}$ by
  \cref{lem:lsvi_bellman_linear}, since $Q$ takes values in $[0,1]$.
  Thus $f_Q\ldef\tri{\cdot,\theta\subs{Q}\sups{\Mstar}}$ lies in the
  linear class $\cF$ of Part 2.

  We now verify the remaining hypotheses of \cref{lem:uniform_regression}. Rewards and value functions lie in
  $[0,1]$, so $y_i(Q)$ and $f_Q(x_i)$ lie in $[0,2]$ and the noise is
  bounded by $R=2$. Both $Q\mapsto{}y_i(Q)$ and
  $Q\mapsto\cTm[\Mstar]_hQ$ are $1$-Lipschitz in the sup-norm.

  \emph{Conclusion.} The least-squares solution in \lsviucb is
  $\Qhat_h\ind{t}=\fhat_Q$ for $Q=\Qbar_{h+1}\ind{t}\in\wb{\cQ}$. This
  choice of $Q$ depends on the data, but Part 2 holds simultaneously for
  all $Q\in\wb{\cQ}$, so it applies. Taking a union bound
  over $t$ and $h$ gives \cref{eq:19}.
\end{proof}

\keepwithbox{16}
\paragraph{Establishing optimism}

The next lemma shows that closeness of the regression estimate to the
Bellman backup on the data $\crl*{(s\subs{h}\sups{i},
  a\subs{h}\sups{i})}_{i<t}$ translates into closeness at an arbitrary
$(s,a)$ pair as long as $\phi(s,a)$ is sufficiently covered by the
data collected so far. This, in turn, implies that
$\Qbar\ind{t}_1,\ldots,\Qbar\ind{t}_H$ are optimistic.
\begin{lem}[Optimism for \lsviucb]
  \label{lem:lsvi_ucb_optimism}
With $R$ chosen as in \cref{prop:lsvi_ucb},
whenever the event in \pref{lem:lsvi_ucb_confidence} occurs, we have
that for all $(s,a,h)$ and $t\in[T]$,
\begin{align}
  \label{eq:20}
  \abs*{\Qhat_h\ind{t}(s,a)-\brk*{\cTm[\Mstar]_h\Qbar_{h+1}\ind{t}}(s,a)}
  \leq{} \sqrt{R}\cdot{}\nrm*{\phi(s,a)}_{(\wb{\Sigma}_h\ind{t})^{-1}} = b\subs{h,\delta}\ind{t}(s,a).
\end{align}
and
\begin{align}
	\label{eq:lsvi_optimistic_Q}
  \Qbar_h\ind{t}(s,a) \geq{} \Qstar_h(s,a).
\end{align}
\end{lem}
\begin{proof}[\pfref{lem:lsvi_ucb_optimism}]
Writing the Bellman backup, via \pref{lem:lsvi_bellman_linear}, as
  \begin{align*}
	  \brk*{\cTm[\Mstar]_h\Qbar_{h+1}\ind{t}}(s,a) = \tri*{\phi(s,a), \theta_h\ind{t}} 
  \end{align*}
  for some $\theta_h\ind{t}\in\bbR^d$ with $\nrm{\theta_h\ind{t}}_2\leq2\sqrt{d}$, we have that
	\begin{align*}
		\abs*{\Qhat_h\ind{t}(s,a)-\brk*{\cTm[\Mstar]_h\Qbar_{h+1}\ind{t}}(s,a)} &= \abs*{\inner{\phi(s,a),\thetahat_h\ind{t}-\theta_h\ind{t}}} \\
		&= \abs*{\inner{(\wb{\Sigma}_h\ind{t})^{-1/2}\phi(s,a),(\wb{\Sigma}_h\ind{t})^{1/2}(\thetahat_h\ind{t}-\theta_h\ind{t})}}\\
		&\leq \nrm[\big]{\phi(s,a)}_{(\wb{\Sigma}_h\ind{t})^{-1}} \cdot \nrm[\big]{\thetahat_h\ind{t}-\theta_h\ind{t}}_{\wb{\Sigma}_h\ind{t}}.
	\end{align*}
	\pref{lem:lsvi_ucb_confidence} then implies \eqref{eq:20}, since
	\begin{align}
		\nrm[\big]{\thetahat_h\ind{t}-\theta_h\ind{t}}_{\wb{\Sigma}_h\ind{t}}^2 &= 	(\thetahat_h\ind{t}-\theta_h\ind{t})^\tr \prn*{\sum_{i<t}\phi(s_h\ind{i},a_h\ind{i})
        \phi(s_h\ind{i},a_h\ind{i})^{\trn} + I}(\thetahat_h\ind{t}-\theta_h\ind{t}) \\
		&= \sum_{i<t}\prn*{\Qhat\ind{t}_h(s_h\ind{i},a_h\ind{i})-\brk*{\cTm[\Mstar]_h\Qbar_{h+1}\ind{t}}(s_h\ind{i},a_h\ind{i})}^2 +  \nrm[\big]{\thetahat_h\ind{t}-\theta_h\ind{t}}^2
	\end{align}
	and $\nrm[\big]{\thetahat_h\ind{t}-\theta_h\ind{t}}^2\lesssim d$ by \eqref{eq:norm_of_bellman_backups}.  
	
	To show \eqref{eq:lsvi_optimistic_Q}, we proceed by induction on $\Vbar\ind{t}_h \geq \Vstar_h$, as in the proof of \pref{lem:ucb_vi_optimism}.
	We start with the base case $h=H+1$, which has $\Vbar\ind{t}_{H+1} = \Vstar_{H+1} \equiv 0$. Assuming $\Vbar\ind{t}_{h+1} \geq \Vstar_{h+1}$, we first observe that $\cTm[\Mstar]_h$ is monotone and $\cTm[\Mstar]_h\Vbar_{h+1}\ind{t} \geq \cTm[\Mstar]_h\Vstar_{h+1} = \Qstar_{h}$. Hence,
	\begin{align}
		\Qhat_h\ind{t} &=  \Qhat_h\ind{t} - \cTm[\Mstar]_h\Vbar_{h+1} + \cTm[\Mstar]_h\Vbar_{h+1} \\
		&\geq \Qhat_h\ind{t} - \cTm[\Mstar]_h\Vbar_{h+1} + \Qstar_{h} \\
		&\geq -b\subs{h,\delta}\ind{t} + \Qstar_{h}
	\end{align}
	and thus $\Qhat_h\ind{t}+ b\subs{h,\delta}\ind{t} \geq \Qstar_{h}$.
        Since $\Qstar_h\in[0,1]$, clipping to $[0,1]$ preserves
        this inequality, so $\Qbar_h\ind{t}\geq\Qstar_h$. In particular,
        the lower clipping is inactive on this event.
        This, in turn, implies $\Vbar_h\ind{t}\geq \Vstar_h$.
\end{proof}

\paragraph{Finishing the proof}
With the technical results above established, the proof of \cref{prop:lsvi_ucb} follows fairly quickly.
      
\begin{proof}[\pfref{prop:lsvi_ucb}]
  Apply \pref{lem:lsvi_ucb_confidence} with failure
  probability $\delta/2$; the choice of $R$ in \cref{prop:lsvi_ucb}
  accommodates this change by increasing $c$. On this event,
  since $\Qbar$ is optimistic by \pref{lem:lsvi_ucb_optimism}, we have that for each timestep $t$,
  \begin{align*}
    \fstar(\pistar) - \fstar(\pihat\ind{t})
    &\leq{}
    \En_{s_1\sim{}d_1}\brk*{\Vbar_1\ind{t}(s_1)} - \fstar(\pihat\ind{t}) \\
          &=  \sum_{h=1}^{H}\Enm{\Mstar}{\pihat\ind{t}}\brk*{\Qbar\ind{t}_h(s_h,a_h) - \brk*{\cTm[\Mstar]_h\Qbar_{h+1}\ind{t}}(s_h,a_h)}
  \end{align*}
  by \pref{lem:bellman_residual}. Using the definition of
  $\bhdelta\ind{t}$ and \pref{lem:lsvi_ucb_optimism}, we have
  \begin{align*}
    \sum_{h=1}^{H}\Enm{\Mstar}{\pihat\ind{t}}\brk*{\Qbar\ind{t}_h(s_h,a_h) -
    \brk*{\cTm[\Mstar]_h\Qbar_{h+1}\ind{t}}(s_h,a_h)}
    \approxleq \sqrt{R}\sum_{h=1}^{H}\Enm{\Mstar}{\pihat\ind{t}}\brk*{\nrm*{\phi(s_h,a_h)}_{(\wb{\Sigma}_h\ind{t})^{-1}}}.
  \end{align*}
  Summing over all timesteps $t$ gives
  \begin{align*}
    \Reg
    \mathrel{\approxleq} \sqrt{R}\sum_{t=1}^{T}\sum_{h=1}^{H}\Enm{\Mstar}{\pihat\ind{t}}\brk*{\nrm*{\phi(s_h,a_h)}_{(\wb{\Sigma}_h\ind{t})^{-1}}}.
  \end{align*}
  Each feature norm is at most one, so their sum over an
  episode lies in $[0,H]$. Azuma's inequality (the martingale form of
  \cref{lem:hoeffding}) implies that, with probability at least
  $1-\delta/2$, the preceding bound is at most a constant times
  \begin{align*}
    \sqrt{R}\sum_{t=1}^{T}\sum_{h=1}^{H}\nrm*{\phi(s\ind{t}_h,a\ind{t}_h)}_{(\wb{\Sigma}_h\ind{t})^{-1}}
    + H\sqrt{RT\log(2/\delta)}.
  \end{align*}
  The elliptic potential lemma (\pref{lem:elliptic_potential}) now allows
  us to bound
  \begin{align*}
    \sum_{t=1}^{T}\nrm*{\phi(s\ind{t}_h,a\ind{t}_h)}_{(\wb{\Sigma}_h\ind{t})^{-1}}
    \approxleq \sqrt{dT\log(1+T/d)}
  \end{align*}
  for each $h$. Substituting $R=c\cdot{}d^2\log(2dHT/\delta)$
  and bounding both $\log(1+T/d)$ and $\log(2/\delta)$ by
  $\log(2dHT/\delta)$ gives \cref{eq:lsvi_ucb_regret}.
  A union bound over the two events gives the claimed probability.

      \end{proof}

\subsection{Bellman Rank and \bilinucb}
\label{sec:rl_bellman_rank}
In this section, we continue our study of value-based methods, which
assume access to a class $\cQ$ of state-action value functions such that
$\Qstar\in\cQ$. In \cref{sec:rl_linear}, we saw that the Low-Rank MDP
assumption facilitates sample-efficient reinforcement learning when
$\cQ$ is a class of linear functions, but what if we want
to learn with \emph{nonlinear} functions such as neural networks? To
this end, we introduce a structural property, \emph{Bellman
  rank} \citep{jiang2017contextual,du2021bilinear,jin2021bellman}, which allows for sample-efficient learning with general
classes $\cQ$, and subsumes a number of well-studied MDP families, including:
\begin{multicols}{2}
  \begin{itemize}
  \item Low-Rank MDPs \citep{yang2019sample,jin2020provably,agarwal2020flambe}
  \item Block MDPs and reactive POMDPs
  \citep{du2019latent}.
  \item MDPs with Linear $Q^{\star}$ and $V^{\star}$ \citep{du2021bilinear}.
  \item MDPs with low occupancy complexity \citep{du2021bilinear}.
  \item Linear mixture MDPs
  \citep{modi2020sample,ayoub2020model}.
  \item Linear dynamical systems (LQR) \citep{dean2020sample}.
  \end{itemize}
\end{multicols}
\noindent \cref{sec:bellman_rank_example} defines and works through several of these examples.

\paragraph{Building intuition}%
\newcommand{\thetaq}{\theta\sups{Q}}%
\newcommand{\wq}{w\sups{Q}}%
Bellman rank is a property of the underlying MDP $\Mstar$ which gives
a way of controlling \emph{distribution shift}---that is, how many
times a deliberate algorithm can be surprised by a substantially new
state distribution $d^{\sss{M},\pi}$ when it updates its policy. To motivate the property, let us revisit the low-rank MDP
model. Let $M$ be a low-rank MDP with feature map
$\phi(s,a)\in\bbR^{d}$, and let
$Q_h(s,a)=\tri*{\phi(s,a),\thetaq_h}$ be an arbitrary linear value
function. Observe that since $M$ is a Low-Rank MDP, we have
$\brk*{\cTm_h{}Q}(s,a)=\tri*{\phi(s,a),\tilde{\theta}_h\sups{M,Q}}$,
where
$\tilde{\theta}_h\sups{M,Q}\ldef{}\wm_h+\int\mum_h(s')\max_{a'}Q_{h+1}(s',a')ds'$. As
a result, for any policy $\pi$, we can write the Bellman residual for $Q$ as
\begin{align}
  \Enm{M}{\pi}\brk*{Q_h(s_h,a_h)-r_h-\max_{a}Q_{h+1}(s_{h+1},a)}
  &=
    \tri*{\Enm{M}{\pi}\brk*{\phi(s_h,a_h)},\thetaq_h-\tilde{\theta}_h\sups{M,Q}}\\
  &=   \tri*{\Xm_h(\pi),\Wm_h(Q)},
    \label{eq:low_rank_mdp_bellman}
\end{align}
where $\Xm_h(\pi)\ldef{}\Enm{M}{\pi}\brk*{\phi(s_h,a_h)}\in\bbR^{d}$ is
  an ``embedding'' that depends on $\pi$ but not $Q$, and
  $\Wm_h(Q)\ldef{}\thetaq_h-\tilde{\theta}_h\sups{M,Q}\in\bbR^{d}$
  is an embedding that depends on $Q$ but not $\pi$ (both embeddings
  depend on $M$). Notably, if we view the Bellman residual as a
  huge $\Pi\times\cQ$ matrix $\cE_h(\cdot,\cdot)\in\bbR^{\Pi\times\cQ}$ (we write $\cE^{\sss{M}}_h$ when the dependence on the MDP $M$ needs to be explicit) with 
        \begin{equation}
          \cE_h(\pi,Q) \ldef
          \Enm{M}{\pi}\brk*{Q_h(s_h,a_h)-\prn*{r_h+\max_{a}Q_{h+1}(s_{h+1},a)}},
          \label{eq:bellman_residual}
        \end{equation}
        then the property \cref{eq:low_rank_mdp_bellman} implies that $\rank(\cE_h(\cdot,\cdot))\leq{}d$. Bellman rank is
        an abstraction of this property.\footnote{Bellman rank was
          originally introduced in the pioneering work of \citet{jiang2017contextual}. The definition of
          Bellman rank we present, which is slightly different from the
          original definition, is taken from the later work of
          \citet{du2021bilinear,jin2021bellman}, and is often referred
          to as \emph{Q-type} Bellman rank.
        }
        \begin{definition}[Bellman rank]
          For an MDP $M$ with value function class $\cQ$ and policy
          class $\Pi$, the Bellman rank is defined
          as
          \begin{align}
            \label{eq:bellman_rank}
            \dimbel(M)=\max_{h\in\brk{H}}\rank(\crl{\cE_h(\pi,Q)}_{\pi\in\Pi,Q\in\cQ}).
          \end{align}
          Equivalently, Bellman rank is the smallest dimension $d$
          such that for all $h$, there exist embeddings
          $\Xm_h(\pi),\Wm_h(Q)\in\bbR^{d}$ such that
          \begin{equation}
            \label{eq:bellman_rank_factorization}
            \cE_h(\pi,Q)=\tri*{\Xm_h(\pi),\Wm_h(Q)}
            \end{equation}
          for all $\pi\in\Pi$
          and $Q\in\cQ$.
        \end{definition}
        The factorization in \pref{eq:bellman_rank_factorization}
        is what controls distribution shift. Informaly, it says that there are only $d$ directions
        in which the state distribution of a new policy can surprise the
        learner, as far as the class $\cQ$ is concerned. This property
        was used implicitly in the analysis of \lsviucb, and holds in
        many settings beyond the low-rank MDP model as we will soon see.\loose

  \subsubsection{The BiLinUCB Algorithm}
  \newcommand{\cEhat}{\wh{\cE}}

We now present an algorithm, \bilinucb \citep{du2021bilinear}, which
attains low regret for MDPs with low Bellman rank under the
realizability assumption that
\[
\Qstar\in\cQ.
\]
Like many of the
algorithms we have covered, \bilinucb is based on confidence sets and
optimism, though the way we construct the confidence sets and
implement optimism is new.

\paragraph{PAC versus regret}
For technical reasons, we do not directly give a regret bound for
\bilinucb. Instead, we prove a \emph{PAC} (``probably
approximately correct'') guarantee. For PAC, the algorithm plays for
$T$ episodes, then outputs a final policy $\pihat$, and its
performance is measured via
\begin{align}
  \label{eq:pac}
  \fstar(\pistar) - \fstar(\pihat).
\end{align}
That is, instead of considering cumulative performance as with regret,
we are only concerned with final performance. For PAC, we want to
ensure that $\fstar(\pistar) - \fstar(\pihat)\leq\veps$ for some
$\veps\ll 1$ using a number of
episodes that is polynomial in $\veps^{-1}$ and other problem
parameters. This is an easier task than achieving low regret: If we have an algorithm
that ensures that $\En\brk*{\Reg}\approxleq{}\sqrt{CT}$ for some
problem-dependent constant $C$, we can turn this into an algorithm that
achieves PAC error $\veps$ using $\bigoh\prn*{\frac{C}{\veps^2}}$ episodes
via online-to-batch conversion. In the other direction, if we have an
algorithm that achieves PAC error $\veps$ using
$\bigoh\prn*{\frac{C}{\veps^2}}$ episodes, one can use this to achieve
$\En\brk*{\Reg}\approxleq{}C^{1/3}T^{2/3}$ using a simple
explore-then-commit approach; this is lossy, but is the best one can
hope for in general.

\paragraph{Algorithm overview}
\bilinucb proceeds in $K$ iterations, each of which consists of $n$
episodes. The algorithm maintains a confidence set
$\cQ\ind{k}\subseteq\cQ$ of value functions (generalizing the
confidence sets we constructed for structured bandits in
\cref{sec:structured}), with the property that  $\Qstar\in\cQ$ with high
probability. Each iteration $k$ consists of two parts:
\begin{itemize}
\item Given the current confidence set $\cQ\ind{k}$, the algorithm
  computes a value function $Q\ind{k}$ and corresponding policy
  $\pi\ind{k}\ldef\pi\subs{Q\ind{k}}$ that is \emph{optimistic on
    average}
  \[
    Q\ind{k}=\argmax_{Q\in\cQ\ind{k}}\En_{s_1\sim{}d_1}\brk*{Q_1(s_1,\piq(s_1))}.
  \]
  The main novelty here is that we are only aiming for optimism with
  respect to the initial state distribution.
\item Using the new policy $\pi\ind{k}$, the algorithm gathers $n$
  episodes and uses these to compute estimators $\crl{\cEhat_h\ind{k}(Q)}_{h\in\brk{H}}$
  which approximate the Bellman residual $\cE_h(\pi\ind{k},Q)$
  for all $Q\in\cQ$. Then, in \pref{eq:bilinucb_confidence_set}, the algorithm computes the new confidence
  set $\cQ\ind{k+1}$ by restricting to value functions for which the
  estimated Bellman residual is small for
  $\pi\ind{1},\ldots,\pi\ind{k}$. Eliminating value functions with large
  Bellman residual is a natural idea, because we know from the Bellman
  equation that $\Qstar$
  has zero Bellman residual.
\end{itemize}

\begin{algbox}[frametitle={\algtitle{alg:bilinucb}{\bilinucb}}]
	\begin{algorithmic}
          \Require $\beta>0$; iteration count $K\in\bbN$; batch size $n\in\bbN$.
          \Ensure Policy $\pihat$.
          \State $\cQ\ind{1}\gets\cQ$.
	\For{iteration $k=1,\ldots,K$}
        \State \multiline{Compute optimistic value function:
        \[
          Q\ind{k}=\argmax_{Q\in\cQ\ind{k}}\En_{s_1\sim{}d_1}\brk*{Q_1(s_1,\piq(s_1))}.
        \]
        and let $\pi\ind{k}\ldef{}\pi\subs{Q\ind{k}}$.}
      \For{$l=1,\ldots,n$}
      \State Execute $\pi\ind{k}$ for an episode
      and observe trajectory
      $(s_1\ind{k,l},a_1\ind{k,l},r_1\ind{k,l}),\ldots,
      (s_H\ind{k,l},a_H\ind{k,l},r_H\ind{k,l})$.
      \EndFor{}
      \State \multiline{Compute confidence set
        \begin{align}
          \cQ\ind{k+1}=\crl*{Q\in\cQ\mid{}
          \sum_{i\leq{}k}(\cEhat_h\ind{i}(Q))^2 \leq
          \beta^2\quad\forall{}h\in\brk{H}},
          \label{eq:bilinucb_confidence_set}
        \end{align}
        where
        \[
          \cEhat_h\ind{i}(Q)\ldef{}\frac{1}{n}\sum_{l=1}^{n}\prn*{Q_h(s_h\ind{i,l},a_h\ind{i,l})-r_h\ind{i,l}
          - \max_{a\in\cA}Q_{h+1}(s_{h+1}\ind{i,l},a)}.
        \]}
      \EndFor{}
      \State Let $\wh{k}=\argmax_{k\in\brk{K}}\wh{V}\ind{k}$, where
      $\wh{V}\ind{k}\ldef{}\frac{1}{n}\sum_{l=1}^{n}\sum_{h=1}^{H}r_h\ind{k,l}$.
      \State \Return $\pihat=\pi\ind{\wh{k}}$.
	\end{algorithmic}
\end{algbox}

\paragraph{Main guarantee}
The main result for this section is the following PAC guarantee for \bilinucb.
\begin{theorem}[Sample complexity of \bilinucb]
  \label{prop:bilinucb_regret}
  Suppose rewards are nonnegative with $\sum_{h=1}^{H}r_h\leq 1$ almost
  surely, and let $\cQ$ be a finite class with $\Qstar\in\cQ$,
  $Q_h(s,a)\in[0,1]$, and $Q_{H+1}\equiv 0$ for every $Q\in\cQ$. Suppose $\Mstar$ has Bellman rank $\dimbel(\Mstar)\leq{} d$, with the corresponding factors bounded as $\nrm*{\Xmstar_h(\pi)}_2\leq B_X$ and
  $\nrm*{\Wmstar_h(Q)}_2\leq B_W$ for all $h\in[H]$,
  $\pi\in\Pi_{\cQ}$, and $Q\in\cQ$, where $B_X,B_W\geq 1$.

  Fix $\veps,\delta\in(0,1)$. For sufficiently large absolute constants
  $c_1,c_2,c_3>0$, set
  \[
    K=\ceil{c_1Hd\log\prn*{2+\frac{B_X^2B_W^2n}{d}}},
    \qquad
    \beta^2=c_2K\frac{\log(2\abs{\cQ}HK/\delta)}{n},
  \]
  and choose $n\geq c_3H^2K\log(2\abs{\cQ}HK/\delta)/\veps^2$.
  Then \bilinucb returns a policy $\pihat$ satisfying
  $\fstar(\pistar)-\fstar(\pihat)\leq\veps$ with probability at least
  $1-\delta$. In particular, choosing $n$ appropriately gives this guarantee using
  \[
    \bigoht\prn*{\frac{H^4d^2\log(2\abs{\cQ}/\delta)}{\veps^2}}
  \]
  episodes, where $\bigoht$ also suppresses logarithmic dependence on $B_XB_W$.
\end{theorem}

That is, low Bellman rank suffices
to learn a near-optimal policy, with sample complexity that depends polynomially
on $d$, $H$, and $\veps^{-1}$, and logarithmically on $\abs{\cQ}$. This reflects that \bilinucb is able to generalize across the state
space, with $d$ and $\log\abs{\cQ}$ controlling the degree of generalization. The basic principles at play are:
\begin{itemize}
\item By choosing $Q\ind{k}$ optimistically, we ensure that the
  suboptimality of the algorithm is controlled by the Bellman residual
  for $Q\ind{k}$ \emph{on-policy}, that is, under the state
  distribution of the policy the algorithm actually executes, similar
  to what we saw for \ucbvi and \lsviucb; as noted above,
  the optimism here is only on average over the initial state distribution. 
\item The confidence set construction
  \pref{eq:bilinucb_confidence_set} explicitly removes value
  functions that have large Bellman residual on the policies
  encountered so far. The key role of the Bellman rank property is to
  ensure that there are only $\bigoht(d)$ ``effective'' state
  distributions that lead to substantially different values for the
  Bellman residual, which means that eventually, only value functions
  with low residual will remain. The Bellman rank property is only used for analysis,
  and the algorithm does not explicitly compute or estimate the factorization.
\end{itemize}

For the low-rank MDP example above, take bounded linear
$Q$-functions with $\nrm*{\thetaq_h}_2\leq 2\sqrt d$.
Then $\Xm_h(\pi)=\Enm{M}{\pi}\brk*{\phi(s_h,a_h)}$ has norm at most
one, while $\Wm_h(Q)=\thetaq_h-\tilde\theta_h^{M,Q}$ has norm at most
$4\sqrt d$ by \pref{lem:lsvi_bellman_linear}.
Thus we can take $B_X=1$ and $B_W=4\sqrt d$, which affects only
logarithmic factors in the guarantee.

  \paragraph{Regret bounds}
The \bilinucb algorithm can be lifted to provide a regret guarantee via a 
explore-then-commit strategy:  Run the algorithm for $T_0$ episodes to
learn a $\veps$-optimal policy, then commit to this policy for the
remaining rounds. It is a simple exercise to show that by choosing
$T_0$ appropriately, this approach gives
\[
  \Reg \leq\bigoht\prn*{
    (H^4d^2\log(\abs{\cQ}/\delta))^{1/3}\cdot{}T^{2/3}
    }.
  \]
Here too, $\bigoht$ suppresses logarithmic dependence on $B_XB_W$.
Under an additional assumption known as \emph{Bellman completeness},
which asks that Bellman backups of functions in $\cQ$ remain in
$\cQ$, it is possible to attain $\sqrt{T}$ with a variant of this
algorithm that uses a slightly different confidence set construction
\citep{jin2021bellman}. Low-rank MDPs are Bellman complete
by \cref{lem:lsvi_bellman_linear}, which is why \lsviucb attains
$\sqrt{T}$ regret there.

One may ask why optimism is used at all when the goal is only
a PAC guarantee. At its core, optimism is a device for choosing informative
policies: the optimistic value function is either nearly correct, in
which case its greedy policy is nearly optimal, or wrong, in which case
executing its policy exposes the error. For general reinforcement learning (compared to say, bandits, where crude exploration suffices) a PAC learning algorithm needs this
property as much as a regret minimization algorithm does. The only difference is that a PAC learner may enjoy better quantitative guarantees, as it need not pay for exploration in every round.

\subsubsection{Proof of \cref*{prop:bilinucb_regret}}
Recall the factorization in \pref{prop:bilinucb_regret}:
\begin{align*}
  \cE_h(\pi,Q)=\tri*{\Xmstar_h(\pi),\Wmstar_h(Q)}
  \qquad\forall\pi\in\Pi_{\cQ},\;Q\in\cQ.
\end{align*}
For simplicity, we give the proof for $B_X=B_W=1$.\footnote{For general
bounds, the regularization term $\sqrt\lambda+\beta$ in
\pref{eq:bilinucb_intermediate} becomes $\sqrt\lambda B_W+\beta$, and the
potential logarithm becomes $\log(1+B_X^2K/(\lambda d))$.
Choosing $\lambda=\beta^2/B_W^2$ gives the stated dependence on $B_XB_W$.}

\paragraph{Technical lemmas}
Before proceeding, we state two technical
lemmas. The first lemma establishes validity for the confidence set
$\cQ\ind{k}$ constructed by \bilinucb.
\begin{lem}[Validity of the \bilinucb confidence sets]
  \label{lem:bilinucb_confidence}
For any $\delta>0$, if we set $\beta^2 =
c\cdot{}K\frac{\log\abs{\cQ}+\log(HK/\delta)}{n}$, where $c>0$ is
sufficiently large absolute constant, then with probability at least
$1-\delta$, for all $k\in\brk{K}$:
\begin{enumerate}
\item All $Q\in\cQ\ind{k}$ have
  \begin{equation}
    \label{eq:bilinucb_confidence1}
    \sum_{i<k}(\cE_h(\pi\ind{i},Q))^2 \approxleq\beta^2\quad\forall{}h\in\brk{H}.
  \end{equation}
\item $\Qstar\in\cQ\ind{k}$.
\end{enumerate}
  
\end{lem}
\begin{proof}[\pfref{lem:bilinucb_confidence}]
Using Hoeffding's inequality and a union bound (\pref{lem:hoeffding}),
we have that with probability at least $1-\delta$, for all
$k\in\brk{K}$, $h\in\brk{H}$, and $Q\in\cQ$,
\begin{align}
  \label{eq:bilinucb_conv}
  \abs*{\cEhat_h\ind{k}(Q)-\cE_h(\pi\ind{k},Q)} \leq C\cdot{} \sqrt{\frac{\log(\abs{\cQ}HK/\delta)}{n}},
\end{align}
where $C$ is an absolute constant.

To prove Part 1, we observe that for all $k$, using the AM-GM inequality, we have that for all $Q\in\cQ$,
\begin{align*}
  \sum_{i<k}(\cE_h(\pi\ind{i}, Q))^2
  &\leq 2\sum_{i<k}(\cEhat\ind{i}_h(Q))^2 + 2\sum_{i<k}(\cE_h(\pi\ind{i},Q) -
    \cEhat\ind{i}_h(Q))^2.
\end{align*}
For $Q\in\cQ\ind{k}$, the definition of $\cQ\ind{k}$ implies that
$\sum_{i<k}(\cEhat\ind{i}_h(Q))^2\leq\beta^2$, while
\pref{eq:bilinucb_conv} gives
\[
  \sum_{i<k}(\cE_h(\pi\ind{i},Q)-\cEhat_h\ind{i}(Q))^2
  \leq C^2K\frac{\log(\abs{\cQ}HK/\delta)}{n}
  \leq\beta^2,
\]
provided the constant $c$ in the definition of $\beta^2$ is at least $C^2$.
Combining these bounds proves Part 1.

For Part 2, Bellman optimality gives
$\cE_h(\pi,\Qstar)=0$ for every $\pi$. Thus, for all $k$ and $h$,
\begin{align*}
  \sum_{i<k}(\cEhat_h\ind{i}(\Qstar))^2
  &=\sum_{i<k}(\cEhat_h\ind{i}(\Qstar)-\cE_h(\pi\ind{i},\Qstar))^2 \\
  &\leq C^2(k-1)\frac{\log(\abs{\cQ}HK/\delta)}{n} \\
  &\leq C^2K\frac{\log(\abs{\cQ}HK/\delta)}{n}
  \leq\beta^2,
\end{align*}
where we used \pref{eq:bilinucb_conv} and the same choice of $c$.
Consequently, $\Qstar\in\cQ\ind{k}$ for all $k$, as required.
\end{proof}

The next result shows that whenever the event in the previous lemma
holds, the value functions constructed by \bilinucb are optimistic.
\begin{lem}[Consequences of the confidence event]
  \label{lem:bilinucb_good_event}
  Whenever the event in \pref{lem:bilinucb_confidence} occurs, the
  following properties hold:
  \begin{enumerate}
  \item Define
    \begin{align}
      \label{eq:bilinucb_sigma}
      \Sigma_h\ind{k}=\sum_{i<k}\Xmstar_h(\pi\ind{i})\Xmstar_h(\pi\ind{i})^{\trn}.
    \end{align}
    For all $k\in\brk{K}$, all $Q\in\cQ\ind{k}$ satisfy
    \begin{align}
      \label{eq:bilinucb_elliptic_bound}
            \nrm*{\Wmstar_h(Q)}_{\Sigma_h\ind{k}}^2\approxleq \beta^2.
    \end{align}
  \item For all $k$, $Q\ind{k}$ is optimistic in the sense that
    \begin{align}
      \label{eq:bilinucb_optimism}
      \En_{s_1\sim{}d_1}\brk*{Q\ind{k}_1(s_1,\piq(s_1))}
      \geq \En_{s_1\sim{}d_1}\brk[\big]{\Qstar_1(s_1,\pistar(s_1))}
      = \fstar(\pistar).
    \end{align}
  \end{enumerate}
  
\end{lem}

\begin{proof}[\pfref{lem:bilinucb_good_event}]
For Part 1, recall that by the Bilinear class property, we can write
$\cE_h(\pi\ind{k},Q)=\tri*{\Xmstar(\pi\ind{k}),\Wmstar(Q)}$, so that
\pref{eq:bilinucb_confidence1} implies that
\begin{align*}
  \nrm*{\Wmstar(Q)}_{\Sigma_h\ind{k}}^2
  = \sum_{i<k}\tri*{\Xmstar(\pi\ind{i}),\Wmstar(Q)}^2
  = \sum_{i<k}(\cE_h(\pi\ind{i},Q))^2 \approxleq\beta^2.
\end{align*}
For Part 2, we observe that for all $k$, since
$\Qstar\in\cQ\ind{k}$, we have
\begin{align*}
  \En_{s_1\sim{}d_1}\brk*{Q\ind{k}_1(s_1,\piq(s_1))}
    &= \sup_{Q\in\cQ^k}\En_{s_1\sim{}d_1}\brk[\big]{Q_1(s_1,\piq(s_1))} \\
  &\geq \En_{s_1\sim{}d_1}\brk[\big]{\Qstar_1(s_1,\pistar(s_1))}\\
  &= \fstar(\pistar).
\end{align*}
\end{proof}

\paragraph{Proving the main result}
Equipped with the lemmas above, we prove \cref{prop:bilinucb_regret}.
\begin{proof}[\pfref{prop:bilinucb_regret}]

We first prove a generic bound on the suboptimality of each policy
$\pi\ind{k}$ for $k\in\brk{K}$.
Let us
condition on the event in \pref{lem:bilinucb_confidence}, applied
with failure probability $\delta/2$. Whenever this event occurs,
\pref{lem:bilinucb_good_event} implies that $Q\ind{k}$ is optimistic,
so we can bound
\begin{align}
  \label{eq:29}
  \fstar(\pistar) - \fstar(\pi\ind{k})
  &\leq{}
  \En_{s_1\sim{}d_1}\brk*{Q_1\ind{k}(s_1,\pi\subs{Q\ind{k}}(s_1))} -
    \fstar(\pi\ind{k}) \\
    &=
      \sum_{h=1}^{H}\Enm{\Mstar}{\pi\ind{k}}\brk*{Q_h\ind{k}(s_h,a_h)
      - r_h - \max_{a\in\cA}Q_{h+1}\ind{k}(s_{h+1},a)}\\
      &=
        \sum_{h=1}^{H}\tri*{\Xmstar_h(\pi\ind{k}),\Wmstar_h(Q\ind{k})},
\end{align}
where the first equality uses the Bellman residual decomposition
(\pref{lem:bellman_residual}), and the second equality uses the
Bellman rank assumption. For any $\lambda\geq{}0$, using Cauchy-Schwarz, we can bound
\begin{align*}
  \sum_{h=1}^{H}\tri*{\Xmstar_h(\pi\ind{k}),\Wmstar_h(Q\ind{k})}
  &\leq{}
    \sum_{h=1}^{H}\nrm*{\Xmstar_h(\pi\ind{k})}_{(\lambda{}I+\Sigma_h\ind{k})^{-1}}\nrm*{\Wmstar_h(Q\ind{k})}_{\lambda{}I+\Sigma_h\ind{k}}
    \end{align*}
    For each $h\in\brk{H}$, applying the bound in
    \pref{eq:bilinucb_elliptic_bound}
    gives
    \[
      \nrm*{\Wmstar_h(Q\ind{k})}_{\lambda{}I+\Sigma_h\ind{k}}
      \leq{}\sqrt{\lambda\nrm*{\Wmstar_h(Q\ind{k})}_2^2 + \beta^2}
      \leq\lambda^{1/2}+\beta,
    \]
    where we have used that $\nrm*{\Wmstar_h(Q\ind{k})}_2\leq{}1$ by
    assumption. This allows us to bound

  \begin{align}
  \sum_{h=1}^{H}\tri*{\Xmstar_h(\pi\ind{k}),\Wmstar_h(Q\ind{k})}  &\approxleq{}
                                                                    (\lambda^{1/2}+\beta)\cdot{}\sum_{h=1}^{H}\nrm*{\Xmstar_h(\pi\ind{k})}_{(\lambda{}I+\Sigma_h\ind{k})^{-1}}.
    \label{eq:bilinucb_intermediate}
  \end{align}
  If we can find a policy $\pi\ind{k}$ for which the \rhs of
  \pref{eq:bilinucb_intermediate} is small, this policy will be
  guaranteed to have low regret. The following lemma shows that such a
  policy is guaranteed to exist.
\begin{lem}[Existence of a well-covered iteration]
  \label{lem:bilinucb_potential}
  For any $\lambda>0$, as long as $K\geq{}Hd \log
  \prn*{1+\lambda^{-1}K/d}$, there exists $k\in\brk{K}$ such that
    \begin{align}
    \label{eq:bilinucb_inverse}
      \nrm*{\Xmstar_h(\pi\ind{k})}_{(\lambda{}I+\Sigma_h\ind{k})^{-1}}^2
      \approxleq
      \frac{Hd \log \prn*{1+\lambda^{-1}K/d}}{K}\quad\forall{}h\in\brk{H}.
  \end{align}
\end{lem}
We choose $\lambda=\beta^2$, which implies that it suffices to take
$K\approxgeq{}Hd\log(1+n/d)$ to satisfy the condition in
\pref{lem:bilinucb_potential}. By choosing $k$ to satisfy \pref{eq:bilinucb_inverse} and plugging this
bound into \pref{eq:bilinucb_intermediate}, we conclude that the
policy $\pi\ind{k}$ has
\begin{align}
  \fstar(\pistar) - \fstar(\pi\ind{k}) &\approxleq{} H\beta\sqrt{\frac{Hd\log(1+\beta^{-2}K/d)}{K}}
    \approxleq \bigoht\prn*{H^{3/2}\sqrt{\frac{d\log(\abs{\cQ}/\delta)}{n}}}\approxleq\veps\label{eq:bilinucb_final}
\end{align}
as desired. 

Finally, we need to argue that the policy $\pihat$ returned by the
algorithm is at least as good as $\pi\ind{k}$. This is
straightforward and we only sketch the argument: By Hoeffding's inequality and a union bound, we
have that with probability at least $1-\delta/2$, for all $k$,
\[
  \abs*{\fstar(\pi\ind{k})-\Vhat\ind{k}}\approxleq \sqrt{\frac{\log(2K/\delta)}{n}},
\]
which implies that $\fstar(\pihat)\approxgeq\fstar(\pi\ind{k}) -
\sqrt{\frac{\log(2K/\delta)}{n}}$. The error term here is of lower
order than \pref{eq:bilinucb_final}.
A union bound gives overall success probability at least $1-\delta$.
  
\end{proof}

\paragraph{Proof of \cref*{lem:bilinucb_potential}}
\begin{proof}[\pfref{lem:bilinucb_potential}]
The proof uses the logarithmic form of the elliptic potential
lemma, \cref{eq:elliptic_potential_log} in \cref{lem:elliptic_potential}.
      For any $\lambda>0$, applying \cref{eq:elliptic_potential_log} for each $h\in\brk{H}$
      and summing gives
      \begin{align*}
        \sum_{k=1}^K\sum_{h=1}^{H} \log\prn[\big]{1 + \norm{\Xmstar_h(\pi\ind{k})}^2_{(\lambda{}I+\Sigma_h\ind{k})^{-1}}} \leq
        Hd \log \prn*{1+\lambda^{-1}K/d}.
      \end{align*}
      This implies that there exists $k$ such that
      \[
\sum_{h=1}^{H} \log\prn[\big]{1 + \norm{\Xmstar_h(\pi\ind{k})}^2_{(\lambda{}I+\Sigma_h\ind{k})^{-1}}} \leq
\frac{Hd \log \prn*{1+\lambda^{-1}K/d}}{K},
\]
which means that for all $h\in\brk{H}$, $\log\prn[\big]{1 + \norm{\Xmstar_h(\pi\ind{k})}^2_{(\lambda{}I+\Sigma_h\ind{k})^{-1}}} \leq
\frac{Hd \log \prn*{1+\lambda^{-1}K/d}}{K}$, or equivalently:
\[
\norm{\Xmstar_h(\pi\ind{k})}^2_{(\lambda{}I+\Sigma_h\ind{k})^{-1}} \leq
\exp\prn*{\frac{Hd \log \prn*{1+\lambda^{-1}K/d}}{K}}-1.
\]
As long as $K\geq{}Hd \log \prn*{1+\lambda^{-1}K/d}$, using that
$e^{x}\leq{}1+2x$ for $0\leq{}x\leq{}1$, we have
\begin{align*}
  \norm{\Xmstar_h(\pi\ind{k})}^2_{(\lambda{}I+\Sigma_h\ind{k})^{-1}}
  \leq{} 2\frac{Hd \log \prn*{1+\lambda^{-1}K/d}}{K}.
\end{align*}
  
\end{proof}

\subsection{Examples and Generalizations of Bellman Rank}
\label{sec:bellman_rank_example}

This subsection works through some well known examples of classes with low Bellman rank, computing the factorization and rank $d$ for each. Crucially, we see that many of these examples support low rank even when the value function class under consideration is nonlinear.

\subsubsection{Examples}

We start with familiar examples, then introduce new models that allow for nonlinear
function approximation.

\begin{example}[Tabular MDPs]
If $M$ is a tabular MDP with $\abs{\cS}\leq{}S$ and
$\abs{\cA}\leq{}A$, we can write the Bellman residual for any function
$Q$ and policy $\pi$ as
\begin{align*}
  \cE_h(\pi,Q) &=
          \Enm{M}{\pi}\brk*{Q_h(s_h,a_h)-\prn*{r_h+\max_{a}Q_{h+1}(s_{h+1},a)}}
  \\
  &=
    \sum_{s,a}\dmpi_h(s,a)\En\sups{M}\brk*{Q_h(s,a)-\prn*{r_h+\max_{a'}Q_{h+1}(s_{h+1},a')}\mid{}s_h=s,a_h=a}.
\end{align*}
It follows that if we define
\[\Xm_h(\pi)=\crl*{\dmpi_h(s,a)}_{s\in\cS,a\in\cA}\in\bbR^{SA}\] and
\[\Wm_h(Q)=\crl*{\En\sups{M}\brk*{Q_h(s,a)-\prn*{r_h+\max_{a'}Q_{h+1}(s_{h+1},a')}\mid{}s_h=s,a_h=a}}_{s\in\cS,a\in\cA}\in\bbR^{SA},\]
we have
\[
\cE_h(\pi,Q) = \tri*{\Xm_h(\pi),\Wm_h(Q)}.
\]
This shows that $\dimbel(M)\leq{}SA$.
\end{example}

\begin{example}[Low-Rank MDPs]
The calculation in \pref{eq:low_rank_mdp_bellman} shows that by
choosing
$\Xm_h(\pi)\ldef{}\Enm{M}{\pi}\brk*{\phi(s_h,a_h)}\in\bbR^{d}$ and
$\Wm_h(Q)\ldef{}\thetaq_h-\tilde{\theta}_h\sups{M,Q}\in\bbR^{d}$,
any Low-Rank MDP $M$ has $\dimbel(M)\leq{}d$. For this
class, \bilinucb gives a weaker guarantee than \lsviucb, because it does
not exploit Bellman completeness; see the discussion of regret bounds
in \cref{sec:rl_bellman_rank}.
  
\end{example}
We now explore Bellman rank for some MDP families that have not already
been covered.
\begin{example}[Low Occupancy Complexity]
  \newcommand{\thetampi}{\theta^{\sss{M},\pi}}
  An MDP $M$ is said to have \emph{low occupancy complexity} if there
  exists a feature map $\phim(s,a)\in\bbR^{d}$ such that for all
  $\pi$, there exists $\thetampi_h\in\bbR^{d}$ such that
  \begin{equation}
    \label{eq:low_occupancy}
    \dmpi_h(s,a) = \tri*{\phim(s,a),\thetampi_h}.
  \end{equation}
  Note that neither $\phim$ nor $\thetampi$ is assumed to be known to the learner.
If $M$ has low occupancy complexity, then for any value function $Q$
and policy $\pi$, we have
\begin{align*}
  \cE_h(\pi,Q) &=
          \Enm{M}{\pi}\brk*{Q_h(s_h,a_h)-\prn*{r_h+\max_{a}Q_{h+1}(s_{h+1},a)}}
  \\
               &=
    \sum_{s,a}\dmpi_h(s,a)\cdot\En\sups{M}\brk*{Q_h(s,a)-\prn*{r_h+\max_{a'}Q_{h+1}(s_{h+1},a')}\mid{}s_h=s,a_h=a}
  \\
                 &=
                   \sum_{s,a}\tri*{\phim(s,a),\thetampi_h}\cdot\En\sups{M}\brk*{Q_h(s,a)-\prn*{r_h+\max_{a'}Q_{h+1}(s_{h+1},a')}\mid{}s_h=s,a_h=a}
                     \\
               &=
\tri*{\thetampi_h,\sum_{s,a}\phim(s,a)\cdot\En\sups{M}\brk*{Q_h(s,a)-\prn*{r_h+\max_{a'}Q_{h+1}(s_{h+1},a')}\mid{}s_h=s,a_h=a}}.
\end{align*}
It follows that if we define
\[\Xm_h(\pi)=\thetampi_h\]
and
\[\Wm_h(Q)=\sum_{s,a}\phim(s,a)\cdot\En\sups{M}\brk*{Q_h(s,a)-\prn*{r_h+\max_{a'}Q_{h+1}(s_{h+1},a')}\mid{}s_h=s,a_h=a},
\]
then $\cE_h(\pi,Q) = \tri*{\Xm_h(\pi),\Wm_h(Q)}$, which shows that
$\dimbel(M)\leq{}d$.

This setting includes tabular MDPs\footnote{Low-rank
MDPs satisfy a related condition on state occupancies. This condition can
also be handled within the bilinear framework, with a modified exploration
rule; see \citet{du2021bilinear}.} and allows for \emph{nonlinear} function
approximation: whenever the occupancies satisfy \cref{eq:low_occupancy},
the Bellman rank is at most $d$ for \emph{any class $\cQ$}, including
classes of neural networks or other nonlinear models.
\end{example}

We close with two more new examples.
\begin{example}[LQR]
  \newcommand{\Amstar}{A\sups{M}}
    \newcommand{\Bmstar}{B\sups{M}}
A classical problem in continuous control is the Linear Quadratic
Regulator, or LQR. Here, we have $\cS=\cA=\bbR^{d}$, and states evolve
via
\begin{align*}
  s_{h+1}=\Amstar{}s_{h} + \Bmstar{}a_h + \zeta_h,
\end{align*}
where $\zeta_h\sim{}\cN(0,I)$, and $s_1\sim{}\cN(0,I)$. We assume that
rewards have the form\footnote{LQR is typically stated in terms of
  losses; we negate because we consider rewards. Because the noise is Gaussian, the rewards in the LQR setting are technically unbounded, and fall outside this section's 
  $[0,1]$-bounded reward assumption; this requires a straightforward extension of the results in \cref{sec:rl_linear}.}
\[
r_h = -s_h^{\trn}\Lambda\sups{M}s_h - a_h^{\trn}\Gamma\sups{M}a_h
\]
for matrices $\Lambda\sups{M},\Gamma\sups{M}\psdgeq{}0$. A classical
result, dating back to Kalman, is that the optimal controller for this
system is a linear mapping of the form
\[
\pi\subs{M,h}(s) = K\sups{M}_hs,
\]
and that the value function $\Qmstar(s,a) =
(s,a)^{\trn}P_h\sups{M}(s,a)$ is quadratic. Hence, it suffices to take $\cQ$ to be the set of all quadratic functions in
$(s,a)$. With this choice, it can be shown that
$\dimbel(M)\leq{}d^2+1$. The basic idea is to choose
$\Xm_h(\pi)=(\mathrm{vec}(\Enm{M}{\pi}\brk*{s_hs_h^{\trn}}),
1)$ and use the quadratic structure of the value functions.
  \end{example}

 \begin{example}[Linear $\Qstar$/$\Vstar$]
Linearity of $\Qstar$ alone is not enough
(\cref{prop:realizability_insufficient_linear}), but linearity of both
$\Qstar$ and $\Vstar$ is.

Consider an MDP $M$. Assume that there known feature maps $\phi(s,a)\in\bbR^{d}$ and
$\psi(s')\in\bbR^{d}$ such that
\[
\Qmstar_h(s,a)=\tri*{\phi(s,a),\thetam_h},\mathand\Vmstar_h(s)=\tri*{\psi(s),\wm_h}.
\]
Let
\[
\cQ=\crl*{Q\mid{} Q_h(s,a)=\tri*{\phi(s,a),\theta_h}: \theta_h\in\bbR^d
  ,\exists{} w\; \max_{a\in\cA}\tri*{\phi(s,a),\theta_h}=\tri*{\psi(s),w}\;\forall{}s}.
\]
Then $\dimbel(M) \leq{} 2d$. We do not prove this result, but the
basic idea is to choose $\Xm_h(\pi) = \Enm{M}{\pi}\brk*{(\phi(s_h,a_h),\psi(s_{h+1}))}$.
  \end{example}

  See \citet{jiang2017contextual,du2021bilinear} for further examples.

\subsubsection{Generalizations: V-Type Bellman Rank}
\newcommand{\Platentm}[1][M]{P^{\sss{#1},\mathrm{latent}}}
\newcommand{\phistarm}{\phi\sups{M}}
While we gave \pref{eq:bellman_rank_factorization} as the definition
for Bellman rank, there are many variations on the assumption that also
lead to low regret. One well-known variant is \emph{V-type}
Bellman rank \citep{jiang2017contextual}, which asserts that for all $\pi\in\Pi$ and $Q\in\cQ$,
\begin{align}
  \En^{\sss{M},\pi}
  \En\sups{M}_{s_{h+1}\mid{}s_h,a_h\sim{}\piq(s_h)}
  \brk*{
  Q_h(s_h,a_h)
  -r_h - \max_{a\in\cA}Q_{h+1}(s_{h+1},a)
  }
  =\tri*{\Xm_h(\pi),\Wm_h(Q)}.\label{eq:v_type}
\end{align}

This is the same as the definition
\pref{eq:bellman_rank_factorization} (which is typically referred to
as \emph{Q-type} Bellman rank), except that we take $a_h=\piq(s_h)$
instead of $a_h=\pi(s_h)$.\footnote{The name ``V-type'' refers to the
  fact that \pref{eq:v_type} only depends on $Q$ through the induced
  $V$-function $s\mapsto{}Q_h(s,\piq(s))$, while
  \pref{eq:bellman_rank_factorization} depends on the full
  $Q$-function, hence ``Q-type''.} With an appropriate modification,
\bilinucb can be shown to give sample complexity guarantees that scale
with V-type Bellman rank instead of Q-type. This definition captures
meaningful classes of tractable RL models that are not captured by
the Q-type definition \pref{eq:bellman_rank_factorization}, with a
canonical example being \emph{Block MDPs}.
\begin{example}[Block MDP]
The Block MDP
\citep{jiang2017contextual,du2019latent,misra2020kinematic} is a model
in which the (``observed'') state space $\cS$ is
large/high-dimensional, but the dynamics are governed by a (small)
latent state space $\cZ$. Formally, a Block MDP
$M=\crl{\cS,\cA,\crl{\Pm_h},\crl{\Rm_h},d_1}$
is defined based on an (unobserved) \emph{latent state space} $\cZ$, with
$z_h$ denoting the latent state at layer $h$. We first describe the
dynamics for the latent space. Given initial latent state $z_1$, the latent states evolve via
\[
  z_{h+1}\sim{}\Platentm_h(z_h,a_h).
\]
The latent state $z_h$ is not observed. Instead, we observe
\[
s_h\sim{}\qm_h(z_h),
\]
where $\qm_h:\cZ\to\Delta(\cS)$ is an \emph{emission distribution} with
the property that $\supp(q_h(z))\cap\supp(q_h(z'))=\emptyset$ if
$z\neq z'$. This property (\emph{decodability}) ensures that there exists a unique mapping $\phistarm_h:\cS\to\cZ$ that maps the observed state $s_h$ to the corresponding latent state $z_h$. We assume that $\Rm_h(s,a)=\Rm_h(\phistarm_h(s),a)$, which implies that optimal policy $\pim$ depends only on the endogenous latent state, i.e. $\pi\subs{M,h}(s)=\pi\subs{M,h}(\phistarm_h(s))$.

The main challenge of learning in Block MDPs is that the decoder
$\phistarm$ is not known to the learner in advance. Indeed, given
access to the decoder, one can obtain regret
$\poly(H,\abs{\cZ},\abs{\cA})\cdot\sqrt{T}$ by applying tabular
reinforcement learning algorithms to the latent state space. In light
of this, the aim of the Block MDP setting is to obtain sample
complexity guarantees that are independent of the size of the observed
state space $\abs{\cS}$, and scale as $\poly(\abs{\cZ}, \abs{\cA},
H,\log\abs{\cF})$, where $\cF$ is an appropriate class of function
approximators (typically either a value function class $\cQ$
containing $\Qmstar$ or a class of decoders $\Phi$ that attempts to
model $\phistarm$ directly).

We now show that the Block MDP setting admits low V-type Bellman
rank. Observe that we can write
\begin{align*}
  &%
      \En^{\sss{M},\pi}
  \En\sups{M}_{s_{h+1}\mid{}a_h,s_h\sim{}\piq(s_h)}
  \brk*{
  Q_h(s_h,\piq(s_h))
  -r_h - \max_{a\in\cA}Q_{h+1}(s_{h+1},\piq(s_{h+1}))
    }\\
  &=\sum_{z\in\cZ}\dmpi_h(z)\En_{s\sim{}\qm_h(z)}\En\sups{M}_{s_{h+1}\mid{}s_h,a_h\sim{}\piq(s_h)}
  \brk*{
  Q_h(s_h,\piq(s_h))
  -r_h - \max_{a\in\cA}Q_{h+1}(s_{h+1},\piq(s_{h+1}))
  }.
\end{align*}
where $\dmpi_h(z)$ denotes the probability that the latent state
at layer $h$ is $z$ under $\pi$. This implies that we can take
\[
\Xm_h(\pi) = \crl*{\dmpi_h(z)}_{z\in\cZ}
\]
and
\[
  \Wm_h(Q) = \crl*{
   \En_{s\sim{}\qm_h(z)}\En\sups{M}_{s_{h+1}\mid{}s_h,a_h\sim{}\piq(s_h)}
  \brk*{
  Q_h(s_h,\piq(s_h))
  -r_h - \max_{a\in\cA}Q_{h+1}(s_{h+1},\piq(s_{h+1}))
  } 
    }_{z\in\cZ}
  \]
  so that the V-type Bellman rank is at most $\abs{\cZ}$. This means
  that as long as $\cQ$ contains $\Qmstar$, we can obtain sample complexity guarantees that scale
  with $\abs{\cZ}$ rather than $\abs{\cS}$, as desired.
\end{example}
There are a number of variants of Bellman rank, including Bilinear
rank \citep{du2021bilinear} and Bellman-Eluder dimension
\citep{jin2021bellman}, which subsume and slightly generalize both Bellman rank definitions.

\subsection{Estimation-to-Decisions for Low Bellman Rank}
\label{sec:dec_bellman_rank}
\label{sec:dec_bellman}
An alternative to \bilinucb is to run \etd (\cref{alg:etd}) over a realizable model class $\cM$ with low Bellman rank. By \cref{prop:upper_main}, all
that is required is a bound on the \CompShort for the class at hand.
The following result
\citep{foster2021statistical} shows that the
\CompText is always bounded for classes with low Bellman rank.
\begin{prop}[DEC for Bellman rank]
  \label{prop:dec_bellman_rank}
  Let $\cM$ be a class of MDPs with nonnegative rewards and
  $\sum_{h=1}^{H}r_h\leq 1$ almost surely under every model and policy.
  Suppose every $Q\in\cQ_{\cM}$ satisfies $Q_h(s,a)\in[0,1]$ for all
  $h,s,a$, and every $M\in\cM$ has Bellman rank at most $d$ relative to
  the induced classes $\cQ_{\cM}$ and $\Pi_{\cQ_{\cM}}$ from
  \pref{eq:rl_induced_classes}. Then, for every $\gamma>0$, we have
  \begin{align}
    \label{eq:22}
    \comp(\cM) \approxleq \frac{H^2d}{\gamma}.
  \end{align}
\end{prop}
By \cref{eq:upper_main_hp}, \etd with $\gamma\propto H\sqrt{dT/\EstHel(\cM,T,\delta)}$ has
$\Reg\approxleq{}H\sqrt{dT\cdot\EstHel(\cM,T,\delta)}$ with probability at
least $1-\delta$ whenever we have access to a realizable model class
satisfying the assumptions of
\pref{prop:dec_bellman_rank}. As a special
case, for any finite class $\cM$, averaged exponential
weights has online log-loss regret $\RegOLLog\leq\log\abs{\cM}$.
The conversion in \pref{lem:log_loss_hellinger} controls
$\EstHel$, which in turn gives the RL regret bound
\begin{align}
\En\brk*{\Reg}\approxleq{}H\sqrt{dT\log\abs{\cM}}\label{eq:bellman_rank_model}
\end{align}
in expectation, via the first inequality in
\cref{eq:upper_main}. For tabular MDPs, $d=SA$, and this recovers the
guarantee of \cref{prop:e2d_tabular} with an improved dependence on the
horizon.
We do not
prove \pref{prop:dec_bellman_rank}, but interested readers can refer
to \citet[Theorem~7.1]{foster2021statistical}. The proof uses a more general
version of the \pcigw algorithm from \pref{sec:dec_tabular}, incorporating
optimal design in the space of policies; see \citet[Theorem~7.2]{foster2021statistical}.

\paragraph{Value-based guarantees via optimistic estimation}
Estimation-to-Decisions as described above is a model-based method, and in general, the model estimation complexity $\log\abs{\cM}$ in
\cref{eq:bellman_rank_model} can be arbitrarily large compared to the
complexity $\log\abs{\cQ}$ for a realizable value function class
(consider the low-rank MDP---since $\mu$ is unknown, it is not
possible to construct a small model class $\cM$). To
derive value-based guarantees along the lines of what \bilinucb
achieves in \cref{prop:bilinucb_regret}, a natural approach is to
replace the Hellinger distance appearing in the \CompShort with a
divergence tailored to value function iteration, following the
development in \cref{sec:general_divergence,sec:optimistic}. One such
choice is the divergence
\begin{align*}
  \Dsbrpi{Q}{M}
  = \sum_{h=1}^{H}\prn*{\Enm{M}{\pi}\brk*{Q_h(s_h,a_h)-\prn*{r_h+\max_{a}Q_{h+1}(s_{h+1},a)}}}^2,
\end{align*}
which measures the squared bellman residual for an estimated value
function under $M$. With this choice, we appeal to the optimistic \etd
algorithm (\etdopt) from \cref{sec:optimistic}. One can show that the
optimistic \CompShort for this divergence is bounded as 
\begin{align*}
    \ocompsbr(\cM) \approxleq{} \frac{H\cdot{}d}{\gamma}.
\end{align*}
This implies that \etdopt, with an appropriate choice of estimation
algorithm tailored to $\Dsbrpi{\cdot}{\cdot}$, achieves
\begin{align*}
\En\brk*{\Reg} \approxleq (H^2d\log\abs{\cQ})^{1/2}T^{3/4}.
\end{align*}
Note that due to the
asymmetric nature of $\Dsbrpi{\cdot}{\cdot}$, it is critical to appeal
to optimistic estimation to derive this result. Indeed, the
non-optimistic generalized DEC $\compgen[\Dsbrshort]$ does not enjoy a
polynomial bound. See \citet{foster2022note} for details. 

\subsection{Concluding Remarks and Broader Outlook}
\label{sec:rl_conclusion}

This section showed that for reinforcement learning with
function approximation, realizability alone, whether of the model, the
optimal value function, or the optimal policy, does not permit
sample-efficient learning, and that additional structure controlling
distribution shift is required. We developed two such structures:
low-rank transitions, exploited by \lsviucb through least-squares value
iteration with elliptic bonuses (\cref{prop:lsvi_ucb}), and the more
general notion of low Bellman rank, exploited by \bilinucb through
optimism on average and elimination of value functions with large
Bellman residual (\cref{prop:bilinucb_regret}), and, through the bound
on the \CompShort in \cref{prop:dec_bellman_rank}, by \etd. We close
with a survey of the broader landscape of tractability, and with open
directions for reinforcement learning and the theory developed in
this monograph as a whole.

\paragraph{The landscape of tractability}
Within the Bellman-rank family, other algorithms include OLIVE
\citep{jiang2017contextual}, which eliminates value functions using a
single adversarially chosen policy per round, and GOLF
\citep{jin2021bellman}, which attains $\sqrt{T}$ regret under Bellman
completeness and is analyzed through the Bellman-eluder dimension, a
combination of Bellman rank with the eluder dimension of
\citet{russo2013eluder}. Bilinear classes \citep{du2021bilinear} and
witness rank \citep{sun2019model} extend these ideas to model-based
function approximation, and \citet{chen2022unified} and
\citet{zhong2022posterior} unify the resulting guarantees through the
\CompShort and related measures. Representation learning for low-rank and
block MDPs is treated by \citet{agarwal2020flambe},
\citet{uehara2022representation}, \citet{zhang2022efficient}, and
\citet{mhammedi2023efficient}, and inherent Bellman error offers an
alternative to Bellman completeness \citep{zanette2020learning}. For an expanded textbook treatment of theoretical guarantees in reinforcement
learning with function approximation, see \citet{agarwal2019reinforcement}.

Bellman rank is one of several structural conditions under
which sample-efficient reinforcement learning with function
approximation is possible, and it is worth knowing the others. The first
is \emph{coverability} \citep{xie2023role,amortila2024harnessing,mhammedi2024power}, which asks that a single
(unknown) distribution over states cover the occupancy of every policy,
and is satisfied by tabular, low-rank, and block MDPs; this facilitates a different type of control over distribution shift which subsumes some settings beyond the scope of Bellman rank.

The second is \emph{exogenous} structure, in
which part of the state evolves independently of the agent's actions
\citep{dietterich2018discovering}; for exogenous block MDPs, the
endogenous latent state can be recovered with sample complexity
independent of the exogenous part
\citep{efroni2021provable,efroni2022sample,lamb2023guaranteed}.

The
third is \emph{partial observability}: while learning in general POMDPs
is intractable, structural conditions such as undercompleteness
\citep{jin2020sample}, observability and short-term memory
\citep{golowich2022learning,efroni2022provable}, and weak revealing or
$B$-stability \citep{liu2022partially,chen2022partially} permit
sample-efficient learning, as does the predictive state representation
model \citep{zhan2023pac}.

These conditions are incomparable in general. Coverability captures
classes outside the scope of Bellman rank, and the converse also holds;
and no single condition is known to capture every class for which
sample-efficient learning is possible. Characterizing such classes is
an ongoing research program.

\paragraph{Topics beyond the scope of the monograph}
Two large parts of the reinforcement learning literature fall outside the scope of this monograph, as discussed in \cref{sec:scope}. \emph{Offline reinforcement learning} removes exploration
altogether: the learner receives a fixed dataset and must find a good
policy without further interaction, so that the central issues become
coverage of the data and the realizability and completeness of the
function class \citep{munos2008finite,chen2019information,xie2020q,wang2020statistical,foster2022offline},
with pessimism playing the role that optimism plays here
\citep{jin2021pessimism,rashidinejad2021bridging,xie2021bellman,zhan2022offline};
see \citet{levine2020offline} for a survey.

\emph{Policy gradient
methods}
\citep{williams1992simple,sutton2000policy,schulman2015trust,schulman2017proximal},
which optimize a parameterized policy directly and are the workhorse of
practical deep reinforcement learning, are analyzed by
\citet{agarwal2021theory}; their guarantees require the data
distribution to already cover the relevant states, and comparatively
little is known about exploration for policy gradient methods beyond
the tabular and linear settings \citep{cai2020provably,agarwal2020pc},
even though they succeed in practice in domains where strong exploration
is not required.

\paragraph{Open directions and broader outlook}
Stepping back,
we close with a broader view of where the theory developed in this
monograph stands and where it is headed.

\emph{Closing the statistical complexity gap.} The most immediate question is a
purely technical one. The upper and lower bounds of
\cref{sec:general_dm,sec:lower} match up to the cost of estimation: the
upper bounds pay $\log\abs{\cM}$ for a finite class and the lower bounds
do not, and the radii at which the two are evaluated differ by the same
factor. For specific classes the gap can be closed by hand, and
\cref{sec:lower_tighter} and \citet{chen2024assouad} describe some more general techniques, but a unified account of the statistical complexity of
interactive decision making, with matching upper and lower bounds for
every class, does not yet exist. The gap is most pronounced in
reinforcement learning, where model-based guarantees scale with the
complexity of the model class and value-based guarantees with that of
the value function class; value-based guarantees obtainable
through the \CompShort (\cref{sec:dec_bellman_rank}) are not yet sharp.

\emph{Computational foundations.} This monograph has taken a
purely statistical perspective. Our upper bounds are attained by
algorithms that involve primitives like solving a min-max problem at every round, or maintaining confidence sets over value functions, and we have not asked whether
these steps can be carried out efficiently. For contextual bandits the
answer is yes, which is what makes the reduction to regression oracles
practical. For reinforcement learning with function approximation the
picture is different: there are settings in which sample-efficient
learning is possible but computationally efficient learning is not,
under standard complexity-theoretic assumptions
\citep{golowich2024exploration,rohatgi2025necessary}. A computational
theory of reinforcement learning, one that says which structural
conditions permit computationally efficient exploration and which do not, is in its
infancy, and we regard it as among the most interesting technical problems in
the area.\loose

\emph{Rethinking the foundations.} One can read the state of
affairs in this section more pessimistically. Two facts stand out.
First, unlike contextual bandits, where exploration algorithms are now
deployed at scale (\cref{sec:spectrum}), provably efficient
exploration algorithms for reinforcement learning have not made their way
into practice in any robust form. Second, every positive result in this
section rests on some kind of structural assumption. These assumptions are delicate,
and designing algorithms around them has been a slow process, proceeding one
assumption and one paper at a time. Taken together, these may be a sign
that the problem as we have posed it, learning to act from scratch in an
unknown environment with only a function class as prior knowledge, is
simply too hard, and that progress requires relaxing the setting rather
than lengthening the list of assumptions.\loose

\emph{Pre-trained models.} Pre-trained models offer one such
relaxation, and the most consequential one to date. In the language
model paradigm, reinforcement learning is not used to learn from scratch
but to adapt a model that already encodes a great deal about the task,
through post-training with verifiable or human feedback
\citep{christiano2017deep,stiennon2020learning,ouyang2022training,bai2022training,rafailov2024direct}.
This has been enormously successful, and a science is beginning
to emerge around questions like what the pre-trained model must provide for post-training to
succeed \citep{huang2025best,chen2025when,foster2025good}. From the
viewpoint of this monograph, the pre-trained model plays the role that
the class $\cM$ played for us, that of prior knowledge about the
environment, but in a far richer and less explicit form, and the
statistical and computational complexity of decision making relative to such a prior is
largely uncharted.

Language models are, from
the standpoint of this monograph, a simple decision making problem: a
single decision per prompt, or at most a short horizon. The ideas
developed here around exploration, credit assignment, and distribution
shift matter most for longer-horizon problems, including robotics,
embodied agents, and scientific discovery. We hope to see successful
approaches in these settings to combine the exploration principles of
this monograph with pre-training and imitation learning
\citep{pomerleau1988alvinn,ross2011reduction,foster2024behavior}, the
latter supplying the coverage that exploration from scratch cannot, and
understanding such combinations is wide open.

\emph{Beyond the black box.} Finally, throughout this monograph
we have treated estimation as a black box: an oracle that returns a
model or a value function together with a guarantee, whose internals we
never examine. This was deliberate, and is what allows us to state one principle covering
every function class. But the models now in use are deep neural networks,
and a black-box view discards most of what is known about them: the
representations they learn, the inductive biases of their architectures
and of the optimizers that train them, and the way their behavior scales
with data and compute. Each of these bears on decision making. How learning and exploration interact inside a neural network, rather than across an oracle interface, may be where the
next theory (or science) of decision making begins.

\subsection{Bibliographic Remarks}\label{sec:rl_bib}
\cref{prop:realizability_insufficient} is adapted from
\citet{krishnamurthy2016pac}, who introduced the rich-observation setting
that later became the Block MDP.
\cref{prop:realizability_insufficient_linear} is due to
\citet{weisz2021exponential}, who proved the lower bound for planning
with a generative model; \citet{wang2021exponential} provided a simpler construction and showed that the hardness persists in online RL, even under a constant action suboptimality gap. The low-rank MDP model, in the form used here, is from
\citet{yang2019sample,jin2020provably}, who use the name \emph{linear
  MDP}. \lsviucb (\cref{alg:lsviucb}) and its
analysis---\cref{lem:lsvi_bellman_linear}, \cref{prop:lsvi_ucb},
\cref{lem:lsvi_ucb_confidence}, and \cref{lem:lsvi_ucb_optimism}---follow
\citet{jin2020provably}.

Bellman rank was introduced by \citet{jiang2017contextual} in
the V-type form of \cref{eq:v_type}; the Q-type definition in
\cref{eq:bellman_rank_factorization} is from
\citet{du2021bilinear,jin2021bellman}. \bilinucb (\cref{alg:bilinucb}),
\cref{prop:bilinucb_regret}, and \cref{lem:bilinucb_confidence},
\cref{lem:bilinucb_good_event}, and \cref{lem:bilinucb_potential} follow
\citet{du2021bilinear}, specialized from bilinear classes to Bellman rank
and to finite $\cQ$; the logarithmic form of the elliptic potential
lemma, \cref{eq:elliptic_potential_log}, is
\citet[Lemma~19.4]{lattimore2020bandit}. The examples in
\cref{sec:bellman_rank_example} are drawn from
\citet{jiang2017contextual} and \citet{du2021bilinear}.
\cref{prop:dec_bellman_rank} is
\citet[Theorem~7.1]{foster2021statistical}, and the value-based guarantee
via optimistic estimation is from \citet{foster2022note}.

\section*{Acknowledgments}
\addcontentsline{toc}{section}{Acknowledgments}
We thank Fan Chen, Margalit Glasgow, Noah Golowich, Yanjun Han, Sham Kakade, Jian Qian, Yunbei Xu, and Yunzong Xu for many helpful discussions on the topics in the manuscript over the years, and for being amazing research collaborators and friends. We also thank the students of MIT's 9.522 course, particularly the Fall 2022 and Fall 2023 iterations, for invaluable feedback on the content and presentation of the material in the monograph.

\paragraph{Statement on AI use}
The authors used AI tools in the final stages of this project for copyediting, restructuring, and fixing technical errors. All technical
content, proofs, and final wording were written and checked by the
authors, who take full responsibility for the manuscript.

\newpage

\bibliography{refs}

\begin{thebibliography}{227}
\providecommand{\natexlab}[1]{#1}
\providecommand{\url}[1]{\texttt{#1}}
\expandafter\ifx\csname urlstyle\endcsname\relax
  \providecommand{\doi}[1]{doi: #1}\else
  \providecommand{\doi}{doi: \begingroup \urlstyle{rm}\Url}\fi

\bibitem[Abbasi-Yadkori et~al.(2011)Abbasi-Yadkori, P{\'a}l, and
  Szepesv{\'a}ri]{abbasi2011improved}
Y.~Abbasi-Yadkori, D.~P{\'a}l, and C.~Szepesv{\'a}ri.
\newblock Improved algorithms for linear stochastic bandits.
\newblock In \emph{Advances in Neural Information Processing Systems}, 2011.

\bibitem[Abe and Long(1999)]{abe1999associative}
N.~Abe and P.~M. Long.
\newblock Associative reinforcement learning using linear probabilistic
  concepts.
\newblock In \emph{Proceedings of the Sixteenth International Conference on
  Machine Learning}, pages 3--11. Morgan Kaufmann Publishers Inc., 1999.

\bibitem[Agarwal and Zhang(2022)]{agarwal2022model}
A.~Agarwal and T.~Zhang.
\newblock Model-based {RL} with optimistic posterior sampling: Structural
  conditions and sample complexity.
\newblock In \emph{Advances in Neural Information Processing Systems},
  volume~35, pages 35284--35297, 2022.

\bibitem[Agarwal et~al.(2013)Agarwal, Foster, Hsu, Kakade, and
  Rakhlin]{agarwal2013stochastic}
A.~Agarwal, D.~P. Foster, D.~Hsu, S.~M. Kakade, and A.~Rakhlin.
\newblock Stochastic convex optimization with bandit feedback.
\newblock \emph{SIAM Journal on Optimization}, 23\penalty0 (1):\penalty0
  213--240, 2013.

\bibitem[Agarwal et~al.(2014)Agarwal, Hsu, Kale, Langford, Li, and
  Schapire]{agarwal2014taming}
A.~Agarwal, D.~Hsu, S.~Kale, J.~Langford, L.~Li, and R.~Schapire.
\newblock Taming the monster: A fast and simple algorithm for contextual
  bandits.
\newblock In \emph{International Conference on Machine Learning}, pages
  1638--1646, 2014.

\bibitem[Agarwal et~al.(2016)Agarwal, Bird, Cozowicz, Hoang, Langford, Lee, Li,
  Melamed, Oshri, Ribas, Sen, and Slivkins]{agarwal2016making}
A.~Agarwal, S.~Bird, M.~Cozowicz, L.~Hoang, J.~Langford, S.~Lee, J.~Li,
  D.~Melamed, G.~Oshri, O.~Ribas, S.~Sen, and A.~Slivkins.
\newblock Making contextual decisions with low technical debt.
\newblock \emph{arXiv:1606.03966}, 2016.

\bibitem[Agarwal et~al.(2020{\natexlab{a}})Agarwal, Henaff, Kakade, and
  Sun]{agarwal2020pc}
A.~Agarwal, M.~Henaff, S.~Kakade, and W.~Sun.
\newblock {PC-PG}: Policy cover directed exploration for provable policy
  gradient learning.
\newblock In \emph{Advances in Neural Information Processing Systems},
  2020{\natexlab{a}}.

\bibitem[Agarwal et~al.(2020{\natexlab{b}})Agarwal, Kakade, Krishnamurthy, and
  Sun]{agarwal2020flambe}
A.~Agarwal, S.~Kakade, A.~Krishnamurthy, and W.~Sun.
\newblock {FLAMBE}: Structural complexity and representation learning of low
  rank {MDP}s.
\newblock In \emph{Advances in Neural Information Processing Systems},
  volume~33, 2020{\natexlab{b}}.

\bibitem[Agarwal et~al.(2021)Agarwal, Kakade, Lee, and
  Mahajan]{agarwal2021theory}
A.~Agarwal, S.~M. Kakade, J.~D. Lee, and G.~Mahajan.
\newblock On the theory of policy gradient methods: Optimality, approximation,
  and distribution shift.
\newblock \emph{Journal of Machine Learning Research}, 22\penalty0
  (98):\penalty0 1--76, 2021.

\bibitem[Agarwal et~al.(2022)Agarwal, Jiang, Kakade, and
  Sun]{agarwal2019reinforcement}
A.~Agarwal, N.~Jiang, S.~M. Kakade, and W.~Sun.
\newblock \emph{Reinforcement Learning: Theory and Algorithms}.
\newblock 2022.
\newblock Monograph draft, Version 3 (January 31, 2022),
  \url{https://rltheorybook.github.io}.

\bibitem[Agarwal et~al.(2023)Agarwal, Jin, and Zhang]{agarwal2023vo}
A.~Agarwal, Y.~Jin, and T.~Zhang.
\newblock {VO$Q$L}: Towards optimal regret in model-free {RL} with nonlinear
  function approximation.
\newblock In \emph{Conference on Learning Theory}, volume 195 of
  \emph{Proceedings of Machine Learning Research}, pages 987--1063. PMLR, 2023.

\bibitem[Agrawal(1995)]{agrawal1995sample}
R.~Agrawal.
\newblock Sample mean based index policies by {$O(\log n)$} regret for the
  multi-armed bandit problem.
\newblock \emph{Advances in Applied Probability}, 27\penalty0 (4):\penalty0
  1054--1078, 1995.

\bibitem[Agrawal and Goyal(2012)]{agrawal2012analysis}
S.~Agrawal and N.~Goyal.
\newblock Analysis of {T}hompson sampling for the multi-armed bandit problem.
\newblock In \emph{Conference on Learning Theory}, volume~23, pages
  39.1--39.26. PMLR, 2012.

\bibitem[Agrawal and Goyal(2013)]{agrawal2013thompson}
S.~Agrawal and N.~Goyal.
\newblock Thompson sampling for contextual bandits with linear payoffs.
\newblock In \emph{International Conference on Machine Learning}, pages
  127--135. PMLR, 2013.

\bibitem[Ali and Silvey(1966)]{ali1966general}
S.~M. Ali and S.~D. Silvey.
\newblock A general class of coefficients of divergence of one distribution
  from another.
\newblock \emph{Journal of the Royal Statistical Society: Series B},
  28\penalty0 (1):\penalty0 131--142, 1966.

\bibitem[Alon et~al.(2015)Alon, Cesa-Bianchi, Dekel, and Koren]{alon2015online}
N.~Alon, N.~Cesa-Bianchi, O.~Dekel, and T.~Koren.
\newblock Online learning with feedback graphs: Beyond bandits.
\newblock In \emph{Conference on Learning Theory}, 2015.

\bibitem[Amin et~al.(2011)Amin, Kearns, and Syed]{amin2011bandits}
K.~Amin, M.~Kearns, and U.~Syed.
\newblock Bandits, query learning, and the haystack dimension.
\newblock In \emph{Proceedings of the 24th Annual Conference on Learning
  Theory}, pages 87--106. JMLR Workshop and Conference Proceedings, 2011.

\bibitem[Amortila et~al.(2024)Amortila, Foster, Jiang, Sekhari, and
  Xie]{amortila2024harnessing}
P.~Amortila, D.~J. Foster, N.~Jiang, A.~Sekhari, and T.~Xie.
\newblock Harnessing density ratios for online reinforcement learning.
\newblock In \emph{International Conference on Learning Representations}, 2024.

\bibitem[Anthony and Bartlett(2009)]{anthony2009neural}
M.~Anthony and P.~L. Bartlett.
\newblock \emph{Neural network learning: Theoretical foundations}.
\newblock Cambridge University Press, 2009.

\bibitem[Audibert and Bubeck(2009)]{audibert2009minimax}
J.-Y. Audibert and S.~Bubeck.
\newblock Minimax policies for adversarial and stochastic bandits.
\newblock In \emph{Proceedings of the 22nd Annual Conference on Learning Theory
  (COLT)}, pages 217--226, 2009.

\bibitem[Auer(2002)]{auer2002using}
P.~Auer.
\newblock Using confidence bounds for exploitation-exploration trade-offs.
\newblock \emph{Journal of Machine Learning Research}, 3\penalty0
  (Nov):\penalty0 397--422, 2002.

\bibitem[Auer et~al.(2002{\natexlab{a}})Auer, Cesa-Bianchi, and
  Fischer]{auer2002finite}
P.~Auer, N.~Cesa-Bianchi, and P.~Fischer.
\newblock Finite-time analysis of the multiarmed bandit problem.
\newblock \emph{Machine learning}, 47\penalty0 (2-3):\penalty0 235--256,
  2002{\natexlab{a}}.

\bibitem[Auer et~al.(2002{\natexlab{b}})Auer, Cesa-Bianchi, Freund, and
  Schapire]{auer2002non}
P.~Auer, N.~Cesa-Bianchi, Y.~Freund, and R.~E. Schapire.
\newblock The nonstochastic multiarmed bandit problem.
\newblock \emph{SIAM Journal on Computing}, 32\penalty0 (1):\penalty0 48--77,
  2002{\natexlab{b}}.

\bibitem[Auer et~al.(2007)Auer, Ortner, and Szepesv{\'a}ri]{auer2007improved}
P.~Auer, R.~Ortner, and C.~Szepesv{\'a}ri.
\newblock Improved rates for the stochastic continuum-armed bandit problem.
\newblock In \emph{International Conference on Computational Learning Theory},
  pages 454--468. Springer, 2007.

\bibitem[Ayoub et~al.(2020)Ayoub, Jia, Szepesv{\'a}ri, Wang, and
  Yang]{ayoub2020model}
A.~Ayoub, Z.~Jia, C.~Szepesv{\'a}ri, M.~Wang, and L.~Yang.
\newblock Model-based reinforcement learning with value-targeted regression.
\newblock In \emph{International Conference on Machine Learning}, pages
  463--474. PMLR, 2020.

\bibitem[Azar et~al.(2017)Azar, Osband, and Munos]{azar2017minimax}
M.~G. Azar, I.~Osband, and R.~Munos.
\newblock Minimax regret bounds for reinforcement learning.
\newblock In \emph{International Conference on Machine Learning}, pages
  263--272, 2017.

\bibitem[Bai et~al.(2022)Bai, Jones, Ndousse, Askell, Chen, DasSarma, Drain,
  Fort, Ganguli, Henighan, Joseph, Kadavath, Kernion, Conerly, El-Showk,
  Elhage, Hatfield-Dodds, Hernandez, Hume, Johnston, Kravec, Lovitt, Nanda,
  Olsson, Amodei, Brown, Clark, McCandlish, Olah, Mann, and
  Kaplan]{bai2022training}
Y.~Bai, A.~Jones, K.~Ndousse, A.~Askell, A.~Chen, N.~DasSarma, D.~Drain,
  S.~Fort, D.~Ganguli, T.~Henighan, N.~Joseph, S.~Kadavath, J.~Kernion,
  T.~Conerly, S.~El-Showk, N.~Elhage, Z.~Hatfield-Dodds, D.~Hernandez, T.~Hume,
  S.~Johnston, S.~Kravec, L.~Lovitt, N.~Nanda, C.~Olsson, D.~Amodei, T.~Brown,
  J.~Clark, S.~McCandlish, C.~Olah, B.~Mann, and J.~Kaplan.
\newblock Training a helpful and harmless assistant with reinforcement learning
  from human feedback.
\newblock \emph{arXiv:2204.05862}, 2022.

\bibitem[Bartlett et~al.(2005)Bartlett, Bousquet, and
  Mendelson]{bartlett2005local}
P.~L. Bartlett, O.~Bousquet, and S.~Mendelson.
\newblock Local {R}ademacher complexities.
\newblock \emph{The Annals of Statistics}, 33\penalty0 (4):\penalty0
  1497--1537, 2005.

\bibitem[Bellemare et~al.(2016)Bellemare, Srinivasan, Ostrovski, Schaul,
  Saxton, and Munos]{bellemare2016unifying}
M.~G. Bellemare, S.~Srinivasan, G.~Ostrovski, T.~Schaul, D.~Saxton, and
  R.~Munos.
\newblock Unifying count-based exploration and intrinsic motivation.
\newblock In \emph{Advances in Neural Information Processing Systems}, 2016.

\bibitem[Bellman(1954)]{bellman1954theory}
R.~Bellman.
\newblock The theory of dynamic programming.
\newblock \emph{Bulletin of the American Mathematical Society}, 60\penalty0
  (6):\penalty0 503--515, 1954.

\bibitem[Bietti et~al.(2021)Bietti, Agarwal, and
  Langford]{bietti2018contextual}
A.~Bietti, A.~Agarwal, and J.~Langford.
\newblock A contextual bandit bake-off.
\newblock \emph{Journal of Machine Learning Research}, 22\penalty0
  (133):\penalty0 1--49, 2021.

\bibitem[Bilodeau et~al.(2020)Bilodeau, Foster, and Roy]{bilodeau2020tight}
B.~Bilodeau, D.~J. Foster, and D.~Roy.
\newblock Tight bounds on minimax regret under logarithmic loss via
  self-concordance.
\newblock In \emph{International Conference on Machine Learning}, 2020.

\bibitem[Birg{\'e} and Massart(1998)]{birge1998minimum}
L.~Birg{\'e} and P.~Massart.
\newblock Minimum contrast estimators on sieves: exponential bounds and rates
  of convergence.
\newblock \emph{Bernoulli}, 4\penalty0 (3):\penalty0 329--375, 1998.

\bibitem[Black et~al.(2023)Black, Ilhan, Marchini, and
  Markeviciute]{black2023adaptex}
W.~Black, E.~Ilhan, A.~Marchini, and V.~Markeviciute.
\newblock {AdaptEx}: A self-service contextual bandit platform.
\newblock In \emph{Proceedings of the 17th ACM Conference on Recommender
  Systems}, 2023.
\newblock \doi{10.1145/3604915.3608870}.

\bibitem[Boucheron et~al.(2005)Boucheron, Bousquet, and
  Lugosi]{boucheron2005theory}
S.~Boucheron, O.~Bousquet, and G.~Lugosi.
\newblock Theory of classification: A survey of some recent advances.
\newblock \emph{ESAIM: Probability and Statistics}, 9:\penalty0 323--375, 2005.

\bibitem[Bousquet et~al.(2004)Bousquet, Boucheron, and
  Lugosi]{bousquet2004introduction}
O.~Bousquet, S.~Boucheron, and G.~Lugosi.
\newblock Introduction to statistical learning theory.
\newblock In \emph{Advanced Lectures on Machine Learning}, pages 169--207.
  Springer, 2004.

\bibitem[Brafman and Tennenholtz(2002)]{brafman2002r}
R.~I. Brafman and M.~Tennenholtz.
\newblock {R-max}---a general polynomial time algorithm for near-optimal
  reinforcement learning.
\newblock \emph{Journal of Machine Learning Research}, 3:\penalty0 213--231,
  2002.

\bibitem[Bubeck and Cesa-Bianchi(2012)]{bubeck2012regret}
S.~Bubeck and N.~Cesa-Bianchi.
\newblock Regret analysis of stochastic and nonstochastic multi-armed bandit
  problems.
\newblock \emph{Foundations and Trends{\textregistered} in Machine Learning},
  5\penalty0 (1):\penalty0 1--122, 2012.

\bibitem[Bubeck and Eldan(2016)]{bubeck2016multi}
S.~Bubeck and R.~Eldan.
\newblock Multi-scale exploration of convex functions and bandit convex
  optimization.
\newblock In \emph{Conference on Learning Theory}, pages 583--589, 2016.

\bibitem[Bubeck et~al.(2015)Bubeck, Dekel, Koren, and Peres]{bubeck2015bandit}
S.~Bubeck, O.~Dekel, T.~Koren, and Y.~Peres.
\newblock Bandit convex optimization: $\sqrt{T}$ regret in one dimension.
\newblock In \emph{Conference on Learning Theory}, pages 266--278, 2015.

\bibitem[Bubeck et~al.(2017)Bubeck, Lee, and Eldan]{bubeck2017kernel}
S.~Bubeck, Y.~T. Lee, and R.~Eldan.
\newblock Kernel-based methods for bandit convex optimization.
\newblock In \emph{Proceedings of the 49th Annual ACM SIGACT Symposium on
  Theory of Computing}, pages 72--85, 2017.

\bibitem[Burda et~al.(2019)Burda, Edwards, Storkey, and
  Klimov]{burda2018exploration}
Y.~Burda, H.~Edwards, A.~Storkey, and O.~Klimov.
\newblock Exploration by random network distillation.
\newblock In \emph{International Conference on Learning Representations}, 2019.

\bibitem[Cai et~al.(2020)Cai, Yang, Jin, and Wang]{cai2020provably}
Q.~Cai, Z.~Yang, C.~Jin, and Z.~Wang.
\newblock Provably efficient exploration in policy optimization.
\newblock In \emph{International Conference on Machine Learning}, 2020.

\bibitem[Cesa-Bianchi and Lugosi(1999)]{cesabianchi99logloss}
N.~Cesa-Bianchi and G.~Lugosi.
\newblock Minimax regret under log loss for general classes of experts.
\newblock In \emph{Proceedings of the Twelfth Annual Conference on
  Computational Learning Theory (COLT)}, pages 12--18. ACM, 1999.

\bibitem[Cesa-Bianchi and Lugosi(2006)]{cesa2006prediction}
N.~Cesa-Bianchi and G.~Lugosi.
\newblock \emph{Prediction, Learning, and Games}.
\newblock Cambridge University Press, 2006.

\bibitem[Cesa-Bianchi et~al.(2004)Cesa-Bianchi, Conconi, and
  Gentile]{cesa2004generalization}
N.~Cesa-Bianchi, A.~Conconi, and C.~Gentile.
\newblock On the generalization ability of on-line learning algorithms.
\newblock \emph{IEEE Transactions on Information Theory}, 50\penalty0
  (9):\penalty0 2050--2057, 2004.

\bibitem[Chapelle and Li(2011)]{chapelle2011empirical}
O.~Chapelle and L.~Li.
\newblock An empirical evaluation of {T}hompson sampling.
\newblock In \emph{Advances in Neural Information Processing Systems}, 2011.

\bibitem[Chen and Rakhlin(2025)]{chen2025decision}
F.~Chen and A.~Rakhlin.
\newblock Decision making in changing environments: Robustness, query-based
  learning, and differential privacy.
\newblock In \emph{Conference on Learning Theory}, volume 291 of
  \emph{Proceedings of Machine Learning Research}, pages 983--985. PMLR, 2025.

\bibitem[Chen et~al.(2023)Chen, Mei, and Bai]{chen2022partially}
F.~Chen, S.~Mei, and Y.~Bai.
\newblock Partially observable {RL} with {B}-stability: Unified structural
  condition and sharp sample-efficient algorithms.
\newblock In \emph{International Conference on Learning Representations}, 2023.

\bibitem[Chen et~al.(2024)Chen, Foster, Han, Qian, Rakhlin, and
  Xu]{chen2024assouad}
F.~Chen, D.~J. Foster, Y.~Han, J.~Qian, A.~Rakhlin, and Y.~Xu.
\newblock {Assouad, Fano, and Le Cam with Interaction: A Unifying Lower Bound
  Framework and Characterization for Bandit Learnability}.
\newblock In \emph{Advances in Neural Information Processing Systems},
  volume~37, pages 75585--75641, 2024.

\bibitem[Chen et~al.(2025)Chen, Mei, and Bai]{chen2022unified}
F.~Chen, S.~Mei, and Y.~Bai.
\newblock Unified algorithms for {RL} with decision-estimation coefficients:
  {PAC}, reward-free, preference-based learning, and beyond.
\newblock \emph{The Annals of Statistics}, 53\penalty0 (1):\penalty0 426--456,
  2025.

\bibitem[Chen et~al.(2026)Chen, Huang, Golowich, Malladi, Block, Ash,
  Krishnamurthy, and Foster]{chen2025when}
F.~Chen, A.~Huang, N.~Golowich, S.~Malladi, A.~Block, J.~T. Ash,
  A.~Krishnamurthy, and D.~J. Foster.
\newblock The coverage principle: How pre-training enables post-training.
\newblock In \emph{International Conference on Learning Representations
  (ICLR)}, 2026.

\bibitem[Chen and Jiang(2019)]{chen2019information}
J.~Chen and N.~Jiang.
\newblock Information-theoretic considerations in batch reinforcement learning.
\newblock In \emph{International Conference on Machine Learning}, 2019.

\bibitem[Christiano et~al.(2017)Christiano, Leike, Brown, Martic, Legg, and
  Amodei]{christiano2017deep}
P.~F. Christiano, J.~Leike, T.~Brown, M.~Martic, S.~Legg, and D.~Amodei.
\newblock Deep reinforcement learning from human preferences.
\newblock \emph{Advances in Neural Information Processing Systems}, 2017.

\bibitem[Chu et~al.(2011)Chu, Li, Reyzin, and Schapire]{chu2011contextual}
W.~Chu, L.~Li, L.~Reyzin, and R.~E. Schapire.
\newblock Contextual bandits with linear payoff functions.
\newblock In \emph{International Conference on Artificial Intelligence and
  Statistics}, 2011.

\bibitem[Cover(1967)]{cover1967behavior}
T.~M. Cover.
\newblock Behavior of sequential predictors of binary sequences.
\newblock In \emph{Transactions of the Fourth Prague Conference on Information
  Theory, Statistical Decision Functions, Random Processes}, pages 263--272,
  Prague, 1967. Publishing House of the Czechoslovak Academy of Sciences.

\bibitem[Cover(1991)]{cover1991universal}
T.~M. Cover.
\newblock Universal portfolios.
\newblock \emph{Mathematical Finance}, 1\penalty0 (1):\penalty0 1--29, 1991.

\bibitem[Csisz{\'a}r(1967)]{csiszar1967information}
I.~Csisz{\'a}r.
\newblock Information-type measures of difference of probability distributions
  and indirect observations.
\newblock \emph{Studia Scientiarum Mathematicarum Hungarica}, 2:\penalty0
  299--318, 1967.

\bibitem[Dani et~al.(2008)Dani, Hayes, and Kakade]{dani2008stochastic}
V.~Dani, T.~P. Hayes, and S.~M. Kakade.
\newblock Stochastic linear optimization under bandit feedback.
\newblock In \emph{Conference on Learning Theory (COLT)}, 2008.

\bibitem[Dann et~al.(2021)Dann, Mohri, Zhang, and Zimmert]{dann2021provably}
C.~Dann, M.~Mohri, T.~Zhang, and J.~Zimmert.
\newblock A provably efficient model-free posterior sampling method for
  episodic reinforcement learning.
\newblock \emph{Advances in Neural Information Processing Systems},
  34:\penalty0 12040--12051, 2021.

\bibitem[Dean et~al.(2020)Dean, Mania, Matni, Recht, and Tu]{dean2020sample}
S.~Dean, H.~Mania, N.~Matni, B.~Recht, and S.~Tu.
\newblock On the sample complexity of the linear quadratic regulator.
\newblock \emph{Foundations of Computational Mathematics}, 20\penalty0
  (4):\penalty0 633--679, 2020.

\bibitem[Devroye et~al.(1996)Devroye, Gy\"{o}rfi, and Lugosi]{DevroyeLugosi96}
L.~Devroye, L.~Gy\"{o}rfi, and G.~Lugosi.
\newblock \emph{A Probabilistic Theory of Pattern Recognition}.
\newblock Springer, 1996.

\bibitem[Dietterich et~al.(2018)Dietterich, Trimponias, and
  Chen]{dietterich2018discovering}
T.~G. Dietterich, G.~Trimponias, and Z.~Chen.
\newblock Discovering and removing exogenous state variables and rewards for
  reinforcement learning.
\newblock In \emph{International Conference on Machine Learning}, pages
  1262--1270, 2018.

\bibitem[Domingues et~al.(2021)Domingues, M{\'e}nard, Kaufmann, and
  Valko]{domingues2021episodic}
O.~D. Domingues, P.~M{\'e}nard, E.~Kaufmann, and M.~Valko.
\newblock Episodic reinforcement learning in finite {MDPs}: Minimax lower
  bounds revisited.
\newblock In \emph{Algorithmic Learning Theory}, pages 578--598. PMLR, 2021.

\bibitem[Donoho and Liu(1987)]{donoho1987geometrizing}
D.~L. Donoho and R.~C. Liu.
\newblock Geometrizing rates of convergence, {I}.
\newblock Technical Report 137a, Department of Statistics, University of
  California, Berkeley, 1987.

\bibitem[Donoho and Liu(1991)]{donoho1991geometrizingii}
D.~L. Donoho and R.~C. Liu.
\newblock Geometrizing rates of convergence, {II}.
\newblock \emph{The Annals of Statistics}, 19\penalty0 (2):\penalty0 633--667,
  1991.

\bibitem[Du et~al.(2019)Du, Krishnamurthy, Jiang, Agarwal, Dudik, and
  Langford]{du2019latent}
S.~Du, A.~Krishnamurthy, N.~Jiang, A.~Agarwal, M.~Dudik, and J.~Langford.
\newblock Provably efficient {RL} with rich observations via latent state
  decoding.
\newblock In \emph{International Conference on Machine Learning}, pages
  1665--1674. PMLR, 2019.

\bibitem[Du et~al.(2021)Du, Kakade, Lee, Lovett, Mahajan, Sun, and
  Wang]{du2021bilinear}
S.~S. Du, S.~M. Kakade, J.~D. Lee, S.~Lovett, G.~Mahajan, W.~Sun, and R.~Wang.
\newblock Bilinear classes: A structural framework for provable generalization
  in {RL}.
\newblock \emph{International Conference on Machine Learning}, 2021.

\bibitem[Duchi(2023)]{duchi-notes}
J.~C. Duchi.
\newblock Lecture notes on statistics and information theory.
\newblock Lecture notes, Stanford University, 2023.

\bibitem[Dudik et~al.(2011)Dudik, Hsu, Kale, Karampatziakis, Langford, Reyzin,
  and Zhang]{dudik2011efficient}
M.~Dudik, D.~Hsu, S.~Kale, N.~Karampatziakis, J.~Langford, L.~Reyzin, and
  T.~Zhang.
\newblock Efficient optimal learning for contextual bandits.
\newblock In \emph{Proceedings of the Twenty-Seventh Conference on Uncertainty
  in Artificial Intelligence}, pages 169--178. AUAI Press, 2011.

\bibitem[Ecoffet et~al.(2021)Ecoffet, Huizinga, Lehman, Stanley, and
  Clune]{ecoffet2021first}
A.~Ecoffet, J.~Huizinga, J.~Lehman, K.~O. Stanley, and J.~Clune.
\newblock First return, then explore.
\newblock \emph{Nature}, 590\penalty0 (7847):\penalty0 580--586, 2021.

\bibitem[Efroni et~al.(2022{\natexlab{a}})Efroni, Foster, Misra, Krishnamurthy,
  and Langford]{efroni2022sample}
Y.~Efroni, D.~J. Foster, D.~Misra, A.~Krishnamurthy, and J.~Langford.
\newblock Sample-efficient reinforcement learning in the presence of exogenous
  information.
\newblock In \emph{Conference on Learning Theory}, 2022{\natexlab{a}}.

\bibitem[Efroni et~al.(2022{\natexlab{b}})Efroni, Jin, Krishnamurthy, and
  Miryoosefi]{efroni2022provable}
Y.~Efroni, C.~Jin, A.~Krishnamurthy, and S.~Miryoosefi.
\newblock Provable reinforcement learning with a short-term memory.
\newblock In \emph{International Conference on Machine Learning},
  2022{\natexlab{b}}.

\bibitem[Efroni et~al.(2022{\natexlab{c}})Efroni, Misra, Krishnamurthy,
  Agarwal, and Langford]{efroni2021provable}
Y.~Efroni, D.~Misra, A.~Krishnamurthy, A.~Agarwal, and J.~Langford.
\newblock Provable {RL} with exogenous distractors via multistep inverse
  dynamics.
\newblock In \emph{International Conference on Learning Representations
  (ICLR)}, 2022{\natexlab{c}}.

\bibitem[Flaxman et~al.(2005)Flaxman, Kalai, and McMahan]{flaxman2005online}
A.~D. Flaxman, A.~T. Kalai, and H.~B. McMahan.
\newblock Online convex optimization in the bandit setting: gradient descent
  without a gradient.
\newblock In \emph{Proceedings of the sixteenth annual ACM-SIAM symposium on
  Discrete algorithms}, pages 385--394, 2005.

\bibitem[Foster and Krishnamurthy(2021)]{foster2021efficient}
D.~J. Foster and A.~Krishnamurthy.
\newblock Efficient first-order contextual bandits: Prediction, allocation, and
  triangular discrimination.
\newblock \emph{Neural Information Processing Systems (NeurIPS)}, 2021.

\bibitem[Foster and Rakhlin(2020)]{foster2020beyond}
D.~J. Foster and A.~Rakhlin.
\newblock Beyond {UCB}: Optimal and efficient contextual bandits with
  regression oracles.
\newblock \emph{International Conference on Machine Learning (ICML)}, 2020.

\bibitem[Foster et~al.(2018{\natexlab{a}})Foster, Agarwal, Dud{\'\i}k, Luo, and
  Schapire]{foster2018practical}
D.~J. Foster, A.~Agarwal, M.~Dud{\'\i}k, H.~Luo, and R.~E. Schapire.
\newblock Practical contextual bandits with regression oracles.
\newblock \emph{International Conference on Machine Learning},
  2018{\natexlab{a}}.

\bibitem[Foster et~al.(2018{\natexlab{b}})Foster, Kale, Luo, Mohri, and
  Sridharan]{foster2018logistic}
D.~J. Foster, S.~Kale, H.~Luo, M.~Mohri, and K.~Sridharan.
\newblock Logistic regression: The importance of being improper.
\newblock \emph{Conference on Learning Theory}, 2018{\natexlab{b}}.

\bibitem[Foster et~al.(2018{\natexlab{c}})Foster, Rakhlin, and
  Sridharan]{foster2018online}
D.~J. Foster, A.~Rakhlin, and K.~Sridharan.
\newblock Online learning: Sufficient statistics and the burkholder method.
\newblock \emph{Conference on Learning Theory}, 2018{\natexlab{c}}.

\bibitem[Foster et~al.(2020)Foster, Gentile, Mohri, and
  Zimmert]{foster2020adapting}
D.~J. Foster, C.~Gentile, M.~Mohri, and J.~Zimmert.
\newblock Adapting to misspecification in contextual bandits.
\newblock \emph{Advances in Neural Information Processing Systems}, 33, 2020.

\bibitem[Foster et~al.(2021{\natexlab{a}})Foster, Kakade, Qian, and
  Rakhlin]{foster2021statistical}
D.~J. Foster, S.~M. Kakade, J.~Qian, and A.~Rakhlin.
\newblock The statistical complexity of interactive decision making.
\newblock \emph{arXiv preprint arXiv:2112.13487}, 2021{\natexlab{a}}.

\bibitem[Foster et~al.(2021{\natexlab{b}})Foster, Rakhlin, Simchi-Levi, and
  Xu]{foster2020instance}
D.~J. Foster, A.~Rakhlin, D.~Simchi-Levi, and Y.~Xu.
\newblock Instance-dependent complexity of contextual bandits and reinforcement
  learning: A disagreement-based perspective.
\newblock In \emph{Conference on Learning Theory (COLT)}, volume 134 of
  \emph{Proceedings of Machine Learning Research}, page 2059,
  2021{\natexlab{b}}.

\bibitem[Foster et~al.(2022{\natexlab{a}})Foster, Krishnamurthy, Simchi-Levi,
  and Xu]{foster2022offline}
D.~J. Foster, A.~Krishnamurthy, D.~Simchi-Levi, and Y.~Xu.
\newblock Offline reinforcement learning: Fundamental barriers for value
  function approximation.
\newblock In \emph{Conference on Learning Theory}, 2022{\natexlab{a}}.

\bibitem[Foster et~al.(2022{\natexlab{b}})Foster, Rakhlin, Sekhari, and
  Sridharan]{foster2022complexity}
D.~J. Foster, A.~Rakhlin, A.~Sekhari, and K.~Sridharan.
\newblock On the complexity of adversarial decision making.
\newblock In \emph{Advances in Neural Information Processing Systems},
  volume~35, 2022{\natexlab{b}}.

\bibitem[Foster et~al.(2023{\natexlab{a}})Foster, Foster, Golowich, and
  Rakhlin]{foster2023multiagent}
D.~J. Foster, D.~P. Foster, N.~Golowich, and A.~Rakhlin.
\newblock On the complexity of multi-agent decision making: From learning in
  games to partial monitoring.
\newblock In \emph{Conference on Learning Theory}, 2023{\natexlab{a}}.

\bibitem[Foster et~al.(2023{\natexlab{b}})Foster, Golowich, and
  Han]{foster2023tight}
D.~J. Foster, N.~Golowich, and Y.~Han.
\newblock Tight guarantees for interactive decision making with the
  decision-estimation coefficient.
\newblock In \emph{Conference on Learning Theory (COLT)}, volume 195 of
  \emph{Proceedings of Machine Learning Research}, pages 3969--4043,
  2023{\natexlab{b}}.

\bibitem[Foster et~al.(2023{\natexlab{c}})Foster, Golowich, Qian, Rakhlin, and
  Sekhari]{foster2022note}
D.~J. Foster, N.~Golowich, J.~Qian, A.~Rakhlin, and A.~Sekhari.
\newblock Model-free reinforcement learning with the decision-estimation
  coefficient.
\newblock In \emph{Advances in Neural Information Processing Systems},
  volume~36, 2023{\natexlab{c}}.

\bibitem[Foster et~al.(2024{\natexlab{a}})Foster, Block, and
  Misra]{foster2024behavior}
D.~J. Foster, A.~Block, and D.~Misra.
\newblock Is behavior cloning all you need? understanding horizon in imitation
  learning.
\newblock \emph{Advances in Neural Information Processing Systems},
  37:\penalty0 120602--120666, 2024{\natexlab{a}}.

\bibitem[Foster et~al.(2024{\natexlab{b}})Foster, Han, Qian, and
  Rakhlin]{foster2024online}
D.~J. Foster, Y.~Han, J.~Qian, and A.~Rakhlin.
\newblock Online estimation via offline estimation: An information-theoretic
  framework.
\newblock In \emph{Advances in Neural Information Processing Systems},
  volume~37, 2024{\natexlab{b}}.

\bibitem[Foster et~al.(2025)Foster, Mhammedi, and Rohatgi]{foster2025good}
D.~J. Foster, Z.~Mhammedi, and D.~Rohatgi.
\newblock Is a good foundation necessary for efficient reinforcement learning?
  {T}he computational role of the base model in exploration.
\newblock In \emph{Conference on Learning Theory}, volume 291 of
  \emph{Proceedings of Machine Learning Research}, pages 2026--2142. PMLR,
  2025.

\bibitem[Freund and Schapire(1997)]{freund1997decision}
Y.~Freund and R.~E. Schapire.
\newblock A decision-theoretic generalization of on-line learning and an
  application to boosting.
\newblock \emph{Journal of Computer and System Sciences}, 55\penalty0
  (1):\penalty0 119--139, 1997.

\bibitem[Garivier and Capp{\'e}(2011)]{garivier2011kl}
A.~Garivier and O.~Capp{\'e}.
\newblock The {KL-UCB} algorithm for bounded stochastic bandits and beyond.
\newblock In \emph{Proceedings of the 24th annual Conference On Learning
  Theory}, pages 359--376, 2011.

\bibitem[Gerrits(2020)]{gerrits2020vowpal}
J.~Gerrits.
\newblock {VowpalWabbit} 8.9.0 release notes.
\newblock \url{https://vowpalwabbit.org/blog/vowpalwabbit-8.9.0.html}, 2020.
\newblock Accessed 2026-09-16.

\bibitem[Glasgow and Rakhlin(2023)]{glasgow2023tight}
M.~Glasgow and A.~Rakhlin.
\newblock Tight bounds for $\gamma$-regret via the decision-estimation
  coefficient.
\newblock \emph{arXiv preprint arXiv:2303.03327}, 2023.

\bibitem[Golowich et~al.(2022)Golowich, Moitra, and
  Rohatgi]{golowich2022learning}
N.~Golowich, A.~Moitra, and D.~Rohatgi.
\newblock Learning in observable {POMDPs}, without computationally intractable
  oracles.
\newblock In \emph{Advances in Neural Information Processing Systems}, 2022.

\bibitem[Golowich et~al.(2024)Golowich, Moitra, and
  Rohatgi]{golowich2024exploration}
N.~Golowich, A.~Moitra, and D.~Rohatgi.
\newblock Exploration is harder than prediction: Cryptographically separating
  reinforcement learning from supervised learning.
\newblock In \emph{2024 IEEE 65th Annual Symposium on Foundations of Computer
  Science (FOCS)}, pages 1953--1967. IEEE, 2024.

\bibitem[Hannan(1957)]{Hannan57}
J.~Hannan.
\newblock Approximation to {B}ayes risk in repeated play.
\newblock \emph{Contributions to the Theory of Games}, 3:\penalty0 97--139,
  1957.

\bibitem[Hazan(2016)]{hazan2016introduction}
E.~Hazan.
\newblock Introduction to online convex optimization.
\newblock \emph{Foundations and Trends{\textregistered} in Optimization},
  2\penalty0 (3-4):\penalty0 157--325, 2016.

\bibitem[Hazan and Kale(2015)]{hazan2015online}
E.~Hazan and S.~Kale.
\newblock An online portfolio selection algorithm with regret logarithmic in
  price variation.
\newblock \emph{Mathematical Finance}, 25\penalty0 (2):\penalty0 288--310,
  2015.

\bibitem[He et~al.(2025)He, Wu, and Jindal]{he2025uber}
L.~He, Y.~Wu, and G.~Jindal.
\newblock Enhancing personalized {CRM} communication with contextual bandit
  strategies.
\newblock Uber Engineering Blog,
  \url{https://www.uber.com/blog/enhancing-personalized-crm/}, 2025.
\newblock Accessed 2026-09-16.

\bibitem[Howard et~al.(2020)Howard, Ramdas, McAuliffe, and
  Sekhon]{howard2020time}
S.~R. Howard, A.~Ramdas, J.~McAuliffe, and J.~Sekhon.
\newblock Time-uniform {C}hernoff bounds via nonnegative supermartingales.
\newblock \emph{Probability Surveys}, 17:\penalty0 257--317, 2020.

\bibitem[Huang et~al.(2025{\natexlab{a}})Huang, Block, Foster, Rohatgi, Zhang,
  Simchowitz, Ash, and Krishnamurthy]{huang2025self}
A.~Huang, A.~Block, D.~J. Foster, D.~Rohatgi, C.~Zhang, M.~Simchowitz, J.~T.
  Ash, and A.~Krishnamurthy.
\newblock Self-improvement in language models: {T}he sharpening mechanism.
\newblock In \emph{International Conference on Learning Representations},
  2025{\natexlab{a}}.

\bibitem[Huang et~al.(2025{\natexlab{b}})Huang, Block, Liu, Jiang,
  Krishnamurthy, and Foster]{huang2025best}
A.~Huang, A.~Block, Q.~Liu, N.~Jiang, A.~Krishnamurthy, and D.~J. Foster.
\newblock Is best-of-{N} the best of them? coverage, scaling, and optimality in
  inference-time alignment.
\newblock In \emph{International Conference on Machine Learning (ICML)}, volume
  267 of \emph{Proceedings of Machine Learning Research}, pages 25075--25126,
  2025{\natexlab{b}}.

\bibitem[Jaksch et~al.(2010)Jaksch, Ortner, and Auer]{jaksch2010near}
T.~Jaksch, R.~Ortner, and P.~Auer.
\newblock Near-optimal regret bounds for reinforcement learning.
\newblock \emph{Journal of Machine Learning Research}, 11:\penalty0 1563--1600,
  2010.

\bibitem[Jayram(2009)]{jayram2009hellinger}
T.~Jayram.
\newblock Hellinger strikes back: A note on the multi-party information
  complexity of and.
\newblock In \emph{International Workshop on Approximation Algorithms for
  Combinatorial Optimization}, pages 562--573. Springer, 2009.

\bibitem[Jiang and Agarwal(2018)]{jiang2018open}
N.~Jiang and A.~Agarwal.
\newblock Open problem: The dependence of sample complexity lower bounds on
  planning horizon.
\newblock In \emph{Conference On Learning Theory}, pages 3395--3398. PMLR,
  2018.

\bibitem[Jiang et~al.(2017)Jiang, Krishnamurthy, Agarwal, Langford, and
  Schapire]{jiang2017contextual}
N.~Jiang, A.~Krishnamurthy, A.~Agarwal, J.~Langford, and R.~E. Schapire.
\newblock Contextual decision processes with low {Bellman} rank are
  {PAC}-learnable.
\newblock In \emph{International Conference on Machine Learning}, pages
  1704--1713, 2017.

\bibitem[Jin et~al.(2018)Jin, Allen-Zhu, Bubeck, and Jordan]{jin2018q}
C.~Jin, Z.~Allen-Zhu, S.~Bubeck, and M.~I. Jordan.
\newblock Is {Q}-learning provably efficient?
\newblock In \emph{Advances in Neural Information Processing Systems}, 2018.

\bibitem[Jin et~al.(2020{\natexlab{a}})Jin, Kakade, Krishnamurthy, and
  Liu]{jin2020sample}
C.~Jin, S.~Kakade, A.~Krishnamurthy, and Q.~Liu.
\newblock Sample-efficient reinforcement learning of undercomplete {POMDPs}.
\newblock \emph{Advances in Neural Information Processing Systems}, 33,
  2020{\natexlab{a}}.

\bibitem[Jin et~al.(2020{\natexlab{b}})Jin, Krishnamurthy, Simchowitz, and
  Yu]{jin2020reward}
C.~Jin, A.~Krishnamurthy, M.~Simchowitz, and T.~Yu.
\newblock Reward-free exploration for reinforcement learning.
\newblock In \emph{International Conference on Machine Learning}, pages
  4870--4879. PMLR, 2020{\natexlab{b}}.

\bibitem[Jin et~al.(2020{\natexlab{c}})Jin, Yang, Wang, and
  Jordan]{jin2020provably}
C.~Jin, Z.~Yang, Z.~Wang, and M.~I. Jordan.
\newblock Provably efficient reinforcement learning with linear function
  approximation.
\newblock In \emph{Conference on Learning Theory}, pages 2137--2143,
  2020{\natexlab{c}}.

\bibitem[Jin et~al.(2021{\natexlab{a}})Jin, Liu, and
  Miryoosefi]{jin2021bellman}
C.~Jin, Q.~Liu, and S.~Miryoosefi.
\newblock Bellman eluder dimension: New rich classes of {RL} problems, and
  sample-efficient algorithms.
\newblock \emph{Advances in Neural Information Processing Systems},
  34:\penalty0 13406--13418, 2021{\natexlab{a}}.

\bibitem[Jin et~al.(2021{\natexlab{b}})Jin, Yang, and Wang]{jin2021pessimism}
Y.~Jin, Z.~Yang, and Z.~Wang.
\newblock Is pessimism provably efficient for offline {RL}?
\newblock In \emph{International Conference on Machine Learning},
  2021{\natexlab{b}}.

\bibitem[Juditsky et~al.(2008)Juditsky, Rigollet, and
  Tsybakov]{juditsky2008learning}
A.~Juditsky, P.~Rigollet, and A.~B. Tsybakov.
\newblock Learning by mirror averaging.
\newblock \emph{The Annals of Statistics}, 36\penalty0 (5):\penalty0
  2183--2206, 2008.

\bibitem[Jun and Zhang(2020)]{jun2020crush}
K.-S. Jun and C.~Zhang.
\newblock Crush optimism with pessimism: Structured bandits beyond asymptotic
  optimality.
\newblock \emph{Advances in Neural Information Processing Systems}, 33, 2020.

\bibitem[Kaelbling et~al.(1998)Kaelbling, Littman, and
  Cassandra]{kaelbling1998planning}
L.~P. Kaelbling, M.~L. Littman, and A.~R. Cassandra.
\newblock Planning and acting in partially observable stochastic domains.
\newblock \emph{Artificial Intelligence}, 101\penalty0 (1--2):\penalty0
  99--134, 1998.

\bibitem[Kakade(2003)]{kakade2003sample}
S.~M. Kakade.
\newblock \emph{On the sample complexity of reinforcement learning}.
\newblock PhD thesis, Gatsby Computational Neuroscience Unit, University
  College London, 2003.

\bibitem[Kalai and Vempala(2002)]{kalai2002efficient}
A.~Kalai and S.~Vempala.
\newblock Efficient algorithms for universal portfolios.
\newblock \emph{Journal of Machine Learning Research}, 3:\penalty0 423--440,
  2002.

\bibitem[Kaufmann et~al.(2012)Kaufmann, Korda, and Munos]{kaufmann2012thompson}
E.~Kaufmann, N.~Korda, and R.~Munos.
\newblock Thompson sampling: An asymptotically optimal finite-time analysis.
\newblock In \emph{International Conference on Algorithmic Learning Theory},
  pages 199--213, 2012.

\bibitem[Kearns and Singh(2002)]{kearns2002near}
M.~Kearns and S.~Singh.
\newblock Near-optimal reinforcement learning in polynomial time.
\newblock \emph{Machine Learning}, 49\penalty0 (2--3):\penalty0 209--232, 2002.

\bibitem[Kiefer and Wolfowitz(1960)]{kiefer1960equivalence}
J.~Kiefer and J.~Wolfowitz.
\newblock The equivalence of two extremum problems.
\newblock \emph{Canadian Journal of Mathematics}, 12:\penalty0 363--366, 1960.

\bibitem[Kleinberg(2004)]{kleinberg2004nearly}
R.~Kleinberg.
\newblock Nearly tight bounds for the continuum-armed bandit problem.
\newblock \emph{Advances in Neural Information Processing Systems},
  17:\penalty0 697--704, 2004.

\bibitem[Kleinberg et~al.(2008)Kleinberg, Slivkins, and
  Upfal]{kleinberg2008multi}
R.~Kleinberg, A.~Slivkins, and E.~Upfal.
\newblock Multi-armed bandits in metric spaces.
\newblock In \emph{Proceedings of the fortieth annual ACM symposium on Theory
  of computing}, pages 681--690, 2008.

\bibitem[Kleinberg et~al.(2019)Kleinberg, Slivkins, and
  Upfal]{kleinberg2019bandits}
R.~Kleinberg, A.~Slivkins, and E.~Upfal.
\newblock Bandits and experts in metric spaces.
\newblock \emph{Journal of the ACM (JACM)}, 66\penalty0 (4):\penalty0 1--77,
  2019.

\bibitem[Koltchinskii(2006)]{koltchinskii2006local}
V.~Koltchinskii.
\newblock Local {Rademacher} complexities and oracle inequalities in risk
  minimization.
\newblock \emph{The Annals of Statistics}, 34\penalty0 (6):\penalty0
  2593--2656, 2006.

\bibitem[Krishnamurthy et~al.(2016)Krishnamurthy, Agarwal, and
  Langford]{krishnamurthy2016pac}
A.~Krishnamurthy, A.~Agarwal, and J.~Langford.
\newblock {PAC} reinforcement learning with rich observations.
\newblock In \emph{Advances in Neural Information Processing Systems}, pages
  1840--1848, 2016.

\bibitem[Lai and Robbins(1985)]{lai1985asymptotically}
T.~L. Lai and H.~Robbins.
\newblock Asymptotically efficient adaptive allocation rules.
\newblock \emph{Advances in Applied Mathematics}, 6\penalty0 (1):\penalty0
  4--22, 1985.

\bibitem[Lamb et~al.(2023)Lamb, Islam, Efroni, Didolkar, Misra, Foster, Molu,
  Chari, Krishnamurthy, and Langford]{lamb2023guaranteed}
A.~Lamb, R.~Islam, Y.~Efroni, A.~Didolkar, D.~Misra, D.~J. Foster, L.~Molu,
  R.~Chari, A.~Krishnamurthy, and J.~Langford.
\newblock Guaranteed discovery of control-endogenous latent states with
  multi-step inverse models.
\newblock \emph{Transactions on Machine Learning Research}, 2023.

\bibitem[Langford and Zhang(2008)]{langford2008epoch}
J.~Langford and T.~Zhang.
\newblock The epoch-greedy algorithm for multi-armed bandits with side
  information.
\newblock In \emph{Advances in Neural Information Processing Systems}, pages
  817--824, 2008.

\bibitem[Laplace(1814)]{laplace1814essai}
P.-S. Laplace.
\newblock \emph{Essai philosophique sur les probabilit{\'e}s}.
\newblock Courcier, Paris, 1814.
\newblock English translation: \emph{A Philosophical Essay on Probabilities},
  Dover, 1951.

\bibitem[Lattimore(2020)]{lattimore2020improved}
T.~Lattimore.
\newblock Improved regret for zeroth-order adversarial bandit convex
  optimisation.
\newblock \emph{Mathematical Statistics and Learning}, 2\penalty0 (3):\penalty0
  311--334, 2020.

\bibitem[Lattimore and Gy{\"o}rgy(2021)]{lattimore2021mirror}
T.~Lattimore and A.~Gy{\"o}rgy.
\newblock Mirror descent and the information ratio.
\newblock In \emph{Conference on Learning Theory}, pages 2965--2992. PMLR,
  2021.

\bibitem[Lattimore and Szepesv{\'a}ri(2017)]{lattimore2017end}
T.~Lattimore and C.~Szepesv{\'a}ri.
\newblock The end of optimism? {A}n asymptotic analysis of finite-armed linear
  bandits.
\newblock In \emph{International Conference on Artificial Intelligence and
  Statistics}, pages 728--737, 2017.

\bibitem[Lattimore and Szepesv{\'a}ri(2019)]{lattimore2019information}
T.~Lattimore and C.~Szepesv{\'a}ri.
\newblock An information-theoretic approach to minimax regret in partial
  monitoring.
\newblock In \emph{Conference on Learning Theory}, pages 2111--2139. PMLR,
  2019.

\bibitem[Lattimore and Szepesv{\'a}ri(2020{\natexlab{a}})]{lattimore2020bandit}
T.~Lattimore and C.~Szepesv{\'a}ri.
\newblock \emph{Bandit algorithms}.
\newblock Cambridge University Press, 2020{\natexlab{a}}.

\bibitem[Lattimore and
  Szepesv{\'a}ri(2020{\natexlab{b}})]{lattimore2020exploration}
T.~Lattimore and C.~Szepesv{\'a}ri.
\newblock Exploration by optimisation in partial monitoring.
\newblock In \emph{Conference on Learning Theory}, pages 2488--2515. PMLR,
  2020{\natexlab{b}}.

\bibitem[Le~Cam(1973)]{lecam1973convergence}
L.~Le~Cam.
\newblock Convergence of estimates under dimensionality restrictions.
\newblock \emph{The Annals of Statistics}, 1\penalty0 (1):\penalty0 38--53,
  1973.

\bibitem[Lecu{\'e} and Mendelson(2009)]{lecue2009aggregation}
G.~Lecu{\'e} and S.~Mendelson.
\newblock Aggregation via empirical risk minimization.
\newblock \emph{Probability Theory and Related Fields}, 145\penalty0
  (3):\penalty0 591--613, 2009.

\bibitem[Levine et~al.(2020)Levine, Kumar, Tucker, and Fu]{levine2020offline}
S.~Levine, A.~Kumar, G.~Tucker, and J.~Fu.
\newblock Offline reinforcement learning: Tutorial, review, and perspectives on
  open problems.
\newblock \emph{arXiv preprint arXiv:2005.01643}, 2020.

\bibitem[Li et~al.(2022)Li, Kamath, Foster, and Srebro]{li2021eluder}
G.~Li, P.~Kamath, D.~J. Foster, and N.~Srebro.
\newblock Understanding the eluder dimension.
\newblock \emph{Advances in Neural Information Processing Systems},
  35:\penalty0 23737--23750, 2022.

\bibitem[Li et~al.(2010)Li, Chu, Langford, and Schapire]{li2010contextual}
L.~Li, W.~Chu, J.~Langford, and R.~E. Schapire.
\newblock A contextual-bandit approach to personalized news article
  recommendation.
\newblock In \emph{Proceedings of the 19th International Conference on World
  Wide Web}, pages 661--670. ACM, 2010.

\bibitem[Littlestone and Warmuth(1994)]{littlestone1994weighted}
N.~Littlestone and M.~K. Warmuth.
\newblock The weighted majority algorithm.
\newblock \emph{Information and Computation}, 108\penalty0 (2):\penalty0
  212--261, 1994.

\bibitem[Liu et~al.(2022)Liu, Chung, Szepesv{\'a}ri, and Jin]{liu2022partially}
Q.~Liu, A.~Chung, C.~Szepesv{\'a}ri, and C.~Jin.
\newblock When is partially observable reinforcement learning not scary?
\newblock In \emph{Conference on Learning Theory}, pages 5175--5220, 2022.

\bibitem[Luo et~al.(2018)Luo, Wei, and Zheng]{luo2018efficient}
H.~Luo, C.-Y. Wei, and K.~Zheng.
\newblock Efficient online portfolio with logarithmic regret.
\newblock In \emph{Advances in Neural Information Processing Systems}, 2018.

\bibitem[Mannor and Shamir(2011)]{mannor2011bandits}
S.~Mannor and O.~Shamir.
\newblock From bandits to experts: On the value of side-observations.
\newblock In \emph{Advances in Neural Information Processing Systems}, 2011.

\bibitem[Mhammedi et~al.(2023)Mhammedi, Block, Foster, and
  Rakhlin]{mhammedi2023efficient}
Z.~Mhammedi, A.~Block, D.~J. Foster, and A.~Rakhlin.
\newblock Efficient model-free exploration in low-rank {MDPs}.
\newblock In \emph{Advances in Neural Information Processing Systems}, 2023.

\bibitem[Mhammedi et~al.(2024)Mhammedi, Foster, and Rakhlin]{mhammedi2024power}
Z.~Mhammedi, D.~J. Foster, and A.~Rakhlin.
\newblock The power of resets in online reinforcement learning.
\newblock \emph{Advances in Neural Information Processing Systems},
  37:\penalty0 12334--12407, 2024.

\bibitem[Misra et~al.(2020)Misra, Henaff, Krishnamurthy, and
  Langford]{misra2020kinematic}
D.~Misra, M.~Henaff, A.~Krishnamurthy, and J.~Langford.
\newblock Kinematic state abstraction and provably efficient rich-observation
  reinforcement learning.
\newblock In \emph{International Conference on Machine Learning}, pages
  6961--6971. PMLR, 2020.

\bibitem[Mnih et~al.(2015)Mnih, Kavukcuoglu, Silver, Rusu, Veness, Bellemare,
  Graves, Riedmiller, Fidjeland, Ostrovski, et~al.]{mnih2015human}
V.~Mnih, K.~Kavukcuoglu, D.~Silver, A.~A. Rusu, J.~Veness, M.~G. Bellemare,
  A.~Graves, M.~Riedmiller, A.~K. Fidjeland, G.~Ostrovski, et~al.
\newblock Human-level control through deep reinforcement learning.
\newblock \emph{Nature}, 518\penalty0 (7540):\penalty0 529--533, 2015.

\bibitem[Modi et~al.(2020)Modi, Jiang, Tewari, and Singh]{modi2020sample}
A.~Modi, N.~Jiang, A.~Tewari, and S.~Singh.
\newblock Sample complexity of reinforcement learning using linearly combined
  model ensembles.
\newblock In \emph{International Conference on Artificial Intelligence and
  Statistics}, pages 2010--2020. PMLR, 2020.

\bibitem[Munos and Szepesv{\'a}ri(2008)]{munos2008finite}
R.~Munos and C.~Szepesv{\'a}ri.
\newblock Finite-time bounds for fitted value iteration.
\newblock \emph{Journal of Machine Learning Research}, 9:\penalty0 815--857,
  2008.

\bibitem[Opper and Haussler(1999)]{opper99logloss}
M.~Opper and D.~Haussler.
\newblock Worst case prediction over sequences under log loss.
\newblock In G.~Cybenko, D.~P. O'Leary, and J.~Rissanen, editors, \emph{The
  Mathematics of Information Coding, Extraction and Distribution}, volume 107
  of \emph{The IMA Volumes in Mathematics and its Applications}, pages 81--90.
  Springer, New York, 1999.

\bibitem[Orabona(2019)]{orabona2019modern}
F.~Orabona.
\newblock A modern introduction to online learning.
\newblock \emph{arXiv preprint arXiv:1912.13213}, 2019.

\bibitem[Orseau et~al.(2017)Orseau, Lattimore, and Legg]{orseau2017soft}
L.~Orseau, T.~Lattimore, and S.~Legg.
\newblock {Soft-Bayes}: Prod for mixtures of experts with log-loss.
\newblock In \emph{International Conference on Algorithmic Learning Theory},
  pages 372--399. PMLR, 2017.

\bibitem[Osband and Van~Roy(2016)]{osband2016lower}
I.~Osband and B.~Van~Roy.
\newblock On lower bounds for regret in reinforcement learning.
\newblock \emph{arXiv preprint arXiv:1608.02732}, 2016.

\bibitem[Osband and Van~Roy(2017)]{osband2017posterior}
I.~Osband and B.~Van~Roy.
\newblock Why is posterior sampling better than optimism for reinforcement
  learning?
\newblock In \emph{International Conference on Machine Learning}, pages
  2701--2710. PMLR, 2017.

\bibitem[Osband et~al.(2013)Osband, Russo, and Van~Roy]{osband2013more}
I.~Osband, D.~Russo, and B.~Van~Roy.
\newblock (more) efficient reinforcement learning via posterior sampling.
\newblock \emph{Advances in Neural Information Processing Systems},
  26:\penalty0 3003--3011, 2013.

\bibitem[Osband et~al.(2016)Osband, Blundell, Pritzel, and
  Van~Roy]{osband2016deep}
I.~Osband, C.~Blundell, A.~Pritzel, and B.~Van~Roy.
\newblock Deep exploration via bootstrapped {DQN}.
\newblock In \emph{Advances in Neural Information Processing Systems}, pages
  4026--4034, 2016.

\bibitem[Ouyang et~al.(2022)Ouyang, Wu, Jiang, Almeida, Wainwright, Mishkin,
  Zhang, Agarwal, Slama, Ray, Schulman, Hilton, Kelton, Miller, Simens, Askell,
  Welinder, Christiano, Leike, and Lowe]{ouyang2022training}
L.~Ouyang, J.~Wu, X.~Jiang, D.~Almeida, C.~Wainwright, P.~Mishkin, C.~Zhang,
  S.~Agarwal, K.~Slama, A.~Ray, J.~Schulman, J.~Hilton, F.~Kelton, L.~Miller,
  M.~Simens, A.~Askell, P.~Welinder, P.~Christiano, J.~Leike, and R.~Lowe.
\newblock Training language models to follow instructions with human feedback.
\newblock \emph{Advances in Neural Information Processing Systems}, 2022.

\bibitem[Pathak et~al.(2017)Pathak, Agrawal, Efros, and
  Darrell]{pathak2017curiosity}
D.~Pathak, P.~Agrawal, A.~A. Efros, and T.~Darrell.
\newblock Curiosity-driven exploration by self-supervised prediction.
\newblock In \emph{International Conference on Machine Learning}, 2017.

\bibitem[Polyanskiy(2020)]{polyanskiy2020lecture}
Y.~Polyanskiy.
\newblock Information theoretic methods in statistics and computer science.
\newblock Class notes, EPFL 2019 / ENSAE-CREST 2020.
  \url{https://people.lids.mit.edu/yp/homepage/sdpi_course.shtml}, 2020.

\bibitem[Polyanskiy and Wu(2014)]{polyanskiy2014lecture}
Y.~Polyanskiy and Y.~Wu.
\newblock Lecture notes on information theory.
\newblock Lecture notes for MIT 6.441 / UIUC ECE 563 / Yale S\&DS 664, 2014.
\newblock Superseded by \emph{Information Theory: From Coding to Learning},
  Cambridge University Press, 2025.

\bibitem[Polyanskiy and Wu(2026)]{polyanskiy2019dualizing}
Y.~Polyanskiy and Y.~Wu.
\newblock Dualizing {L}e {C}am's method for functional estimation {I}: General
  theory.
\newblock \emph{The Annals of Statistics}, 54\penalty0 (1):\penalty0 1--24,
  2026.
\newblock Preprint: arXiv:1902.05616.

\bibitem[Pomerleau(1988)]{pomerleau1988alvinn}
D.~A. Pomerleau.
\newblock {ALVINN}: An autonomous land vehicle in a neural network.
\newblock In \emph{Advances in Neural Information Processing Systems}, 1988.

\bibitem[Puterman(1994)]{puterman1994markov}
M.~L. Puterman.
\newblock \emph{Markov Decision Processes: Discrete Stochastic Dynamic
  Programming}.
\newblock John Wiley \& Sons, 1994.

\bibitem[Rafailov et~al.(2023)Rafailov, Sharma, Mitchell, Manning, Ermon, and
  Finn]{rafailov2024direct}
R.~Rafailov, A.~Sharma, E.~Mitchell, C.~D. Manning, S.~Ermon, and C.~Finn.
\newblock Direct preference optimization: Your language model is secretly a
  reward model.
\newblock \emph{Advances in Neural Information Processing Systems}, 2023.

\bibitem[Rakhlin and Sridharan(2012)]{StatNotes2012}
A.~Rakhlin and K.~Sridharan.
\newblock Statistical learning and sequential prediction, 2012.
\newblock Available at {\small
  \url{http://www.mit.edu/~rakhlin/courses/stat928/stat928_notes.pdf}}.

\bibitem[Rakhlin and Sridharan(2017)]{rakhlin2017equivalence}
A.~Rakhlin and K.~Sridharan.
\newblock On equivalence of martingale tail bounds and deterministic regret
  inequalities.
\newblock In \emph{Conference on Learning Theory}, pages 1704--1722, 2017.

\bibitem[Rakhlin et~al.(2015)Rakhlin, Sridharan, and
  Tewari]{rakhlin2015sequential}
A.~Rakhlin, K.~Sridharan, and A.~Tewari.
\newblock Sequential complexities and uniform martingale laws of large numbers.
\newblock \emph{Probability Theory and Related Fields}, 161\penalty0
  (1-2):\penalty0 111--153, 2015.

\bibitem[Rashidinejad et~al.(2021)Rashidinejad, Zhu, Ma, Jiao, and
  Russell]{rashidinejad2021bridging}
P.~Rashidinejad, B.~Zhu, C.~Ma, J.~Jiao, and S.~Russell.
\newblock Bridging offline reinforcement learning and imitation learning: A
  tale of pessimism.
\newblock In \emph{Advances in Neural Information Processing Systems}, 2021.

\bibitem[Rendle et~al.(2010)Rendle, Freudenthaler, and
  Schmidt{-}Thieme]{RendleFS10}
S.~Rendle, C.~Freudenthaler, and L.~Schmidt{-}Thieme.
\newblock Factorizing personalized {M}arkov chains for next-basket
  recommendation.
\newblock In \emph{International Conference on World Wide Web ({WWW})}, pages
  811--820. {ACM}, 2010.

\bibitem[Rissanen(1986)]{rissanen1986complexity}
J.~Rissanen.
\newblock Complexity of strings in the class of {M}arkov sources.
\newblock \emph{IEEE Transactions on Information Theory}, 32\penalty0
  (4):\penalty0 526--532, 1986.

\bibitem[Robbins(1952)]{robbins1952some}
H.~Robbins.
\newblock Some aspects of the sequential design of experiments.
\newblock \emph{Bulletin of the American Mathematical Society}, 58\penalty0
  (5):\penalty0 527--535, 1952.

\bibitem[Rohatgi and Foster(2025)]{rohatgi2025necessary}
D.~Rohatgi and D.~J. Foster.
\newblock Necessary and sufficient oracles: Toward a computational taxonomy for
  reinforcement learning.
\newblock In \emph{Conference on Learning Theory}, volume 291 of
  \emph{Proceedings of Machine Learning Research}, pages 4838--4936. PMLR,
  2025.

\bibitem[Ross et~al.(2011)Ross, Gordon, and Bagnell]{ross2011reduction}
S.~Ross, G.~Gordon, and J.~A. Bagnell.
\newblock A reduction of imitation learning and structured prediction to
  no-regret online learning.
\newblock In \emph{International Conference on Artificial Intelligence and
  Statistics}, 2011.

\bibitem[Russo and Van~Roy(2013)]{russo2013eluder}
D.~Russo and B.~Van~Roy.
\newblock Eluder dimension and the sample complexity of optimistic exploration.
\newblock In \emph{Advances in Neural Information Processing Systems}, pages
  2256--2264, 2013.

\bibitem[Russo and Van~Roy(2014)]{russo2014learning}
D.~Russo and B.~Van~Roy.
\newblock Learning to optimize via posterior sampling.
\newblock \emph{Mathematics of Operations Research}, 39\penalty0 (4):\penalty0
  1221--1243, 2014.

\bibitem[Russo and Van~Roy(2016)]{russo2016information}
D.~Russo and B.~Van~Roy.
\newblock An information-theoretic analysis of {T}hompson sampling.
\newblock \emph{Journal of Machine Learning Research}, 17\penalty0
  (68):\penalty0 1--30, 2016.

\bibitem[Russo and Van~Roy(2018)]{russo2018learning}
D.~Russo and B.~Van~Roy.
\newblock Learning to optimize via information-directed sampling.
\newblock \emph{Operations Research}, 66\penalty0 (1):\penalty0 230--252, 2018.

\bibitem[Schulman et~al.(2015)Schulman, Levine, Abbeel, Jordan, and
  Moritz]{schulman2015trust}
J.~Schulman, S.~Levine, P.~Abbeel, M.~Jordan, and P.~Moritz.
\newblock Trust region policy optimization.
\newblock In \emph{International Conference on Machine Learning}, 2015.

\bibitem[Schulman et~al.(2017)Schulman, Wolski, Dhariwal, Radford, and
  Klimov]{schulman2017proximal}
J.~Schulman, F.~Wolski, P.~Dhariwal, A.~Radford, and O.~Klimov.
\newblock Proximal policy optimization algorithms.
\newblock \emph{arXiv preprint arXiv:1707.06347}, 2017.

\bibitem[Shalev-Shwartz and Ben-David(2014)]{shalev2014understanding}
S.~Shalev-Shwartz and S.~Ben-David.
\newblock \emph{Understanding machine learning: From theory to algorithms}.
\newblock Cambridge university press, 2014.

\bibitem[Shtar'kov(1987)]{shtarkov1987universal}
Y.~M. Shtar'kov.
\newblock Universal sequential coding of single messages.
\newblock \emph{Problemy Peredachi Informatsii}, 23\penalty0 (3):\penalty0
  3--17, 1987.
\newblock English translation: Problems of Information Transmission,
  23(3):175--186, 1987.

\bibitem[Silver et~al.(2016)Silver, Huang, Maddison, Guez, Sifre, Van
  Den~Driessche, Schrittwieser, Antonoglou, Panneershelvam, Lanctot,
  et~al.]{silver2016mastering}
D.~Silver, A.~Huang, C.~J. Maddison, A.~Guez, L.~Sifre, G.~Van Den~Driessche,
  J.~Schrittwieser, I.~Antonoglou, V.~Panneershelvam, M.~Lanctot, et~al.
\newblock Mastering the game of {G}o with deep neural networks and tree search.
\newblock \emph{Nature}, 529\penalty0 (7587):\penalty0 484--489, 2016.

\bibitem[Simchi-Levi and Xu(2022)]{simchi2020bypassing}
D.~Simchi-Levi and Y.~Xu.
\newblock Bypassing the monster: A faster and simpler optimal algorithm for
  contextual bandits under realizability.
\newblock \emph{Mathematics of Operations Research}, 47\penalty0 (3):\penalty0
  1904--1931, 2022.

\bibitem[Sion(1958)]{sion1958minimax}
M.~Sion.
\newblock On general minimax theorems.
\newblock \emph{Pacific J. Math.}, 8:\penalty0 171--176, 1958.

\bibitem[Slivkins(2019)]{slivkins2019introduction}
A.~Slivkins.
\newblock Introduction to multi-armed bandits.
\newblock \emph{Foundations and Trends{\textregistered} in Machine Learning},
  12\penalty0 (1-2):\penalty0 1--286, 2019.

\bibitem[Stiennon et~al.(2020)Stiennon, Ouyang, Wu, Ziegler, Lowe, Voss,
  Radford, Amodei, and Christiano]{stiennon2020learning}
N.~Stiennon, L.~Ouyang, J.~Wu, D.~Ziegler, R.~Lowe, C.~Voss, A.~Radford,
  D.~Amodei, and P.~F. Christiano.
\newblock Learning to summarize with human feedback.
\newblock \emph{Advances in Neural Information Processing Systems},
  33:\penalty0 3008--3021, 2020.

\bibitem[Sun et~al.(2019)Sun, Jiang, Krishnamurthy, Agarwal, and
  Langford]{sun2019model}
W.~Sun, N.~Jiang, A.~Krishnamurthy, A.~Agarwal, and J.~Langford.
\newblock Model-based {RL} in contextual decision processes: {PAC} bounds and
  exponential improvements over model-free approaches.
\newblock In \emph{Conference on Learning Theory}, pages 2898--2933. PMLR,
  2019.

\bibitem[Sutton and Barto(2018)]{sutton2018reinforcement}
R.~S. Sutton and A.~G. Barto.
\newblock \emph{Reinforcement learning: An introduction}.
\newblock MIT press, 2018.

\bibitem[Sutton et~al.(1999)Sutton, McAllester, Singh, and
  Mansour]{sutton2000policy}
R.~S. Sutton, D.~A. McAllester, S.~P. Singh, and Y.~Mansour.
\newblock Policy gradient methods for reinforcement learning with function
  approximation.
\newblock In \emph{Advances in Neural Information Processing Systems},
  volume~12, 1999.

\bibitem[Tang et~al.(2017)Tang, Houthooft, Foote, Stooke, Chen, Duan, Schulman,
  De~Turck, and Abbeel]{tang2017exploration}
H.~Tang, R.~Houthooft, D.~Foote, A.~Stooke, X.~Chen, Y.~Duan, J.~Schulman,
  F.~De~Turck, and P.~Abbeel.
\newblock \#{E}xploration: A study of count-based exploration for deep
  reinforcement learning.
\newblock In \emph{Advances in Neural Information Processing Systems}, 2017.

\bibitem[Tewari and Murphy(2017)]{tewari2017ads}
A.~Tewari and S.~A. Murphy.
\newblock From ads to interventions: Contextual bandits in mobile health.
\newblock In J.~M. Rehg, S.~A. Murphy, and S.~Kumar, editors, \emph{Mobile
  Health: Sensors, Analytic Methods, and Applications}, pages 495--517.
  Springer, 2017.

\bibitem[Thompson(1933)]{thompson1933likelihood}
W.~R. Thompson.
\newblock On the likelihood that one unknown probability exceeds another in
  view of the evidence of two samples.
\newblock \emph{Biometrika}, 25\penalty0 (3/4):\penalty0 285--294, 1933.

\bibitem[Tsybakov(2003)]{tsybakov2003optimal}
A.~B. Tsybakov.
\newblock Optimal rates of aggregation.
\newblock In \emph{Learning Theory and Kernel Machines}, pages 303--313.
  Springer, 2003.

\bibitem[Tsybakov(2009)]{tsybakov2008introduction}
A.~B. Tsybakov.
\newblock \emph{Introduction to Nonparametric Estimation}.
\newblock Springer Series in Statistics. Springer, New York, 2009.

\bibitem[Uehara et~al.(2022)Uehara, Zhang, and Sun]{uehara2022representation}
M.~Uehara, X.~Zhang, and W.~Sun.
\newblock Representation learning for online and offline {RL} in low-rank
  {MDPs}.
\newblock In \emph{International Conference on Learning Representations}, 2022.

\bibitem[van~de Geer(2000)]{Sara00}
S.~A. van~de Geer.
\newblock \emph{Empirical Processes in {M}-{E}stimation.}
\newblock Cambridge University Press, 2000.

\bibitem[Vangara and Egg(2024)]{vangara2024adyen}
A.~Vangara and A.~Egg.
\newblock Contextual bandits in payment processing: Non-uniform exploration and
  supervised learning.
\newblock \emph{arXiv preprint arXiv:2412.00569}, 2024.

\bibitem[Vapnik(1998)]{vapnik1998statistical}
V.~N. Vapnik.
\newblock \emph{Statistical Learning Theory}.
\newblock Wiley, New York, 1998.

\bibitem[Vapnik and Chervonenkis(1971)]{vapnik1971uniform}
V.~N. Vapnik and A.~Y. Chervonenkis.
\newblock On the uniform convergence of relative frequencies of events to their
  probabilities.
\newblock \emph{Theory of Probability and Its Applications}, 16\penalty0
  (2):\penalty0 264--280, 1971.

\bibitem[Vovk(1990)]{vovk1990aggregating}
V.~Vovk.
\newblock Aggregating strategies.
\newblock In \emph{Proceedings of the Third Annual Workshop on Computational
  Learning Theory (COLT)}, pages 371--383. Morgan Kaufmann, 1990.

\bibitem[Vovk(1995)]{vovk1995game}
V.~Vovk.
\newblock A game of prediction with expert advice.
\newblock In \emph{Proceedings of the eighth annual conference on computational
  learning theory}, pages 51--60. ACM, 1995.

\bibitem[Wang et~al.(2021{\natexlab{a}})Wang, Foster, and
  Kakade]{wang2020statistical}
R.~Wang, D.~P. Foster, and S.~M. Kakade.
\newblock What are the statistical limits of offline {RL} with linear function
  approximation?
\newblock In \emph{International Conference on Learning Representations
  (ICLR)}, 2021{\natexlab{a}}.

\bibitem[Wang et~al.(2021{\natexlab{b}})Wang, Wang, and
  Kakade]{wang2021exponential}
Y.~Wang, R.~Wang, and S.~M. Kakade.
\newblock An exponential lower bound for linearly-realizable {MDPs} with
  constant suboptimality gap.
\newblock In \emph{Advances in Neural Information Processing Systems},
  volume~34, 2021{\natexlab{b}}.

\bibitem[Weisz et~al.(2021)Weisz, Amortila, and
  Szepesv{\'a}ri]{weisz2021exponential}
G.~Weisz, P.~Amortila, and C.~Szepesv{\'a}ri.
\newblock Exponential lower bounds for planning in {MDPs} with
  linearly-realizable optimal action-value functions.
\newblock In \emph{Algorithmic Learning Theory}, pages 1237--1264. PMLR, 2021.

\bibitem[Williams(1992)]{williams1992simple}
R.~J. Williams.
\newblock Simple statistical gradient-following algorithms for connectionist
  reinforcement learning.
\newblock \emph{Machine Learning}, 8:\penalty0 229--256, 1992.

\bibitem[Wong and Shen(1995)]{wong1995probability}
W.~H. Wong and X.~Shen.
\newblock Probability inequalities for likelihood ratios and convergence rates
  of sieve {MLEs}.
\newblock \emph{The Annals of Statistics}, 23\penalty0 (2):\penalty0 339--362,
  1995.

\bibitem[Xie and Jiang(2020)]{xie2020q}
T.~Xie and N.~Jiang.
\newblock Q* approximation schemes for batch reinforcement learning: A
  theoretical comparison.
\newblock In \emph{Conference on Uncertainty in Artificial Intelligence}, pages
  550--559. PMLR, 2020.

\bibitem[Xie et~al.(2021)Xie, Cheng, Jiang, Mineiro, and
  Agarwal]{xie2021bellman}
T.~Xie, C.-A. Cheng, N.~Jiang, P.~Mineiro, and A.~Agarwal.
\newblock Bellman-consistent pessimism for offline reinforcement learning.
\newblock In \emph{Advances in Neural Information Processing Systems}, 2021.

\bibitem[Xie et~al.(2023)Xie, Foster, Bai, Jiang, and Kakade]{xie2023role}
T.~Xie, D.~J. Foster, Y.~Bai, N.~Jiang, and S.~M. Kakade.
\newblock The role of coverage in online reinforcement learning.
\newblock In \emph{ICLR}, 2023.

\bibitem[Xie et~al.(2025)Xie, Foster, Krishnamurthy, Rosset, Awadallah, and
  Rakhlin]{xie2024exploratory}
T.~Xie, D.~J. Foster, A.~Krishnamurthy, C.~Rosset, A.~Awadallah, and
  A.~Rakhlin.
\newblock Exploratory preference optimization: Harnessing implicit
  {Q}*-approximation for sample-efficient {RLHF}.
\newblock In \emph{International Conference on Learning Representations
  (ICLR)}, 2025.

\bibitem[Xu and Zeevi(2023)]{xu2023bayesian}
Y.~Xu and A.~Zeevi.
\newblock Bayesian design principles for frequentist sequential learning.
\newblock In \emph{International Conference on Machine Learning}, 2023.

\bibitem[Yang and Wang(2019)]{yang2019sample}
L.~Yang and M.~Wang.
\newblock Sample-optimal parametric {Q}-learning using linearly additive
  features.
\newblock In \emph{International Conference on Machine Learning}, pages
  6995--7004. PMLR, 2019.

\bibitem[Yang and Barron(1999)]{yang1999information}
Y.~Yang and A.~Barron.
\newblock Information-theoretic determination of minimax rates of convergence.
\newblock \emph{Annals of Statistics}, 27\penalty0 (5):\penalty0 1564--1599,
  1999.

\bibitem[Yao et~al.(2014)Yao, Szepesv{\'{a}}ri, Pires, and Zhang]{YaoSPZ14}
H.~Yao, C.~Szepesv{\'{a}}ri, B.~{\'{A}}. Pires, and X.~Zhang.
\newblock Pseudo-{MDPs} and factored linear action models.
\newblock In \emph{{IEEE} Symposium on Adaptive Dynamic Programming and
  Reinforcement Learning ({ADPRL})}, pages 1--9. {IEEE}, 2014.

\bibitem[Zanette and Brunskill(2019)]{zanette2019tighter}
A.~Zanette and E.~Brunskill.
\newblock Tighter problem-dependent regret bounds in reinforcement learning
  without domain knowledge using value function bounds.
\newblock In \emph{International Conference on Machine Learning}, pages
  7304--7312. PMLR, 2019.

\bibitem[Zanette et~al.(2020)Zanette, Lazaric, Kochenderfer, and
  Brunskill]{zanette2020learning}
A.~Zanette, A.~Lazaric, M.~Kochenderfer, and E.~Brunskill.
\newblock Learning near optimal policies with low inherent {B}ellman error.
\newblock In \emph{International Conference on Machine Learning}, 2020.

\bibitem[Zhan et~al.(2022)Zhan, Huang, Huang, Jiang, and Lee]{zhan2022offline}
W.~Zhan, B.~Huang, A.~Huang, N.~Jiang, and J.~D. Lee.
\newblock Offline reinforcement learning with realizability and single-policy
  concentrability.
\newblock In \emph{Conference on Learning Theory}, pages 2730--2775, 2022.

\bibitem[Zhan et~al.(2023)Zhan, Uehara, Sun, and Lee]{zhan2023pac}
W.~Zhan, M.~Uehara, W.~Sun, and J.~D. Lee.
\newblock {PAC} reinforcement learning for predictive state representations.
\newblock In \emph{International Conference on Learning Representations}, 2023.

\bibitem[Zhang(2006)]{zhang2006from}
T.~Zhang.
\newblock From $\epsilon$-entropy to {KL}-entropy: Analysis of minimum
  information complexity density estimation.
\newblock \emph{The Annals of Statistics}, 34\penalty0 (5):\penalty0
  2180--2210, 2006.

\bibitem[Zhang(2022)]{zhang2022feel}
T.~Zhang.
\newblock Feel-good {T}hompson sampling for contextual bandits and
  reinforcement learning.
\newblock \emph{SIAM Journal on Mathematics of Data Science}, 4\penalty0
  (2):\penalty0 834--857, 2022.

\bibitem[Zhang et~al.(2022)Zhang, Song, Uehara, Wang, Agarwal, and
  Sun]{zhang2022efficient}
X.~Zhang, Y.~Song, M.~Uehara, M.~Wang, A.~Agarwal, and W.~Sun.
\newblock Efficient reinforcement learning in block {MDPs}: A model-free
  representation learning approach.
\newblock In \emph{International Conference on Machine Learning}, 2022.

\bibitem[Zhang et~al.(2021)Zhang, Ji, and Du]{zhang2021reinforcement}
Z.~Zhang, X.~Ji, and S.~Du.
\newblock Is reinforcement learning more difficult than bandits? a near-optimal
  algorithm escaping the curse of horizon.
\newblock In \emph{Conference on Learning Theory}, pages 4528--4531. PMLR,
  2021.

\bibitem[Zhong et~al.(2026)Zhong, Xiong, Zheng, Wang, Wang, Yang, and
  Zhang]{zhong2022posterior}
H.~Zhong, W.~Xiong, S.~Zheng, L.~Wang, Z.~Wang, Z.~Yang, and T.~Zhang.
\newblock Generalized eluder coefficient: A unified framework for interactive
  decision making in {M}arkov decision processes, partially observable {M}arkov
  decision processes, and beyond.
\newblock \emph{Mathematics of Operations Research}, 2026.
\newblock Articles in advance.

\bibitem[Zhu et~al.(2024)Zhu, de~Salvo~Braz, Bhandari, Jiang, Wan, Efroni,
  Wang, Xu, Guo, Nikulkov, Korenkevych, Dogan, Cheng, Wu, and Xu]{zhu2024pearl}
Z.~Zhu, R.~de~Salvo~Braz, J.~Bhandari, D.~Jiang, Y.~Wan, Y.~Efroni, L.~Wang,
  R.~Xu, H.~Guo, A.~Nikulkov, D.~Korenkevych, U.~Dogan, F.~Cheng, Z.~Wu, and
  W.~Xu.
\newblock Pearl: A production-ready reinforcement learning agent.
\newblock \emph{Journal of Machine Learning Research}, 25\penalty0
  (273):\penalty0 1--30, 2024.

\end{thebibliography}

\newpage

\appendix

\section{Technical Tools}
\label{sec:technical}

\subsection{Probabilistic Inequalities}
\label{app:probabilistic}

\subsubsection{Proof of Bernstein's Inequality}
\label{app:bernstein_proof}
\begin{proof}[\pfref{lem:bernstein}]
  Set $X_i=Z_i-\En Z_i$. If $v=0$, then $X_i=0$ almost surely and
  the result is immediate. Otherwise $B>0$. For $0\leq\lambda<3/B$,
  boundedness permits expansion of the exponential in expectation.
  Since $\En X_i=0$, $\En X_i^2=v^2$, and
  $|\En X_i^k|\leq B^{k-2}v^2$ for $k\geq2$,
  \begin{align*}
    \En e^{\lambda X_i}
    &\leq 1+v^2\sum_{k=2}^{\infty}\frac{\lambda^kB^{k-2}}{k!}
    \leq 1+\frac{\lambda^2v^2}{2(1-\lambda B/3)}
    \leq\exp\left\{\frac{\lambda^2v^2}{2(1-\lambda B/3)}\right\}.
  \end{align*}
  Here $k!\geq2\cdot3^{k-2}$ for $k\geq2$. Independence and Markov's
  inequality imply, for any $s>0$ and $0<\lambda<3/B$,
  \[
    \Prob{\sum_{i=1}^T X_i\geq s}
    \leq\exp\left\{-\lambda s+
      \frac{T\lambda^2v^2}{2(1-\lambda B/3)}\right\}.
  \]
  Put $u=\log(1/\delta)$, $a=\sqrt{2u/(Tv^2)}$, and choose
  $\lambda=a/(1+Ba/3)$. For
  $s=\sqrt{2Tv^2u}+Bu/3$, substitution makes the exponent exactly $-u$.
  Dividing the resulting bound by $T$ proves the claim.
\end{proof}

\subsubsection{Tail Bounds with Stopping Times}

\begin{lem}[Hoeffding's inequality with adaptive stopping time]
  \label{lem:hoeffding_adaptive}
  For \iid random variables $Z_1,\ldots,Z_T$ taking values in $[a,b]$ almost surely, with probability at least $1-\delta$,
  \begin{align}
    \label{eq:hoeffding_adaptive1}
    \frac{1}{T'}\sum_{i=1}^{T'} Z_i - \En\brk*{Z}\leq (b-a)\sqrt{\frac{\log(T/\delta)}{2T'}}\qquad\forall{}1\leq{}T'\leq{}T.
  \end{align}
  As a consequence, for any random time $\tau\in\brk{T}$,
  with probability at least $1-\delta$,
    \begin{align}
    \label{eq:hoeffding_adaptive2}
      \frac{1}{\tau}\sum_{i=1}^{\tau} Z_i - \En\brk*{Z}\leq (b-a)\sqrt{\frac{\log(T/\delta)}{2\tau}}.
  \end{align}
  The choice of $\tau$ may depend arbitrarily on the data.
\end{lem}
\begin{proof}[\pfref{lem:hoeffding_adaptive}]
  \pref{lem:hoeffding} states that for any fixed $T'\in\brk{T}$, with probability at least $1-\delta$,
  \[
    \frac{1}{T'}\sum_{i=1}^{T'} Z_i - \En\brk*{Z}\leq (b-a)\sqrt{\frac{\log(1/\delta)}{2T'}}.
  \]
  \pref{eq:hoeffding_adaptive1} follows by applying this result with $\delta'=\delta/T$ and taking a union bound over all $T$ choices for $T'\in\brk{T}$.
  On this event, the bound holds at every $T'\in\brk{T}$,
  so substituting $T'=\tau$ proves \eqref{eq:hoeffding_adaptive2}.
\end{proof}

\subsubsection{Tail Bounds for Martingales}

  \begin{lem}[Martingale Chernoff bound]
    \label{lem:martingale_chernoff}
    For any sequence of real-valued random variables $\prn{X_t}_{t\leq{}T}$ adapted to a
    filtration $\prn{\filt_t}_{t\leq{}T}$, it holds that with probability at least
    $1-\delta$, for all $T'\leq{}T$,
    \begin{equation}
      \label{eq:martingale_chernoff}
      \sum_{t=1}^{T'}X_t \leq
      \sum_{t=1}^{T'}\log\prn*{\En_{t-1}\brk*{e^{X_t}}} + \log(\delta^{-1}),
    \end{equation}
    where $\En_{t-1}\brk*{\cdot}=\En\brk*{\cdot\mid\filt_{t-1}}$ denotes conditional expectation.
  \end{lem}
\begin{proof}[\pfref{lem:martingale_chernoff}]
  Write $\psi_t=\log\prn*{\En_{t-1}\brk{e^{X_t}}}$ and define
  \[
    Z_0=1,\qquad
    Z_\tau=\exp\prn*{\sum_{t=1}^{\tau}(X_t-\psi_t)}
      =\prod_{t=1}^{\tau}e^{X_t-\psi_t}.
  \]
  If $\psi_t=+\infty$, interpret the corresponding factor $Z_t$ as zero.
  In this case, the right-hand side of \eqref{eq:martingale_chernoff}
  is infinite for all $T'\geq t$, so the desired bound is automatic.

  To proceed, we observe that the process $(Z_\tau)_{\tau=0}^T$ is a nonnegative supermartingale.
  Indeed, for each $1\leq\tau\leq T$, on the event $\psi_\tau<\infty$,
  \begin{align*}
    \En_{\tau-1}\brk{Z_\tau}
    &=Z_{\tau-1}e^{-\psi_\tau}\En_{\tau-1}\brk{e^{X_\tau}}
      =Z_{\tau-1},
  \end{align*}
  while on the event $\psi_\tau=+\infty$ the conditional expectation
  is zero. Hence $\En_{\tau-1}\brk{Z_\tau}\leq Z_{\tau-1}$ in either
  case. Since $Z_0=1$, Ville's inequality
  (e.g., \citet{howard2020time}) gives
  \[
    \bbP\prn*{\max_{1\leq\tau\leq T}Z_\tau\geq\delta^{-1}}
      \leq\delta.
  \]
  With probability at least $1-\delta$, therefore,
  $Z_{T'}<\delta^{-1}$ for all $1\leq T'\leq T$.
  Taking logarithms whenever $\psi_1,\ldots,\psi_{T'}$ are finite
  proves \eqref{eq:martingale_chernoff}. \qedhere
\end{proof}

The next result is a martingale counterpart to Bernstein's
inequality (\cref{lem:bernstein}).
\begin{lem}[Freedman's inequality (Bernstein for martingales)]
  \label{lem:freedman}
  Let $(X_t)_{t\leq{T}}$ be a real-valued martingale difference
  sequence adapted to a filtration $\prn{\filt_t}_{t\leq{}T}$. If
  $\abs*{X_t}\leq{}R$ almost surely, then for any $\eta\in(0,1/R)$, with probability at least $1-\delta$, for all $T'\leq{}T$,
    \[
      \sum_{t=1}^{T'} X_t \leq{} \eta\sum_{t=1}^{T'} \En_{t-1}\brk*{X_t^{2}} + \frac{\log(\delta^{-1})}{\eta}.
    \]
  \end{lem}
  \begin{proof}[\pfref{lem:freedman}]
  Without loss of generality, let $R=1$, and fix $\eta\in(0,1)$. The result follows by invoking \pref{lem:martingale_chernoff} with $\eta X_t$ in place of $X_t$, and by the facts that $e^a \leq 1+a+(e-2)a^2$ for $a\leq 1$ and $1+b\leq e^b$ for all $b\in\reals$.
  \end{proof}

  The following result follows by applying \pref{lem:freedman} to both tails.
  
\begin{lem}[Multiplicative Freedman inequality]
      \label{lem:multiplicative_freedman}
            Let $(X_t)_{t\leq{T}}$ be a sequence of random
      variables adapted to a filtration $\prn{\filt_{t}}_{t\leq{}T}$. If
  $0\leq{}X_t\leq{}R$ almost surely, then with probability at least
  $1-\delta$,
  \begin{align*}
    &\sum_{t=1}^{T}X_t \leq{}
                        \frac{3}{2}\sum_{t=1}^{T}\En_{t-1}\brk*{X_t} +
                        2R\log(2\delta^{-1}),
    \intertext{and}
      &\sum_{t=1}^{T}\En_{t-1}\brk*{X_t} \leq{} 2\sum_{t=1}^{T}X_t + 4R\log(2\delta^{-1}).
  \end{align*}
\end{lem}
\begin{proof}[\pfref{lem:multiplicative_freedman}]
  Suppose $R>0$, and write
  $\mu_t=\En_{t-1}\brk{X_t}$ and $Y_t=X_t-\mu_t$.
  Then $\En_{t-1}\brk{Y_t}=0$, $|Y_t|\leq R$, and
  \[
    \En_{t-1}\brk{Y_t^2}
      =\En_{t-1}\brk{X_t^2}-\mu_t^2
      \leq R\mu_t.
  \]
  Apply \pref{lem:freedman} to $(Y_t)$ and $(-Y_t)$ with
  $\eta=1/(2R)\in(0,1/R)$ and failure probability $\delta/2$ for
  each application. Set $S=\sum_{t=1}^T X_t$,
  $M=\sum_{t=1}^T\mu_t$, and $L=\log(2/\delta)$.
  A union bound gives, with probability at least $1-\delta$,
  \[
    S-M\leq\frac12 M+2RL,
    \qquad
    M-S\leq\frac12 M+2RL.
  \]
  Rearranging yields
  \[
    S\leq\frac32 M+2RL,
    \qquad M\leq2S+4RL.
  \]
\end{proof}

\subsection{Information Theory}

\subsubsection{Properties of Hellinger Distance}

\begin{lem}[Hellinger distance for conditional distributions]
  \label{lem:hellinger_pair}
    For any distributions $\bbP$ and $\bbQ$ over a pair of random variables $(X,Y)$,
    \[
      \En_{X\sim{}\bbP_X}\brk*{\Dhels{\bbP_{Y\mid{}X}}{\bbQ_{Y\mid{}X}}} \leq{} 
      4\Dhels{\bbP_{X,Y}}{\bbQ_{X,Y}}.      
    \]
  \end{lem}
  \begin{proof}[\pfref{lem:hellinger_pair}]
    Since squared Hellinger distance is an $f$-divergence, we have
    \begin{align*}
      \En_{X\sim{}\bbP_X}\brk*{\Dhels{\bbP_{Y\mid{}X}}{\bbQ_{Y\mid{}X}}}
      & = \Dhels{\bbP_{Y\mid{}X}\tens{}\bbP_X}{\bbQ_{Y\mid{}X}\tens{}\bbP_X}.
        \end{align*}
        Next, using that Hellinger distance satisfies the
        triangle inequality, along with the elementary inequality
        $(a+b)^2\leq{}2(a^2+b^2)$, we have,
        \begin{align*}
\En_{X\sim{}\bbP_X}\brk*{\Dhels{\bbP_{Y\mid{}X}}{\bbQ_{Y\mid{}X}}}      & \leq{} 2\Dhels{\bbP_{Y\mid{}X}\tens{}\bbP_X}{\bbQ_{Y\mid{}X}\tens{}\bbQ_X}
        + 2\Dhels{\bbQ_{Y\mid{}X}\tens{}\bbP_X}{\bbQ_{Y\mid{}X}\tens{}\bbQ_X}\\
      & = 2\Dhels{\bbP_{X,Y}}{\bbQ_{X,Y}}
        + 2\Dhels{\bbP_X}{\bbQ_X}\\
      & \leq{} 4\Dhels{\bbP_{X,Y}}{\bbQ_{X,Y}},
    \end{align*}
    where the final line follows from the data processing inequality.
  \end{proof}

\subsubsection{Change-of-Measure Inequalities}
\label{sec:change_of_measure}

The first inequality below is a sub-Gaussian variant of Pinsker's
inequality; see \citet{russo2016information}.
\begin{lem}[Pinsker for \subg random variables]
  \label{lem:pinsker_subgaussian}
  Suppose that $X\sim\bbP$ and $Y\sim\bbQ$ are both $\sigma^2$-\subgaussian. Then
  \[
    \abs{\En_{\bbP}\brk*{X}-\En_{\bbQ}\brk*{Y}}
    \leq{}\sqrt{2\sigma^2\cdot\Dkl{\bbP}{\bbQ}}.
  \]
\end{lem}

    The next inequality is from \citet{foster2021statistical}.
    \begin{lem}[Multiplicative Pinsker-type inequality for Hellinger distance]
  \label{lem:mp_min}
  Let $\bbP$ and $\bbQ$ be probability measures on $(\cX,\filt)$. For all
  $h:\cX\to\bbR$ with $0\leq{}h(X)\leq{}R$ almost surely under $\bbP$
  and $\bbQ$, we have
  \begin{align}
    \label{eq:mp_min_sqrt}
    \abs*{\En_{\bbP}\brk*{h(X)} - \En_{\bbQ}\brk*{h(X)}} \leq \sqrt{2R(\En_{\bbP}\brk*{h(X)} + \En_{\bbQ}\brk*{h(X)})\cdot\Dhels{\bbP}{\bbQ}}.
  \end{align}
In particular, 
  \begin{align}
    \label{eq:mp_min}
\En_{\bbP}\brk*{h(X)}
    &\leq{} 3\En_{\bbQ}\brk*{h(X)} + 4R\cdot\Dhels{\bbP}{\bbQ}.
  \end{align}
\end{lem}
\begin{proof}[\pfref{lem:mp_min}]
Let a measurable event $A$ be fixed. Let $p = \bbP(A)$ and
$q=\bbQ\prn{A}$. Then we have
\[
\frac{(p-q)^{2}}{2(p+q)} \leq{}(\sqrt{p}-\sqrt{q})^{2}\leq{}\Dhels{(p,1-p)}{(q,1-q)}\leq{}\Dhels{\bbP}{\bbQ},
\]
where the third inequality is the data-processing inequality. It follows that
\[
\abs*{p-q}\leq{}\sqrt{2(p+q) \Dhels{\bbP}{\bbQ}},
\]
To deduce the final result for $R=1$, we observe that
$\En_{\bbP}\brk*{h(X)}=\int_{0}^{1}\bbP\prn{h(X)>t}dt$ and likewise
for $\En_{\bbQ}\brk*{h(X)}$, then apply Jensen's inequality. The result for
general $R$ follows by rescaling.

The inequality in \pref{eq:mp_min} follows by applying the AM-GM
inequality to \pref{eq:mp_min_sqrt} and rearranging.

\end{proof}

\subsection{Minimax Theorem}
\label{sec:minimax_appendix}
    We recall Sion's minimax theorem \citep{sion1958minimax}.
  \begin{lem}[Sion's Minimax Theorem]
          \label{lem:sion}
          Let $\cX$ and $\cY$ be convex sets in linear topological
          spaces, and assume $\cX$ is compact. Let
          $f:\cX\times\cY\to\bbR$ be such that (i) $f(x, \cdot)$ is
          concave and upper semicontinuous over $\cY$ for all
          $x\in\cX$ and (ii) $f(\cdot,y)$ is convex and lower
          semicontinuous over $\cX$ for all $y\in\cY$. Then
          \begin{equation}
            \label{eq:minimax_theorem}
            \inf_{x\in\cX}\sup_{y\in\cY}f(x,y) = \sup_{y\in\cY}\inf_{x\in\cX}f(x,y).
          \end{equation}
        \end{lem}  

 \end{document}